%% file: SpectralMethods-nowplain.tex
\PassOptionsToClass {plain,biber}{nowfnt}
\documentclass[MAL,biber]{nowfnt} % creates the journal version, needs biber version
\usepackage{amsmath,amsbsy,amstext,amssymb,amsthm,mathtools,comment,dsfont}
\usepackage{color}
\usepackage{bbm}
\usepackage{algorithm,bm}

\usepackage{appendix}

\input{chapters/commands.tex}

\newcommand{\mycommentbegin}{\begin{comment}} 
\newcommand{\mycommentend}{\end{comment}}

\newcommand{\bz}{{\boldsymbol z}}

\newcommand{\bu}{{\boldsymbol u}}

\newcommand{\bM}{{\boldsymbol M}}

\newcommand{\ba}{{\boldsymbol a}}

\newcommand{\bx}{{\boldsymbol x}}
\newcommand{\bU}{{\boldsymbol U}}
\newcommand{\by}{{\boldsymbol y}}

\newcommand{\bR}{{\boldsymbol R}}

\newcommand{\bB}{{\boldsymbol B}}

\newcommand{\bY}{{\boldsymbol Y}}
\newcommand{\bP}{{\boldsymbol P}}

\newcommand{\bI}{{\boldsymbol I}}

\newcommand{\vertiii}[1]{{\left\vert\kern-0.25ex\left\vert\kern-0.25ex\left\vert #1 
    \right\vert\kern-0.25ex\right\vert\kern-0.25ex\right\vert}}

\newcommand{\vertiiibig}[1]{{\big\vert\kern-0.25ex\big\vert\kern-0.25ex\big\vert #1 
    \big\vert\kern-0.25ex\big\vert\kern-0.25ex\big\vert}}

\newcommand{\vertiiiplain}[1]{{\vert\kern-0.25ex\vert\kern-0.25ex\vert #1 
    \vert\kern-0.25ex\vert\kern-0.25ex\vert}}

\definecolor{yxc}{RGB}{255,0,0}
\definecolor{yjc}{RGB}{125,0,0}
\definecolor{cm}{RGB}{0,0,200}
\definecolor{yly}{RGB}{0,150,0}
\definecolor{dacong}{RGB}{88,178,220}

\newtheorem{assumption}{\textbf{Assumption}}[chapter]
\newtheorem{claim}{\textbf{Claim}}[chapter]

\newcommand{\mynewline}{}

\title{Spectral Methods for Data Science: A Statistical Perspective}
\maintitleauthorlist{
Yuxin Chen \\
Princeton University \\
yuxin.chen@princeton.edu \\
\and
Yuejie Chi \\
Carnegie Mellon University \\
yuejiechi@cmu.edu \\
\and
Jianqing Fan\\
Princeton University \\
jqfan@princeton.edu \\
\and
Cong Ma\\
University of Chicago \\
congm@uchicago.edu
}

\issuesetup
{%
 copyrightowner={A.~Heezemans and M.~Casey},
 pubyear       = 2020,
 }

\usepackage{mwe}

\author[1]{Chen,Yuxin}
\author[2]{Chi,Yuejie}
\author[3]{Fan,Jianqing}
\author[4]{Ma,Cong}

\affil[1]{Princeton University; yuxin.chen@princeton.edu}
\affil[2]{Carnegie Mellon University; yuejiechi@cmu.edu}
\affil[3]{Princeton University; jqfan@princeton.edu}
\affil[4]{University of Chicago; congm@uchicago.edu}

\articledatabox{\nowfntstandardcitation}

\begin{document}

\makeabstracttitle
\renewcommand{\mynewline}{}

\begin{abstract}

\input{chapters/abstract.tex}

\end{abstract}

\input{chapters/intro.tex}

\newpage

\input{chapters/L2.tex}

\newpage

\input{chapters/matrix_recovery.tex}

\newpage

\input{chapters/Linf.tex}

%\newpage
%\input{chapters/generalization.tex}

\newpage
\input{chapters/conclusion.tex}

%\newpage
%\input{chapters/tensors.tex}

%\newpage
%\input{chapters/algorithms.tex}

\newpage

 % end of main matter

\input{chapters/ack.tex}

%\appendix 

%\end{document}

%BACKMATTER SEE DOCUMENTATION
\backmatter  % references, restarts sample

\printbibliography

\end{document}

%% file: chapters/commands.tex
\providecommand{\cref}[1]{Chapter~\ref{chap:#1}}

\providecommand{\norm}[1]{\lVert#1\rVert}

\providecommand{\set}[1]{\left\{#1\right\}}

\providecommand{\va}{\bm{a}}

\usepackage{dsfont}

%% file: chapters/abstract.tex
Spectral methods have emerged as a simple yet surprisingly effective approach for extracting information from massive, noisy and incomplete data. In a nutshell, spectral methods refer to a collection of algorithms built upon the eigenvalues (resp.~singular values) and eigenvectors (resp.~singular vectors) of some properly designed  matrices constructed from data. A diverse array of applications have been found in machine learning, imaging science, financial and econometric modeling, and signal processing, including recommendation systems, community detection, ranking, structured matrix recovery, tensor data estimation, joint shape matching, blind deconvolution, financial investments, risk managements,  treatment evaluations, causal inference, amongst others. Due to their simplicity and effectiveness, spectral methods are not only used  as a stand-alone estimator, but also frequently employed to facilitate other more sophisticated algorithms to enhance performance.

While the studies of spectral methods can be traced back to classical matrix perturbation theory and the method of moments, the past decade has witnessed tremendous theoretical advances in demystifying  their efficacy through the lens of statistical modeling, with the aid of concentration inequalities and non-asymptotic random matrix theory.
This monograph aims to present a systematic, comprehensive, yet accessible introduction to spectral methods from a modern statistical perspective,
highlighting their algorithmic implications in diverse large-scale applications.
In particular, our exposition gravitates around several central questions that span various applications: how to characterize the sample efficiency of spectral methods in reaching a target level of statistical accuracy, and how to assess their stability in the face of random noise, missing data, and adversarial corruptions?
In addition to conventional $\ell_2$ perturbation analysis, we present a systematic $\ell_{\infty}$ and $\ell_{2,\infty}$ perturbation theory for eigenspace and singular subspaces, which has only recently become available owing to a powerful ``leave-one-out'' analysis framework.

%% file: chapters/intro.tex
\chapter{Introduction}
\label{cha:introduction}

In contemporary science and engineering applications, the volume of available data  is growing at an enormous rate. The emergence of this trend is  due to recent technological advances that have enabled the collection, transmission, storage and processing of data from every corner of our life, in the forms of images, videos, network traffic, email logs, electronic health records, genomic and genetic measurements, high-frequency financial trades, grocery transactions, online exchanges, and so on.
In the meantime, modern applications often require reasonings about an unprecedented scale of features or parameters of interest.
This gives rise to the pressing demand of developing {\em low-complexity} algorithms that can effectively distill actionable insights from large-scale and high-dimensional data. In addition to the curse of dimensionality, the challenge is further compounded when the data in hand are  noisy, messy, and contain missing features.

Towards addressing the above challenges, \emph{spectral methods} have emerged as a simple yet surprisingly effective approach to information extraction from massive and noisy data. In a nutshell, spectral methods refer to a collection of algorithms built upon the eigenvectors (resp.~singular vectors) and eigenvalues (resp.~singular values) of some properly designed  matrices generated from data. Remarkably, spectral methods lend themselves to a diverse array of applications in practice, including community detection in networks \citep{newman2006finding,abbe2017community,rohe2011spectral,mcsherry2001spectral}, angular synchronization in cryo-EM \citep{singer2011three,singer2011angular}, joint image alignment \citep{chen2016projected}, clustering \citep{von2007tutorial,ng2002spectral}, ranking \citep{negahban2016rank,chen2015spectral,chen2017spectral}, dimensionality reduction \citep{belkin2003laplacian}, low-rank matrix estimation \citep{achlioptas2007fast, keshavan2010matrix}, tensor estimation \citep{montanari2016spectral,cai2019tensor},  covariance and precision matrix estimation \citep{fan2013large,fan2020robust}, shape reconstruction \citep{li2004fast}, econometric and financial modeling \citep{fan2021recent}, among others.
Motivated by their applicability to numerous real-world problems, this monograph seeks to offer a unified and comprehensive treatment towards establishing the theoretical underpinnings for spectral methods, particularly through a statistical lens.

\section{Motivating applications}
\label{sec:motivating_examples}

At the heart of spectral methods is the idea that the eigenvectors or singular vectors of certain data matrices reveal crucial information pertaining to the targets of interest.
We single out a few examples that epitomize this idea.

\paragraph{Clustering.}

\begin{figure}
	\begin{center}
		\begin{tabular}{ccc}
			\includegraphics[width=0.25\textwidth]{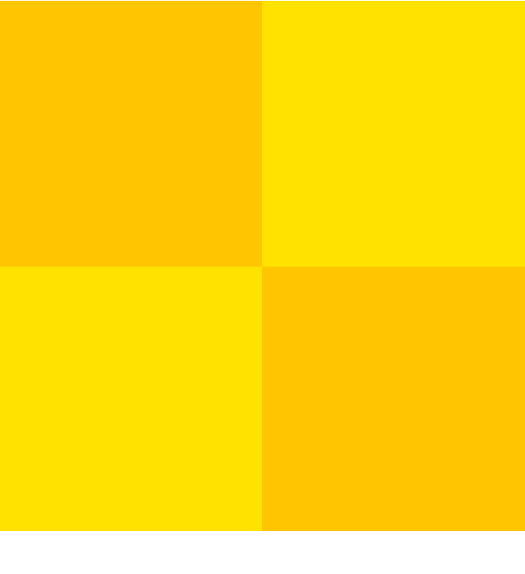} & \includegraphics[width=0.25\textwidth]{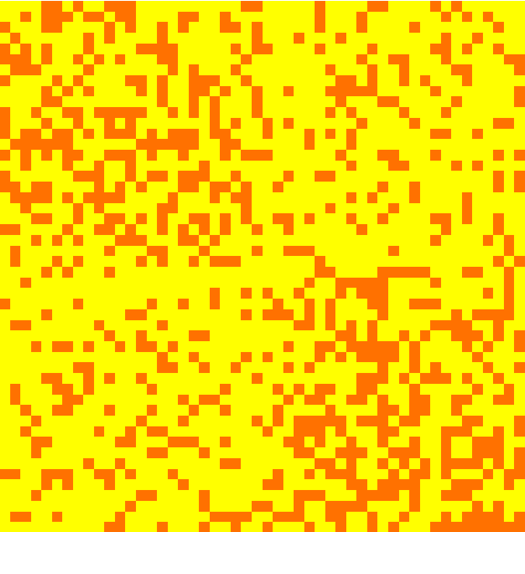} &   \includegraphics[width=0.4\textwidth]{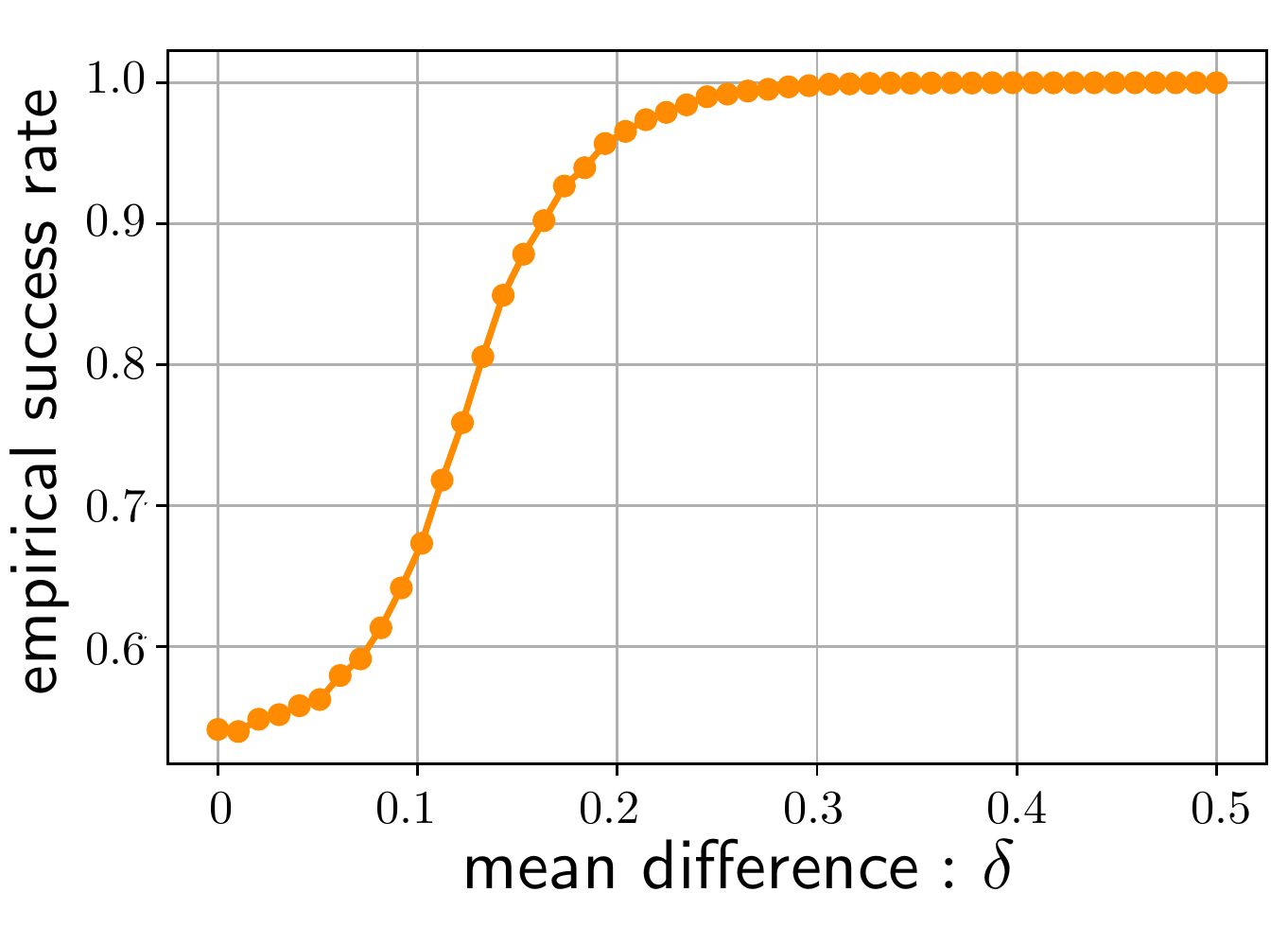} \tabularnewline
			(a) & (b) & (c) \tabularnewline
		\end{tabular}
	\end{center}
	\caption{Spectral methods for clustering. We plot in (a)  an ideal structure of the adjacency matrix $\bm{A}$ in \eqref{eq:defn-A-motivation-CD}, and in (b) a noisy version which is a realization from the stochastic block model, where $A_{i,j}$ is an independent Bernoulli variable with mean $\frac{1+\delta}{2}$ (resp.~$\frac{1-\delta}{2}$) if $i$ and $j$ belong to the same group (resp.~different groups).  We report in (c) the empirical success rate of the spectral method over 200 Monte Carlo trials in correctly clustering $n=100$ individuals as the mean difference $\delta$ varies.
		\label{fig:clustering-motivation}}
	
\end{figure}
Clustering corresponds to the grouping of individuals based on their mutual similarities,
which constitutes a fundamental task in unsupervised learning and spans numerous applications such as image segmentation (e.g., grouping pixels based on the objects they represent in an image) \citep{browet2011community}
and community detection (e.g., grouping users on the basis of their social circles) \citep{fortunato2016community}.
For concreteness, let us take a look at a simple scenario with $n$ individuals such that: (1) there exists a latent partitioning that divides all individuals into two groups, with the first $n/2$ individuals belonging to the first group and the rest  belonging to the second group (without loss of generality);
and (2) we observe pairwise similarity measurements generated based on their group memberships.
Ideally, if we know whether any two individuals belong to the same group or not, then we can form an adjacency matrix $\bm{A}=[A_{i,j}]_{1\leq i,j\leq n}$ such that
\begin{equation}
	A_{i,j}=\begin{cases}
1,\qquad & \text{if }(i,j)\text{ belongs to the same group},\\
0, & \text{else}.
\end{cases}
	\label{eq:defn-A-motivation-CD}
\end{equation}
As a key observation, this matrix $\bm{A}$, as illustrated in Figure \ref{fig:clustering-motivation}(a), turns out to be a rank-2 matrix
\[
	\bm{A} =
	\left[\begin{array}{cc}
		\bm{1}_{n/2}\bm{1}^{\top}_{n/2}\\
 & \bm{1}_{n/2}\bm{1}^{\top}_{n/2}
\end{array}\right]=\frac{1}{2}\bm{1}_{n}\bm{1}^{\top}_{n} + \frac{1}{2}\left[\begin{array}{c}
\bm{1}_{n/2}\\
-\bm{1}_{n/2}
\end{array}\right]\left[\begin{array}{cc}
\bm{1}^{\top}_{n/2} & -\bm{1}^{\top}_{n/2}\end{array}\right],
\]
where $\bm{1}_n$ represents an $n$-dimensional all-one vector. 
After subtracting $\frac{1}{2}\bm{1}_{n}\bm{1}^{\top}_{n}$ from $\bm{A}$, the eigenvector $\bm{u}_2\coloneqq [\begin{array}{cc}
\bm{1}^{\top}_{n/2} & -\bm{1}^{\top}_{n/2}\end{array}]$ of the remaining component uncovers the underlying group structure; namely, all positive entries of $\bm{u}_2$ represent one group, with all negative entries of $\bm{u}_2$  reflecting another group.
In reality, however, we typically only get to collect imprecise information about whether two individuals belong to the same group,
thus resulting in a corrupted version of $\bm{A}$ (see Figure~\ref{fig:clustering-motivation}(b)). Fortunately, the eigenvector (the one corresponding to $\bm{u}_2$ above) of the observed data matrix (with proper arrangement) might continue to be informative, as long as the noise level is not overly high.
To illustrate the practical applicability, we plot in Figure~\ref{fig:clustering-motivation}(c) the numerical performance of this approach, which allows for perfect clustering of all individuals  for a wide range of noisy scenarios. Similar ideas continue to fare well on the clustering of real data, where we illustrate in Figure~\ref{fig:dolphin-motivation} that the penultimate eigenvector of a Laplacian matrix (also known as the Fiedler vector) of an undirected social network reveals  two communities of 62 dolphins residing in Doubtful Sound, New Zealand.

\begin{figure}[t]
\begin{center}
\begin{tabular}{cc}
	\includegraphics[width=0.42\textwidth]{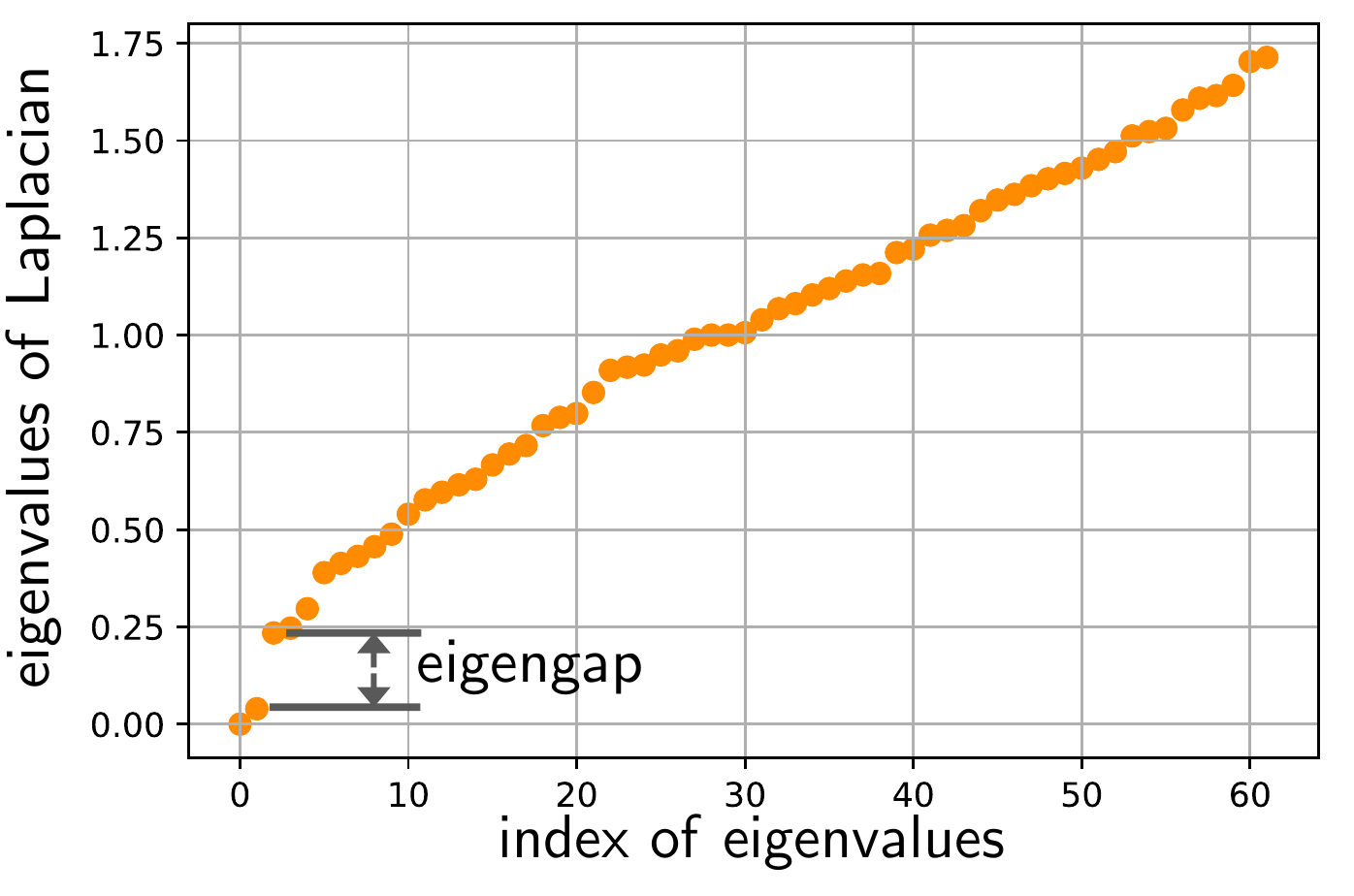}  &   \includegraphics[width=0.47\textwidth,height=1.4in]{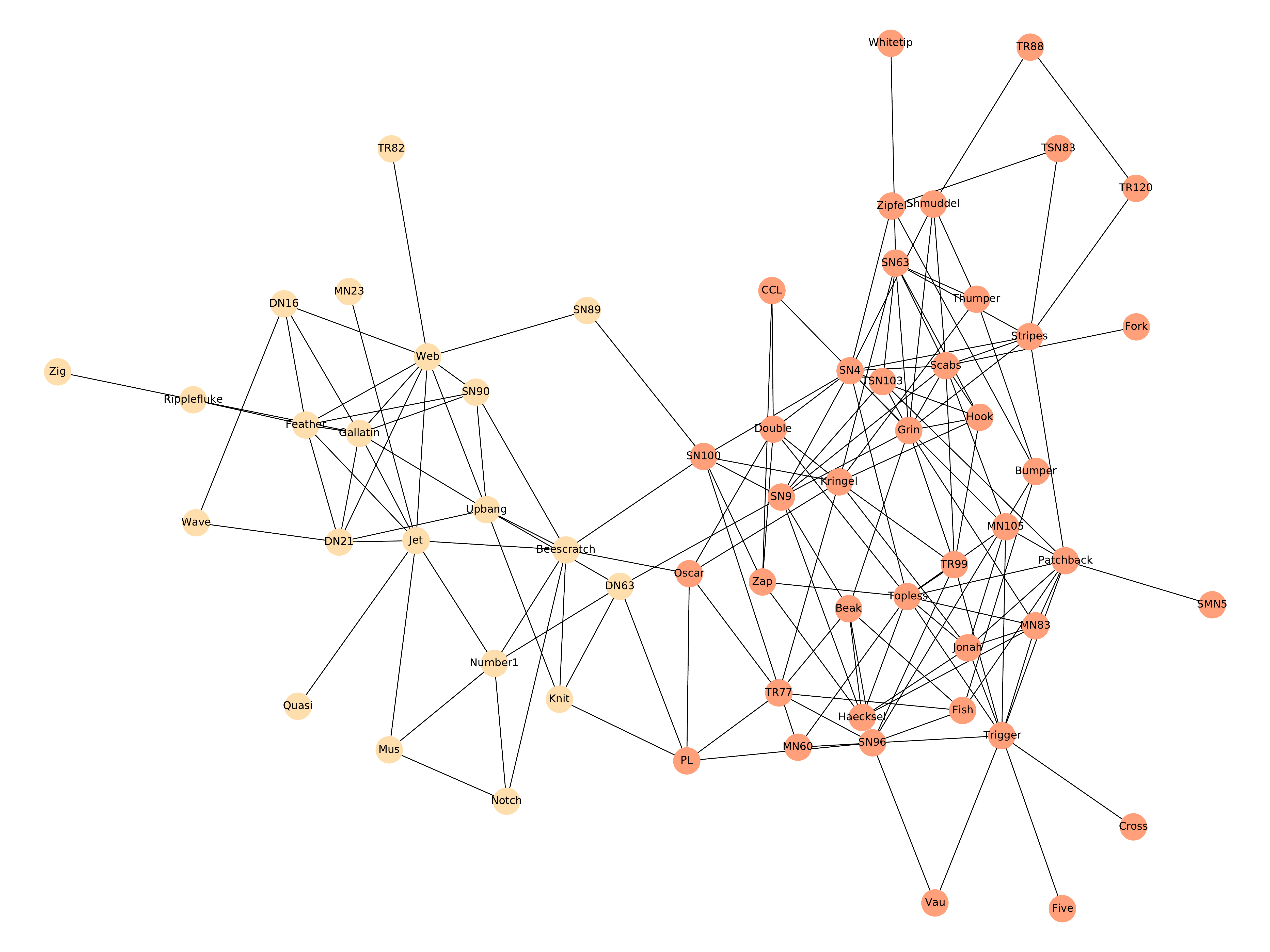} \tabularnewline
	(a) & (b) \tabularnewline
\end{tabular}	
\end{center}	
	\caption{Illustration of spectral clustering for 62 dolphins residing in Doubtful Sound, New Zealand. (a) plots the spectrum of the  Laplacian matrix of an undirected social network of frequent associations, and (b) illustrates the  two communities recovered using the penultimate eigenvector of the Laplacian matrix. Data source: \citet{lusseau2003bottlenose}.}
	\label{fig:dolphin-motivation}

\end{figure}

 \begin{figure}[t]
\begin{center}
\begin{tabular}{ccccc}
\hspace{-0.15in}\includegraphics[width=0.215\textwidth]{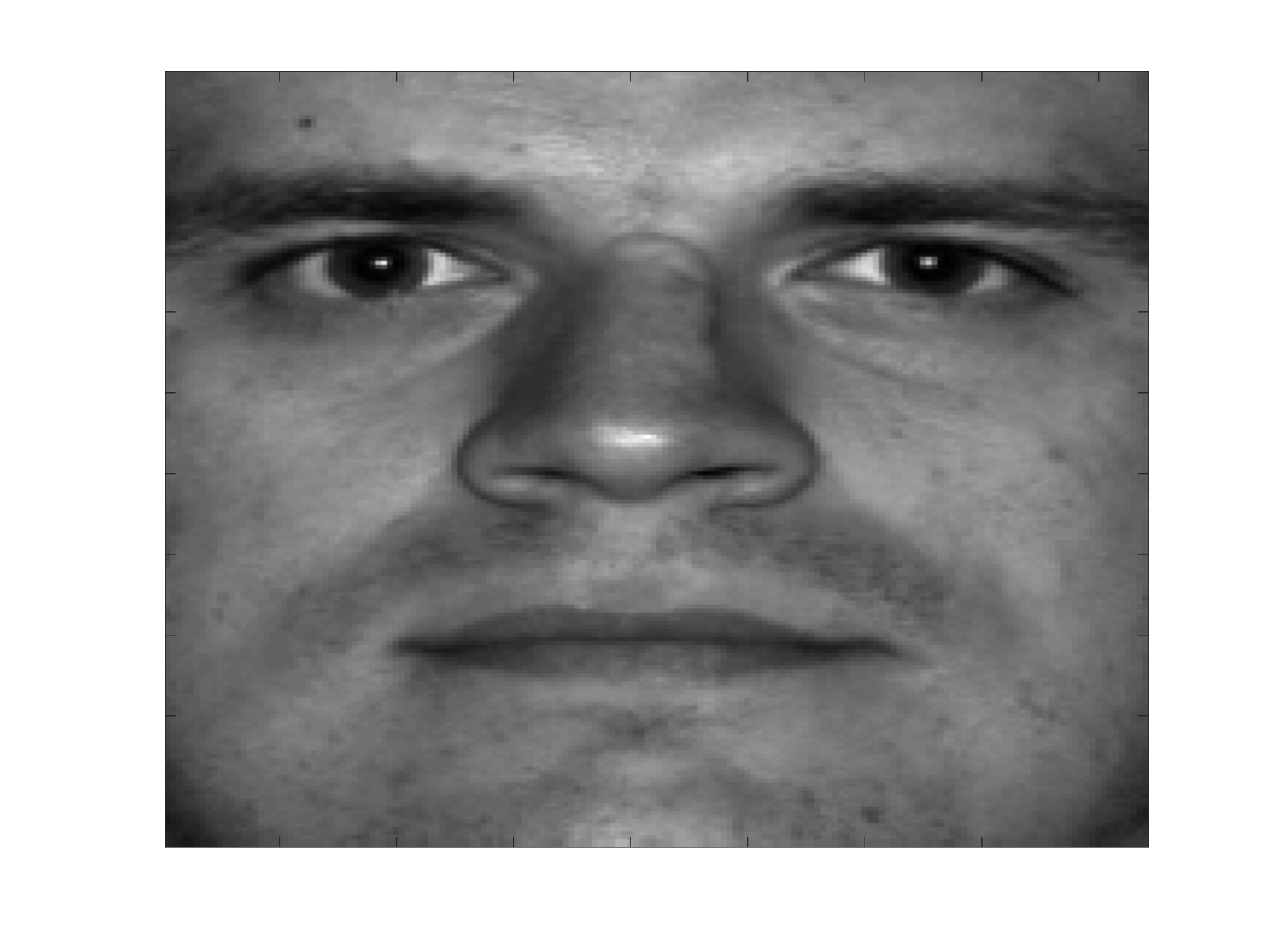} &
\hspace{-0.22in}\includegraphics[width=0.215\textwidth]{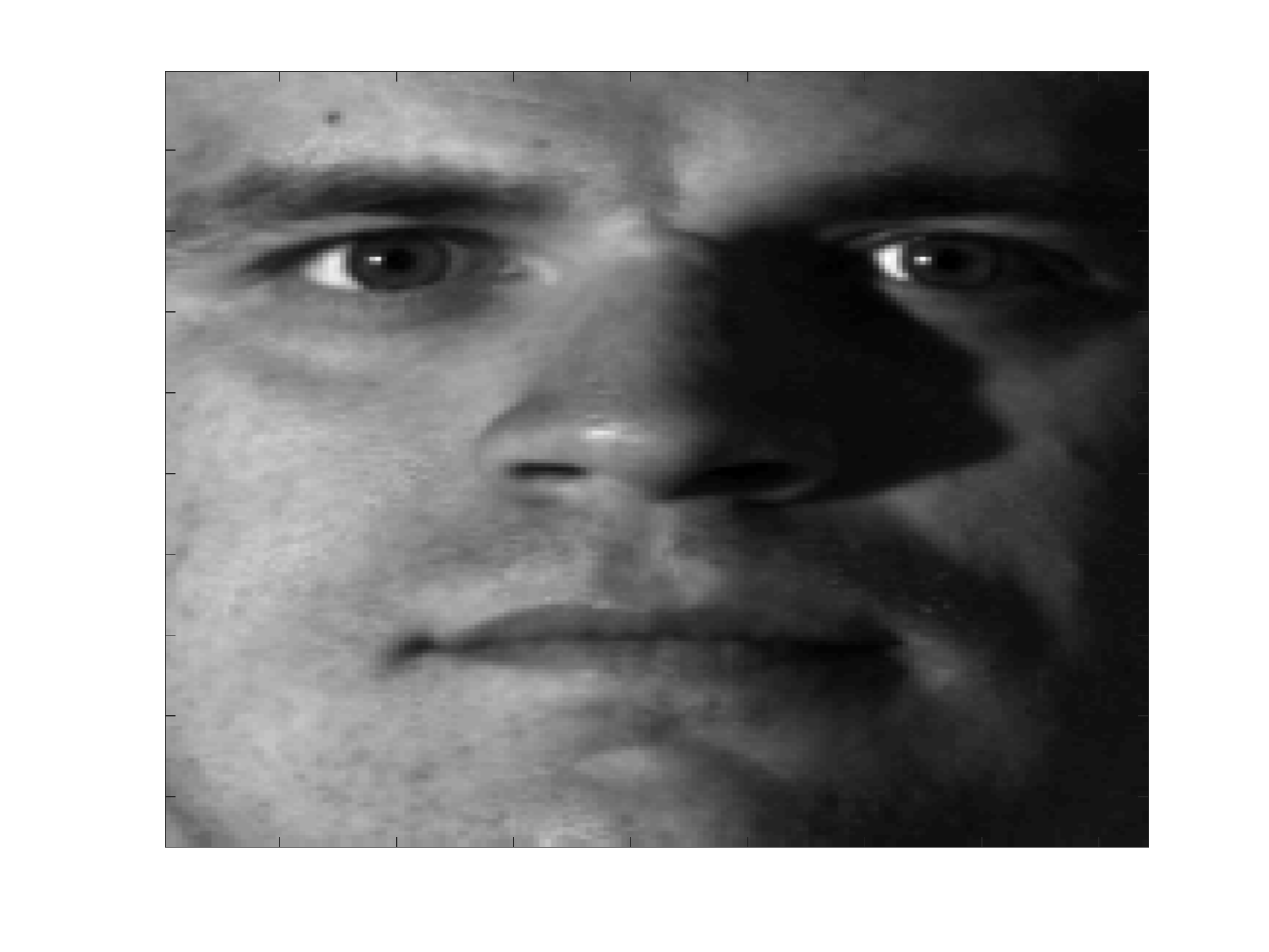} &
\hspace{-0.22in}\includegraphics[width=0.215\textwidth]{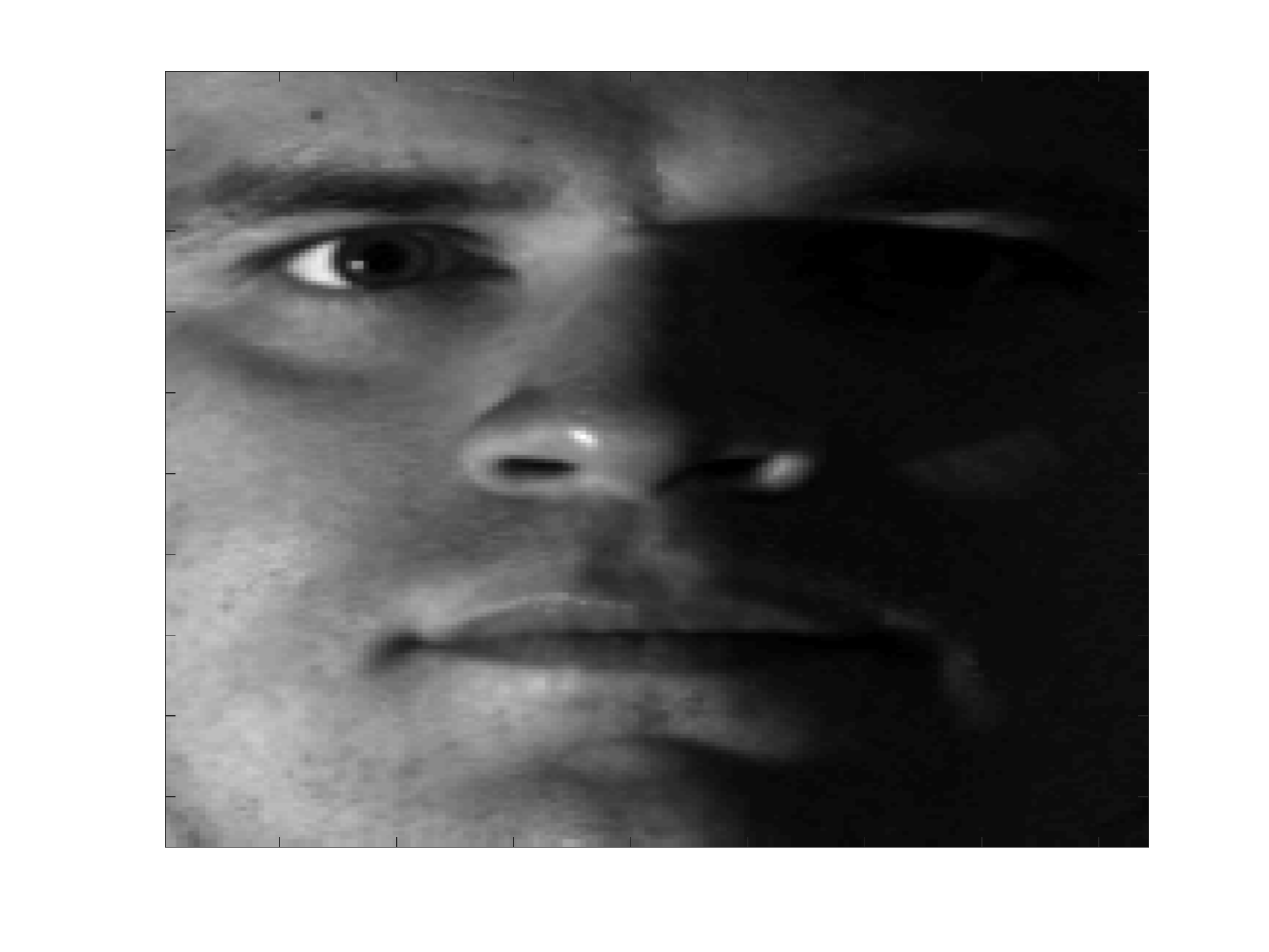} &
\hspace{-0.22in}\includegraphics[width=0.215\textwidth]{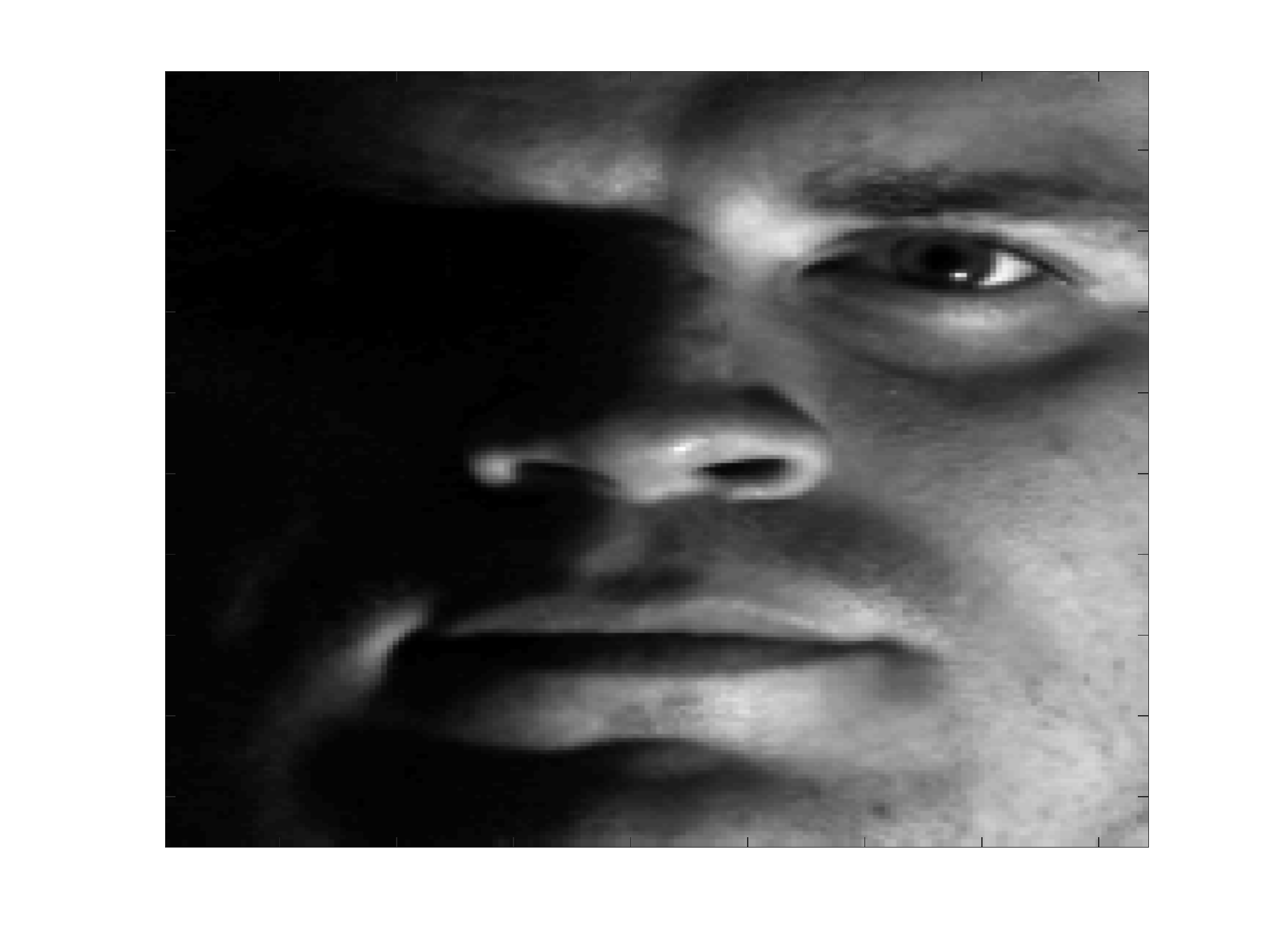} &
\hspace{-0.22in}\includegraphics[width=0.215\textwidth]{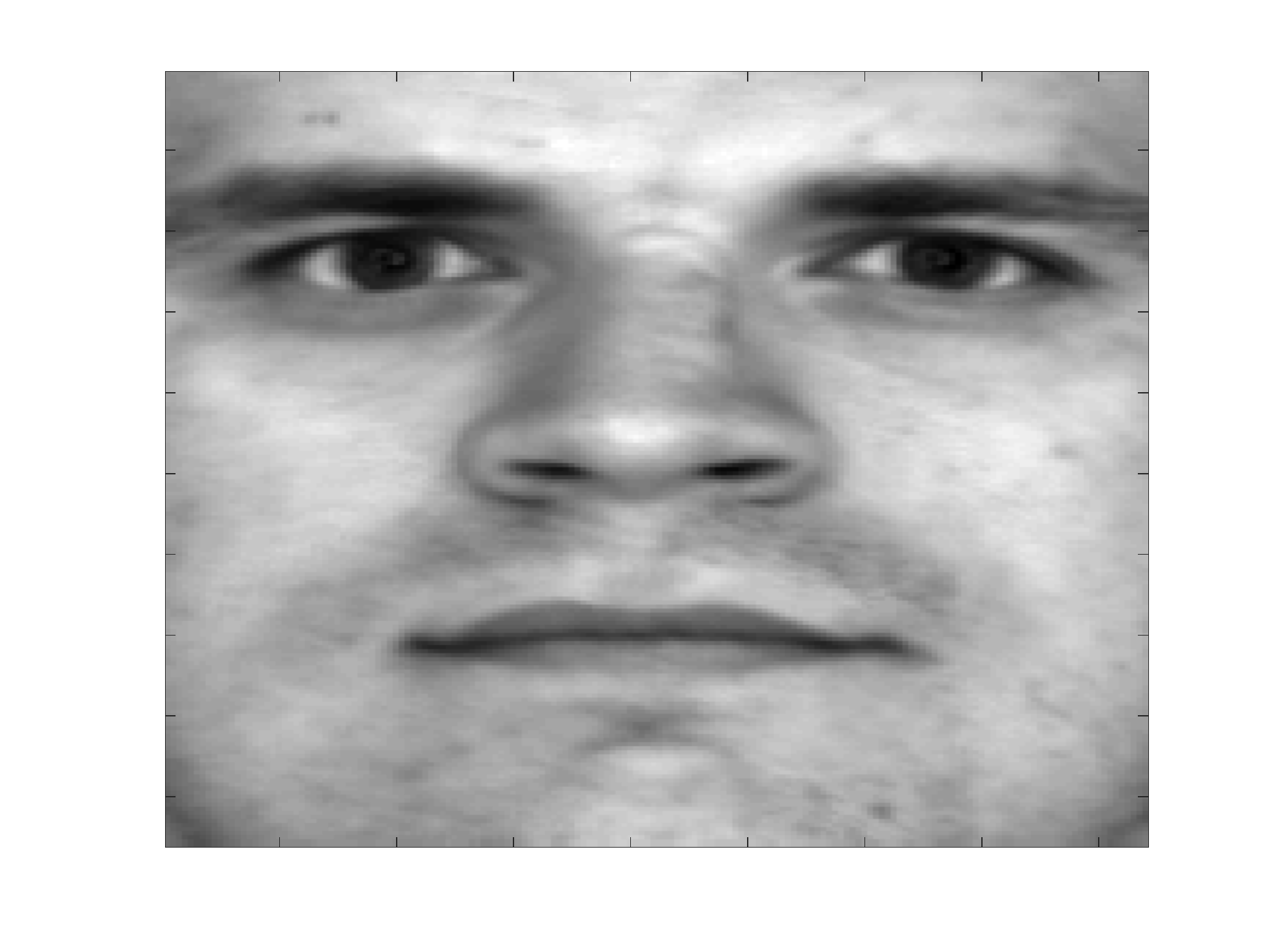}
%\\
%(a) & (b) & (c) & (d) & (e)
\end{tabular}
\end{center}
	 \caption{Illustration of the eigenface using the Cropped YaleB dataset \citep{georghiades2001few}. The first four images are sampled from this dataset, representing typical images taken under different illumination conditions with various occlusions. The last one represents the eigenface (i.e., the first principal component) of this dataset.\label{fig:eigenface}}
\end{figure}

\paragraph{Principal component analysis (PCA).}
PCA is arguably one of the most commonly employed tools for data exploration and visualization.
Given a collection of data samples $\bm{x}_1, \cdots, \bm{x}_n\in\mathbb{R}^p$,
PCA seeks to identify a rank-$r$ subspace that explains most of the variability of the data.
This is particularly well-grounded when, say, the sample vectors $\{\bm{x}_i\}_{1\leq i\leq n}$ reside primarily within a common rank-$r$ subspace---denoted by $\bm{U}^{\star}$.
To extract out this principal subspace,  it is instrumental to examine the following sample covariance matrix
\begin{align*}
	\bm{M} = \frac{1}{n} \sum_{i=1}^n \bx_i \bx_i ^\top .
\end{align*}
If all sample vectors approximately lie within $\bm{U}^{\star}$, then one might be able to infer
$\bm{U}^{\star}$ by inspecting the rank-$r$ leading eigenspace of $\bm{M}$ (or its variants), provided that the signal-to-noise ratio exceeds some reasonable level.  This reflects the role of spectral methods in enabling meaningful dimensionality reduction and factor analysis.

In practice, a key benefit of  PCA is its ability to  remove  nuance factors in, and  extract out salient features from,  each data point.  As an illustration, the first four images of Figure \ref{fig:eigenface} are representative ones sampled from a face dataset \citep{georghiades2001few}, which correspond to faces of the same person under different illumination and occlusion conditions. In contrast, the ``eigenface'' \citep{turk1991face} depicted in the last image of Figure \ref{fig:eigenface}  corresponds to the first principal component (i.e., $r=1$), which effectively removes the nuance factors and highlights the feature of the face.

\begin{figure}
\begin{center}
\begin{tabular}{cc}
	\includegraphics[width=0.3\textwidth]{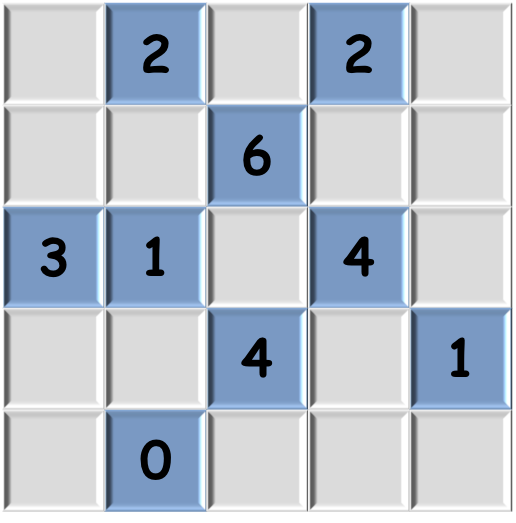} \qquad\qquad &   \includegraphics[width=0.47\textwidth]{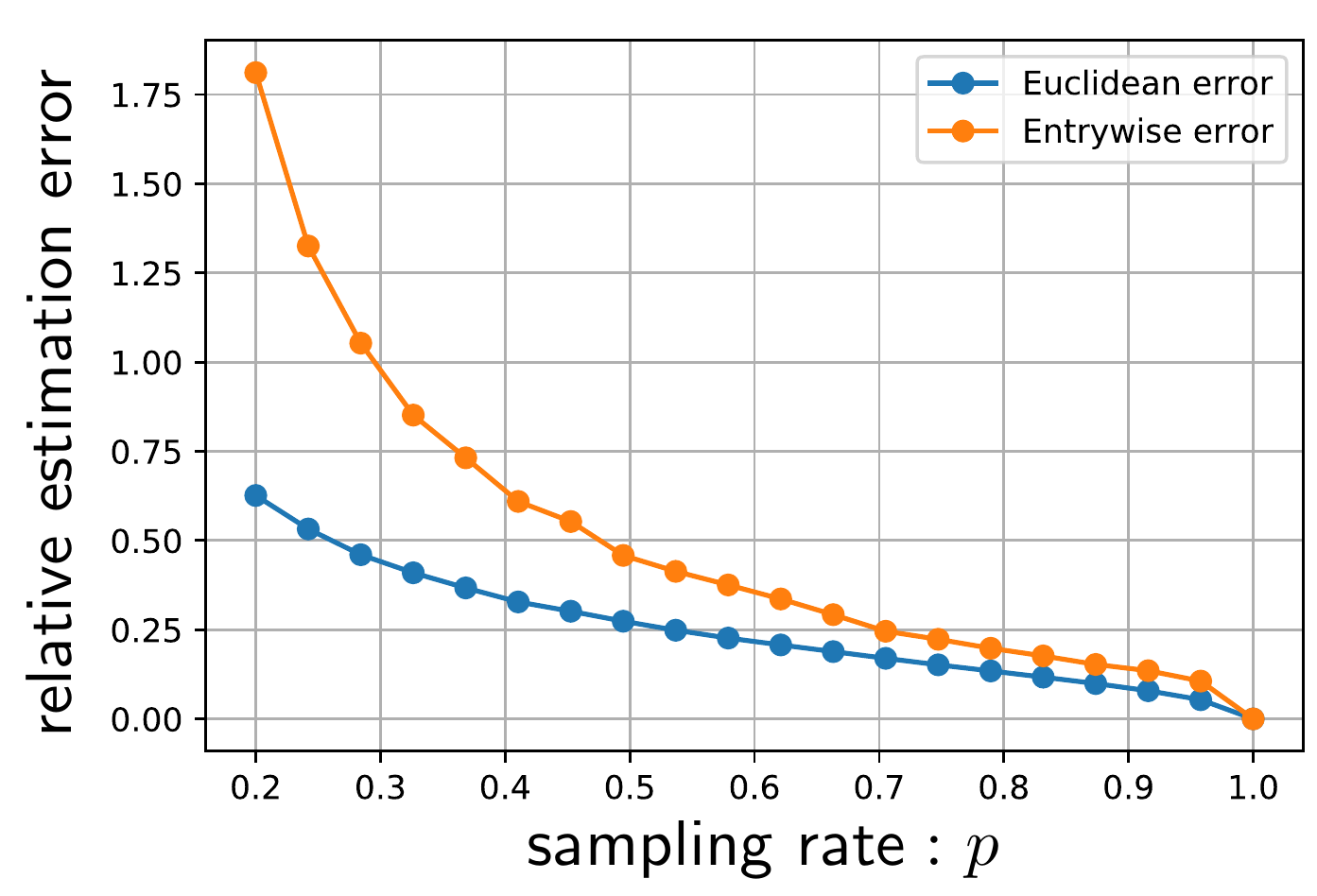} \tabularnewline
	(a) & (b) \tabularnewline
\end{tabular}
\end{center}	
	\caption{Spectral methods for matrix recovery with missing data, where (a) is an illustration of missing data and (b) reports the empirical estimation errors of spectral methods as the sampling rate $p$ varies.  Both the relative Euclidean error $\frac{\|\widehat{\bm{M}}-\bm{M}^{\star}\|_{\mathrm{F}}}{\|\bm{M}^{\star}\|_{\mathrm{F}}}$ and the relative entrywise error $\frac{\|\widehat{\bm{M}}-\bm{M}^{\star}\|_{\infty}}{\|\bm{M}^{\star}\|_{\infty}}$ are plotted (with $\widehat{\bm{M}}$ denoting the matrix estimate and $\|\cdot\|_{\infty}$ the entrywise $\ell_{\infty}$ norm).
	\label{fig:matrix-completion-motivation}}

\end{figure}

\paragraph{Matrix recovery in the face of missing data.}

A proliferation of big-data applications has to deal with matrix estimation in the presence of missing data,  either due to the infeasibility to acquire complete observations of a massive data matrix \citep{davenport2016overview} such as the Netflix problem in recommender systems (as users only watch and rate a small fraction of movies), or because of the incentive to accelerate computation by means of sub-sampling \citep{mahoney2016lecture}.
Imagine that we are asked to estimate a large matrix $\bm{M}^{\star}=[M_{i,j}^{\star}]_{1\leq i,j\leq n}$, even though a dominant fraction of its entries are unseen.  While in general we cannot  predict anything about the missing entries,  reliable estimation might become possible if $\bm{M}^{\star}$ is known {\em a priori} to enjoy a low-rank structure, as is the case in many applications like structure from motion \citep{tomasi1992shape} and sensor network localization \citep{javanmard2013localization}. This low-rank assumption motivates the use of spectral methods.  More specifically, suppose the entries of $\bm{M}^{\star}$ are randomly sampled such that each entry  is observed independently with probability $p \in (0,1]$. An unbiased estimate $\bm{M}=[M_{i,j}]_{1\leq i,j\leq n}$ of $\bm{M}^{\star}$ can be readily obtained via  rescaling and zero filling (also called the inverse probability weighting method):
\begin{equation*}
	M_{i,j}=\begin{cases}
\frac{1}{p}M_{i,j}^{\star},\quad & \text{if the }(i,j)\text{-th}\text{ entry is observed},\\
0, & \text{else}.
\end{cases}
\end{equation*}
To capture the assumed low-rank structure of $\bm{M}^{\star}$, it is natural to resort to the best rank-$r$ approximation of $\bm{M}$ (with $r$  the true rank of $\bm{M}^{\star}$),
computable through the rank-$r$ singular value decomposition of $\bm{M}$.  Given its (trivial) success when $p=1$,  we expect the algorithm  to perform well when  $p$ is close to 1.  The key question, however, is where the algorithm stands if the vast majority of the entries is missing.
While we shall illuminate this in Chapters~\ref{chap:application-L2} and \ref{cha:Linf-theory}, Figure~\ref{fig:matrix-completion-motivation} provides some immediate numerical assessment,
which demonstrates the appealing performance of spectral methods---in terms of both Euclidean and entrywise estimation errors---even when  the missing rate is quite high.

\begin{figure}
\begin{center}
\begin{tabular}{cc}
	\includegraphics[width=0.4\textwidth]{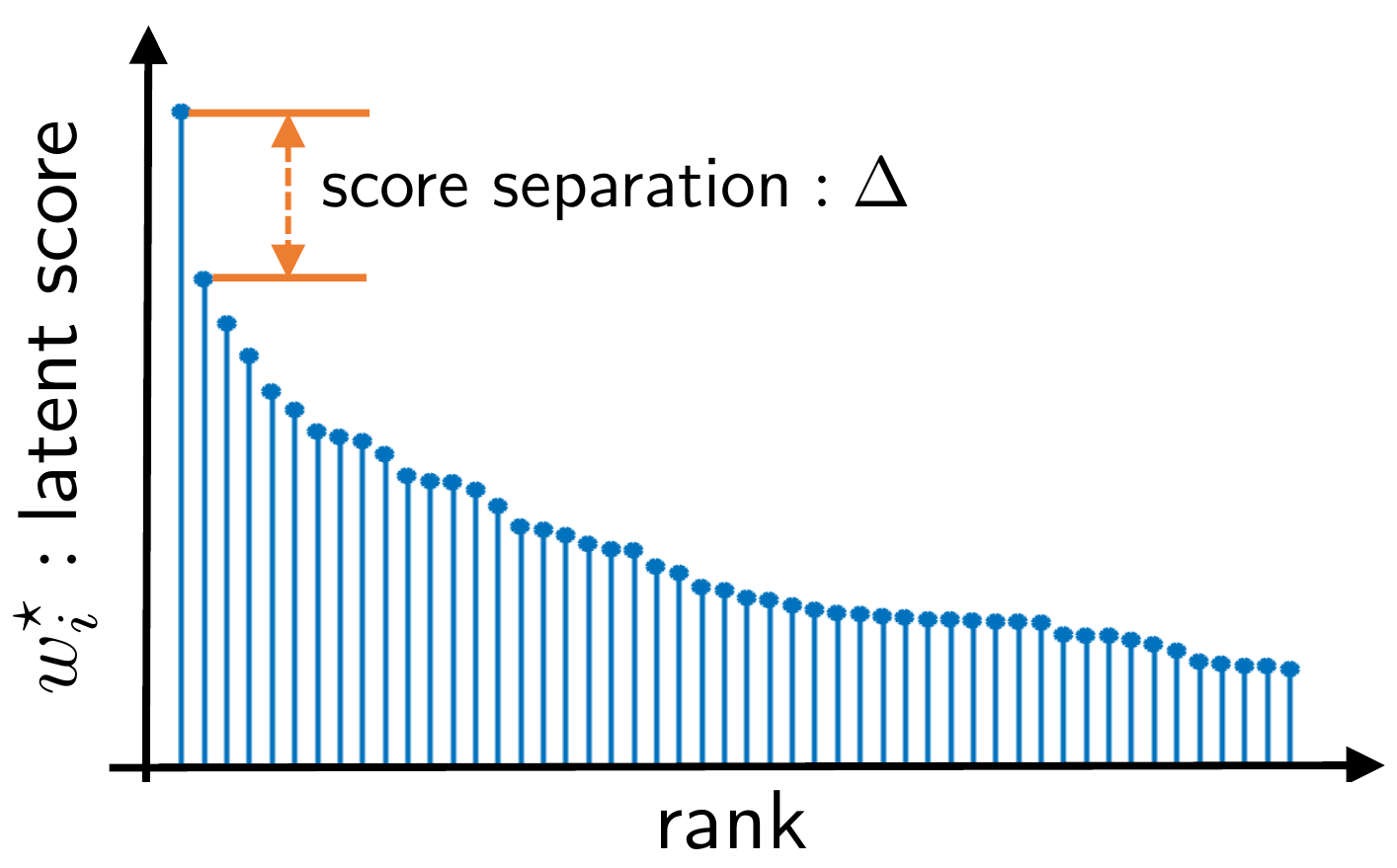} \qquad &
	\includegraphics[width=0.4\textwidth]{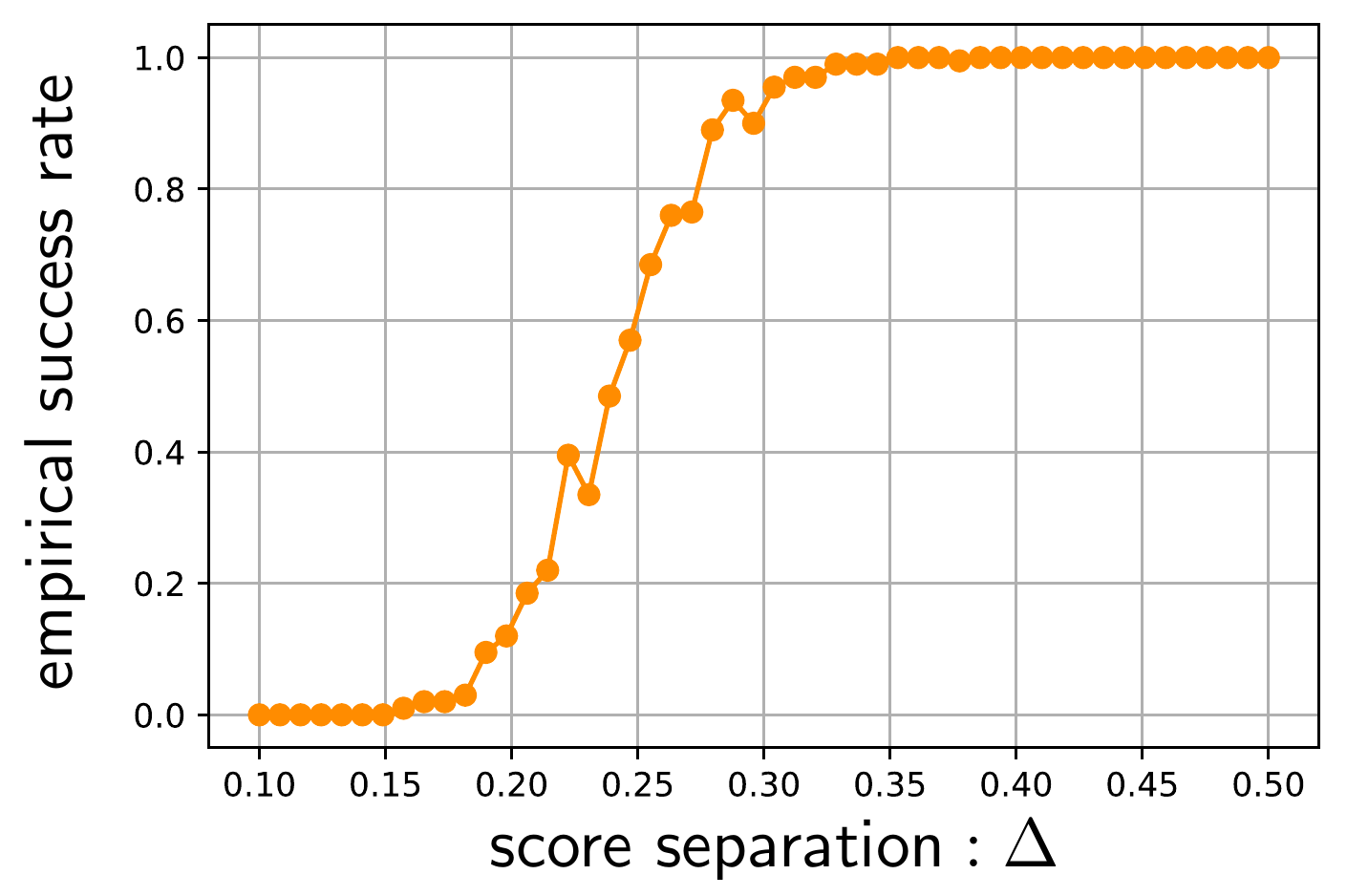} \tabularnewline
	(a) & (b) \tabularnewline
\end{tabular}
\end{center}	
	\caption{Spectral methods for ranking from pairwise comparisons.   (a)  illustrates the latent preference  scores $\{w_i^{\star}\}$ that govern the ranking of items. The empirical success rates in correctly identifying the top-ranked item  are plotted in (b) as $\Delta$ varies, where   $\Delta$ represents the separation between the score of the top item and that of the second-ranked item.
	\label{fig:ranking-motivation}}

\end{figure}

\paragraph{Ranking from pairwise comparisons.}
Another important application of spectral methods arises from the context of ranking,  a task  of central importance in, say, web search and recommendation systems.
In a variety of scenarios, humans find it difficult to simultaneously rank many items, but relatively easier to express pairwise preferences.
This gives rise to the problem of ranking based on pairwise comparisons.
More specifically, imagine we are given a collection of $n$ items, and wish to identify top-ranked items based on pairwise preferences (with uncertainties in comparison outcomes) between observed pairs of items.
A classical statistical model proposed by \citet{bradley1952rank, luce2012individual} postulates the existence of a set of latent positive scores $\{w_i^{\star}\}_{1\leq i\leq n}$---each associated with an item---that determines the ranks of these items.
 The outcome of the comparison between items $i$ and $j$ is generated in a way that
 \begin{align*}
	 \mathbb{P}(i\text{ beats }j) = \frac{w_i^{\star}}{w_i^{\star}+w_j^{\star}}, \qquad 1\leq i,j\leq n.
\end{align*}
As it turns out, the preference scores  are closely related to the stationary distribution of a Markov chain associated with the above probability kernel,
thus forming the basis of spectral ranking algorithms.
To elucidate it in a little more detail, let us construct a probability transition matrix $\bm{P}^{\star}=[P_{i,j}^{\star}]_{1\leq i,j\leq n}$ with
\begin{align*}
	P_{i,j}^{\star}=\begin{cases}
\frac{1}{n}\cdot\frac{w_{j}^{\star}}{w_{i}^{\star}+w_{j}^{\star}}, & \text{if }i\neq j,\\
		1-\sum_{l:l\neq i}P_{i,l}^{\star}, \qquad & \text{if }i=j.
\end{cases}
\end{align*}
Clearly, it forms a probability transition matrix since each element is nonnegative and the entries in each row add up to one.
It is straightforward to verify that the score vector $\bm{w}^{\star}\coloneqq [w_i^{\star}]_{1\leq i\leq n}$ satisfies $\bm{w}^{\star\top} = \bm{w}^{\star\top} \bm{P}^{\star}$, namely $\bm{w}^{\star}$ is a left eigenvector of $\bm{P}^{\star}$ associated with eigenvalue one. A candidate  method then consists of (i) forming an unbiased estimate of $\bm{P}^{\star}$ (which can be easily obtained using pairwise comparison outcomes), (ii) computing its left eigenvector (in fact, the leading left eigenvector), and (iii) reporting the ranking result in accordance with the order of the elements in this eigenvector.
This spectral ranking scheme, which shares similar spirit with the celebrated {\em PageRank} algorithm \citep{page1999pagerank},
exhibits intriguing performance when identifying the top-ranked items, as showcased in the numerical experiments in Figure \ref{fig:ranking-motivation}(b).

\paragraph{A unified theme.}  In all preceding applications, the core ideas underlying the development of spectral methods can be described in a unified fashion:
\begin{itemize}
	\item[1.] Identify a key matrix $\bm{M}^{\star}$---which is typically unobserved---whose eigenvectors or singular vectors disclose the information being sought after;

	\item[2.] Construct a surrogate matrix $\bm{M}$ of $\bm{M}^{\star}$ using the data samples in hand, and compute the corresponding eigenvectors or singular vectors of this surrogate matrix.
\end{itemize}
Viewed in this light,  this monograph aims to identify key factors---e.g., certain spectral structure of $\bm{M}^{\star}$ as well as the size of the approximation error $\bm{M}-\bm{M}^{\star}$---that exert main influences on the efficacy of the resultant spectral methods.

\section{A modern statistical perspective}

The idea of spectral methods can be traced back to early statistical literature on methods of moments (e.g., \citet{pearson1894contributions,hansen1982large}),
where one seeks to extract key parameters of the probability distributions of interest by examining the empirical  moments of data.
While classical matrix perturbation theory lays a sensible foundations for the analysis of spectral methods \citep{stewart1990matrix},
the theoretical understanding can be considerably enhanced through the lens of statistical modeling---a way of thinking that has flourished in the past decade.
To the best of our knowledge, however,
a systematic and comprehensive introduction to the modern statistical foundation of spectral methods, as well as an overview of recent advances, is previously unavailable.

The current monograph aims to fill this gap by developing a coherent and accessible treatment of spectral methods from
a modern statistical perspective. Highlighting  algorithmic implications that inform practice,
our exposition gravitates around the following central questions:
how to characterize the sample efficiency of spectral methods in reaching a prescribed accuracy level,
and how to assess the stability of spectral methods in the face of random noise, missing data, and adversarial corruptions?
We underscore several distinguishing features of our treatment compared to prior studies:
\begin{itemize}
	\item In comparison to the worst-case performance guarantees derived solely based on classical matrix perturbation theory,
our statistical treatment emphasizes the benefit of harnessing the ``typical'' behavior of data models,
which offers key insights into how to harvest performance gains by leveraging intrinsic properties of data generating mechanisms.

	\item In contrast to classical asymptotic theory~\citep{van2000asymptotic},
		we adopt a non-asymptotic (or finite-sample) analysis framework that draws on tools from recent developments of concentration inequalities \citep{tropp2015introduction} and high-dimensional statistics \citep{wainwright2019high}. This framework accommodates the scenario where both the sample size and the number of features are enormous,
and unveils a clearer and more complete picture about the interplay and trade-off between salient model parameters.

\end{itemize}

Another unique feature of this monograph is a principled introduction of {\em fine-grained entrywise analysis} (e.g., a theory studying $\ell_{\infty}$ eigenvector perturbation),
which reflects cutting-edge research activities in this area.
This is particularly important when, for example, demonstrating the feasibility of exact clustering or perfect ranking in the aforementioned applications.
In truth, an effective entrywise analysis framework cannot be  readily obtained from classical matrix analysis alone,
and has only recently become available owing to the emergence of modern statistical toolboxes. In particular, we shall present
a powerful framework, called {\em leave-one-out analysis}, that proves effective and versatile for delivering fine-grained performance guarantees for spectral methods in a variety of problems.

\section{Organization}

We now present a high-level overview of the structure of this monograph.
\begin{itemize}
\item
Chapter~\ref{cha:matrix-perturbation} reviews the fundamentals of classical matrix perturbation theory for spectral analysis, focusing on $\ell_2$-type distances measured by the spectral norm and the Frobenius norm. This chapter covers the celebrated Davis-Kahan $\sin\bm{\Theta}$ theorem for eigenspace perturbation, the Wedin theorem for singular subspace perturbation, and an extension to probability transition matrices, laying the algebraic foundations for the remaining chapters.

\item
Chapter~\ref{chap:application-L2} explores the utility of $\ell_2$ matrix perturbation theory when paired with statistical tools,
presenting a unified recipe for statistical analysis empowered by non-asymptotic matrix tail bounds.
We develop spectral methods for a variety of statistical data science applications,
and derive nearly tight theoretical guarantees (up to logarithmic factors) based on this unified recipe.

\item Chapter~\ref{cha:Linf-theory} develops fine-grained perturbation theory for spectral analysis in terms of  $\ell_{\infty}$ and $\ell_{2,\infty}$ metrics, based on a leave-one-out analysis framework rooted in probability theory. Its effectiveness is demonstrated through concrete applications including community recovery and matrix completion.
  This analysis framework also enables a non-asymptotic distributional theory for spectral methods, which paves the way for uncertainty quantification in applications like noisy matrix completion.

\item  Chapter~\ref{chapter:conclusion} concludes this monograph by identifying a few directions that are worthy of future investigation.

\end{itemize}

\noindent While this monograph pursues a coherent and accessible treatment that might appeal to a broad audience,
it does not necessarily deliver the sharpest possible results for the applications discussed herein in terms of the logarithmic terms and/or pre-constants.
The bibliographic notes at the end of each chapter contain information about the state-of-the-art theory for each application as a pointer to further readings.

\section{What is not here and complementary readings}

The topics presented in this monograph do not cover the tensor decomposition methods studied in another recent strand of work \citep{anandkumar2014tensor}.  While such tensor-based methods are also sometimes referred to as spectral methods, their primary focus is to invoke tensor decomposition to learn latent variables,  based on higher-order moments estimated from data samples.
We elect not to discuss this class of methods but instead refer the interested reader to the recently published monograph by \citet{MAL-057}.
Another monograph by \citet{kannan2009spectral} provides an in-depth computational and algorithmic treatment of spectral methods from the perspective of theoretical computer science.
The applications and results covered therein (e.g., fast matrix multiplication)  complement the ones presented in the current monograph.
In addition, spectral methods have been frequently employed to initialize nonconvex optimization algorithms. We will not elaborate on the nonconvex optimization aspect here but instead recommend the reader to the recent overview article by \citet{chi2019nonconvex}.  Finally, spectral methods are widely adopted to estimate high-dimensional covariance and precision matrices, and extract latent factors for econometric and statistical modeling. This topic alone has a huge literature, and we refer the interested reader to \citet{fan2020statistical} for in-depth discussions.

\section{Notation}

Before moving forward, let us introduce some notation that will be used throughout this monograph.

First of all, we reserve boldfaced symbols  for vectors, matrices and tensors.
For any matrix $\bm{A}$, let $\sigma_{j}(\bm{A})$ (resp.~$\lambda_{j}(\bm{A})$) represent its $j$-th largest singular value (resp.~eigenvalue).
In particular, $\sigma_{\max}(\bm{A})$ (resp.~$\lambda_{\max}(\bm{A})$) stands for the largest singular value (resp.~eigenvalue) of $\bm{A}$,
while $\sigma_{\min}(\bm{A})$ (resp.~$\lambda_{\min}(\bm{A})$) indicates the smallest singular value (resp.~eigenvalue) of $\bm{A}$.
We use $\bm{A}^{\top}$ to denote the transpose of $\bm{A}$, and
let $\bm{A}_{i,\cdot}$ and $\bm{A}_{\cdot,i}$ indicate the $i$-th row and the $i$-th column of $\bm{A}$, respectively.
We follow standard conventions by letting
$\bm{I}_{n}$ be the $n\times n$ identity matrix, $\bm{1}_{n}$ the $n$-dimensional all-one vector, and $\bm{0}_{n}$ the $n$-dimensional all-zero vector;
we shall often suppress the subscript as long as it is clear from the context.
The $i$-th standard basis vector is denoted by $\bm{e}_i$ throughout.
The notation $\mathcal{O}^{n\times r}$ ($r\leq n$) represents the set of all $n\times r$ orthonormal matrices (whose columns are orthonormal). Moreover, we refer to $[n]$ as the set $\{1,\cdots, n\}$.

Next, we  turn to vector and matrix norms. For any vector $\bm{v}$, we denote by $\|\bm{v}\|_2$, $\|\bm{v}\|_1$ and $\|\bm{v}\|_{\infty}$  its $\ell_2$ norm, $\ell_1$ norm and $\ell_{\infty}$ norm, respectively.
 For any matrix $\bm{A}=[A_{i,j}]_{1\leq i\leq m,1\leq j\leq n}$,  we let $\|\bm{A}\|$, $\|\bm{A}\|_{*}$, $\|\bm{A}\|_{\mathrm{F}}$ and $\|\bm{A}\|_{\infty}$  represent respectively its spectral norm (i.e., the largest singular value of $\bm{A}$), its nuclear norm (i.e., the sum of singular values of $\bm{A}$), its Frobenius norm (i.e., $\|\bm{A}\|_{\mathrm{F}} \coloneqq \sqrt{\sum_{i,j}A_{i,j}^2}$), and its entrywise $\ell_{\infty}$ norm (i.e., $\|\bm{A}\|_{\infty} \coloneqq \max_{i,j}|A_{i,j}|$).
We also refer to $\|\bm{A}\|_{2,\infty}$  as the $\ell_{2,\infty}$ norm  of $\bm{A}$,
defined as $\|\bm{A}\|_{2,\infty} \coloneqq \max_i \|\bm{A}_{i,\cdot}\|_2$.
Similarly, we define the $\ell_{\infty,2}$ norm  of $\bm{A}$ as $\|\bm{A}\|_{\infty,2} \coloneqq \|\bm{A}^{\top}\|_{2,\infty}$.
In addition, for any matrices $\bm{A}=[A_{i,j}]_{1\leq i\leq m,1\leq j\leq n}$ and $\bm{B}=[B_{i,j}]_{1\leq i\leq m,1\leq j\leq n}$,
the inner product of $\bm{A}$ and $\bm{B}$ is defined as and denoted by $\langle \bm{A}, \bm{B} \rangle = \sum_{1\leq i\leq m,1\leq j\leq n} A_{i,j} B_{i,j} = \mathsf{Tr}(\bm{A}^{\top}\bm{B})$.

When it comes to diagonal matrices,
 we employ $\mathsf{diag}([\theta_1, \theta_2, \cdots, \theta_r])$ to abbreviate the diagonal matrix with diagonal elements $\theta_1, \cdots, \theta_r$.
For any diagonal matrix $\bm{\Theta} = \mathsf{diag}([\theta_1, \theta_2, \cdots, \theta_r])$,
we adopt the shorthand notation $\sin \bm{\Theta} \coloneqq \mathsf{diag}([\sin\theta_1, \sin\theta_2, \cdots, \sin\theta_r])$; the notation $\sin^2\bm{\Theta}$, $\cos\bm{\Theta}$, and $\cos^2\bm{\Theta}$ is defined analogously.

Finally, this monograph makes heavy use of the following standard notation:
(1) $f(n)=O\left(g(n)\right)$ or
$f(n)\lesssim g(n)$ means that there exists a universal constant $c>0$ such
that $\left|f(n)\right|\leq c|g(n)|$ holds for all sufficiently large $n$; (2) $f(n)\gtrsim g(n)$ means that there exists a universal constant $c>0$ such
that $|f(n)|\geq c\left|g(n)\right|$ holds for all sufficiently large $n$;
 (3) $f(n)\asymp g(n)$ means that there exist universal constants $c_{1},c_{2}>0$
such that $c_{1}|g(n)|\leq|f(n)|\leq c_{2}|g(n)|$ holds for all sufficiently large $n$;
and (4) $f(n)=o(g(n))$ indicates that $f(n)/g(n)\rightarrow 0$ as $n\rightarrow \infty$.
Additionally, we sometimes use $f(n)\gg g(n)$ (resp.~$f(n)\ll g(n)$) to indicate that there exists some sufficiently large (resp.~small) universal constant $c >0$ such that $|f(n)|\geq c \left|g(n)\right|$ (resp.~$|f(n)|\leq c \left|g(n)\right|$).

%% file: chapters/L2.tex
\chapter[Classical spectral analysis: $\ell_2$ perturbation theory]{Classical spectral analysis: $\ell_2$ perturbation theory}
\label{cha:matrix-perturbation}

Characterizing the performance of spectral methods requires understanding the perturbation of eigenspaces and/or that of singular subspaces.
Classical matrix perturbation theory (e.g., \citet{stewart1990matrix}) offers elementary toolkits that prove effective for this purpose, which we review in this chapter.

Setting the stage, consider a real-valued matrix $\bm{M}^{\star}$ and its perturbed version as follows
\begin{align}
	\bm{M} = \bm{M}^{\star} + \bm{E},
	\label{eq:perturbed-M}
\end{align}
where $\bm{E} = 	\bm{M} -  \bm{M}^{\star}$ denotes a real-valued perturbation or error matrix. In statistical applications, $\bm{M}$ can be an observed or estimated data matrix such as the sample covariance matrix, and $\bm{M}^{\star}$ is the target matrix such as the population covariance matrix.  This chapter primarily aims to address the following questions by means of elementary linear algebra:
\begin{enumerate}

	\item For a symmetric matrix $\bm{M}^{\star}$, how does the eigenspace change in response to a symmetric perturbation matrix $ \bm{E}$?
	\item For a general matrix $\bm{M}^{\star}$, how is the singular subspace affected as a result of the perturbation matrix $ \bm{E}$?
\end{enumerate}
We shall also explore eigenvector perturbation for a special class of asymmetric matrices: probability transition matrices.

\input{chapters/prelim}

\section{Preliminaries: Distance and angles between subspaces}
\label{sec:preliminary-distance-angles}

In order to develop perturbation theory for eigenspaces and singular subspaces, we first need to delineate a metric that quantifies the proximity of two subspaces in a meaningful way.

\subsection{Setup and notation}
\label{sec:setting-distance-angle}

Consider two $r$-dimensional subspaces  $\mathcal{U}^{\star}$ and ${\mathcal{U}}$ in $\mathbb{R}^{n}$, where $1\le r\le n$.
One can represent these two subspaces by two matrices
 $\bm{U}^{\star} \in \mathbb{R}^{n\times r}$
and ${\bm{U}} \in \mathbb{R}^{n\times r}$, whose columns form an orthonormal
basis of $\mathcal{U}^{\star}$ and ${\mathcal{U}}$, respectively.
Here and throughout, we shall use $\mathcal{U}$ and its matrix representation $\bm{U}$ interchangeably whenever it is clear from the context.

For the sake of convenience, we further introduce two $n\times (n-r)$ matrices $\bm{U}^{\star}_{\perp}$ and $\bm{U}_{\perp}$, such that $[\bm{U}^{\star},\bm{U}_{\perp}^{\star}]$ and $[\bm{U},\bm{U}_{\perp}]$ are both $n\times n$ orthonormal matrices.
In other words, $\bm{U}_{\perp}^{\star}$ and $\bm{U}_{\perp}$ represent the orthogonal complement of $\bm{U}^{\star}$ and $\bm{U}$, respectively.

\subsection{Distance metrics and principal angles}
\label{sec:introduction-distance-principal-angles}

\paragraph{Global rotational ambiguity.} To measure the distance between the two subspaces $\mathcal{U}$ and ${\mathcal{U}}^{\star}$,
a naive idea is to employ the ``metric'' $\vertiiiplain{\bm{U}-{\bm{U}}^{\star}}$,
where $\vertiii{\cdot}$ is a certain norm of interest (e.g., the spectral norm or the Frobenius norm).
An immediate drawback arises, however, since this ``metric'' does not
take into account the global rotational ambiguity---namely, for any rotation matrix $\bm{R}\in\mathcal{O}^{r\times r}$,
the columns of the matrix $\bm{U}\bm{R}$ also form a valid orthonormal basis of $\mathcal{U}$.
This means that even when the two subspaces $\mathcal{U}$ and ${\mathcal{U}}^{\star}$ coincide, one might still have $\vertiiiplain{\bm{U}-{\bm{U}^{\star}}} \neq 0$, depending on how we rotate these matrices.

\paragraph{Valid choices of distance and angles.}
The takeaway of the above discussion is that any meaningful metric employed to measure the proximity of two subspaces should account for the rotational ambiguity properly.
In what follows, we single out a few widely used metrics that meet such a requirement.

\begin{enumerate}

\item \emph{Distance with optimal rotation.}  Given the global rotational ambiguity, it is natural to first adjust the rotation matrix suitably before computing the distance.  One choice is to measure the distance upon optimal rotation, namely,
\begin{align}
	\mathsf{dist}_{\vertiii{\cdot}}\big(\bm{U},{\bm{U}}^{\star}\big) \coloneqq
	\min _{\bm{R}\in \mathcal{O}^{r\times r}} \vertiiibig{ \bm{U} \bm{R} - \bm{U}^{\star}  } ,
	\label{defn:dist-rotation}
\end{align}
where $\vertiii{\cdot}$ is a certain norm to be chosen (e.g., the spectral norm or the Frobenius norm).

\item \emph{Distance between projection matrices.}  As an established fact, the projection matrix onto a subspace $\mathcal{U}$---given by $\bm{U}\bm{U}^{\top}$---is unique and unaffected by how $\bm{U}$ is rotated (since $\bm{U}\bm{U}^{\top}=\bm{U}\bm{R}\bm{R}^{\top}\bm{U}^{\top}$ for any rotation matrix $\bm{R}\in \mathcal{O}^{r\times r}$).
The rotational invariance of the projection matrix motivates us to define the distance between  $\mathcal{U}$
and ${\mathcal{U}}^{\star}$ as follows
\begin{equation}
	\mathsf{dist}_{\mathsf{p},\vertiii{\cdot}}\big(\bm{U},{\bm{U}}^{\star}\big) \coloneqq \vertiiibig{\bm{U} \bm{U}^{\top}- {\bm{U}}^{\star} {\bm{U}}^{\star\top}},
	\label{eq:algebraic-dist}
\end{equation}
where, as usual, $\vertiii{\cdot}$ is a certain matrix norm of interest, and the subscript $\mathsf{p}$ stands for projection.

%, by using Weyl's inequality for singular values (Lemma~\ref{lemma:weyls-singular-value} with $\bm{A}=0$)
\item \emph{Geometric construction via principal angles}. Let $\sigma_{1}\geq \sigma_{2}\geq \cdots\geq \sigma_{r} \geq 0$ be the singular values of $\bm{U}^{\top}{\bm{U}}^{\star}$,
arranged in descending order. Given that $\| \bm{U}^{\top}{\bm{U}}^{\star} \| \leq \| \bm{U} \| \, \| {\bm{U}}^{\star} \|=1 $,
all the singular values
$\{\sigma_{i}\}_{i=1}^r$ fall within the interval $[0,1]$.
Therefore,
one can define the principal angles (or canonical angles) between the two subspaces of interest as
\begin{align}
	\theta_{i} \coloneqq \arccos \left(\sigma_{i}\right)\qquad\text{for all }1\leq i\leq r,
	\label{defn:principal-angles}
\end{align}
which clearly satisfy
\begin{align}
	0\leq\theta_{1}\leq \cdots\leq\theta_{r}\leq\pi/2.
	\label{eq:principal-angle-range}
\end{align}
To see why this definition makes sense, consider the simplest
		example where $r=1$.  In this case, the principal angle $\theta_{1}$ coincides with the conventionally defined angle between two unit vectors $\bm{U}$ and ${\bm{U}}^{\star}$. Armed with these angles, one might measure the distance between the subspaces
$\mathcal{U}$ and ${\mathcal{U}}^{\star}$ through the following metric
\begin{equation}
	\mathsf{dist}_{\mathsf{sin},\vertiii{\cdot}}\big(\bm{U}, {\bm{U}}^{\star} \big) \coloneqq \vertiii{\sin\bm{\Theta}},\label{eq:geometric-dist}
\end{equation}
where $\vertiii{\cdot}$ is again some matrix norm to be selected, and
\begin{equation}
	\bm{\Theta}\coloneqq {\footnotesize \left[\begin{array}{ccc}
\theta_{1}\\
 & \ddots\\
 &  & \theta_{r}
	\end{array}\right] }
	,\quad
	\sin\bm{\Theta}\coloneqq {\footnotesize \left[\begin{array}{ccc}
\sin\theta_{1}\\
 & \ddots\\
 &  & \sin\theta_{r}
	\end{array}\right] } .
	\label{eq:defn-Theta-sin-Theta}
\end{equation}
%
%Note that $\sin \theta_i = \sqrt{1-\sigma_i^2}$ and 	$\mathsf{dist}_{\mathsf{sin},\vertiii{\cdot}}\big(\bm{U}, {\bm{U}}^{\star} \big) =  \sqrt{1-\sigma_r^2}$ when the spectral norm is used and it equals to $\sqrt{r - \sigma_1^2 - \cdots - \sigma_r^2}$ when the Frobenius norm is used.
With slight abuse of notation, we can define other diagonal matrices such as $\cos \bm{\Theta}$ analogously, where $\cos(\cdot)$ is applied in an entrywise manner to the diagonal elements of $\bm{\Theta}$. Such matrices will be useful for future discussions.

\end{enumerate}

\subsection{Intimate connections between the distance metrics}
It turns out that the   metrics \eqref{defn:dist-rotation}, (\ref{eq:algebraic-dist}) and
 (\ref{eq:geometric-dist}) introduced above are tightly related, as we shall explain in this subsection.
 The proofs of all the results in this subsection are deferred to Section~\ref{sec:proof-preliminary-dist-angles}.

To begin with, we take a look at the relation between $\mathsf{dist}_{\mathsf{p},\vertiii{\cdot}}(\cdot,\cdot)$ and $\mathsf{dist}_{\mathsf{sin},\vertiii{\cdot}}(\cdot,\cdot)$,  which is perhaps best illuminated by the following lemma.

\begin{lemma}\label{lemma:connection-algebraic-geometric}
	Consider the settings of Section~\ref{sec:setting-distance-angle}. If $2r\leq n$, then the  singular values of $\bm{U}\bm{U}^{\top}-{\bm{U}}^{\star}{\bm{U}}^{\star\top}$ (including zeros)
are given by
\[
\underbrace{\sin\theta_{r},\, \sin\theta_{r},\, \sin\theta_{r-1},\, \sin\theta_{r-1},\,\cdots,\,\sin\theta_{1},\,\sin\theta_{1}}_{2r},\;\underbrace{0,\,0,\,\cdots,\,0}_{n-2r}.
\]
\end{lemma}
In a nutshell, Lemma~\ref{lemma:connection-algebraic-geometric} establishes an explicit link between (a) the difference of the projection matrices and (b) the principal angles between the two  subspaces of interest.
This lemma and its analysis unveil the following crucial equivalence relation under two of our favorite norms---the spectral norm and the Frobenius norm; in light of this, we might refer to these metrics as the $\sin{\bm{\Theta}}$ distances from time to time.

\begin{lemma}\label{prop:unitary-norm-property}
Consider the settings of Section~\ref{sec:setting-distance-angle}, and recall the definition of $\sin\bm{\Theta}$ in \eqref{eq:defn-Theta-sin-Theta}. For any $1\le r \le n$, one has
\begin{subequations}
\begin{align}
	\big\Vert \bm{U}\bm{U}^{\top}-{\bm{U}}^{\star}{\bm{U}}^{\star\top}\big\Vert
	& =\left\Vert \sin\bm{\Theta}\right\Vert  = \big\Vert \bm{U}_{\perp}^{\top}\bm{U}^{\star}\big\Vert = \big\Vert \bm{U}^{\top}\bm{U}_{\perp}^{\star}\big\Vert ; \label{eq:spectral-norm-equiv} \\
	\tfrac{1}{\sqrt{2}} \big\Vert \bm{U}\bm{U}^{\top}-{\bm{U}}^{\star}{\bm{U}}^{\star\top}\big\Vert _{\mathrm{F}}
	& =\left\Vert \sin\bm{\Theta}\right\Vert _{\mathrm{F}}
	=  \big\Vert \bm{U}_{\perp}^{\top}\bm{U}^{\star}\big\Vert_{\mathrm{F}} =   \big\Vert \bm{U}^{\top}\bm{U}_{\perp}^{\star} \big\Vert_{\mathrm{F}} .
	\label{eq:fro-norm-equiv}
\end{align}
\end{subequations}
\end{lemma}

Next, we move on to demonstrate the (near) equivalence of $\mathsf{dist}_{\vertiii{\cdot}}(\cdot,\cdot)$ and $\mathsf{dist}_{\mathsf{p},\vertiii{\cdot}}(\cdot,\cdot)$ under the above-mentioned two norms.

\begin{lemma}\label{prop:rotation-UR}
	Under the settings of Section~\ref{sec:setting-distance-angle}, for any $1 \leq r \leq n$,
	one has\footnote{It is straightforward to verify that the upper bounds on both $\min_{\bm{R}\in \mathcal{O}^{r\times r}}\big\|\bm{U}\bm{R}-\bm{U}^{\star}\big\|$ and $\min_{\bm{R}\in\mathcal{O}^{r\times r}}\left\Vert \bm{U}\bm{R}-\bm{U}^{\star}\right\Vert _{\mathrm{F}}$  are attainable when $\bm{U} = [1,0]^\top$ and $\bm{U}^\star = [0,1]^\top$.}
\begin{align*}
	\|\bm{U}\bm{U}^{\top} - \bm{U}^{\star}\bm{U}^{\star\top} \|
	&\leq
	\min_{\bm{R}\in \mathcal{O}^{r\times r}}\big\|\bm{U}\bm{R}-\bm{U}^{\star}\big\|
	%\mathsf{dist} \big(\bm{U},\bm{U}^{\star}\big)
	\leq \sqrt{2} \|\bm{U}\bm{U}^{\top} - \bm{U}^{\star}\bm{U}^{\star\top} \|; \\
	\tfrac{1}{\sqrt{2}} \|\bm{U}\bm{U}^{\top} - \bm{U}^{\star}\bm{U}^{\star\top} \|_{\mathrm{F}}
	%\mathsf{dist}_{\mathrm{F}} \big(\bm{U},\bm{U}^{\star}\big)
	&\leq
	\min_{\bm{R}\in\mathcal{O}^{r\times r}}\left\Vert \bm{U}\bm{R}-\bm{U}^{\star}\right\Vert _{\mathrm{F}}
	%\mathsf{dist}_{\mathrm{F}} \big(\bm{U},\bm{U}^{\star}\big)
	\leq \|\bm{U}\bm{U}^{\top} - \bm{U}^{\star}\bm{U}^{\star\top} \|_{\mathrm{F}}.
	%\mathsf{dist} \big(\bm{U},\bm{U}^{\star}\big).
	%\label{eq:equiv-rotation-dist-spec}
\end{align*}
%\end{subequations}
\end{lemma}
In words, $\mathsf{dist}_{\vertiii{\cdot}}(\cdot,\cdot)$ and $\mathsf{dist}_{\mathsf{p},\vertiii{\cdot}}(\cdot,\cdot)$
are equivalent up to a factor of $\sqrt{2}$, when $\vertiii{\cdot}$ is the spectral norm or the Frobenius norm.

\subsection{The distance metrics of choice in this monograph}

In conclusion, the following metrics, which are seemingly distinct at first glance,  are (nearly) equivalent in measuring the distance between two subspaces $\bm{U}$ and $\bm{U}^{\star}$:
\begin{align}
	\mathrm{1)} \quad & \vertiiibig{\bm{U}\bm{U}^{\top} - \bm{U}^{\star}\bm{U}^{\star\top}}  \notag\\
	\mathrm{2)} \quad & \vertiiibig{\sin\bm{\Theta}}  \notag\\
	\mathrm{3)} \quad & \vertiiibig{ \bm{U}_{\perp}^{\top}\bm{U}^{\star} }=\vertiiibig{ \bm{U}^{\top}\bm{U}^{\star}_{\perp} }  \notag \\
	\mathrm{4)} \quad & \min_{\bm{R}\in \mathcal{O}^{r\times r}} \vertiiibig{ \bm{U}\bm{R} - \bm{U}^{\star} } \notag
\end{align}
when $\vertiii{\cdot}$ represents either the spectral norm  or the Frobenius norm.
Viewed in this light, we shall mainly concentrate on the following metrics throughout the rest of this monograph:
\begin{subequations}
\label{eq:dist_UUstar}
\begin{align}
	\mathsf{dist}(\bm{U},\bm{U}^{\star}) &\coloneqq \min_{\bm{R}\in \mathcal{O}^{r\times r}} \big\| \bm{U}\bm{R} - \bm{U}^{\star} \big\|;
	\label{eq:dist_UUstar-spectral} \\
	%\big\| \bm{U} \bm{U}^{\top}  -  \bm{U}^{\star} \bm{U}^{\star\top} \big\|, \\
	\mathsf{dist}_{\mathrm{F}}(\bm{U},\bm{U}^{\star}) &\coloneqq  \min_{\bm{R}\in \mathcal{O}^{r\times r}} \big\| \bm{U}\bm{R} - \bm{U}^{\star} \big\|_{\mathrm{F}}.
	\label{eq:dist_UUstar-fro}
	%\big\| \bm{U} \bm{U}^{\top}  -  \bm{U}^{\star} \bm{U}^{\star\top} \big\|_{\mathrm{F}},
\end{align}
\end{subequations}

\section{Perturbation theory for eigenspaces}
\label{sec:perturbation}

Armed with the above metrics for subspace distances, we are in a position to identify key factors that affect the perturbation of eigenvectors and eigenspaces.

\subsection{Setup and notation}
\label{sec:setting-davis-kahan}

Let $\bm{M}^{\star}$ and $\bm{M}=\bm{M}^{\star}+\bm{E}$ be two $n\times n$ real symmetric
matrices. We express the eigendecomposition
of $\bm{M}^{\star}$ and $\bm{M}$ as follows
%
%\begin{subequations}
\begin{alignat}{2}
\bm{M}^{\star} & =\sum_{i=1}^{n}\lambda_{i}^{\star}\bm{u}_{i}^{\star}\bm{u}_{i}^{\star\top} & & =\left[\begin{array}{cc}
\bm{U}^{\star} & \bm{U}_{\perp}^{\star}\end{array}\right]\left[\begin{array}{cc}
\bm{\Lambda}^{\star} & \bm{0}\\
\bm{0} & \bm{\Lambda}_{\perp}^{\star}
\end{array}\right]\left[\begin{array}{c}
\bm{U}^{\star\top}\\
\bm{U}_{\perp}^{\star\top}
\end{array}\right];  \label{eq:M-e-decompose} \\
\bm{M} & =\sum_{i=1}^{n}\lambda_{i}\bm{u}_{i}\bm{u}_{i}^{\top} & & =\left[\begin{array}{cc}
\bm{U} & \bm{U}_{\perp}\end{array}\right]\left[\begin{array}{cc}
\bm{\Lambda} & \bm{0}\\
\bm{0} & \bm{\Lambda}_{\perp}
\end{array}\right]\left[\begin{array}{c}
\bm{U}^{\top}\\
\bm{U}_{\perp}^{\top}
\end{array}\right]. \label{eq:Mstar-e-decompose}
\end{alignat}
%\end{subequations}
%
Here, $\{\lambda_{i}\}$ (resp.~$\{\lambda_{i}^{\star}\}$)  denote
the eigenvalues of $\bm{M}$ (resp.~$\bm{M}^{\star}$), and $\bm{u}_{i}$ (resp.~$\bm{u}_i^{\star}$)
stands for the eigenvector associated with the eigenvalue $\lambda_{i}$ (resp.~$\lambda_i^{\star}$).
Additionally, we take
\begin{align*}
	  \bm{U} &\coloneqq[\bm{u}_{1},\cdots,\bm{u}_{r}]\in\mathbb{R}^{n\times r}, \qquad&\bm{U}_{\perp} &\coloneqq[\bm{u}_{r+1},\cdots,\bm{u}_{n}]\in\mathbb{R}^{n\times (n-r)},\\
	  \bm{\Lambda}&\coloneqq\mathsf{diag}\big([\lambda_{1},\cdots,\lambda_{r}]\big),\qquad &\bm{\Lambda}_{\perp}  &\coloneqq \mathsf{diag}\big([\lambda_{r+1},\cdots,\lambda_{n}]\big).
\end{align*}
The matrices $\bm{U}^{\star}$, $\bm{U}_{\perp}^{\star}$, $\bm{\Lambda}^{\star}$, and $\bm{\Lambda}_{\perp}^{\star}$ are defined analogously.

\subsection{A warm-up example}

In general, the eigenvector/eigenspace of a real symmetric matrix might change drastically even upon a small perturbation.
To understand this, consider the following toy example borrowed from~\citet{hsu2016notes}:
\[
\bm{M}^{\star}=\left[\begin{array}{cc}
1+\epsilon & 0\\
0 & 1-\epsilon
\end{array}\right],
~~
\bm{E}=\left[\begin{array}{cc}
-\epsilon & \epsilon\\
\epsilon & \epsilon
\end{array}\right],~~
\bm{M}=\left[\begin{array}{cc}
1 & \epsilon\\
\epsilon & 1
\end{array}\right],
\]
where $0<\epsilon<1$ can be arbitrarily small. It is straightforward
to check that the leading eigenvectors of $\bm{M}^{\star}$ and $\bm{M}$ are given respectively by
\[
\bm{u}_{1}^{\star}=\left[\begin{array}{c}
1\\
0
\end{array}\right],
\qquad\text{and}\qquad
\bm{u}_{1}=\frac{1}{\sqrt{2}} \left[\begin{array}{c}
1 \\
1
\end{array}\right] .
\]
Consequently, we have
\begin{align}
	\big\|\bm{u}_{1}\bm{u}_{1}^{\top} - \bm{u}_{1}^{\star}\bm{u}_{1}^{\star\top}\big\| = \frac{1}{\sqrt{2}}, 
	\quad \text{and} \quad
	%\mathsf{dist}_{\mathrm{F}}(\bm{u}_{1}, \bm{u}_{1}^{\star})
	\big\|\bm{u}_{1}\bm{u}_{1}^{\top} - \bm{u}_{1}^{\star}\bm{u}_{1}^{\star\top}\big\|_{\mathrm{F}}= 1,
\end{align}
which are both quite large regardless of the size of $\epsilon$ or the size of the perturbation $\|\bm{E}\|$.

On closer inspection, this ``pathological'' behavior comes up due to the
fact that perturbation size $\epsilon$  is comparable to the
eigengap of $\bm{M}^{\star}$ (namely, $\lambda_{1}(\bm{M}^{\star})-\lambda_{2}(\bm{M}^{\star})=2\epsilon$).
This hints at the important role played by the eigengap in influencing eigenspace perturbation.

\subsection{The Davis-Kahan sin$\bm{\Theta}$ theorem}

At the core of classical eigenspace perturbation theory lies the landmark result of \citet{davis1970rotation},
which delivers powerful eigenspace perturbation bounds in terms of the size of the perturbation matrix as well as the associated eigengap.
Here and throughout, for any symmetric matrix $\bm{A}$, we denote by $\mathsf{eigenvalues}(\bm{A})$ the set of eigenvalues of $\bm{A}$.

\begin{theorem}[Davis-Kahan's sin$\bm{\Theta}$ theorem]
\label{thm:davis-kahan}
Consider the settings in Section~\ref{sec:setting-davis-kahan}. Assume that
\begin{subequations}
\label{eq:assumption-eigenvalues-DK}
\begin{align}
	\mathsf{eigenvalues}(\bm{\Lambda}^{\star}) &\subseteq [\alpha,\beta] ,\\
	\mathsf{eigenvalues}(\bm{\Lambda}_{\perp}) &\subseteq (-\infty, \alpha-\Delta] \cup [\beta+\Delta, \infty)
\end{align}
\end{subequations}
for some quantities $\alpha,\beta\in \mathbb{R}$ and eigengap $\Delta>0$. Then one has
\begin{subequations}
\label{eq:davis-kahan-conclusion-general}
\begin{align}
\mathsf{dist}\big(\bm{U},\bm{U}^{\star}\big) & \leq\sqrt{2}\|\sin\bm{\Theta}\|\leq\frac{\sqrt{2}\big\|\bm{E}\bm{U}^{\star}\big\|}{\Delta}\leq\frac{\sqrt{2}\|\bm{E}\|}{\Delta};\\
\mathsf{dist}_{\mathrm{F}}\big(\bm{U},\bm{U}^{\star}\big) & \leq\sqrt{2}\|\sin\bm{\Theta}\|_{\mathrm{F}}\leq\frac{\sqrt{2}\big\|\bm{E}\bm{U}^{\star}\big\|_{\mathrm{F}}}{\Delta}\leq\frac{\sqrt{2r}\|\bm{E}\|}{\Delta}.
\end{align}
\end{subequations}
This conclusion remains valid if Assumption~\eqref{eq:assumption-eigenvalues-DK} is replaced by
\begin{subequations}
\label{eq:assumption-eigenvalues-DK-reverse}
\begin{align}
	\mathsf{eigenvalues}(\bm{\Lambda}^{\star}) &\subseteq (-\infty, \alpha-\Delta] \cup [\beta+\Delta, \infty) ; \\
	\mathsf{eigenvalues}(\bm{\Lambda}_{\perp}) &\subseteq [\alpha,\beta].
\end{align}
\end{subequations}
\end{theorem}

\begin{remark}\label{rmk:generalized_DK}
In fact, Theorem~\ref{thm:davis-kahan} can be generalized to accommodate any unitarily invariant norm $\vertiii{\cdot}$, in the sense that
\begin{align}
\vertiii{\sin\bm{\Theta}}
\leq\frac{\vertiii{\bm{E}\bm{U}^{\star}}}{\Delta }.
%\leq \frac{\vertiii{\bm{E}}}{ \Delta }.
\end{align}

\end{remark}

The proof of Theorem~\ref{thm:davis-kahan} and Remark~\ref{rmk:generalized_DK} is quite elementary and can be found in Section~\ref{subsection2.2.4}.

\begin{remark}
As we shall demonstrate in Chapter \ref{cha:Linf-theory}, the above bounds that involve $\|\bm{E}\bm{U}^{\star}\|$ and $\|\bm{E}\bm{U}^{\star}\|_{\mathrm{F}}$ are particularly useful when $\bm{E}$ exhibits special structure (e.g., row sparsity or column sparsity).
\end{remark}
Theorem~\ref{thm:davis-kahan} is commonly referred to as the Davis-Kahan sin$\bm{\Theta}$ theorem, given that it concerns the sin${\bm\Theta}$ distance between subspaces.
Both bounds scale linearly with the perturbation size, and are inversely proportional to the eigengap $\Delta$. Informally, if we view $\|\bm{E}\|$ as the noise size and interpret the eigengap as the ``signal strength'' (which dictates how easy it is to distinguish the $r$ eigenvalues of interest from the remaining spectrum), then Theorem~\ref{thm:davis-kahan}
asserts that the eigenspace perturbation degrades gracefully as the signal-to-noise-ratio decreases.

The careful reader might notice that Theorem~\ref{thm:davis-kahan}  stays silent on the allowable size $\|\bm{E}\|$ of the perturbation.  Note, however, that a restriction on $\|\bm{E}\|$ is somewhat hidden in Assumptions (\ref{eq:assumption-eigenvalues-DK}) and (\ref{eq:assumption-eigenvalues-DK-reverse}).
When the eigenvalues in $\bm{\Lambda}^{\star}$ (resp.~$\bm{\Lambda}$) and $\bm{\Lambda}^{\star}_{\perp}$ (resp.~$\bm{\Lambda}_{\perp}$) are suitably ordered,
it is oftentimes more convenient to work with the following corollary, which makes apparent the constraint on the  size $\|\bm{E}\|$ with regard to the eigengap of $\bm{M}^{\star}$.
\begin{corollary}
\label{cor:davis-kahan-conclusion-corollary}
	Consider the settings in Section~\ref{sec:setting-davis-kahan}. Suppose that $|\lambda_1^{\star} | \geq |\lambda_2^{\star} | \geq \cdots \geq  |\lambda_r^{\star} | >  |\lambda_{r+1}^{\star} | \geq \cdots \geq |\lambda_n^{\star}|$ and $|\lambda_1 | \geq |\lambda_2 | \geq \cdots \geq |\lambda_n|$ (i.e., the eigenvalues are sorted by their magnitudes).
If $\|\bm{E}\|< (1-1/\sqrt{2}) (|\lambda_{r}^{\star}|-|\lambda_{r+1}^{\star}|)$, then
\begin{subequations}
\label{eq:davis-kahan-conclusion-corollary}
\begin{align}
\mathsf{dist}\big(\bm{U},\bm{U}^{\star}\big) & \leq\sqrt{2}\|\sin\bm{\Theta}\|\leq\frac{2 \big\|\bm{E}\bm{U}^{\star}\big\|}{|\lambda_{r}^{\star}|-|\lambda_{r+1}^{\star}|}\leq\frac{2\|\bm{E}\|}{|\lambda_{r}^{\star}|-|\lambda_{r+1}^{\star}|};\\
	\mathsf{dist}_{\mathrm{F}}\big(\bm{U},\bm{U}^{\star}\big) & \leq\sqrt{2}\|\sin\bm{\Theta}\|_{\mathrm{F}}\leq\frac{2 \big\|\bm{E}\bm{U}^{\star}\big\|_{\mathrm{F}}}{|\lambda_{r}^{\star}|-|\lambda_{r+1}^{\star}|}\leq\frac{2\sqrt{r}\|\bm{E}\|}{|\lambda_{r}^{\star}|-|\lambda_{r+1}^{\star}|}.
\end{align}
\end{subequations}
\end{corollary}
The proof of Corollary~\ref{cor:davis-kahan-conclusion-corollary} is also given in Section~\ref{subsection2.2.4}.

\subsection{Proof of the Davis-Kahan \texorpdfstring{sin$\bm{\Theta}$}{TEXT} theorem}\label{subsection2.2.4}

\paragraph{Proof of Theorem \ref{thm:davis-kahan}.}

The proof proceeds by controlling the distance metric $\vertiiibig{\bm{U}_{\perp}^{\top}\bm{U}^{\star}}$, where $\vertiii{\cdot}$ denotes a unitarily invariant norm.

We start by proving the theorem under Assumption~\eqref{eq:assumption-eigenvalues-DK}, and
	claim that it suffices to consider the case where
	\begin{align}
		\alpha=-\beta \leq 0.
		\label{eq:alpha-beta-equal}
	\end{align}
	%for some $\beta\geq0$.
	%
	In fact, if this condition is violated, then one can employ a ``centering'' trick by enforcing global offset to $\bm{M}^{\star}$ and $\bm{M}$  as follows
	\begin{align*}
		\bm{M}^{\star}_{\mathsf{c}} = \bm{M}^{\star} - \frac{\alpha + \beta}{2} \bm{I}_{n},
		\quad \text{and} \quad
		\bm{M}_{\mathsf{c}} = \bm{M} - \frac{\alpha + \beta}{2} \bm{I}_{n}.
	\end{align*}
	It is straightforwardly seen that (a) $\bm{M}^{\star}_{\mathsf{c}}$ (resp.~$\bm{M}_{\mathsf{c}}$) and $\bm{M}^{\star}$ (resp.~$\bm{M}$) share the same eigenvectors; (b) the eigenvalues of $\bm{M}^{\star}_{\mathsf{c}}$ (resp.~$\bm{M}_{\mathsf{c}}$) associated with $\bm{U}^{\star}$ (resp.~$\bm{U}_{\perp}$) reside within $[-\gamma,\gamma]$ (resp.~$(-\infty, -\gamma - \Delta] \cup [\gamma+\Delta,\infty)$), where $\gamma = \frac{\beta-\alpha}{2} \geq 0$. Consequently, this reduces to a scenario that resembles \eqref{eq:alpha-beta-equal}.   In addition, we isolate two
 immediate consequences of  Assumptions~\eqref{eq:assumption-eigenvalues-DK} and \eqref{eq:alpha-beta-equal} that prove useful:
\begin{equation}
	\|\bm{\Lambda}^{\star}\|\leq\beta, 
	\qquad\text{and}\qquad
	\sigma_{\min}(\bm{\Lambda}_{\perp}) \geq \beta+\Delta,
	\label{eq:key-estimate}
\end{equation}
where we recall that $\sigma_{\min}(\bm{\Lambda}_{\perp})$ is the minimal singular value of $\bm{\Lambda}_{\perp}$.

Armed with the above spectral conditions, we are prepared to study $\bm{U}_{\perp}^{\top}\bm{U}^{\star}$.
This is controlled through the following identity (obtained by the definition of eigenvectors): 
\begin{align}
	\bm{U}_{\perp}^{\top} (\bM - \bM^\star) \bU^\star = \bm{\Lambda}_{\perp}\bm{U}_{\perp}^{\top}\bm{U}^{\star}-\bm{U}_{\perp}^{\top}\bm{U}^{\star}\bm{\Lambda}^{\star}, \label{eq:sylvester}
\end{align}
Let $\bR \coloneqq \left(\bm{M}-\bm{M}^{\star}\right)\bm{U}^{\star} = \bm{E} \bm{U}^{\star}$.
The triangle inequality then tells us that
\begin{align}
\vertiiibig{\bm{U}_{\perp}^{\top}\bm{R}} & \geq\vertiiibig{\bm{\Lambda}_{\perp}\bm{U}_{\perp}^{\top}\bm{U}^{\star}}-\vertiiibig{\bm{U}_{\perp}^{\top}\bm{U}^{\star}\bm{\Lambda}^{\star}} \notag\\
	& \geq  \sigma_{\min}(\bm{\Lambda}_{\perp})  \vertiiibig{\bm{U}_{\perp}^{\top}\bm{U}^{\star}}- \|\bm{\Lambda}^{\star}\| \cdot \vertiiibig{\bm{U}_{\perp}^{\top}\bm{U}^{\star}} \notag\\
 %& \geq(\beta+\Delta)\vertiiibig{\bm{U}_{\perp}^{\top}\bm{U}^{\star}}-\beta\vertiiibig{\bm{U}_{\perp}^{\top}\bm{U}^{\star}} \notag\\
	& \geq(\beta+\Delta-\beta)\vertiiibig{\bm{U}_{\perp}^{\top}\bm{U}^{\star}}= \Delta\,\vertiiibig{\bm{U}_{\perp}^{\top}\bm{U}^{\star}}.
	\label{eq:Uperp-R-LB-Delta1}
\end{align}
Here, the middle line follows from Lemma~\ref{prop:unitary_norm_relation} in Section~\ref{subsec:Matrix-analysis}, whereas the last inequality arises from the properties (\ref{eq:key-estimate}).
As a consequence,
\[
\vertiiibig{\bm{U}_{\perp}^{\top}\bm{U}^{\star}}\leq\frac{\vertiiibig{\bm{U}_{\perp}^{\top}\bm{R}}}{\Delta}\leq\frac{\vertiii{\bm{R}}}{\Delta}=\frac{\vertiiibig{\bm{E}\bm{U}^{\star}}}{\Delta},
\]
where the second inequality follows again from Lemma~\ref{prop:unitary_norm_relation} and $\|\bm{U}_{\perp}\|=1$.
%Identifying $\vertiiiplain{\bm{U}_{\perp}^{\top}\bm{U}^{\star}}$ with $\vertiii{\sin\bm{\Theta}}$ concludes the proof.
%
When $\vertiii{\cdot}$ is either the spectral norm or the Frobenius norm, combining the preceding inequality with Lemmas \ref{prop:unitary-norm-property}-\ref{prop:rotation-UR} and the facts $\|\bm{U}^{\star}\|=1$ and $\|\bm{U}^{\star}\|_{\mathrm{F}}=\sqrt{r}$ immediately establishes the theorem for this case.

Next, we turn to the scenario where
 Assumption (\ref{eq:assumption-eigenvalues-DK-reverse}) is in effect; it can be analyzed in a similar manner and hence we remark only on the difference.
Assuming \eqref{eq:alpha-beta-equal} holds without loss of generality, we have
\begin{equation}
	\|\bm{\Lambda}_{\perp}\|\leq\beta,\qquad\text{and}\qquad
	\sigma_{\min}(\bm{\Lambda}^{\star}) \geq \beta+\Delta.
	\label{eq:key-estimate-case2}
\end{equation}
Applying the triangle inequality  to (\ref{eq:sylvester}) in a different way yields
\begin{align*}
\vertiiibig{\bm{U}_{\perp}^{\top}\bm{R}} & \geq\vertiiibig{\bm{U}_{\perp}^{\top}\bm{U}^{\star}\bm{\Lambda}^{\star}}-\vertiiibig{\bm{\Lambda}_{\perp}\bm{U}_{\perp}^{\top}\bm{U}^{\star}}\\
 & \geq\sigma_{\min}(\bm{\Lambda}^{\star})\vertiiibig{\bm{U}_{\perp}^{\top}\bm{U}^{\star}}-\|\bm{\Lambda}_{\perp}\| \cdot \vertiiibig{\bm{U}_{\perp}^{\top}\bm{U}^{\star}}\\
 & \geq(\beta+\Delta-\beta)\vertiiibig{\bm{U}_{\perp}^{\top}\bm{U}^{\star}}=\Delta\vertiiibig{\bm{U}_{\perp}^{\top}\bm{U}^{\star}},
\end{align*}
a conclusion that coincides with \eqref{eq:Uperp-R-LB-Delta1}. The rest of the proof is the same as the one in the previous case.

Before concluding, we remark that $\vertiiibig{\bm{U}_{\perp}^{\top}\bm{U}^{\star}}=\vertiiibig{\sin\bm{\Theta}}$
holds for any unitarily invariant norm $\vertiii{\cdot}$; see \citet[Lemma 2.1]{li1998relative}. This together with the above analysis leads to Remark~\ref{rmk:generalized_DK}.

\paragraph{Proof of Corollary \ref{cor:davis-kahan-conclusion-corollary}.}
We first examine the spectral ranges of $\bm{\Lambda}_{\perp}$ and $\bm{\Lambda}^{\star}$.
Let $\lambda_i(\bm{M}^{\star})$ (resp.~$\lambda_i(\bm{M})$) be the $i$-th largest eigenvalue of $\bm{M}^{\star}$ (resp.~$\bm{M}$), sorted by their values (as opposed to their magnitudes). Then Weyl's inequality (cf.~Lemma \ref{lemma:weyl} in Section~\ref{subsec:Matrix-analysis}) asserts that
\[
	|\lambda_{i}(\bm{M})-\lambda_{i}(\bm{M}^{\star})|  \leq\|\bm{E}\|, \qquad 1\leq i\leq n.
	%\leq  \Big( 1 - \frac{1}{\sqrt{2}} \Big) |\lambda_{r}^{\star}|
\]
Suppose that $\bm{M}^{\star}$ has $r_1$ positive (resp.~$r_2=r-r_1$ negative) eigenvalues whose magnitudes exceed $|\lambda_{r+1}^{\star}|$.
Then for any $i$ obeying $1\leq i\leq r_{1}$ or $i>n-r_{2}$, the triangle inequality gives
\begin{align*}
\big|\lambda_{i}(\bm{M})\big| & \geq\big|\lambda_{i}(\bm{M}^{\star})\big|-\|\bm{E}\|>|\lambda_{r}^{\star}|-\big(1-1/\sqrt{2}\big)(|\lambda_{r}^{\star}|-|\lambda_{r+1}^{\star}|)\\
 & > |\lambda_{r+1}^{\star}|+\big(1-1/\sqrt{2}\big)(|\lambda_{r}^{\star}|-|\lambda_{r+1}^{\star}|)\geq|\lambda_{r+1}^{\star}|+\|\bm{E}\|,
\end{align*}
where the last inequality arises from our assumption on $\|\bm{E}\|$.
On the contrary, if $r_1 < i \leq n-r_2$, then one has
\begin{align*}
	\big|\lambda_{i}(\bm{M})\big| & \leq  \big|\lambda_{r+1}^{\star}\big| + \|\bm{E}\| .
	%\leq \big|\lambda_{r+1}^{\star}\big| + \big(1-{1}/{\sqrt{2}}\big) \big( |\lambda_{r}^{\star}| - \big|\lambda_{r+1}^{\star}\big| \big) \\
	%& < |\lambda_{r}^{\star}|-\big(1-{1}/{\sqrt{2}}\big) ( |\lambda_{r}^{\star}| - |\lambda_{r+1}^{\star}|).
\end{align*}
As a consequence, there are exactly $r$ (resp.~$n-r$) eigenvalues of $\bm{M}$ whose magnitudes exceed (resp.~lie below) $|\lambda_{r+1}^{\star}|+\|\bm{E}\|$.

The above observation together with the ordering $|\lambda_1 | \geq |\lambda_2 | \geq \cdots \geq |\lambda_n|$ implies
\[
\mathsf{eigenvalues}(\bm{\Lambda}_{\perp}) \subseteq \big[-|\lambda_{r+1}^{\star}|-\|\bm{E}\|, \, |\lambda_{r+1}^{\star}|+ \|\bm{E}\| \big].
\]
In addition, the assumption that $|\lambda_1^{\star} | \geq |\lambda_2^{\star} | \geq \cdots \geq |\lambda_n^{\star}|$ tells us that
\[
\mathsf{eigenvalues}(\bm{\Lambda}^{\star}) \subseteq \big(-\infty, -|\lambda_{r}^{\star}| \big] \cup \big[|\lambda_{r}^{\star}|, \infty\big).
\]
Taking $\beta = -\alpha= |\lambda_{r+1}^{\star}|+ \|\bm{E}\|$ and $\Delta=|\lambda_{r}^\star| -|\lambda_{r+1}^{\star}| - \|\bm{E}\| > (|\lambda_{r}^\star|-|\lambda_{r+1}^{\star}|) / \sqrt{2}$,
we can invoke Theorem~\ref{thm:davis-kahan} under Assumption~(\ref{eq:assumption-eigenvalues-DK-reverse}) to establish the advertised results.

\section{Perturbation theory for singular subspaces}

There is no shortage of scenarios where the data matrices under consideration are asymmetric or  rectangular. In these cases, one is often asked to study singular value decomposition (SVD) rather than eigendecomposition. Fortunately, the eigenspace perturbation theory can be naturally extended to accommodate perturbation of singular subspaces.

\subsection{Setup and notation}
\label{sec:setup-SVD}

Let $\bm{M}^{\star}$ and $\bm{M}=\bm{M}^{\star}+\bm{E}$ be two matrices
in $\mathbb{R}^{n_1 \times n_2}$ (without loss of generality, we assume $n_1\leq n_2$), whose SVDs
are given respectively by
\begin{alignat}{2}
\bm{M}^{\star} & =\sum_{i=1}^{n_1}\sigma_{i}^{\star}\bm{u}_{i}^{\star}\bm{v}_{i}^{\star\top} & & =\left[\begin{array}{cc}
\bm{U}^{\star} & \bm{U}_{\perp}^{\star}\end{array}\right]\left[\begin{array}{ccc}
\bm{\Sigma}^{\star} & \bm{0} & \bm{0}\\
\bm{0} & \bm{\Sigma}_{\perp}^{\star} & \bm{0}
\end{array}\right]\left[\begin{array}{c}
\bm{V}^{\star\top}\\
\bm{V}_{\perp}^{\star\top}
\end{array}\right];  \label{eq:Mstar-SVD} \\
\bm{M} & =\sum_{i=1}^{n_1}\sigma_{i}\bm{u}_{i}\bm{v}_{i}^{\top} & & =\left[\begin{array}{cc}
\bm{U} & \bm{U}_{\perp}\end{array}\right]\left[\begin{array}{ccc}
\bm{\Sigma} & \bm{0} & \bm{0}\\
\bm{0} & \bm{\Sigma}_{\perp} & \bm{0}
\end{array}\right]\left[\begin{array}{c}
\bm{V}^{\top}\\
\bm{V}_{\perp}^{\top}
\end{array}\right]. \label{eq:M-SVD}
\end{alignat}
Here, $\sigma_{1}\geq\cdots\geq\sigma_{n_1}$ (resp.~$\sigma_{1}^{\star}\geq\cdots\geq\sigma_{n_1}^{\star}$)
stand for the singular values of $\bm{M}$ (resp.~$\bm{M}^{\star}$)
arranged in descending order, $\bm{u}_{i}$ (resp\@.~$\bm{u}_{i}^{\star}$)
denotes the left singular vector associated with the singular value
$\sigma_{i}$ (resp.~$\sigma_{i}^{\star}$), and $\bm{v}_{i}$ (resp\@.~$\bm{v}_{i}^{\star}$)
represents the right singular vector associated with $\sigma_{i}$ (resp.~$\sigma_{i}^{\star}$).
In addition, we denote
\begin{align*}
	 \bm{\Sigma}  &\coloneqq\mathsf{diag}\big([\sigma_{1},\cdots,\sigma_{r}]\big),\quad &\bm{\Sigma}_{\perp} &\coloneqq\mathsf{diag}\big([\sigma_{r+1},\cdots,\sigma_{n_1}]\big),\\
	 \bm{U}  &\coloneqq[\bm{u}_{1},\cdots,\bm{u}_{r}]\in \mathbb{R}^{n_1\times r},\qquad &\bm{U}_{\perp} &\coloneqq[\bm{u}_{r+1},\cdots,\bm{u}_{n_1}]\in \mathbb{R}^{n_1\times (n_1-r)},\\
	 \bm{V}  &\coloneqq[\bm{v}_{1},\cdots,\bm{v}_{r}]\in \mathbb{R}^{n_2\times r},\qquad &\bm{V}_{\perp} &\coloneqq[\bm{v}_{r+1},\cdots,\bm{v}_{n_2}] \in \mathbb{R}^{n_2\times (n_2-r)}.
\end{align*}
The matrices $\bm{\Sigma}^{\star},\bm{\Sigma}_{\perp}^{\star},\bm{U}^{\star},\bm{U}_{\perp}^{\star},\bm{V}^{\star},\bm{V}_{\perp}^{\star}$
are defined analogously.

\subsection{Wedin's sin$\bm\Theta$ theorem}

%Per-{\AA}ke
\citet{wedin1972perturbation} developed a perturbation bound for singular subspaces that parallels the Davis-Kahan sin$\bm{\Theta}$ theorem for eigenspaces. In what follows, we present a version that is convenient for subsequent discussions in this monograph.
\begin{theorem}[Wedin's sin$\bm{\Theta}$ theorem]
\label{thm:wedin}
Consider the settings in Section~\ref{sec:setup-SVD}. If $\|\bm{E}\|<\sigma_{r}^{\star}-\sigma_{r+1}^{\star}$, then one has
\begin{align*}
\max\left\{ \mathsf{dist}\big(\bm{U},\bm{U}^{\star}\big),\mathsf{dist}\big(\bm{V},\bm{V}^{\star}\big)\right\}
	& \leq\frac{ \sqrt{2} \max\big\{ \|\bm{E}^{\top}\bm{U}^{\star}\|,\|\bm{E}\bm{V}^{\star}\|\big\} }{\sigma_{r}^{\star}-\sigma_{r+1}^{\star}-\|\bm{E}\|};\\
\max\left\{ \mathsf{dist}_{\mathrm{F}}\big(\bm{U},\bm{U}^{\star}\big),\mathsf{dist}_{\mathrm{F}}\big(\bm{V},\bm{V}^{\star}\big)\right\}
	& \leq\frac{\sqrt{2}\max\big\{ \|\bm{E}^{\top}\bm{U}^{\star}\|_{\mathrm{F}},\|\bm{E}\bm{V}^{\star}\|_{\mathrm{F}}\big\} }{\sigma_{r}^{\star}-\sigma_{r+1}^{\star}-\|\bm{E}\|}.
\end{align*}
\end{theorem}

This theorem simultaneously controls the perturbation of left and right singular subspaces. As a worthy note, both the interaction between $\bm{E}$ and $\bm{U}^{\star}$, and that between $\bm{E}$ and $\bm{V}^{\star}$, come into play in determining the perturbation bounds.
In particular, if $\|\bm{E}\|< (1-1/\sqrt{2})(\sigma_{r}^{\star}-\sigma_{r+1}^{\star})$, then one can apply Lemma~\ref{prop:unitary_norm_relation} in Section~\ref{subsec:Matrix-analysis} to obtain
\begin{subequations}
\label{eq:wedin-simpler-friendly}
\begin{align}
\max\left\{ \mathsf{dist}\big(\bm{U},\bm{U}^{\star}\big),\mathsf{dist}\big(\bm{V},\bm{V}^{\star}\big)\right\}
	& \leq\frac{ 2 \|\bm{E}\| }{\sigma_{r}^{\star}-\sigma_{r+1}^{\star}}, \\
\max\left\{ \mathsf{dist}_{\mathrm{F}}\big(\bm{U},\bm{U}^{\star}\big),\mathsf{dist}_{\mathrm{F}}\big(\bm{V},\bm{V}^{\star}\big)\right\}
	& \leq \frac{ 2\sqrt{r}\|\bm{E}\| }{\sigma_{r}^{\star}-\sigma_{r+1}^{\star}},
\end{align}
\end{subequations}
akin to the eigenspace perturbation bounds \eqref{eq:davis-kahan-conclusion-corollary}.

\subsection{Proof of the Wedin \texorpdfstring{sin$\bm{\Theta}$}{TEXT} theorem}
We now present a proof of the Wedin theorem.
Similar to the proof of the Davis-Kahan theorem, we start by bounding $\vertiiibig{\bm{U}_{\perp}^{\top}\bm{U}^{\star}}$,
where $\vertiii{\cdot}$ stands for any unitarily invariant
norm. To this end, it is seen that
\begin{align}
\bm{U}_{\perp}^{\top}\bm{U}^{\star}
	& =\bm{U}_{\perp}^{\top}\big(\bm{U}^{\star}\bm{\Sigma}^{\star}\bm{V}^{\star\top}\big)\bm{V}^{\star}\bm{\Sigma}^{\star-1}\nonumber \\
	& =\bm{U}_{\perp}^{\top}\left(\bm{M}-\bm{E}-\bm{U}_{\perp}^{\star}\bm{\Sigma}_{\perp}^{\star}\bm{V}_{\perp}^{\star\top}\right)\bm{V}^{\star}\bm{\Sigma}^{\star-1}\nonumber \\
 & =\bm{U}_{\perp}^{\top}\left(\bm{U}\bm{\Sigma}\bm{V}^{\top}+\bm{U}_{\perp}\bm{\Sigma}_{\perp}\bm{V}_{\perp}^{\top}-\bm{E}-\bm{U}_{\perp}^{\star}\bm{\Sigma}_{\perp}^{\star}\bm{V}_{\perp}^{\star\top}\right)\bm{V}^{\star}\bm{\Sigma}^{\star-1}\nonumber \\
 & =\bm{\Sigma}_{\perp}\bm{V}_{\perp}^{\top}\bm{V}^{\star}\bm{\Sigma}^{\star-1}-\bm{U}_{\perp}^{\top}\bm{E}\bm{V}^{\star}\bm{\Sigma}^{\star-1} .
	\label{eq:wedin-identity}
\end{align}
Here, the first identity is valid as long as $\bm{\Sigma}^{\star}$ is invertible (which is guaranteed since $\sigma_{\min}(\bm{\Sigma}^{\star})=\sigma_r^{\star} > \sigma_{r+1}^{\star}+ \|\bm{E}\|>0$ under our assumption),
the second line follows from the identities $\bm{M}-\bm{E}=\bm{M}^{\star}$ and $\bm{M}^{\star}=\bm{U}^{\star}\bm{\Sigma}^{\star}\bm{V}^{\star\top}+\bm{U}_{\perp}^{\star}\bm{\Sigma}_{\perp}^{\star}\bm{V}_{\perp}^{\star\top}$, the third line holds since $\bm{M}=\bm{U}\bm{\Sigma}\bm{V}^{\top}+\bm{U}_{\perp}\bm{\Sigma}_{\perp}\bm{V}_{\perp}^{\top}$,
whereas the last identity exploits the property
\[
\bm{U}_{\perp}^{\top}\bm{U}=\bm{0},\qquad \text{and}\qquad \bm{V}_{\perp}^{\star\top}\bm{V}^{\star}=\bm{0}.
\]
Applying the triangle inequality and Lemma~\ref{prop:unitary_norm_relation} in Section~\ref{subsec:Matrix-analysis} to the identity (\ref{eq:wedin-identity})
yields
\begin{align}
\vertiiibig{\bm{U}_{\perp}^{\top}\bm{U}^{\star}}
	& \leq \|\bm{\Sigma}_{\perp}\| \cdot \vertiiibig{\bm{V}_{\perp}^{\top}\bm{V}^{\star}} \cdot \|\bm{\Sigma}^{\star-1}\|
	 +  \| \bm{U}_{\perp}^{\top}\|\cdot \vertiiibig{\bm{E}\bm{V}^{\star}} \cdot \|\bm{\Sigma}^{\star-1}\| \notag\\
	 & =  \sigma_{r+1} \cdot \vertiiibig{\bm{V}_{\perp}^{\top}\bm{V}^{\star}} \cdot \frac{1}{\sigma_r^{\star}}
	 + \vertiiibig{\bm{E}\bm{V}^{\star}} \cdot \frac{1}{\sigma_r^{\star}}
	   \notag\\
 & \leq\frac{\sigma_{r+1}^{\star}+\|\bm{E}\|}{\sigma_{r}^{\star}}\vertiiibig{\bm{V}_{\perp}^{\top}\bm{V}^{\star}}+\frac{\vertiii{\bm{E}\bm{V}^{\star}}}{\sigma_{r}^{\star}}.
	\label{eq:Uperp-Ustar-UB}
\end{align}
Here, the second line uses the properties $\|\bm{\Sigma}^{\star-1}\|=1/\sigma_{r}^{\star}$ and $\|\bm{\Sigma}_{\perp}\|= \sigma_{r+1}$, while
the last inequality follows from Weyl's inequality $\sigma_{r+1}\leq \sigma_{r+1}^{\star}+\|\bm{E}\|$ (cf.~Lemma~\ref{lemma:weyls-singular-value} in Section~\ref{subsec:Matrix-analysis}). Repeating the same argument yields
\begin{align}
\vertiiibig{\bm{V}_{\perp}^{\top}\bm{V}^{\star}}\leq\frac{\vertiiibig{\bm{E}^{\top}\bm{U}^{\star}}}{\sigma_{r}^{\star}}
	+\frac{\sigma_{r+1}^{\star}+\|\bm{E}\|}{\sigma_{r}^{\star}}\vertiiibig{\bm{U}_{\perp}^{\top}\bm{U}^{\star}}.
	\label{eq:Vperp-Vstar-UB}
\end{align}

To finish up, combine the inequalities \eqref{eq:Uperp-Ustar-UB} and \eqref{eq:Vperp-Vstar-UB} to obtain
\begin{align*}
 & \max\big\{ \vertiiibig{\bm{U}_{\perp}^{\top}\bm{U}^{\star}},\vertiiibig{\bm{V}_{\perp}^{\top}\bm{V}^{\star}}\big\}
  \leq \frac{  \max\big\{  \vertiiibig{\bm{E}^{\top}\bm{U}^{\star}},\vertiiibig{\bm{E}\bm{V}^{\star}}  \big\} }{\sigma_r^{\star}}  \\
	&	\qquad\qquad\qquad\qquad
	+ \frac{\sigma_{r+1}^{\star}+\|\bm{E}\|}{\sigma_{r}^{\star}} \max\big\{ \vertiiibig{\bm{U}_{\perp}^{\top}\bm{U}^{\star}},\vertiiibig{\bm{V}_{\perp}^{\top}\bm{V}^{\star}}\big\} .
\end{align*}
When $\|\bm{E}\|<\sigma_r^{\star} - \sigma_{r+1}^{\star}$, we can rearrange terms to arrive at
\begin{align*}
 & \max\big\{ \vertiiibig{\bm{U}_{\perp}^{\top}\bm{U}^{\star}},\vertiiibig{\bm{V}_{\perp}^{\top}\bm{V}^{\star}}\big\}
	\leq\frac{\max\big\{ \vertiiibig{\bm{E}^{\top}\bm{U}^{\star}},\vertiiibig{\bm{E}\bm{V}^{\star}}\big\} }{\sigma_{r}^{\star}-\sigma_{r+1}^{\star}-\|\bm{E}\|}.
\end{align*}
The proof is then completed by invoking Lemmas~\ref{prop:unitary-norm-property} and \ref{prop:rotation-UR}.

\input{chapters/DK_asym.tex}

\section{Appendix: Proofs of auxiliary lemmas in Section~\ref{sec:preliminary-distance-angles}}
\label{sec:proof-preliminary-dist-angles}

\subsection{Proof of Lemma \ref{lemma:connection-algebraic-geometric}}
%Let $\bm{U}_{\perp}\in\mathbb{R}^{n\times(n-r)}$
%(resp.~${\bm{U}}^{\star}_{\perp}\in\mathbb{R}^{n\times(n-r)}$)  be a matrix whose columns form an orthonormal
%basis for the orthogonal complement of $\mathcal{U}$ (resp.~${\mathcal{U}}^{\star}$).
Given that singular values are unitarily invariant, it suffices to look at
the singular values of the following matrix
\begin{align}
\left[\begin{array}{c}
\bm{U}^{\top}\\
\bm{U}_{\perp}^{\top}
\end{array}\right]\big(\bm{U}\bm{U}^{\top}- {\bm{U}}^{\star} {\bm{U}}^{\star\top}\big)
	\big[
{\bm{U}}_{\perp}^{\star} , {\bm{U}}^{\star} \big]=\left[\begin{array}{cc}
	\bm{U}^{\top}{\bm{U}}_{\perp}^{\star} & \bm{0}\\
\bm{0} & -\bm{U}_{\perp}^{\top} {\bm{U}}^{\star}
\end{array}\right].
	\label{eq:transform-UU-UU}
\end{align}
Consequently, the singular values of $\bm{U}\bm{U}^{\top}- {\bm{U}}^{\star} {\bm{U}}^{\star\top}$
are composed of those of $\bm{U}^{\top}{\bm{U}}_{\perp}^{\star}$ and those of
$\bm{U}_{\perp}^{\top} {\bm{U}}^{\star}$ combined. It then boils down to characterizing the spectrum of $\bm{U}^{\top}{\bm{U}}_{\perp}^{\star}$ and $\bm{U}_{\perp}^{\top} {\bm{U}}^{\star}$.

To pin down the singular values of $\bm{U}^{\top}{\bm{U}}_{\perp}^{\star}$, we first turn attention to the eigenvalues of $\bm{U}^{\top}\bm{U}_{\perp}^{\star}\bm{U}_{\perp}^{\star\top}\bm{U}$. Assuming that the SVD of $\bm{U}^{\top}\bm{U}^{\star}$ is
given by $\bm{X}\bm{\Sigma}\bm{Y}^{\top}$ (where $\bm{X}$ and $\bm{Y}$
are $r\times r$ orthonormal matrices, and $\bm{\Sigma}$ is diagonal), we can derive
\begin{align}
\bm{U}^{\top}\bm{U}_{\perp}^{\star}\bm{U}_{\perp}^{\star\top}\bm{U} & =\bm{U}^{\top}\big(\bm{I}_{n}-\bm{U}^{\star}\bm{U}^{\star\top}\big)\bm{U}=\bm{U}^{\top}\bm{U}-\bm{U}^{\top}\bm{U}^{\star}\bm{U}^{\star\top}\bm{U} \notag\\
	& =\bm{I}_{r}-\bm{X}\bm{\Sigma}^{2}\bm{X}^{\top}=\bm{X}\big(\bm{I}_{r}-\cos^{2}\bm{\Theta}\big)\bm{X}^{\top} \notag\\
	& = \bm{X} \big( \sin^{2}\bm{\Theta} \big) \bm{X}^{\top} .
	\label{eq:U-Uperp-equiv}
\end{align}
Here, the penultimate identity follows from our construction (cf. \eqref{defn:principal-angles}), where
	we define $\cos\bm{\Theta} \coloneqq \mathsf{diag}([\cos\theta_1,\cdots,\cos\theta_r])$.
Therefore, for any $1\le i \le r$, the $i$-th largest singular value of $\bm{U}^{\top}\bm{U}_{\perp}^{\star}$ obeys
\[
\sigma_{i}\big(\bm{U}^{\top}\bm{U}_{\perp}^{\star}\big)=\sqrt{\lambda_{i}\big(\bm{U}^{\top}\bm{U}_{\perp}^{\star}\bm{U}_{\perp}^{\star\top}\bm{U} \big)}
	%=\sqrt{1-\cos^{2}\theta_{r+1-i}}
	=\sin\theta_{r+1-i},
\]
which results from the ordering in \eqref{eq:principal-angle-range}.  This means that, if $r\leq n-r$, then the singular values of $\bm{U}^{\top}{\bm{U}}_{\perp}^{\star}$ are precisely given by $\{\sin\theta_{i}\}_{1\leq i\leq r}$.
Repeating this argument reveals that the singular values of $\bm{U}_{\perp}^{\top} {\bm{U}}^{\star}$
are also $\{\sin\theta_{i}\}_{1\leq i\leq r}$ if $r\leq n-r$.

Combining the above observations thus completes the proof.

\subsection{Proof of Lemma \ref{prop:unitary-norm-property}}
A closer inspection  of the proof of Lemma \ref{lemma:connection-algebraic-geometric} (in particular,  \eqref{eq:U-Uperp-equiv} and the orthonormality of $\bm{X}$) reveals that
\begin{align*}
\big\|\bm{U}^{\top}\bm{U}_{\perp}^{\star}\big\| & =\sqrt{\big\|\bm{U}^{\top}\bm{U}_{\perp}^{\star}\bm{U}_{\perp}^{\star\top}\bm{U}\big\|}=\sqrt{\|\bm{X}\big(\sin^{2}\bm{\Theta}\big)\bm{X}^{\top}\|}=\|\sin\bm{\Theta}\|,\\
\big\|\bm{U}^{\top}\bm{U}_{\perp}^{\star}\big\|_{\mathrm{F}} & =\sqrt{\mathsf{Tr}(\bm{U}^{\top}\bm{U}_{\perp}^{\star}\bm{U}_{\perp}^{\star\top}\bm{U})}=\sqrt{\mathsf{Tr}(\bm{X}\big(\sin^{2}\bm{\Theta}\big)\bm{X}^{\top})}\\
 & =\sqrt{\mathsf{Tr}(\bm{X}^{\top}\bm{X}\sin^{2}\bm{\Theta})}=\sqrt{\mathsf{Tr}(\sin^{2}\bm{\Theta})}=\|\sin\bm{\Theta}\|_{\mathrm{F}},
\end{align*}
where we have used the basic property $\mathsf{Tr}(\bm{A}\bm{B})=\mathsf{Tr}(\bm{B}\bm{A})$. Similarly,
\begin{align*}
\big\|\bm{U}_{\perp}^{\top}\bm{U}^{\star}\big\| =\|\sin\bm{\Theta}\|, 
	\qquad \text{and} \qquad
\big\|\bm{U}_{\perp}^{\top}\bm{U}^{\star}\big\|_{\mathrm{F}} =\|\sin\bm{\Theta}\|_{\mathrm{F}}.
\end{align*}
Note that the above identities hold for all $1\leq r\leq n$. In addition, the relation \eqref{eq:transform-UU-UU} tells us that
\begin{subequations}
\begin{align}
	\big\|\bm{U}\bm{U}^{\top}-\bm{U}^{\star}\bm{U}^{\star\top}\big\| &= \max\big\{\big\|\bm{U}^{\top}\bm{U}_{\perp}^{\star}\big\|,\big\|\bm{U}_{\perp}^{\top}\bm{U}^{\star}\big\|\big\}; \\
	\big\|\bm{U}\bm{U}^{\top}-\bm{U}^{\star}\bm{U}^{\star\top}\big\|_{\mathrm{F}} & =\Big(\big\|\bm{U}^{\top}\bm{U}_{\perp}^{\star}\big\|_{\mathrm{F}}^{2}+\big\|\bm{U}_{\perp}^{\top}\bm{U}^{\star}\big\|_{\mathrm{F}}^{2}\Big)^{1/2}.
\end{align}
\end{subequations}
Putting the above identities together immediately establishes the advertised results.

\subsection{Proof of Lemma~\ref{prop:rotation-UR}}
\label{proof:equiv-rotation-dist}

As before, suppose that the SVD of $\bm{U}^{\top}\bm{U}^{\star}$
is given by $\bm{X}\bm{\Sigma}\bm{Y}^{\top}$, where $\bm{X}$ and $\bm{Y}$
are $r\times r$ orthonormal matrices whose columns contain the left singular vectors and the right singular vectors of $\bm{U}^{\top}\bm{U}^{\star}$, respectively, and $\bm{\Sigma}\in \mathbb{R}^{r\times r}= \cos\bm{\Theta}$ is a diagonal matrix whose diagonal entries correspond to the singular values of $\bm{U}^{\top}\bm{U}^{\star}$.

\paragraph{The spectral norm upper bound.}
We first observe that
\begin{align}
\|\bm{U}\bm{X}\bm{Y}^{\top}-\bm{U}^{\star}\|^{2} & =\|(\bm{U}\bm{X}\bm{Y}^{\top}-\bm{U}^{\star})^{\top}(\bm{U}\bm{X}\bm{Y}^{\top}-\bm{U}^{\star})\|\nonumber\\
 & =\|2\bm{I}_{r}-\bm{Y}\bm{X}^{\top}\bm{U}^{\top}\bm{U}^{\star}-\bm{U}^{\star\top}\bm{U}\bm{X}\bm{Y}^{\top}\| \nonumber\\
	& =\|2\bm{I}_{r}-\bm{Y}\bm{X}^{\top}\bm{X}\bm{\Sigma}\bm{Y}^{\top}-\bm{Y}\bm{\Sigma}\bm{X}^{\top}\bm{X}\bm{Y}^{\top}\| \nonumber\\
 & =2\|\bm{Y}(\bm{I}_{r}-\bm{\Sigma})\bm{Y}^{\top}\|=2\|\bm{I}_{r}-\bm{\Sigma}\|.
	\label{eq:UXT-Ustar-UB1}
\end{align}
Here, the penultimate line relies on the singular value decomposition $\bm{U}^{\top}\bm{U}^{\star}=\bm{X}\bm{\Sigma}\bm{Y}^{\top}$,
while the two identities in the last line result from  the orthonormality of $\bm{X}$ and $\bm{Y}$, respectively. In addition, note that
\begin{align*}
	\|\bm{I}_{r}-\bm{\Sigma}\| &= \|\bm{I}_{r}-\cos\bm{\Theta}\|\leq\|\bm{I}_{r}-\cos^{2}\bm{\Theta}\| \\
	& =\|\sin^{2}\bm{\Theta}\|=\|\sin\bm{\Theta}\|^2.
\end{align*}
This taken together with \eqref{eq:UXT-Ustar-UB1} leads to
\begin{align*}
	\min_{\bm{R}\in \mathcal{O}^{r\times r}}\big\|\bm{U}\bm{R}-\bm{U}^{\star}\big\|
	\leq \big\|\bm{U}\bm{X}\bm{Y}^{\top}-\bm{U}^{\star}\big\| \leq \sqrt{2} \|\sin \bm{\Theta} \| ,
	% \mathsf{dist}(\bm{U},\bm{U}^{\star}).
\end{align*}
where the first inequality holds since $\bm{X}$ and $\bm{Y}$ are both orthonormal matrices and hence $\bm{X}\bm{Y}^{\top}$ is also orthonormal.

\paragraph{The spectral norm lower bound.}
On the other hand, we make the observation that
\begin{align}
	& \min_{\bm{R}\in\mathcal{O}^{r\times r}}\big\|\bm{U}\bm{R}-\bm{U}^{\star}\big\|^{2}  =\min_{\bm{R}\in\mathcal{O}^{r\times r}}\big\|(\bm{U}\bm{R}-\bm{U}^{\star})^{\top}(\bm{U}\bm{R}-\bm{U}^{\star})\big\|\nonumber \\
 & \qquad\qquad =\min_{\bm{R}\in\mathcal{O}^{r\times r}}\big\|\bm{R}^{\top}\bm{U}^{\top}\bm{U}\bm{R}+\bm{U}^{\star\top}\bm{U}^{\star}-\bm{R}^{\top}\bm{U}^{\top}\bm{U}^{\star}-\bm{U}^{\star\top}\bm{U}\bm{R}\big\|\nonumber \\
 & \qquad\qquad{=}\min_{\bm{R}\in\mathcal{O}^{r\times r}}\big\|2\bm{I}_{r}-\bm{R}^{\top}\bm{X}\bm{\Sigma}\bm{Y}^{\top}-\bm{Y}\bm{\Sigma}\bm{X}^{\top}\bm{R}\big\|,
	\label{eq:relation-i-123456}
\end{align}
where the last relation holds since $\bm{X}\bm{\Sigma}\bm{Y}^{\top}$ is the
SVD of $\bm{U}^{\top}\bm{U}^{\star}$. Continue the derivation to obtain
\begin{align}
\eqref{eq:relation-i-123456}  & \overset{(\mathrm{i})}{=}\min_{\bm{Q}\in\mathcal{O}^{r\times r}}\big\|2\bm{I}_{r}-\bm{Q}\bm{\Sigma}\bm{Y}^{\top}-\bm{Y}\bm{\Sigma}\bm{Q}^{\top}\big\|\nonumber \\
 & \overset{(\mathrm{ii})}{=}\min_{\bm{Q}\in\mathcal{O}^{r\times r}}\big\|2\bm{Q}^{\top}\bm{Q}-\bm{Q}^{\top}\bm{Q}\bm{\Sigma}\bm{Y}^{\top}\bm{Q}-\bm{Q}^{\top}\bm{Y}\bm{\Sigma}\bm{Q}^{\top}\bm{Q}\big\|\nonumber \\
 &  =\min_{\bm{Q}\in\mathcal{O}^{r\times r}}\big\|2\bm{I}_{r}-\bm{\Sigma}\bm{Y}^{\top}\bm{Q}-\bm{Q}^{\top}\bm{Y}\bm{\Sigma}\big\|\nonumber \\
 &  \overset{(\mathrm{iii})}{=}\min_{\bm{O}\in\mathcal{O}^{r\times r}}\big\|2\bm{I}_{r}-\bm{\Sigma}\bm{O}-\bm{O}^{\top}\bm{\Sigma}\big\|.\label{eq:UR-Ustar-identity}
\end{align}
Here, (i) follows by setting $\bm{Q}=\bm{R}^{\top}\bm{X}$
(since both $\bm{X}$ and $\bm{R}$ are orthonormal matrices), (ii)
results from the unitary invariance of the spectral norm, whereas (iii) holds
by setting $\bm{O}=\bm{Y}^{\top}\bm{Q}$. Moreover, recognizing that
$\|\bm{\Sigma}\bm{O}\|\leq\|\bm{\Sigma}\| \cdot \|\bm{O}\|\leq1$ (and hence
$2\bm{I}_{r}-\bm{\Sigma}\bm{O}-\bm{O}^{\top}\bm{\Sigma}\succeq\bm{0}$),
one can obtain
\begin{align}
\min_{\bm{O}\in\mathcal{O}^{r\times r}}\big\|2\bm{I}_{r}-\bm{\Sigma}\bm{O}-\bm{O}^{\top}\bm{\Sigma}\big\|
	& = \min_{\bm{O}\in\mathcal{O}^{r\times r}}\lambda_{\max} \big( 2\bm{I}_{r}-\bm{\Sigma}\bm{O}-\bm{O}^{\top}\bm{\Sigma}\big) \notag\\
	& =\min_{\bm{O}\in\mathcal{O}^{r\times r}}\max_{\bm{u}:\|\bm{u}\|_{2}=1}\bm{u}^{\top}\big(2\bm{I}_{r}-\bm{\Sigma}\bm{O}-\bm{O}^{\top}\bm{\Sigma}\big)\bm{u}\nonumber \\
 & =\min_{\bm{O}\in\mathcal{O}^{r\times r}}\max_{\bm{u}:\|\bm{u}\|_{2}=1}\big(2-2\bm{u}^{\top}\bm{\Sigma}\bm{O}\bm{u}\big)\nonumber \\
 & \geq\min_{\bm{O}\in\mathcal{O}^{r\times r}}\big(2-2\bm{e}_{r}^{\top}\bm{\Sigma}\bm{O}\bm{e}_{r}\big)\nonumber \\
 & = 2-2\cos\theta_{r}\max_{\bm{O}\in\mathcal{O}^{r\times r}}\bm{e}_{r}^{\top}\bm{O}\bm{e}_{r}\nonumber \\
 & \geq 2-2\cos\theta_{r}
	=4\sin^{2}(\theta_{r}/2).\label{eq:UR-star-inequality}
\end{align}
Here, the inequality follows by taking $\bm{u}$ to be $\bm{e}_{r}$
(recall that by construction, $\sigma_{r}=\cos\theta_{r}\geq 0$ is the
smallest singular value of $\bm{\Sigma}$), and the penultimate line
holds by combining the facts $|\bm{e}_{r}^{\top}\bm{O}\bm{e}_{r}|\leq\|\bm{O}\|=1$
and $\bm{e}_{r}^{\top}\bm{e}_{r}=1$. Putting (\ref{eq:UR-star-inequality})
and (\ref{eq:UR-Ustar-identity}) together yields
\begin{align*}
\min_{\bm{R}\in\mathcal{O}^{r\times r}}\big\|\bm{U}\bm{R}-\bm{U}^{\star}\big\| & \geq\sqrt{4\sin^{2}(\theta_{r}/2)}=2\sin(\theta_{r}/2)=\|2\sin(\bm{\Theta}/2)\|\nonumber \\
 & \geq\|\sin\bm{\Theta}\|,
	%=\mathsf{dist}(\bm{U},\bm{U}^{\star})
\end{align*}
where we again use the inequality $2\sin(\theta/2) \geq \sin \theta$ for all $\theta\in [0,\pi/2]$.

Finally, invoking the relation $\|\sin\bm{\Theta}\|=  \|\bm{U}\bm{U}^{\top} - \bm{U}^{\star}\bm{U}^{\star\top}\|$ (see Lemma~\ref{prop:unitary-norm-property}) establishes the claimed spectral norm bounds.

\paragraph{The Frobenius norm upper bound.}
  Regarding the Frobenius norm upper bound, one sees that
\begin{align}
	& \big\Vert \bm{U}\bm{X}\bm{Y}^{\top}-\bm{U}^{\star}\big\Vert _{\mathrm{F}}^{2}  =\left\Vert \bm{U}\right\Vert _{\mathrm{F}}^{2}+\big\|\bm{U}^{\star}\big\|_{\mathrm{F}}^{2}-2\mathsf{Tr}\big(\bm{Y}\bm{X}^{\top}\bm{U}^{\top}\bm{U}^{\star}\big) \notag\\
 & \qquad\qquad \overset{(\mathrm{i})}{=}r+r-2\mathsf{Tr}\big(\bm{Y}\bm{X}^{\top}\bm{X}\bm{\Sigma}\bm{Y}^{\top}\big)
  \overset{(\mathrm{ii})}{=}2r-2\mathsf{Tr}\left(\bm{\Sigma}\right) ,
	\label{eq:relation-UB-12689}
\end{align}
where (i) holds since $\bm{U}$ and $\bm{U}^{\star}$ are both $n\times r$
matrices with orthonormal columns, and (ii) follows since $\bm{X}^{\top}\bm{X}=\bm{Y}^{\top}\bm{Y}=\bm{I}$ (and hence $\mathsf{Tr}(\bm{Y}\bm{X}^{\top}\bm{X}\bm{\Sigma}\bm{Y}^{\top})=\mathsf{Tr}(\bm{Y}^{\top}\bm{Y}\bm{X}^{\top}\bm{X}\bm{\Sigma})=\mathsf{Tr}(\bm{\Sigma})$).
Furthermore,
\begin{align*}
2r-2\mathsf{Tr}\left(\bm{\Sigma}\right)
	& \overset{(\mathrm{iii})}{=}2\sum\nolimits_{i} (1-\cos\theta_{i})
	\leq2\sum\nolimits_{i} (1-\cos^{2}\theta_{i})\\
 & =2\left\Vert \sin\bm{\Theta}\right\Vert _{\mathrm{F}}^{2} = \big\Vert \bm{U}\bm{U}^{\top}-{\bm{U}}^{\star}{\bm{U}}^{\star\top}\big\Vert _{\mathrm{F}}^2,
\end{align*}
where (iii) holds by construction (cf. \eqref{defn:principal-angles}), and the last identity results from Lemma~\ref{prop:unitary-norm-property}. This
taken collectively with  \eqref{eq:relation-UB-12689}
reveals that
\begin{align*}
\min_{\bm{R}\in\mathcal{O}^{r\times r}}\left\Vert \bm{U}\bm{R}-\bm{U}^{\star}\right\Vert _{\mathrm{F}}^{2}
	& \leq\big\Vert \bm{U}\bm{X}\bm{Y}^{\top}-\bm{U}^{\star}\big\Vert _{\mathrm{F}}^{2}
	\leq
	%2\left\Vert \sin\bm{\Theta}\right\Vert _{\mathrm{F}}^{2} \\
	%=\mathsf{dist}_{\mathrm{F}}^{2}\big(\bm{U},\bm{U}^{\star}\big).
	\big\Vert \bm{U}\bm{U}^{\top}-{\bm{U}}^{\star}{\bm{U}}^{\star\top}\big\Vert _{\mathrm{F}}^2  ,
\end{align*}
where the first inequality holds since $\bm{X}$ and $\bm{Y}$ are both orthonormal matrices and hence $\bm{X}\bm{Y}^{\top}$ is also orthonormal.

\paragraph{The Frobenius norm lower bound.}

With regards to the Frobenius norm lower bound, it is seen that
\begin{align}
\min_{\bm{R}\in\mathcal{O}^{r\times r}}\big\|\bm{U}\bm{R}-\bm{U}^{\star}\big\|_{\mathrm{F}}^{2} & =\min_{\bm{R}\in\mathcal{O}^{r\times r}}\Big\{\|\bm{U}\bm{R}\|_{\mathrm{F}}^{2}+\|\bm{U}^{\star}\|_{\mathrm{F}}^{2}-2\big\langle\bm{U}\bm{R},\bm{U}^{\star}\big\rangle\Big\}\nonumber \\
 & \overset{(\mathrm{i})}{=} 2\min_{\bm{R}\in\mathcal{O}^{r\times r}}\Big\{ r-\big\langle\bm{R},\bm{U}^{\top}\bm{U}^{\star}\big\rangle\Big\}\nonumber \\
 & \overset{(\mathrm{ii})}{=} 2\min_{\bm{R}\in\mathcal{O}^{r\times r}}\Big\{ r-\big\langle\bm{R},\bm{X}\bm{\Sigma}\bm{Y}^{\top}\big\rangle\Big\} ,
	\label{eq:relation-ii-67890}
\end{align}
where (i) holds since $\|\bm{U}\|_{\mathrm{F}}=\|\bm{U}^{\star}\|_{\mathrm{F}}=\sqrt{r}$,
and (ii) relies on the SVD $\bm{X}\bm{\Sigma}\bm{Y}^{\top}$ of $\bm{U}^{\top}\bm{U}^{\star}$.
Continue the derivation to obtain
\begin{align}
\eqref{eq:relation-ii-67890}
	& \overset{(\mathrm{iii})}{=}  2\min_{\bm{Q}\in\mathcal{O}^{r\times r}}\Big\{ r-\big\langle\bm{Q},\cos\bm{\Theta}\big\rangle\Big\}
	 \overset{(\mathrm{iv})}{\geq} 2\min_{\bm{Q}\in\mathcal{O}^{r\times r}}\Big\{ r-\|\bm{Q}\|\,\|\cos\bm{\Theta}\|_{*}\Big\}\nonumber \\
 & =2 \big( r-\sum\nolimits_{i}\cos\theta_{i}\big).\label{eq:min-UR-Ustar-fro-1}
\end{align}
Here, (iii) sets $\bm{Q}=\bm{X}^{\top}\bm{R}\bm{Y}$ and identifies
$\bm{\Sigma}$ as $\cos\bm{\Theta}$,
(iv) comes from the elementary inequality $\langle \bm{A}, \bm{B} \rangle \leq \|\bm{A}\|\,\|\bm{B}\|_*$,
whereas the last line follows
since $\cos\theta_{i}\geq0$. Additionally, it is easily seen that
\begin{align}
%2r-2\sum\nolimits_{i} \cos\theta_{i}
\eqref{eq:min-UR-Ustar-fro-1}
	& =2\sum\nolimits_{i} (1-\cos\theta_{i})=4\sum\nolimits_{i} \sin^{2}(\theta_{i}/2)\nonumber \\
 & \geq \sum\nolimits_{i} \sin^{2}\theta_{i}
	%=\|\sin\bm{\Theta}\|_{\mathrm{F}}^{2}\nonumber \\
  =\frac{1}{2} \big\Vert \bm{U}\bm{U}^{\top}-{\bm{U}}^{\star}{\bm{U}}^{\star\top}\big\Vert _{\mathrm{F}}^2 ,
	\label{eq:min-UR-Ustar-fro-2}
\end{align}
where the penultimate relation follows from the elementary inequality
$2\sin(\theta/2)\geq\sin\theta$ (which holds for any $0\leq\theta\leq\pi/2$), and
the last line invokes Lemma \ref{prop:unitary-norm-property}.
Combining the inequalities (\ref{eq:min-UR-Ustar-fro-1}) and (\ref{eq:min-UR-Ustar-fro-2}),
we establish the claimed lower bound.

\section{Notes}

\paragraph{Additional resources on matrix perturbation theory.}
Matrix perturbation theory is a firmly established topic that has been extensively studied in the past several decades.
Two classic books that offer in-depth discussions of  perturbation theory for eigenspaces and singular subspaces are \citet{stewart1990matrix,sun1987perturbation}.
Other valuable resources on this topic include \citet{bhatia2013matrix,horn2012matrix}.
The exposition herein is  largely influenced by the excellent lecture notes by \citet{montanari2011notes,hsu2016notes}.
In addition, the book \citep{kato2013perturbation} offers a more abstract  treatment of perturbation theory from the viewpoint of linear operators. Several variants of the sin$\bm{\Theta}$ theorem amenable to statistical analysis are available in the statistics literature as well (e.g., \citet{MR3371006,vu2013minimaxSPCA,cai2018rate,zhang2018heteroskedastic}). 
%
% Yuxin: move the following to notes in Chapter 3 
%For distributions with symmetric innovation, the empirical top eigenspaces are unbiased
%\citep{fan2019distributed}; for general distributions, the eigenspace perturbation bias is of second order \citep{koltchinskii2016asymptotics,fan2019distributed}.

\paragraph{Extensions.}
We point out several well-known extensions of the theorems presented in this chapter.
To begin with, the current exposition restricts attention to the real case for simplicity, while in fact all results herein generalize to the complex-valued case \citep{stewart1990matrix}. In addition, Theorem~\ref{thm:wedin} together with Lemma~\ref{prop:rotation-UR} reveals the existence of two rotation matrices $\bm{R}_{U}$ and $\bm{R}_{V}$ obeying
\begin{align*}
	\max\big\{ \|\bm{U}\bm{R}_{U}-\bm{U}^{\star}\|_{\mathrm{F}}, \|\bm{V}\bm{R}_{V}-\bm{V}^{\star}\|_{\mathrm{F}} \big\}
	\leq\frac{\sqrt{2}\max\big\{ \|\bm{E}^{\top}\bm{U}^{\star}\|_{\mathrm{F}},\|\bm{E}\bm{V}^{\star}\|_{\mathrm{F}}\big\} }{\sigma_{r}^{\star}-\sigma_{r+1}^{\star}-\|\bm{E}\|},
\end{align*}
but falls short of illuminating the connection between $\bm{R}_{U}$ and $\bm{R}_{V}$. An extension derived in \citet{dopico2000sintheta} establishes a similar perturbation bound even when $\bm{R}_U$ and $\bm{R}_V$ are taken to be the same rotation matrix.

%% file: chapters/prelim.tex
\section{Preliminaries: Basics of matrix analysis} \label{subsec:Matrix-analysis}

We begin this chapter by gathering a few elementary materials in matrix analysis that prove useful for our theoretical development. The readers familiar with matrix analysis can proceed directly to Section~\ref{sec:preliminary-distance-angles}.

\paragraph{Unitarily invariant norms.}

Among all matrix norms, the family of unitarily invariant norms defined below is of central interest, which subsumes  as special cases the spectral norm $\|\cdot\|$ and the Frobenius norm $\|\cdot\|_{\mathrm{F}}$.
\begin{definition}
A matrix norm $\vertiii{\cdot}$ on $\mathbb{R}^{m\times n}$ is said to be unitarily invariant if
\[
	\vertiii{\bm{A}}=\vertiiibig{\bm{U}^{\top}\bm{A}\bm{V}}
\]
holds for any matrix $\bm{A}\in\mathbb{R}^{m\times n}$ and any two square orthonormal
matrices $\bm{U}\in\mathcal{O}^{m\times m}$ and $\bm{V}\in\mathcal{O}^{n\times n}$.
\end{definition}
This class of matrix norms enjoys several useful properties, as summarized in the following lemma. The proof can be found in \citet[Theorem 3.9]{stewart1990matrix}.
\begin{lemma}
\label{prop:unitary_norm_relation}
For any unitarily invariant norm $\vertiii{\cdot}$, one has 
\begin{align*}
\vertiii{\bm{A}\bm{B}} & \leq\vertiii{\bm{A}} \cdot  \left\Vert \bm{B}\right\Vert, &  & \vertiii{\bm{A}\bm{B}}\leq\vertiii{\bm{B}} \cdot \left\Vert \bm{A}\right\Vert,\\
\vertiii{\bm{A}\bm{B}} & \geq\vertiii{\bm{A}} \, \sigma_{\min}\left(\bm{B}\right), &  & \vertiii{\bm{A}\bm{B}}\geq\vertiii{\bm{B}} \, \sigma_{\min}\left(\bm{A}\right).
\end{align*}
%
%where $\sigma_{\min}(\bm{A})$ denotes the smallest singular value of a matrix $\bm{A}$. 
\end{lemma}

\paragraph{Perturbation bounds for eigenvalues and singular values.}

Next, we review classical perturbation
bounds for eigenvalues of symmetric matrices and for singular values of general matrices. 

\begin{lemma}[Weyl's inequality for eigenvalues]
\label{lemma:weyl}
Let $\bm{A},\bm{E}\in\mathbb{R}^{n\times n}$ be two real symmetric matrices.
	For every $1\leq i\leq n$, the $i$-th largest eigenvalues of $\bm{A}$ and $\bm{A}+\bm{E}$ obey
\begin{equation}
	\left|\lambda_{i}\left(\bm{A}\right)-\lambda_{i}\left(\bm{A} +\bm{E}\right)\right|\leq\left\Vert \bm{E}\right\Vert .
	\label{eq:weyl}
\end{equation}
\end{lemma}
\begin{proof}See Equation (1.63) in \citet{Tao2012RMT}. \end{proof}

\begin{lemma}[Weyl's inequality for singular values]
\label{lemma:weyls-singular-value}
Let $\bm{A},\bm{E}\in\mathbb{R}^{m\times n}$ be two general matrices.
Then for every $1\leq i\leq \min\{m,n\}$, the $i$-th largest singular values of $\bm{A}$ and $\bm{A}+\bm{E}$ obey 
\[
	\left|\sigma_{i}\left(\bm{A}+\bm{E}\right)-\sigma_{i}\left(\bm{A}\right)\right|\leq\left\Vert \bm{E}\right\Vert .
\]
\end{lemma}
\begin{proof}See Exercise 1.3.22 in \citet{Tao2012RMT}. \end{proof}

An immediate implication of Lemma~\ref{lemma:weyl} (resp.~Lemma~\ref{lemma:weyls-singular-value}) is that the eigenvalues of a real symmetric matrix (resp.~the singular values of a general matrix) 
 are stable vis-\`a-vis small perturbations.

%% file: chapters/DK_asym.tex
\section{Eigenvector perturbation for probability transition matrices}
\label{sec:eigenvector-theory-DK}

Thus far, our eigenvector perturbation analysis has been constrained to the set of \emph{symmetric}
matrices. Note, however, that the utility of eigenvectors is by no means confined to symmetric matrices.
In fact, eigenvector analysis  plays a vital role in studying  asymmetric matrices as well, most notably the family of probability transition matrices of Markov chains.  This section explores how to extend eigenvector perturbation theory to accommodate an important class of probability transition matrices associated with \emph{reversible} Markov chains.

\subsection{Background, setup and notation}
\label{subsec:Setup-and-notation-prob-matrix}

Before presenting the formulation, we remind the readers that a matrix $\bm{P}\in\mathbb{R}^{n\times n}$
is a probability transition matrix if it is composed of non-negative entries with
each row summing to $1$, which is used to describe the state transition of a Markov chain over a set of $n$ states.
Of special interest is the stationary distribution of $\bm{P}$, denoted by a probability vector $\bm{\pi}=[\pi_i]_{1\leq i\leq n}$,
that satisfies
\begin{equation}\label{eq2.30}
\bm{\pi}\geq\bm{0},\qquad\bm{1}^{\top}\bm{\pi}=1,\qquad\text{and}\qquad\bm{\pi}^{\top}\bm{P}=\bm{\pi}^{\top}.
\end{equation}
In words, the distribution $\bm{\pi}$ is invariant with respect to $\bm{P}$.
Clearly, $\bm{\pi}$ is the left eigenvector of $\bm{P}$ associated with eigenvalue $1$, with the corresponding right eigenvector given by $\bm{1}$.  By the  Gershgorin circle theorem (see, e.g., \citet{olver2006applied}), the modulus of all eigenvalues must be bounded by the maximum of the row sum, which is  1.  Given that $1$ is an eigenvalue of $\bm{P}$, the largest modulus of the eigenvalues of $\bm{P}$ is precisely 1, and therefore $\bm{\pi}$ is the {\em leading} left eigenvector of $\bm{P}$.
In addition, a Markov chain is said to be reversible when the following {\em detailed balance} equations are satisfied:
\begin{equation}
	\pi_{i} P_{i,j} =\pi_{j} P_{j,i},  \qquad \text{for all }1\leq i,j\leq n, \label{eq:detailed-balance}
\end{equation}
where $\bm{\pi}=[\pi_i]_{1\leq i\leq n}$ is the stationary distribution obeying~\eqref{eq2.30}.  It will be seen in the proof of Theorem~\ref{thm:DK_asym} that all eigenvalues of such a matrix $\bP$ are real. For readers who wish  an introduction to the basics of Markov chains, 
we recommend  the monograph by \citet{bremaud2013markov}.

In this section, we consider the probability transition matrix $\bm{P}^{\star} \in \mathbb{R}^{n\times n}$ of a reversible Markov chain, as well as its perturbed version---also in the form of a probability transition matrix:
\begin{align*}
	\bm{P}=\bm{P}^{\star}+\bm{E} \in \mathbb{R}^{n\times n}.
\end{align*}
The leading left eigenvectors of $\bm{P}^{\star}$ and $\bm{P}$---or equivalently, the vectors representing their stationary distributions---are denoted by $\bm{\pi}^{\star}$ and $\bm{\pi}$, respectively.
Here, we allow $\bm{E}$ to be fairly general, meaning that $\bm{P}$ does not necessarily represent a reversible Markov chain. The question is: how does the matrix $\bm{E}$ affect the perturbation $\bm{\pi} - \bm{\pi}^{\star}$ of the leading left eigenvector of interest?

Additionally, we find it helpful to introduce several notation frequently used in the studies of Markov chains.
Instead of operating under the usual $\ell_{2}$ norm, the stationary
distribution $\bm{\pi}$ equips us with a new set of norms.
Specifically,  for a strictly positive probability vector $\bm{\pi}=[\pi_i]_{1\leq i\leq n}$, any vector $\bm{x}=[x_i]_{1\leq i\leq n}$ and any matrix $\bm{A}$,
it is useful to introduce the vector norm $\|\bm{x}\|_{\bm{\pi}} \coloneqq \sqrt{\sum_i \pi_i x_i^2}$
and the corresponding matrix norm $\|\bm{A}\|_{\bm{\pi}} \coloneqq  \sup_{\|\bm{x}\|_{\bm{\pi}}=1}\|\bm{A}\bm{x}\|_{\bm{\pi}}$.

\subsection{Perturbation of the leading eigenvector}
\label{sec:perturbation-theory-MC}

Now we are ready to present the perturbation bound for the leading left
eigenvector of a probability transition matrix, a result originally developed in \citet{chen2017spectral}.

\begin{theorem}
\label{thm:DK_asym}
	Consider the settings in Section~\ref{subsec:Setup-and-notation-prob-matrix}. Suppose that $\bm{P}^{\star}$ represents a reversible Markov chain, whose stationary distribution vector $\bm{\pi}^{\star}$ is strictly positive.  Assume that
\begin{equation}
	\left\Vert \bm{E}\right\Vert _{\bm{\pi}^{\star}}
	< 1-\max\big\{ \lambda_{2}(\bm{P}^{\star}),-\lambda_{n}(\bm{P}^{\star})\big\} .
	\label{eq:spectral-gap-condition}
\end{equation}
Then one has
\[
	\|\bm{\pi}-\bm{\pi}^{\star}\|_{\bm{\pi}^{\star}}
	\leq
	\frac{\big\Vert \bm{\pi}^{\star\top}\bm{E}\big\Vert _{\bm{\pi}^{\star}}}{1-\max\big\{ \lambda_{2}(\bm{P}^{\star}),-\lambda_{n}(\bm{P}^{\star})\big\}
	-\left\Vert \bm{E}\right\Vert _{\bm{\pi}^{\star}}}.
\]
\end{theorem}
The similarity between Theorem~\ref{thm:DK_asym}
and Corollary~\ref{cor:davis-kahan-conclusion-corollary} is noteworthy. Indeed,
recalling that the largest eigenvalue of the probability transition matrix $\bm{P}^{\star}$ is precisely 1,
one might view $1-\max\left\{ \lambda_{2}(\bm{P}^{\star}),-\lambda_{n}(\bm{P}^{\star})\right\} $
as the gap between the first and the second largest eigenvalues of
$\bm{P}^{\star}$ (in magnitude),
akin to the eigengap $|\lambda_{r}^{\star}|-|\lambda_{r+1}^{\star}|$ in Corollary~\ref{cor:davis-kahan-conclusion-corollary} (with $r=1$).
In words, Theorem~\ref{thm:DK_asym} guarantees that as long as the size of the
perturbation matrix $\bm{E}$ is not too large, 
the  perturbation of the leading left eigenvector---or equivalently, the perturbation of the stationary distribution of the associated Markov chain---is proportional to the size of the noise when projected onto the direction $\bm{\pi}^{\star}$, as measured by $\Vert \bm{\pi}^{\star\top}\bm{E}\Vert _{\bm{\pi}^{\star}}$.
As we shall demonstrate in Section~\ref{sec:ranking}, this perturbation theory delivers powerful techniques for analyzing the ranking problem  described previously in Chapter~\ref{cha:introduction}.

\begin{remark}
	Sensitivity and perturbation analyses for the steady-state distributions of Markov chains have been studied in the literature; see, e.g., \citet{mitrophanov2005sensitivity,liu2012perturbation,jiang2017unified,rudolf2018perturbation} and the references therein. 
\end{remark}

\subsection{Proof of Theorem~\ref{thm:DK_asym}} \label{subsec:proof-DK-asym}

Since $\bm{\pi}^{\star}$ and $\bm{\pi}$ denote respectively the leading left eigenvectors of $\bm{P}^{\star}$ and $\bm{P}$, namely,
\[
\bm{\pi}^{\star\top}\bm{P}^{\star}=\bm{\pi}^{\star\top},\qquad\text{and}\qquad\bm{\pi}^{\top}\bm{P}=\bm{\pi}^{\top},
\]
the perturbation $\bm{\pi}-\bm{\pi}^{\star}$ admits the following decomposition
\begin{align*}
&\bm{\pi}^{\top}-\bm{\pi}^{\star\top}  =\bm{\pi}^{\top}\bm{P}-\bm{\pi}^{\star\top}\bm{P}^{\star}=\left(\bm{\pi}-\bm{\pi}^{\star}\right)^{\top}\bm{P}+\bm{\pi}^{\star\top}\left(\bm{P}-\bm{P}^{\star}\right)\\
 &\quad =\left(\bm{\pi}-\bm{\pi}^{\star}\right)^{\top}\left(\bm{P}-\bm{P}^{\star}\right)+\left(\bm{\pi}-\bm{\pi}^{\star}\right)^{\top}\bm{P}^{\star}+\bm{\pi}^{\star\top}\left(\bm{P}-\bm{P}^{\star}\right)\\
 &\quad =\left(\bm{\pi}-\bm{\pi}^{\star}\right)^{\top}\left(\bm{P}-\bm{P}^{\star}\right)
	+\left(\bm{\pi}-\bm{\pi}^{\star}\right)^{\top}\big(\bm{P}^{\star}-\bm{1}\bm{\pi}^{\star\top}\big) + \bm{\pi}^{\star\top}\left(\bm{P}-\bm{P}^{\star}\right).
\end{align*}
Here, the last relation hinges upon the fact that $\bm{\pi}$ and $\bm{\pi}^{\star}$
are probability vectors, and hence $\left(\bm{\pi}-\bm{\pi}^{\star}\right)^{\top}\bm{1}= 1 - 1 = 0$.
Apply the triangle inequality with respect to the norm $\|\cdot\|_{\bm{\pi}^{\star}}$ to obtain
\begin{align*}
\|\bm{\pi}-\bm{\pi}^{\star}\|_{\bm{\pi}^{\star}}
  & \leq \big\|\left(\bm{\pi}-\bm{\pi}^{\star}\right)^{\top}\left(\bm{P}-\bm{P}^{\star}\right) \big\|_{\bm{\pi}^{\star}}
	+ \big\|\left(\bm{\pi}-\bm{\pi}^{\star}\right)^{\top}\big(\bm{P}^{\star}-\bm{1}\bm{\pi}^{\star\top}\big) \big\|_{\bm{\pi}^{\star}} \nonumber\\
&\qquad\quad + \big\|\bm{\pi}^{\star\top}\left(\bm{P}-\bm{P}^{\star}\right) \big\|_{\bm{\pi}^{\star}}\nonumber \\
 & \leq\left( \|\bm{P}-\bm{P}^{\star}\|_{\bm{\pi}^{\star}}+ \big\|\bm{P}^{\star}-\bm{1}\bm{\pi}^{\star\top} \big\|_{\bm{\pi}^{\star}}\right)  \|\bm{\pi}-\bm{\pi}^{\star}\|_{\bm{\pi}^{\star}}\nonumber\\
 &\qquad\quad + \big\|\bm{\pi}^{\star\top}\left(\bm{P}-\bm{P}^{\star}\right) \big\|_{\bm{\pi}^{\star}},
%\label{eq:proof-dk-asymm}
\end{align*}
where the last line relies on the definition of the matrix norm $\|\cdot\|_{\bm{\pi}^{\star}}$. Rearranging terms, we are left with
\[
	\|\bm{\pi}-\bm{\pi}^{\star}\|_{\bm{\pi}^{\star}}
	\leq\frac{\big\|\bm{\pi}^{\star\top}\left(\bm{P}-\bm{P}^{\star}\right)\big\|_{\bm{\pi}^{\star}}}{1-\|\bm{P}-\bm{P}^{\star}\|_{\bm{\pi}^{\star}}-\big\|\bm{P}^{\star}-\bm{1}\bm{\pi}^{\star\top}\big\|_{\bm{\pi}^{\star}}},
\]
with the proviso that $\|\bm{P}-\bm{P}^{\star}\|_{\bm{\pi}^{\star}}+ \|\bm{P}^{\star}-\bm{1}\bm{\pi}^{\star\top}\|_{\bm{\pi}^{\star}}<1$.
The proof would then be completed as long as one could justify that
\begin{equation}
	\big\|\bm{P}^{\star}-\bm{1}\bm{\pi}^{\star\top} \big\|_{\bm{\pi}^{\star}}
	=\max\big\{ \lambda_{2}(\bm{P}^{\star}),-\lambda_{n}(\bm{P}^{\star}) \big\} .
	\label{eq:spectral-gap}
\end{equation}

\begin{proof}[Proof of the identity (\ref{eq:spectral-gap})]
Let $\bm{\pi}^{\star}=[\pi_i^{\star}]_{1\leq i\leq n}$, and define a diagonal matrix $\bm{\Pi}^{\star} = \mathsf{diag}([\pi_{1}^{\star}, \cdots, \pi_{n}^{\star}])\in\mathbb{R}^{n\times n}$.
%whose diagonal entries are given by
%
%\[
%\Pi_{i,i}^{\star}=\pi_{i}^{\star} \qquad\text{for all }1\leq i\leq n.
%\]
%
From the definition of the norm $\|\cdot\|_{\bm{\pi}^{\star}}$ (both the matrix version and the vector version), it is easily seen that for any matrix $\bm{A}$,
\begin{align}
	& \|\bm{A}\|_{\bm{\pi}^{\star}}  =\sup_{\bm{x}\neq\bm{0}}\frac{\|\bm{A}\bm{x}\|_{\bm{\pi}^{\star}}}{\|\bm{x}\|_{\bm{\pi}^{\star}}}=\sup_{\bm{x}\neq\bm{0}}\frac{\big\|\big(\bm{\Pi}^{\star}\big)^{1/2}\bm{A}\big(\bm{\Pi}^{\star}\big)^{-1/2}\big(\bm{\Pi}^{\star}\big)^{1/2}\bm{x}\big\|_{2}}{\big\|\big(\bm{\Pi}^{\star}\big)^{1/2}\bm{x}\big\|_{2}} \notag\\
	& \quad =\sup_{\bm{v}\neq\bm{0}}\frac{\big\|\big(\bm{\Pi}^{\star}\big)^{1/2}\bm{A}\big(\bm{\Pi}^{\star}\big)^{-1/2}\bm{v}\big\|_{2}}{\big\|\bm{v}\big\|_{2}}=\big\|\big(\bm{\Pi}^{\star}\big)^{1/2}\bm{A}\big(\bm{\Pi}^{\star}\big)^{-1/2}\big\|,  \label{eq:pi-norm-matrix-equation}
\end{align}
with $\|\cdot\|$ the usual spectral norm,  where the last line replaces $\big(\bm{\Pi}^{\star}\big)^{1/2}\bm{x}$ with $\bm{v}$.
Consequently, we obtain
\begin{align*}
\|\bm{P}^{\star}-\bm{1}\bm{\pi}^{\star\top}\|_{\bm{\pi}^{\star}} & =\|\left(\bm{\Pi}^{\star}\right)^{1/2}\big(\bm{P}^{\star}-\bm{1}\bm{\pi}^{\star\top}\big)\left(\bm{\Pi}^{\star}\right)^{-1/2}\|\\
	& =\big\Vert \bm{S}^{\star}- \bm{\pi}^{\star}_{1/2}(\bm{\pi}^{\star}_{1/2})^{\top} \big\Vert ,
\end{align*}
where we define $\bm{S}^{\star}\coloneqq\left(\bm{\Pi}^{\star}\right)^{1/2}\bm{P}^{\star}\left(\bm{\Pi}^{\star}\right)^{-1/2}$ and $\bm{\pi}^{\star}_{1/2} \coloneqq \big[ \sqrt{\pi_i^{\star} } \, \big]_{1\leq i\leq n}$.
Several basic properties regarding $\bm{S}^{\star}$ are in order; see \citet[Chapter 6.2]{bremaud2013markov}.
\begin{itemize}
	\item[(a)] Since $\bm{P}^{\star}$ represents a reversible Markov chain with stationary distribution $\bm{\pi}^{\star}$,   			
		the matrix $\bm{S}^{\star}$ is symmetric, whose eigenvalues are real-valued. This can be verified by the detailed balance equations~(\ref{eq:detailed-balance}).
	\item[(b)] Given that $\bm{S}^{\star}$ is obtained via a similarity transformation of $\bm{P}^{\star}$,
		we see that $\bm{S}^{\star}$ and $\bm{P}^{\star}$ share the same set of eigenvalues.
		This can easily be verified from the definition of eigenvectors:
	$$
	   \bm{S}^{\star} \bm{\xi} = \lambda \bm{\xi}  \quad \Longleftrightarrow \quad
	   \bm{P}^{\star}\left(\bm{\Pi}^{\star}\right)^{-1/2} \bm{\xi} = \lambda \left(\bm{\Pi}^{\star}\right)^{-1/2}  \bm{\xi}.
	$$
	\item[(c)] In particular, $\lambda_{1}(\bm{S}^{\star})=\lambda_{1}(\bm{P}^{\star})=1$,
		and  $\bm{\pi}^{\star}_{1/2}$ is precisely the eigenvector of $\bm{S}^{\star}$ associated with  $\lambda_{1}(\bm{S}^{\star})=1$.
		Thus, from the eigendecomposition of the symmetric matrix $\bm{S}^{\star}$,
		it is easy to see that the eigenvalues of $\bm{S}^{\star} - \bm{\pi}^{\star}_{1/2}(\bm{\pi}^{\star}_{1/2})^{\top}$ are $0,
	\lambda_{2}(\bm{S}^{\star}), \cdots, \lambda_{n}(\bm{S}^{\star})$.
\end{itemize}
Taking the preceding facts collectively, we reach
\begin{align*}
	& \big\Vert \bm{S}^{\star} - \bm{\pi}^{\star}_{1/2}(\bm{\pi}^{\star}_{1/2})^{\top} \big\Vert
	\overset{(\mathrm{i})}{=} \max \big\{ \big| \lambda_{2}(\bm{S}^{\star}) \big|, \big| \lambda_{n}(\bm{S}^{\star}) \big| \big\}  \\
	& \qquad = \max \big\{ \lambda_{2}(\bm{S}^{\star}),-\lambda_{n}(\bm{S}^{\star}) \big\}
	 \overset{\mathrm{(ii)}}{=}\max\big\{ \lambda_{2}(\bm{P}^{\star}),-\lambda_{n}(\bm{P}^{\star}) \big\} .
\end{align*}
Here, (i) relies on Property (c), while (ii) follows from Property (b). This concludes the proof. \end{proof}

%% file: chapters/matrix_recovery.tex
\chapter[Applications of $\ell_2$ perturbation theory to data science]{Applications of $\ell_2$ perturbation theory \\ to data science}
\label{chap:application-L2}

This chapter develops tailored spectral methods for several important applications arising in statistics, machine learning and signal processing.
As it turns out,  these methods are all variations of a common recipe: extracting the information of interest  from the eigenspace (resp.~singular spaces) and eigenvalues (resp.~singular values) of a certain matrix $\bm{M}$ properly constructed from data. The inspiration stems from the observation that: the corresponding quantities of $\bm{M}^{\star}=\mathbb{E}[\bm{M}]$---when properly constructed and under appropriate statistical models---might faithfully reveal the information being sought after.
The classical $\ell_2$ perturbation theory introduced in Chapter~\ref{cha:matrix-perturbation}, when paired with modern probabilistic tools reviewed in Section~\ref{sec:matrix-Bernstein}, uncovers appealing performance of spectral methods in numerous applications by controlling the size of the perturbation $\bm{E} := \bm{M} - \bm{M}^{\star} $.
The vignettes in this chapter provide ample evidence regarding the benefits of
harnessing the statistical nature of the acquired data.

\input{chapters/matrix_concentration.tex}

\input{chapters/matrix_denoising.tex}

\input{chapters/PCA.tex}
\input{chapters/community_detection.tex}
\input{chapters/gaussian_mixture.tex}

\input{chapters/ranking.tex}

\input{chapters/phase_retrieval.tex}

\input{chapters/matrix_completion.tex}

\input{chapters/tensor_completion.tex}

%\input{chapters/gaussian_mixture.tex}

 \section{Notes}

This section provides further pointers to the applications studied in this chapter, and singles out a brief list of applications we have omitted.

Before proceeding, it is worth pointing out several important facts. First, for many applications (e.g., phase retrieval, matrix and tensor completion), the spectral method alone does not allow for perfect reconstruction of the unknowns even when it is information-theoretically feasible to do so.  Instead, the spectral method frequently serves as a suitable initialization step for these applications, and its estimate can often be further refined by means of nonconvex optimization algorithms like gradient descent and alternating minimization; see \citet{chi2019nonconvex,jain2017non} for overviews of recent advances.
Second, throughout this chapter, we have assumed that the underlying matrix is exactly low-rank, and in addition the spectral methods deployed know the correct rank. However, in reality, data matrices are rarely exactly low-rank. It is therefore of great importance to develop and analyze methods that can handle such misspecified cases. When the reconstruction error of the matrix is considered, several methods are capable of achieving graceful tradeoff between the estimation error and the approximation error, without knowing the correct rank, e.g. e.g.,~\citet{koltchinskii2011nuclear,chatterjee2014universal,negahban2011estimation}. 
In addition, further discussions (e.g., convex relaxation approaches and nonconvex landscape analysis) about several of these applications can be found in \citet{candes2014mathematics,wainwright2019high,wright2020high,zhang2020symmetry}.

% study of spectral methods, particularly for the applications studied in this chapter.

%Similar spectral techniques were also proposed in \citep{keshavan2010matrix, jain2013low}, for initializing algorithms for low-rank matrix completion.

\paragraph{PCA, factor models and covariance estimation.}

PCA and factor models are among the most classic and extensively studied topics in statistics \citep{anderson1962introduction,fan2020statistical,wainwright2019high}. The model considered in Section~\ref{sec:formulation-PCA} has been studied by, for example, \citet{johnstone2001distribution,paul2007asymptotics,nadler2008finite, perry2016optimality, xie2018sequential, wang2017asymptotics, fan2018spiked, bao2020statistical}  under the name of  spiked covariance models,  
covering both the finite-sample regime and high-dimensional asymptotics.
A more recent strand of work extended the theory to accommodate heteroskedastic noise and missing data (including heterogeneous missing patterns) \citep{lounici2014high,zhang2018heteroskedastic,cai2019subspace,zhu2019high}, as well as exponential family distributions~\citep{liu2018pca}. In addition to providing the distance control between the spectral estimate and the true principle subspace, \citet{koltchinskii2016asymptotics} and \citet{fan2019distributed} also studied the bias of the spectral estimate under various types of data distributions.
It is clearly impossible to review the enormous literature in a monograph of this length; the interested reader is referred to the overview papers \citet{johnstone2018pca,fan2018PCA,vaswani2018rethinking,balzano2018streaming}  and the recent books
  \citet{fan2020statistical,wainwright2019high} for overviews of contemporary developments on this topic (with particular emphasis on high-dimensional data).
In addition, this monograph does not account for the sparsity structure, or a superposition of low-rank and sparsity structure,  where are commonly imposed on either the covariance matrix or the precision matrix \citep{JohLu09,ma2013sparse,vu2012minimax,cai2013sparse,candes2011robust,chandrasekaran2011rank,chandrasekaran2010latent}. These additional structural assumptions play a crucial role in further dimension reduction and are useful for, say, learning graphical models, video surveillance in computer vision, and portfolio allocation and risk managements in finance;   see \citet{fan2020statistical,ma2016GPCA,wainwright2019high,wright2020high} for more detailed discussions.

% When the covariance of $\bm{\eta}_{i}$ is sparse, the bias can be negligible when the factors are strong, namely, $\lambda_r$ is large.  In this case, conventional PCA continues to apply.  See   \citet[Chapter 10]{fan2020statistical}.

\paragraph{Applications of PCA in statistical and econometric modeling.}   PCA has been widely applied to estimate dimension-reduced spaces in multiple-index models \citep{Li:92, duan1991slicing,cook2007fisher,xia2009adaptive,li2018sufficient}, and latent factors in econometric modeling \citep{forni2000generalized, stock2002forecasting, bai2002determining,bai2003inferential,bai2009panel,ahn2013eigenvalue,fan2015power,fan2016projected}.
For recent reviews of this topic, we refer the readers to \citet{stock2016dynamic} for dynamic factor models with applications to macroeconomics, to \citet{bai2016econometric} for time series and panel data models, to \citet{fan2020robust} for robust factor models and large covariance estimation, and to \citet{fan2021recent}  for  factor models and their broader applications to econometric learning.  In particular, factor models have been frequently employed to adjust correlated covariates in high-dimensional model selection, large-scale inference, predictions, treatment evaluations, among others;  see \citet{fan2020robust, fan2021recent} and the references therein.

\paragraph{Graph clustering and community recovery.}

Spectral methods---possibly coupled with other subsequent refining schemes like
$k$-means---are among the most widely used algorithms for graph clustering
\citep{mcsherry2001spectral,rohe2011spectral,balakrishnan2011noise,chaudhuri2012spectral,fishkind2013consistent,sarkar2015role,jin2015fast,gao2017achieving,zhang2020theoretical,newman2013spectral,chen2015phase,zhang2020detecting,le2015estimating,jin2015fast,le2018concentration,chen2020global}.
While a large fraction of earlier papers required the average vertex degree  to be significantly larger than  $\log n$,
\citet{lei2015consistency} broadened the coverage of the theory by accommodating sparse graphs with average  degrees as low as $O(\log n)$. This, however, should be differentiated from the ultra-sparse regime with average degrees $O(1)$; in this scenario, spectral methods based on  vanilla adjacency matrices no longer work, and  more intelligent designs are needed  to effectively detect the communities \citep{coja2010graph,massoulie2014community,chin2015stochastic,le2015sparse}.
The theory available for spectral clustering extends far beyond the two-community SBM presented herein, examples including SBMs with growing communities \citep{rohe2011spectral}, degree-corrected SBMs \citep{lei2015consistency,lei2014generic}, graphs with locality \citep{chen2016community}, mixed membership models \citet{fan2019simple,han2019universal}, hyper-graphs \citep{ahn2018hypergraph,michoel2012alignment,cole2020exact}, and directed graphs \citep{wang2020spectral}. An abundance of other paradigms, most notably convex relaxation, have also proved effective for clustering \citep{jalali2011clustering,amini2013pseudo,abbe2014exact,hajek2015achieving,cai2015robust,zhao2012consistency,li2018convex,zhang2020theoretical,yuan2018community,fei2018exponential,fei2019achieving}. We recommend the article \citet{abbe2017community} for an overview of recent developments.

%\citep{florescu2016spectral}

 \paragraph{Gaussian mixture models.}

The Gaussian mixture model is among the most classic and convenient statistical models to capture the effect of multi-modal and heterogeneous data
(e.g., \citet{pearson1894contributions,titterington1985statistical,xu1996convergence,dasgupta1999learning,hsu2013learning,kalai2010efficiently,balakrishnan2014statistical,xu2016global,jin2016local,fei2018hidden,jin2017phase,dan2020sharp,han2020eigen}).
Unlike parameter estimation (e.g., estimating the centers) that does not require center separation \citep{wu2020optimal},
the feasibility of reliable clustering in Gaussian mixture models
is dictated by the minimum center separation \citep{lu2016statistical,cai2018rate,ndaoud2018sharp,giraud2019partial,chen2020cutoff}.
While spectral methods naturally come into mind for the clustering task and have been frequently applied in the literature \citep{von2007tutorial,vempala2004spectral,kannan2008spectral,kumar2010kmeans,awasthi2012improved},
sharp statistical analysis of spectral clustering (and its variants)
has been lacking until recently \citep{ndaoud2018sharp,loffler2019optimality,srivastava2019robust,abbe2020ell_p}.
While it might be tempting to impose a minimum spectral gap requirement on the matrix $\bm{\Theta}^{\star}$ (cf. \eqref{eq:defn-Theta-star-F-star-GMM}) in order to invoke the $\sin\bm{\Theta}$ theorems,
such a condition can be dropped as long as an appropriate spectral clustering scheme is employed \citep{loffler2019optimality}.
Encouragingly, spectral clustering (with the aid of $k$-means) also achieves information-theoretically optimal mis-clustering rate exponents for a couple of scenarios  \citep{loffler2019optimality,abbe2020ell_p}.

% \citep{el2008operator}

 \paragraph{Ranking from pairwise comparisons.}

 Deploying spectral methods to address ranking tasks has a long history, dating back at least to \citet{seeley1949net}. We refer the readers to~\citet{vigna2016spectral} for a historical account of this subject. The specific instance of spectral methods  introduced here was due to \citet{negahban2016rank}, and has been subsequently analyzed in multiple papers \citep{rajkumar2014statistical,chen2015spectral,jang2016top,chen2020partial}. It bears close similarity to the celebrated PageRank algorithm heavily used by Google~\citep{page1999pagerank}.  \citet{negahban2016rank} developed the first $\ell_{2}$ statistical guarantees when estimating the underlying score vector, accounting for missing data and general comparison graphs as well. The $\ell_2$ guarantees for random comparison graphs were further sharpened in \citet{chen2015spectral} (which  closed the logarithmic gap).
Note, however, that the $\ell_2$ score estimation error bounds alone typically do not imply the ranking accuracy. Motivated by this inadequacy, \citet{chen2015spectral} directly studied the top-$K$ ranking accuracy, by demonstrating the  optimality of spectral ranking followed by an iterative refinement scheme.
However, this left open another question regarding whether the follow-up refinement step is necessary in achieving optimal ranking accuracy.
\citet{jang2016top} attempted to address this question by establishing desired ranking accuracy of spectral methods when the number of pairwise comparisons available is large. A complete picture was subsequently obtained by \citet{chen2017spectral}, which proved the optimality of spectral methods  in top-$K$ ranking all the way to the sample-starved regime. Moving beyond exact top-$K$ ranking, the recent work \citet{chen2020partial} studied the capability (and limitations) of spectral methods in handling
partial recovery of the top-$K$ ranked items. Moving beyond the BTL model, there are also a number of other ranking models that have been extensively studied in the literature (e.g., the Plackett-Luce model for multi-way comparisons \citep{hunter2004mm,hajek2014minimax,oh2014learning,agarwal2018accelerated}, the stochastically transitive model \citep{shah2016stochastically,shah2019feeling}), which are beyond the scope of the present monograph.

% \citet{vempala2004spectral,awasthi2012improved,cai2018rate,kannan2008spectral,loffler2019optimality,yi2014alternating,abbe2020ell_p}

%\citep{kumar2010kmeans}

\paragraph{Phase retrieval.}

\citet{netrapalli2015phase} proposed the first spectral method (cf.~Section~\ref{sec:pr-alg}) for phase retrieval,
and established the performance guarantees when the sample size exceeds $m\gtrsim n\log^3 n$.
The theoretical support was then tightened by \citet{candes2015phase}, allowing the sample size to be as low as $m \asymp n\log n$.
Similar theory was provided for the random coded diffraction pattern model in \citet{candes2015phase}. Several variations and generalizations of the spectral method have been further proposed to improve performance. The first order-wise optimal spectral method  for phase retrieval was proposed by  \citet{chen2015solving}, based on the truncation idea. This method has multiple variants \citep{zhang2016provable,li2017nonconvex, wang2017solving}, and has been shown to be robust against noise and corruptions. The precise asymptotic characterization of the spectral method was first obtained in \citet{lu2020phase}. Based on this characterization, \citet{mondelli2017fundamental,luo2019optimal} later devised optimal designs of spectral methods in phase retrieval when the sensing matrix follows the Gaussian design, where its sensitivity to model mismatch was studied in \citet{monardo2019sensitivity}. \citet{ma2019spectral,dudeja2020analysis} explored similar questions when the sensing matrix is Haar distributed (e.g., an isotropically random unitary matrix).
The spectral method presented herein has been used to seed a follow-up procedure that in turn enhances estimation accuracy; see, e.g., \citet{netrapalli2015phase,candes2015phase,sanghavi2017local,goldstein2018phasemax,bahmani2016phase,ma2017implicit,chandra2019phasepack,dhifallah2017phase,qu2019convolutional,zhang2017nonconvex,ma2019optimization,tan2019phase,jeong2017convergence,cai2019fast,salehi2018precise}. An alternative to the spectral method, based on a nullspace approach, has been proposed in \citet{chen2017phase}. \citet{fannjiang2020numerics} provided an extensive discussion on initialization strategies for algorithmic phase retrieval, including but not limited to various forms of spectral methods.
Sparse phase retrieval is another important topic when the signal of interest is assumed to be a sparse vector; we refer the interested reader to
\citet{li2013sparse,oymak2012simultaneously,chen2015exact,cai2016optimal,wang66sparse,jagatap2019sample,yang2019misspecified,soltanolkotabi2017structured,yang2016sparse,zhang2018compressive,salehi2018learning,shechtman2014gespar,yuan2019phase,eldar2014phase} and additional references cited therein.

\paragraph{Matrix completion.}

Regarding matrix completion, the spectral method was originally proposed in \citet{achlioptas2007fast,keshavan2010matrix} to estimate (approximately) low-rank matrices in the face of missing data and random corruptions.
Similar to phase retrieval, 
the estimate returned by the spectral method is employed as a suitable initialization to enable fast convergence of nonconvex iterative procedures;
see, e.g., \citet{keshavan2010matrix,keshavan2009matrix,jain2013low,hardt2014understanding,sun2016guaranteed,chen2015fast,zheng2016convergence,boumal2015low,wei2016guarantees,chen2020nonconvex,ma2017implicit,zhang2018fast,jin2016provable,charisopoulos2019lowrank}.
Moreover, there are several nuclear norm penalized estimators that also bear close relevance to the spectral method,
e.g., \citet{koltchinskii2011nuclear}.
We also remark in passing that there are other estimators that can effectively handle the case when the underlying matrix is not exactly low-rank, including but not limited to Universal Singular Value Thresholding \citep{chatterjee2014universal} and its soft-thresholded version \citep{koltchinskii2011nuclear}.
In addition, while our discussion focuses on clean data and uniform random sampling patterns,  it is of great importance to study various noisy and quantized scenarios \citep{keshavan2009matrix,candes2010matrix,cao2015poisson,klopp2014noisy,chen2015fast,mcrae2019low,davenport20141,ma2017implicit,zhang2018primal,krahmer2019convex}, 
as well as non-uniform or deterministic sampling patterns \citep{foucart2019weighted,negahban2012restricted,shapiro2018matrix}.

% We have to omit these due to the space limit.

 \paragraph{Tensor completion and estimation.}

Unfolding-based spectral methods  have been frequently adopted to deal with various tensor estimation problems including tensor PCA, tensor decomposition, tensor completion,
and so on \citep{richard2014a,montanari2016spectral,han2020an,zhang2018tensor,cai2019subspace,cai2019tensorOR,xia2019on,xia2017statistically,moitra2019spectral,liu2020tensor,zhang2020islet,xia2020inference,tong2021scaling}.
When it comes to tensor completion, the first near-optimal $\ell_2$ statistical analysis of spectral methods was due to \citet{montanari2016spectral},
which was subsequently extended by \citet{cai2019subspace} to enable $\ell_{2,\infty}$ error control.
The readers interested in higher-order tensors (beyond third-order tensors) can consult \citet{montanari2016spectral,richard2014a}.
In addition, the theory and algorithm presented herein focus attention on the regime where $r< n$, and fall short of accommodating ``over-complete'' tensors when $r$ rises above $n$. Certain ``contraction'' tricks are needed in order to cope with the over-complete regime; see \citet{hopkins2016fast,montanari2016spectral}.

\paragraph{An extensive but non-exhaustive list of other applications.} Finally, the list of applications discussed in this monograph is clearly far from comprehensive. Spectral methods have been successfully applied to a plethora of other problems,  including but not limited to the following topics:
\begin{itemize}
\setlength\itemsep{0.2em}
\item { matrix sensing:} \citet{tu2016low,zheng2015convergent,ma2019beyond,chen2020learning,tong2020low,tong2020accelerating,lee2017near};
\item { phase synchronization and group synchronization:} \citet{singer2011angular,abbe2020entrywise,ling2020near};
\item { joint matching and map synchronization:} \citet{chen2014near,pachauri2013solving,shen2016normalized,bajaj2018smac,sun2018joint,sun2019k,huang2019tensor,huang2019learning};
\item { covariance sketching and quadratic sensing:} \citet{li2019nonconvex,charisopoulos2019lowrank,sanghavi2017local,chi2017subspace};
%\item { ranking from pairwise comparisons:}
\item { blind deconvolution and blind calibration:} \citet{li2016deconvolution,ma2017implicit,huang2018blind,charisopoulos2019composite,chen2020convex,li2018blind,cambareri2016non};
\item { blind demixing:} \citet{ling2019regularized,dong2018nonconvex};
\item { low-rank phase retrieval and phaseless PCA:} \citet{vaswani2017low,nayer2019phaseless,vaswani2020nonconvex};
\item { canonical correlation analysis (CCA):} \citet{cai2018rate,ge2016efficient};
\item { sparse PCA:} \citet{amini2008high,JohLu09};
\item { mixed linear regression:} \citet{yi2014alternating,ghosh2020alternating,kwon2020minimax};
\item { finding hidden cliques:} \citet{alon1998finding};
\item { joint image alignment:} \citet{chen2016projected};
\item { robust subspace recovery and robust PCA:} \citet{yi2016fast,netrapalli2014non,cherapanamjeri2017nearly,tong2020low,zhu2018dual,maunu2019well};
\item { contextual stochastic block models:} \citet{binkiewicz2017covariate,abbe2020ell_p};
\item { learning neural networks:} \citet{zhong2017recovery,fu2020guaranteed};
%\item { inference for network data:} \citet{fan2019simple};
\item { topic modeling:} \citet{ke2017new};
\item { crowd sourcing:} \citet{ghosh2011moderates,dalvi2013aggregating,karger2013efficient,zhang2014spectral,karger2014budget};
\item { meta learning:} \citet{kong2020meta,du2020few,tripuraneni2020provable};
\item { subspace clustering:} \citet{eriksson2012high,li2019theory};
\item { state aggregation and compression of Markov chains:} \citet{zhang2020spectral,duan2019state};
\item {causal inference}: \citet{amjad2018robust};
\item { passive imaging:} \citet{lee2018spectral}.
\end{itemize}
For the sake of conciseness, we have chosen not to detail these applications,
but instead recommend the interested reader to the above articles and the references therein.

%% file: chapters/matrix_concentration.tex
\section{Preliminaries: Matrix tail bounds}
\label{sec:matrix-Bernstein}

In order to invoke the sin$\bm{\Theta}$ theorems (Theorems~\ref{thm:davis-kahan} and \ref{thm:wedin}), an important ingredient lies in developing a tight upper bound on the spectral norm $\|\bm{E}\|$ of the perturbation matrix $\bm{E}$. This is where statistical/probabilistic tools play a major role.
Rather than presenting an encyclopedia of probabilistic techniques (which can be gleaned from \citet{tropp2015introduction,vershynin2016high,boucheron2013concentration,wainwright2019high,tropp2011freedman,raginsky2013concentration,howard2020time}), this monograph singles out only two useful matrix concentration inequalities that suffice for the applications considered herein.

\subsubsection{The (truncated) matrix Bernstein inequality}
The first result is an extension of the celebrated matrix Bernstein inequality \citep{oliveira2009concentration, tropp2012user,hopkins2016fast}. 
This is an elegant and convenient tail bound for the sum of independent random matrices, resulting in effective performance guarantees for a diverse array of statistical applications.   We refer the interested reader to \citet{tropp2015introduction} for a highly accessible introduction of the classical matrix Bernstein inequality,
and \citet[Section~A.2.2]{hopkins2016fast} for a proof of the truncated variant stated in Theorem~\ref{thm:matrix-Bernstein}.

\begin{theorem}[Truncated matrix Bernstein]
\label{thm:matrix-Bernstein}
Let $\{\bm{X}_{i}\}_{1\leq i\leq m}$ be a sequence of independent real random matrices with dimension $n_{1}\times n_{2}$. Suppose that for all $1\leq i\leq m$,
%
%\begin{align}
%\mathbb{E}\left[\bm{X}_{i}\right]=\bm{0}
%	\quad \text{and} \quad
%	\left\Vert \bm{X}_{i}\right\Vert \leq L\qquad\text{for all } i.
%	\label{eq:mean-zero-bound-B}
%\end{align}
%
\begin{subequations} 	\label{eq:mean-zero-bound-B-truncated}
\begin{align}
	\mathbb{P}\big\{\left\Vert \bm{X}_{i}- \mathbb{E}\left[\bm{X}_{i}\right] \right\Vert \geq L \big\} & \leq q_{0}  \\
	\big\| \mathbb{E}\left[\bm{X}_{i}\right]- \mathbb{E}\left[\bm{X}_{i}\mathbbm1\big\{\left\Vert \bm{X}_{i}\right\Vert < L \big\}\right] \big\| & \leq q_{1}
\end{align}	
\end{subequations}
hold for some quantities $0 \leq q_0 \leq 1$ and $q_1 \geq 0$. In addition, define  the matrix variance statistic $v$  as
\begin{align}
	v \coloneqq
	\max\Bigg\{ & \Bigg\Vert \sum_{i=1}^m\mathbb{E}\Big[(\bm{X}_{i}-\mathbb{E}\left[\bm{X}_{i}\right])(\bm{X}_{i}-\mathbb{E}\left[\bm{X}_{i}\right])^{\top}\Big]\Bigg\Vert , \nonumber \\
 &	\qquad   \Bigg\Vert \sum_{i=1}^m\mathbb{E}\Big[(\bm{X}_{i}-\mathbb{E}\left[\bm{X}_{i}\right])^{\top}(\bm{X}_{i}-\mathbb{E}\left[\bm{X}_{i}\right])\Big]\Bigg\Vert \Bigg\} .
	\label{defn:variance-statistic}
\end{align}
Then for all $t\geq  mq_1$, one has
\[
	\mathbb{P}\left( \left\Vert \sum_{i=1}^m(\bm{X}_{i}-\mathbb{E}\left[\bm{X}_{i}\right])\right\Vert \geq t\right)\leq\left(n_{1}+n_{2}\right) \exp\left(\frac{-(t- mq_1)^{2}/2}{v+L(t-mq_1)/3}\right) + mq_0.
\]
\end{theorem}
\begin{remark}
Note that when the $\bm{X}_{i}$'s are  i.i.d.~zero-mean random matrices, the matrix variance statistic simplifies  to
$$
  v = m \max \Big\{ \big\Vert \mathbb{E}\big[\bm{X}_{i}\bm{X}_{i}^{\top}\big]\big\Vert,
	\big\Vert \mathbb{E}\big[\bm{X}_{i}^{\top} \bm{X}_{i}\big]\big\Vert \Big\}.
$$
\end{remark}
% When specialized to the i.i.d.~vector case with $n_2 = 1$ and zero mean, one has $v = m \mathsf{tr}(\mathbb{E}[ \bX_i \bX_i^\top])$,
% which is $m$ times the trace of the covariance matrix of $\bX_i$.

To make it more user-friendly, we record a straightforward consequence of Theorem~\ref{thm:matrix-Bernstein} as follows.
\begin{corollary}
\label{thm:matrix-Bernstein-friendly-truncated}
	Suppose the assumptions of Theorem~\ref{thm:matrix-Bernstein} hold, and set $n\coloneqq\max\{n_1,n_2\}$. For any $a\geq 2$, with probability exceeding $1-2n^{-a+1}-mq_0$ one has
	\begin{align}
		\left\Vert \sum_{i=1}^m (\bm{X}_{i} - \mathbb{E}\left[\bm{X}_{i}\right]) \right\Vert \leq \sqrt{2a v \log n } + \frac{2a}{3}L\log n  + mq_1.
	\end{align}
\end{corollary}
In order to make effective use of the above results (particularly when handling unbounded random matrices), it is advisable to take $L$ as a high-probability bound on $\|\bm{X}_i -  \mathbb{E}\left[\bm{X}_{i}\right]  \|$.
The rationale is simple: by properly truncating $\bm{X}_i$ based on the level $L$,
we end up with a {\em bounded} sequence that is more convenient to work with while not deviating much  from the original sequence.
In particular,  if all $\|\bm{X}_i-  \mathbb{E}\left[\bm{X}_{i}\right] \|$ are bounded by a deterministic quantity which is set to be $L$,
then both $q_0$ and $q_1$ vanish, thus eliminating the need of enforcing truncation.
In this case, Corollary~\ref{thm:matrix-Bernstein-friendly-truncated} simplifies to a user-friendly version of the standard matrix Bernstein inequality, which we record below for ease of reference.
\begin{corollary}[Matrix Bernstein]
\label{thm:matrix-Bernstein-friendly}
		Let $\{\bm{X}_{i}\}_{1 \leq i \leq m}$
be a set of independent real random matrices with dimension $n_{1}\times n_{2}$. Suppose that
\begin{align}
\mathbb{E}\left[\bm{X}_{i}\right]=\bm{0},
	\quad \text{and} \quad
	\left\Vert \bm{X}_{i}\right\Vert \leq L,\qquad\text{for all } i.
	\label{eq:mean-zero-bound-B}
\end{align}
Set $n\coloneqq\max\{n_1,n_2\}$, and recall the definition of variance statistic in~\eqref{defn:variance-statistic}. For any $a\geq 2$, with probability exceeding $1-2n^{-a+1}$ one has
	\begin{align}
		\Bigg\Vert \sum_{i=1}^m \bm{X}_{i} \Bigg\Vert \leq \sqrt{2a v \log n } + \frac{2a}{3}L\log n.
	\end{align}
\end{corollary}

By virtue of the above  inequalities, the key to bounding  $\left\Vert \sum\nolimits_{i}\bm{X}_{i}\right\Vert$ largely lies in controlling the following  two crucial quantities:
\begin{align*}
	\sqrt{v\log n} \qquad \text{and} \qquad L\log n,
\end{align*}
where the former depends on the number $m$ of random matrices involved.

\subsubsection{Spectral norm of random matrices with independent entries}
An important family of random matrices that merits special attention comprises the ones with independent random entries,
that is, matrices of the form $\bm{X}=[X_{i,j}]_{1\leq i,j\leq n}$ with independent $X_{i,j}$'s.
While the spectral norm of such a matrix can also be analyzed via matrix Bernstein (by treating $\bm{X}$ as the sum of independent random matrices $X_{i,j}\bm{e}_i\bm{e}_j^{\top}$),
this approach is typically loose in terms of the logarithmic factor.
Motivated by the abundance of such random matrices in practice,
we record below a strengthened non-asymptotic spectral norm bound, which is of significant utility and is tighter than what matrix Bernstein has to offer for this case.
\begin{theorem}
	\label{thm:tighter-spectral-normal-ramon}
	Consider a \emph{symmetric} random matrix $\bm{X}=[X_{i,j}]_{1\leq i,j\leq n}$ in $\mathbb{R}^{n\times n}$, whose entries are independently generated and obey		%
	\begin{align}
		\label{eq:Xij-condition-iid}
		\mathbb{E}[X_{i,j}]=0, \quad \text{and} \quad  |X_{i,j}|\leq B,  \qquad 1\leq i,j\leq n.
	\end{align}	
	Define
	\begin{align}
		\label{eq:Xij-row-sum-variance}
		\nu \coloneqq \max_i \sum\nolimits_j \mathbb{E}[X_{i,j}^2].
	\end{align}
	Then there exists some universal constant $c>0$ such that for any $t\geq 0$,
	\begin{align} \label{eq3.8}
		\mathbb{P}\Big\{ \|\bm{X}\|\geq4\sqrt{\nu}+t \Big\} \leq n\exp\Big(-\frac{t^{2}}{cB^{2}}\Big).		
	\end{align}
\end{theorem}

This result, which appeared in \citet[Remark 3.13]{bandeira2016sharp}, can be established via  tighter control of the expected spectral norm in conjunction with Talagrand's concentration inequality. Two remarks are in order.
\begin{itemize}
\item
First, it is easy to see that the result extends to asymmetric matrices with independent entries, using the standard ``dilation trick'' (see, e.g., \citep[Section~2.1.17]{tropp2015introduction}). Specifically, for an asymmetric random matrix $\bm{X}\in\mathbb{R}^{n_1\times n_2}$, let us introduce the symmetric dilation $\mathcal{S}(\bm{X})$ of $\bm{X}$:
\[
\mathcal{S}(\bm{X})\coloneqq\left[\begin{array}{cc}
\bm{0} & \bm{X}\\
\bm{X}^{\top} & \bm{0}
\end{array}\right] \in \mathbb{R}^{ (n_1+n_2) \times (n_1+ n_2)},
\]
which enjoys the desired symmetry and can be analyzed directly using Theorem~\ref{thm:tighter-spectral-normal-ramon}. The resulting bound on $\|\mathcal{S}(\bm{X})\|$ can be translated back to $\|\bm{X}\|$ via the elementary identity $\|\bm{X}\| = \|\mathcal{S}(\bm{X})\|$. For conciseness, we will occasionally apply Theorem~\ref{thm:tighter-spectral-normal-ramon} directly to asymmetric matrices without invoking the dilation trick.

\item As a useful corollary, if we know {\em a priori} that $\mathbb{E}[X_{i,j}^2]\leq \sigma^2$ for all $1\leq i,j\leq n$, then Theorem~\ref{thm:tighter-spectral-normal-ramon} implies that
\begin{align}
	\|\bm{X}\| \leq 4\sigma \sqrt{n} + \widetilde{c} B \sqrt{ \log n }
	\label{eq:X-spectral-norm-iid-special-ramon}
\end{align}
with probability at least $1-n^{-8}$ for some  constant $\widetilde{c} >0$. To see this, it suffices to set $\widetilde{c} = \sqrt{9c}$ and take $t = B\sqrt{9c \log n} $ in \eqref{eq3.8}.

\end{itemize}

\begin{remark}
	The inequality \eqref{eq:X-spectral-norm-iid-special-ramon} continues to hold if we replace $n^{-8}$ with $n^{-\alpha}$ for any positive constant $\alpha>0$.
	Here and below, we often go with the artificial choice like $n^{-8}$ since it is small enough for our purpose.
\end{remark}

%% file: chapters/matrix_denoising.tex
\section{Low-rank matrix denoising}
\label{sec:matrix-denoising-L2}

To catch a glimpse of the effectiveness of the approach we have introduced,
let us start by trying it out on a warm-up example: low-rank matrix denoising.

\subsection{Problem formulation and algorithm}\label{sec:matrix-denoising-setup}  Consider an unknown rank-$r$ symmetric matrix $\bm{M}^{\star} \in \mathbb{R}^{n\times n}$ with eigendecomposition $\bm{M}^{\star}=\bm{U}^{\star}\bm{\Lambda}^{\star}\bm{U}^{\star\top}$, where the columns of $\bm{U}^{\star}\in \mathbb{R}^{n\times r}$ are orthonormal, and $\bm{\Lambda}^{\star}\in \mathbb{R}^{r\times r}$ is a diagonal matrix containing the nonzero eigenvalues $\{\lambda_{i}^{\star}\}$ of $\bm{M}^{\star}$. Assume that $|\lambda_1^{\star}|\geq |\lambda_2^{\star}|\geq \cdots \geq |\lambda_r^{\star}|>0$.  Suppose that  we observe a noisy copy
\[
\bm{M}=\bm{M}^{\star}+\bm{E},
\]
 where $\bm{E}=[E_{i,j}]_{1\leq i,j\leq n}$ is a symmetric noise matrix. It is assumed that the entries $\{{E}_{i,j}\}_{i\geq j}$ are independently generated obeying
\begin{align}
	E_{i,j} \overset{\mathrm{i.i.d.}}{\sim} \mathcal{N}(0,\sigma^2) , \qquad i\geq j.
\end{align}
The aim is to estimate the eigenspace $\bm{U}^{\star}$ from the data matrix $\bm{M}$. Despite its simplicity, this problem  has been extensively studied in the literature \citep{koltchinskii2016perturbation,bao2018singular,ding2020high,xia2019normal,li2021minimax}. It also  bears close relevance to the famous angular/phase synchronization problem \citep{singer2011angular,bandeira2017tightness}.

In order to estimate the low-rank factors specified by $\bm{U}^{\star}$, a natural scheme is to resort to the rank-$r$ leading eigenspace of the data matrix $\bm{M}$.  More precisely, denote by $\lambda_1,\cdots,\lambda_n$ the eigenvalues of $\bm{M}$ sorted by their magnitudes, i.e.,
\begin{align}
	|\lambda_1| \geq |\lambda_2| \geq \cdots \geq |\lambda_n|,
\end{align}
and let $\bm{u}_1$, $\cdots$, $\bm{u}_n$ represent the associated eigenvectors.
This spectral method returns  $\bm{U} = [\bm{u}_1,\cdots,\bm{u}_r]\in \mathbb{R}^{n\times r}$ as an estimate of $\bm{U}^{\star}$.

\subsection{Performance guarantees}
\label{sec:performance-denoising-L2}

\paragraph{Statistical accuracy of the spectral estimate.}

We now examine the accuracy of the above spectral estimate. Towards this,
a key step lies in bounding the spectral norm of the noise matrix
$\bm{E}$. We claim for the moment that (which will be established in Section~\ref{sec:proof-eqn-noise-matrix-E-bound-denoising})
\begin{align}
	\|\bm{E}\| \leq 5\sigma\sqrt{n}
	\label{eq:noise-matrix-E-bound-denoising}
\end{align}
with probability at least $1-O(n^{-8})$.
Armed with this claim and the fact $\lambda_{r+1}^{\star}=0$,
we are in a situation where it is quite easy to see how the Davis-Kahan theorem applies.
According to Corollary \ref{cor:davis-kahan-conclusion-corollary}, with probability greater than $1-O(n^{-8})$ one has
\begin{align}
	\mathsf{dist}\big(\bm{U},\bm{U}^{\star}\big)\leq\frac{2\|\bm{E}\|}{|\lambda_{r}^{\star}|}\leq\frac{10\sigma\sqrt{n}}{|\lambda_{r}^{\star}|} ,
	\label{eq:dist-U-Ustar-denoising-L2}
\end{align}
provided that the noise variance is sufficiently small obeying $ \sigma \sqrt{n} \leq \frac{1-1/\sqrt{2}}{5} |\lambda_r^{\star}|$ so that $\|\bm{E}\|\leq (1-1/\sqrt{2})|\lambda_r^{\star}|$.

\paragraph{Tightness and optimality.}
The tightness of the statistical guarantee~\eqref{eq:dist-U-Ustar-denoising-L2} can be assessed when compared with the minimax lower bound. For instance, it is well-known in the literature (e.g., \citet[Theorem 3]{cheng2020tackling}) that: even for the case with $r=1$, one cannot hope to achieve $\mathsf{dist}\big(\widehat{\bm{U}},\bm{U}^{\star}\big)=o(\sigma \sqrt{n}/|\lambda_{r}^{\star}| )$---in a minimax sense---regardless of the estimator $\widehat{\bm{U}}$ in use. 
Consequently, the spectral method turns out to be orderwise statistically optimal for low-rank matrix denoising.

%(in addition to the proof of Corollary \ref{cor:davis-kahan-conclusion-corollary})

\paragraph{Additional useful results: eigenvalue and matrix estimation.}
	Before concluding, we record several immediate consequences of the above analysis that will be useful later on. Specifically, assuming that $ \sigma \sqrt{n} \leq \frac{1-1/\sqrt{2}}{5} |\lambda_r^{\star}|$, we see from Weyl's inequality (cf.~Lemma \ref{lemma:weyl}) that
	\begin{align}
		|\lambda_{i}| & \leq\|\bm{E}\|\leq5\sigma\sqrt{n},
		\qquad
		\text{for all }i\geq r+1
		\label{eq:lambda-i-bound-E-denoising}
	\end{align}
	%
	% and \cm{This is wrong.}
	% %
	% \begin{align}
	% 	|\lambda_{i} - \lambda_i^{\star}| & \leq\|\bm{E}\|\leq5\sigma\sqrt{n}, \qquad
	% 	\text{for all }i\leq r
	% 	\label{eq:lambda-1-bound-E-denoising}
	% \end{align}
	%
with probability $1-O(n^{-8})$.

We further remark on the Euclidean statistical accuracy when estimating the unknown matrix $\bm{M}^{\star}$ using $\widehat{\bm{M}} \coloneqq \bm{U}\bm{\Lambda}\bm{U}^{\top}$, where $\bm{\Lambda} \coloneqq \mathsf{diag}\big([\lambda_{1},\cdots,\lambda_{r}]\big)$.
It is seen from the triangle inequality that
\begin{align}
\big\|\bm{U}\bm{\Lambda}\bm{U}^{\top}-\bm{M}^{\star}\big\| & \leq\big\|\bm{M}-\bm{M}^{\star}\big\|+\big\|\bm{U}\bm{\Lambda}\bm{U}^{\top}-\bm{M}\big\|\nonumber \\
 & =\big\|\bm{E}\big\|+\big|\lambda_{r+1}\big|\leq2\big\|\bm{E}\big\|,
	\label{eq:ULambdaU-Mstar-norm-denoising}
\end{align}
where the last inequality relies on \eqref{eq:lambda-i-bound-E-denoising}.
Since the rank of $\bm{U}\bm{\Lambda}\bm{U}^{\top}-\bm{M}^{\star}$ is at most $2r$,
with probability at least $1-O(n^{-8})$ one has
\begin{align}
	\big\|\bm{U}\bm{\Lambda}\bm{U}^{\top}-\bm{M}^{\star}\big\|_{\mathrm{F}}
	& \leq\sqrt{2r} \, \big\|\bm{U}\bm{\Lambda}\bm{U}^{\top}-\bm{M}^{\star}\big\|  \leq 2\sqrt{2r}\, \|\bm{E}\| \nonumber \\
	& \leq 10 \sigma\sqrt{2nr}.
	\label{eq:top-r-approx-Fro-bound}
\end{align}

\subsection{Proof of the inequality \eqref{eq:noise-matrix-E-bound-denoising} on $\|\bm{E}\|$}
\label{sec:proof-eqn-noise-matrix-E-bound-denoising}

We plan to employ Theorem~\ref{thm:tighter-spectral-normal-ramon}.
Given that Gaussian entries are unbounded, we introduce a truncated copy $\widetilde{\bm{E}}=[\widetilde{E}_{i,j}]_{1\leq i,j\leq n}$ defined as follows
\begin{equation}
	\widetilde{E}_{i,j}\coloneqq E_{i,j} \mathbbm{1}\big\{|E_{i,j}|\leq 5\sigma\sqrt{\log n}\big\}, \quad 1\leq i,j\leq n.
	\label{eq:defn-Eij-truncated-denoising}
\end{equation}
Two properties are in place.
\begin{itemize}
\item It is readily seen from the property of Gaussian distributions
that
\[
\mathbb{P}\big\{ E_{i,j}=\widetilde{E}_{i,j}\big\}\geq1-n^{-12},\qquad1\leq i,j\leq n,
\]
which combined with the union bound leads to
\begin{equation}
\mathbb{P}\big\{\bm{E}=\widetilde{\bm{E}}\big\}\geq1-n^{-10}.\label{eq:prob-E-Etilde-equivalence-denoising}
\end{equation}
\item Given that $B\coloneqq\max_{i,j}|\widetilde{E}_{i,j}|\leq 5\sigma\sqrt{\log n}$,
we can invoke Theorem~\ref{thm:tighter-spectral-normal-ramon} (or more directly,
(\ref{eq:X-spectral-norm-iid-special-ramon})) to demonstrate that
\[
\|\widetilde{\bm{E}}\|\leq4\sigma\sqrt{n}+O(B\log n)\leq5\sigma\sqrt{n}
\]
for sufficiently large $n$, 
with probability exceeding $1-O(n^{-8})$. Here, we implicitly use the fact that $\mathbb{E}[\widetilde{{E}}_{i,j}^2] \leq \mathbb{E}[E_{i,j}^2 ] = \sigma^2$.
\end{itemize}
Combining the above two observations implies that
\[
\|\bm{E}\|=\|\widetilde{\bm{E}}\|\leq5\sigma\sqrt{n}
\]
 with probability exceeding $1-O(n^{-8})$, as claimed.
 %\end{proof}

%\begin{remark}
%	Before concluding, we record several immediate consequences of the above analysis (in addition to the proof of Corollary \ref{cor:davis-kahan-conclusion-corollary}) that will prove useful later on. Specifically, assuming that $ \sigma \sqrt{n} \leq \frac{1-1/\sqrt{2}}{5} |\lambda_r^{\star}|$ , we have --- with probability $1-O(n^{-8})$ --- that
%	%
%	\begin{align}
%		|\lambda_{i}| & \leq\|\bm{E}\|\leq5\sigma\sqrt{\log n},
%		\qquad
%		\text{for all }i\geq r+1.
%		\label{eq:lambda-i-bound-E-denoising}
%	\end{align}
%	%
%	In the special case where $r=1$, with probability $1-O(n^{-8})$ one has
%	%
%	\begin{align}
%		|\lambda_{1} - \lambda_1^{\star}| & \leq\|\bm{E}\|\leq5\sigma\sqrt{\log n}.
%		\label{eq:lambda-1-bound-E-denoising}
%	\end{align}
%	%
%\end{remark}

%% file: chapters/PCA.tex
\section{Principal component analysis and factor models}
\label{sec:PCA}

Principal component analysis (PCA) and factor models  \citep{jolliffe1986principal,lawley1962factor,fan2020statistical}---which serve as an effective unsupervised learning tool for exploring and understanding data---arise frequently in data-intensive applications in economics,  finance, psychology, signal processing, speech, neuroscience, traffic data analysis, among other things \citep{stock2002forecasting,mccrae1992introduction, scharf1991svd, chen2015reduced, balzano2018streaming,fan2020robust}.
PCA and factor models not only allow for dimensionality reduction,
but also provide intermediate means for data visualization, noise removal, anomaly detection, and other downstream tasks.
In this section, we investigate a simple, yet broadly applicable, factor model.

\subsection{Problem formulation and assumptions}
\label{sec:formulation-PCA}

Dependence of high-dimensional measurements is a stylized feature in data science.  To model the dependence among observed high-dimensional data, we assume that there are latent factors that drive the dependence, with a loading matrix that describes how each component depends on the latent factors and an idiosyncratic noise that captures the remaining part.
To set the stage, imagine we have collected a set of $n$ independent sample vectors
$\bm{x}_{i} \in \mathbb{R}^{p}$, $1\leq i\leq n$ obeying
\begin{equation}
	\bm{x}_{i}=\bm{L}^{\star}\bm{f}_{i}+\bm{\eta}_{i},\qquad1\leq i\leq n.
	\label{eq:samples-PCA}
\end{equation}
Here, $\bm{f}_{i}\in\mathbb{R}^{r}$ is a vector of latent factors, $\bm{L}^{\star}\in\mathbb{R}^{p\times r}$ represents a factor
loading matrix that is not known {\em a priori}, whereas $\bm{\eta}_{i}\in\mathbb{R}^{p}$
stands for additive random noise or the idiosyncratic part that cannot be explained by the latent factor $\bm{f}_i$.
Informally, the samples $\{\bm{x}_{i}\}$ are, in some sense, assumed to be approximately embedded in a low-dimensional subspace encoded by the loading matrix $\bm{L}^{\star}$,
which describes how each component of data $\bm{x}_{i}$ depends on the factor $\bm f_i$ and captures the inter-dependency across different variables.
In the language of PCA, the subspace spanned by $\bm{L}^{\star}$ specifies the $r$ principal components underlying this sequence of data samples.
A common goal thus amounts to estimating the subspace spanned by the loading matrix $\bm{L}^{\star}$ and the latent factors $\{\bm{f}_i\}$. 
In the PCA literature, the subspace  represented by $\bm{L}^{\star}$ is commonly referred  to as the principal subspace.

In this monograph, we concentrate on the following tractable statistical model for pedagogical reasons.
See  \citet[Chapter 10]{fan2020statistical} for more general settings (including, say, heavy-tailed distributions and non-isotropic noise covariance matrices).
\begin{assumption}
\label{assumption:factor-model-f-eta}
The vectors $\bm{f}_{i}$ and $\bm{\eta}_{i}$ ($1\leq i\leq n$) are all independently generated according to
\begin{equation}
	\bm{f}_{i}\overset{\mathrm{i.i.d.}}{\sim}\mathcal{N}(\bm{0},\bm{I}_{r}),
	\qquad \text{and} \qquad
	\bm{\eta}_{i}\overset{\mathrm{i.i.d.}}{\sim}\mathcal{N} \big( \bm{0}, \sigma^2\bm{I}_{p} \big).
	\label{eq:random-model-fi-eta-i}
\end{equation}
\end{assumption}
Moreover, we assume without loss of generality that $\bm{L}^{\star}=\bm{U}^{\star} (\bm{\Lambda}^{\star})^{1/2}$, where the columns of $\bm{U}^{\star}\in \mathbb{R}^{p\times r}$ are composed of orthonormal vectors, and $\bm{\Lambda}^{\star}=\mathsf{diag}\big( [\lambda_1^{\star}, \cdots, \lambda_r^{\star}] \big)$ is an $r$-dimensional diagonal matrix obeying $\lambda_1^{\star}\geq \cdots \geq \lambda_r^{\star}>0$.
Throughout this section, we denote by $$\kappa \coloneqq \lambda_1^{\star}\,/\,\lambda_r^{\star}$$ the condition number of the low-rank matrix $\bm{L}^{\star}\bm{L}^{\star\top}= \bm{U}^{\star}\bm{\Lambda}^{\star}\bm{U}^{\star\top}$.

\subsection{Algorithm}
\label{sec:algorithm-PCA}

As a starting point, it is readily seen under Assumption~\ref{assumption:factor-model-f-eta} that
\begin{align}
	\bm{x}_{i}\sim\mathcal{N}\big(\bm{0},\bm{M}^{\star}\big)
	\quad\quad
	\text{with }\bm{M}^{\star} \coloneqq \bm{U}^{\star}\bm{\Lambda}^{\star}\bm{U}^{\star\top} + \sigma^2 \bm{I}_p.
	\label{eq:defn-Mstar-PCA}
\end{align}
In brief, the covariance matrix $\bm{M}^{\star}$  is a low-rank matrix superimposed by a scaled identity matrix;
for this reason, this model is also frequently referred to as the spiked covariance model \citep{johnstone2001distribution}.
The key takeaway is that  the top-$r$ eigenspace of the covariance matrix $\bm{M}^{\star}$ in \eqref{eq:defn-Mstar-PCA} coincides with the $r$-dimensional principal subspace being sought after (i.e., the one spanned by $\bm{L}^{\star}$ or $\bm{U}^{\star}$).

 The above observation motivates a simple spectral algorithm, which begins by computing a sample covariance matrix
\begin{equation}
	\bm{M} \coloneqq \frac{1}{n} \sum_{i=1}^{n}\bm{x}_{i}\bm{x}_{i}^{\top},
	\label{eq:defn-sample-covariance}
\end{equation}
followed by  computation of the rank-$r$ eigendecomposition $\bm{U}\bm{\Lambda}\bm{U}^{\top}$
of $\bm{M}$. Here, $\bm{\Lambda}\in\mathbb{R}^{r\times r}$ is a
diagonal matrix whose diagonal entries entail the $r$ largest eigenvalues $\lambda_{1}\geq\cdots\geq\lambda_{r}$
of $\bm{M}$, and $\bm{U}\coloneqq[\bm{u}_{1},\cdots,\bm{u}_{r}]\in\mathbb{R}^{p\times r}$
with $\bm{u}_{i}$ representing the eigenvector of $\bm{M}$ associated with $\lambda_{i}$.
The spectral algorithm studied herein then returns $\bm{U}$ as the  estimate for the principal subspace $\bm{U}^{\star}$.

\begin{remark}
	In the presence of missing data or heteroskedastic noise (meaning that the  variance of the noise entries varies across different entries),  the second part of the covariance matrix $\bm{M}^{\star}$ (i.e., $\sigma^{2}\bm{I}_{p}$ in~\eqref{eq:defn-Mstar-PCA}) might no longer be a scaled identity. Under such circumstances,
one might need to carefully adjust the diagonal entries of  $\bm{M}$ in order for the  algorithm to succeed; see, e.g., \citet{lounici2014high,loh2012high,zhang2018heteroskedastic,cai2019subspace,zhu2019high,yan2021inference}. The reader might  consult Section~\ref{sec:tensor-completion} for an introduction to a commonly adopted diagonal deletion idea to address the aforementioned issue.  
\end{remark}

\subsection{Performance guarantees}

This subsection develops statistical guarantees for the spectral method described above by invoking the eigenspace perturbation theory introduced previously.
The first step is to establish a connection between the sample covariance $\bm{M}$ and the true covariance $\bm{M}^{\star}$.
Defining $\bm{F}\coloneqq[\bm{f}_{1},\cdots,\bm{f}_{n}]\in\mathbb{R}^{r\times n}$
and $\bm{Z}\coloneqq[\bm{\eta}_{1},\cdots,\bm{\eta}_{n}]\in\mathbb{R}^{p\times n}$,
one can easily compute that
\begin{align}
\bm{M} & = \frac{1}{n} (\bm{L}^{\star}\bm{F}+\bm{Z})(\bm{L}^{\star}\bm{F}+\bm{Z})^{\top}=\bm{M}^{\star}+\bm{E},
\label{eq:connection-M-Mstar-E-PCA}
\end{align}
where $\bm{M}^{\star}$ is defined in \eqref{eq:defn-Mstar-PCA}, and
\begin{align}
\bm{E}  & \coloneqq\bm{L}^{\star}\Big(\frac{1}{n}\bm{F}\bm{F}^{\top}-\bm{I}_r \Big)\bm{L}^{\star\top}+\frac{1}{n}\bm{L}^{\star}\bm{F}\bm{Z}^{\top} +\frac{1}{n}\bm{Z}\bm{F}^{\top}\bm{L}^{\star\top}  \notag \\
  &\qquad +  \Big( \frac{1}{n}\bm{Z}\bm{Z}^{\top} -\sigma^2 \bm{I}_p \Big).
\label{eq:defn-E-PCA}
%\\
%\bm{M}^{\star} & \coloneqq\bm{\Sigma}^{\star}+\sigma^{2}\bm{I}.
%\label{eq:defn-Mstar-PCA}
\end{align}

To apply the Davis-Kahan theorem, we are in need of controlling the size of the perturbation matrix $\bm{E}$.
This is achieved by the following lemma, whose proof is deferred to Section~\ref{sec:proof-lemmas:perturbation-size-E-PCA}.
\begin{lemma}
\label{lemma:perturbation-size-E-PCA}
Consider the settings in Section~\ref{sec:formulation-PCA}. Suppose that $n\geq c r\log^{3}(n+p)$ for some sufficiently large constant $c>0$. Then with probability exceeding $1-O((n+p)^{-10})$, one has
\begin{align*}
	\|\bm{E}\| & \lesssim\Bigg(\lambda_{1}^{\star}\sqrt{\frac{r}{n}}+\sigma\sqrt{\frac{\lambda_{1}^{\star}p}{n}}+\sigma^{2}\sqrt{\frac{p}{n}} + \frac{\sigma^2 p\log^{\frac{3}{2}}(n+p)}{n} \Bigg)\log^{\frac{1}{2}}(n+p).
\end{align*}
\end{lemma}

With Lemma~\ref{lemma:perturbation-size-E-PCA} in place, we are ready to present the following theorem that controls the estimation error of the spectral algorithm.
\begin{theorem}
\label{thm:perturbation-bound-PCA-l2}
Consider the settings in Section~\ref{sec:formulation-PCA}.
Suppose that $n \geq C \big(\kappa^{2}r+ r\log^2(n+p)+ \frac{\kappa\sigma^{2}p}{\lambda_{r}^{\star}}+\frac{\sigma^{4}p}{(\lambda_{r}^{\star})^{2}}\big)\log^{3} (n+p)$ for some sufficiently large constant $C>0$.
%and that $p \geq r$.
Then with probability  at least $1-O((n+p)^{-10})$, the following holds:
\begin{align}
\mathsf{dist}\big(\bm{U},\bm{U}^{\star}\big)
	\lesssim\Bigg(\frac{\sigma}{\sqrt{\lambda_{r}^{\star}}}\sqrt{\frac{\kappa p}{n}}+\frac{\sigma^{2}}{\lambda_{r}^{\star}}\sqrt{\frac{p}{n}} + \kappa\sqrt{\frac{r}{n}}\Bigg)\log^{\frac{1}{2}}(n+p).
\label{eq:thm-perturbation-PCA-l2}
\end{align}
\end{theorem}
\begin{remark}
	\label{remark:additional-term-PCA}
	The third term $\kappa\sqrt{(r \log(n+p)) /n} $ on the right-hand side of \eqref{eq:thm-perturbation-PCA-l2} arises due to the randomness of $\{\bm{f}_i\}$ but not that of $\{\bm{\eta}_i\}$. If our goal is instead to estimate the eigenspace of $\bm{L}^{\star} (\frac{1}{n}\sum_i\bm{f}_i\bm{f}_i^{\top})\bm{L}^{\star\top}$ as opposed to that of $\bm{L}^{\star}\bm{L}^{\star\top}$, then this term can be erased.
\end{remark}

To interpret what Theorem~\ref{thm:perturbation-bound-PCA-l2} conveys,
we include a few remarks in the sequel, focusing on the simple scenario where $\kappa =O(1)$.
In view of Remark~\ref{remark:additional-term-PCA}, we shall ignore the term $\kappa\sqrt{(r \log(n+p)) /n} $ in the discussion below.

\paragraph{Linear vs.~quadratic dependency on the noise level.}  In comparison to the matrix denoising task (cf.~Section~\ref{sec:performance-denoising-L2}) where $\mathsf{dist}\big(\bm{U},\bm{U}^{\star}\big)$ scales linearly with the noise level $\sigma$ (cf.~\eqref{eq:dist-U-Ustar-denoising-L2}), the above performance guarantees for PCA exhibit contrasting behavior in two different regimes depending on the strength of the signal-to-noise ratio (SNR), measured in terms of $\lambda_{r}^{\star} / \sigma^2$:
\begin{itemize}
\item
	When the SNR is sufficiently large with $\lambda_{r}^{\star} / \sigma^2   \gtrsim 1$,
	then the dominant factor in \eqref{eq:thm-perturbation-PCA-l2} is the term $\sigma \big( {\frac{  p \log(n+p)}{\lambda_{r}^{\star} n}} \big)^{1/2}$, which scales linearly with the noise level.
\item
When the SNR drops below the threshold $\lambda_{r}^{\star} / \sigma^2 \lesssim 1$, then  the term $\frac{\sigma^{2}}{\lambda_{r}^{\star}}\sqrt{\frac{p\log(n+p)}{n}} $---which scales quadratically with the noise level---enters the picture and becomes the dominant effect.
\end{itemize}
In truth, the quadratic term emerges since our spectral method operates upon the sample covariance matrix,  which inevitably contains second moments of the noise components.

\paragraph{Tightness and optimality.} Natural questions arise as to  whether the performance guarantees in Theorem~\ref{thm:perturbation-bound-PCA-l2} are tight,
and whether the statistical accuracy can be further improved by designing more intelligent algorithms. These questions can be addressed by looking into the fundamental statistical limits. As established in the literature \citep{zhang2018heteroskedastic,cai2019subspace}, one cannot hope to achieve
\begin{align}
	\mathsf{dist}\big( \widehat{\bm{U}},\bm{U}^{\star}\big)
	=o \Bigg( \frac{\sigma}{\sqrt{\lambda_{r}^{\star}}}\sqrt{\frac{p}{n}}+\frac{\sigma^{2}}{\lambda_{r}^{\star}}\sqrt{\frac{p}{n}} \Bigg)
	\label{eq:minimax-lower-bound-PCA}
\end{align}
in a minimax sense, regardless of the choice of the estimator $\widehat{\bm{U}}$; see, e.g., \citet[Theorem 2]{zhang2018heteroskedastic} for a precise statement. Comparing \eqref{eq:minimax-lower-bound-PCA} with Theorem~\ref{thm:perturbation-bound-PCA-l2} reveals the near statistical optimality of the spectral method (modulo some log factor), and  confirms the tightness of the eigenspace perturbation theory when applied to this problem.

\paragraph{Proof of Theorem~\ref{thm:perturbation-bound-PCA-l2}.}
We first make the observation that
\begin{align*}
\lambda_{1}(\bm{M}^{\star}) & \geq\cdots\geq\lambda_{r}(\bm{M}^{\star})>\lambda_{r+1}(\bm{M}^{\star})=\cdots=\lambda_{p}(\bm{M}^{\star})=\sigma^{2}>0,\\
 & \qquad\quad\text{and}\quad \lambda_{r}(\bm{M}^{\star})-\lambda_{r+1}(\bm{M}^{\star})=\lambda_{r}^{\star}.
\end{align*}
The Davis-Kahan sin$\bm{\Theta}$ theorem (cf.~Corollary \ref{cor:davis-kahan-conclusion-corollary})
thus implies that: if the perturbation size obeys $\|\bm{E}\|\leq(1-1/\sqrt{2})\lambda_{r}^{\star}$,
then one has
\begin{align*}
 & \mathsf{dist}\big(\bm{U},\bm{U}^{\star}\big)  \leq\frac{2\|\bm{E}\|}{\lambda_{r}(\bm{M}^{\star})-\lambda_{r+1}(\bm{M}^{\star})}=\frac{2\|\bm{E}\|}{\lambda_{r}^{\star}}\\
	& \quad \lesssim \frac{1}{\lambda_{r}^{\star}}\Bigg(\lambda_{1}^{\star}\sqrt{\frac{r}{n}}+\sigma\sqrt{\frac{\lambda_{1}^{\star}p}{n}}+\sigma^{2}\sqrt{\frac{p}{n}} + \frac{\sigma^{2}p\log^{\frac{3}{2}}(n+p)}{n} \Bigg)\log^{\frac{1}{2}}(n+p) \\
	& \quad \asymp \Bigg(\kappa\sqrt{\frac{r}{n}}+\frac{\sigma}{\sqrt{\lambda_{r}^{\star}}}\sqrt{\frac{\kappa p}{n}}+\frac{\sigma^{2}}{\lambda_{r}^{\star}}\sqrt{\frac{p}{n}}  \Bigg)\log^{\frac{1}{2}}(n+p) .
\end{align*}
Here, the penultimate inequality results from Lemma~\ref{lemma:perturbation-size-E-PCA};
the last line is valid as long as $n\gtrsim (\sigma^2 / \lambda_{1}^{\star}) p\log^{3}(n+p)$---a condition that would hold under the assumption of this theorem---so that the fourth term is dominated by the second one in the parenthesis of the penultimate line. Finally, it is immediately seen from Lemma~\ref{lemma:perturbation-size-E-PCA}
that the condition $\|\bm{E}\|\leq(1-1/\sqrt{2})\lambda_{r}^{\star}$ would hold
under the assumption of this theorem.
%
%\end{proof}

\subsection{Proof of Lemma~\ref{lemma:perturbation-size-E-PCA}}
\label{sec:proof-lemmas:perturbation-size-E-PCA}

We start by applying the triangle inequality to \eqref{eq:defn-E-PCA} as follows
\begin{align}
\|\bm{E}\| & \leq \big\|\bm{L}^{\star}\big\|^{2}\Big\|\frac{1}{n}\bm{F}\bm{F}^{\top}-\bm{I}_{r}\Big\|+\|\bm{L}^{\star} \|\,\Big\|\frac{1}{n}\bm{F}\bm{Z}^{\top}\Big\|+\Big\|\frac{1}{n}\bm{Z}\bm{F}^{\top}\Big\|\,\big\|\bm{L}^{\star} \big\|\nonumber \\
 & \qquad+\Big\|\frac{1}{n}\bm{Z}\bm{Z}^{\top}-\sigma^{2}\bm{I}_{p}\Big\|.
\label{eq:defn-E-decompose-PCA-triangle}
\end{align}
In order to develop an upper bound on this quantity, one needs to control the spectral
norm of $\frac{1}{n}\bm{F}\bm{F}^{\top}-\bm{I}_{r}$, $\frac{1}{n}\bm{F}\bm{Z}^{\top}$,
$\frac{1}{n}\bm{Z}\bm{F}^{\top}$ and $\frac{1}{n}\bm{Z}\bm{Z}^{\top}-\sigma^{2}\bm{I}_{p}$.
All of these terms share similar randomness structure, namely, they are all averages of independent zero-mean random matrices.
As a result, the truncated matrix Bernstein inequality in Corollary~\ref{thm:matrix-Bernstein-friendly-truncated} becomes applicable.
In what follows, we shall only demonstrate how to control the size
of $\frac{1}{n}\bm{F}\bm{Z}^{\top}$; the other terms can be bounded similarly.

Write
$
\bm{F}\bm{Z}^{\top}=\sum_{i=1}^{n}\bm{f}_{i}\bm{\eta}_{i}^{\top}
$.
Since the entries of $\bm{F}\bm{Z}^{\top}$ might be unbounded, we
start by identifying an appropriate truncation level. From standard
properties about Gaussian distributions and the union bound, it is
straightforward to verify that
\[
\mathbb{P}\left\{  \|\bm{f}_{i}\big\|_{\infty}\leq5\sqrt{\log(n+p)}\text{ and }\|\bm{\eta}_{i}\big\|_{\infty}\leq5\sigma\sqrt{\log(n+p)}\right\} \geq1-(n+p)^{-11.5}.
\]
One can further derive
%$\|\bm{f}_{i}\bm{\eta}_{i}^{\top}\big\|\leq\|\bm{f}_{i}\big\|_{2}\|\bm{\eta}_{i}\big\|_{2}\leq\sqrt{rp}\|\bm{f}_{i}\big\|_{\infty}\|\bm{\eta}_{i}\big\|_{\infty}$,
%
\[
	\|\bm{f}_{i}\bm{\eta}_{i}^{\top}\big\|
	\leq\|\bm{f}_{i}\big\|_{2}\|\bm{\eta}_{i}\big\|_{2} \leq\sqrt{rp}\,\|\bm{f}_{i}\big\|_{\infty}\|\bm{\eta}_{i}\big\|_{\infty}\leq25\sqrt{rp}\sigma\log(n+p)
\]
with probability greater than $1-(n+p)^{-11.5}$. In other words,
with the choice $L\coloneqq 25\sqrt{rp}\sigma\log(n+p)$ one has
\[
\mathbb{P}\left\{ \|\bm{f}_{i}\bm{\eta}_{i}^{\top}\big\|\geq L\right\} \leq(n+p)^{-11.5}\eqqcolon q_{0}.
\]
Additionally, the symmetry of Gaussian distributions implies
\[
	\mathbb{E}\big[\bm{f}_{i}\bm{\eta}_{i}^{\top}\big]
	- \mathbb{E}\Big[\bm{f}_{i}\bm{\eta}_{i}^{\top}\mathbbm1\big\{\|\bm{f}_{i}\bm{\eta}_{i}^{\top}\big\| < L\big\}\Big]
	% =- \mathbb{E}\Big[\bm{f}_{i}\bm{\eta}_{i}^{\top}\mathbbm1\big\{\|\bm{f}_{i}\bm{\eta}_{i}^{\top}\big\|\leq L\big\}\Big]
	=0.
\]

To invoke the truncated Bernstein inequality, it remains to determine
the variance statistic. Towards this end, letting $\bB_{i} = \bm{f}_{i}\bm{\eta}_{i}^{\top}$, we observe that
\begin{align*}
\mathbb{E}\big[\bm{B}_{i}\bm{B}_{i}^{\top}\big] & =\mathbb{E}\big[\bm{f}_{i}\bm{\eta}_{i}^{\top}\bm{\eta}_{i}\bm{f}_{i}^{\top}\big]=\mathbb{E}\big[\bm{\eta}_{i}^{\top}\bm{\eta}_{i}\big]\mathbb{E}\big[\bm{f}_{i}\bm{f}_{i}^{\top}\big]=p\sigma^{2}\bm{I}_r,\\
\mathbb{E}\big[\bm{B}_{i}^{\top}\bm{B}_{i}\big] & =\mathbb{E}\big[\bm{\eta}_{i}\bm{f}_{i}^{\top}\bm{f}_{i}\bm{\eta}_{i}^{\top}\big]=\mathbb{E}\big[\bm{f}_{i}^{\top}\bm{f}_{i}\big]\mathbb{E}\big[\bm{\eta}_{i}\bm{\eta}_{i}^{\top}\big]=r\sigma^{2}\bm{I}_p,
\end{align*}
thus leading to
\[
	v\coloneqq \max\Big\{\Big\|\sum_{i}\mathbb{E}\big[\bm{B}_{i}\bm{B}_{i}^{\top}\big]\Big\|,\Big\|\sum_{i}\mathbb{E}\big[\bm{B}_{i}^{\top}\bm{B}_{i}\big]\Big\|\Big\}
	= np\sigma^{2},
\]
where we use the fact that $r \leq p$.
Taking these bound together and applying the truncated matrix Bernstein theorem
(see Corollary~\ref{thm:matrix-Bernstein-friendly-truncated}) demonstrate that if $n\gtrsim r\log^3(n+p)$, one has
\begin{subequations}
\label{eq:norm-noise-bounds-PCA}
\begin{align}
	& \frac{1}{n}\big\|\bm{F}\bm{Z}^{\top}\big\|  \lesssim \frac{1}{n} \sqrt{v\log(n+p)}+ \frac{1}{n} L\log(n+p) \notag \\
	& \quad \asymp\sigma\sqrt{\frac{p\log(n+p)}{n}}+\frac{\sqrt{rp}}{n}\sigma\log^2(n+p)
 	 \asymp\sigma\sqrt{\frac{p\log(n+p)}{n}}
	 \label{eq:norm-bound-FZtop-PCA}
\end{align}
with probability at least $1-O\big((n+p)^{-10}\big)-nq_{0}=1-O\big((n+p)^{-10}\big)$.

Repeating the above analysis yields that: if $n\gtrsim r\log^{3}(n+p)$,  with probability at least $1-O((n+p)^{-10})$ one has
\begin{align}
\Big\| \frac{1}{n} \bm{F}\bm{F}^{\top}-\bm{I}_{r}\Big\| & \lesssim \sqrt{\frac{r\log(n+p)}{n}},\label{eq:norm-bound-FFt-PCA}\\
\Big\| \frac{1}{n} \bm{Z}\bm{Z}^{\top}-\sigma^{2}\bm{I}_{p} \Big\| & \lesssim \sigma^{2}\sqrt{\frac{p\log(n+p)}{n}} + \frac{\sigma^{2}p\log^2(n+p)}{n}.
\label{eq:norm-bound-ZZt-PCA}
\end{align}
\end{subequations}
Note that we do not get rid of the second term on the right-hand side of \eqref{eq:norm-bound-ZZt-PCA} since we do not assume $n\gtrsim p\log^{3}(n+p)$.

Substituting the above results \eqref{eq:norm-noise-bounds-PCA} into \eqref{eq:defn-E-decompose-PCA-triangle} and recognizing the basic fact
$\|\bm{L}^{\star}\|=\|\bm{U}^{\star}(\bm{\Lambda}^{\star})^{1/2}\|\leq\|(\bm{\Lambda}^{\star})^{1/2}\|=\sqrt{\lambda_{1}^{\star}}$,
we conclude that
\begin{align*}
	\|\bm{E}\| & \lesssim\Bigg(\lambda_{1}^{\star}\sqrt{\frac{r}{n}}+\sigma\sqrt{\frac{\lambda_{1}^{\star}p}{n}}+\sigma^{2}\sqrt{\frac{p}{n}} + \frac{\sigma^2 p\log^{\frac{3}{2}}(n+p)}{n} \Bigg)\log^{\frac{1}{2}}(n+p).
\end{align*}

%% file: chapters/community_detection.tex
\section{Graph clustering and community recovery}
\label{sec:community-detection}

\begin{figure}
\begin{center}
\includegraphics[width=0.85\textwidth]{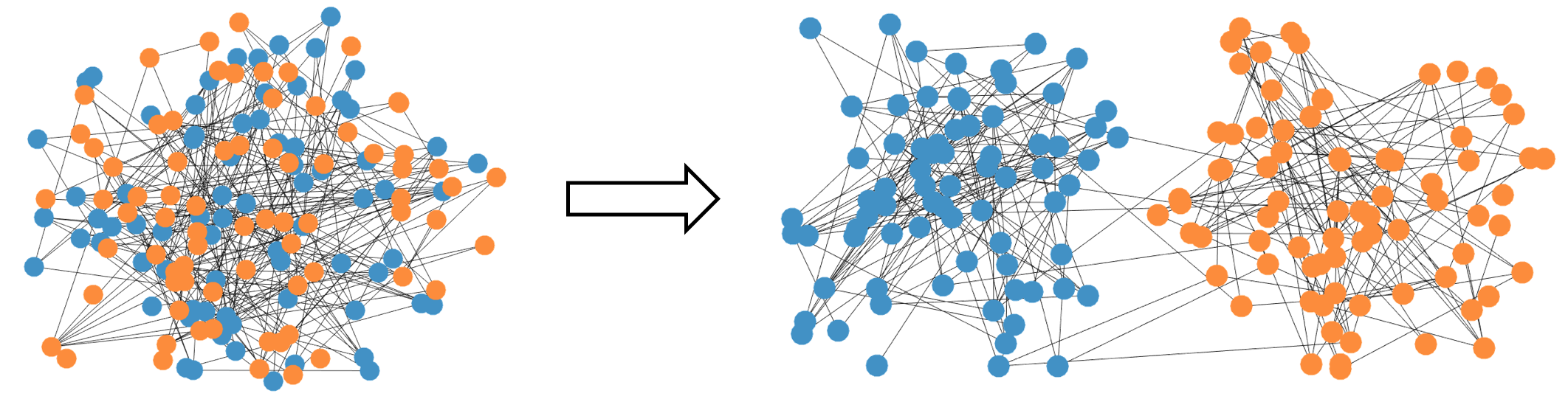}
\end{center}
\caption{Illustration of graph clustering and community recovery, where one wishes to cluster all nodes into two communities based on the edges in the graph.} \label{fig:community}
\end{figure}

Next, we move on to a central problem that permeates data science applications: clustering. An important formulation that falls under this category is graph clustering or community recovery, which aims to cluster individuals into different communities based on  pairwise measurements of their relationships, each of which reveals information about whether or not two individuals belong to the same community \citep{abbe2017community}; see Figure~\ref{fig:community} for an illustration. There has been a recent explosion of interest in this problem, due to its wide applicability in, say, social network analysis \citep{azaouzi2019community}, image segmentation \citep{browet2011community}, shape mapping in computer vision \citep{huang2013consistent}, haplotype phasing in genome sequencing \citep{chen2016community}, to name just a few.
This section explores the capability of spectral methods in application to graph clustering;
 we will revisit the clustering problem again in Section~\ref{sec:Gaussian-mixture} for another common formulation.

\subsection{Problem formulation and assumptions}
\label{sec:setup-community-detection}

In this section, we formulate the graph clustering problem via the well-renowned {\em stochastic block model} (SBM) introduced in \citet{holland1983stochastic}---an idealized generative model that commonly serves as a  theoretical benchmark for evaluating community recovery algorithms.

Consider an undirected graph $\mathcal{G}=(\mathcal{V},\mathcal{E})$ that comprises $n$ vertices, where $\mathcal{V}$ and   $\mathcal{E}$ denote the vertex set and the edge set of $\mathcal{G}$, respectively.  The $n$ vertices, labelled by $1,\cdots, n$,
exhibit  community structures and can be grouped into two non-overlapping communities of {\em equal} sizes. Here and throughout,  $n$ is assumed to be an even number, so that each community contains exactly $n/2$ vertices.
To encode the community memberships, we assign $n$ binary-valued variables $x_i^{\star} \in \{1,-1\}$ ($1\leq i \leq n$) to the  vertices in a way that
\begin{align*}
	x_i^{\star} = \begin{cases} 1, \qquad & \text{if vertex } i \text{ belongs to the 1st community}, \\
				   -1, \quad  & \text{otherwise}. \end{cases}
\end{align*}

The SBM assumes that the set $\mathcal{E}$ of (undirected) edges  is generated randomly based on the community memberships of the incident vertices. To be precise, each pair $(i,j)$ of vertices  is  connected by an edge independently with probability $p$ (resp.~$q$) if $i$ and $j$ belong to the same community (resp.~different communities). The resultant connectivity pattern is represented by an adjacency matrix
$\bm{A}=[A_{i,j}]_{1\leq i,j\leq n}\in \{0,1\}^{n\times n}$, such that for each pair $(i,j)$,
\begin{align}
	A_{i,j} = \begin{cases} 1, \qquad & \text{if }(i,j)\in \mathcal{E}, \\ 0, & \text{otherwise}.   \end{cases}
\end{align}
By convention, we take the diagonal entries to be $A_{i,i}=0$ for all $1\leq i\leq n$. As a remark, the matrix $\bm{A}$ is symmetric since $\mathcal{G}$ is an undirected graph,  with upper triangular elements being realizations of independent Bernoulli random variables with mean either $p$ (if two nodes are in the same community) or $q$ (otherwise).  In addition, it is assumed throughout that $p>q>0$, implying that there are in expectation more within-community edges than across-community edges.

Based on the adjacency matrix $\bm{A}$ generated by the SBM,
the goal is to identify the latent community memberships of the vertices. To phrase it in mathematical terms,
the aim is to reconstruct the vector  $\bm{x}^{\star}=[x_i^{\star}]_{1\leq i\leq n}\in \{1,-1\}^n$ modulo the global sign,
namely, recovering either $\bm{x}^{\star}$ or $-\bm{x}^{\star}$.
This is all one can hope for, as there is
absolutely no basis to distinguish the names of two groups.
%$\bm{x}^{\star}$ from $-\bm{x}^{\star}$ given merely the edge information.

\subsection{Algorithm: spectral clustering}
\label{sec:algorithm-community-detection}

Now we describe a spectral method.
To simplify presentation, it is assumed without loss of generality that: $x_i^{\star}=1$ for any $1\leq i\leq n/2$, and $x_i^{\star}=-1$ for any $i > n/2$.

A starting point for the algorithm design is to examine the mean of the adjacency matrix,  given as follows
\begin{align*}
	 \mathbb{E}[\bm{A}]  	=
	\left[\begin{array}{cc}
		p\,\bm{1}_{n/2}\bm{1}^{\top}_{n/2} & q\,\bm{1}_{n/2}\bm{1}^{\top}_{n/2}\\
		q\,\bm{1}_{n/2}\bm{1}^{\top}_{n/2} & p\,\bm{1}_{n/2}\bm{1}^{\top}_{n/2}
	\end{array}\right] -   p \bm{I}.
\end{align*}
As revealed by the above calculation,   the  matrix constructed below
\begin{align}
	\bm{M} = \bm{A} -\frac{p+q}{2} \bm{1}_n\bm{1}^{\top}_n  + p\bm{I}
	\label{eq:M-data-matrix-SBM}
\end{align}
exhibits an approximate rank-1 structure, in the sense that its mean
\begin{align} \label{eq:M-data-matrix-expectation-SBM}
	\bm{M}^{\star} \coloneqq \mathbb{E}[\bm{M}] =
	\frac{p-q}{2}\left[\begin{array}{c}
		\bm{1}_{n/2}\\
		-\bm{1}_{n/2}
	\end{array}  \right] \left[\begin{array}{cc}
\bm{1}^{\top}_{n/2} & -\bm{1}^{\top}_{n/2}\end{array}\right]
\end{align}
is a rank-1 matrix.  The leading eigenvalue of $\bm{M}^{\star}$ and its associated eigenvector are given respectively by
\begin{align}
	\lambda^{\star}\coloneqq\frac{(p-q)n}{2},  
	\quad \text{and} \quad
	\bm{u}^{\star}  \coloneqq \frac{1}{\sqrt{n}}
	\left[\begin{array}{c}
		\bm{1}_{n/2}\\
		-\bm{1}_{n/2}
	\end{array}\right].
	\label{eq:leading-evalue-evector-M-SBM}
\end{align}
Crucially, the eigenvector  $\bm{u}^{\star}$ encapsulates the precise community structure we seek to recover: all positive entries of $\bm{u}^{\star}$ correspond to vertices from one community, while the remaining ones form another community.

Inspired by the above calculation, a candidate spectral clustering algorithm consists of eigendecomposition followed by entrywise rounding:
\begin{itemize}
	\item[1.] Compute the leading eigenvector $\bm{u}$ of  $\bm{M}$ (constructed in \eqref{eq:M-data-matrix-SBM});
	\item[2.] Compute the estimate $\bm{x}=[x_i]_{1\leq i\leq n}$  such that for any $1\leq i\leq n$,
	\begin{align}
		  x_{i}= \mathsf{sgn}(u_i)  =
		  \begin{cases}
1,\quad & \text{if }u_{i}>0,\\
-1, \quad & \text{if }u_i \leq 0.
\end{cases}
	\end{align}
\end{itemize}
%
% Here, for any $z\in \mathbb{R}$, the sign function obeys $\mathsf{sgn}(z)=1$ if $z>0$ and $\mathsf{sgn}(z)=-1$ otherwise. 
In words, the community memberships are estimated in accordance with the signs of the entries of the leading eigenvector of $\bm{M}$, namely, the entries with the same signs are declared to come from the same cluster.

%which has two non-vanishing eigenvalues and their associated eigenvectors
%$$
%  	\lambda_1^* = \frac{(p+q)n}{2},~\lambda_2^* = \frac{(p-q)/n}{2},  \mbox{ and }  \bu_1^* = \frac{1}{\sqrt{n}}\mathbf{1}, ~ 	\bu_2^*=\frac{1}{\sqrt{n}}\left[\begin{array}{c}
%  		\bm{1}_{n/2}\\
%  		-\bm{1}_{n/2}
%  	\end{array}  \right ].
%$$
%Therefore, the second eigenvector reveals the true membership of the community.  We can work directly with the second eigenvector, but we here work with the first eigenvector by shifting the first component $\lambda_1^* \bu_1^*  (\bu_1^* )^T$  in the eigen-decomposition.  

\begin{remark}
	The above algorithm requires prior knowledge of the parameters $p$ and $q$ when constructing  $\bm{M}$. It is also feasible to develop a ``model-agnostic'' alternative by, for instance, looking at the second eigenvector of $\bm{A}$ (since the second eigenvector of $\mathbb{E}[\bm{A}]$ turns out to be precisely $\bm{u}^{\star}$), which does not rely on prior information about $p$ and $q$ at all; see, e.g., \citet{abbe2020entrywise} for details.  Here, we adopt the above model-dependent version primarily for convenience of exposition.  %To make the method feasible, we note that $ \mathbb{E} \sum_{i,j} A_{i, j} = n^{2}(p+q)/2 - n p$ and we can estimate $p+q$ by $2n^{-2} \sum_{i,j}A_{i,j}$ in the construction of $M$.
\end{remark}

\subsection{Performance guarantees: almost exact recovery}
\label{sec:theory-community-detection}

The spectral method enjoys appealing statistical guarantees for recovering the community structure of the SBM, which can be readily obtained by invoking the $\ell_2$ eigenvector perturbation theory.  To demonstrate this, we begin by developing an upper bound on the spectral norm of the perturbation matrix $\bm{E}\coloneqq \bm{M} - \bm{M}^{\star}$,
postponing the proof to Section~\ref{sec:proof-auxiliary-lemmas-community-detection}.
\begin{lemma} \label{lemma:E-spectral-norm-SBM}
	Consider the settings in Section~\ref{sec:setup-community-detection}, and suppose that $np\gtrsim \log n$. Then with probability at least $1-O(n^{-8})$, one has
	\begin{align}
		\| \bm{E} \| \lesssim \sqrt{np} .
	\end{align}
\end{lemma}
This spectral norm bound, in conjunction with the Davis-Kahan sin$\bm{\Theta}$ theorem, leads to the following theoretical support for the spectral method introduced in Section~\ref{sec:algorithm-community-detection}.
\begin{theorem}
	\label{thm:community-detection-weak}
	Consider the setting in Section~\ref{sec:setup-community-detection}, and suppose that
	\begin{align}
		p \gtrsim\frac{\log n}{n}, 
		\qquad\text{and}\qquad \sqrt{\frac{p}{n}} =o(p-q).
		\label{eq:assumption-p-q-SBM-L2}
	\end{align}
	With probability exceeding $1-O(n^{-8})$, the spectral method achieves
	\begin{align*}
		\frac{1}{n}\sum_{i=1}^{n}\mathbbm1\big\{ x_{i} = x_{i}^{\star}\big\} = 1- o(1), 
		\quad \text{or} \quad
		\frac{1}{n}\sum_{i=1}^{n}\mathbbm1\big\{ x_{i} = - x_{i}^{\star}\big\} = 1- o(1) .
	\end{align*}
\end{theorem}
It is noteworthy that the metric
\[
	\min \Bigg\{  \frac{1}{n}\sum_{i=1}^{n}\mathbbm1\big\{ x_{i} \neq x_{i}^{\star}\big\},\, \frac{1}{n}\sum_{i=1}^{n}\mathbbm1\big\{ x_{i} \neq - x_{i}^{\star}\big\} \Bigg\}
\]
can be understood as the mis-clustering rate.
In a nutshell, Theorem~\ref{thm:community-detection-weak} asserts that with the assistance of simple rounding (i.e., the $\mathsf{sgn}(\cdot)$ operation), the spectral method allows for {\em almost} exact community recovery---namely, correctly clustering all but a vanishing fraction of the vertices---assuming satisfaction of Condition~\eqref{eq:assumption-p-q-SBM-L2}.  Note that ``almost exact recovery'' is also referred to as ``weak consistency'' in the literature \citep{abbe2017community}.

Let us take a moment to interpret the recovery condition in \eqref{eq:assumption-p-q-SBM-L2}.
The first requirement  in Condition~\eqref{eq:assumption-p-q-SBM-L2} ensures the presence of sufficiently many edges  in the observed graph, while still permits the graph to be fairly sparse (with average vertex degrees as low as the order of $\log n$).
The second  requirement in Condition~\eqref{eq:assumption-p-q-SBM-L2}---which imposes a lower bound on the separation  between  the edge densities $p$ and $q$---guarantees that the within-community edges can be
adequately differentiated from across-community edges. As a more concrete example, consider the scenario where $p\asymp (\log n)/{n}$
(so that each vertex is only expected to be incident to $O(\log n)$ edges).
In this case, the second requirement in Condition~\eqref{eq:assumption-p-q-SBM-L2} can be translated into
\begin{align*}
	p-q \gg {\sqrt{\log n}}\,/\,{n}, \qquad \text{if } p\asymp  (\log n)\,/\,n.
\end{align*}
This indicates that the separation  $p-q$ is allowed to be considerably smaller than the edge densities, even in this low-edge-density regime.
In comparison, in another extreme case with $p \asymp 1$ (so that each vertex is likely to be connected with a constant fraction of other vertices), the second requirement in Condition~\eqref{eq:assumption-p-q-SBM-L2} reads
\begin{align*}
	p-q \gg 1/\sqrt{n} , \qquad \text{if } p \asymp  1,
\end{align*}
thereby allowing the edge density difference to be even $\sqrt{n}$ times smaller than the edge densities themselves.

It is worth highlighting that the spectral method is not merely capable of correctly clustering all but a diminishing fraction of vertices;
in fact, it allows for simultaneous and exact recovery for {\em all} vertices under slightly modified conditions. Establishing this stronger assertion requires developing a significantly strengthened $\ell_{\infty}$-based eigenvector perturbation theory, which will be elucidated in Section~\ref{sec:community-detection-linf}.
The discussion about the statistical optimality of this spectral method is  postponed to Section~\ref{sec:community-detection-linf} as well.

\paragraph{Proof of Theorem~\ref{thm:community-detection-weak}.}
It is readily seen from Lemma \ref{lemma:E-spectral-norm-SBM} that with with probability at least  $1-O(n^{-8})$,
\[
	\|\bm{E}\| \leq \Big(1-\frac{1}{\sqrt{2}}\Big) \frac{n(p-q)}{2}  = \Big(1-\frac{1}{\sqrt{2}}\Big) \lambda^{\star},
\]
provided that Condition~\eqref{eq:assumption-p-q-SBM-L2} holds. Here, $\lambda^{\star}$ is defined in \eqref{eq:leading-evalue-evector-M-SBM}. Apply Corollary~\ref{cor:davis-kahan-conclusion-corollary} to yield that with probability at least  $1-O(n^{-8})$,
\begin{align}
	\mathsf{dist}\big(\bm{u},\bm{u}^{\star}\big)\leq\frac{2\|\bm{E}\|}{\lambda^{\star}}\lesssim\frac{\sqrt{np}}{n(p-q)}	=o(1),
	\label{eq:L2-bound-u-ustar-CD-L2}
\end{align}
where the last relation follows from Condition \eqref{eq:assumption-p-q-SBM-L2}.

Assume, without loss of generality, that $\|\bm{u}-\bm{u}^{\star}\|_{2}=\mathsf{dist}\big(\bm{u},\bm{u}^{\star}\big)$.
We shall pay attention to the set
$$
	\mathcal{N}\coloneqq \big\{ i\mid |u_{i}-u_{i}^{\star}| \geq {1}/{\sqrt{n}}\big\}.
$$
In view of the rounding procedure: for any $i$ obeying $x_{i}\neq x_{i}^{\star}$, one necessarily has $\mathsf{sgn}(u_{i})\neq \mathsf{sgn}(u_{i}^{\star})$,
thus indicating that $|u_i - u_{i}^{\star}| \geq |u_{i}^{\star}| = 1/\sqrt{n}$ and hence $i\in \mathcal{N}$.
Combining the $\ell_2$ bound \eqref{eq:L2-bound-u-ustar-CD-L2} and the definition of $\mathcal{N}$, we can easily verify that
\[
|\mathcal{N}|\leq\frac{\|\bm{u}-\bm{u}^{\star}\|_{2}^{2}}{(1/\sqrt{n})^{2}}=o\big(n\big),
\]
which in turn leads to the advertised result
\[
	\frac{1}{n}\sum_{i=1}^{n} \mathbbm{1} \big\{ x_{i}\neq x_{i}^{\star}\big\}
	\leq \frac{1}{n}\sum_{i=1}^{n} \mathbbm{1} \Bigg \{ |u_i - u_{i}^{\star}| \geq \frac{1}{\sqrt{n}} \Bigg\}
	= \frac{|\mathcal{N}|}{n}=o(1) .
\]
%
%\end{proof}

\subsection{Proof of Lemma~\ref{lemma:E-spectral-norm-SBM}}
\label{sec:proof-auxiliary-lemmas-community-detection}

We intend to apply Theorem~\ref{thm:tighter-spectral-normal-ramon} to establish this lemma.
First, observe from the definition $\bm{E}=\bm{M}-\bm{M}^{\star}=\bm{A}-\mathbb{E}[\bm{A}]$ that
\[
	\big|E_{i,j}\big|\leq  \max_{i,j}\big|A_{i,j}\big| = 1.
\]
In addition, the variance of $E_{i,j}$ is upper bounded by
\begin{align*}
	 \mathbb{E}\big[E_{i,j}^{2}\big]  = \mathsf{Var}(A_{i,j})
	 \leq \mathbb{E}\big[A_{i,j}^{2}\big]
 	 \overset{(\mathrm{i})}{\leq}\max\{p,q\}\overset{(\mathrm{ii})}{=} p
\end{align*}
for any $(i,j)$,
where %(i) is valid since $E_{i,j}$ and $A_{i,j}$ differ only by a constant offset,
%(ii) holds since $\mathsf{Var}(A_{i,j})\leq \mathbb{E}[A_{i,j}^2]$,
(i) follows since $A_{i,j}$ is a Bernoulli random variable with mean either $p$ or $q$, and (ii) is due to the assumption $p>q$.
The bound \eqref{eq:X-spectral-norm-iid-special-ramon} and the condition $np\gtrsim \log n$ thus imply that
\begin{align}
	\|\bm{E}\|= \big\| \bm{M} - \mathbb{E}[\bm{M}] \big\| \lesssim \sqrt{np} + O(\sqrt{\log n}) \asymp \sqrt{np}
\end{align}
with probability exceeding $1-O(n^{-8})$.

%% file: chapters/gaussian_mixture.tex
\section{Clustering in Gaussian mixture models}
\label{sec:Gaussian-mixture}

% We now revisit an important task in unsupervised learning---clustering, which aims to group unlabeled data points into a few clusters (so that the data within the same cluster share similar characteristics). 

This section is also concerned with clustering, with the aim of grouping unlabeled data points into a few clusters (so that the data within the same cluster share similar characteristics). 
In contrast to the graph clustering setting in Section~\ref{sec:community-detection} where only pairwise measurements are available, this section  assumes direct access to data samples for each individual. Spectral methods---possibly with the aid of subsequent refinement like $k$-means---continue to be remarkably effective  for this setting,  achieving practical success in, say, image segmentation \citep{shi2000normalized}, text separation \citep{reynolds1995robust}, climate modeling \citep{lin2017statistical}, and heterogeneity modeling in precision medicine and marketing \citep{fan2014challenges}. Motivated by the empirical successes, understanding the theoretical properties of spectral clustering has garnered growing attention recently. In particular, Gaussian mixture models emerge as a succinct model of attack,  providing elegant yet intuitive abstractions to pivotal quantities that dictate the feasibility of  spectral clustering.

\subsection{Gaussian mixture models and assumptions}
\label{sec:setup-gaussian-mixture}

\paragraph{Model and goal.}
Imagine that we have collected $n$ independent samples $\{\bm{x}_{i}\}_{1\leq i\leq n}$,
generated from a mixture of $r$ spherical Gaussians with respective centers $\bm{\theta}_{1}^{\star},\cdots,\bm{\theta}_{r}^{\star}\in\mathbb{R}^{p}$.
More precisely, for each sample vector $\bm{x}_{i}\in\mathbb{R}^{p}$, we assume the existence of a predetermined, yet {\em a priori} unknown,
cluster membership variable $\xi_{i}^{\star}\in [r]$ such that
\begin{equation}
\bm{x}_{i}=\begin{cases}
\bm{\theta}_{1}^{\star}+\bm{\eta}_{i},\qquad & \text{if }\xi_{i}^{\star}=1,\\
\quad\,\,\vdots & \quad\,\vdots\\
\bm{\theta}_{r}^{\star}+\bm{\eta}_{i}, & \text{if }\xi_{i}^{\star}=r,
\end{cases}\label{eq:data-model-r-Gaussian-GM}
\end{equation}
where the noise vector $\bm{\eta}_{i}\sim\mathcal{N}(\bm{0},\bm{I}_{p})$ is independently generated across the samples.
In words, $\xi_{i}^{\star}$ indicates which Gaussian component a sample is generated from.
Clustering in this Gaussian mixture model can, therefore,
be posed as recovering the set of cluster membership variables $\{\xi_{i}^{\star}\}_{1\leq i\leq n}$ (modulo the global permutation ambiguity).

\paragraph{Assumptions.} To simplify our exposition, we impose the following assumptions throughout this section. As a worthy note,  this assumption is often non-essential and can be significantly relaxed, which we shall remark on momentarily in Remark~\ref{remark:assumptions_gmm}.
\begin{assumption}
	\label{assumption:gaussian-mixture}
	The centers are independently generated obeying
	\[
		\bm{\theta}_{i}^{\star} ~\overset{\mathrm{i.i.d.}}{\sim}~ \mathcal{N}\Bigg( \bm{0}, \frac{\Delta^2}{2p} \bm{I}_p \Bigg),
		\qquad 1\leq i\leq r
	\]
	for some parameter $\Delta>0$.
\end{assumption}
Under this assumption, standard Gaussian concentration inequalities  \citep{vershynin2016high} tell us that, with high probability (for large $p$ and $n=\mathrm{poly}(p)$),
\begin{align*}
& ~~~~~~ \big\|\bm{\theta}_{i}^{\star}\big\|_{2}^{2}  =(1+o(1))\frac{\Delta^{2}}{2},
\quad~~
\big|\bm{\theta}_{i}^{\star\top}\bm{\theta}_{j}^{\star}\big|= o(\Delta^2), \\
%O\left(\Delta\sqrt{\frac{\log n}{p}}\right),\\
& \big\|\bm{\theta}_{i}^{\star}-\bm{\theta}_{j}^{\star}\big\|_{2}^{2}
=\big\|\bm{\theta}_{i}^{\star}\big\|_{2}^{2}+\big\|\bm{\theta}_{j}^{\star}\big\|_{2}^{2}-2\bm{\theta}_{i}^{\star\top}\bm{\theta}_{j}^{\star}=(1+o(1))\Delta^{2}
\end{align*}
hold for any pair $i\neq j$, where $o(1)$ denotes a vanishingly small quantity as $p$ approaches infinity. The indication is that the parameter $\Delta$ reflects (approximately) the separation between any pair of centers.

For simplicity of presentation, it is further assumed that there are exactly $n/r$ samples drawn from each of the $r$ Gaussian components.
Without loss of generality, we assume that
\begin{align}
	\xi_{i}^{\star}=l,\qquad\text{if }\left\lceil \frac{\,i\,}{r}\right\rceil =l
	\label{eq:zi-assignment-GM}
\end{align}
for any $1\leq i\leq n$, where the ceiling function $\lceil x \rceil$ represents the least integer greater than or equal to the number $x\in \mathbb{R}$.
%the ceiling operator such as  $\left\lceil .01 \right\rceil = 1$ and $ \left\lceil 2.2 \right\rceil = 3$.
In other words, the first batch of $n/r$ samples is drawn from the first Gaussian component, the second batch comes from the second component, and so on.
It is worth pointing out that this assumed assignment information \eqref{eq:zi-assignment-GM} is unavailable when running the spectral clustering algorithm.

\subsection{Algorithm and rationale}
\label{sec:algorithm-gaussian-mixture}

\paragraph{Motivation: spectral structure of the data matrix.}

In order to develop a spectral clustering algorithm, it is instrumental to first examine the spectral feature of the following data matrix
\begin{align}
	\bm{X}\coloneqq[\bm{x}_{1},\cdots,\bm{x}_{n}]=\mathbb{E}[\bm{X}]+\underset{\eqqcolon\,\bm{Z}}{\underbrace{[\bm{\eta}_{1},\cdots,\bm{\eta}_{n}]}}.
	\label{eq:defn-Z-X-EX-GMM}
\end{align}
Clearly, $\mathbb{E}[\bm{X}] \in \mathbb{R}^{p \times n}$ exhibits a rank-$r$ structure:
\begin{align*}
	\mathbb{E}[\bm{X}] & =\big[\bm{\theta}_{1}^{\star},\cdots,\bm{\theta}_{1}^{\star},\bm{\theta}_{2}^{\star},\cdots,\bm{\theta}_{2}^{\star},\cdots,\bm{\theta}_{r}^{\star},\cdots,\bm{\theta}_{r}^{\star}\big]
 	= \bm{\Theta}^{\star}\bm{F}^{\star\top},
\end{align*}
where we define
\begin{align}
\bm{\Theta}^{\star}\coloneqq\left[\bm{\theta}_{1}^{\star},\cdots,\bm{\theta}_{r}^{\star}\right] \in \mathbb{R}^{p\times r},
\quad
%\text{and}\quad
\bm{F}^{\star}\coloneqq {\footnotesize\left[\begin{array}{cccc}
\bm{1}_{\frac{n}{r}}\\
 & \bm{1}_{\frac{n}{r}}\\
 &  & \ddots\\
 &  &  & \bm{1}_{\frac{n}{r}}
\end{array}\right] }\in \mathbb{R}^{n\times r}.
\label{eq:defn-Theta-star-F-star-GMM}
\end{align}

Similarly, the Gram matrix $\bm{X}^{\top}\bm{X}$ also inherits this rank-$r$ structure in the following sense (albeit in the form of a ``spiked'' structure due to the presence of noise):
\begin{align}
	\mathbb{E}\big[\bm{X}^{\top}\bm{X}\big]
	& =\mathbb{E}[\bm{X}]^{\top}\mathbb{E}[\bm{X}]+\mathbb{E}\big[\bm{Z}^{\top}\bm{Z}\big]
	=\bm{F}^{\star}\bm{\Theta}^{\star\top}\bm{\Theta}^{\star}\bm{F}^{\star\top}+p\bm{I}_n .
	\label{eq:mean-XT-X-GM}
\end{align}
Recognizing that  $\bm{F}^{\star}$ encodes all the cluster membership information,
one is motivated to attempt information extraction from the rank-$r$ eigenspace of $\bm{X}^{\top}\bm{X}$,
akin to the PCA algorithm introduced in Section~\ref{sec:algorithm-PCA}.

\paragraph{Algorithm: spectral clustering followed by $k$-means.}

With the preceding spectral properties in mind, we are ready to present a spectral clustering algorithm tailored to this Gaussian mixture model.
Given that the eigenspace of $\bm{X}^{\top}\bm{X}$ might only approximate $\bm{F}^{\star}$ up to global rotation,
we include a follow-up $k$-means scheme \citep{macqueen1967some}
 to produce a valid clustering outcome based on the spectral estimate.
\begin{itemize}
	\item[1.] Compute the leading rank-$r$ eigenspace $\bm{U}\in \mathbb{R}^{n\times r}$ of $\bm{X}^{\top}\bm{X}$.

	\item[2.] Compute $\bm{Y}= \mathcal{P} \big( \bm{U}\bm{U}^{\top} \big) \in \mathbb{R}^{n\times n}$,
		where the operator $\mathcal{P}(\cdot)$ projects each column onto the unit sphere,
		i.e., 
\begin{align*}
	%\label{eq:projection-unit-ball-GMM}
	\mathcal{P}(\bm{Z}) \coloneqq \Big[ \frac{\bm{z}_{1}}{\|\bz_1\|_2},\cdots,  \frac{\bz_n}{\|\bm{z}_{n}\|_2}\Big] 
\end{align*}
	%
% Clearly, the global scaling factor $\sqrt{n/r}$ is not needed.
	for any matrix $\bm{Z}=[\bm{z}_1,\cdots,\bm{z}_n]$. 
	As will be discussed below, the projection step is not necessary, and we can also simply take $\bY = \bU \bU^\top$.
		%,  but improves somewhat the  performance; see \eqref{eq:Y-Ystar-dist-GMM} below in the proof. \yxc{check}

	\item[3.] Let $\bm{y}_i$ represent the $i$-th column of $\bm{Y}$, and apply the $k$-means algorithm
	(with $k=r$) to the vectors $\{\by_i\}_{1 \leq i \leq n}$  to find the cluster centers and cluster labels for all individuals; namely, we  compute
\begin{equation}
\Big(\big\{\widehat{\xi}_{i}\big\}_{i=1}^{n},\big\{\widehat{\bm{\vartheta}}_{i}\big\}_{i=1}^{r}\Big)
= \hspace{-0.2em}
\underset{\xi_{1},\cdots,\xi_{n}\in[r],\,\bm{\vartheta}_{1},\cdots,\bm{\vartheta}_{r}\in\mathbb{R}^{n}}{\arg\min}\sum_{i=1}^{n}\big\|\bm{y}_{i}-\bm{\vartheta}_{\xi_{i}}\big\|_{2}^{2} .
\label{eq:k-means-GMM}
\end{equation}
%
%The $k$-mean algorithm alternatively optimizes $\big\{\xi_{i}\big\}_{i=1}^{n}$  given $\big\{\bm{\vartheta}_{i}\big\}_{i=1}^{r}$ and  $\big\{\bm{\vartheta}_{i}\big\}_{i=1}^{r}$ given $\big\{\xi_{i}\big\}_{i=1}^{n}$.  Given the current estimate of centers $\big\{\widehat{\bm{\vartheta}}_{i}\big\}_{i=1}^{r}$, classify $\bm{x}_i$ to its nearest estimated center, i.e., assign $\widehat{\xi}_{i}$ to the label of that of the nearest center ; and given the class labels $\big\{\widehat{\xi}_{i}\big\}_{i=1}^{n}$, use the sample means of the clusters to update  $\big\{\widehat{\bm{\vartheta}}_{i}\big\}_{i=1}^{r}$.  The initialization is typically $r$ randomly selected data points.  Since the optimization is non-convex, multiple random initializations are used and the optimal output, in terms of target function \eqref{eq:k-means-GMM}, is used as the final estimate.
\end{itemize}
The algorithm then returns $\big\{\widehat{\xi}_{i}\big\}_{1\leq i\leq n}$ as the clustering result.
Interestingly, Step 1 bears similarity with the spectral algorithm for graph clustering,
since  we essentially generate a pairwise similarity measurement for each pair $(i,j)$ using the inner product $\langle \bm{x}_i, \bm{x}_j \rangle$.

\begin{remark}
	The $k$-means formulation \eqref{eq:k-means-GMM}---which minimizes the sum of squared distance between each data point and the center of its associated cluster---is
	an integer program and intractable in general \citep{aloise2009np}.
	Fortunately, computationally feasible solutions are available either under sufficient minimum center separation or when suitably initialized
	\citep{lloyd1982least,vempala2004spectral,lu2016statistical,peng2007approximating,awasthi2015relax,mixon2017clustering,iguchi2017probably}.
	An in-depth account of this computational aspect is beyond the scope of this monograph,
	and the interested reader is referred to \citet{li2020birds,loffler2019optimality} for details.
\end{remark}

%\citep{peng2007approximating,awasthi2015relax,mixon2017clustering,iguchi2017probably,li2020birds}

\paragraph{Further explanations.}

We take a moment to explain why $k$-means is applied to the columns of $\bm{Y}$.
Recall that the central object the spectral algorithm seeks to approximate
is the leading rank-$r$ eigenspace of $\bm{F}^{\star}\bm{\Theta}^{\star\top}\bm{\Theta}^{\star}\bm{F}^{\star\top}$ (cf.~\eqref{eq:mean-XT-X-GM}).
For convenience, suppose we have the eigendecomposition  $\bm{\Theta}^{\star\top}\bm{\Theta}^{\star}=\bm{U}_{\theta}\bm{\Sigma}_{\theta}\bm{U}_{\theta}^{\top}$,
where $\bm{U}_{\theta}\in\mathcal{O}^{r\times r}$ is  orthonormal and $\bm{\Sigma}_{\theta}\in\mathbb{R}^{r\times r}$ is diagonal.
This results in the decomposition
\begin{equation}
\bm{F}^{\star}\bm{\Theta}^{\star\top}\bm{\Theta}^{\star}\bm{F}^{\star\top}=\frac{n}{r}\cdot\Big(\underset{\eqqcolon\,\bm{U}^{\star}}{\underbrace{\sqrt{\frac{r}{n}}\bm{F}^{\star}\bm{U}_{\theta}}}\Big)\bm{\Sigma}_{\theta}\Big(\sqrt{\frac{r}{n}}\bm{F}^{\star}\bm{U}_{\theta}\Big)^{\top}.
\label{eq:FTheta-Theta-F-eigenspace}
\end{equation}
Apparently, the matrix $\bm{U}^{\star}\in\mathbb{R}^{n\times r}$ defined above has orthonormal columns and, as a result,
represents the eigenspace of $\bm{F}^{\star}\bm{\Theta}^{\star\top}\bm{\Theta}^{\star}\bm{F}^{\star\top}$.
The idea is that if the spectral estimate $\bm{U}$ approximates $\bm{U}^{\star}\in\mathbb{R}^{n\times r}$ well,
then the matrices $\sqrt{\frac{n}{r}} \, \bm{U}\bm{U}^{\top}$ and $\bm{Y}$ constructed above are hopefully close to the following matrix
\begin{equation}
\bm{Y}^{\star}  \coloneqq  \sqrt{\frac{n}{r}} \, \bm{U}^{\star}\bm{U}^{\star\top}
%=\bm{F}^{\star}\bm{U}_{\theta}\bm{U}_{\theta}^{\top}\bm{F}^{\star\top}
%= \sqrt{\frac{r}{n}} \, \bm{F}^{\star}\bm{F}^{\star\top}
= \sqrt{\frac{r}{n}} \left[\begin{array}{ccc}
\bm{1}_{\frac{n}{r}}\bm{1}_{\frac{n}{r}}^{\top}\\
 & \ddots\\
 &  & \bm{1}_{\frac{n}{r}}\bm{1}_{\frac{n}{r}}^{\top}
\end{array}\right].
\label{eq:Ystar-definition-GMM}
\end{equation}
As can be easily seen, the data points belonging to the same ground-truth cluster are associated with identical columns in $\bm{Y}^{\star}$;
for instance, each of the first $n/r$ samples---which belongs to the first cluster---corresponds to a column of $\bm{Y}^{\star}$ given by
$\sqrt{\frac{r}{n}}\,  {\scriptsize \left[\begin{array}{c}
\bm{1}_{n/r}\\
\bm{0}
\end{array}\right]}$.
Therefore, clustering  the columns of $\bm{Y}$ via $k$-means is expected to unveil the underlying cluster structure,
provided that $\bm{Y}$ is sufficiently close to $\bm{Y}^{\star}$.
In summary,  spectral estimation (Steps 1-2) effectively leads to a new vector $\bm{y}_i$ for each point, which enjoys substantially enhanced signal-to-noise ratio compared to $\bm{x}_i$ and boosts the chance for $k$-means to succeed.

We shall also explain the projection operation enforced in Step 2 of the algorithm.
Given that each column of $\bm{Y}^{\star} $ has unit $\ell_2$ norm, projecting each column of  $ \bm{U}\bm{U}^{\top}$ onto the unit sphere ensures that no column of $\bm{Y}$ has an abnormal size.
Note, however, that this projection step is  non-essential and is introduced here 
primarily to simplify the mathematical analysis. Spectral clustering is expected to succeed even in the absence of such a projection step \citep{loffler2019optimality}.
% Yuxin: projection is actually needed in the proof to show that || y_i ||_2 \leq 1
% %with expectation of improved empirical performance.  Indeed, the proof goes through without this projection step; see \eqref{eq:Y-Ystar-dist-GMM}.

\begin{remark}
	Another variation of spectral clustering is to directly apply the $k$-means algorithm to cluster the rows of $\bU$ (or some properly rescaled version of them) \citep{loffler2019optimality}.  To explain the rationale, we note that under the assumption \eqref{eq:zi-assignment-GM}, $\bU^{\star}$ necessarily consists of $r$ blocks of identical rows as follows:
$$
	\bU^{\star} = \sqrt{\frac{r}{n}} \begin{pmatrix}
		\bm{1}_{n/r} \bm{\nu}_1^\top\\
   	\vdots\\
   	\bm{1} _{n/r} \bm{\nu}_r^\top
   \end{pmatrix} ,
$$
where $\bm{\nu}_1^\top, \cdots, \bm{\nu}_r^\top$ are the orthonormal rows of the matrix $\bU_\theta$ (cf.~\eqref{eq:FTheta-Theta-F-eigenspace}).
Consequently,  clustering the rows of $\bU^{\star}$ reveals exactly the true cluster assignments of all individuals.
The idea of our spectral analysis below applies to this method as well; we leave it to the reader as an exercise.
\end{remark}

\subsection{Performance guarantees}
\label{sec:theory-gaussian-mixture}

Now, we turn to characterizing the clustering performance of the above spectral algorithm.
We shall focus attention on the mis-clustering rate as the performance metric.  As the cluster labels in $[r]$ can be arbitrarily permuted,
the mis-clustering rate associated with the labels $\{\widehat{\xi}_{i}\}$ returned by our algorithm is defined as
\[
	\ell_{\mathsf{mis}}\big(\{\widehat{\xi}_{i}\},\{\xi_{i}^{\star}\}\big)
	\coloneqq \min_{\phi\in\Pi} ~ \frac{1}{n}\sum_{i=1}^{n}\mathbbm{1}\left\{ \phi(\widehat{\xi}_{i})\neq\xi_{i}^{\star}\right\} ,
\]
where $\Pi$ is the set of permutations of $[r]$.
In words, this metric captures the average number of mislabeled data points, after accounting for global permutation.
For notational convenience, we shall set $\bm{M}^{\star}\coloneqq\mathbb{E}\big[\bm{X}^{\top}\bm{X}\big]$ and
$\bm{E}\coloneqq\bm{X}^{\top}\bm{X} - \mathbb{E}\big[\bm{X}^{\top}\bm{X}\big]$ throughout this section.

The first step towards analyzing the statistical accuracy of the spectral algorithm lies in developing a perturbation bound on $\|\bm{U}\bm{U}^{\top} - \bm{U}^{\star}\bm{U}^{\star\top}\|$, where $\bm{U}^{\star}$ (cf.~\eqref{eq:FTheta-Theta-F-eigenspace}) represents the leading rank-$r$ eigenspace of $\bm{M}^{\star}$.
This can be accomplished via the Davis-Kahan theorem, which requires us to first control the size of the perturbation $\bm{E}$.
\begin{lemma}
	\label{lem:perturbation-norm-Gaussian-mixture}
	Consider the settings in Section~\ref{sec:setup-gaussian-mixture}, and suppose $p\gtrsim r \log ^3 (n+p)$.
	Then with probability at least $1-O((n+p)^{-10})$, one has
	\begin{align*}
		\|\bm{E}\|  \lesssim\frac{\Delta n\sqrt{\log(n+p)}}{\sqrt{r}}+\sqrt{np\log(n+p)}+n\log^{2}(n+p) .
	\end{align*}
\end{lemma}
%
%Another part that needs to be taken care of is developing a lower bound on the spectral gap of $\bm{M}^{\star}$.
%In view of \eqref{eq:mean-XT-X-GM} and \eqref{eq:FTheta-Theta-F-eigenspace}, the spectral gap of $\bm{M}^{\star}$ relies heavily on the spectral property of $\bm{\Theta}^{\star\top}\bm{\Theta}^{\star}$, which is studied through the following lemma.
%%
%\begin{lemma}
%	\label{lemma:sigma-min-Thetastar-GMM}
%	Consider the setting in Section~\ref{sec:setup-gaussian-mixture}, and suppose that  $p\geq C_{2}r^{2}\log p$ for some sufficiently large constant $C_2>0$.
%	Then with probability at least $1-O(p^{-8})$, the matrix $\bm{\Theta}^{\star}$ defined in \eqref{eq:defn-Theta-star-F-star-GMM} obeys
%	\[
%		\lambda_{\min}\big(\bm{\Theta}^{\star\top}\bm{\Theta}^{\star}\big) \geq \Delta^2 / 4.
%	\]
%\end{lemma}
%%

%
The proof of this lemma can be found in Section~\ref{sec:proof-auxiliary-lemmas-GM}.
Equipped with the above perturbation bound, we are ready to present our statistical guarantees for spectral clustering.
\begin{theorem}
	\label{thm:GMM-performance}
	Consider the setting and assumptions in Section~\ref{sec:setup-gaussian-mixture}, and suppose that $r=O(1)$ and $p\gtrsim \log^3 n$.
	% Suppose that  $\Delta\geq C_{1} \sqrt{r}\max \big\{ \big(\frac{p\log(n+p)}{n}\big)^{1/4}, \log(n+p) \big\}$ for some sufficiently large constant $C_1>0$.
	With probability at least $1-O(p^{-10})$, the mis-clustering rate of the spectral algorithm in Section~\ref{sec:algorithm-gaussian-mixture} achieves
	 $$ \ell_{\mathsf{mis}}\big(\{\widehat{\xi}_{i}\},\{\xi_{i}^{\star}\}\big) = o(1) ,$$
	with the proviso that
	\begin{equation}
		\log(n+p)=o(\Delta)\quad\text{and}\quad\Big(\frac{p\log(n+p)}{n}\Big)^{1/4}=o(\Delta).
		\label{eq:condition-Delta-thm-GMM}
	\end{equation} 	
\end{theorem}

Before embarking on the proof of this theorem, we discuss briefly the implications of this theorem.
 In order to ensure a vanishingly small mis-clustering rate, it suffices for the center separation $\Delta$ to exceed
\[
	\Delta\gtrsim \begin{cases}
		\mathrm{poly}\log(n+p), & \text{if }p\leq n,\\
		\big(\frac{p}{n}\big)^{1/4} \mathrm{poly}\log(n+p),\qquad & \text{if }p\geq n.
	\end{cases}
\]
This separation condition matches the minimax lower bound  up to some logarithmic term \citep{cai2018rate,ndaoud2018sharp}.
In particular,  in the high-dimensional case where $p\geq n$, the required separation condition changes fairly gracefully with the aspect ratio $p/n$.

\begin{remark}\label{remark:assumptions_gmm}
	As alluded to previously, Assumption \ref{assumption:gaussian-mixture} can be significantly relaxed. For example,
	the Gaussianity assumption therein is unnecessary; (almost) exact clustering is plausible once the minimum center separation exceeds a certain threshold,
	regardless of how $\{\bm{\theta}_i^{\star}\}$ are generated.
	To achieve this generality, however, the algorithm might need to be properly modified. Roughly speaking, in addition to $\bm{U}$,
	it is sensible to also exploit information contained in the eigenvalues of $\bm{X}^{\top}\bm{X}$
	(which is crucial for, say, the scenario where all centers $\{\bm{\theta}_i^{\star}\}$ are perfectly aligned except for the scaling factors).
	We recommend the readers to \citet{loffler2019optimality} for detailed discussions.
\end{remark}

\paragraph{Proof of Theorem~\ref{thm:GMM-performance}.}
The proof consists of two steps: controlling the perturbation $\|\bm{U}\bm{U}^{\top}-\bm{U}^{\star}\bm{U}^{\star\top}\|_{\mathrm{F}}$ (and hence $\|\bm{Y}-\bm{Y}^{\star}\|_{\mathrm{F}}$),
and demonstrating that the follow-up $k$-means  performs well.

%The proof consists of two steps, which we elaborate on separately.

%\paragraph{Step 1: controlling $\|\bm{Y}-\bm{Y}^{\star}\|$.}

The first step  is to control $\|\bm{U}\bm{U}^{\top}-\bm{U}^{\star}\bm{U}^{\star\top}\|_{\mathrm{F}}$,
built upon Lemma~\ref{lem:perturbation-norm-Gaussian-mixture} and a lower bound on the spectral gap of $\bm{M}^{\star}$.
Observe that
\begin{align*}
\lambda_{r}\big(\bm{M}^{\star}\big)-\lambda_{r+1}\big(\bm{M}^{\star}\big)
& =\lambda_{\min}\big(\bm{F}^{\star}\bm{\Theta}^{\star\top}\bm{\Theta}^{\star}\bm{F}^{\star\top}\big)
 =\frac{n}{r}\lambda_{\min}\big(\bm{\Theta}^{\star\top}\bm{\Theta}^{\star}\big) ,
%\label{eq:lambda-r-Mstar-GMM-12345}
\end{align*}
where the last identity holds since  $\bm{F}^{\star} \bm{F}^{\star\top} = \frac{n}{r} \bI_r$ according to the definition \eqref{eq:defn-Theta-star-F-star-GMM}.
This motivates us to look at the spectral property of $\bm{\Theta}^{\star\top}\bm{\Theta}^{\star}$.
Given that $\bm{\Theta}^{\star}$ is composed of i.i.d.~Gaussian entries (cf.~Assumption~\ref{assumption:gaussian-mixture}),
invoking the bound (\ref{eq:norm-bound-ZZt-PCA}) with proper rescaling gives
\[
	\big\Vert \bm{\Theta}^{\star\top}\bm{\Theta}^{\star}-\mathbb{E}\big[\bm{\Theta}^{\star\top}\bm{\Theta}^{\star}\big]\big\Vert
	\lesssim\Delta^{2}\left(\sqrt{\frac{r\log p}{p}}+\frac{r\log^{2}p}{p}\right)\leq\frac{\Delta^{2}}{4}
\]
with probability exceeding $1-O(p^{-10})$, provided that $p\geq C_{2}r\log^{2}p$ for some sufficiently large constant $C_{2}>0$.
Further, it is self-evident that $\mathbb{E}\big[\bm{\Theta}^{\star\top}\bm{\Theta}^{\star}\big] = \frac{1}{2}\Delta^{2} \bm{I}_{r}$.
Weyl's inequality (see Lemma~\ref{lemma:weyl}) then guarantees that
\begin{align*}
\lambda_{\min}\big(\bm{\Theta}^{\star\top}\bm{\Theta}^{\star}\big)
& \geq\lambda_{\min}\big(\mathbb{E}\big[\bm{\Theta}^{\star\top}\bm{\Theta}^{\star}\big]\big)
  -\big\Vert \bm{\Theta}^{\star\top}\bm{\Theta}^{\star}-\mathbb{E}\big[\bm{\Theta}^{\star\top}\bm{\Theta}^{\star}\big] \big\Vert \\
 & \geq {\Delta^{2}}/{2} - {\Delta^{2}}/{4} = {\Delta^{2}}/{4}.
\end{align*}
Combine the preceding inequalities to arrive at
\begin{align}
\lambda_{r}\big(\bm{M}^{\star}\big)-\lambda_{r+1}\big(\bm{M}^{\star}\big)
= \frac{n}{r}\lambda_{\min}\big(\bm{\Theta}^{\star\top}\bm{\Theta}^{\star}\big)
 \geq   \frac{n\Delta^2}{ 4 r} .
\label{eq:eigen-gap-GMM-12345}
\end{align}
%

%where the penultimate relation holds since the columns of $\bm{F}^{\star}$ are orthogonal to each other and $\|\bm{F}^{\star}\|=\sqrt{n/r}$,  and the last line follows from Lemma~\ref{lemma:sigma-min-Thetastar-GMM}.

By virtue of Lemma~\ref{lem:perturbation-norm-Gaussian-mixture} and \eqref{eq:eigen-gap-GMM-12345}, if the following condition
$$\Delta\geq C_{1}\max\Big\{ \Big(\frac{r^{2}p\log(n+p)}{n}\Big)^{1/4},\sqrt{r}\log(n+p) \Big\}$$
holds for some large enough constant $C_1>0$, then it is guaranteed that $\|\bm{E}\|\leq (1-1/\sqrt{2}) (\lambda_{r}\big(\bm{M}^{\star}\big)-\lambda_{r+1}\big(\bm{M}^{\star}\big)) $.
This in turn allows us to invoke the Davis-Kahan theorem
(namely, Corollary~\ref{cor:davis-kahan-conclusion-corollary})
to obtain
\begin{align}
 & \big\|\bm{U}\bm{U}^{\top}-\bm{U}^{\star}\bm{U}^{\star\top}\big\|_{\mathrm{F}}
 \leq\frac{\sqrt{2r}\,\|\bm{E}\|}{\lambda_{r}\big(\bm{M}^{\star}
 \big)-\lambda_{r+1}\big(\bm{M}^{\star}\big)} \leq \sqrt{r} \varepsilon
 %\nonumber\\
 %& \qquad\lesssim\frac{\Delta r\sqrt{\log(n+p)}+r^{1.5}
% \sqrt{\frac{p\log(n+p)}{n}}+r^{1.5}\log^{2}(n+p)}{\Delta^{2}}
\label{eq:U-Ustar-dist-GMM}
\end{align}
with probability exceeding $1-O(p^{-8})$,
where the last line arises from \eqref{eq:eigen-gap-GMM-12345}
and Lemma~\ref{lem:perturbation-norm-Gaussian-mixture}, and
\[
	\varepsilon \asymp \frac{\Delta\sqrt{r\log(n+p)}
+ r\sqrt{\frac{p\log(n+p)}{n}} + r \log^{2}(n+p)}{\Delta^{2}}.
\]

From the construction of $\bm{Y}$ and \eqref{eq:Ystar-definition-GMM},
one can propagate the bound \eqref{eq:U-Ustar-dist-GMM} to $\|\bm{Y}-\bm{Y}^{\star}\|_{\mathrm{F}}$ as follows:
\begin{align}
	\big\|\bm{Y}-\bm{Y}^{\star}\big\|_{\mathrm{F}}^{2} & \overset{\mathrm{(i)}}{=} \Big\|\mathcal{P}\Big(\sqrt{\frac{n}{r}}\,\bm{U}\bm{U}^{\top}\Big)-\bm{Y}^{\star}\Big\|_{\mathrm{F}}^{2}
 \overset{\mathrm{(ii)}}{\leq} 4\Big\|\sqrt{\frac{n}{r}}\,
 \bm{U}\bm{U}^{\top}-\bm{Y}^{\star}\Big\|_{\mathrm{F}}^{2} \notag\\
 & =\frac{4n}{r}\big\|\bm{U}\bm{U}^{\top}-\bm{U}^{\star}
 \bm{U}^{\star\top}\big\|_{\mathrm{F}}^{2}\lesssim \varepsilon^{2}n.
\label{eq:Y-Ystar-dist-GMM}
\end{align}
Here, the first identity (i) holds since the operator $\mathcal{P}$ is invariant to global scaling.
Regarding the inequality (ii),  it follows from standard inequality regarding Euclidean projection (e.g.,
\citet[Lemma~15]{soltanolkotabi2017structured}), which we postpone to the end of this proof.

% where the last relation holds true as long as
% %
% \[
% \sqrt{r}\Bigg(\frac{p\log(n+p)}{n}\Bigg)^{1/4}=o(\Delta)\quad\text{and}\quad\sqrt{r}\log(n+p)=o(\Delta).
% \]
% %

The next step then amounts to translating the perturbation bound \eqref{eq:Y-Ystar-dist-GMM} into clustering accuracy guarantees (after $k$-means is applied).
This is accomplished through the following key lemma, to be established in Section~\ref{sec:proof-auxiliary-lemmas-GM}.
\begin{lemma}
	\label{lemma:k-means-GMM}
	Suppose that the matrix $\bm{Y}$ obtained in the spectral algorithm in Section~\ref{sec:algorithm-gaussian-mixture} satisfies
	\begin{equation}
		\big\|\bm{Y}-\bm{Y}^{\star}\big\|_{\mathrm{F}}^{2}\leq\varepsilon^{2}n,
		\label{eq:assumption-Y-Ystar-Fro-GMM}
	\end{equation}
	where $\varepsilon>0$ is a quantity obeying $\varepsilon\leq c_{3}r^{-4}$ for some sufficiently small constant $c_3>0$.
	Then the mis-clustering rate obeys
	$$\ell_{\mathsf{mis}}\big(\{\widehat{\xi}_{i}\},\{\xi_{i}^{\star}\}\big)  \leq 2r \varepsilon^{1/4}.$$
\end{lemma}
As a consequence of Lemma~\ref{lemma:k-means-GMM}, the mis-clustering rate is $o(1)$ as long as $\varepsilon r^4 = o(1)$,
a condition that is guaranteed  under the assumptions \eqref{eq:condition-Delta-thm-GMM} and $r=O(1)$.
This establishes Theorem~\ref{thm:GMM-performance}.

\begin{proof}[Proof of the inequality (ii) in \eqref{eq:Y-Ystar-dist-GMM}]
For any vector $\bm{v}$ residing in the unit sphere and any other vector
$\bm{w}$, we have
\begin{align*}
\|\mathcal{P}(\bm{w})-\bm{w}\|_{2}^{2} & =\|\mathcal{P}(\bm{w})-\bm{v}+\bm{v}-\bm{w}\|_{2}^{2}\\
 & =\|\mathcal{P}(\bm{w})-\bm{v}\|_{2}^{2}+\|\bm{v}-\bm{w}\|_{2}^{2}+2\langle\mathcal{P}(\bm{w})-\bm{v},\bm{v}-\bm{w}\rangle\\
 & \geq\|\mathcal{P}(\bm{w})-\bm{v}\|_{2}^{2}+
 \|\mathcal{P}(\bm{w})-\bm{w}\|_{2}^{2}
 +2\langle\mathcal{P}(\bm{w})-\bm{v},\bm{v}-\bm{w}\rangle,
\end{align*}
where the last inequality follows since $\mathcal{P}$ denotes the projection onto the unit sphere and $\bm{v}$ lies in the unit sphere.
%,and hence $\|\mathcal{P}(\bm{w})-\bm{w}\|_2\leq \|\bm{v}-\bm{w}\|_2$.
Cancelling out the common term $\|\mathcal{P}(\bm{w})-\bm{w}\|_{2}^{2}$ and invoking Cauchy-Schwarz lead to
\begin{align*}
\|\mathcal{P}(\bm{w})-\bm{v}\|_{2}^{2} & \leq-2\langle\mathcal{P}(\bm{w})-\bm{v},\bm{v}-\bm{w}\rangle\leq2\|\mathcal{P}(\bm{w})-\bm{v}\|_{2}\|\bm{v}-\bm{w}\|_{2},
\end{align*}
and therefore,
\[
\|\mathcal{P}(\bm{w})-\bm{v}\|_{2}\leq2\|\bm{w}-\bm{v}\|_{2}, 
\quad\text{or}\quad\|\mathcal{P}(\bm{w})-\bm{v}\|_{2}^{2}
\leq4\|\bm{w}-\bm{v}\|_{2}^{2}.
\]
This inequality clearly extends to the matrix counterpart, thus establishing the claimed result.
\end{proof}

%\end{proof}

\subsection{Proof of auxiliary lemmas}
\label{sec:proof-auxiliary-lemmas-GM}

\paragraph{Proof of Lemma~\ref{lem:perturbation-norm-Gaussian-mixture}.}

To begin with, let us decompose $\bm{E}$ as follows
\begin{align}
	\bm{E} & =\bm{X}^{\top}\bm{X}-\mathbb{E}\big[\bm{X}^{\top}\bm{X}\big]
	  = \big(\bm{\Theta}^{\star} \bm{F}^{\star\top} +\bm{Z}\big)^{\top}\big(\bm{\Theta}^{\star} \bm{F}^{\star\top} +\bm{Z}\big)
		- \mathbb{E}\big[\bm{X}^{\top}\bm{X}\big] \notag\\
 	& = \bm{F}^{\star}\bm{\Theta}^{\star\top}\bm{Z} + \bm{Z}^{\top}\bm{\Theta}^{\star}\bm{F}^{\star\top} + \bm{Z}^{\top}\bm{Z}-p\bm{I}_n,
	\label{eq:defn-E-decomposition-GMM}
\end{align}
where we have used the notation in \eqref{eq:defn-Z-X-EX-GMM} and \eqref{eq:defn-Theta-star-F-star-GMM}, as well as the identity (\ref{eq:mean-XT-X-GM}).
As it turns out, similar terms have already been controlled in the proof for PCA (see Section~\ref{sec:proof-lemmas:perturbation-size-E-PCA}). More precisely, the first term in \eqref{eq:defn-E-decomposition-GMM} obeys
\begin{align*}
\big\|\bm{F}^{\star}\bm{\Theta}^{\star\top}\bm{Z}\big\| &
%=\sqrt{\frac{n}{r}}\big\|\sqrt{\frac{r}{n}}\bm{F}^{\star}
%\bm{\Theta}^{\star\top}\bm{Z}\big\|
\overset{(\mathrm{i})}
{=}\sqrt{\frac{n}{r}}\big\|\bm{\Theta}^{\star\top}\bm{Z}\big\|
  \overset{(\mathrm{ii})}{\lesssim}  \sqrt{\frac{n}{r}}\cdot\frac{\Delta}{\sqrt{p}}\sqrt{np\log(n+p)} \\
	& \asymp  \frac{\Delta n\sqrt{\log(n+p)}}{\sqrt{r}}
\end{align*}
with probability exceeding $1-O\big((n+p)^{-10}\big)$, provided that $p\gtrsim r \log ^3 (n+p)$.
Here, the first relation (i) holds true since $\sqrt{\frac{r}{n}}\bm{F}^{\star}$ contains orthonormal columns and the spectral norm is unitarily invariant,
while (ii) invokes the high-probability bound (\ref{eq:norm-bound-FZtop-PCA}).
When it comes to the third term of \eqref{eq:defn-E-decomposition-GMM},
the bound \eqref{eq:norm-bound-ZZt-PCA} readily implies that
\[
	\big\|\bm{Z}^{\top}\bm{Z}-p\bm{I}_n\big\|  \lesssim  \sqrt{np\log(n+p)}+n\log^{2}(n+p)
\]
with probability at least $1-O((n+p)^{-10})$.
Substituting the preceding two bounds into \eqref{eq:defn-E-decomposition-GMM} and applying the triangle inequality, we reach
\begin{align*}
\|\bm{E}\| & \leq2\big\|\bm{F}^{\star}\bm{\Theta}^{\star\top}\bm{Z}\big\|+\big\|\bm{Z}^{\top}\bm{Z}-p\bm{I}_n\big\|\\
 & \lesssim\frac{\Delta n\sqrt{\log(n+p)}}{\sqrt{r}}+\sqrt{np\log(n+p)}+n\log^{2}(n+p) .
\end{align*}

\paragraph{Proof of Lemma~\ref{lemma:k-means-GMM} (analysis for $k$-means).}

Given the class labels $\{\xi_i\}_{i=1}^n$, the optimization of the cluster centers  $\{\bm{\vartheta}_i\}_{i=1}^r$ in the $k$-means formulation \eqref{eq:k-means-GMM} is achieved by the sample means of each cluster. Thus, the $k$-means formulation \eqref{eq:k-means-GMM} can be equivalently posed as solving
\begin{align}
	%\underset{\xi_{1},\cdots,\xi_{n}\in[r],\,\bm{\vartheta}_{1},\cdots,\bm{\vartheta}_{r}\in\mathbb{R}^{p}}{\min}\sum_{i=1}^{n}\big\|\bm{y}_{i}-\bm{\vartheta}_{\xi_{i}}\big\|_{2}^{2}
	 \underset{\mathcal{C}\in\Xi}{
		 \mathrm{minimize}} ~ \sum_{l=1}^{r} \sum_{i\in\mathcal{C}_{l}}\Big\|\bm{y}_{i}-\frac{1}{|\mathcal{C}_{l}|}\sum_{j\in\mathcal{C}_{l}}\bm{y}_{j}\Big\|_{2}^{2}
	\label{eq:k-means-formulation-Cl-GMM}
\end{align}
where $\mathcal{C}=\{\mathcal{C}_{1},\cdots,\mathcal{C}_{r}\}$ represents the cluster assignment,
and $\Xi$ denotes the set of all $r$-partitions of $[n]$ (i.e., $r$ disjoint subsets whose union equals $[n]$).

In order to tackle this formulation, a key ingredient of the proof lies in the following deviation bound that allows one to replace $\bm{y}_{i}$ with the truth $\bm{y}_{i}^{\star}$,
as long as the cluster size is sufficiently large.
\begin{claim}
\label{claim:deviation-S-GMM}
Consider any set $\mathcal{S}\subseteq[n]$ with cardinality $c_{s}n$ for some quantity $c_s >0$.
Suppose that (\ref{eq:assumption-Y-Ystar-Fro-GMM}) holds with $\varepsilon\leq c_{s}^{2}$. Then one has
\begin{align}
	\Bigg|\sum_{i\in\mathcal{S}}\Big\|\bm{y}_{i}-\frac{1}{|\mathcal{S}|}\sum_{j\in\mathcal{S}}\bm{y}_{j}\Big\|_{2}^{2}-\sum_{i\in\mathcal{S}}\Big\|\bm{y}_{i}^{\star}-\frac{1}{|\mathcal{S}|}\sum_{j\in\mathcal{S}}\bm{y}_{j}^{\star}\Big\|_{2}^{2} \Bigg|
	\leq  6 c_{s}\sqrt{\varepsilon}n.
	\label{eq:claim:deviation-S-GMM}
\end{align}
\end{claim}
With Claim~\ref{claim:deviation-S-GMM} in place, we are positioned to establish Lemma~\ref{lemma:k-means-GMM}  by contradiction; that is, we intend to demonstrate that any cluster assignment that differs too much from the ground-truth clusters cannot possibly be the $k$-means solution.
In what follows, we denote by $\mathcal{C}_{l}^{\star}$ the $l$-th ground-truth cluster $(1\leq l\leq r)$,
and let $\{\mathcal{C}_{1},\cdots,\mathcal{C}_{r}\}$ represent the minimizer of \eqref{eq:k-means-formulation-Cl-GMM} whenever it is clear from the context.

\bigskip
\noindent {\em Step 1: developing an upper bound on \eqref{eq:k-means-formulation-Cl-GMM}.}
To begin with, we derive an upper bound on the optimal objective value of \eqref{eq:k-means-formulation-Cl-GMM},
which serves as a reference in assessing the (sub)-optimality of other cluster assignments.
By virtue of Claim~\ref{claim:deviation-S-GMM} and the assumption $|\mathcal{C}_{l}^{\star}|=n/r$, one has
\[
\Bigg|\sum_{i\in\mathcal{C}_{l}^{\star}}\Big\|\bm{y}_{i}-\frac{1}{|\mathcal{C}_{l}^{\star}|}\sum_{j\in\mathcal{C}_{l}^{\star}}\bm{y}_{j}\Big\|_{2}^{2}-\sum_{i\in\mathcal{C}_{l}^{\star}}\Big\|\bm{y}_{i}^{\star}-\frac{1}{|\mathcal{C}_{l}^{\star}|}\sum_{j\in\mathcal{C}_{l}^{\star}}\bm{y}_{j}^{\star}\Big\|_{2}^{2}\Bigg|\leq\frac{6\sqrt{\varepsilon}n}{r},
\]
with the proviso that $\varepsilon \leq 1/r^2$.
Note that by construction,  for the ``ideal'' fitting,
one has $\sum_{i\in\mathcal{C}_{l}^{\star}}\big\|\bm{y}_{i}^{\star}-\frac{1}{|\mathcal{C}_{l}^{\star}|}\sum_{j\in\mathcal{C}_{l}^{\star}}\bm{y}_{j}^{\star}\big\|_{2}^{2}=0$.
Using this fact and summing the above inequality over
all $1\leq l \leq r$, we arrive at
\begin{align}
\sum_{l=1}^{r}\sum_{i\in\mathcal{C}_{l}^{\star}}\Big\|\bm{y}_{i}-\frac{1}{|\mathcal{C}_{l}^{\star}|}\sum_{j\in\mathcal{C}_{l}^{\star}}\bm{y}_{j}\Big\|_{2}^{2}
& \leq 6\sqrt{\varepsilon}n.
\label{eq:error-true-cluster-GMM}
\end{align}
Consequently, due to the assumed optimality of $\mathcal{C}$ w.r.t.~\eqref{eq:k-means-formulation-Cl-GMM},
replacing $\{\mathcal{C}_{l}^\star\}_{l=1}^r$ in \eqref{eq:error-true-cluster-GMM} by $\{\mathcal{C}_{l}\}_{l=1}^r$ can only further improve the objective value:
\begin{align}
\sum_{l=1}^{r}\sum_{i\in\mathcal{C}_{l}}\Big\|\bm{y}_{i}-\frac{1}{|\mathcal{C}_{l}|}\sum_{j\in\mathcal{C}_{l}}\bm{y}_{j}\Big\|_{2}^{2}
	& \leq \sum_{l=1}^{r}\sum_{i\in\mathcal{C}_{l}^{\star}}\Big\|\bm{y}_{i}-\frac{1}{|\mathcal{C}_{l}^{\star}|}\sum_{j\in\mathcal{C}_{l}^{\star}}\bm{y}_{j}\Big\|_{2}^{2}
	\notag\\
&  \leq 6\sqrt{\varepsilon}n.
\label{eq:error-optimizer-cluster-GMM}
\end{align}

\bigskip
\noindent {\em Step 2: showing that no cluster can be too large.}
Suppose that there exists a cluster $\mathcal{C}_{l}$ $(1\leq l\leq r)$ that is too large in the sense that
\begin{align}
	\label{eq:Cl-lower-bound-contradition-GMM}
	|\mathcal{C}_{l}|\geq\frac{(1+c_{\varepsilon})n}{r}
\end{align}
for some quantity $c_{\varepsilon}>0$. We would like to show that this is impossible unless $c_{\varepsilon}$ is  small;
in fact,  in light of Claim~\ref{claim:deviation-S-GMM},
we need only to establish a lower bound on the second term in \eqref{eq:claim:deviation-S-GMM} (with $\mathcal{S}=\mathcal{C}_{l}$)
so that it leads to a contradiction with the upper bound \eqref{eq:error-optimizer-cluster-GMM}.
Towards this end, we start with the elementary decomposition of the sum of squared errors:
\begin{align}
\sum_{i\in\mathcal{C}_{l}}\Big\|\bm{y}_{i}^{\star}-\frac{1}{|\mathcal{C}_{l}|}\sum_{j\in\mathcal{C}_{l}}\bm{y}_{j}^{\star}\Big\|_{2}^{2} & =\Bigg(\sum_{i\in\mathcal{C}_{l}}\big\|\bm{y}_{i}^{\star}\big\|_{2}^{2}\Bigg)-|\mathcal{C}_{l}|\cdot\Big\|\frac{1}{|\mathcal{C}_{l}|}\sum_{j\in\mathcal{C}_{l}}\bm{y}_{j}^{\star}\Big\|_{2}^{2}\nonumber \\
 & =|\mathcal{C}_{l}| \Big(1-\Big\|\frac{1}{|\mathcal{C}_{l}|}\sum_{j\in\mathcal{C}_{l}}\bm{y}_{j}^{\star}\Big\|_{2}^{2}\Big),\label{eq:fitting-error-cluster-GMM-123}
\end{align}
where we have invoked the fact that $\big\|\bm{y}_{i}^{\star}\big\|_{2}=1$.
It thus comes down to controlling $\big\|\sum_{j\in\mathcal{C}_{l}}\bm{y}_{j}^{\star}\big\|_{2}^{2}$.
For notational convenience, for any $1\leq\tau\leq r$, we set $n_{l,\tau} \coloneqq \big|\mathcal{C}_{l}\cap\mathcal{C}_{\tau}^{\star}\big|$
(namely, the number of points in $\mathcal{C}_{l}$ coming from the $\tau$-th ground-truth cluster),
and let $\bm{y}_{(\tau)}^{\star}$ represent the vector associated with the $\tau$-th cluster
(namely, $\bm{y}_{(\tau)}^{\star}=\bm{y}_{j}^{\star}$ for any $j\in\mathcal{C}_{\tau}^{\star}$).
Armed with this set of notation, we can write
\begin{equation}
\Big\|\frac{1}{|\mathcal{C}_{l}|}\sum_{j\in\mathcal{C}_{l}}\bm{y}_{j}^{\star}\Big\|_{2}^{2}
= \Big\|\sum_{\tau=1}^{r}\frac{n_{l,\tau}}{|\mathcal{C}_{l}|}\bm{y}_{(\tau)}^{\star}\Big\|_{2}^{2}
= \sum_{\tau=1}^{r}\Big\|\frac{n_{l,\tau}}{|\mathcal{C}_{l}|}\bm{y}_{(\tau)}^{\star}\Big\|_{2}^{2}
=\frac{\sum_{\tau=1}^{r}n_{l,\tau}^{2}}{|\mathcal{C}_{l}|^{2}} .
\label{eq:mean-yjstar-decompose-456-GMM}
\end{equation}
Here, the penultimate identity holds since $\langle\bm{y}_{(i)}^{\star},\bm{y}_{(j)}^{\star}\rangle=0$ for any $i\neq j$,
while the last relation relies on the fact that $\big\|\bm{y}_{(\tau)}^{\star}\big\|_{2}=1$.
Using $0\leq n_{l,\tau} \leq n/r$ and $\sum_{\tau=1}^{r}n_{l,\tau} = |\mathcal{C}_l| $, we have
\begin{align*}
	\eqref{eq:mean-yjstar-decompose-456-GMM}  & \leq   \frac{ \big( \max_{\tau} n_{l,\tau} \big) \big( \sum_{\tau} n_{l,\tau} \big) }{|\mathcal{C}_{l}|^2}
	\leq \frac{(n/r) \cdot |\mathcal{C}_l|}{|\mathcal{C}_{l}|^2}
	%= \frac{n/r}{|\mathcal{C}_{l}|}
\leq\frac{1}{1+c_{\varepsilon}},
\end{align*}
where the first inequality comes from the basic fact that $\|\bm{a}\|_2^2 \leq \|\bm{a}\|_{\infty}\|\bm{a}\|_{1}$ for any vector $\bm{a}$, and the last relation arises from the assumed cardinality constraint on $\mathcal{C}_{l}$ (cf.~(\ref{eq:Cl-lower-bound-contradition-GMM})).

%To further upper bound (\ref{eq:mean-yjstar-decompose-456-GMM}), the following claim proves useful.
%%
%\begin{claim}
%	\label{claim:simple-opt-GMM}
%	Suppose $\{a_{\tau}^{\star}\}_{1\leq\tau\leq r}$ is the optimizer of the following problem:
%%
%\begin{align*}
%\text{maximize}\quad  & f\big(\{a_{\tau}\}_{1\leq\tau\leq r}\big) \coloneqq A^{-2}\sum\nolimits_{\tau=1}^{r}a_{\tau}^{2}\\
%\text{subject to}\quad  & 0\leq a_{\tau}\leq A_{0}\ (1\leq\tau\leq r),\ \sum\nolimits_{\tau=1}^{r}a_{\tau}=A,
%\end{align*}
%
%where $A_0, A>0$.
%Then one has $f\big(\{a_{\tau}^{\star}\}_{1\leq\tau\leq r}\big)\leq A_{0}/A$.
%\end{claim}

%%
%With the assistance of Claim~\ref{claim:simple-opt-GMM}, we can derive
%
%\begin{align*}
%	\eqref{eq:mean-yjstar-decompose-456-GMM}  & \leq\frac{n/r}{|\mathcal{C}_{l}|}\leq\frac{1}{1+c_{\varepsilon}},
%\end{align*}
%
%which relies on the fact $n_{l,\tau}\leq |\mathcal{C}_{\tau}^{\star}| = n/r$ and the assumption \eqref{eq:Cl-lower-bound-contradition-GMM}.
Substitution into (\ref{eq:fitting-error-cluster-GMM-123}) yields
\begin{align}
\sum_{i\in\mathcal{C}_{l}}\Big\|\bm{y}_{i}^{\star}-\frac{1}{|\mathcal{C}_{l}|}\sum_{j\in\mathcal{C}_{l}}\bm{y}_{j}^{\star}\Big\|_{2}^{2}
& \geq|\mathcal{C}_{l}| \Big(1-\frac{1}{1+c_{\varepsilon}}\Big)
\geq%\frac{c_{\varepsilon}}{1+c_{\varepsilon}}
c_{\varepsilon}
\frac{n}{r},
\label{eq:fitting-error-cluster-GMM-123-1}
\end{align}
where the last inequality again arises from the assumption \eqref{eq:Cl-lower-bound-contradition-GMM}. %with $c_{\varepsilon}>0$.
This in turn demonstrates that
\begin{align}
 & \sum_{l=1}^{r}\sum_{i\in\mathcal{C}_{l}}\Big\|\bm{y}_{i}-\frac{1}{|\mathcal{C}_{l}|}\sum_{j\in\mathcal{C}_{l}}\bm{y}_{j}\Big\|_{2}^{2}\geq\sum_{i\in\mathcal{C}_{l}}\Big\|\bm{y}_{i}-\frac{1}{|\mathcal{C}_{l}|}\sum_{j\in\mathcal{C}_{l}}\bm{y}_{j}\Big\|_{2}^{2} \notag\\
 & \qquad \geq \sum_{i\in\mathcal{C}_{l}}\Big\|\bm{y}_{i}^{\star}-\frac{1}{|\mathcal{C}_{l}|}\sum_{j\in\mathcal{C}_{l}}\bm{y}_{j}^{\star}\Big\|_{2}^{2}-6\frac{|\mathcal{C}_{l}|}{n}\sqrt{\varepsilon}n
	\label{eq:cluster-SS-lower-bound-GMM-345}\\
 & \qquad \geq %\frac{c_{\varepsilon}}{1+c_{\varepsilon}}
   c_{\varepsilon}  \frac{n}{r}-6\sqrt{\varepsilon}n
   > 6\sqrt{\varepsilon}n, \notag
\end{align}
where \eqref{eq:cluster-SS-lower-bound-GMM-345} results from  Claim~\ref{claim:deviation-S-GMM} when $\varepsilon\leq 1/r^2$, and the last inequality follows as long as $12\sqrt{\varepsilon} < %\frac{c_{\varepsilon}}{1+c_{\varepsilon}}
c_{\varepsilon} /r$.
Comparing this  with \eqref{eq:error-optimizer-cluster-GMM} leads to  contradiction with the optimality assumption of $\{\mathcal{C}_l\}$.

\bigskip
\noindent {\em Step 3: showing that no cluster can be too small.}
Suppose now that there exists a cluster $\mathcal{C}_{i}$ ($1\leq i\leq r$) obeying
\[
	|\mathcal{C}_{i}|\leq\frac{\big(1-c_{\varepsilon}(r-1)\big)n}{r}.
\]
Then from the pigeonhole principle, one can find another cluster $\mathcal{C}_{l}$ ($1\leq l\leq r$) with cardinality exceeding
\[
	|\mathcal{C}_{l}|\geq\frac{(1+c_{\varepsilon})n}{r};
\]
otherwise the total size obeys $\sum_{k=1}^r |\mathcal{C}_{k}| < \frac{ (1-c_{\varepsilon}(r-1) )n}{r} + (r-1)\frac{(1+c_{\varepsilon})n}{r} \leq n$ and $\{\mathcal{C}_l\}$ is infeasible.
The above condition on $|\mathcal{C}_{l}|$ coincides with the assumption \eqref{eq:Cl-lower-bound-contradition-GMM} in Step 2,
which, as a result of previous arguments, cannot possibly hold.
To conclude, for all $1\leq i\leq r$, one necessarily has
\begin{align}
	\label{eq:Cl-lower-bound-GMM-357}
	|\mathcal{C}_{i}| > \frac{\big(1-c_{\varepsilon}(r-1)\big)n}{r} \geq \frac{\big(1-c_{\varepsilon}r \big)n}{r}.
\end{align}

\bigskip
\noindent {\em Step 4: showing that each $\mathcal{C}_{l}$ is mainly composed of points from a true (and distinct) cluster.}
Suppose that there exists a cluster $\mathcal{C}_{l}$ ($1\leq l\leq r$) whose dominant component obeys
\[
	\max_{1\leq\tau\leq r}n_{l,\tau}  \leq  (1-2rc_{\varepsilon})n/r,
\]
where we recall that $n_{l,\tau}=|\mathcal{C}_{l}\cap \mathcal{C}_{\tau}^{\star}|$.
Under this assumption, we have
\begin{align*}
	\eqref{eq:mean-yjstar-decompose-456-GMM}
	 \leq \frac{\big(\max_{\tau}n_{l,\tau}\big)\big( \sum_{\tau}n_{l,\tau}\big)}{|\mathcal{C}_{l}|^2}
	%\leq \frac{\big[ (1-2rc_{\varepsilon})n/r \big] \cdot  |\mathcal{C}_{l}| }{|\mathcal{C}_{l}|^2} \\
	\leq \frac{[(1-2rc_{\varepsilon})n/r]\cdot |\mathcal{C}_l|}{|\mathcal{C}_{l}|^2}
	\leq \frac{1-2rc_{\varepsilon}}{1-rc_{\varepsilon}},
\end{align*}
where the last inequality relies on the lower bound \eqref{eq:Cl-lower-bound-GMM-357}.
Substitution into (\ref{eq:fitting-error-cluster-GMM-123}) gives
\begin{align*}
\sum_{i\in\mathcal{C}_{l}}\Big\|\bm{y}_{i}^{\star}-\frac{1}{|\mathcal{C}_{l}|}\sum_{j\in\mathcal{C}_{l}}\bm{y}_{j}^{\star}\Big\|_{2}^{2}
& \geq|\mathcal{C}_{l}| \Big(1-\frac{1-2rc_{\varepsilon}}{1-rc_{\varepsilon}}\Big)
\overset{\mathrm{(i)}}{\geq} \frac{(1-rc_{\varepsilon})n}{r} \cdot \frac{rc_{\varepsilon}}{1-rc_{\varepsilon}} \\
 & =c_{\varepsilon}n > 12\sqrt{\varepsilon}n,
\end{align*}
where (i) arises again from \eqref{eq:Cl-lower-bound-GMM-357},
and the last relation is valid once $c_{\varepsilon}> 12\sqrt{\varepsilon}$.
This taken collectively with the inequality \eqref{eq:cluster-SS-lower-bound-GMM-345} yields
\begin{align*}
\sum_{l=1}^{r}\sum_{i\in\mathcal{C}_{l}}\Big\|\bm{y}_{i}-\frac{1}{|\mathcal{C}_{l}|}\sum_{j\in\mathcal{C}_{l}}\bm{y}_{j}\Big\|_{2}^{2} & \geq\sum_{i\in\mathcal{C}_{l}}\Big\|\bm{y}_{i}^{\star}-\frac{1}{|\mathcal{C}_{l}|}\sum_{j\in\mathcal{C}_{l}}\bm{y}_{j}^{\star}\Big\|_{2}^{2}-6\sqrt{\varepsilon}n\\
 & > 6 \sqrt{\varepsilon}n,
\end{align*}
which, however, contradicts the upper bound \eqref{eq:error-optimizer-cluster-GMM}.
Consequently, the dominant component in every cluster $1\leq l\leq r$ must obey
\begin{align}
	\max_{1\leq\tau\leq r}n_{l,\tau}  >  (1-2rc_{\varepsilon})n/r.
	\label{eq:lower-bound-n-l-tau-GMM}
\end{align}
In particular, if $2rc_{\varepsilon}< 1/2$, then $\max_{1\leq\tau\leq r}n_{l,\tau} > n/(2r)$.
An immediate consequence is that: the dominant components of the clusters $\{\mathcal{C}_l\}$ must come from distinct ground-truth clusters.

\bigskip
\noindent {\em Step 5: putting all this together.}
Armed with the preceding bound \eqref{eq:lower-bound-n-l-tau-GMM} and the remark thereafter,
it is straightforward to verify the following result on the mis-clustering rate:
\[
\ell_{\mathsf{mis}}\big(\{\widehat{\xi}_{i}\},\{\xi_{i}^{\star}\}\big)\leq1-\frac{\sum_{l=1}^{r}\max_{1\leq\tau\leq r}n_{l,\tau}}{n}\leq2rc_{\varepsilon}.
\]
Finally, setting $c_{\varepsilon} = \varepsilon^{1/4}$ leads to the advertised result,
provided that $\varepsilon\leq c_{3}r^{-4}$ for some sufficiently small constant $c_3>0$.

\paragraph{Proof of Claim~\ref{claim:deviation-S-GMM}.}
Before proceeding, let us take a quick look at how many columns
of $\bm{Y}$ might deviate considerably from their counterparts in
$\bm{Y}^{\star}$. To be precise, let us introduce the following set
\begin{equation}
	\mathcal{N}_{\mathsf{large}}\coloneqq\left\{ i\mid\|\bm{y}_{i}-\bm{y}_{i}^{\star}\|_{2}^{2}\geq\varepsilon\right\} .
	\label{eq:defn-Nlarge-GMM}
\end{equation}
Clearly, its cardinality is necessarily bounded above by
\begin{equation}
	\big|\mathcal{N}_{\mathsf{large}}\big|\leq\frac{\big\|\bm{Y}-\bm{Y}^{\star}\big\|_{\mathrm{F}}^{2}}{\varepsilon}\leq\frac{\varepsilon^{2}n}{\varepsilon}=\varepsilon n.
	\label{eq:size-Nlarge-UB-GMM}
\end{equation}
The starting point of the proof is the elementary identities
\begin{align}
 &  \sum_{i\in\mathcal{S}}\Big\|\bm{y}_{i}-\frac{1}{|\mathcal{S}|}\sum_{j\in\mathcal{S}}\bm{y}_{j}\Big\|_{2}^{2}-\sum_{i\in\mathcal{S}}\Big\|\bm{y}_{i}^{\star}-\frac{1}{|\mathcal{S}|}\sum_{j\in\mathcal{S}}\bm{y}_{j}^{\star}\Big\|_{2}^{2}  \nonumber \\
 & \quad=\sum_{i\in\mathcal{S}}\big\|\bm{y}_{i}\big\|_{2}^{2}-|\mathcal{S}|\cdot\Big\|\frac{1}{|\mathcal{S}|}\sum_{j\in\mathcal{S}}\bm{y}_{j}\Big\|_{2}^{2}  - \Big( \sum_{i\in\mathcal{S}}\big\|\bm{y}_{i}^{\star}\big\|_{2}^{2}-|\mathcal{S}|\cdot\Big\|\frac{1}{|\mathcal{S}|}\sum_{j\in\mathcal{S}}\bm{y}_{j}^{\star}\Big\|_{2}^{2} \Big) \nonumber\\
 & \quad = |\mathcal{S}|\cdot \Bigg( \Big\|\frac{1}{|\mathcal{S}|}\sum_{j\in\mathcal{S}}\bm{y}_{j}^{\star}\Big\|_{2}^{2} -\Big\|\frac{1}{|\mathcal{S}|}\sum_{j\in\mathcal{S}}\bm{y}_{j}\Big\|_{2}^{2} \Bigg), \label{eq:sum-y-S-decompose-GMM}
\end{align}
where the last inequality follows from the fact that $\big\|\bm{y}_{i}\big\|_{2}=\big\|\bm{y}_{i}^{\star}\big\|_{2}=1$.
To bound the right-hand side of (\ref{eq:sum-y-S-decompose-GMM}),
we make the observation that
\begin{align}
\Big\|\frac{1}{|\mathcal{S}|}\sum_{j\in\mathcal{S}}\big(\bm{y}_{j}-\bm{y}_{j}^{\star}\big)\Big\|_{2} & \leq\frac{1}{|\mathcal{S}|}\sum_{j\in\mathcal{S}\backslash\mathcal{N}_{\mathsf{large}}}\big\|\bm{y}_{j}-\bm{y}_{j}^{\star}\big\|_{2}+\frac{1}{|\mathcal{S}|}\sum_{j\in\mathcal{S}\cap\mathcal{N}_{\mathsf{large}}}\big\|\bm{y}_{j}-\bm{y}_{j}^{\star}\big\|_{2}\notag\\
 & \overset{(\mathrm{i})}{\leq}\sqrt{\varepsilon}+\frac{2\big|\mathcal{N}_{\mathsf{large}}\big|}{|\mathcal{S}|}\overset{(\mathrm{ii})}{\leq}\sqrt{\varepsilon}+\frac{2\varepsilon}{c_{s}}\leq 3\sqrt{\varepsilon}. \label{eq:sum-y-diff}
\end{align}
Here, (i) holds true since any column outside $\mathcal{N}_{\mathsf{large}}$
satisfies $\big\|\bm{y}_{j}-\bm{y}_{j}^{\star}\big\|_{2} < \sqrt{\varepsilon}$,
and any column coming from $\mathcal{N}_{\mathsf{large}}$ obeys
$\big\|\bm{y}_{j}-\bm{y}_{j}^{\star}\big\|_{2}\leq\big\|\bm{y}_{j}\big\|_{2}+\big\|\bm{y}_{j}^{\star}\big\|_{2} = 2$;
(ii) follows from (\ref{eq:size-Nlarge-UB-GMM}) and the assumption
$|\mathcal{S}|=c_{s}n$; and the last inequality relies on the assumption
$\varepsilon\leq c_{s}^{2}$.
In addition,
\begin{subequations}\label{eq:sum-y-ub}
\begin{align}
\Big\|\frac{1}{|\mathcal{S}|}\sum_{j\in\mathcal{S}}\bm{y}_{j}^{\star}\Big\|_{2} &\leq\frac{1}{|\mathcal{S}|}\sum_{j\in\mathcal{S}}\big\|\bm{y}_{j}^{\star}\big\|_{2}=1; \\
\Big\|\frac{1}{|\mathcal{S}|}\sum_{j\in\mathcal{S}}\bm{y}_{j} \Big\|_{2} &\leq\frac{1}{|\mathcal{S}|}\sum_{j\in\mathcal{S}}\big\|\bm{y}_{j} \big\|_{2}=1.
\end{align}
\end{subequations}
Combining  \eqref{eq:sum-y-diff} and \eqref{eq:sum-y-ub} and applying the triangle inequality, we arrive at
\begin{align}
& \Bigg|\Big\|\frac{1}{|\mathcal{S}|}\sum_{j\in\mathcal{S}}\bm{y}_{j}\Big\|_{2}^{2}-\Big\|\frac{1}{|\mathcal{S}|}\sum_{j\in\mathcal{S}}\bm{y}_{j}^{\star}\Big\|_{2}^{2}\Bigg| \notag\\
& \qquad  =\Bigg|\Big\|\frac{1}{|\mathcal{S}|}\sum_{j\in\mathcal{S}}\bm{y}_{j}\Big\|_{2} -\Big\|\frac{1}{|\mathcal{S}|}\sum_{j\in\mathcal{S}}\bm{y}_{j}^{\star}\Big\|_{2} \Bigg|\cdot \Bigg|\Big\|\frac{1}{|\mathcal{S}|}\sum_{j\in\mathcal{S}}\bm{y}_{j}\Big\|_{2} + \Big\|\frac{1}{|\mathcal{S}|}\sum_{j\in\mathcal{S}}\bm{y}_{j}^{\star}\Big\|_{2} \Bigg| \notag\\
& \qquad \leq 2\Bigg|\Big\|\frac{1}{|\mathcal{S}|}\sum_{j\in\mathcal{S}}\bm{y}_{j}\Big\|_{2} -\Big\|\frac{1}{|\mathcal{S}|}\sum_{j\in\mathcal{S}}\bm{y}_{j}^{\star}\Big\|_{2} \Bigg| \notag \\
& \qquad \leq 2 \Big\|\frac{1}{|\mathcal{S}|}\sum_{j\in\mathcal{S}}(\bm{y}_{j}-\bm{y}_{j}^{\star})\Big\|_{2}
	%\nonumber \\
%	+ 2\Big\|\frac{1}{|\mathcal{S}|}\sum_{j\in\mathcal{S}}(\bm{y}_{j}-\bm{y}_{j}^{\star})\Big\|_{2}\Big\|\frac{1}{|\mathcal{S}|}\sum_{j\in\mathcal{S}}\bm{y}_{j}^{\star}\Big\|_{2}\nonumber \\
  \leq 6\sqrt{\varepsilon}.\label{eq:mean-difference-123-GMM}
\end{align}
%9\varepsilon+\leq 15 \sqrt{\varepsilon}
Finally, plugging in (\ref{eq:mean-difference-123-GMM})
into (\ref{eq:sum-y-S-decompose-GMM}),
we conclude that
\begin{align*}
 & \Bigg| \sum_{i\in\mathcal{S}}\Big\|\bm{y}_{i}-\frac{1}{|\mathcal{S}|}\sum_{j\in\mathcal{S}}\bm{y}_{j}\Big\|_{2}^{2}-\sum_{i\in\mathcal{S}}\Big\|\bm{y}_{i}^{\star}-\frac{1}{|\mathcal{S}|}\sum_{j\in\mathcal{S}}\bm{y}_{j}^{\star}\Big\|_{2}^{2} \Bigg|\\
 & \quad = |\mathcal{S}|
 \cdot\Bigg|\Big\|\frac{1}{|\mathcal{S}|}\sum_{j\in\mathcal{S}}\bm{y}_{j}\Big\|_{2}^{2}-\Big\|\frac{1}{|\mathcal{S}|}\sum_{j\in\mathcal{S}}\bm{y}_{j}^{\star}\Big\|_{2}^{2}  \Bigg|\\
 & \quad\leq   6 \sqrt{\varepsilon}|\mathcal{S}| = 6 c_{s}\sqrt{\varepsilon}n
\end{align*}
as claimed, where the last relation holds true since $|\mathcal{S}|=c_s n$.

%% file: chapters/ranking.tex
\section{Ranking from pairwise comparisons}
\label{sec:ranking}

The ranking task---which seeks to identify a consistent ordering of several items based on
(partially) revealed preference information about them---is encountered in numerous contexts
including web search, crowd sourcing, social choice, peer grading, and so on \citep{dwork2001rank,chen2013pairwise,caplin1991aggregation,shah2013case}.
Of particular interest is the ``preference-based'' observation model,
in which we are only given relative comparisons of a few items (as opposed to individual scores of them).
In practice, comparison data of this kind abound,
partly because humans often find it easier to make a preference over two or a couple of items than to assign specific ratings to many individual ones.
The emergence of crowdsourcing platforms such as Amazon Mechanical Turk further widens the availability of comparison data,
where binary judgments over pairs of items are often solicited from a pool of non-experts.
In this section, we concentrate on pairwise comparisons and
explore the potential of spectral methods for the ranking task.

\subsection{The Bradley-Terry-Luce model and assumptions \label{subsec:Problem-setup-and-ranking}}

To formulate the problem in a statistically sound manner, we introduce a classical parametric model,
called the Bradley-Terry-Luce (BTL) model~\citep{bradley1952rank,ford1957solution,luce2012individual},
to describe the generating process of pairwise comparisons.

\paragraph{Latent preference scores.}
Imagine that there are $n$ items to be ranked.
A key component of the BTL model is the assignment of a latent preference score to each item; more concretely, the BTL model hypothesizes on the existence of an unseen preference score vector
\begin{equation}
	\bm{w}^{\star}=[w_{1}^{\star},w_{2}^{\star},\cdots, w_{n}^{\star}]^{\top},
	\label{eq:ranking-score}
\end{equation}
with $w_{i}^{\star}>0$ assigned to the $i$-th item ($1\leq i\leq n$).
The ranks of these items are therefore determined exclusively by their (relative) preference scores: an item with a larger score is ranked higher.
Throughout this section, we denote by $\kappa$ a sort of  condition number as follows
\begin{equation}
	\kappa \coloneqq\frac{\max_{1\leq i\leq n}w_{i}^{\star}}{\min_{1\leq i\leq n}w_{i}^{\star}}.
	\label{eq:ranking-well-condition}
\end{equation}

%$[w_{\min}, w_{\max}]$ the dynamic range of the scores (so that  $w_i[w_{\min}, w_{\max}]$ for all $1\leq i\leq n$), and let

\paragraph{Pairwise comparisons.}
Equipped with the aforementioned score
vector, the BTL model posits that: the probability of an item winning a paired comparison is determined entirely by the
relative scores of the two items involved. To be precise, when comparing  every pair $(i,j)$ of items, the model assumes that
\begin{equation}
	\mathbb{P} \big\{ \text{item }j\text{ is preferred over item }i \big\}
	= \frac{w_{j}^{\star}}{w_{i}^{\star}+ w_{j}^{\star}},
	\label{eq:BTL}
\end{equation}
asserting that an item assigned a higher preference score is more likely to win.
In this section, we assume access to a comparison between every pair of items. To be precise, for each pair
$(i,j)$ ($1\le i<j\le n$), we observe an independent binary comparison outcome $y_{i,j}$ following the BTL model \eqref{eq:BTL}:
\[
y_{i,j}=\begin{cases}
1, & \text{with probability }\frac{w_{j}^{\star}}{w_{i}^{\star}+w_{j}^{\star}},\\
0, & \text{otherwise},
\end{cases}
\]
where $y_{i,j}=1$ means item $j$ beats item $i$ and $y_{i,j}=0$ otherwise.
By convention, we set $y_{i,j}=1-y_{j,i}$ for all $i>j$.
%
%To simplify the notation hereafter, we denote
%%
%\begin{equation}
%	y_{i,j}^{\star}\coloneqq\frac{w_{j}^{\star}}{w_{i}^{\star}+w_{j}^{\star}}.
%	\label{eq:true-comparison-prob}
%\end{equation}
%%
%

%

\paragraph{Goal. }
With the BTL parametric model in mind,
a natural strategy is to start by estimating the underlying scores $\{ w_i^{\star}\}$ based on the pairwise comparisons in hand, followed by a ranking step performed in  accordance with the estimated scores.
In this section, we shall focus on characterizing the statistical accuracy of spectral methods in accomplishing the meta task of preference score estimation,
and will remark in passing on the ranking step that follows. Obviously, from \eqref{eq:BTL}, we can only hope for estimating $\{ w_i^{\star}\}$ up to some global scaling ambiguity.

% Motivated by this idea, this section explores the capability of spectral methods in accomplishing the meta task of score estimation.
%
%Due to the scale invariance of the BTL observation model (see~(\ref{eq:BTL})),
%it is assumed without loss of generality that
%%
%\begin{equation}
% 	\sum_{i=1}^{n}w_{i}^{\star}=1.
%	\label{eq:BTL-normalization}
%\end{equation}
%%

\subsection{A spectral ranking algorithm \label{subsec:Spectral-method-for-ranking}}

At first glance, the recipe we have introduced for designing  spectral methods seems to have no direct bearing on the BTL model.
Somewhat unexpectedly, a closer inspection unveils an intimate connection between the BTL model and a reversible Markov chain,
whose stationary distribution embodies crucial information about the score vector of interest.
This in turn lays a solid foundation for the spectral algorithm described below, originally developed by \citet{negahban2016rank}.

The first step is to convert the pairwise comparison data
$\{y_{i,j}\}_{i\neq j}$ into a probability transition matrix $\bm{P}=[P_{i,j}]_{1\leq i,j\leq n}$, in a way that
\begin{equation}
P_{i,j}=\begin{cases}
\frac{1}{n}y_{i,j}, & \text{if }i\neq j,\\
1-\sum_{j:j\neq i}\frac{1}{n}y_{i,j}, & \text{otherwise}.
\end{cases}
\label{defn:ranking-P-empirical}
\end{equation}
By construction of $\bm{P}$, all of its entries are non-negative and the entries in each row sum up to one,
thus confirming that  $\bm{P}$ is a probability transition matrix.
The spectral algorithm then computes the leading left eigenvector $\bm{\pi}$ of  $\bm{P}$, returning it as
 the estimate for the underlying score vector $\bm{w}^{\star}$.

To make sense of the rationale behind this algorithm, it is helpful to
look at the mean $\bm{P}^{\star} =[P_{i,j}^{\star}]_{1\leq i,j\leq n} \coloneqq\mathbb{E}[\bm{P}]$, which obeys
\begin{equation}
P_{i,j}^{\star}=\begin{cases}
	\frac{1}{n} \frac{w_j^{\star}}{w_i^{\star}+w_j^{\star}}, & \text{if }i\neq j,\\
	1- \frac{1}{n}  \sum_{j:j\neq i}\frac{w_j^{\star}}{w_i^{\star}+w_j^{\star}}, & \text{otherwise}.
\end{cases}\label{defn:ranking-P-star}
\end{equation}
Clearly, this matrix $\bm{P}^{\star}$ is a probability transition matrix as well.
As can be straightforwardly verified, the vector $\bm{\pi}^{\star}=[\pi_{i}^{\star}]_{1\leq i\leq n}$ defined by
\begin{align}
	%\pi_i^{\star} = \frac{w_i^{\star}}{\sum_{l=1}^n w_l^{\star}},\quad 1\leq i\leq n
	%\qquad \text{or} \qquad
	\bm{\pi}^{\star} = \frac{1}{\bm{1}^{\top}\bm{w}^{\star}} \bm{w}^{\star}
	\label{defn:pi-i-star-BTL}
\end{align}
satisfies the following conditions:
\begin{itemize}
	\item $\bm{\pi}^{\star}$ is a probability vector (i.e., $\pi_i^{\star}\geq 0$ for all $i$ and $\sum_i\pi_i^{\star}=1$);
	\item $\bm{\pi}^{\star}$ satisfies the detailed balance equations as follows:
\[
	\pi_{i}^{\star}P_{i,j}^{\star} = \pi_{j}^{\star}P_{j,i}^{\star},\qquad\text{for all }(i,j).
\]
\end{itemize}
Classical Markov chain theory \citep{bremaud2013markov} thus tells us that
$\bm{P}^{\star}$ represents a reversible Markov chain, whose stationary distribution is precisely given by $\bm{\pi}^{\star}$ (this can easily be verified using the definition of the stationary distribution) and corresponds to the normalized preference scores.
As a consequence, we hold the intuition that:  $\bm{\pi}$ is close to $\bm{\pi}^{\star}$---and hence $\bm{w}^{\star}$ up to some global scaling---as long as $\bm{P}$ approximates $\bm{P}^{\star}$ reasonably well.

\subsection{Performance guarantees}
\label{sec:theory-ranking}

This subsection develops theoretical support for the above spectral ranking algorithm,
based on the eigenvector perturbation theory developed previously for probability transition matrices in Section~\ref{sec:eigenvector-theory-DK}.
For notational convenience, we shall use $\bm{E}\coloneqq\bm{P}-\bm{P}^{\star}$ to
denote the difference of the above two transition matrices of interest.

By virtue of Theorem~\ref{thm:DK_asym},
the perturbation of the stationary distribution of a reversible Markov
chain $\bm{P}^{\star}$ is dictated by two important quantities: (i)
the spectral gap $1-\max\left\{ \lambda_{2}(\bm{P}^{\star}),-\lambda_{n}(\bm{P}^{\star})\right\} $,
and (ii) the noise size $\|\bm{E}\|_{\bm{\pi}^{\star}}$ (recall the definition of $\|\cdot\|_{\bm{\pi}^{\star}}$  in Section~\ref{subsec:Setup-and-notation-prob-matrix}).
These two quantities are controlled respectively via the following two lemmas, whose
proofs can be found in Section~\ref{subsec:Proof-of-auxilliary-ranking}.

\begin{lemma}
	\label{lemma:ranking-gap}
	Consider the settings and notation in Sections~\ref{subsec:Problem-setup-and-ranking} and~\ref{subsec:Spectral-method-for-ranking}. It follows that
\[
1-\max\big\{ \lambda_{2}(\bm{P}^{\star}),-\lambda_{n}(\bm{P}^{\star}) \big\} \geq\frac{1}{2\kappa^{2}},
\]
where we recall the definition of $\kappa$  in (\ref{eq:ranking-well-condition}).
\end{lemma}
\begin{lemma}
	\label{lemma:ranking-noise}
Consider the settings and notation in
Sections~\ref{subsec:Problem-setup-and-ranking} and \ref{subsec:Spectral-method-for-ranking}, and recall that $\bm{E}\coloneqq\bm{P}-\bm{P}^{\star}$.
With probability at least $1-O(n^{-8})$,
\[
	\|\bm{E}\|_{\bm{\pi}^{\star}}\leq \sqrt{\kappa}\,\|\bm{E}\| \lesssim\sqrt{\frac{\kappa\log n}{n}}.
\]
\end{lemma}
Now we are well prepared to assess the quality of the spectral estimate, as summarized below, whose proof is given at the end of this subsection.
\begin{theorem}
\label{thm:ranking}
Consider the settings and algorithm in Sections~\ref{subsec:Problem-setup-and-ranking}
and \ref{subsec:Spectral-method-for-ranking}. Suppose that $n\geq C\kappa^{5}\log n$
for some sufficiently large constant $C>0$. Then with probability
exceeding $1-O(n^{-8})$, one has
\begin{align}
	\frac{\|\bm{\pi}-\bm{\pi}^{\star}\|_{2}}{\|\bm{\pi}^{\star}\|_{2}}\lesssim\kappa^{2.5}\sqrt{\frac{\log n}{n}}.
	\label{eq:pi-pistar-quality-ranking}
\end{align}
\end{theorem}
Given the construction \eqref{defn:pi-i-star-BTL} of $\bm{\pi}^{\star}$, this theorem implies the existence of a scalar $z>0$ such that
\[
	\frac{\| z\bm{\pi} - \bm{w}^{\star}\|_{2}}{\|\bm{w}^{\star}\|_{2}}\lesssim\kappa^{2.5}\sqrt{\frac{\log n}{n}}
\]
holds with high probability. It is worth noting that one cannot possibly retrieve the global scaling factor $z$, due to the invariance of the BTL observation model under global scaling (cf.~(\ref{eq:BTL})).

To interpret the effectiveness of this theorem, consider, for example, the case when $\kappa=O(1)$ (so that all the latent scores
$w_{i}^{\star}$ are about the same order). Theorem~\ref{thm:ranking}
tells us that the relative estimation error of $\bm{\pi}$ is vanishing
as the number $n$ of items increases. As it turns out, this statistical error rate \eqref{eq:pi-pistar-quality-ranking}
 is near minimax-optimal up to a logarithmic
factor; see~\citet[Theorem 3]{negahban2016rank}. In fact, with a more careful
analysis, one can further eliminate this extra $\log n$ factor and establish (orderwise) minimax optimality of this algorithm,
as has been done in \citet[Theorem 5.2]{chen2017spectral}.

Caution needs to be exercised, however, that high score estimation accuracy alone does not necessarily imply appealing ranking accuracy.
For instance, if the goal is to identify the top-$K$ ranked items---a problem commonly referred to as ``top-$K$ ranking'' \citep{chen2015spectral}---then the ranking accuracy also relies heavily on the separation between the score of the $K$-th ranked item and that of the $(K+1)$-th ranked item (namely, whether the set of top-$K$ ranked items is sufficiently distinguishable from the remaining ones). Fortunately, the spectral ranking algorithm introduced in this section remains minimax optimal when it comes to top-$K$ ranking, through a refined $\ell_\infty$ perturbation theory to be introduced in Chapter~\ref{cha:Linf-theory}. The interested reader is referred to \citet{chen2017spectral} for details.

\begin{proof}[Proof of Theorem~\ref{thm:ranking}]
%Recall that $\bm{E}=\bm{P}-\bm{P}^{\star}$ denotes
%the noise matrix.
Invoke Theorem~\ref{thm:DK_asym} to see that
\begin{align}
	\|\bm{\pi}-\bm{\pi}^{\star}\|_{\bm{\pi}^{\star}}
	& \leq \frac{\big\Vert \bm{\pi}^{\star\top}\bm{E}\big\Vert _{\bm{\pi}^{\star}}}
	{1-\max\left\{ \lambda_{2}(\bm{P}^{\star}),-\lambda_{n}(\bm{P}^{\star})\right\} -\left\Vert \bm{E}\right\Vert _{\bm{\pi}^{\star}}} \nonumber\\
	& \leq 4 \kappa^{2}\big\Vert \bm{\pi}^{\star\top}\bm{E}\big\Vert _{\bm{\pi}^{\star}},
	\label{eq:pi-perturbation-UB1-ranking}
\end{align}
provided that
\begin{align}
	1-\max\big\{ \lambda_{2}(\bm{P}^{\star}),-\lambda_{n}(\bm{P}^{\star})\big\}
	-\left\Vert \bm{E}\right\Vert _{\bm{\pi}^{\star}}\geq1/(4\kappa^{2}) .
	\label{eq:condition-E-eigen-gap-bound-ranking}
\end{align}
From Lemma~\ref{lemma:ranking-gap} and Lemma~\ref{lemma:ranking-noise},
we know that Condition~\eqref{eq:condition-E-eigen-gap-bound-ranking} holds true with probability at least $1-O(n^{-8})$,
with the proviso that $n\geq C\kappa^{5}\log n$ for some sufficiently large constant $C>0$.

Additionally, letting $\pi_{\min}^{\star} \coloneqq \min_i\pi_i^{\star} $ and $\pi_{\max}^{\star} \coloneqq \max_i\pi_i^{\star} $,
we can easily see from the definition of $\|\cdot\|_{\bm{\pi}^{\star}}$ (i.e., $\|\bm{v}\|_{\bm{\pi}^{\star}}=\sqrt{\sum_i \pi_i^{\star}v_i^2 }$ for any vector $\bm{v}$)  that
\[
	\|\bm{v}\|_{\bm{\pi}^{\star}} \overset{\mathrm{(i)}}{\leq} \sqrt{\pi^{\star}_{\max }} \, \|\bm{v}\|_2 , \qquad \text{and} \qquad
	\|\bm{v}\|_2  \overset{\mathrm{(ii)}}{\leq}  \frac{1}{\sqrt{  \pi^{\star}_{\min} } } \, \|\bm{v}\|_{\bm{\pi}^{\star}} ,
\]
which allows us to further obtain
\begin{align*}
\|\bm{\pi}-\bm{\pi}^{\star}\|_{2} & \leq\frac{1}{\sqrt{\pi_{\min}^{\star}}}\|\bm{\pi}-\bm{\pi}^{\star}\|_{\bm{\pi}^{\star}}\leq\frac{4\kappa^{2}}{\sqrt{\pi_{\min}^{\star}}}\|\bm{\pi}^{\star\top}\bm{E}\|_{\bm{\pi}^{\star}}\leq4\kappa^{2.5}\|\bm{\pi}^{\star\top}\bm{E}\|_{2}\\
 & \leq4\kappa^{2.5}\|\bm{E}\|\,\|\bm{\pi}^{\star}\|_{2}.
\end{align*}
Here, the first inequality comes from (ii), the second inequality is a consequence of \eqref{eq:pi-perturbation-UB1-ranking}, whereas the third one results from (i).
The proof is then completed by applying the high-probability bound $\|\bm{E}\|\lesssim\sqrt{(\log n)/n}$ derived in Lemma~\ref{lemma:ranking-noise}.
\end{proof}

\subsection{Proof of auxiliary lemmas \label{subsec:Proof-of-auxilliary-ranking}}

Before delving into the proof, we state a general comparison theorem, which is attributed to \citet{diaconis1993comparison},
that relates the spectral gap of a reversible Markov chain with that
of another (possibly more tractable) reversible chain. We refer the interested reader
to \citet[Lemma 6]{negahban2016rank} for a proof of the following result.
\begin{lemma}
	\label{lemma:gap-comparison}
	Consider two reversible Markov chains over the state space $\{1,2,\cdots,n\}$.
	Let $\bm{P}$ and $\bm{\pi}$ (resp.~$\widetilde{\bm{P}}$ and $\widetilde{\bm{\pi}}$) denote
	the transition matrix and the stationary distribution of the first (resp.~second) chain.
	In addition, set
\[
	\alpha\coloneqq\min_{i,j}
	\frac{\pi_i P_{i,j}}{\widetilde{\pi}_i \widetilde{P}_{i,j}},
	\qquad\text{and}\qquad
	\beta\coloneqq\max_i\frac{\pi_i}{\widetilde{\pi}_i}.
\]
Then one has
\[
	\frac{1-\max\left\{ \lambda_{2}(\bm{P}),-\lambda_{n}(\bm{P})\right\} }{1-\max\big\{ \lambda_{2}(\widetilde{\bm{P}}),-\lambda_{n}(\widetilde{\bm{P}})\big\} }
	\geq\frac{\alpha}{\beta}.
\]
\end{lemma}
Armed with this comparison lemma, we are ready to present the proof of Lemma \ref{lemma:ranking-gap}.

\paragraph{Proof of Lemma \ref{lemma:ranking-gap}.}
In order to control the spectral gap with the aid of Lemma~\ref{lemma:gap-comparison}, we construct an auxiliary transition matrix
\[
\bm{Q}^{\star}=\frac{1}{n}\bm{1}\bm{1}^{\top},
\]
which clearly corresponds to a reversible Markov chain with  stationary distribution $\bm{u}^{\star}=(1/n)\cdot\bm{1}$.  The eigengap of this newly constructed reversible Markov chain is
$$1-\max\left\{ \lambda_{2}(\bm{Q}^{\star}),-\lambda_{n}(\bm{Q}^{\star})\right\} =1,$$
since $\lambda_{2}(\bm{Q}^{\star})=\lambda_{n}(\bm{Q}^{\star})=0$.  Therefore, we only need to bound $\alpha$ and $\beta$.

Recalling the construction of $\bm{P}^{\star}$ in (\ref{defn:ranking-P-star}),
we can straightforwardly check that
\[
	\pi_{i}^{\star}P^{\star}_{i,j} %= \pi_{i}^{\star} \frac{1}{n}\frac{w_{j}^{\star}}{w_{i}^{\star}+w_{j}^{\star}}
	= \frac{1}{n}\frac{\pi_{i}^{\star}\pi_{j}^{\star}}{\pi_{i}^{\star}+\pi_{j}^{\star}} \geq\frac{1}{2n}\min\{\pi_{i}^{\star},\pi_{j}^{\star}\}
\]
for every $i\neq j$,  and in addition,
\[
\pi_{i}^{\star}P_{i,i}^{\star}=\pi_{i}^{\star}\Bigg[1-\sum_{j:j\neq i}\frac{1}{n}\frac{w_{j}^{\star}}{w_{i}^{\star}+w_{j}^{\star}}\Bigg]\geq\pi_{i}^{\star}\Bigg[1-\sum_{j:j\neq i}\frac{1}{n}\Bigg]=\frac{1}{n}\pi_{i}^{\star}.
\]
Combining the previous two inequalities, we obtain
\[
	\min_{i,j} \big( \pi^{\star}_i P^{\star}_{i,j} \big)
	\geq\frac{1}{2n}\min_{1\leq i\leq n}\pi_{i}^{\star}
	= \frac{1}{2n\kappa}\max_{1\leq i\leq n}\pi_{i}^{\star}
	\geq\frac{1}{2n^{2}\kappa},
\]
where the last relation holds since $\max_{ i}\pi_{i}^{\star}\geq \frac{1}{n}\sum_i\pi_i^{\star} = \frac{1}{n}$.
This together with the construction of $\bm{Q}^{\star}$ further leads to
\[
	\alpha \coloneqq \min_{i,j}\frac{\pi_{i}^{\star}P^{\star}_{i,j}}{u_{i}^{\star}Q^{\star}_{i,j}}
	=n^{2}\min_{i,j} \big( \pi^{\star}_i P^{\star}_{i,j} \big)
	\geq\frac{1}{2\kappa}.
\]
In regard to $\beta$, it is seen that
\[
	\beta\coloneqq \max_i \frac{\pi^{\star}_i}{u^{\star}_i} = n\max_{ i}\pi_{i}^{\star}
	= n\kappa \min_{i}\pi_{i}^{\star}
	\leq \kappa,
\]
where the final inequality follows since $\min_i\pi_i\leq \frac{1}{n}\sum_i\pi_i = \frac{1}{n}$. With the preceding bounds on $\alpha$ and $\beta$ in place,
Lemma~\ref{lemma:gap-comparison} informs us that
\[
\frac{1-\max\left\{ \lambda_{2}(\bm{P}^{\star}),-\lambda_{n}(\bm{P}^{\star})\right\} }{1-\max\left\{ \lambda_{2}(\bm{Q}^{\star}),-\lambda_{n}(\bm{Q}^{\star})\right\} }\geq\frac{\alpha}{\beta}\geq\frac{1}{2\kappa^{2}}.
\]
This together with the aforementioned  eigengap for $\bm{Q}^{\star}$
establishes the advertised result.

\paragraph{Proof of Lemma \ref{lemma:ranking-noise}.}

Let $\bm{D} \coloneqq \mathsf{diag}(\sqrt{\pi_1^{\star}},\cdots,\sqrt{\pi_n^{\star}})$. We have seen from the proof in Section~\ref{subsec:proof-DK-asym} (cf.~\eqref{eq:pi-norm-matrix-equation}) that
\[
	\|\bm{E}\|_{\bm{\pi}^{\star}} =\|\bm{D}\bm{E}\bm{D}^{-1}\| \leq  \|\bm{D}\|\, \|\bm{E}\| \, \|\bm{D}^{-1}\| = \frac{  \sqrt{\max_i \pi_i^{\star}} } {  \sqrt{\min_i \pi_i^{\star}} }  \|\bm{E}\|
	=  \sqrt{\kappa} \, \|\bm{E}\|,
\]
where $\kappa$ is defined in (\ref{eq:ranking-well-condition}).
Therefore, it suffices to bound $\|\bm{E}\|$.

% We first make the observation that
% %
% \begin{align*}
% \|\bm{E}\|_{\bm{\pi}^{\star}} & =\sup_{\bm{x}\neq\bm{0}}\frac{\|\bm{E}\bm{x}\|_{\bm{\pi}^{\star}}}{\|\bm{x}\|_{\bm{\pi}^{\star}}}=\sup_{\bm{x}\neq\bm{0}}\frac{\|\bm{D}\bm{E}\bm{D}^{-1}\bm{D}\bm{x}\|_{2}}{\|\bm{D}\bm{x}\|_{2}}=\sup_{\bm{v}\neq\bm{0}}\frac{\|\bm{D}\bm{E}\bm{D}^{-1}\bm{v}\|_{2}}{\|\bm{v}\|_{2}}\\
%  & =\|\bm{D}\bm{E}\bm{D}^{-1}\|,
% \end{align*}
%
% where $\bm{D} \coloneqq \mathsf{diag}(\sqrt{\pi_1^{\star}},\cdots,\sqrt{\pi_n^{\star}})$. Here, the second relation makes use of the identity $\|\bm{x}\|_{\bm{\pi}^{\star}} = \|\bm{D}\bm{x}\|_2$ (see the definition of the vector norm $\|\cdot\|_{\bm{\pi}^{\star}}$), the third relation replaces $\bm{D}\bm{x}$ with $\bm{v}$,  while the last one holds true due to the definition of the spectral norm.  As a result,
% %

By construction of $\bm{P}$ and $\bm{P}^{\star}$ (see \eqref{defn:ranking-P-empirical} and \eqref{defn:ranking-P-star}), we see that
\begin{equation}
	E_{i,j}=P_{i,j}-P_{i,j}^{\star}=\frac{1}{n} \big( y_{i,j}- \mathbb{E} [ y_{i,j}] \big) \label{eq:ranking-noise-off-diag}
\end{equation}
for any $i\neq j$.
In addition, for all $1\leq i\leq n$, it follows that
\begin{align}
E_{i,i} & =P_{i,i}-P_{i,i}^{\star}
%=\Big(1-\sum_{j:j\neq i}P_{i,j}\Big)-\Big(1-\sum_{j:j\neq i}P_{i,j}^{\star}\Big)
= -\sum_{j:j\neq i}E_{i,j}
  =-\frac{1}{n}\sum_{j:j\neq i}\big( y_{i,j}- \mathbb{E} [ y_{i,j}] \big) .\label{eq:ranking-noise-diag}
\end{align}
In view of these identities, we shall decompose the matrix $\bm{E}$
into three parts: the upper triangular part (denoted by $\bm{E}_{\mathsf{upper}}$),
the diagonal part (denoted by $\bm{E}_{\mathsf{diag}}$), and the lower triangular
part (denoted by $\bm{E}_{\mathsf{lower}}$). Clearly,  the triangle inequality gives
\begin{equation}
\|\bm{E}\|\leq\|\bm{E}_{\mathsf{upper}}\|+\|\bm{E}_{\mathsf{diag}}\|+\|\bm{E}_{\mathsf{lower}}\|.\label{eq:ranking-triangle}
\end{equation}
In the sequel, we deal with these three terms separately.

Let us start with the diagonal part $\bm{E}_{\mathsf{diag}}$.
In view of the definition of the spectral norm, we know that
\[
\|\bm{E}_{\mathsf{diag}}\|=\max_{1\leq i\leq n}|E_{i,i}|=\max_{1\leq i\leq n} \Big| \sum_{j:j\neq i}E_{i,j} \Big|,
\]
where the last relation arises from (\ref{eq:ranking-noise-diag}). Fix any
$i$, then it is easily seen that $\sum_{j:j\neq i}E_{i,j}$ is a sum of
independent zero-mean random variables $\{E_{i,j}\}$, which can be controlled via the Bernstein inequality.
Specifically, observe that
\[
\max_{j:j\neq i}|E_{i,j}|=\max_{j:j\neq i} \frac{1}{n} \big| y_{i,j}-\mathbb{E} [ y_{i,j}]   \big|
\leq\frac{1}{n} \eqqcolon B_1
\]
and,  in addition,
\[
	v_1\coloneqq\sum_{j:j\neq i}\mathbb{E}[E_{i,j}^{2}]=\frac{1}{n^{2}}\sum_{j:j\neq i}\mathsf{Var}\left(y_{i,j}\right)\leq\frac{1}{n},
\]
where the last inequality follows since the variance of a Bernoulli random variable is no larger than $1$.
Apply the Bernstein inequality (cf.~Corollary~\ref{thm:matrix-Bernstein-friendly}) and the union
bound over $1\leq i\leq n$ to demonstrate that
\begin{align*}
	\|\bm{E}_{\mathsf{diag}}\| & =\max_{1\leq i\leq n} \Big|\sum_{j:j\neq i}E_{i,j} \Big|
	\lesssim \sqrt{v_1\log n} + B_1 \log n   \\
	& \lesssim  \sqrt{\frac{\log n}{n}}+\frac{\log n}{n}\asymp\sqrt{\frac{\log n}{n}}
\end{align*}
holds with probability at least $1-O(n^{-8})$.
% This finishes the upper bound on the diagonal part.

We now move on to the upper triangular part $\bm{E}_{\mathsf{upper}}$,
whose entries $\{E_{i,j}\}_{i<j}$ are independent. Invoking Theorem~\ref{thm:tighter-spectral-normal-ramon} (see the remark about asymmetric version right after Theorem~\ref{thm:tighter-spectral-normal-ramon}) with the bounds on $B_1$ and $v_1$ established above,
we arrive at
\[
	\|\bm{E}_{\mathsf{upper}}\|\lesssim \sqrt{v_1 } + B_1 \log n \lesssim \sqrt{\frac{1}{n}} + \frac{\log n}{n} \asymp  \sqrt{\frac{1}{n}}
\]
with probability at least $1-O(n^{-8})$. Similar arguments lead to
the same upper bound on $\|\bm{E}_{\mathsf{lower}}\|$, which
we omit  for brevity.

Substituting the upper bounds on $\|\bm{E}_{\mathsf{diag}}\|$, $\|\bm{E}_{\mathsf{upper}}\|$ and $\|\bm{E}_{\mathsf{lower}}\|$ into (\ref{eq:ranking-triangle}), we immediately establish
the desired bound.

%% file: chapters/phase_retrieval.tex
\section{Phase retrieval and solving quadratic systems of equations}
\label{sec:phase_retrieval}

Phase retrieval is a fundamental problem arising in numerous imaging applications such as X-ray crystallography,  diffraction imaging, and so on \citep{fienup1982phase,shechtman2015phase,candes2012phaselift,candes2015phase,jaganathan2015phase}.
In physics, phase retrieval is concerned with estimating a specimen by observing the intensities (or squared modulus) of the diffracted
waves scattered by the object without knowing their phases. The advent of this problem is attributed to the physical limitation that the optical sensors are unable to record the phases of the diffracted waves. Put another way, in phase retrieval, we only have access to  measurements that are quadratic functions of the object of interest, and aim at estimating the unknown object up to global phase.  This gives rise to the problem of solving quadratic systems of equations, to be formulated below.

\begin{figure}
\begin{center}
\includegraphics[width=0.8\textwidth]{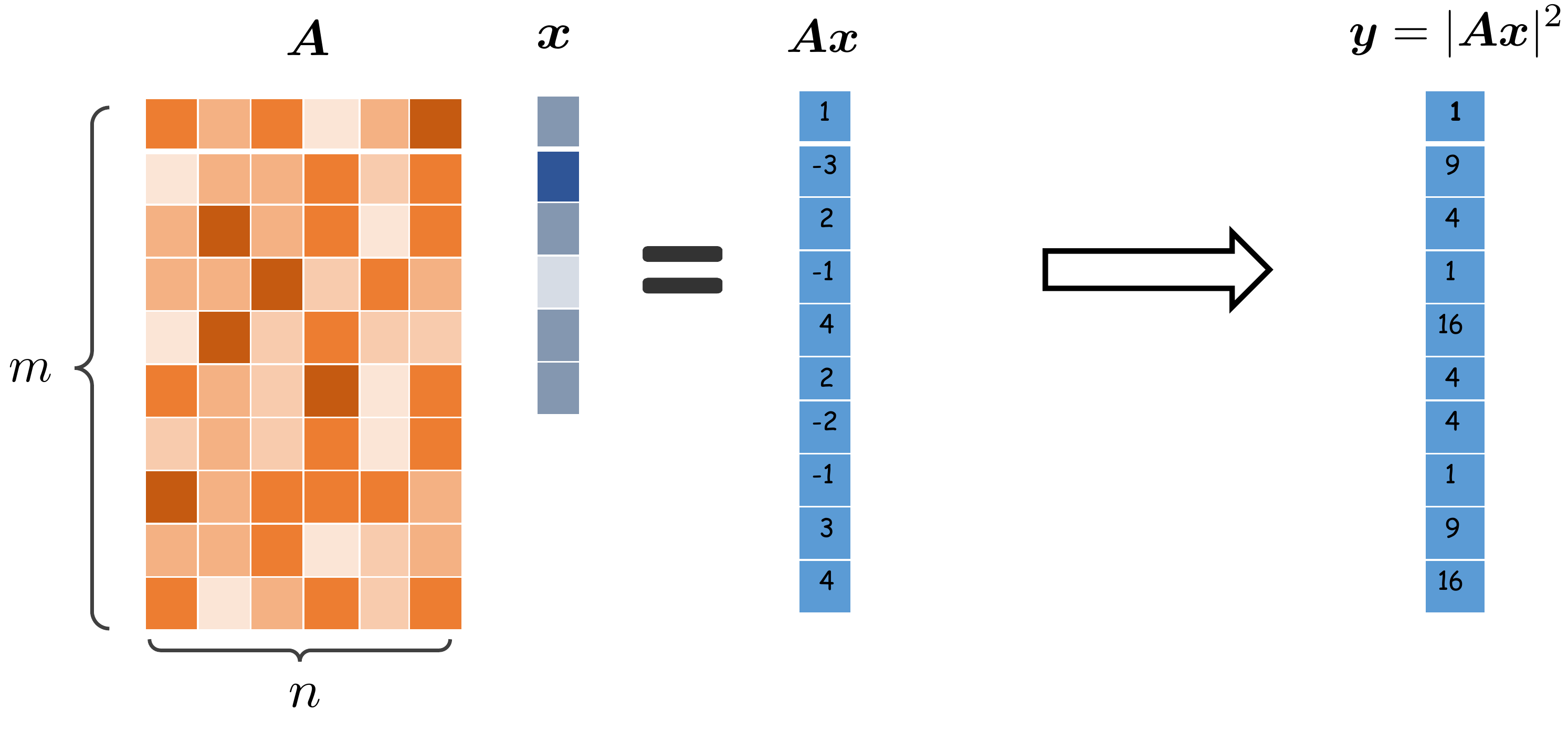}
\end{center}
	\caption{Illustration of phase retrieval and solving quadratic systems of equations, where only the intensities of linear measurements are collected. Here, $\bm{A}=[\bm{a}_1,\cdots,\bm{a}_m]^{\top}$, and $|\bm{z}|^2\coloneqq [|z_1|^2, \cdots, |z_m|^2]^{\top}$ for any vector $\bm{z}=[z_i]_{1\leq i\leq m}$. } \label{fig:phase_retrieval}
\end{figure}

\subsection{Problem formulation and assumptions}
\label{sec:formulation-PR}

Suppose that we are interested in reconstructing an unknown signal
$\bm{x}^{\star}\in\mathbb{R}^{n}$, but only have access to a collection of $m$
quadratic measurements on the linear combinations of its entries as follows:
\begin{equation}
	y_{i}=(\bm{a}_{i}^{\top}\bm{x}^{\star})^{2},\qquad1\leq i\leq m,
	\label{eq:PR-samples}
\end{equation}
where $\bm{a}_{i}=[a_{i,1},\cdots,a_{i,n}]^{\top} \in\mathbb{R}^{n}$ is the design vector known {\em a priori}. 
See Figure~\ref{fig:phase_retrieval} for an illustration of this measurement model.
The question is: when can we hope to reconstruct $\bm{x}^{\star}$, in an accurate and efficient fashion, on the basis of these nonlinear equations?

%
% In addition to physical science, solving quadratic systems of equations
% is also closely related to mixed linear regression \citep{chen2014convex}
% and learning neural nets with quadratic activations \citep{soltanolkotabi2017theoretical,li2017algorithmic}.
%

As is well known, solving quadratic systems of equations is, in general, NP hard.\footnote{See the reduction to the NP-hard stone problem in~\citet{chen2015solving}.}
Additional assumptions are therefore needed to enable tractable recovery.
Here, we adopt a Gaussian design model commonly studied in the literature.
\begin{assumption}
\label{assump:pr-gaussian}
	The design vectors $\{\bm{a}_i\}_{1\leq i\leq m}$ are independently generated  obeying $\bm{a}_{i}\overset{\text{i.i.d.}}{\sim}\mathcal{N}(\bm{0},\bm{I}_{n})$.
\end{assumption}

\subsection{Algorithm}\label{sec:pr-alg}

The Gaussian design model (cf.~Assumption~\ref{assump:pr-gaussian}) allows meaningful estimation of the unknown object~$\bm{x}^{\star}$ via the (by now) familiar spectral method.
Let us start by arranging the data into the following matrix
\begin{equation}
\bm{M}\coloneqq\frac{1}{m}\sum_{i=1}^{m}y_{i}\bm{a}_{i}\bm{a}_{i}^{\top}=\frac{1}{m}\sum_{i=1}^{m}(\bm{a}_{i}^{\top}\bm{x}^{\star})^{2}\bm{a}_{i}\bm{a}_{i}^{\top},\label{eq:D_PR_org}
\end{equation}
which can be viewed as a weighted sample covariance matrix of the design
vectors $\{\bm{a}_{i}\}$. The spectral method then estimates $\bm{x}^{\star}$ by
\begin{align}
	\bm{x}=\sqrt{\frac{\lambda_{1}}{3}}\,\bm{u}_{1},\label{eq:init-PR}
\end{align}
where $\bm{u}_{1}$ (resp.~$\lambda_{1}=\lambda_{1}(\bm{M})$) indicates
the leading eigenvector (resp.~eigenvalue) of the matrix
$\bm{M}$. This simple approach has been suggested for phase retrieval since the work of~\citet{netrapalli2015phase}.

To explain the rationale of this approach, it is instrumental
to look at the mean of $\bm{M}$ under Assumption~\ref{assump:pr-gaussian}.
Specifically, simple calculation (which we include at the end of this subsection) gives
\begin{equation}
	\bm{M}^{\star}\coloneqq \mathbb{E}[\bm{M}] = \mathbb{E}\big[(\bm{a}_{i}^{\top}\bm{x}^{\star})^{2}\bm{a}_{i}\bm{a}_{i}^{\top}\big]=2\bm{x}^{\star}\bm{x}^{\star\top}+\|\bm{x}^{\star}\|_{2}^{2}\,\bm{I}_{n}.
	\label{eq:mtx_LLN}
\end{equation}
It is self-evident that (a) the leading eigenvector of $\bm{M}^{\star}$ is precisely given by $\pm \bm{x}^{\star} / \|\bm{x}^{\star} \|_2$, and (b) the leading eigenvalue of $\bm{M}^{\star}$ is given by $3\|\bm{x}^{\star}\|_{2}^{2}$ by \eqref{eq:mtx_LLN}.
From now on, we shall set
\begin{align}
	\bm{u}_{1}^{\star} \coloneqq \bm{x}^{\star} / \|\bm{x}^{\star} \|_2,
	\qquad \text{and} \qquad
	\lambda_1^{\star} \coloneqq 3\|\bm{x}^{\star}\|_{2}^{2},
	\label{defn:u1-lambda1-star-PR}
\end{align}
which implies $\bm{x}^{\star} = \sqrt{\lambda_1^\star/3} \,\bm{u}_{1}^{\star}$ and hence explains the estimator constructed in \eqref{eq:init-PR}.
The above argument further hints that: the spectral estimate $\bm{x}$ converges to the ground truth $\pm \bm{x}^{\star}$ in the large-sample limit with $m\rightarrow \infty$ (so that $\bm{M}\rightarrow \mathbb{E}[\bm{M}] =\bm{M}^{\star}$). The question, however, boils down to  where this algorithm stands in the more realistic finite-sample scenario.
\begin{remark}
The expression \eqref{defn:u1-lambda1-star-PR} indicates that $\bm{x}^{\star} =  \|\bm{x}^{\star} \|_2\,  \bu_1^\star$.
From the law of large numbers, one expects
$$
	\frac{1}{m} \sum_{i=1}^m y_i  ~\to~   \mathbb{E}[ y_i] = \mathbb {E} \big[ (\ba^{\top} \bx^\star )^2 \big] =   \|  \bx^\star\|^2_2,
$$
with probability approaching one.
Thus, an alternative estimator is
\begin{equation}
	\label{eq3.43}
	\widehat{\bx} =  \Big( \frac{1}{m} \sum\nolimits_{i=1}^m y_i  \Big)^{1/2} \bu_1.
\end{equation}
%
%which is precisely the version analyzed in \citet{candes2015phase}.
\end{remark}

\paragraph{Derivation of \eqref{eq:mtx_LLN}.}

The $(i,j)$-th entry of  $\mathbb{E}[\bm{M}]$ is given by
$$
\mathbb{E}[M_{j,k}]=\mathbb{E}
\big[\big(\big(\bm{a}_{i}^{\top}\bm{x}^{\star}\big)^{2}\bm{x}^{\star}\bm{x}^{\star\top}\big)_{j,k} \big]
=\mathbb{E}\big[(a_{i,1}x_{i,1}^{\star}+\cdots+a_{i,n}x_{i,n}^{\star})^{2}a_{i,j}a_{i,k}\big] .
$$
Expanding  terms and using the moments of Gaussian variables yield
\begin{align*}
\mathbb{E}[M_{j,k}] & =\mathbb{E}\big[2a_{i,j}^{2}a_{i,k}^{2}x_{i,j}^{\star}x_{i,k}^{\star}\big]=2x_{i,j}^{\star}x_{i,k}^{\star}\qquad\text{if }j\neq k;\\
\mathbb{E}[M_{j,j}] & =\mathbb{E}\big[a_{i,j}^{4}\big(x_{i,j}^{\star}\big)^{2}\big]+\sum_{l:\,l\neq j}\mathbb{E}\big[a_{i,l}^{2}a_{i,j}^{2}\big(x_{i,l}^{\star}\big)^{2}\big]=3\big(x_{i,j}^{\star}\big)^{2}+\sum_{l:\,l\neq j}\big(x_{i,l}^{\star}\big)^{2}\\
 & = 2\big(x_{i,j}^{\star}\big)^{2} + \|\bm{x}^{\star}\|_{2}^{2}.
\end{align*}
Putting these together leads to the expression \eqref{eq:mtx_LLN}.

%
%after expanding the quadratic term is $2 \mathbb{E} a_i^2 a_j^2 x_i^\star x_j^\star = 2 x_i^\star x_j^\star$ for $i\not = j$ and $\mathbb{E} a_i^4 (x_i^\star)^2 = 3 (x_i^\star)^2$ for $i = j$.  Putting this in the matrix form, we get

%The derivation of 	\eqref{eq:mtx_LLN}  reveals more generally that
%$$ \mathbb{E}[\bm{M}] =2\bm{x}^{\star}\bm{x}^{\star\top}+(\mathbb{E} a_1^4 - 2)\|\bm{x}^{\star}\|_{2}^{2}\,\bm{I}_{n}
%$$
%when $\{a_i\}_{i=1}^n$ are independent with
%$$
%\E a_i = 0, \quad \E a_i^2 = 0, \quad \E a_i^3 = 0, \quad \mbox{and} \quad  \E a_i^4 = E a_1^4
%$$
%so that $\bm{x}^{\star} = \pm  \|\bm{x}^{\star} \|_2  \bu_1^\star$ continue to hold.

\subsection{Performance guarantees}

Developing theoretical support for the aforementioned spectral method hinges upon characterizing the proximity of $\lambda_{1}$ and $\lambda_{1}^{\star}$
and that of $\bm{u}_{1}$ and $\bm{u}_{1}^{\star}$, both of which rely largely on bounding $\bm{M}-\bm{M}^{\star}$.
In what follows, we start by controlling  $\|\bm{M}-\bm{M}^{\star}\|$, with the proof  postponed to Section~\ref{sec:proof-auxiliary-lemmas-PR-L2}.
\begin{lemma}
\label{lemma:pr-noise-bound}
Consider the settings in Section~\ref{sec:formulation-PR}.
There exist  sufficiently large constants $c,C>0$ such that if $m\geq Cn\log^{3}m$, then with probability at least $1-O(m^{-10})$ one has
\begin{equation}
	\|\bm{M}-\bm{M}^{\star}\|\leq c\sqrt{\frac{n\log^{3}m}{m}}\|\bm{x}^{\star}\|_{2}^{2} \leq\frac{1}{10}\|\bm{x}^{\star}\|_{2}^{2}.
	\label{eq:pr-noise-bound}
\end{equation}
\end{lemma}
\begin{remark}
	The sample size requirement can be further relaxed to $m\geq Cn\log n$ with a more careful treatment \citep{candes2015phase,ma2017implicit}.
	For the sake of conciseness, however, we do not strive to shave the log factors here.
\end{remark}

With the above bound in mind, we are ready to characterize the statistical accuracy of the spectral method for phase retrieval.

\begin{theorem}
\label{thm:PR-L2-performance}
Suppose the assumptions of Lemma~\ref{lemma:pr-noise-bound} hold, then with probability at least $1-O(m^{-10})$, the following holds
\[
	\min \{ \| \bm{x} - \bm{x}^{\star}\|_2, \| \bm{x} + \bm{x}^{\star}\|_2 \} \leq3c\sqrt{\frac{n\log^3 m}{m}}\|\bm{x}^{\star}\|_{2} .
	%\leq \frac{3}{10}\|\bm{x}^{\star}\|_{2}
\]
%Here $c>0$ is the same universal constant as in Lemma~\ref{lemma:pr-noise-bound}.
\end{theorem}
%\begin{remark} The log factors in Theorem~\ref{thm:PR-L2-performance} can be eliminated by a truncated spectral method designed to handle the heavy-tailed nature of quadratic Gaussian measurements; see Section~\ref{sec:improved-spectral-truncated}.
%\end{remark}

As can be seen from Theorem~\ref{thm:PR-L2-performance}, when the number $m$ of measurements obeys $m\gg n\log^{3}m$, the relative accuracy of the spectral estimates (i.e., $\min \{ \| \bm{x} - \bm{x}^{\star}\|_2, \| \bm{x} + \bm{x}^{\star}\|_2 \}  / \|\bm{x}^{\star}\|_{2}$) becomes considerably smaller than $1$, thus indicating consistent estimation. This should be contrasted with the minimax lower bounds derived in the literature \citep{cai2015rop,eldar2014phase}, which assert that no estimator can achieve a vanishingly small relative estimation error if $m$ is orderwise smaller than $n$. All this corroborates the power of spectral methods for solving the phase retrieval problem.

\paragraph{Proof of Theorem~\ref{thm:PR-L2-performance}.}

Lemma~\ref{lemma:pr-noise-bound} and Weyl's inequality (see Lemma~\ref{lemma:weyl}) yield
\begin{align}
	 |\lambda_{1}-\lambda_{1}^{\star}|\leq\|\bm{M}-\bm{M}^{\star}\|\leq c\sqrt{\frac{n\log^3 m}{m}}\|\bm{x}^{\star}\|_{2}^{2}
	& \leq\|\bm{x}^{\star}\|_{2}^{2},\label{eq:pr-eig-pert}
\end{align}
As a result,  by using $\lambda_{1}^\star = 3\|\bm{x}^{\star}\|_{2}^{2}$, we have
\begin{equation}
	\lambda_{1} \geq \lambda_{1}^\star - \|\bm{x}^{\star}\|_{2}^{2} = 3\|\bm{x}^{\star}\|_{2}^{2}-\|\bm{x}^{\star}\|_{2}^{2} = 2\|\bm{x}^{\star}\|_{2}^{2}.
	\label{eq:pr-eig-lower-bound}
\end{equation}
In addition, note that $\lambda_{1}^{\star}=\lambda_{1}(\bm{M}^{\star})=3\|\bm{x}^{\star}\|_{2}^{2}$
and $\lambda_{i}(\bm{M}^{\star})=\|\bm{x}^{\star}\|_{2}^{2}$ for all $i\geq 2$. The
bound (\ref{eq:pr-noise-bound}) on $\bm{M}-\bm{M}^{\star}$ indicates that
\[
\|\bm{M}-\bm{M}^{\star}\|\leq(1-1/\sqrt{2})\big[ \lambda_{1}(\bm{M}^{\star})-\lambda_{2}(\bm{M}^{\star}) \big],
\]
which allows one to invoke the Davis-Kahan $\sin\bm{\Theta}$ theorem (cf.~Corollary~\ref{cor:davis-kahan-conclusion-corollary}) to obtain
\begin{align}
	\mathsf{dist}(\bm{u}_{1},\bm{u}_{1}^{\star})\leq\frac{2\|\bm{M}-\bm{M}^{\star}\|}{ \lambda_{1}(\bm{M}^{\star})-\lambda_{2}(\bm{M}^{\star}) }\leq2c\sqrt{\frac{n\log^3 m}{m}}.
	\label{eq:dist-u1-u1star-PR}
\end{align}
Without loss of generality, we shall assume $\|\bm{u}_{1}-\bm{u}_{1}^{\star}\|_{2}=\mathsf{dist}(\bm{u}_{1},\bm{u}_{1}^{\star})$ in the sequel.

Now we are ready to control our target quantity $\mathsf{dist}(\bm{x},\bm{x}^{\star})$.
In view of the definition (\ref{eq:init-PR}) of $\bm{x}$, one has
\begin{align}
\|\bm{x}-\bm{x}^{\star}\|_{2}
 & = \Big\| \sqrt{\lambda_{1}/3}\,\bm{u}_{1}-\|\bm{x}^{\star}\|_{2}\bm{u}_{1}^{\star} \Big\|_{2}\nonumber \\
 & \leq \Big\| \Big(\sqrt{\lambda_{1}/3}-\|\bm{x}^{\star}\|_{2} \Big)\,\bm{u}_{1} \Big\|_{2}+\|\bm{x}^{\star}\|_{2}\|\bm{u}_{1}-\bm{u}_{1}^{\star}\|_{2}\nonumber \\
 & \leq \Big|\sqrt{\lambda_{1}/3}-\|\bm{x}^{\star}\|_{2} \Big| + 2c \sqrt{\frac{n\log^3 m}{m}}\|\bm{x}^{\star}\|_{2}.\label{eq:pr-first-step}
\end{align}
Here, the second line applies the triangle inequality, and the last line
arises from the facts $\|\bm{u}_{1}\|_{2}=1$ and 
\eqref{eq:dist-u1-u1star-PR}. %$\|\bm{u}_{1}-\bm{u}_{1}^{\star}\|_{2}=\mathsf{dist}(\bm{u}_{1},\bm{u}_{1}^{\star})\leq2c\sqrt{(n\log^3 m)/m}$.
It then boils down to controlling $|\sqrt{\lambda_{1}/3}-\|\bm{x}^{\star}\|_{2}|$,
for which (\ref{eq:pr-eig-pert}) and (\ref{eq:pr-eig-lower-bound})
prove useful. A little algebra reveals that
\begin{align}
	\Big|\sqrt{\lambda_{1}/3}\,-\|\bm{x}^{\star}\|_{2} \Big|
	%& =\frac{1}{\sqrt{3}}\left|\sqrt{\lambda_{1}}-\sqrt{3}\|\bm{x}^{\star}\|_{2}\right|
	=\frac{1}{\sqrt{3}}\frac{\big|\lambda_{1}-3\|\bm{x}^{\star}\|_{2}^{2} \big|}{\sqrt{\lambda_{1}}+\sqrt{3}\|\bm{x}^{\star}\|_{2}}
	\leq c\sqrt{\frac{n\log^3 m}{m}}\|\bm{x}^{\star}\|_{2},
	\label{eq:pr-second-step}
\end{align}
where the last relation relies on the bounds (\ref{eq:pr-eig-pert})
and (\ref{eq:pr-eig-lower-bound}).

Taking collectively (\ref{eq:pr-first-step}) and (\ref{eq:pr-second-step})
concludes the proof.

\subsection{Extensions}
\label{sec:improved-spectral-truncated}

The spectral algorithm described in Section~\ref{sec:pr-alg}, while enjoying appealing statistical guarantees, is improvable in multiple aspects.
In this subsection, we briefly discuss two central issues:  sample efficiency and robustness against outliers.

\subsubsection{Improving sample efficiency}

Thus far, the spectral algorithm we have discussed requires the sample size to exceed $m \gtrsim n  \log^3 m$.
While this can be improved to $m \gtrsim n \log n$ via tighter analysis \citep{candes2015phase,ma2017implicit},
it remains suboptimal due to the presence of the log factor.
What happens in the sample-starved regime where the sample size $m$ is on the same order as the number $n$ of unknowns?
Is it possible to achieve the information-theoretic sampling limit for this problem?
As it turns out, in order to attain the desired statistical accuracy in the sample-starved regime,
we have to modify the standard recipe by applying appropriate preprocessing steps before forming the data matrix $\bm{M}$, as we shall explain momentarily.

%By the triangle inequality, it suffice to give a lower bound for  the spectrum of the matrix $\bm{M}$ previously constructed in \eqref{eq:D_PR_org}.

\paragraph{Why is the algorithm in Section~\ref{sec:pr-alg} suboptimal?}

Before introducing the improved spectral algorithm, we take a closer look at the lower bound of the approximation error $\|\bm{M} - \bm{M}^\star\|$ for the sample-starved regime.  
Clearly,
\begin{equation*}
	\norm{\bm{M}} \ge \frac{\va_j^\top \bm{M} \va_j}{\| \va_j \|_2^2}
	=  \frac{1}{m} \sum_{i=1}^m y_i \frac{(\va_i^\top \va_j)^2}{\| \va_j \|_2^2} \geq \frac{1}{m}  y_j  \| \va_j \|_2^2
\end{equation*}
holds for any $1\leq j\leq m$. Taking $j = i^\ast  \coloneqq \arg \max_i  y_i$, we obtain
\begin{equation}
	\label{eq:bound_id2}
	\norm{\bm{M}} \ge \frac{(\max_i y_i) \, \| \va_{i^\ast} \|_2^2}{m}.
\end{equation}
Under the i.i.d.~Gaussian design, $\set{y_i / \| \bm{x}^{\star} \|_2^2}_{1\leq i\leq m}$ forms a collection of i.i.d.~$\chi^2$ random variables with 1 degree of freedom. Classical  Gaussian concentration results \citep{Ferguson:1996,vershynin2016high} tell us that
\begin{align*}
	\max_{1\le i\le m}\,y_{i}  =(2+o(1))\|\bm{x}^{\star}\|_{2}^{2}\log m,\quad
	\|\bm{a}_{i}\|_{2}^{2}  = (1+o(1)) n, ~~ 1 \leq i\leq m
\end{align*}
with probability approaching one as $n$ grows, as long as $m=\mathrm{poly}(n)$.
Substitution into \eqref{eq:bound_id2} implies that
\begin{equation*}
	%\label{eq:PR_sc}
	\norm{\bm{M}} \ge \big(2+o(1)\big)  \frac{ n \log m}{  m} \|\bm{x}^{\star}\|_2^2  \gg \|\bm{x}^{\star}\|_2^2
\end{equation*}
once $m\ll n \log m$, which combined with \eqref{eq:mtx_LLN} further yields
\[
	\norm{\bm{M} - \bm{M}^{\star}} \geq \norm{\bm{M}} - \norm{\bm{M}^{\star}}
	= \norm{\bm{M}} - 3\|\bm{x}^{\star}\|_2^2 \gg \norm{\bm{M}^{\star}}.
\]
In other words, the deviation between $\bm{M}$ and $\bm{M}^{\star}$ is not as well-controlled as desired in the regime with $m\ll n \log m$, and hence classical matrix perturbation theory (e.g., the Davis-Kahan theorem) does not support the use of the spectral algorithm based on $\bm{M}$ in this case.

\paragraph{Spectral methods with data preprocessing.}

The above diagnosis  suggests a natural remedy:
since the culprit lies in the  large influence $\max_i y_i$  has brought to bear on the leading eigenvector,
it is advisable to downweight the effect of any excessively large $y_i$.
This is precisely the key idea behind the {\em truncated spectral method} proposed by \citet{chen2015solving}---as well as other variations proposed thereafter---that provably improves the sample efficiency of spectral methods.

More specifically, instead of using the matrix $\bm{M}$ constructed in \eqref{eq:D_PR_org},
we resort to a properly preprocessed data matrix
\begin{equation}
	\label{eq:D_PR_T}
	\bm{M}_{\mathcal{T}} \coloneqq \frac{1}{m} \sum_{i=1}^m \mathcal{T}(y_i) \,\va_i \va_i^\top
\end{equation}
with $\mathcal{T}$ some preprocessing function, and produce, by \eqref{eq3.43}, an estimate%\footnote{Here, $\frac{1}{m} \sum_{i=1}^m y_i $ serves as an  estimate of $\|\bm{x}^{\star}\|_2^2$ since $\frac{1}{m} \sum_{i=1}^m y_i  \to   E y_i = \mathbb {E} (\ba^T \bx^\star )^2 =   \|  \bx^\star\|^2$.}
\begin{equation}
	\label{eq:estimate-preprocess-PR}
	\bm{x}_{\mathcal{T}} = \Big( \frac{1}{m} \sum\nolimits_{i=1}^m y_i \Big)^{1/2} \, \bm{u}_{1,\mathcal{T}},
\end{equation}
where  $\bm{u}_{1,\mathcal{T}}$  denotes  the leading eigenvector of $\bm{M}_{\mathcal{T}}$.
A few representative examples of $\mathcal{T}$ are in order.
%, where $c_{\tau}>0$ denotes some sufficiently large constant:
%
\begin{itemize}
	\item {\em mean-based truncation} \citep{chen2015solving}:
	\begin{equation}
		\label{eq:trimming}
		\mathcal{T}(y) \coloneqq y \mathbbm{1} \big\{ y \leq \alpha_1 \overline{y} \big\}, \quad
		\overline{y} \coloneqq \frac{1}{m} \sum\nolimits_{i=1}^m y_i,
	\end{equation}
	where $\alpha_1>0$ is some sufficiently large constant;

	\item {\em median-based truncation} \citep{zhang2016provable,zhang2018median}:
	\begin{equation}
		\label{eq:median-truncation}
		\mathcal{T}(y) \coloneqq y \mathbbm{1}\big\{ y \leq \alpha_2 y_{\mathsf{med}} \big\}, \quad
		y_{\mathsf{med}}  \coloneqq \mathsf{median}\big(\, \{ y_i \} \,\big),
	\end{equation}
	where $\alpha_2>0$ is some sufficiently large constant; 	
	
	\item {\em orthogonality promotion} \citep{wang2017solving,duchi2019solving}
	\begin{equation}
		\label{eq:TAF}
		\mathcal{T}(y) \coloneqq  \mathbbm{1}\big\{ y \geq y_{(\alpha_3 m)} \big\},
	\end{equation}
	where $y_{(1)}\geq y_{(2)}\geq \cdots \geq y_{(m)}$ denote the order statistics of $\{y_i\}$,
	and $0<\alpha_3<1$ is some properly chosen constant;

	\item {\em optimal preprocessing} \citep{mondelli2017fundamental,luo2019optimal}:
	\begin{equation}
		\label{eq:optimal-truncation}
		\mathcal{T}(y) \coloneqq \frac{y\,/\,\overline{y} -1}{ y\,/\,\overline{y} +\sqrt{2m/n}-1},\quad
		\overline{y} \coloneqq \frac{1}{m} \sum\nolimits_{i=1}^m y_i.
	\end{equation}
		
\end{itemize}
\begin{remark}
	To be more precise, $\mathcal{T}(y)$ depends not only on $y$ but also some statistics about $\{y_i\}$ (e.g., empirical mean).
	Here, we suppress the dependency on such additional statistics mainly to simplify notation.
\end{remark}
In words, the first two choices discard any measurement $y_i$ that is too large (compared to the order of either the empirical mean or empirical median),
the third one selects a subset of measurements that are most aligned with the unknown signal and scales their contributions to a measurement-invariant level,
while the last one effectively behaves as a shrinkage operator once $y_i$ rises above the empirical mean.
The following theorem---which was first established in \citet{chen2015solving} for the version \eqref{eq:trimming} and subsequently extended to other alternatives \citep{zhang2016provable,wang2017solving,lu2020phase,mondelli2017fundamental,luo2019optimal}---confirms the effectiveness and importance of proper preprocessing  in enabling order-optimal sample complexity.
The interested reader is referred to these papers for the proofs.
\begin{theorem}
	\label{thm:truncated-PR}
	Consider the settings in Section~\ref{sec:phase_retrieval}. Fix any constant $\varepsilon > 0$, and suppose  $m \ge c_0 n $ for some sufficiently large constant $c_0>0$ that is independent of $n$ and $m$ but possibly dependent on $\varepsilon$.  Then the spectral estimate \eqref{eq:estimate-preprocess-PR} equipped with the above choices of $\mathcal{T}$ obeys
\[
	\min \{ \| \bm{x}_{\mathcal{T}} - \bm{x}^{\star}\|_2, \| \bm{x}_{\mathcal{T}} + \bm{x}^{\star}\|_2 \} \le \varepsilon \,\| \bm{x}^{\star} \|_2
\]
with probability at least $1 - O(n^{-2})$, provided that the parameters $\alpha_1,\alpha_2,\alpha_3$ are suitably chosen in \eqref{eq:trimming}--\eqref{eq:TAF}.
\end{theorem}

\begin{figure}
\begin{center}
	\begin{tabular}{c}
	\includegraphics[width=0.55\textwidth]{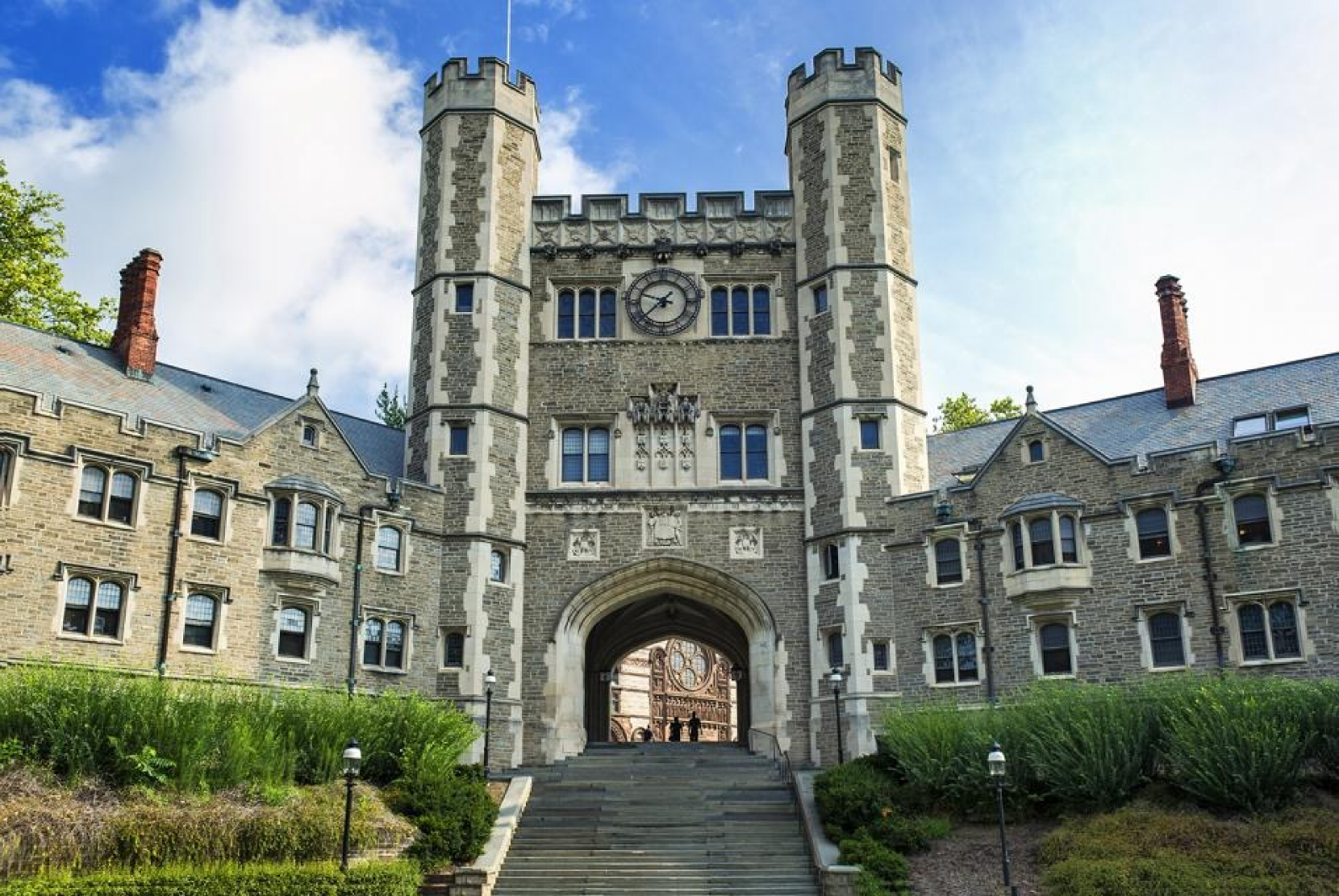}
	\tabularnewline
	(a)
	\tabularnewline
	\includegraphics[width=0.55\textwidth]{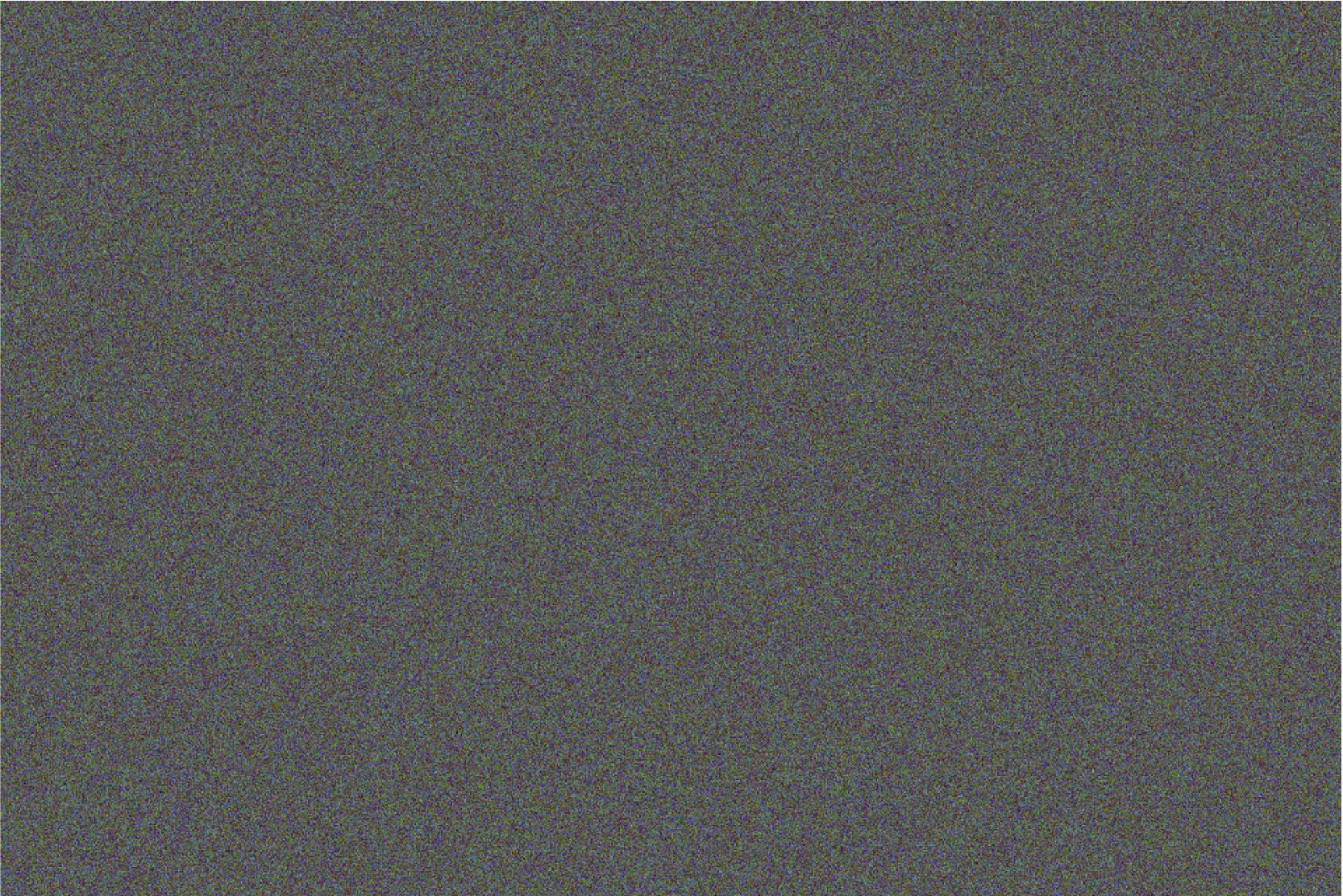}
	\tabularnewline
	(b)
	\tabularnewline
	\includegraphics[width=0.55\textwidth]{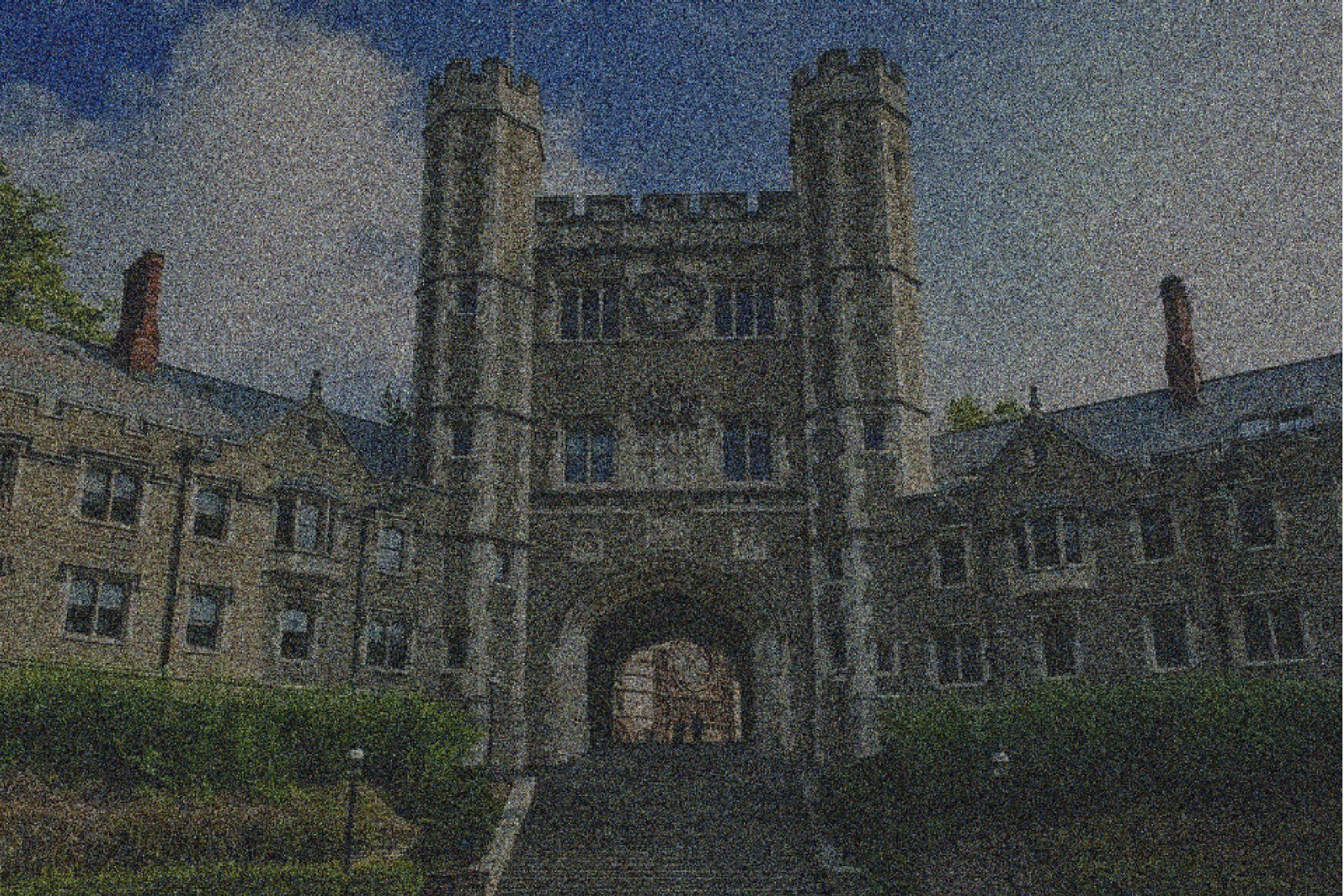}
	\tabularnewline
	(c)
	\tabularnewline
\end{tabular}
\end{center}
	\caption{Numerical performance of spectral methods for phase retrieval under coded diffraction patterns (see~\citet{candes2015phase} for details) when $m=10n$. (a) the original image $\bm{x}^{\star}$ (which is $613,760$-dimensional); (b) the estimate of the spectral method in Section~\ref{sec:pr-alg}; (c) the estimate of the mean-truncated spectral method (cf.~\eqref{eq:trimming}). 
	There are in total $10$ groups of measurements each of size $n$; to generate each group of measurements, the entries of the signal $\bm{x}^{\star}$ are first independently multiplied by random variables uniformly over $\{1,-1,i,-i\}$, followed by an application of the discrete Fourier transform. 
	\label{fig:phase_retrieval_twf}} 
\end{figure}

%
%% paragraph about simulation

To demonstrate the practicability of preprocessing, we depict in
Figure~\ref{fig:phase_retrieval_twf} the numerical performance of  the mean-truncated spectral method (i.e., the choice \eqref{eq:trimming}) in comparison to the vanilla version described in Section~\ref{sec:pr-alg}.
These numerical experiments corroborate the clear advantage of proper preprocessing in the sample-starved regime.

Finally, we remark that in addition to order-wise statistical guarantees,  \citet{lu2020phase} further pinned down sharp characterization of the error bounds (including the pre-constants) in this sample-starved regime. Leveraging such sharp analyses,  \citet{mondelli2017fundamental} identified the information-theoretic optimal choice \eqref{eq:optimal-truncation}, in the sense that it leads to an estimate strictly better than a random guess (a.k.a.~weak recovery) whenever it is information-theoretically possible. \citet{luo2019optimal} further showed that this choice is uniformly optimal, meaning that it leads to the smallest principal angle between $\bm{x}_{\mathcal{T}}$ and $\bm{x}^{\star}$ uniformly over all sampling ratios when $m$ is on the same order of $n$.

\subsubsection{Robustness vis-\`a-vis adversarial outliers}

Another practical consideration that merits special attention is that
the collected samples are sometimes susceptible to adversarial entries
(due to, say, sensor failures or malicious attacks).
To formulate this in more formal terms, consider the following modified measurement model \citep{zhang2016provable,hand2017phaselift,hand2016corruption}:
\begin{equation}
	\label{eq:robust_pr_model}
	y_i = \begin{cases}
	(\ba_i^\top\bx^{\star})^2, \qquad  & i \notin \mathcal{S}_{\mathsf{outlier}}, \\
		\mathsf{arbitrary} ,\qquad & i \in \mathcal{S}_{\mathsf{outlier}}.
	\end{cases}
\end{equation}
Here, $\mathcal{S}_{\mathsf{outlier}} \subseteq \{1,\cdots,m\}$ represents the unknown subset of indices associated with outliers,
which is of cardinality  $|\mathcal{S}_{\mathsf{outlier}}|=\alpha  m$ for some $0<\alpha<1$.
In particular, the  measurements coming from $\mathcal{S}_{\mathsf{outlier}}$ might be corrupted arbitrarily.
The goal is to reliably estimate $\bm{x}^{\star}$ even when the measurements are grossly corrupted.

Unfortunately, the vanilla spectral method presented in Section~\ref{sec:pr-alg} might not function properly
even in the presence of a single outlier; for instance, if the magnitude of this outlier is excessively large, then the leading eigenvector of $\bm{M}$ will be heavily biased by this outlier.
As a result, the spectral method needs to be properly adjusted in order to combat  the adverse effect of outliers.

To circumvent this issue,  we first remind the readers of a classical finding in robust statistics \citep{huber2004robust}: the median statistic is oftentimes robust against adversarial outliers.
Leveraging this finding to the phase retrieval context,
one might naturally employ the median of the measurements $\{y_i\}_{1\leq i\leq m}$ as a tool to help detect any excessively large outlier. In fact, this is precisely the idea behind the median-truncated scheme presented
in \eqref{eq:median-truncation}, whose capability in dealing with outliers has been established in \citet{zhang2016provable} for phase retrieval and \citet{li2017nonconvex} for low-rank matrix recovery.
\begin{theorem}
	\label{thm:truncated-PR-init}
	Consider the measurement model in \eqref{eq:robust_pr_model},
	and the i.i.d.~Gaussian design in Assumption~\ref{assump:pr-gaussian}.
	Fix any  $\varepsilon > 0$. There exist some constants  $c_0>0$ and $0<c_1<1$ such that
	if  $m\geq c_{0}n$ and $\alpha\leq c_1$,
	then the spectral estimate \eqref{eq:estimate-preprocess-PR} equipped with \eqref{eq:median-truncation} obeys
\[
	\min \{ \| \bm{x}_{\mathcal{T}} - \bm{x}^{\star}\|_2, \| \bm{x}_{\mathcal{T}} + \bm{x}^{\star}\|_2 \}  \le \varepsilon \| \bm{x}^{\star} \|_2
\]
with probability at least $1 - O(n^{-2})$.
\end{theorem}
In a nutshell, Theorem~\ref{thm:truncated-PR-init} reveals that a median-truncated spectral method  achieves consistent estimation even when a constant fraction of the measurements are corrupted in an arbitrary manner.  All this is guaranteed to happen even when the number $m$ of samples is on the same order as $n$,  thus further enhancing the resilience of spectral methods in the presence of adversarial corruptions. The interested reader is referred to \citet{zhang2016provable} for the proof of this theorem.

\subsection{Proof of auxiliary lemmas}
\label{sec:proof-auxiliary-lemmas-PR-L2}

\paragraph{Proof of Lemma~\ref{lemma:pr-noise-bound}.}
Given that $\{\bm{a}_{i}\}$ is rotationally invariant,  we  assume without
loss of generality that $\bm{x}^{\star}=\bm{e}_{1}$, where
$\bm{e}_{1}$ is the first standard basis vector. Thus, our task can be translated into bounding
\[
	\bm{M}-\bm{M}^{\star} =  \frac{1}{m}\sum_{i=1}^{m} \Big\{ a_{i,1}^{2}\bm{a}_{i}\bm{a}_{i}^{\top}-\big(2\bm{e}_{1}\bm{e}_{1}^{\top}+\bm{I}_{n}\big) \Big\}
	\eqqcolon \frac{1}{m}\sum_{i=1}^{m} \Big( \bm{B}_i - \mathbb{E}\big[\bm{B}_i \big] \Big),
\]
where $a_{i,1}$ denotes the first entry of the vector $\bm{a}_i$, and $\bm{B}_i\coloneqq a_{i,1}^{2}\bm{a}_{i}\bm{a}_{i}^{\top}$.

% Preview source code from paragraph 6 to 7

In order to deal with the unboundedness of Gaussian random variables,
we resort to the truncated matrix Bernstein inequality, which requires
us to first set a proper truncation level. In view of the Gaussianity
of $\bm{a}_{i}$ and the union bound, one has $\|\bm{a}_{i}\|_{\infty}\leq 5\sqrt{\log m}$ for all $1\leq i\leq m$
with probability at least $1-m^{-11.5}$; on this event, one would have
%This in turn reveals that
%
\begin{equation}
\label{eq:Bi-norm-early}
\big\|\bm{B}_{i}\big\|=\big|a_{i,1}\big|^{2}\big\|\bm{a}_{i}\big\|_{2}^{2}\leq n\big\|\bm{a}_{i}\big\|_{\infty}^{4}\leq5^4 n\log^{2}m .
\end{equation}
%
%holds with probability greater than $1-m^{-11.5}$.
Therefore, taking
$L\coloneqq 5^4 n\log^{2}m$ leads to
\[
\mathbb{P}\big\{\big\|\bm{B}_{i}\big\|\geq L\big\}\leq m^{-11.5}\eqqcolon q_{0}.
\]
Further, truncating at this level does not incur much bias; to be precise, we claim that (with the proof deferred to the end of this subsection)
\begin{align}
	q_{1} & \coloneqq\big\|\mathbb{E}\big[\bm{B}_{i}\mathbbm1\{\|\bm{B}_{i}\|<L\}-\mathbb{E}[\bm{B}_{i}]\big]\big\| \lesssim m^{-3}.
	\label{eq:q1-bound-UB-PR}
\end{align}

The next step is to characterize the variance statistic. Towards
this end, it is first seen from the definition of $\bm{B}_i$ that
\begin{equation}
\mathbb{E}\big[\left(\bm{B}_{i}-\mathbb{E}[\bm{B}_{i}]\right)^{2}\big]\preceq\mathbb{E}[\bm{B}_{i}^{2}]
% =\mathbb{E}\big[a_{i,1}^{4}\bm{a}_{i}\bm{a}_{i}^{\top}\bm{a}_{i}\bm{a}_{i}^{\top}\big]
=\mathbb{E}\big[a_{i,1}^{4}\big\|\bm{a}_{i}\big\|_{2}^{2}\bm{a}_{i}\bm{a}_{i}^{\top}\big].\label{eq:E-Bi-square-PR}
\end{equation}
As can be easily verified, $\mathbb{E}\big[a_{i,1}^{4}\big\|\bm{a}_{i}\big\|_{2}^{2}\bm{a}_{i}\bm{a}_{i}^{\top}\big]$
is a diagonal matrix, whose diagonal entries obey
\[
\Big(\mathbb{E}\big[a_{i,1}^{4}\big\|\bm{a}_{i}\big\|_{2}^{2}\bm{a}_{i}\bm{a}_{i}^{\top}\big]\Big)_{l,l}=\mathbb{E}\big[a_{i,1}^{4}a_{i,l}^{2}\sum\nolimits _{j}a_{i,j}^{2}\big]=\sum\nolimits _{j}\mathbb{E}\big[a_{i,1}^{4}a_{i,l}^{2}a_{i,j}^{2}\big]\lesssim n
\]
for any $1\leq l\leq n$, where the last relation follows from the
property of standard Gaussians. This taken together with
(\ref{eq:E-Bi-square-PR}) gives
\[
v\coloneqq\Big\|\sum_{i}\mathbb{E}\big[\left(\bm{B}_{i}-\mathbb{E}[\bm{B}_{i}]\right)^{2}\big]  \Big \|\leq \sum_{i} \max _{l}\Big|\Big(\mathbb{E}\big[a_{i,1}^{4}\big\|\bm{a}_{i}\big\|_{2}^{2}\bm{a}_{i}\bm{a}_{i}^{\top}\big]\Big)_{l,l}\Big|\lesssim mn.
\]
%
%\[
%\Longrightarrow\qquad v\coloneqq\Big\|\sum\nolimits _{i}\mathbb{E}\big[\left(\bm{B}_{i}-\mathbb{E}[\bm{B}_{i}]\right)^{2}\big]\Big\|\lesssim mn.
%\]
%
Invoking the truncated Bernstein inequality
in Corollary \ref{thm:matrix-Bernstein-friendly-truncated} then yields: with probability at least $1-O(m^{-10})-mq_0=1-O(m^{-10})$, one has
\begin{align*}
\big\|\bm{M}-\bm{M}^{\star}\big\| & \lesssim\frac{1}{m}\sqrt{v\log m}+\frac{L}{m}\log m + \frac{mq_1}{m} \\
 & \lesssim\sqrt{\frac{n\log m}{m}}+\frac{n\log^{3}m}{m} +\frac{1}{m^3} \lesssim\sqrt{\frac{n\log^{3}m}{m}}
\end{align*}
as desired, with the proviso that $m\gtrsim n\log^{3}m$.

\begin{proof}[Proof of the inequality~\eqref{eq:q1-bound-UB-PR}] We begin by employing the relation \eqref{eq:Bi-norm-early} to help modify the truncation event as follows:
\begin{align*}
q_{1} & = \big\|\mathbb{E}\big[\bm{B}_{i}\mathbbm1\{\|\bm{B}_{i}\|<L\}-\mathbb{E}[\bm{B}_{i}]\big]\big\|
=\big\|\mathbb{E}\big[\bm{B}_{i}\mathbbm1\{\|\bm{B}_{i}\|\geq L\}\big]\big\|\\
 & \leq\big\|\mathbb{E}\big[\bm{B}_{i}\mathbbm1\{n\|\bm{a}_{i}\|_{\infty}^{4}\geq L\}\big]\big\|
  =\big\|\mathbb{E}\big[\bm{B}_{i}\mathbbm1\{\|\bm{a}_{i}\|_{\infty}\geq\overline{L}\}\big]\big\|,
\end{align*}
where
%the inequality arises from , and
we define $\overline{L}\coloneqq(L/n)^{1/4}=5 \sqrt{\log m}$. It
is easily seen that $\mathbb{E}\big[\bm{B}_{i}\mathbbm1\{\|\bm{a}_{i}\|_{\infty}\geq\overline{L}\}\big]$
is a diagonal matrix and, therefore,
\begin{align}
 q_{1}  & \leq \max_l \Big|\Big(\mathbb{E}\big[\bm{B}_{i}\mathbbm1\{\|\bm{a}_{i}\|_{\infty}\geq\overline{L}\}\big]\Big)_{l,l}\Big|
	\overset{\mathrm{(i)}}{=} \max_l  \mathbb{E}\big[a_{i,1}^{2}a_{i,l}^{2}\mathbbm1\{\|\bm{a}_{i}\|_{\infty}\geq\overline{L}\}\big] \nonumber\\
	& \overset{\mathrm{(ii)}}{\leq} 0.5 \max_l  \mathbb{E}\big[(a_{i,1}^{4}+a_{i,l}^{4})\mathbbm1\{\|\bm{a}_{i}\|_{\infty}\geq\overline{L}\}\big]
   = \mathbb{E}\big[a_{i,1}^{4}\mathbbm1\{\|\bm{a}_{i}\|_{\infty}\geq\overline{L}\}\big] \nonumber\\
 & \leq \mathbb{E}\big[a_{i,1}^{4}\mathbbm1\{\big|a_{i,1}\big|\geq\overline{L}\}\big]+ \mathbb{E}\big[a_{i,1}^{4}\mathbbm1\{\max\nolimits_{j\neq1}\big|a_{i,j}\big|\geq\overline{L}\}\big],
\label{eq:q1-UB-PR-1234}
\end{align}
where (i) relies on the definition of $\bm{B}_i$, and (ii) comes from the AM-GM inequality. With regards to the first term of \eqref{eq:q1-UB-PR-1234}, observe that
\begin{align*}
	& \mathbb{E}\big[a_{i,1}^{4}\mathbbm{1}\{|a_{i,1}|\geq 5\sqrt{\log m}\}\big]
		=\int_{5\sqrt{\log m}}^{\infty}\frac{2\xi^{4}}{\sqrt{2\pi}}e^{-\xi^{2}/2} \mathrm{d}\xi \\
	& \quad \leq2\int_{5\sqrt{\log m}}^{\infty}\frac{e^{-\xi^{2}/4}}{\sqrt{2\pi}} \mathrm{d}\xi=2\mathbb{P}\big\{|a_{i,1}|\geq2.5\sqrt{\log m} \big\}\lesssim\frac{1}{m^3}
\end{align*}
for $m$ sufficiently large,
where we have used the fact that $\xi^{4}e^{-\xi^{2}/2}\leq e^{-\xi^{2}/4}$
for $\xi\geq5\sqrt{\log m}$. Regarding the second term of \eqref{eq:q1-UB-PR-1234}, note that
\begin{align*}
	& \mathbb{E}\Big[a_{i,1}^{4}\mathbbm1\Big\{\max_{j\neq1}\big|a_{i,j}\big|\geq\overline{L}\Big\}\Big]  =\mathbb{E}\big[a_{i,1}^{4}\big]\mathbb{E}\Big[\mathbbm1\Big\{\max_{j\neq1}\big|a_{i,j}\big|\geq\overline{L}\Big\}\Big]\\
 	& \qquad\qquad =3\mathbb{P}\Big\{\max_{j\neq1}\big|a_{i,j}\big|\geq5\sqrt{\log m}\Big\}\lesssim m^{-10},
\end{align*}
where the first identity uses the independence between $a_{i,1}$ and $\{a_{i,j}\}_{j\neq 1}$.
Substituting the preceding bounds into \eqref{eq:q1-UB-PR-1234} establishes \eqref{eq:q1-bound-UB-PR}. \end{proof}

%% file: chapters/matrix_completion.tex
\section{Matrix completion}\label{sec:mc-l2}

A pressing challenge often encountered in data science applications is estimation and learning in the face of \emph{missing data}.
To elucidate how to tackle this challenge via spectral methods, we delve into the renowned matrix completion problem in this section,
followed by another application called tensor completion in Section~\ref{sec:tensor-completion}.

Imagine that one observes a small subset of the entries in a large unknown matrix
and seeks to fill in all missing entries. An archetypal example is collaborative filtering, where one aims to predict the users' preferences on a collection of products based on partially revealed user-product ratings.   See Figure~\ref{fig:completion} for an illustration.  The problem, often referred to as matrix completion, is apparently ill-posed in general, as there are (much) fewer measurements than the unknowns.

%This is indeed a factor model in Section~\ref{sec:PCA} with missing data.  

Fortunately, if the matrix of interest exhibits certain low-dimensional structure, then reliable recovery becomes feasible.
A commonly encountered example of this kind concerns the case when the target matrix enjoys a low-rank structure. Again, take collaborative filtering for example:  the user-product rating matrix  might be well explained by a relatively small number of latent factors connecting users' preferences with products' attributes, thus resulting in an approximately low-rank matrix.  Motivated by its fundamental importance,  recent years have witnessed a flurry of research activity in studying low-rank matrix completion \citep{candes2009exact, keshavan2010matrix, gross2011recovering};
see \citet{chen2018harnessing,davenport2016overview} for overviews of recent developments.
In the sequel, we present a simple yet effective approach enabled by the spectral method, originally proposed in~\citet{achlioptas2007fast,keshavan2010matrix}.

\begin{figure}
\begin{center}
\includegraphics[width=0.35\textwidth]{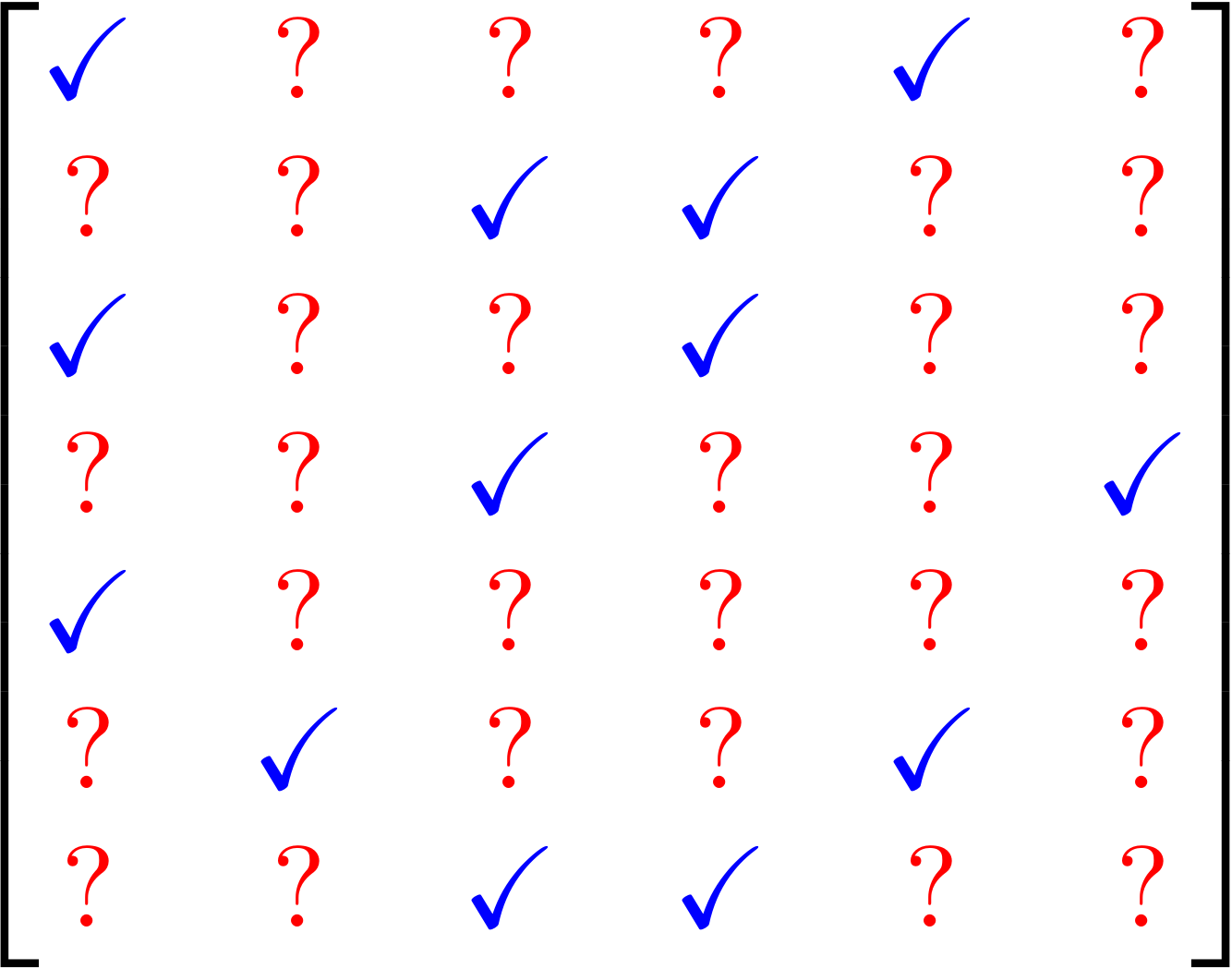}
\end{center}
\caption{Illustration of matrix completion, where each ``$\checkmark$'' stands for an observed entry and each ``$?$'' represents a missing entry.} \label{fig:completion}
\end{figure}

\subsection{Problem formulation and assumptions}
\label{sec:problem-formulation-MC}

Suppose that we are interested in estimating an $n_{1}\times n_{2}$  rank-$r$ matrix  $\bm{M}^{\star} = [{M}^{\star}_{i,j}]_{1\leq i\leq n_1, 1\leq j\leq n_2}$.
Without loss of generality, we assume
$$n_{1}\leq n_{2}.$$  Denote by  $\bm{M}^{\star}=\bm{U}^{\star}\bm{\Sigma}^{\star}\bm{V}^{\star\top}$ the SVD of $\bm{M}^{\star}$, where the columns of $\bm{U}^{\star}\in \mathbb{R}^{n_1\times r}$ (resp.~$\bm{V}^{\star}\in \mathbb{R}^{n_2\times r}$) are the left (resp.~right) singular vectors of $\bm{M}^{\star}$, and $\bm{\Sigma}^{\star}$ is a diagonal matrix whose diagonal entries are the singular values of $\bm{M}^{\star}$. We define the condition number of the matrix $\bm{M}^{\star}$ to be $\kappa \coloneqq \sigma_1(\bM^\star)/\sigma_r(\bM^\star)$.

To capture the presence of missing data, we introduce an index subset $\Omega\subseteq [n_1]\times [n_2]$,
such that each entry ${M}^{\star}_{i,j}$ is observed if and only if $(i,j)\in \Omega$.
The goal is to reconstruct the singular subspaces $\bm{U}^{\star}$ and $\bm{V}^{\star}$, as well as the  full matrix $\bm{M}^{\star}$, based on entries observed over the sampling set $\Omega$.

\paragraph{Random sampling.}
Apparently, not all sampling patterns admit reliable estimation. For instance, if $\Omega$ contains only entries in the top half of the matrix,
then there is in general no hope to predict the bottom half of the matrix. In order to allow for meaningful matrix completion,
this monograph focuses on a natural random observation model commonly adopted in the literature, as formulated below.
\begin{assumption}[Random sampling]
\label{assump:mc-bernoulli}
Each entry of $\bm{M}^{\star}$ is observed independently with probability $0<p<1$, namely,
each $(i,j)\in [ n_{1}] \times [n_{2}]$ is included in $\Omega$ independently with probability $p$.
\end{assumption}
Under this model, we shall view the expected number of observed entries---namely, $pn_1n_2$---as the sample size. In truth, as long as $p$ is not overly small,  the number of observed entries is expected to concentrate  around its mean $pn_1n_2$.

\paragraph{Incoherence conditions.}
Caution needs to be exercised, however,  that the random sampling model alone does not guarantee effective recovery of an arbitrary low-rank matrix $\bm{M}^{\star}$.
Consider, for example, the following rank-1 matrix $\bm{M}^{\star}$ containing a single nonzero entry:
\[
%{\small
\left[\begin{array}{cccc}
1 & 0 & \cdots & 0\\
0 & 0 & \cdots & 0\\
\vdots & \vdots & \ddots & \vdots\\
0 & 0 & \cdots & 0
\end{array}\right] .
%}
\]
If $p=o(1)$, then with probability $1-p=1-o(1)$, the sampling pattern will fail to include the nonzero entry $M_{1,1}^{\star}$,
thus ruling out the possibility of faithful matrix recovery.
Consequently,  one needs to make sure that the sampling pattern does not suppress too much useful information.
Towards this end, the pioneering work \citet{candes2009exact, candes2010NearOptimalMC} singled out an incoherence parameter that plays a vital role.
\begin{definition}
\label{assump:mc-incoherence}
The incoherence parameter $\mu$ of the matrix $\bm{M}^{\star}\in\mathbb{R}^{n_{1}\times n_{2}}$ is defined as
\begin{align*}
	\mu\coloneqq \max\left\{  \frac{ n_1 \|\bm{U}^{\star}\|_{2,\infty}^2 }{r} ,  \frac{ n_2 \|\bm{V}^{\star}\|_{2,\infty}^2 }{r} \right\} .
	%\\
	%\|\bm{V}^{\star}\|_{2,\infty} & \leq\sqrt{\mu/n_{2}}\,\|\bm{V}^{\star}\|_{\mathrm{F}}=\sqrt{\mu r/n_{2}}.
\end{align*}
\end{definition}
\begin{remark}
\label{remark:range-mu-MC}
Recognizing the following basic relation
\[
	\frac{r}{n_1}=\frac{1}{n_1} \|\bm{U}^{\star}\|_{\mathrm{F}}^2 \leq \|\bm{U}^{\star}\|_{2,\infty}^2 \leq \|\bm{U}^{\star}\|^2 = 1
\]
and an analogous one for $\bm{V}^{\star}$, we have $1\leq \mu \leq \max\{n_1,n_2\}/r= n_2/r$.
\end{remark}

In words, a small $\mu$ indicates that the energy of the singular vectors is spread out across different elements, namely,  the singular subspace of $\bm{M}^{\star}$ is not too ``aligned'' with any of the standard basis vectors,
thus ensuring that entrywise observations provide somewhat equalized information about the full spectrum of $\bm{M}^{\star}$.
The following lemma summarizes a few immediate consequences of this definition, with the proof deferred to Section~\ref{sec:proof-auxiliary-lemmas-MC-L2}.
\begin{lemma}\label{claim:mu-incoherence}
Assume that  $\bm{M}^{\star}\in\mathbb{R}^{n_{1}\times n_{2}}$ is $\mu$-incoherent. Then the following relations hold
\begin{subequations}
\label{subeq:mu-incoherence}
\begin{align}
	\|\bm{M}^{\star}\|_{2,\infty} & \leq\sqrt{\mu r/n_{1}}\, \big\| \bm{M}^{\star} \big\|; \quad
	\|\bm{M}^{\star\top}\|_{2,\infty}  \leq\sqrt{\mu r/n_{2}}\, \big\| \bm{M}^{\star} \big\|;\label{eq:mc-row-norm-upper}\\
	& \qquad\quad \|\bm{M}^{\star}\|_{\infty}  \leq\mu r  \big\| \bm{M}^{\star} \big\| /\sqrt{n_{1}n_{2}} . \label{eq:mc-entry-upper}
\end{align}
\end{subequations}
\end{lemma}

\paragraph{Additional notation.}

We find it convenient to introduce a Euclidean projection operator $\mathcal{P}_{\Omega}:\mathbb{R}^{n_{1}\times n_{2}}\mapsto\mathbb{R}^{n_{1}\times n_{2}}$ such that
\begin{equation}\label{defn:Pomega}
\big[\mathcal{P}_{\Omega}(\bm{A}) \big]_{i,j}=\begin{cases}
A_{i,j},\quad & \text{if }(i,j)\in\Omega \\
0, & \text{else}
\end{cases}
\end{equation}
for any matrix $\bm{A} =[A_{i,j}]\in\mathbb{R}^{n_{1}\times n_{2}}$.
With this notation in place, matrix completion amounts to recovering $\bm{M}^{\star}$ on the basis of $\mathcal{P}_{\Omega}(\bm{M}^{\star})$.

\subsection{Algorithm}\label{sec:mc-alg}

To apply the spectral method, the first step is to form a reasonable approximation $\bm{M}$ of the unknown matrix $\bm{M}^{\star}$.
By virtue of the random sampling model (cf.~Assumption~\ref{assump:mc-bernoulli}), a candidate approximation can be obtained from the observed data matrix via inverse probability weighting:
\begin{align}
	\bm{M}   \coloneqq p^{-1}\mathcal{P}_{\Omega}(\bm{M}^{\star}).
	\label{eq:rescaled-M-matrix-completion-l2}
\end{align}
The rationale is that $\bm{M}$ forms an unbiased estimate of the ground truth, namely, $$\mathbb{E}[\bm{M}]=\bm{M}^{\star},$$ where
the expectation is taken over the randomness in $\Omega$.

As a result, the proposed spectral method proceeds by computing the rank-$r$ SVD $\bm{U}\bm{\Sigma}\bm{V}^{\top}$ of the matrix $\bm{M}$ constructed in \eqref{eq:rescaled-M-matrix-completion-l2}, and employing $\bm{U}\in \mathbb{R}^{n_1\times r}$, $\bm{V}\in \mathbb{R}^{n_2\times r}$  and $\bm{U}\bm{\Sigma}\bm{V}^{\top}$ as estimates of $\bm{U}^{\star}$, $\bm{V}^{\star}$ and $\bm{M}^{\star}$, respectively.

\subsection{Performance guarantees}\label{sec:mc-l2-theory}

As before, whether the subspace $\bm{U}$ (resp.~$\bm{V}$) is close
to $\bm{U}^{\star}$ (resp.~$\bm{V}^{\star}$) relies crucially on
the size of the perturbation $\|\bm{M}-\bm{M}^{\star}\|$.
Therefore, we begin by developing an upper bound on this quantity; the proof is based on the matrix Bernstein inequality and is postponed to Section~\ref{sec:proof-auxiliary-lemmas-MC-L2}.

\begin{lemma}
\label{lemma:mc-noise-bound}
Consider the settings in Section~\ref{sec:problem-formulation-MC}. Suppose that $n_{2}p\geq C\mu r\log n_{2}$ for some
 constant $C>0$. Then with probability at least
	$1-O(n_{2}^{-10})$, the matrix $\bm{M}$ constructed in \eqref{eq:rescaled-M-matrix-completion-l2} obeys
\[
	\big\| \bm{M} -\bm{M}^{\star} \big\|
	\lesssim\sqrt{\frac{\mu r\log n_{2}}{n_{1}p}}\, \big\| \bm{M}^{\star} \big\|.
\]
\end{lemma}
With this perturbation bound in place, we are equipped to apply Wedin's $\sin\bm{\Theta}$ theorem to obtain the following results.  The condition on the sample size in Theorem~\ref{thm:mc-l2-subspace} is stronger than that in Lemma~\ref{lemma:mc-noise-bound}, as we need to control the eigengap in the following theorem.
\begin{theorem}
\label{thm:mc-l2-subspace}
Consider the settings in Section~\ref{sec:problem-formulation-MC}.
Suppose that $n_{1}p\geq C_{1}\kappa^{2}\mu r\log n_{2}$ for some sufficiently
large constant $C_{1}>0$. Then with probability exceeding $1-O(n_2^{-10})$,
\begin{align*}
\max\Big\{ \mathsf{dist}\left(\bm{U},\bm{U}^{\star}\right),\mathsf{dist}\left(\bm{V},\bm{V}^{\star}\right)\Big\}  & \lesssim\kappa\sqrt{\frac{\mu r \log n_{2}}{n_{1}p}} .
%;\\
%\max\Big\{ \mathsf{dist}_{\mathrm{F}}\left(\bm{U},\bm{U}^{\star}\right),\mathsf{dist}_{\mathrm{F}}\left(\bm{V},\bm{V}^{\star}\right)\Big\}  & \lesssim\kappa\sqrt{\frac{\mu r^{2} \log n_{2}}{n_{1}p}}.
\end{align*}
\end{theorem}
\begin{proof}
As a direct consequence of Lemma \ref{lemma:mc-noise-bound}, one has
\[
	\big\| \bm{M} -\bm{M}^{\star} \big\|
	\lesssim\sqrt{\frac{\mu r\log n_{2}}{n_{1}p}} \,\big\| \bm{M}^{\star}\big\| \le \Big( 1 - \frac{1}{\sqrt{2}} \Big) \sigma_{r}(\bm{M}^{\star}),
\]
provided that $n_{1}p\geq C_{1}\kappa^{2}\mu r\log n_{2}$ for some  large enough constant $C_{1}>0$.
Apply Wedin's theorem (cf.~\eqref{eq:wedin-simpler-friendly}) and Lemma \ref{lemma:mc-noise-bound} to obtain
\begin{align*}
 & \max\Big\{ \mathsf{dist}\left(\bm{U},\bm{U}^{\star}\right),\mathsf{dist}\left(\bm{V},\bm{V}^{\star}\right) \Big\}
	\leq  \frac{2\big\|\bm{M} -\bm{M}^{\star}\big\|}{\sigma_{r}(\bm{M}^{\star})}
	%\\ % & \qquad\qquad
	%\lesssim\frac{\sqrt{\frac{\mu r\log n_{2}}{n_{1}p}}\,\big\|\bm{M}^{\star}\big\|}{\sigma_{r}(\bm{M}^{\star})}\asymp
	\lesssim \kappa\sqrt{\frac{\mu r\log n_{2}}{n_{1}p}}
\end{align*}
as claimed.
%The proof for the $\mathsf{dist}_{\mathrm{F}}(\cdot,\cdot)$ bounds follows from the same argument and is hence omitted.
%
\end{proof}

As an important implication of Theorem~\ref{thm:mc-l2-subspace}, once the sample size exceeds
$$pn_1n_2 \gg \kappa^{2}\mu r n_2  \log n_{2}, $$
then the spectral estimate achieves consistent estimation in the sense that
\[
	\max\Big\{ \mathsf{dist}\left(\bm{U},\bm{U}^{\star}\right),\mathsf{dist}\left(\bm{V},\bm{V}^{\star}\right)\Big\} = o(1).
\]
Given that $pn_1n_2\gtrsim \mu n_2 r \log n_2$ is an information-theoretic sampling requirement for reliable matrix completion when $r=o(n_1/\log n_2)$ \citep{candes2010NearOptimalMC},
Theorem~\ref{thm:mc-l2-subspace} confirms the near optimality of spectral methods---in terms of the scaling with $n_1$, $n_2$ and $p$---when it comes to consistent subspace estimation.

Before moving forward to the proof, we further characterize the statistical accuracy of $\bm{U}\bm{\Sigma}\bm{V}^{\top}$ in estimating the unknown matrix $\bm{M}^{\star}$. Accomplishing this only requires Lemma~\ref{lemma:mc-noise-bound}, without any need of the singular subspace perturbation theory. This result will also come in handy when we turn to discussing entrywise estimation accuracy in Chapter~\ref{cha:Linf-theory}.
\begin{theorem}
\label{thm:mc-l2-matrix}
Consider the settings in Section~\ref{sec:problem-formulation-MC}.
Suppose that $n_{2}p\geq C\mu r\log n_{2}$ for some sufficiently large constant
$C>0$. Then with probability at least $1-O(n_{2}^{-10})$, one has
\begin{align*}
%\|\bm{U}\bm{\Sigma}\bm{V}^{\top}-\bm{M}^{\star}\| & \lesssim\sqrt{\frac{\mu r}{n_{1}p}\log n_{2}}\,\sigma_{1}(\bm{M}^{\star});\\
\|\bm{U}\bm{\Sigma}\bm{V}^{\top}-\bm{M}^{\star}\|_{\mathrm{F}} & \lesssim\sqrt{\frac{\mu r^{2}\log n_{2}}{n_{1}p}}\, \big\| \bm{M}^{\star} \big\|.
\end{align*}
\end{theorem}
\begin{proof}
First, note
\begin{align*}
\|\bm{U}\bm{\Sigma}\bm{V}^{\top}-\bm{M}^{\star}\| &\leq\|\bm{U}\bm{\Sigma}\bm{V}^{\top}- \bm{M} \|+\|\bm{M} -\bm{M}^{\star}\| \leq2\|\bm{M}-\bm{M}^{\star}\|,
\end{align*}
where the first inequality comes from the triangle inequality, and the second inequality follows from the fact that $\bm{U}\bm{\Sigma}\bm{V}^{\top}$ is the best rank-$r$ approximation to $\bm{M}$,
i.e.,
$$\|\bm{U}\bm{\Sigma}\bm{V}^{\top}-\bm{M}\|=\min_{\bm{Z}: \mathsf{rank}(\bm{Z})\leq r}\|\bm{Z}-\bm{M}\| \leq \|\bm{M}-\bm{M}^{\star}\|. $$
Additionally, it is observed that $\bm{U}\bm{\Sigma}\bm{V}^{\top}-\bm{M}^{\star}$ has rank at most $2r$, which implies
\[
	\|\bm{U}\bm{\Sigma}\bm{V}^{\top}-\bm{M}^{\star}\|_{\mathrm{F}}\leq\sqrt{2r}\|\bm{U}\bm{\Sigma}\bm{V}^{\top}-\bm{M}^{\star}\| \leq 2\sqrt{2r} \big\|\bm{M} - \bm{M}^{\star} \big\|.
\]
This combined with Lemma~\ref{lemma:mc-noise-bound} immediately concludes the proof. \end{proof}

\subsection{Proof of auxiliary lemmas}
\label{sec:proof-auxiliary-lemmas-MC-L2}

\paragraph{Proof of Lemma~\ref{claim:mu-incoherence}.}
First of all, the  $\|\cdot\|_{2,\infty}$ norm of $\bm{M}^{\star}$
can be upper bounded by
\begin{equation*}
\|\bm{M}^{\star}\|_{2,\infty}=\|\bm{U}^{\star}\bm{\Sigma}^{\star}\bm{V}^{\star\top}\|_{2,\infty}\leq\|\bm{U}^{\star}\|_{2,\infty}\|\bm{\Sigma}^{\star}\| \, \|\bm{V}^{\star}\|\leq\sqrt{\frac{\mu r}{n_{1}}} \big\| \bm{M}^{\star} \big\|.
\end{equation*}
Here, the first inequality arises from the elementary bounds
$\|\bm{A}\bm{B}\|_{2,\infty}\leq\|\bm{A}\|_{2,\infty}\|\bm{B}\|$
and $\|\bm{A}\bm{B}\|\leq\|\bm{A}\|\, \|\bm{B}\|$, whereas the last relation
uses Definition~\ref{assump:mc-incoherence}, the orthonormality
of $\bm{V}^{\star}$, and identifies $\|\bm{\Sigma}^{\star}\|$ with
$\| \bm{M}^{\star} \|$. The  bound on $\|\bm{M}^{\star\top}\|_{2,\infty}$ can be derived analogously and is omitted for brevity.

In addition, the matrix $\bm{M}^{\star}$ is elementwise bounded by
\begin{equation*}
\|\bm{M}^{\star}\|_{\infty}=\|\bm{U}^{\star}\bm{\Sigma}^{\star}\bm{V}^{\star\top}\|_{\infty}\leq\|\bm{U}^{\star}\|_{2,\infty}\|\bm{V}^{\star}\|_{2,\infty} \|\bm{\Sigma}^{\star}\|
	\leq\frac{\mu r}{\sqrt{n_{1}n_{2}}} \big\| \bm{M}^{\star} \big\|.
\end{equation*}
Here, the first inequality follows from the fact $\|\bm{A}\bm{B}^{\top}\|_{\infty}\leq\|\bm{A}\|_{2,\infty}\|\bm{B}\|_{2,\infty}$
and the aforementioned one $\|\bm{A}\bm{B}\|_{2,\infty}\leq\|\bm{A}\|_{2,\infty}\|\bm{B}\|$,
while the last inequality again relies on Definition~\ref{assump:mc-incoherence}.
%\end{proof}

\paragraph{Proof of Lemma~\ref{lemma:mc-noise-bound}.} Note that the matrix $\bm{E}\coloneqq p^{-1}\mathcal{P}_{\Omega}(\bm{M}^{\star})-\bm{M}^{\star}$
can be expressed as the sum of $n_{1}n_{2}$ i.i.d.~random matrices
\[
  \frac{1}{p}  \mathcal{P}_{\Omega}(\bm{M}^{\star})-\bm{M}^{\star}=\sum_{i=1}^{n_{1}}\sum_{j=1}^{n_{2}}\underbrace{\big(p^{-1} \delta_{i,j} -1 \big)M_{i,j}^{\star}\bm{e}_{i}\bm{e}_{j}^{\top}}_{\eqqcolon\bm{X}_{i,j}}.
\]
Here, $\delta_{i,j}$ (which indicates whether the $(i,j)$-th entry is observed) follows an independent Bernoulli distribution
with parameter~$p$, and $\bm{e}_{i}$ stands for the $i$-th standard
basis vector of appropriate dimensions. It is easily seen that for each $(i,j)$,
\[
	\mathbb{E}[\bm{X}_{i,j}]=\bm{0}\quad\text{and}\quad\|\bm{X}_{i,j}\|\leq  \frac{1}{p} \|\bm{M}^{\star}\|_{\infty}\leq\frac{\mu r}{p\sqrt{n_{1}n_{2}}} \big\| \bm{M}^{\star} \big\|,
\]
where the last relation results from the entrywise upper bound (\ref{eq:mc-entry-upper})
on $\bm{M}^{\star}$. In order to apply the matrix Bernstein inequality (cf.~Corollary~\ref{thm:matrix-Bernstein-friendly}),
we need to control the variance statistic
\[
	v\coloneqq\max \Big\{ \Big\Vert \sum\nolimits _{i,j}\mathbb{E}\left[\bm{X}_{i,j}\bm{X}_{i,j}^{\top}\right]\Big\Vert ,\Big\Vert \sum\nolimits _{i,j}\mathbb{E}\left[\bm{X}_{i,j}^{\top}\bm{X}_{i,j}\right] \Big\Vert \Big\} .
\]
Regarding the first variance term, we have
\begin{align*}
\sum\nolimits _{i,j}\mathbb{E}\left[\bm{X}_{i,j}\bm{X}_{i,j}^{\top}\right]
	& =\sum\nolimits _{i,j}\mathbb{E}\left[ \big( p^{-1} \delta_{i,j} -1 \big)^{2}(M_{i,j}^{\star})^{2}\bm{e}_{i}\bm{e}_{j}^{\top}\bm{e}_{j}\bm{e}_{i}^{\top}\right]\\
 & =\frac{1-p}{p}\sum\nolimits _{i,j}\big(M_{i,j}^{\star}\big)^{2}\bm{e}_{i}\bm{e}_{i}^{\top}=\frac{1-p}{p}\sum_{i=1}^{n_{1}}\|\bm{M}_{i,\cdot}^{\star}\|_{2}^{2}\bm{e}_{i}\bm{e}_{i}^{\top} \\
	& \preceq \frac{1-p}{p}  \|\bm{M}^{\star}\|_{2,\infty}^{2} \bm{I}_{n_1} \preceq \frac{\mu r}{n_1p}  \|\bm{M}^{\star}\|^{2} \,\bm{I}_{n_1} .
\end{align*}
Here, the first identity arises from the definition of $\bm{X}_{i,j}$,
the second one calculates the variance of Bernoulli random variables,
and the last line relies on the upper bound~(\ref{eq:mc-row-norm-upper}).
Similarly, the second term in the variance statistic enjoys the following characterization:
\[
	\sum\nolimits _{i,j}\mathbb{E}\left[\bm{X}_{i,j}^{\top}\bm{X}_{i,j}\right] \preceq \frac{\mu r}{n_2p} \|\bm{M}^{\star}\|^{2} \,\bm{I}_{n_2}.
\]
Taking the above  relations together and recalling that $n_1\leq n_2$ give
\begin{align*}
v & %\leq\frac{1}{p}\max\left\{ \|\bm{M}^{\star}\|_{2,\infty}^{2},\|\bm{M}^{\star\top}\|_{2,\infty}^{2}\right\}
	% \leq\frac{1}{p}\left\{ \frac{\mu r}{n_{1}} \big\| \bm{M}^{\star} \big\|^2,  \frac{\mu r}{n_{2}} \big\| \bm{M}^{\star} \big\|^2 \right\} \\
	\leq\frac{\mu r}{n_{1}p} \big\| \bm{M}^{\star} \big\|^2.
\end{align*}

With the above bounds in place, invoking matrix Bernstein (see Corollary~\ref{thm:matrix-Bernstein-friendly}) reveals that: with probability at least $1-O(n_2^{-10})$,
\begin{align*}
	\| \bm{E}\| & \lesssim \sqrt{\frac{\mu r \| \bm{M}^{\star} \|^2 \log n_{2}}{n_{1}p} }+\frac{\mu r \| \bm{M}^{\star} \| \log n_{2}}{p\sqrt{n_{1}n_{2}}}
	\asymp \sqrt{\frac{\mu r \| \bm{M}^{\star} \|^2 \log n_{2}}{n_{1}p} } ,
\end{align*}
where the last inequality is valid as long as  $n_{2}p \gtrsim \mu r\log n_{2}$.

%% file: chapters/tensor_completion.tex
\section{Tensor completion}
\label{sec:tensor-completion}

Tensor data, which can be viewed as a higher-order generalization of matrix data,  are routinely used in science and engineering applications to capture multi-way interactions across variables of interest \citep{kolda2009tensor,sidiropoulos2017tensor,anandkumar2014tensor}. Akin to matrix completion,   the problem of tensor completion aims to  reconstruct a (structured) tensor when the vast majority of its entries are unobserved, a task that spans a wide spectrum of applications including  visual data inpainting, harmonic retrieval,  seismic data analysis, and so on \citep{liu2012tensor,chen2013spectral,kreimer2013tensor}.

Apparently, this task cannot possibly be accomplished without exploiting further structural assumptions on the tensor under consideration. Inspired by the success of low-rank matrix completion, we explore the case where the unknown tensor enjoys certain low-rank structure (more specifically, low canonical-polyadic (CP) rank~\citep{kolda2009tensor}). For simplicity, we concentrate on order-three tensors (namely, $\bm{T}=[T_{i,j,k}]\in \mathbb{R}^{n_1\times n_2\times n_3}$), which already capture several fundamental challenges intrinsic to tensor estimation. In addition, we take the dimensionality $n_1 = n_2 = n_3 = n$ for simplicity of presentation.

\begin{figure}
\begin{center}
\includegraphics[width=0.65\textwidth]{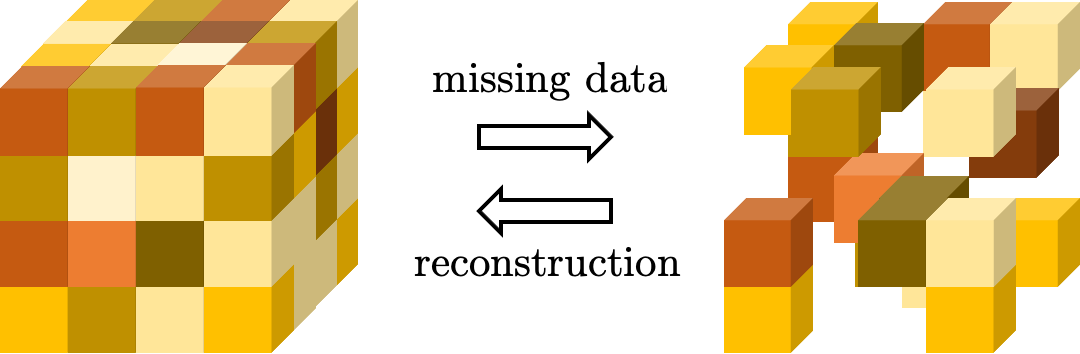}
\end{center}
\caption{Illustration of tensor completion, where we observe partial entries of an order-three tensor.\label{fig:tensor-completion-illustration}}
\end{figure}

\subsection{Problem formulation and assumptions}
\label{sec:setup-tensor-completion}

\paragraph{Notation.} Before describing our models, we introduce several notation that will be useful throughout. For any vectors $\bm{a}=[a_{i}]_{1\leq i\leq n},\bm{b}=[b_{i}]_{1\leq i\leq n},\bm{c}=[c_{i}]_{1\leq i\leq n}\in\mathbb{R}^{n}$,
the tensor $\bm{a}\otimes\bm{b}\otimes\bm{c}$ stands for an $n\times n\times n$
array whose $(i,j,k)$-th entry is given by $a_{i}b_{j}c_{k}$. Additionally, we denote by $\bm{a}\otimes \bm{b}\coloneqq  {\scriptsize \left[\begin{array}{c}
a_{1}\bm{b}\\
\vdots\\
a_{n}\bm{b}
\end{array}\right]}$ the Kronecker product between $\bm{a}$ and $\bm{b}$.  For any tensor $\bm{T}=[T_{i,j,k}]_{1\leq i,j,k\leq n}\in \mathbb{R}^{n\times n\times n}$, we say that $\bm{A}=[A_{i,j}]\in \mathbb{R}^{n\times n^2}$ is the mode-1 matricization of $\bm{T}$, denoted by $$\bm{A}=\mathsf{unfold}(\bm{T}),$$  if $A_{i,(j-1)n+k} = T_{i,j,k}$ for all $(i,j,k)\in [n]\times [n] \times [n]$.

\paragraph{Models and assumptions.}
Suppose the unknown order-three symmetric tensor $\bm{T}^{\star}=[T_{i,j,k}^{\star}]_{1\leq i,j,k\leq n}$  is a superposition of $r$ ($r<n$) rank-one symmetric tensors:
\begin{equation}
\bm{T}^{\star}=\sum_{i=1}^{r}\bm{w}_{i}^{\star}\otimes\bm{w}_{i}^{\star}\otimes\bm{w}_{i}^{\star} \in \mathbb{R}^{n\times n\times n},
\label{eq:defn-true-tensor}
\end{equation}
where $\{\bm{w}_i^{\star}\in \mathbb{R}^{n}\}$ represents a set of latent tensor factors.
What we have available are incomplete observations of the entries of  $\bm{T}^{\star}$. The observed data can be succinctly encoded by an index subset $\Omega \subseteq [n]\times [n] \times [n]$ (called a sampling set) and a tensor $\bm{T}=[T_{i,j,k}]_{1\leq i,j,k\leq n}$ as follows
\begin{equation}
T_{i,j,k}=\begin{cases}
T_{i,j,k}^{\star},\qquad & \text{if }(i,j,k)\in\Omega,\\
0, & \text{else}.
\end{cases}
\label{eq:observed-tensor-T}
\end{equation}
This subsection aims for an intermediate goal, namely, estimating the subspace spanned by $\{\bm{w}_i^{\star}\}_{1\leq i\leq r}$, which often serves as a crucial initial stage towards reliable completion of the whole tensor. The interested reader is referred to \citet{montanari2016spectral,cai2019tensorOR} for subsequent stages of tensor completion algorithms.

Similar to the matrix completion counterpart, we explore a {\em random sampling} pattern such that for all $(i,j,k)\in [n]\times [n] \times [n]$,
\begin{align}
	(i,j,k) \in \Omega \qquad \text{independently with probability }p.
\end{align}
In addition, we define for notational convenience that
\begin{equation}
	\nu_{i}\coloneqq\|\bm{w}_{i}^{\star}\|_{2}^{3},\qquad \nu_{\min}\coloneqq \min_{1\leq i\leq r}\nu_{i},\qquad \nu_{\max} \coloneqq \max_{1\leq i\leq r}\nu_{i},
	\label{eq:defn-wi-wmin-wmax-TC}
\end{equation}
where $\nu_{i}$ reflects the size of the rank-1 component $\bm{w}_{i}^{\star}\otimes \bm{w}_{i}^{\star}\otimes \bm{w}_{i}^{\star}$. The condition number of $\bm{T}^{\star}$ is then defined as $\kappa \coloneqq \nu_{\max} / \nu_{\min}$.

We shall also introduce several incoherence parameters as follows.
\begin{definition}
	Define the incoherence parameters of $\bm{T}^{\star}$ (cf.~\eqref{eq:defn-true-tensor}) as
	\begin{equation}
		\mu_{1}\coloneqq\max_{1\leq i\leq r}\frac{n\|\bm{w}_{i}^{\star}\|_{\infty}^{2}}{\|\bm{w}_{i}^{\star}\|_{2}^{2}},
		\quad \text{and} \quad
		\mu_{2}\coloneqq\max_{i\neq j}\frac{n\big|\langle\bm{w}_{i}^{\star},\bm{w}_{j}^{\star}\rangle\big|^{2}}{\|\bm{w}_{i}^{\star}\|_{2}^{2} \, \|\bm{w}_{j}^{\star}\|_{2}^{2}} .
		\label{eq:defn-incoherence-TC}
	\end{equation}
\end{definition}
Let us explain these parameters in words: small $\mu_1$ and $\mu_2$ reflect that (i) the energy of each tensor factor $\bm{w}_i^{\star}$ is spread out across different entries, and (ii) the factors $\{\bm{w}_i^{\star}\}$  are not too correlated with each other. To simplify presentation, we set
\begin{align*}
	\mu \coloneqq \max \{ \mu_1, \mu_2 \}.
\end{align*}

\subsection{Algorithm}
\label{sec:algorithm-TC}

Unfortunately, it is notoriously difficult to exploit the low-rank structure---and many other low-complexity structure---efficiently in the original tensor space \citep{hillar2013most}.
To circumvent this issue,
a natural strategy thus attempts to matricize the tensor data, followed by an application of suitable low-rank matrix estimation algorithms.   Specifically, let us unfold the tensor $\bm{T}^{\star}$ into an $n\times n^2$ matrix $\bm{A}^{\star}$ as follows
\begin{align}
	\bm{A}^{\star}\coloneqq \mathsf{unfold}\big( \bm{T}^{\star}\big)
	= \sum_{i=1}^{r}\bm{w}_{i}^{\star}\big(\bm{w}_{i}^{\star}\otimes\bm{w}_{i}^{\star}\big)^{\top} \in \mathbb{R}^{n\times n^2}.
	\label{eq:defn-Astar-TC}	
\end{align}
The resulting matrix $\bm{A}^{\star}$ inherits the low-rank structure, as it clearly has rank at most $r$. We shall also matricize the observed data as
\begin{align}
	\bm{A} \coloneqq \mathsf{unfold} ( \bm{T} ) .
	\label{eq:defn-A-TC}
\end{align}

In order to estimate the subspace $\bm{U}^{\star}$ spanned by $\{\bm{w}_i^{\star}\}_{1\leq i\leq r}$ (which is  the column space of $\bm{A}^{\star}$ as well), the spectral method studied here resorts to the rescaled Gram matrix $p^{-2} \bm{A} \bm{A}^{\top}$. As a sanity check, if there is absolutely no missing data (i.e., $p=1$), then $p^{-2} \bm{A} \bm{A}^{\top}$ reduces to $\bm{A}^{\star} \bm{A}^{\star\top}$, whose column space coincides with that of $\bm{A}^{\star}$.
Turning to the scenario with missing data, a close inspection reveals that
\begin{align}
	\frac{1}{p^{2}}\mathbb{E}\big[\bm{A}\bm{A}^{\top}\big]
	=  \bm{A}^{\star}\bm{A}^{\star\top}+ \Big(\frac{1}{p}-1\Big) \mathcal{P}_{\mathsf{diag}} \big(\bm{A}^{\star}\bm{A}^{\star\top}\big),
	\label{eq:expectation-gram-matrix-TC}
\end{align}
where $\mathcal{P}_{\mathsf{diag}} (\cdot)$ denotes the Euclidean projection onto the set of matrices with zero off-diagonal entries.
This, however, makes apparent a severe issue: in the highly subsampled regime (i.e., where $p$ is small), the diagonal components might be excessively large and non-identical, thus destroying the low-rank structure in  \eqref{eq:expectation-gram-matrix-TC}.

To mitigate their undesirable effects, it is advisable to properly adjust the sizes of the diagonal entries \citep{montanari2016spectral,cai2019subspace}.
As it turns out, a simple yet plausible scheme is {\em diagonal deletion}, which exploits only the off-diagonal part as follows
\begin{subequations}
\begin{align}
	\bm{M} \coloneqq \frac{1}{p^{2}} \mathcal{P}_{\mathsf{off}\text{-}\mathsf{diag}}\big(\bm{A}\bm{A}^{\top}\big) .
	%\underset{\text{an off-diagonal matrix}}{\underbrace{\frac{1}{p^{2}}\Big(\bm{A}\bm{A}^{\top}-\mathcal{P}_{\mathsf{diag}}\big(\bm{A}\bm{A}^{\top}\big)\Big)}} .
	% + \underset{\text{a diagonal matrix}}{\underbrace{\frac{1}{p}\mathcal{P}_{\mathsf{diag}}\big(\bm{A}\bm{A}^{\top}\big)}}.	
	\label{defn:gram-matrix-TC}
\end{align}
Here,  $\mathcal{P}_{\mathsf{off}\text{-}\mathsf{diag}}(\cdot)$ stands for the operator that zeros out all diagonal entries of a matrix.
%which zeros out all diagonal entries.
% In words, we zero out the diagonal part of $\bm{A}\bm{A}^{\top}$ so as to mitigate its negative influence.
% the off-diagonal part and the diagonal part of $\bm{A}\bm{A}^{\top}$ are rescaled with different scaling factors. This careful scaling scheme ensures that unbiasedness of $\bm{M}$ in the sense that
%
One can easily verify that, in expectation,
\begin{equation}
	\mathbb{E}[\bm{M}]=   \underset{\eqqcolon\, \bm{M}^{\star} }{\underbrace{  \bm{A}^{\star}\bm{A}^{\star\top} }} - \mathcal{P}_{\mathsf{diag}}\big( \bm{A}^{\star}\bm{A}^{\star\top} \big),
	\label{eq:defn-Mstar-TC}
\end{equation}
\end{subequations}
which stays quite close to the low-rank matrix $\bm{M}^{\star}$ as long as the diagonal entries of $\bm{M}^{\star}$ are small enough.
The spectral method then proceeds by calculating the top-$r$ eigendecomposition $\bm{U}\bm{\Lambda}\bm{U}^{\top}$ of $\bm{M}$ and returning $\bm{U}$ as the subspace estimate. Here, the columns of $\bm{U}\in \mathbb{R}^{n\times r}$ are formed by the $r$ leading eigenvectors of $\bm{M}$, while $\bm{\Lambda}\in \mathbb{R}^{r\times r}$ is a diagonal matrix containing the $r$ leading eigenvalues.

\begin{remark}
	The diagonal deletion idea has been recommended not just for tensor completion, but also for problems including but not limited to bi-clustering \citep{florescu2016spectral}, PCA with missing data and/or heteroskedastic noise  \citep{cai2019subspace,abbe2020ell_p}, and contextual community detection \citep{abbe2020ell_p}.  Instead of diagonal deletion, one might also consider properly rescaling the diagonal entries
	based on the sampling mechanism; see, e.g., \citet{montanari2016spectral,lounici2014high,loh2012high,zhang2018heteroskedastic,zhu2019high}. 
\end{remark}

\subsection{Performance guarantees}
\label{sec:theory-TC}

The aforementioned spectral method can be analyzed by means of the $\ell_{2}$ perturbation theory as well.
As usual, this requires first
controlling the size of $\bm{E} \coloneqq \bm{M}-\bm{M}^{\star}$, where $\bm{M}$ and $\bm{M}^{\star}$ are defined in \eqref{defn:gram-matrix-TC} and \eqref{eq:defn-Mstar-TC}, respectively.
\begin{lemma}
	\label{lem:size-E-TC}
	Consider the settings in Section~\ref{sec:setup-tensor-completion}. There exists some universal constant $C>0$ such that with probability at least $1-O(n^{-7})$,
	\begin{align}
		\|\bm{E}\| \leq C  \Bigg( \frac{\mu^{3/2}r\sqrt{\log n}}{n^{3/2}p} + \sqrt{\frac{\mu^{2}r\log n}{n^{2}p}}  + \frac{\mu r}{n} \Bigg)\nu_{\max}^{2},
	\end{align}
	provided that $p\gtrsim \frac{\mu^{3/2}r\log^{2.5}n}{n^{3/2}}$ and that $\mu \max\{\log n, r^2\kappa^4 \} \leq c_3 n$ for some sufficiently small constant $c_3>0$.
\end{lemma}

In order to apply the Davis-Kahan $\sin\bm{\Theta}$ theorem (cf.~Corollary~\ref{cor:davis-kahan-conclusion-corollary}), another step boils down to characterizing the eigengap of the matrix $\bm{M}^{\star}=\bm{A}^{\star}\bm{A}^{\star\top}$ of interest. Our result is this:
\begin{lemma}
	\label{lemma:spectrum-Astar-AstarT-TC}
	Suppose that $\mu r^{2}\kappa^{4}\leq c_{3}n$ for some sufficiently small constant $c_3>0$. Then the $i$-th largest eigenvalue of $\bm{A}^{\star}\bm{A}^{\star\top}$ obeys
	\begin{align*}
		\lambda_{i}\big(\bm{A}^{\star}\bm{A}^{\star\top}\big) & \in\ \big[\nu_{\min}^{2}/2, 2\nu_{\max}^2 \big], \qquad \text{if }1\leq i\leq r; \\
		\lambda_{i}\big(\bm{A}^{\star}\bm{A}^{\star\top}\big) & =0,\qquad\qquad\qquad\qquad ~\text{if }i\geq r+1.
	\end{align*}
\end{lemma}
The preceding two lemmas, which will be established in Section~\ref{sec:proof-auxiliary-lemmas-TC}, readily lead to the following statistical guarantees for the spectral method presented in Section~\ref{sec:algorithm-TC}.
\begin{theorem}
	\label{thm:performance-TC-l2}
	Consider the settings in Section~\ref{sec:setup-tensor-completion}. Suppose that
	\begin{align}
			\mu r^{2}\kappa^{4}\log n \leq c_4 n
			\qquad \text{and} \qquad
			p\geq c_{5} \frac{\mu^{3/2}\kappa^{2}r \log^{2.5} n} {n^{3/2}}
			% \max \Bigg\{
			%,\frac{\mu^{2}\kappa^{4}r\log n}{n^{2}}\Bigg\}
			\label{eq:condition-TC}
	\end{align}
	hold for some small (resp.~large) enough constant $c_4>0$ (resp.~$c_5>0$). Then with probability at least $1-O(n^{-7})$, one has
	\[
		\mathsf{dist}\big(\bm{U},\bm{U}^{\star}\big)
		\lesssim \frac{\mu^{3/2}\kappa^{2}r\sqrt{\log n}}{n^{3/2}p}+\sqrt{\frac{\mu^{2}\kappa^{4}r\log n}{n^{2}p}}+\frac{\mu\kappa^{2}r}{n} .
	\]
\end{theorem}
\begin{proof}
In view of Lemmas~\ref{lem:size-E-TC}-\ref{lemma:spectrum-Astar-AstarT-TC}, one would have $\|\bm{E}\|\leq (1-1/\sqrt{2}) \lambda_{r} (\bm{M}^{\star} )$ under Condition \eqref{eq:condition-TC}.
%
%\[
%	\mu r\kappa^{2}\leq c_4n, \quad
%	p\geq c_{5}\max\Bigg\{\frac{\mu^{3/2}\kappa^{2}r\sqrt{\log n}}{n^{3/2}},\frac{\mu^{2}\kappa^{4}r\log n}{n^{2}}\Bigg\}
%\]
%
%hold for some small (resp.~large) enough constant $c_4>0$ (resp.~$c_5>0$).
%
Corollary~\ref{cor:davis-kahan-conclusion-corollary} combined with Lemma~\ref{lem:size-E-TC} then tells us that, with probability at least $1-O(n^{-7})$,
\[
	\mathsf{dist}\big(\bm{U},\bm{U}^{\star}\big)
	\leq \frac{2\big\|\bm{E}\big\|}{\lambda_{r}(\bm{M}^{\star})}\lesssim\frac{\Big(\frac{\mu^{3/2}r\sqrt{\log n}}{n^{3/2}p}+\sqrt{\frac{\mu^{2}r\log n}{n^{2}p}}+\frac{\mu r}{n}\Big)\nu_{\max}^{2}}{\nu_{\min}^{2}}
\]
as desired. \end{proof}

Theorem~\ref{thm:performance-TC-l2} is noteworthy for its implication on the sample complexity. To be precise, consider, for simplicity, the scenario where $r,\mu,\kappa = O(1)$.
In order to achieve consistent estimation in the sense that $\mathsf{dist}\big(\bm{U},\bm{U}^{\star}\big)=o(1)$, it suffices for the sample size---which sharply concentrates around $n^3p$ under our model---to exceed
\[
	n^{3}p\gtrsim n^{3/2}\mathrm{poly}\log(n).
\]

The careful reader might immediately remark that this sample complexity remains substantially higher than the information-theoretic limit,
the latter of which is $nr=O(n)$ in this case since there are only $nr$ free parameters. It is worth noting, however, that all polynomial-time algorithms developed in the literature for tensor completion require a sample size at least exceeding the order of $n^{3/2}$ \citep{barak2016noisy}. This hints at the (potential) existence of a computational barrier that prevents one from achieving the information-theoretic limit efficiently. Viewed in this light, the spectral method presented herein already achieves near-optimal sample complexity---when restricted to  computationally tractable algorithms---if the objective is consistent subspace estimation.

\subsection{Proof of auxiliary lemmas}
\label{sec:proof-auxiliary-lemmas-TC}

\paragraph{Proof of Lemma~\ref{lem:size-E-TC}.}
Define the following zero-mean random matrix
\[
	\bm{Z}=p^{-1}\bm{A}-\bm{A}^{\star}.
\]
It is  self-evident that
\[
  p^{-2} \big(\bm{A}\bm{A}^{\top}-\mathbb{E}\big[\bm{A}\bm{A}^{\top}\big]\big)=\bm{A}^{\star}\bm{Z}^{\top}+\bm{Z}\bm{A}^{\star\top}
+ \big( \bm{Z}\bm{Z}^{\top}-\mathbb{E}\big[\bm{Z}\bm{Z}^{\top}\big] \big),
\]
which implies that the identity holds for the off-diagonal part.  By the definitions
\eqref{defn:gram-matrix-TC} and \eqref{eq:defn-Mstar-TC}, it follows from the triangle inequality that
\begin{align}
\big\|\bm{M}-\bm{M}^{\star}\big\| & \leq\big\|\mathcal{P}_{\mathsf{diag}}\big(\bm{A}^{\star}\bm{A}^{\star\top}\big)\big\|+2\big\|\mathcal{P}_{\mathsf{off}\text{-}\mathsf{diag}}\big(\bm{A}^{\star}\bm{Z}^{\top}\big)\big\|\nonumber \\
	& \qquad+\big\|\mathcal{P}_{\mathsf{off}\text{-}\mathsf{diag}}\big(\bm{Z}\bm{Z}^{\top} -\mathbb{E}\big[\bm{Z}\bm{Z}^{\top} \big] \big)\big\|.
	\label{eq:decompose-M-Mstar-Z-Astar-TC}
\end{align}
In the sequel, we shall discuss how to control the three terms on the right-hand side of \eqref{eq:decompose-M-Mstar-Z-Astar-TC} separately.

\bigskip\noindent
{\em Step 1: bounding $\|\mathcal{P}_{\mathsf{diag}} (\bm{A}^{\star}\bm{A}^{\star\top})\|$.}
It is straightforward to verify that
%the term $\mathcal{P}_{\mathsf{diag}} (\bm{A}^{\star}\bm{A}^{\star\top}),$
% whose spectral norm obeys
%
\begin{equation}
\big\|\mathcal{P}_{\mathsf{diag}}\big(\bm{A}^{\star}\bm{A}^{\star\top}\big)\big\|=\max_{1\leq l\leq n}\big\|\bm{A}_{l,\cdot}^{\star}\big\|_{2}^{2}=\big\|\bm{A}^{\star}\big\|_{2,\infty}^{2}.
	\label{eq:P-diag-AAT-bound-TC}
\end{equation}
It thus suffices to  bound $\|\bm{A}^{\star}\|_{2,\infty}$, which we shall discuss momentarily.

%Regarding $\|\bm{A}^{\star}\bm{Z}^{\top}\|$, let us write
%%
%\[
%\bm{A}^{\star}\bm{Z}^{\top}=\sum_{i=1}^{n^{2}}\bm{A}_{\cdot,i}^{\star}\big(\bm{Z}_{\cdot,i}\big)^{\top},
%\]
%%
%which is a sum of independent rank-1 matrices and can be controlled via the matrix Bernstein inequality. Towards this,
%

\bigskip\noindent
{\em Step 2: bounding $\|\mathcal{P}_{\mathsf{off}\text{-}\mathsf{diag}} (\bm{Z}\bm{Z}^{\top} -\mathbb{E}[\bm{Z}\bm{Z}^{\top} ] ) \|$.}
 Define a collection of independent {\em zero-mean} random matrices as follows
\[
	\bm{Q}_{i}\coloneqq  \mathcal{P}_{\mathsf{off}\text{-}\mathsf{diag}}\big(\bm{Z}_{\cdot,i}\bm{Z}_{\cdot,i}^{\top}\big),
	\qquad 1\leq i\leq n^2,
\]
with which we can express
\begin{align}
	\mathcal{P}_{\mathsf{off}\text{-}\mathsf{diag}}\big(\bm{Z}\bm{Z}^{\top} -\mathbb{E}\big[\bm{Z}\bm{Z}^{\top}\big] \big)
	= \sum\nolimits_i  \Big( \bm{Q}_{i} - \mathbb{E}\big[ \bm{Q}_{i} \big] \Big) = \sum\nolimits_i  \bm{Q}_{i} .
	\label{eq:P-offdiag-ZZ-TC}
\end{align}
Here the last relation uses the fact that $\mathcal{P}_{\mathsf{off}\text{-}\mathsf{diag}} (\mathbb{E}[\bm{Z}\bm{Z}^{\top} ] ) = \sum_{i}\mathbb{E}\big[ \bm{Q}_{i} \big] = \bm{0}$.
Recognizing that the entries of $\bm{Z}_{\cdot,i}$ are independently
generated, one can see from straightforward calculations that $\mathbb{E}\big[\bm{Q}_{i}\bm{Q}_{i}^{\top}\big]$
is a diagonal matrix, whose diagonal entries satisfy
\begin{align*}
\Big(\mathbb{E}\big[\bm{Q}_{i}\bm{Q}_{i}^{\top}\big]\Big)_{l,l} & =\mathbb{E}\big[Z_{l,i}^{2}\big]\sum_{j:j\neq l}\mathbb{E}\big[Z_{j,i}^{2}\big]
  = \frac{1-p}{p}\big(A_{l,i}^{\star}\big)^{2}\sum_{j:j\neq l}\frac{1-p}{p}\big(A_{j,i}^{\star}\big)^{2}\\
 & \leq\frac{1}{p^{2}}\big(A_{l,i}^{\star}\big)^{2}\big\|\bm{A}_{\cdot,i}^{\star}\big\|_{2}^{2}
	\leq\frac{1}{p^{2}}\big(A_{l,i}^{\star}\big)^{2}\big\|\bm{A}^{\star}\big\|_{\infty,2}^{2}
\end{align*}
for all $1\leq l\leq n$.  Taking into account all samples yields
\[
	\sum_{i=1}^{n^2}\Big(\mathbb{E}\big[\bm{Q}_{i}\bm{Q}_{i}^{\top}\big]\Big)_{l,l}\leq\frac{1}{p^{2}}\sum_{i=1}^{n^{2}}\big(A_{l,i}^{\star}\big)^{2}\big\|\bm{A}^{\star}\big\|_{\infty,2}^{2}\leq\frac{1}{p^{2}}\big\|\bm{A}^{\star}\big\|_{2,\infty}^{2}\big\|\bm{A}^{\star}\big\|_{\infty,2}^{2}
\]
for any $1\leq l\leq n$,  which together with the diagonal structure of $\mathbb{E}\big[\bm{Q}_{i}\bm{Q}_{i}^{\top}\big]$ leads to an upper bound on the variance statistic
\begin{align*}
	v  \coloneqq\Big\|\sum_{i}\mathbb{E}\big[\bm{Q}_{i}\bm{Q}_{i}^{\top}\big]\Big\|
	= \Big|\max_{l}\Big(\sum_{i}\Big(\mathbb{E}\big[\bm{Q}_{i}\bm{Q}_{i}^{\top}\big]\Big)_{l,l}\Big)\Big|
	\leq\frac{  \big\|\bm{A}^{\star}\big\|_{2,\infty}^{2}\big\|\bm{A}^{\star}\big\|_{\infty,2}^{2} }{p^{2}}.
\end{align*}
In addition, we identify a suitable truncation level and claim that
\begin{subequations}
	\label{eq:Qi-tail-bound-exp-bound-TC}
\begin{align}
\mathbb{P}\Big\{\big\|\bm{Q}_{i}\big\|_{2}\geq L\Big\} & \leq2n^{-7}\eqqcolon q_{0},\\
\Big\|\mathbb{E}\big[\bm{Q}_{i}\mathbbm1\big\{\big\|\bm{Q}_{i}\big\|_{2}\leq L\big\}\big]\Big\| & \leq4n^{-7}p^{-2}\big\|\bm{A}^{\star}\big\|_{\infty,2}^{2}\eqqcolon q_{1},	
\end{align}
\end{subequations}
where we define $L \coloneqq 2\big(4\sqrt{\frac{\log n}{p}}\big\|\bm{A}^{\star}\big\|_{\infty,2}+\frac{6\log n}{p}\big\|\bm{A}^{\star}\big\|_{\infty}\big)^2$.
Armed with these observations, the truncated matrix Bernstein inequality (see Corollary~\ref{thm:matrix-Bernstein-friendly-truncated}) taken together with \eqref{eq:P-offdiag-ZZ-TC} reveals that
\begin{align}
 & \big\|\mathcal{P}_{\mathsf{off}\text{-}\mathsf{diag}}\big(\bm{Z}\bm{Z}^{\top}-\mathbb{E}\big[\bm{Z}\bm{Z}^{\top}\big]\big)\big\|
   \lesssim\sqrt{v\log n}+L\log n+n^{2}q_{1} \notag\\
 &  \lesssim\frac{\sqrt{\log n}}{p}\big\|\bm{A}^{\star}\big\|_{2,\infty}\big\|\bm{A}^{\star}\big\|_{\infty,2}+ \frac{\log^{2}n}{p} \big\|\bm{A}^{\star}\big\|_{\infty,2}^{2}+\frac{\log^{3}n}{p^{2}}\big\|\bm{A}^{\star}\big\|_{\infty}^{2}
	\label{eq:P-off-diag-ZZt-intermediate-TC}
\end{align}
with probability  $1- O(n^{-7}) - nq_0=1- O(n^{-7})$, provided that $p\gtrsim n^{-5}$.

\bigskip\noindent
{\em Step 3: bounding $\|\mathcal{P}_{\mathsf{off}\text{-}\mathsf{diag}}(\bm{A}^{\star}\bm{Z}^{\top})\|$.} This term can be controlled in a similar fashion as $\|\mathcal{P}_{\mathsf{off}\text{-}\mathsf{diag}} (\bm{Z}\bm{Z}^{\top} -\mathbb{E}[\bm{Z}\bm{Z}^{\top} ] ) \|$.  We thus omit the details and only state the result as follows:
\begin{align}
	& \|\mathcal{P}_{\mathsf{off}\text{-}\mathsf{diag}}(\bm{A}^{\star}\bm{Z}^{\top})\|
	\lesssim
	\big\|\bm{A}^{\star}\big\|_{\infty,2}\big\|\bm{A}^{\star}\big\|\sqrt{\frac{\log n}{p}} \notag\\
	& \qquad
	+\sqrt{\frac{\log^{3}n}{p}}\big\|\bm{A}^{\star}\big\|_{\infty,2}^{2} + \|\bm{A}^{\star}\|_{\infty,2}\big\|\bm{A}^{\star}\big\|_{\infty}\frac{\log^{2}n}{p}
	\label{eq:P-offdiag-AZ-intermediate-TC}
\end{align}
holds with probability at least $1-O(n^{-7})$.

\bigskip\noindent
{\em Step 4:}
To finish up,  we are in need of bounding $\|\bm{A}^{\star}\|_{\infty,2}$, $\|\bm{A}^{\star}\|_{2,\infty}$ and $\|\bm{A}^{\star}\|_{\infty}$, which is accomplished in the following lemma.
\begin{lemma}
	\label{lem:size-W-Wlift-TC}
	Suppose that $\mu r^{2}\leq n$. Then one has
\begin{align*}
%	\big\|\bm{W}^{\star}\big\|&\leq\sqrt{2}v_{\max}^{1/3},\qquad\qquad  &\big\|\bm{W}^{\star}\big\|_{2,\infty} &\leq\sqrt{\frac{\mu r}{n}}v_{\max}^{1/3},  \\
	 \big\|\bm{A}^{\star}\big\|_{\infty} \leq\frac{\mu^{3/2}r\nu_{\max}}{n^{3/2}}  , ~~
	\big\|\bm{A}^{\star}\big\|_{\infty,2} \leq \frac{\mu \sqrt{2r}\nu_{\max}}{n} ,
	~~ \big\|\bm{A}^{\star}\big\|_{2,\infty}  \leq\sqrt{\frac{2\mu r}{n}}\nu_{\max} .
\end{align*}
\end{lemma}
Taking Lemma \ref{lem:size-W-Wlift-TC} collectively with \eqref{eq:P-diag-AAT-bound-TC}, \eqref{eq:P-off-diag-ZZt-intermediate-TC}, \eqref{eq:P-offdiag-AZ-intermediate-TC} and combining terms, we arrive at
\begin{align*}
\eqref{eq:P-diag-AAT-bound-TC} + \eqref{eq:P-off-diag-ZZt-intermediate-TC} + \eqref{eq:P-offdiag-AZ-intermediate-TC}
% & \lesssim\Bigg(\frac{\mu^{5/2}r^{3/2}\sqrt{\log n}}{n^{5/2}p}+\frac{\mu^{2}r\log^{2}n}{n^{2}p}+\frac{\mu^{3}r^{2}\log^{3}n}{n^{3}p^{2}}\Bigg)\nu_{\max}^{2}\\
	& \lesssim\Bigg(
	%\frac{\mu^{3}r^{2}\log^{3}n}{n^{3}p^{2}} +\frac{\mu^{2}r\log^{2}n}{n^{2}p}
	\frac{\mu^{3/2}r\sqrt{\log n}}{n^{3/2}p} + \sqrt{\frac{\mu^{2}r\log n}{n^{2}p}}  + \frac{\mu r}{n} \Bigg)\nu_{\max}^{2}
\end{align*}
with probability at least $1-O(n^{-7})$,
provided that $\mu \log n \leq n$ and $p\gtrsim \frac{\mu^{3/2}r\log^{2.5}n}{n^{3/2}}$. This taken together with \eqref{eq:decompose-M-Mstar-Z-Astar-TC} concludes the proof.

%Next, we remark that the second and the third terms on the right-hand side of \eqref{eq:decompose-M-Mstar-Z-Astar-TC} can be controlled via similar arguments. Here, we shall focus primarily on bounding  $\|\mathcal{P}_{\mathsf{off}\text{-}\mathsf{diag}} (\bm{Z}\bm{Z}^{\top} -\mathbb{E}[\bm{Z}\bm{Z}^{\top} ] ) \|$.

\bigskip\noindent
{\em Proof of the relation~\eqref{eq:Qi-tail-bound-exp-bound-TC}.}
We first make note of a connection between $\bm{Q}_{i}$ and $\bm{Z}_{\cdot,i}$ as follows
\begin{equation}
	\big\|\bm{Q}_{i}\big\|\leq\big\|\bm{Z}_{\cdot,i}\bm{Z}_{\cdot,i}^{\top}\big\|+\big\|\mathcal{P}_{\mathsf{diag}}\big(\bm{Z}_{\cdot,i}\bm{Z}_{\cdot,i}^{\top}\big)\big\|\leq2\big\|\bm{Z}_{\cdot,i}\big\|_{2}^{2},
	\label{eq:Qi-Zi-relation-TC}
\end{equation}
which motivates us to first control the size of $\bm{Z}_{\cdot,i}$.
By construction, each entry $Z_{j,i}$ can be written as $Z_{j,i}= \big(\frac{1}{p} \delta_{j,i} -1\big) A_{j,i}^{\star}$, where $\{\delta_{j,i}\}$ is a collection of independent Bernoulli random variables with mean $p$. This observation allows one to derive  % invoke Lemma~\ref{lem:size-W-Wlift-TC} to
\begin{align*}
B_{z} & \coloneqq\max_{i,j}\big|Z_{j,i}\big|\leq \frac{1}{p} \big\|\bm{A}^{\star}\big\|_{\infty} ;\\
	%\leq\frac{\mu^{3/2}r}{pn^{3/2}}\nu_{\max};\\
v_{z} & \coloneqq\mathbb{E}\Big[\big\|\bm{Z}_{\cdot,i}\big\|_{2}^{2}\Big] = \frac{1-p}{p}\sum\nolimits_{j}\big(A_{j,i}^{\star}\big)^{2} \leq\frac{1}{p}\big\|\bm{A}^{\star}\big\|_{\infty,2}^{2} .
	%\leq\frac{2\mu^{2}r}{pn^{2}}\nu_{\max}^{2}.
\end{align*}
The matrix Bernstein inequality (see Corollary~\ref{thm:matrix-Bernstein-friendly})  then yields
\begin{align}
%\big\|\bm{Z}_{\cdot,i}\big\|_{2} & \leq4\sqrt{v_{z}\log n}+6B_{z}\log n \nonumber\\
% & \leq\Bigg(4\sqrt{\frac{2\mu^{2}r\log n}{pn^{2}}}+\frac{6\mu^{3/2}r\log n}{pn^{3/2}}\Bigg)\nu_{\max}
%	\eqqcolon\beta_{z}
\big\|\bm{Z}_{\cdot,i}\big\|_{2} & \leq4\sqrt{v_{z}\log n}+6B_{z}\log n\nonumber\\
 & \leq \Bigg(4\sqrt{\frac{\log n}{p}}\big\|\bm{A}^{\star}\big\|_{\infty,2}+\frac{6\log n}{p}\big\|\bm{A}^{\star}\big\|_{\infty}\Bigg) \eqqcolon\beta_{z}
\label{eq:Zi-UB-betaz-TC}
\end{align}
with probability at least $1-2n^{-7}$, which combined with \eqref{eq:Qi-Zi-relation-TC} gives
\begin{align}
	\mathbb{P}\Big\{ \big\|\bm{Q}_{i}\big\|_{2}  \geq 2\beta_{z}^2 \Big\} \leq 2n^{-7}.
\label{eq:Qi-Zi-UB-betaz-TC}
\end{align}
Recalling that $\mathbb{E}[\bm{Q}_i]=\bm{0}$, one can derive
\begin{align*}
 & \Big\|\mathbb{E}\big[\bm{Q}_{i}\mathbbm1\big\{\big\|\bm{Q}_{i}\big\|_{2}\leq2\beta_{z}^{2}\big\}\big]\Big\|=\Big\|\mathbb{E}\big[\bm{Q}_{i}\big]-\mathbb{E}\big[\bm{Q}_{i}\mathbbm1\big\{\big\|\bm{Q}_{i}\big\|_{2}>2\beta_{z}^{2}\big\}\big]\Big\|\\
	& \quad=\Big\|\mathbb{E}\big[\bm{Q}_{i}\mathbbm1\big\{\big\|\bm{Q}_{i}\big\|_{2}>2\beta_{z}^{2}\big\}\big]\Big\|\overset{(\mathrm{i})}{\leq}\mathbb{P}\big\{\big\|\bm{Q}_{i}\big\|_{2}>2\beta_{z}^2 \big\}\cdot  \frac{2}{p^{2}}\big\|\bm{A}^{\star}_{\cdot,i}\big\|_{2}^{2} \\
 & \quad\overset{(\mathrm{ii})}{\leq}  \frac{4}{n^7p^{2}}\big\|\bm{A}^{\star}\big\|_{\infty,2}^{2} .
	%\overset{(\mathrm{iii})}{\leq}\frac{4\mu^{2}r}{n^{9}p^{2}}\nu_{\max}^{2}\overset{(\mathrm{iv})}{\leq}\frac{1}{n^{8}p^2}\nu_{\max}^{2}.
\end{align*}
Here, (i) relies on \eqref{eq:Qi-Zi-relation-TC} and the fact $\|\bm{Z}_{\cdot,i}\|_2\leq p^{-1}\|\bm{A}^{\star}_{\cdot,i}\|_2$ (by construction), whereas (ii) results from the calculation in \eqref{eq:Zi-UB-betaz-TC}.
% (iii) is a consequence of Lemma~\ref{lem:size-W-Wlift-TC}, whereas (iv) holds provided that $4\mu^2 r \leq n$.

\paragraph{Proof of Lemma~\ref{lemma:spectrum-Astar-AstarT-TC}.}
Define the normalized tensor factors as 
\[
	\overline{\bm{w}}_{i}^{\star}:=\bm{w}_{i}^{\star}/\left\Vert \bm{w}_{i}^{\star}\right\Vert _{2}
	\qquad (1\leq i\leq r),
\]
and it is convenient to introduce the following auxiliary matrices that contain information about them:
\begin{align*}
\overline{\bm{W}}^{\star}\coloneqq \left[\overline{\bm{w}}_{1}^{\star},\cdots,\overline{\bm{w}}_{r}^{\star}\right],\qquad\overline{\bm{W}}_{\mathsf{lift}}^{\star}\coloneqq \left[\overline{\bm{w}}_{1}^{\star}\otimes\overline{\bm{w}}_{1}^{\star},\cdots,\overline{\bm{w}}_{r}^{\star}\otimes\overline{\bm{w}}_{r}^{\star}\right].
\end{align*}
Additionally, we introduce a diagonal matrix $\bm{D}^{\star}\in\mathbb{R}^{r\times r}$ whose diagonal entries are given by
\[
	\big[ \bm{D}^{\star} \big]_{i,i} =\big\| \bm{w}_{i}^{\star} \big\|_{2}^3 = v_{i},\qquad1\leq i\leq r.
\]
The matrices introduced above allow one to express $\bm{A}^{\star} =\overline{\bm{W}}^{\star} \bm{D}^{\star}  \overline{\bm{W}}^{\star\top}_{\mathsf{lift}} $ and $\bm{A}^{\star}\bm{A}^{\star\top}$ as follows
\begin{align}
	\label{defn:Astar-Astar-top-expression-TC}
%\bm{G}^{\star}=
	\bm{A}^{\star}\bm{A}^{\star\top}=\overline{\bm{W}}^{\star} \bm{D}^{\star}  \overline{\bm{W}}^{\star\top}_{\mathsf{lift}} \overline{\bm{W}}^{\star}_{\mathsf{lift}}   \bm{D}^{\star}   \overline{\bm{W}}^{\star\top}.
\end{align}
Clearly, the rank of $\bm{A}^{\star}\bm{A}^{\star\top}$ is bounded above by $r$, and hence it suffices to lower bound $\lambda_i(\bm{A}^{\star}\bm{A}^{\star\top})$ when $i\leq r$.

In order to characterize the spectrum of $\bm{A}^{\star}\bm{A}^{\star\top}$, we  first look at the eigenvalues of $\overline{\bm{W}}^{\star\top} \overline{\bm{W}}^{\star}$ and $\overline{\bm{W}}^{\star\top}_{\mathsf{lift}} \overline{\bm{W}}^{\star}_{\mathsf{lift}}$. Write
\begin{equation}
	\overline{\bm{W}}^{\star\top} \overline{\bm{W}}^{\star}=\bm{I}_{r}+ \bm{R}, 
	\qquad\text{and}\qquad
	\overline{\bm{W}}^{\star\top}_{\mathsf{lift}}\overline{\bm{W}}^{\star}_{\mathsf{lift}}
	=\bm{I}_{r}+{\bm{R}}_{\mathsf{lift}}
	\label{eq:W_top_residual_decomp}
\end{equation}
for some residual matrices $\bm{R},{\bm{R}}_{\mathsf{lift}}\in\mathbb{R}^{r\times r}$ (which are off-diagonal matrices). By virtue of the definition \eqref{eq:defn-incoherence-TC}, we immediately obtain
\begin{equation*}
	\left\Vert \bm{R}\right\Vert _{\infty}\leq\sqrt{\mu /n},
	\qquad\text{and}\qquad
	\big\|{\bm{R}}_{\mathsf{lift}}\big\|_{\infty}\leq\mu /n,
\end{equation*}
thus indicating that
\begin{equation}
	\| \bm{R}\| \leq r\left\Vert \bm{R}\right\Vert _{\infty}\leq r\sqrt{\mu /n},
	%\quad \text{and} \quad
	\quad ~~
	\big\|{\bm{R}}_{\mathsf{lift}} \big\|\leq r \, \big\|{\bm{R}}_{\mathsf{lift}}\big\|_{\infty}\leq\mu r /  n .
	\label{eq:spectral-norm-residual_C_UB}
\end{equation}
Putting these together with \eqref{eq:W_top_residual_decomp} and invoking Weyl's inequality give
\begin{align}
\label{eq:spectraum-WWstar-bar-TC}
\max_{i}\Big|\lambda_{i}\big(\overline{\bm{W}}^{\star\top}\overline{\bm{W}}^{\star}\big)-1\Big| & \leq\left\Vert \bm{R}\right\Vert \leq r\sqrt{\mu /n},
% \\
%\max_{i}\Big|\lambda_{i}\big(\overline{\bm{W}}_{\mathsf{lift}}^{\star\top}\overline{\bm{W}}_{\mathsf{lift}}^{\star}\big)-1\Big| & \leq\big\|\bm{R}_{\mathsf{lift}}\big\|\leq\mu_{2}r/n,
\end{align}
which together with the assumption $\mu r^{2}\leq n$ further reveals that
\begin{align}
\label{eq:W_op_UB}
\big\|\overline{\bm{W}}^{\star}\big\| & =\sqrt{\lambda_{1}\big(\overline{\bm{W}}^{\star\top}\overline{\bm{W}}^{\star}\big)}\leq\sqrt{1+r\sqrt{\mu /n}} \leq 2.
%\big\|\overline{\bm{W}}_{\mathsf{lift}}^{\star}\big\| & =\sqrt{\lambda_{1}\big(\overline{\bm{W}}_{\mathsf{lift}}^{\star\top}\overline{\bm{W}}_{\mathsf{lift}}^{\star}\big)}\leq\sqrt{1+\mu_{2}r/n} \leq 2.
\end{align}

We now return to study $\bm{A}^{\star}\bm{A}^{\star\top}$. In view of \eqref{defn:Astar-Astar-top-expression-TC} and \eqref{eq:W_top_residual_decomp}, one can decompose $\bm{A}^{\star}\bm{A}^{\star\top}$ into the following two terms
\begin{align}
	\bm{A}^{\star}\bm{A}^{\star\top}
	=  \underset{\eqqcolon\, \bm{G}_1 }{\underbrace{  \overline{\bm{W}}^{\star} \big( \bm{D}^{\star} \big)^2 \overline{\bm{W}}^{\star\top} }}
	+  \underset{\eqqcolon\, \bm{G}_2 }{\underbrace{ \overline{\bm{W}}^{\star}  \bm{D}^{\star}  {\bm{R}_{\mathsf{lift}}}  \bm{D}^{\star}   \overline{\bm{W}}^{\star\top} }}.
	\label{eq:decomposition-AAstar-G1-G2-TC}
\end{align}
Making use of the bounds \eqref{eq:spectral-norm-residual_C_UB} and \eqref{eq:W_op_UB} immediately leads to
\[
	\big\|\bm{G}_{2}\big\| \leq  \big\|\overline{\bm{W}}^{\star}\big\|^{2}\big\|\bm{D}^{\star}\big\|^{2}\big\|\bm{R}_{\mathsf{lift}}\big\|\leq 4\mu r \nu_{\max}^2 / n.
\]
Regarding $\bm{G}_1$, it can be directly seen that  the  non-zero eigenvalues of $\bm{G}_1$ coincide with those of $ \bm{D}^{\star}  \overline{\bm{W}}^{\star\top}\overline{\bm{W}}^{\star}  \bm{D}^{\star} $, where the latter can be decomposed into
\[
	 \bm{D}^{\star}  \overline{\bm{W}}^{\star\top}\overline{\bm{W}}^{\star}  \bm{D}^{\star}
	= (\bm{D}^{\star})^2 +  \bm{D}^{\star}  \bm{R}  \bm{D}^{\star} .
\]
As a result, for any $1\leq i\leq r$ one can derive
\begin{align*}
	\left|\lambda_{i}\big(\bm{G}_{1}\big)-\lambda_{i}\big( \big(\bm{D}^{\star}\big)^2 \big)\right| & =\left|\lambda_{i}\big(\bm{D}^{\star} \overline{\bm{W}}^{\star\top}\overline{\bm{W}}^{\star} \bm{D}^{\star} \Big)-\lambda_{i}\big(\big(\bm{D}^{\star}\big)^{2}\big)\right|\\
 & \leq \big\Vert  \bm{D}^{\star} \bm{R} \bm{D}^{\star} \big\Vert
   \leq \left\Vert \bm{D}^{\star}\right\Vert ^{2}\left\Vert \bm{R}\right\Vert \leq r\sqrt{\frac{\mu}{n}}\,\nu_{\max}^{2}.
\end{align*}
This taken together with the decomposition \eqref{eq:decomposition-AAstar-G1-G2-TC} leads to
\[
	\left|\lambda_{i}\big(\bm{A}^{\star}\bm{A}^{\star\top}\big)-\lambda_{i}\big(\bm{G}_{1}\big)\right|\leq\big\|\bm{G}_{2}\big\|\leq\frac{4\mu r\nu_{\max}^{2}}{n},
\]
thus indicating that
\begin{align*}
	& \left|\lambda_{i}\big(\bm{A}^{\star}\bm{A}^{\star\top}\big)-\lambda_{i}\big(\big(\bm{D}^{\star}\big)^{2}\big)\right| \\
	&\qquad\qquad \leq\left|\lambda_{i}\big(\bm{G}_{1}\big)-\lambda_{i}\big(\big(\bm{D}^{\star}\big)^{2}\big)\right|+\left|\lambda_{i}\big(\bm{A}^{\star}\bm{A}^{\star\top}\big)-\lambda_{i}\big(\bm{G}_{1}\big)\right|\\
 	& \qquad\qquad \leq r\sqrt{\frac{\mu}{n}}\,\nu_{\max}^{2}+\frac{4\mu r \nu_{\max}^{2}}{n}
	\leq  8 \max\Big\{ \sqrt{\frac{\mu}{n}}, \frac{\mu }{n} \Big\} r\nu_{\max}^{2}.
\end{align*}
If $16\max\{ \frac{\mu}{n}, \sqrt{\frac{\mu}{n}} \big\} r\nu_{\max}^{2}\leq \nu_{\min}^{2}$, then one has $\big|\lambda_{i}\big(\bm{A}^{\star}\bm{A}^{\star\top}\big)-\lambda_{i}\big(\big(\bm{D}^{\star}\big)^{2}\big)\big| \leq  \nu_{\min}^{2}/2$. In addition, letting $\nu_{(i)}$ be the $i$-th largest element in $\{\nu_i\}_{1\leq i\leq r}$, we have $\lambda_{i}\big( \big(\bm{D}^{\star}\big)^{2} \big) = \nu_{(i)}^2$ and hence arrive at
\begin{align*}
	%\lambda_{i}\big(\bm{A}^{\star}\bm{A}^{\star\top}\big) & \geq  \nu_{(i)}^2  -  \nu_{\min}^{2} / 2 \geq  \\
	\nu_{\min}^{2} / 2 \leq \nu_{(i)}^2  -  \nu_{\min}^{2} / 2 \leq
	\lambda_{i}\big(\bm{A}^{\star}\bm{A}^{\star\top}\big) & \leq  \nu_{(i)}^2  +  \nu_{\min}^{2} / 2 \leq 2 \nu_{\max}^{2}
\end{align*}
for any $1\leq i\leq r$, as claimed.

\paragraph{Proof of Lemma~\ref{lem:size-W-Wlift-TC}.}

Define the following two matrices containing information about the tensor factors:
\begin{subequations}
\begin{align}
\bm{W}^{\star} & \coloneqq\big[\bm{w}_{1}^{\star},\cdots,\bm{w}_{r}^{\star}\big]\in\mathbb{R}^{n\times r},   \label{eq:defn-Wstar_TC}\\
\bm{W}_{\mathsf{lift}}^{\star} & \coloneqq\big[\bm{w}_{1}^{\star}\otimes\bm{w}_{1}^{\star},\cdots,\bm{w}_{r}^{\star}\otimes\bm{w}_{r}^{\star}\big]\in\mathbb{R}^{n^{2}\times r}.
	\label{eq:defn-Wstar-lift-TC}
\end{align}
\end{subequations}
Given that $\bm{W}^{\star\top}\bm{W}^{\star}=\big[\big\langle\bm{w}_{i}^{\star},\bm{w}_{j}^{\star}\big\rangle\big]_{1\leq i,j\leq r}$, its diagonal part satisfies
\begin{align*}
	\Big\| \mathcal{P}_{\mathsf{diag}}\big(\bm{W}^{\star\top}\bm{W}^{\star}\big) \Big\|
	& = \Big\| \mathsf{diag}\Big(\big[\big\|\bm{w}_{i}^{\star}\big\|_{2}^{2}\big]_{1\leq i\leq r}\Big) \Big\|
	%\succeq  v_{\min}^2 \bm{I}
	%\min_{i}\big\|\bm{w}_{i}^{\star}\big\|_{2}^{2}\bm{I} =
	\leq \nu_{\max}^{2/3}.
\end{align*}
In addition, the off-diagonal part of $\bm{W}^{\star\top} \bm{W}^{\star}$ satisfies
\begin{align*}
\Big\|\mathcal{P}_{\mathsf{off}\text{-}\mathsf{diag}}\big( \bm{W}^{\star\top} \bm{W}^{\star}\big)\Big\|
	& \leq\sqrt{\sum_{i\neq j}\big|\big\langle\bm{w}_{i}^{\star},\bm{w}_{j}^{\star}\big\rangle\big|^{2}}\leq\sqrt{\frac{\mu}{n}\sum_{i\neq j}\big\|\bm{w}_{i}^{\star}\big\|_{2}^{2}\big\|\bm{w}_{j}^{\star}\big\|_{2}^{2}}\\
	& \leq \sqrt{\frac{\mu r^{2}}{n}} \max_i \big\|\bm{w}_i^{\star} \big\|_2^2 =  \nu_{\max}^{2/3} \sqrt{\frac{\mu r^{2}}{n}},
\end{align*}
where the second inequality   
%we denote $\mathcal{P}_{\mathsf{off}\text{-}\mathsf{diag}}(\bm{A}) \coloneqq \bm{A} - \mathcal{P}_{\mathsf{diag}}(\bm{A})$, and
 relies on the definition \eqref{eq:defn-incoherence-TC} of the incoherence parameter, and the last relation follows from the definition of $\nu_{\max}$ in \eqref{eq:defn-wi-wmin-wmax-TC}. 
 Consequently, if $\mu r^{2}\leq n$, then
\begin{align}
	\big\|\bm{W}^{\star}\big\|^2 & = \big\| \bm{W}^{\star\top} \bm{W}^{\star}\big\|
	\leq \Big\|\mathcal{P}_{\mathsf{diag}}\big( \bm{W}^{\star\top} \bm{W}^{\star}\big)\Big\|+\Big\|\mathcal{P}_{\mathsf{off}\text{-}\mathsf{diag}}\big( \bm{W}^{\star\top} \bm{W}^{\star}\big)\Big\| \nonumber\\
	& \leq \nu_{\max}^{2/3} + \nu_{\max}^{2/3}\sqrt{\frac{\mu r^{2}}{n}} \leq  2\nu_{\max}^{2/3}.
	\label{eq:Wsquare-norm-bound}
\end{align}
Repeating similar arguments also reveals that
\begin{align}
	\big\|\bm{W}^{\star}_{\mathsf{lift}}\big\|^2  \leq  2\nu_{\max}^{4/3}.
	\label{eq:Wsquare-lift-norm-bound}
\end{align}

Next, it is readily seen from the definition \eqref{eq:defn-incoherence-TC}  that
\begin{align*}
	\big\|\bm{W}^{\star}\big\|_{2,\infty} & \leq\sqrt{r}\max_{i}\|\bm{w}_{i}^{\star}\|_{\infty}\leq\sqrt{\frac{\mu r}{n}}\max_{i}\|\bm{w}_{i}^{\star}\|_{2}
	=\sqrt{\frac{\mu r}{n}} \nu_{\max}^{1/3},\\
\big\|\bm{W}_{\mathsf{lift}}^{\star}\big\|_{2,\infty} & \leq\sqrt{r}\max_{i}\|\bm{w}_{i}^{\star}\|_{\infty}^{2}\leq\frac{\mu \sqrt{r}}{n}\max_{i}\|\bm{w}_{i}^{\star}\|_{2}^{2}
	=\frac{\mu \sqrt{r}}{n} \nu_{\max}^{2/3}.
\end{align*}
Combining these bounds with \eqref{eq:Wsquare-norm-bound} and \eqref{eq:Wsquare-lift-norm-bound} immediately yields
\begin{align*}
\big\|\bm{A}^{\star}\big\|_{\infty,2} & =\Big\Vert\bm{W}^{\star}\big(\bm{W}_{\mathsf{lift}}^{\star}\big)^{\top}\Big\Vert_{\infty,2}\leq\|\bm{W}^{\star}\|\left\Vert \bm{W}_{\mathsf{lift}}^{\star}\right\Vert _{2,\infty}\leq\frac{\mu\sqrt{2r}}{n}\nu_{\max} ,\\
\big\|\bm{A}^{\star}\big\|_{\infty} & =\Big\Vert\bm{W}^{\star}\big(\bm{W}_{\mathsf{lift}}^{\star}\big)^{\top}\Big\Vert_{\infty}\leq\|\bm{W}^{\star}\|_{2,\infty}\left\Vert \bm{W}_{\mathsf{lift}}^{\star}\right\Vert _{2,\infty}\leq\frac{\mu^{3/2}r}{n^{3/2}}\nu_{\max} ,\\
\big\|\bm{A}^{\star}\big\|_{2,\infty} & =\Big\Vert\bm{W}^{\star}\big(\bm{W}_{\mathsf{lift}}^{\star}\big)^{\top}\Big\Vert_{2,\infty}\leq\|\bm{W}^{\star}\|_{2,\infty}\left\Vert \bm{W}_{\mathsf{lift}}^{\star}\right\Vert \leq\sqrt{\frac{2\mu r}{n}}\nu_{\max} .
\end{align*}
%
%as claimed.

%% file: chapters/Linf.tex
\chapter[Fine-grained analysis:  $\ell_{\infty}$ and $\ell_{2,\infty}$ perturbation theory]{Fine-grained spectral analysis: \\ $\ell_{\infty}$ and $\ell_{2,\infty}$ perturbation theory}
\label{cha:Linf-theory}

In a growing number of applications, the $\ell_2$-type distance between subspaces, which is the central subject studied in Chapter~\ref{chap:application-L2}, turns out to be inadequate for performance characterization. Rather, what would be of interest is the entrywise behavior of the eigenvector and the matrix under consideration. This is especially important when the individual entries of the eigenvector or the matrix of interest carry pivotal operational meanings. For example, in a recommendation system, one might be interested in controlling the prediction error of a user's preference on a specific product, which  concerns  a specific entry in a user-product rating matrix; in sensor network localization, one might seek to control the ranging error w.r.t.~a pair of sensors, which corresponds to  entrywise prediction errors in a Euclidean distance matrix; and last but not least, in community recovery, the entries of the leading eigenvector of a certain data matrix might encode the community membership associated with each individual (as explained in Section~\ref{sec:community-detection}).

Tackling the preceding applications calls for development of fine-grained spectral analysis beyond classical $\ell_2$ perturbation theory.  To be more precise, consider once again the observation model
$$
\bm{M}=\bm{M}^{\star} + \bm{E}
$$  previously studied in Chapter~\ref{cha:matrix-perturbation} (cf.~\eqref{eq:perturbed-M}). The sort of fine-grained theory being sought after gravitates around the following questions concerned with $\ell_{\infty}$ and/or $\ell_{2,\infty}$ perturbation:
\begin{itemize}
	\item For a symmetric matrix $\bm{M}^{\star}$, how to characterize the effect of $\bm{E}$ on the $\ell_{\infty}$ perturbation  of the leading eigenvector,
		or the $\ell_{2,\infty}$ perturbation of the rank-$r$ leading eigenspace?
	\item For a general matrix $\bm{M}^{\star}$, how to pin down the $\ell_{\infty}$ perturbation of the leading singular vector, or the $\ell_{2,\infty}$ perturbation  of the rank-$r$ leading singular subspace, in response to the perturbation $\bm{E}$?
	\item How to assess the entrywise estimation error of the matrix estimate produced by the spectral method, and how is it affected by $\bm{E}$?
\end{itemize}
Unfortunately, a direct application of classical $\ell_2$ perturbation theory typically leads to overly crude bounds when coping with the above questions.  In particular, when the $\ell_2$ error is approximately evenly distributed across entries, naively upper bounding the entrywise error by the $\ell_2$ error is often loose by an order-of-magnitude.
In order to conquer such limitations,
this chapter introduces a modern suite of techniques that delivers tight $\ell_\infty$ and $\ell_{2,\infty}$ error control by leveraging the statistical nature of data models.

\input{chapters/matrix_denoising_linf.tex}

\input{chapters/general_theory_linf.tex}

\input{chapters/mc_inf.tex}

\input{chapters/community_detection_Linf.tex}

\input{chapters/distribution_theory.tex}

%\begin{subappendices}

\input{chapters/analysis_Linf.tex}

\input{chapters/analysis_entrywise.tex}
\input{chapters/analysis_Linf_asym.tex}

\input{chapters/analysis_distribution.tex}

%\end{subappendices}

\section{Notes}

\paragraph{Leave-one-out analysis.} The core idea of leave-one-out analysis,  which drops a small amount of randomness to decouple complicated statistical dependency, is deeply rooted in the probability and statistics literature. For instance, an idea of this kind was invoked by \citet{stein1972a} to help establish normal approximation,   was paired with the Stieltjes transform to establish the  limiting spectral law of random matrices (see, e.g., \citep[Section 2.4.3]{Tao2012RMT}), and bears some resemblance to the cavity method in statistical physics \citep{mezard2009information}.
When it comes to statistical estimation, a prominent series of work that unveiled the striking effectiveness of leave-one-out analysis was \citet{el2013robust,el2015impact}, which characterized rigorously the sharp statistical performance (including pre-constants) of M-estimators in high dimension (i.e., a challenging regime where the number of samples is comparable to the number of unknown parameters).
The deep analysis framework developed in these papers inspired much of the follow-up work presented in this chapter. Particularly worth mentioning are: (1) \citet{zhong2017near}: which was the first to determine the entrywise behavior of the generalized projected power method; (2) \citet{abbe2020entrywise,chen2017spectral}: which extended the leave-one-out analysis idea to establish entrywise eigenvector perturbation; and (3) \citet{ma2017implicit,chen2019gradient}: which characterized tight convergence guarantees for nonconvex optimization algorithms with the aid of leave-one-out ideas. For readers' reference, we list below several topics for which leave-one-out analyses prove useful:
\begin{itemize}
	\item Maximum likelihood estimation and M-estimation: \citet{el2013robust,el2015impact,lei2018asymptotics,sur2019likelihood,sur2019modern,chen2017spectral,chen2020partial};
	\item spectral methods: \citet{chen2017spectral,abbe2020entrywise,ma2017implicit,cai2019subspace,lei2019unified,abbe2020ell_p,ling2020near,chen2020partial};
	\item nonconvex optimization for statistical estimation: \citet{ma2017implicit,chen2019gradient, li2019nonconvex,chen2020nonconvex,cai2019tensor,dong2018nonconvex,chen2020convex,wang2021entrywise};
	\item semidefinite relaxation for low-rank matrix factorization: \citet{zhong2017near,ding2020leave,chen2020noisy,chen2020bridging,chen2020convex};
	\item uncertainty quantification and  confidence intervals: \citet{javanmard2018debiasing,chen2019inference,cai2020uncertainty,yan2021inference};
	\item cross validation: \citet{xu2019consistent};
	\item reinforcement learning: \citet{agarwal2020model,li2020breaking,pananjady2020instance,zhang2020model,cui2020minimax,
		wang2021sample}.
\end{itemize}

\paragraph{$\ell_{\infty}$ and $\ell_{2,\infty}$  perturbation theory.}

The $\ell_{\infty}$ and $\ell_{2,\infty}$ perturbation theory for eigenspace and singular subspaces have been investigated in the literature \citep{fan2016ell,cape2019two,eldridge2018unperturbed}, but only scatteredly until very recently.
A modern and systematic framework was established recently,  empowered by the leave-one-out analysis idea.
Its efficacy and tightness  were first demonstrated by \citet{abbe2020entrywise} in a setting that subsumes the one presented herein (with applications to SBMs, phase synchronization and matrix completion), and  by \citet{chen2017spectral} in an asymmetric setting (with application to
 top-$K$ ranking).
 The theoretical framework has subsequently been extended in several aspects.
 For instance, (1) \citet{cai2019subspace} investigated the ``unbalanced'' scenario where the column dimension far exceeds the row dimension of the matrix, resulting in near-optimal fine-grained guarantees for PCA, bi-clustering and tensor completion; (2) \citet{lei2019unified}  expanded the setting by accounting for more flexible noise distributions (including the ones exhibiting certain dependency structure), leading to tight guarantees for, e.g., spectral clustering with more than two communities, and  hierarchical clustering;  (3)  \citet{abbe2020ell_p} explored a more general $\ell_{p}$ perturbation theory that subsumes the $\ell_{\infty}$ perturbation theory as special cases.
Another plausible  approach to study $\ell_{\infty}$ eigenvector perturbation
is to analyze instead the dynamics of an iterative procedure (e.g., the power method) that converges to the leading eigenvector \citep{zhong2017near}, again using the leave-one-out ideas.
This iterative approach offers a perspective complementary to the analysis framework presented herein,
while at the same time playing a pivotal role when studying nonconvex optimization algorithms (see \citet{chi2019nonconvex}).
Moving beyond the leave-one-out analysis framework, $\ell_{\infty}$ and $\ell_{2,\infty}$ eigenspace perturbation theory has been derived via other powerful tools as well, e.g., the Neumann trick \citep{eldridge2018unperturbed,chen2018asymmetry,cheng2020tackling,}, the Procrustes analysis \citep{cape2019two}, and more specialized techniques tailored to Gaussian ensembles \citep{koltchinskii2016perturbation,koltchinskii2016asymptotics,koltchinskii2020efficient}. Perturbations of linear forms and bilinear forms of eigenvectors have also been investigated  in the literature
\citep{koltchinskii2016perturbation,koltchinskii2016asymptotics,chen2018asymmetry,cheng2020tackling,fan2020asymptotic,koltchinskii2020efficient},
which are beyond the scope of this work.
Finally, distributional theory and uncertainty quantification for spectral methods, which
have been recently studied by \citet{xia2019normal, cheng2020tackling, fan2020asymptotic, yan2021inference,agterberg2021entrywise}, 
are still in their infancy. The results presented in Section~\ref{sec:distribution-theory} follow the analysis framework developed in \citet{yan2021inference}.

%% file: chapters/matrix_denoising_linf.tex
%\section{An illustrative example: Rank-1 matrix denoising}
\section{Leave-one-out analysis: An illustrative example}
\label{sec:rank-1-denoising-LOO}

To paint a high-level picture of the core ideas empowering the $\ell_{\infty}$ and $\ell_{2,\infty}$ analysis, we find it helpful to first look at a pedagogical example of rank-1 matrix denoising, a special case of the formulation introduced in Section~\ref{sec:matrix-denoising-setup}.

\subsection{Setup and algorithm}
\label{sec:setup-rank1-denoising}
Suppose that we observe a noisy copy of an unknown rank-1 matrix $\bm{M}^{\star}$ as follows
\begin{equation}
	\bm{M}=\bm{M}^{\star}+\bm{E}=\lambda^{\star}\bm{u}^{\star}\bm{u}^{\star\top}+\bm{E} \in \mathbb{R}^{n\times n},
	\label{eq:defn-matrix-denoising-rank1}
\end{equation}
where $\lambda^{\star}>0$ and $\bm{u}^{\star}\in \mathbb{R}^{n}$
represent the largest eigenvalue  of $\bm{M}^{\star}$ and its associated eigenvector, respectively.  We assume the Gaussian noise model as in Section~\ref{sec:matrix-denoising-setup}, namely, $\bm{E}$ is a symmetric matrix whose upper triangular part comprises of i.i.d.~entries drawn from $\mathcal{N}(0,\sigma^2)$.
 In addition, we remind the readers of the following incoherence parameter $\mu$  (cf.~Definition~\ref{assump:mc-incoherence}):
\begin{align}
	%\|\bm{u}^{\star}\|_{\infty}= \sqrt{\mu/n} = \sqrt{\mu/n} \,\|\bm{u}^{\star}\|_2,
	\mu = n\|\bm{u}^{\star}\|_{\infty}^{2},
	\label{defn:incoherence-rank-1-denoising}
\end{align}
which satisfies $1 \leq \mu \leq n$ in this rank-1 case; see Remark~\ref{remark:range-mu-MC}.

Letting $\lambda$  be the leading eigenvalue of $\bm{M}$ (i.e., $\lambda = \lambda_{1}(\bm{M})$) and $\bm{u}\in \mathbb{R}^{n}$
 the associated eigenvector, the spectral method attempts to estimate $\bm{u}^{\star}$ using $\bm{u}$. In this section, we are particularly interested in controlling the entrywise error, defined in terms of the $\ell_{\infty}$ distance (modulo the global sign):
 \begin{align}
	\mathsf{dist}_{\infty}\big(\bm{u},\bm{u}^{\star}\big)
	\coloneqq \min\big\{ \|\bm{u}-\bm{u}^{\star}\|_{\infty}, \|\bm{u}+\bm{u}^{\star}\|_{\infty} \big\}.
	\label{eq:defn-dist-infty-vector}
\end{align}

\subsection{$\ell_{\infty}$ performance guarantees}

While Section~\ref{sec:performance-denoising-L2} delivers $\ell_2$ estimation guarantees for the spectral estimate $\bm{u}$,
it falls short of characterizing the entrywise behavior---except for the crude and highly suboptimal bound  $\mathsf{dist}_{\infty}(\bm{u},\bm{u}^{\star})\leq \mathsf{dist}(\bm{u},\bm{u}^{\star})$. Encouragingly,  this simple spectral method is provably  accurate in an entrywise fashion, as revealed by the following theorem.
 \begin{theorem}
	 Consider the settings in Section~\ref{sec:setup-rank1-denoising}.
	 There exists some sufficiently small constant $c_0>0$ such that if $\sigma \sqrt{n} \leq c_0 \lambda^{\star}$, then
 \begin{align}
	 \mathsf{dist}_{\infty}\big(\bm{u},\bm{u}^{\star}\big)
	 \lesssim \frac{\sigma (\sqrt{\log n} + \sqrt{\mu} )}{\lambda^{\star}}
 \end{align}
	 holds with probability exceeding $1-O(n^{-8})$.
  \end{theorem}

In particular, if the incoherence parameter obeys $\mu \lesssim \log n$ (the case where no entries of $\bm{u}^{\star}$ are significantly larger in magnitude than the average magnitude), then our $\ell_{\infty}$ bound reads
\begin{align}
	 \mathsf{dist}_{\infty}\big(\bm{u},\bm{u}^{\star}\big)
	 \lesssim \frac{\sigma\sqrt{\log n}}{\lambda^{\star}} ,
 \end{align}
which is about $\sqrt{n/\log n}$ times smaller than the $\ell_2$ error bound \eqref{eq:dist-U-Ustar-denoising-L2}, that is, $\mathsf{dist}\big(\bm{u},\bm{u}^{\star}\big)   \lesssim \frac{\sigma\sqrt{ n}}{\lambda^{\star}}$.
This implies that the estimation errors of $\bm{u}$ are dispersed  more or less evenly across all entries---a message that is previously unavailable from classical $\ell_2$ perturbation theory.

\subsection{Key ingredient and intuition: Leave-one-out estimates}
\label{sec:LOO-estimates-denoising}

To facilitate entrywise analysis, a crucial ingredient lies in the introduction of a set of leave-one-out auxiliary estimates, detailed below.

\paragraph{Construction of leave-one-out estimates.} For each $1\leq l\leq n$, let us construct an auxiliary matrix  $\bm{M}^{(l)}$ as follows
\begin{align}
	\bm{M}^{(l)} \coloneqq \lambda^{\star}\bm{u}^{\star}\bm{u}^{\star\top}+\bm{E}^{(l)} ,
	\label{eq:defn-Ml-denoising}
\end{align}
where the noise matrix $\bm{E}^{(l)}$ is generated according to
\begin{equation}
E_{i,j}^{(l)} \coloneqq
\begin{cases}
E_{i,j},\qquad & \text{if }i\neq l\text{ and }j\neq l, \\
0, & \text{else}.
\end{cases}
\label{eq:Eij-l-defn}
\end{equation}
In words, $\bm{M}^{(l)}$ (resp.~$\bm{E}^{(l)}$) is obtained by leaving out the randomness in the $l$-th row/column of $\bm{M}$ (resp.~$\bm{E}$). The pattern of the leave-one-out construction is illustrated in Figure~\ref{fig:loo_illustration}.
Let $\lambda^{(l)}$ and $\bm{u}^{(l)}$ denote respectively the leading eigenvalue and leading eigenvector of $\bm{M}^{(l)}$; 
these leave-one-out estimates are introduced solely for analysis purpose. 
It is important to recognize that by construction, $\bm{M}^{(l)}$ (and hence $\bm{u}^{(l)}$) is independent of the noise in the $l$-th row/column of $\bm{E}$, a fact that plays a pivotal role in controlling the perturbation of the $l$-th entry of $\bm{u}^{\star}$.

\begin{figure}[t]
\begin{center}
\includegraphics[width=0.85\textwidth]{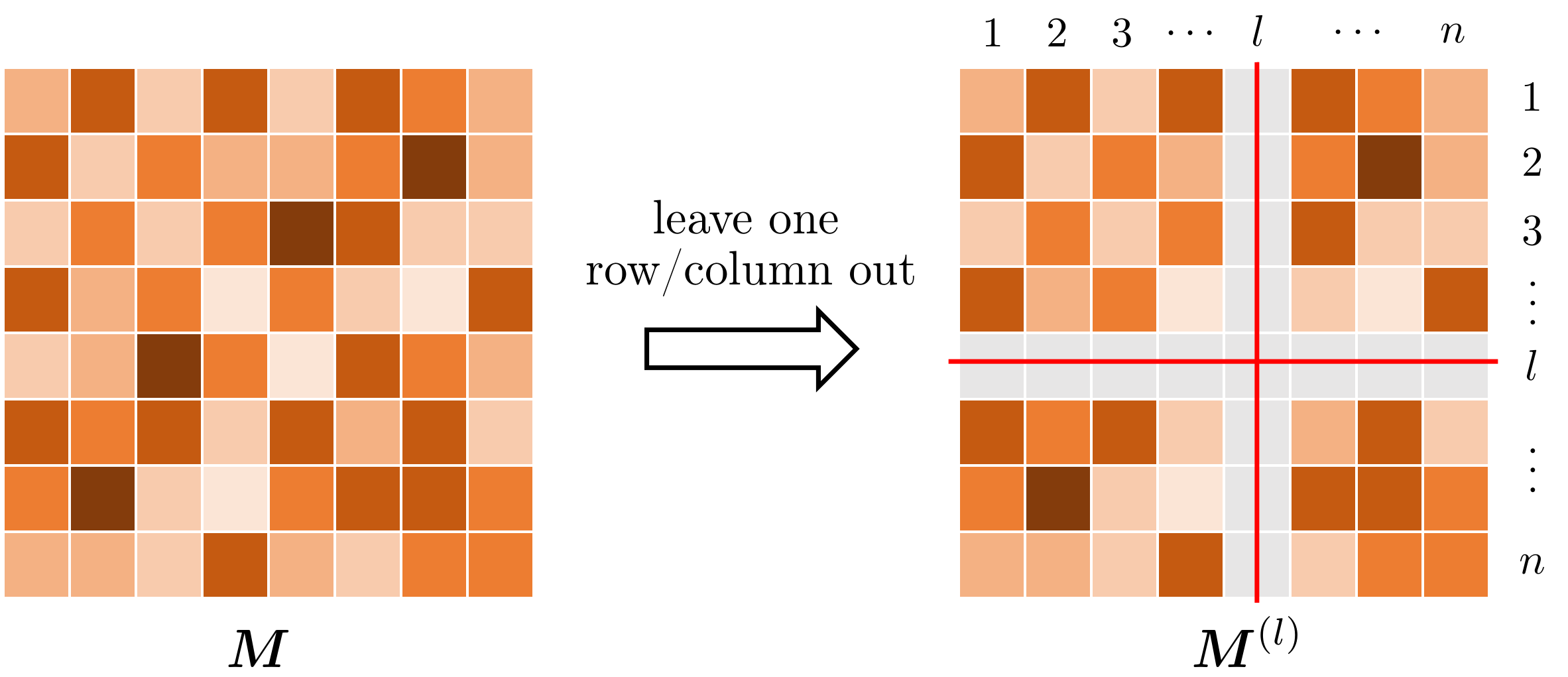}
\end{center}
\caption{Illustration of the leave-one-out auxiliary matrix $\bm{M}^{(l)}$, which removes all the noise in the $l$-th row and the $l$-th column of $\bm{M}$.} \label{fig:loo_illustration}
\end{figure}

\paragraph{Intuition.} Before delving into the proof, let us first explain the rationale at an intuitive level.
\begin{enumerate}
	\item Given that $\bm{u}^{(l)}$ is obtained by dropping only a tiny fraction of the data, we expect $\bm{u}^{(l)}$ to be exceedingly close to $\bm{u}$, i.e.,
	\begin{align}
		\bm{u} \approx \pm \bm{u}^{(l)}.
	\end{align}
	In words,  $\bm{u}^{(l)}$ forms a reliable surrogate of $\bm{u}$, which motivates us to analyze $\bm{u}^{(l)}$ instead (if there are foreseeable benefits to do so).

\item The way we construct $\bm{u}^{(l)}$ makes it particularly convenient to analyze the behavior of the $l$-th entry, denoted by $u_l^{(l)}$.
More specifically, given that $(\lambda^{(l)},\bm{u}^{(l)})$ is an eigenpair of $\bm{M}^{(l)}$, one has (assuming for the moment that $\lambda^{(l)} \neq 0$)
	\begin{align}
		u_{l}^{(l)} & =\frac{1}{\lambda^{(l)}}\bm{M}_{l,\cdot}^{(l)}\bm{u}^{(l)}=\frac{1}{\lambda^{(l)}}\bm{M}_{l,\cdot}^{\star}\bm{u}^{(l)}=\frac{\lambda^{\star}}{\lambda^{(l)}}u_{l}^{\star}\bm{u}^{\star\top}\bm{u}^{(l)}   \label{eq:ul_l-expression-denoising} \\
		& \approx \pm u_{l}^{\star}.
	\end{align}
	Here, the first line follows since, by design, the $l$-th rows of $\bm{M}^{(l)}$
		and $\bm{M}^{\star}$ coincide (both of which are given by $\lambda^{\star}u_{l}^{\star}\bm{u}^{\star\top}$), whereas the second line holds as long as the size $\sigma$ of the noise is sufficiently small, so that $ \lambda^{(l)} / \lambda^{\star} \approx 1$ and $\bm{u}^{\star\top}\bm{u}^{(l)}\approx \pm 1$ according to the $\ell_2$ perturbation theory.
	% where the second identity follows by construction of $\bm{M}^{(l)}$.
	
	%% this point below is justified more rigorously in the formal proof, and hence omitted from the intuition as it is more technical
	
%\item The above $\ell_2$-pertubation theory requires evaluation of $\|\bE^{(l)}\|$.
%	By construction, for any $n$-dimensional vector $\bx$ with restriction $x_l = 0$, we have $\bx^{\top} \bE \bx = \bx^{\top} \bE^{(l)} \bx$.  Therefore,
%$$
%		\| \bE \| \geq \sup_{\|\bx\|_2=1, x_l = 0} \bx^{\top} \bE \bx =  \sup_{\|\bx\|_2=1, x_l = 0}  \bx^{\top} \bE^{(l)} \bx =
%\| \bE ^{(l)}\|
%$$
\end{enumerate}
Combining the above observations suggests that $u_{l} \approx \pm u_{l}^{(l)} \approx \pm u_{l}^{\star}$.

\subsection{Leave-one-out analysis}
\label{sec:LOO-analysis-rank1-denoising}
Now we make rigorous the heuristic argument in the last subsection, which relies heavily on careful statistical analysis.

\subsubsection{Preparation: what we have learned from $\ell_2$ perturbation theory}
Let us start by collecting a few results derived from the $\ell_2$ perturbation theory in Section~\ref{sec:matrix-denoising-L2} for handy reference. Experienced readers can proceed directly to Step 1.

Specifically, suppose that $\sigma\sqrt{n}\leq\frac{1-1/\sqrt{2}}{5}\lambda^{\star}$. Then with probability at least $1-O(n^{-8})$,
\begin{subequations}
\label{eq:L2-perturbation-summary-denoising}
\begin{align}
  \|\bm{E} \|  &\leq 5\sigma \sqrt{n}    & \|\bm{E}^{(l)} \| &\leq \|\bm{E} \| \leq 5\sigma \sqrt{n}  && \\
  \mathsf{dist}( \bm{u} , \bm{u}^{\star} )  &\leq \frac{10\sigma\sqrt{n}}{\lambda^{\star}}  ~~
 & \mathsf{dist}( \bm{u}^{(l)} , \bm{u}^{\star} )  &\leq \frac{10\sigma\sqrt{n}}{\lambda^{\star}} && \\
	| \lambda - \lambda^{\star} |  &\leq 5\sigma \sqrt{n} & | \lambda^{(l)} - \lambda^{\star} |  &\leq 5\sigma \sqrt{n} &&\\
	\max_{j: j\geq 2} | \lambda_j(\bm{M})   |  &\leq 5\sigma \sqrt{n} & \max_{j: j\geq 2}| \lambda_j(\bm{M}^{(l)}) |  &\leq 5\sigma \sqrt{n}
\end{align}
\end{subequations}
hold simultaneously for all $1\leq l\leq n$. Here, the first line arises from~\eqref{eq:noise-matrix-E-bound-denoising}, and the remaining claims follow the same argument as in the proof of Corollary~\ref{cor:davis-kahan-conclusion-corollary}. 
 Consequently, there exist global signs $z, z_l \in \{1,-1\}$  obeying $\|z\bm{u} - \bm{u}^{\star}\|_2= \mathsf{dist}( \bm{u}, \bm{u}^{\star} ) \leq 10\sigma\sqrt{n} / \lambda^{\star}$ and $\|z_l\bm{u}^{(l)} - \bm{u}^{\star}\|_2= \mathsf{dist}( \bm{u}^{(l)}, \bm{u}^{\star} ) \leq 10\sigma\sqrt{n} / \lambda^{\star}$. To simplify presentation, we shall assume
 \begin{subequations}
\label{eq:ul-ustar-WLOG-denoising}
\begin{align}
	\|\bm{u} - \bm{u}^{\star}\|_2 &= \mathsf{dist}( \bm{u}, \bm{u}^{\star} ), \quad \\
	\big\| \bm{u}^{(l)} - \bm{u}^{\star} \big\|_2 &= \mathsf{dist}( \bm{u}^{(l)}, \bm{u}^{\star} ) , \quad 1 \leq l \leq n
\end{align}
\end{subequations}
without loss of generality.
As a simple yet useful byproduct: if ${20\sigma{\sqrt{n}}} < \lambda^{\star}$, then Condition \eqref{eq:ul-ustar-WLOG-denoising} necessarily implies
\begin{align}
	\big\| \bm{u} -  \bm{u}^{(l)} \big\|_2 = \mathsf{dist}\big( \bm{u}, \bm{u}^{(l)}\big), \qquad 1 \leq l \leq n.  	
	\label{eq:u-ul-rotation-denoising}
\end{align}
To see this, combine the triangle inequality and \eqref{eq:L2-perturbation-summary-denoising} to yield
\begin{align*}
	\big\|\bm{u}-\bm{u}^{(l)}\big\|_{2} & \leq\big\|\bm{u}-\bm{u}^{\star}\big\|_{2}+\big\|\bm{u}^{(l)}-\bm{u}^{\star}\big\|_{2}
	\leq\frac{20\sigma\sqrt{n}}{\lambda^{\star}}<1 ,
\end{align*}
which taken collectively with the fact $\|\bm{u} \|_{2} = \|\bm{u}^{(l)}\|_{2}=1$ gives
\[
\big\|\bm{u}+\bm{u}^{(l)}\big\|_{2}^{2}=
2\big\|\bm{u}\big\|_{2}^{2}+2\big\|\bm{u}^{(l)}\big\|_{2}^2
-\big\|\bm{u}-\bm{u}^{(l)}\big\|_{2}^{2}>1>\big\|\bm{u}-\bm{u}^{(l)}\big\|_{2}^{2}.
\]
This together with the definition \eqref{eq:dist_UUstar-spectral} of $\mathsf{dist}(\cdot,\cdot)$ validates  \eqref{eq:u-ul-rotation-denoising}.

\subsubsection{Step 1: bounding the proximity of leave-one-out $\&$ true estimates}

In this step, we seek to control the distance between the true estimate $\bm{u}$ and the leave-one-out estimate $\bm{u}^{(l)}$. Suppose for the moment that
\begin{equation}\label{eq:noise-small-gap}
\|\bm{M}-\bm{M}^{(l)}\| \leq (1 - 1 / \sqrt{2}) \Big(\lambda^{(l)}-  \max_{j\geq 2}\big| \lambda_{j}\big(\bm{M}^{(l)}\big) \big| \Big).
\end{equation}
We can then invoke the Davis-Kahan theorem (cf.~Corollary \ref{cor:davis-kahan-conclusion-corollary}) and the relation \eqref{eq:u-ul-rotation-denoising} to yield
\begin{align}
	\big\|\bm{u}-\bm{u}^{(l)}\big\|_{2} &
	%=\mathsf{dist}\big(\bm{u},\bm{u}^{(l)}\big)
	\leq\frac{2\|\big(\bm{M}-\bm{M}^{(l)}\big)\bm{u}^{(l)}\|_{2}}{\lambda^{(l)}- \max\limits_{j\geq 2}\big| \lambda_{j}\big(\bm{M}^{(l)}\big) \big| }
	\leq \frac{4\|\big(\bm{M}-\bm{M}^{(l)}\big)\bm{u}^{(l)}\|_{2}}{\lambda^{\star}}.
	\label{eq:u-ul-diff-denoising}
\end{align}
Here, the last inequality invokes \eqref{eq:L2-perturbation-summary-denoising} and $20\sigma\sqrt{n}\leq \lambda^{\star}$ to obtain
\begin{align}
	\lambda^{(l)}-  \max_{j\geq 2}\big| \lambda_{j}\big(\bm{M}^{(l)}\big) \big| & \geq  (\lambda^{\star}-  5\sigma \sqrt{n} ) - 5\sigma \sqrt{n} \geq\lambda^{\star}/2 .
	\label{eq:lambda-l-minus-lambda-j-LB1}
	%\\
	%& \geq \lambda^{\star}- \|\bm{E}^{(l)}\| \geq \lambda^{\star}-10\sigma\sqrt{n}\geq\lambda^{\star}/2,
\end{align}
A byproduct of this calculation is that $\lambda^{(l)}$ is positive for all $1\leq l \leq n$.
%
%provided that $20\sigma\sqrt{n}\leq \lambda^{\star}$.

It thus remains to control the term $\|\big(\bm{M}-\bm{M}^{(l)}\big)\bm{u}^{(l)}\|_{2}$ in \eqref{eq:u-ul-diff-denoising}, towards which certain statistical independence proves crucial. Specifically, we observe that (by construction of $\bm{M}^{(l)}$)
\begin{align*}
\big(\bm{M}-\bm{M}^{(l)}\big)\bm{u}^{(l)} & =\bm{e}_{l}\bm{E}_{l,\cdot}\bm{u}^{(l)}+u_{l}^{(l)}(\bm{E}_{\cdot,l} - E_{l,l}\bm{e}_{l}),
\end{align*}
where $\bm{e}_{l}$ is the $l$-th standard basis vector,
% $u_{l}^{(l)}$ is the $l$-th entry of $\bm{u}^{(l)}$,
and $\bm{E}_{l,\cdot}$ (resp.~$\bm{E}_{\cdot,l}$)
denotes the $l$-th row (resp.~column) of $\bm{E}$. By construction, $\bm{u}^{(l)}$ is statistically independent of $\bm{E}_{l,\cdot}$, thus indicating that
\begin{align}
	\bm{E}_{l,\cdot} \bm{u}^{(l)} ~\sim~ \mathcal{N}(0, \sigma^2 \|\bm{u}^{(l)}\|_2^2) = \mathcal{N}(0, \sigma^2 )
\end{align}
conditioned on $\bm{u}^{(l)}$. Hence, with probability at least $1-n^{-10}$,
\begin{align}
	\big| \bm{E}_{l,\cdot} \bm{u}^{(l)} \big| \leq 5\sigma \sqrt{\log n}, \qquad 1\leq l\leq n.
\end{align}
In addition, $\|\bm{E}_{\cdot,l} - E_{l,l}\bm{e}_{l}\|_{2} \leq \|\bm{E}_{\cdot,l}\|_2\leq \|\bm{E}\|\leq 5\sigma \sqrt{n}$ (cf.~\eqref{eq:L2-perturbation-summary-denoising}). Consequently,
\begin{align*}
 & \|\big(\bm{M}-\bm{M}^{(l)}\big)\bm{u}^{(l)}\|_{2}
	\leq \big|\bm{E}_{l,\cdot}\bm{u}^{(l)}\big|  +
	 \big\|\bm{E}_{\cdot,l}\big\|_{2} \cdot \big|u_{l}^{(l)}\big| \\
 & \qquad \leq5\sigma\sqrt{\log n}+\big\|\bm{E}_{\cdot,l}\big\|_{2}\big(\big|u_{l}\big|+\big\|\bm{u}-\bm{u}^{(l)}\big\|_{\infty}\big)\\
 & \qquad \leq5\sigma\sqrt{\log n}+5\sigma\sqrt{n}\|\bm{u}\|_{\infty}+5\sigma\sqrt{n}\big\|\bm{u}-\bm{u}^{(l)}\big\|_{2}.
\end{align*}
Substitution into \eqref{eq:u-ul-diff-denoising} yields
\begin{align*}
\big\|\bm{u}-\bm{u}^{(l)}\big\|_{2}
 & \leq \frac{20\sigma\sqrt{\log n}+20\sigma\sqrt{n}\|\bm{u}\|_{\infty}+20\sigma\sqrt{n}\big\|\bm{u}-\bm{u}^{(l)}\big\|_{2}}{\lambda^{\star}}\\
 & \leq\frac{20\sigma\sqrt{\log n}+20\sigma\sqrt{n}\|\bm{u}\|_{\infty}}{\lambda^{\star}}+\frac{1}{2}\big\|\bm{u}-\bm{u}^{(l)}\big\|_{2},
\end{align*}
provided that $40\sigma\sqrt{n}\leq \lambda^{\star}$.
Rearranging terms and taking the union bound, we demonstrate that with probability at least $1- O(n^{-8})$,
\begin{align}
\big\|\bm{u}-\bm{u}^{(l)}\big\|_{2} & \leq\frac{40\sigma\sqrt{\log n}+40\sigma\sqrt{n}\|\bm{u}\|_{\infty}}{\lambda^{\star}} , \qquad 1\leq l\leq n.
	\label{eq:u-ul-proximity-final-denoising}
\end{align}

\medskip
\begin{proof}[Proof of the relation \eqref{eq:noise-small-gap}] Apply the triangle inequality to see that
\begin{align*}
\|\bm{M}-\bm{M}^{(l)}\|&\leq\|\bm{M}-\bm{M}^{\star}\|+\|\bm{M}^{\star}-\bm{M}^{(l)}\|=\|\bm{E}\|+\|\bm{E}^{(l)}\| \\
&\leq10\sigma\sqrt{n}\leq\lambda^{\star}/10.
\end{align*}
Here, the equality arises from the definition of $\bm{M}$ and $\bm{M}^{\star}$,
the penultimate inequality uses (\ref{eq:L2-perturbation-summary-denoising}), while the last inequality holds as long
as $100\sigma\sqrt{n}\leq\lambda^{\star}$. This together with \eqref{eq:lambda-l-minus-lambda-j-LB1} establishes \eqref{eq:noise-small-gap}.
\end{proof}

\subsubsection{Step 2: analyzing leave-one-out estimates}

We now turn attention to bounding the size of the $l$-th entry $u_l^{(l)}$ of $\bm{u}^{(l)}$. Since $\lambda^{(l)}> 0$, by \eqref{eq:ul_l-expression-denoising}, we have 
%together with the fact $\lambda^{(l)}> 0$ allows us to obtain
%
\begin{align*}
u_{l}^{(l)}-u_{l}^{\star} & %=u_{l}^{\star}\Big(\frac{\lambda^{\star}}{\lambda^{(l)}}\bm{u}^
%{\star\top}\bm{u}^{(l)}-1\Big)
=u_{l}^{\star}\Big(\frac{\lambda^{\star}}{\lambda^{(l)}}\bm{u}^{\star\top}\bm{u}^{(l)}-\bm{u}^{\star\top}\bm{u}^{\star}\Big)\\
 & =u_{l}^{\star}\Big(\frac{\lambda^{\star}-\lambda^{(l)}}{\lambda^{(l)}}\bm{u}^{\star\top}\bm{u}^{(l)}\Big)+u_{l}^{\star}\bm{u}^{\star\top}\big(\bm{u}^{(l)}-\bm{u}^{\star}\big).
\end{align*}
%
%where the first line relies on \eqref{eq:ul-l-expression-denoising}
%and the fact  $\|\bm{u}^{\star}\|_2=1$. Therefore,
The triangle inequality and the Cauchy-Schwarz inequality then give
\begin{align}
\big|u_{l}^{(l)}-u_{l}^{\star}\big|
& \leq\big|u_{l}^{\star}\big|\cdot\frac{\big|\lambda^{\star}-\lambda^{(l)}\big|}{\big|\lambda^{(l)}\big|} \cdot \|\bm{u}^{\star}\|_{2} \cdot \|\bm{u}^{(l)}\|_{2}
	+ \big|u_{l}^{\star}\big|\cdot\|\bm{u}^{\star}\|_{2} \cdot \big\|\bm{u}^{(l)}-\bm{u}^{\star}\big\|_{2} \notag\\
 & \leq\big|u_{l}^{\star}\big|\cdot\frac{10\sigma\sqrt{n}}{\lambda^{\star}}+\big|u_{l}^{\star}\big|\cdot\frac{10\sigma\sqrt{n}}{\lambda^{\star}} \notag\\
 & \leq  \frac{20\sigma\sqrt{n}}{\lambda^{\star}} \big\|\bm{u}^{\star}\big\|_{\infty}.
	\label{eq:ul-ulstar-bound-denoising-rank1}
\end{align}
Here, the second line holds due to Condition (\ref{eq:L2-perturbation-summary-denoising}) and the fact $\big|\lambda^{(l)}\big|\geq\lambda^{\star}/2$; see~(\ref{eq:lambda-l-minus-lambda-j-LB1}).

\subsubsection{Step 3: putting all pieces together}

Putting \eqref{eq:u-ul-proximity-final-denoising} and \eqref{eq:ul-ulstar-bound-denoising-rank1} together, we arrive at
\begin{align}
\big\|\bm{u}-\bm{u}^{\star}\big\|_{\infty} & =\max_{l}\big|u_{l}-u_{l}^{\star}\big|\leq\max_{l}\Big\{\big|u_{l}^{(l)}-u_{l}^{\star}\big|+\big\|\bm{u}-\bm{u}^{(l)}\big\|_{2}\Big\} \notag\\
 & \leq \frac{20\sigma\sqrt{n}}{\lambda^{\star}} \big\|\bm{u}^{\star}\big\|_{\infty}
	+ \frac{40\sigma\sqrt{\log n}+40\sigma\sqrt{n}\|\bm{u}\|_{\infty}}{\lambda^{\star}}.
	\label{eq:u-ustar-inf-first-bound}
\end{align}
The above upper bound, however, involves the term $\|\bm{u}\|_{\infty}$, which can further be bounded by
$\|\bm{u}^{\star}\|_{\infty} + \|\bm{u}-\bm{u}^{\star} \|_{\infty}$.  Substituting this into \eqref{eq:u-ustar-inf-first-bound}, we have
\begin{align*}
\big\|\bm{u}-\bm{u}^{\star}\big\|_{\infty}
%\eqref{eq:u-ustar-inf-first-bound}
%	& \leq\frac{20\sigma\sqrt{n} \big\|\bm{u}^{\star}\big\|_{\infty}}{\lambda^{\star}}
%	+ \frac{40\sigma\sqrt{\log n}  + 40 \sigma \sqrt{n} \big( \|\bm{u}^{\star}\|_{\infty} + \|\bm{u}-\bm{u}^{\star} \|_{\infty} \big)}{\lambda^{\star}}\\
	& \leq \frac{40\sigma\sqrt{\log n} + 60 \sigma \sqrt{n}\, \|\bm{u}^{\star}\|_{\infty} }{\lambda^{\star}}+\frac{1}{2}\big\|\bm{u}-\bm{u}^{\star}\big\|_{\infty},
\end{align*}
provided that $80\sigma\sqrt{n}\leq\lambda^{\star}$. Rearranging terms yields
\begin{align*}
	 \big\|\bm{u}-\bm{u}^{\star}\big\|_{\infty}
	 &\leq \frac{ 80\sigma\sqrt{\log n} + 120 \sigma \sqrt{n}\, \|\bm{u}^{\star}\|_{\infty} }{\lambda^{\star}}
	 % \leq  \frac{200\sigma\sqrt{\log n}\, ( \sqrt{n} \|\bm{u}^{\star}\|_{\infty} )}{\lambda^{\star}} \\
	 = \frac{80\sigma\sqrt{\log n} + 120 \sigma \sqrt{\mu} }{\lambda^{\star}} ,
\end{align*}
where the last identity results from the definition \eqref{defn:incoherence-rank-1-denoising} of $\mu$.

%% file: chapters/general_theory_linf.tex
\section{$\ell_{2,\infty}$ eigenspace perturbation under independent noise}
\label{sec:setup-general-theory-independent}

The appealing entrywise behavior of the eigenvector estimator in Section~\ref{sec:rank-1-denoising-LOO} hints at the promising performance of spectral methods for broader contexts. In this section, we set out to develop a more general framework about $\ell_{2,\infty}$ eigenspace perturbation  that covers a wide spectrum of scenarios.

\subsection{Setup and notation}\label{sec:setup-general-theory}

\paragraph{Ground truth.}
Consider a rank-$r$ symmetric matrix $\bm{M}^{\star}\in \mathbb{R}^{n\times n}$  with eigenvectors $\bm{u}_1^{\star},\cdots,\bm{u}_n^{\star}$ and associated eigenvalues $\lambda_1^{\star},\cdots,\lambda_n^{\star}$ obeying
\begin{align} \label{eq:assumption-lambda-r-positive-setup}
	|\lambda_1^{\star}| \geq |\lambda_2^{\star}| \geq \cdots \geq |\lambda_r^{\star}| >0
	\quad \text{and} \quad
	\lambda_{r+1}^{\star} = \cdots = \lambda_n^{\star}=0.
\end{align}
We shall write the eigendecomposition $\bm{M}^{\star}=\bm{U}^{\star}\bm{\Lambda}^{\star}\bm{U}^{\star\top}$ as usual, where $\bm{\Lambda}^{\star} \coloneqq \mathsf{diag}([\lambda_1^{\star},\cdots,\lambda_r^{\star}])$ and $\bm{U}^{\star}\coloneqq   [\bm{u}_1^{\star},\cdots, \bm{u}_r^{\star}]\in \mathbb{R}^{n\times r}$. Denote the condition number of $\bm{M}^{\star}$ as
\begin{align}
	\label{eq:defn-kappa}
	\kappa \coloneqq {|\lambda_1^{\star}|} \,/\, {|\lambda_r^{\star}|}.
\end{align}
Akin to Definition~\ref{assump:mc-incoherence}, the  incoherence parameter of $\bm{M}^{\star}$ is defined as
\begin{align}
	%\|\bm{U}^{\star}\|_{2,\infty} \coloneqq \max_i \big\| \bm{e}_i^{\top}\bm{U}^{\star} \big\|_{2} =  \sqrt{\frac{\mu r}{n}} .
	\mu \coloneqq \frac{n\|\bm{U}^{\star}\|_{2,\infty}^{2}}{r},
	\label{eq:defn-Ustar-incoherence}
\end{align}
%
%This incoherence parameter, which is crucial in our theoretical development,
%captures how well the energy of $\bm{U}^{\star}$ is spread out across all rows. It is readily seen from the fact $n^{-1/2}\|\bm{U}^{\star}\|_{\mathrm{F}} \leq \|\bm{U}^{\star}\|_{2,\infty}\leq \|\bm{U}^{\star}\|$ that
a parameter that captures how well the energy of $\bm{U}^{\star}$ is spread out across all rows and that obeys (see Remark~\ref{remark:range-mu-MC})
\begin{align}
	1 \leq \mu \leq n/r .
	\label{eq:mu-constraint-general}
\end{align}

\paragraph{Observed data.} What we observe is a corrupted version
\begin{align}
	\bm{M}=\bm{M}^{\star}+\bm{E} \in \mathbb{R}^{n\times n},
	\label{eq:M-noisy-copy-general}
\end{align}
where $\bm{E}$ is a symmetric noise matrix. We denote by $\{\lambda_i\}_{1\leq i\leq n}$ the set of eigenvalues of $\bm{M}$ obeying
\begin{align}
	|\lambda_1| \geq |\lambda_2| \geq \cdots \geq |\lambda_n|,
	\label{eq:lambda-1-n-ordering-magnitude}
\end{align}
and let $\bm{u}_i$ be the eigenvector of $\bm{M}$ associated with $\lambda_i$. We shall introduce the diagonal matrix $\bm{\Lambda}\in \mathbb{R}^{r\times r}$ as $\bm{\Lambda} \coloneqq \mathsf{diag}([\lambda_1,\cdots,\lambda_r])$.

\paragraph{Noise assumptions.}
This section aims to cover a fairly broad class of scenarios of independent noise.
In particular, the noise matrix considered herein is assumed to satisfy the mild conditions listed below.
\begin{assumption}
	\label{assumption-noise-general}
	%[Noise assumptions]
	The entries in the lower triangular part of $\bm{E}=[E_{i,j}]_{1\leq i,j\leq n}$ are independently generated obeying
\begin{align}
	\mathbb{E}[E_{i,j}] = 0, \quad \mathbb{E}[E_{i,j}^2] =: \sigma_{i,j}^2 \leq \sigma^2, \quad |E_{i,j}|\leq B, \quad \text{for all }i\geq j.
	\label{eq:noise-E-condition-general}
\end{align}
	In particular, $\sigma^2$ is taken to be the smallest choice satisfying \eqref{eq:noise-E-condition-general}. 
	Further, it is assumed that
	%there is some positive quantity $c_{\mathsf{b}}=O(1)$ such that
	%
	\begin{align}
		c_{\mathsf{b}} \coloneqq \frac{B}{  \sigma  \sqrt{n/(\mu\log n)} } = O(1).
		\label{eq:assumption-B-sigma}
	\end{align}
\end{assumption}

We emphasize that both $\sigma$ and $B$ are quantities that are allowed to scale with $n$. When $\mu$ is not too large, Condition \eqref{eq:assumption-B-sigma}
allows the maximum magnitude $B$ of each noisy entry to be substantially larger than the typical size $\sigma$.

\paragraph{Goal and algorithm.}
We seek to estimate $\bm{U}^{\star}$ based on $\bm{M}$. Towards this, a simple spectral  method computes the matrix $\bm{U}=[\bm{u}_1,\cdots,\bm{u}_r]\in \mathbb{R}^{n\times r}$ that comprises the top-$r$ leading eigenvectors of $\bm{M}$.

\subsection{$\ell_{2,\infty}$ and $\ell_{\infty}$ theoretical guarantees}
\label{sec:L2inf-Linf-eigen-theory}

The leave-one-out argument introduced before, when properly strengthened, enables powerful  $\ell_{2,\infty}$  performance guarantees for the spectral estimate $\bm{U}$, which concern row-wise perturbation of the eigenspace. Before continuing, we remind the readers of the global rotation ambiguity, namely, in general we cannot expect $\bm{U}$  to be close to $\bm{U}^{\star}$ unless suitable global rotation is taken into account. In light of this, we introduce the following notation that helps identify a proper rotation matrix.
\begin{definition} For any matrix $\bm{Z}$ with SVD $\bm{Z}=\bm{U}_{Z}\bm{\Sigma}_Z\bm{V}_{Z}^{\top}$ (where $\bm{U}_Z$ and $\bm{V}_Z$ represent respectively the left and right singular matrices of $\bm{Z}$, and $\bm{\Sigma}_Z$ is a diagonal matrix composed of the singular values), define
\begin{align}
	\mathsf{sgn}(\bm{Z}) \coloneqq \bm{U}_{Z}  \bm{V}_{Z}^{\top}
	\label{eq:defn-sgn-Z}
\end{align}
to be the matrix sign function of $\bm{Z}$.
\end{definition}
\begin{remark}
The matrix sign function is commonly encountered when aligning two matrices---classically known as the orthogonal Procrustes problem \citep{schonemann1966generalized}.
Consider any two matrices $\widehat{\bB},\bB\in \mathbb{R}^{n\times r}$ with $r\leq n$. Among all rotation matrices,
the one that best aligns  $\widehat \bB$ with $\bB$ is precisely $ \mathsf{sgn}(\widehat \bB^\top \bB)$ (see, e.g., \citep[Appendix D.2.1]{ma2017implicit}), namely,
$$
	\mathsf{sgn}(\widehat \bB^\top \bB) = \underset{\bm{O}\,\in \mathcal{O}^{r\times r}}{\arg\min} ~\| \widehat \bB {\bm O} - \bB\|_{\mathrm{F}}^2.
$$
\end{remark}
With this definition in place, we are ready to state an $\ell_{2,\infty}$ theory adapted from \citet{abbe2020entrywise}.
Compared to the original development in \citet{abbe2020entrywise}, 
the theorem and its proof provided herein are more streamlined versions tailored to the current setting. 
\begin{theorem}
\label{thm:UsgnH-Ustar-MUstar-general}
	Consider the settings and assumptions in Section~\ref{sec:setup-general-theory}. Define $\bm{H} \coloneqq \bm{U}^{\top}\bm{U}^{\star}$. With probability exceeding $1-O(n^{-5})$, one has
\begin{subequations}
\label{eq:UsgnH-Ustar-MUstar-bound-theorem-general}
\begin{align}
& \big\|\bm{U}\mathsf{sgn}(\bm{H})-\bm{U}^{\star}\big\|_{2,\infty}  \lesssim \frac{\sigma\kappa\sqrt{\mu r}+\sigma\sqrt{r\log n}}{ |\lambda_{r}^{\star} | },
\label{eq:UsgnH-Ustar-bound-theorem-general}\\
& \big\|\bm{U}\mathsf{sgn}(\bm{H})-\bm{M}\bm{U}^{\star}\big(\bm{\Lambda}^{\star}\big)^{-1}\big\|_{2,\infty} \notag\\
& \qquad\qquad \lesssim \frac{\sigma\kappa\sqrt{\mu r}}{ |\lambda_{r}^{\star} | }+ \frac{\sigma^{2}\sqrt{rn\log n}+\sigma B\sqrt{\mu r\log^{3}n}}{\big(\lambda_{r}^{\star}\big)^{2}}, 
	%\frac{\sigma\kappa\sqrt{\mu r}}{ |\lambda_{r}^{\star} | }+\frac{\sigma^{2}\sqrt{rn\log n}\,(1+c_{\mathsf{b}}\sqrt{\log n})}{\big(\lambda_{r}^{\star}\big)^{2}},
\label{eq:UsgnH-MUstar-bound-theorem-general}
\end{align}
\end{subequations}
provided that $\sigma \sqrt{n \log n}  \leq c_{\sigma} |\lambda_r^{\star}|$ for some sufficiently small constant $c_{\sigma}>0$.
\end{theorem}
The proof of this theorem can be found in Section~\ref{sec:proof-thm:UsgnH-Ustar-MUstar-general}. 
Note that under the assumption in the theorem, the bound on the right-hand side of 
\eqref{eq:UsgnH-MUstar-bound-theorem-general} is no larger than the one on the right-hand side of \eqref{eq:UsgnH-Ustar-bound-theorem-general}. 
In fact, \eqref{eq:UsgnH-MUstar-bound-theorem-general} could indeed be tighter than \eqref{eq:UsgnH-Ustar-bound-theorem-general} in some important scenarios like community recovery (see Section~\ref{sec:community-detection-linf}).

The $\ell_{2,\infty}$ perturbation theory in Theorem~\ref{thm:UsgnH-Ustar-MUstar-general} accommodates a broad family of noise matrices with independent entries.
In the sequel, we take a moment to interpret several key messages conveyed by this result.

\paragraph{De-localization of estimation errors.} For simplicity, let us concentrate on the case where $\mu,\kappa =O(1)$.   Note that the Davis-Kahan theorem introduced previously results in the following $\ell_2$ estimation guarantees (to be detailed in Section~\ref{sec:preliminary-facts-Linf-general})
\begin{align}
	\mathsf{dist}_{\mathrm{F}}(\bm{U},\bm{U}^{\star})  \leq \sqrt{r} \,\mathsf{dist}(\bm{U},\bm{U}^{\star}) \lesssim  \frac{\sigma \sqrt{nr}}{|\lambda_r^{\star}|}.
	\label{eq:simple-Euclidean-error-general}
\end{align}
In comparison, the $\ell_{2,\infty}$ bound derived in  Theorem~\ref{thm:UsgnH-Ustar-MUstar-general} simplifies to
\begin{align}
	\min_{\bm{R}\,\in \mathcal{O}^{r\times r}}\big\|\bm{U} \bm{R}-\bm{U}^{\star}\big\|_{2,\infty}
	 \leq \big\|\bm{U}\mathsf{sgn}(\bm{H})-\bm{U}^{\star}\big\|_{2,\infty}  \lesssim  \frac{\sigma \sqrt{r\log n}}{ |\lambda_{r}^{\star}| }
	 \label{eq:Linf-eigen-bound-simpler}
\end{align}
under the condition $\mu,\kappa = O(1)$,  which is about $O(\sqrt{n/\log n})$ times smaller than the Euclidean error bound \eqref{eq:simple-Euclidean-error-general}.
This implies that the estimation error of $\bm{U}$ is fairly de-localized and spread out across all rows.

\paragraph{First-order approximation.}
Informally, Theorem~\ref{thm:UsgnH-Ustar-MUstar-general} (and its analysis) unveils the goodness of the first-order approximation
\begin{align}
	\bm{U}\mathsf{sgn}(\bm{H}) \approx \bm{M}\bm{U}^{\star}(\bm{\Lambda}^{\star})^{-1} = \bm{U}^{\star} + \bm{E}\bm{U}^{\star}(\bm{\Lambda}^{\star})^{-1}
\end{align}
uniformly across all rows. An implication of Theorem~\ref{thm:UsgnH-Ustar-MUstar-general} is that $\bm{U}$ might be closer to the first-order approximation $\bm{M}\bm{U}^{\star}(\bm{\Lambda}^{\star})^{-1}$ than to the ground truth $\bm{U}^{\star}$ (namely, the upper bound on the right-hand side of  \eqref{eq:UsgnH-MUstar-bound-theorem-general} is smaller than the bound on the right-hand side of \eqref{eq:UsgnH-Ustar-bound-theorem-general} under the stated conditions). In principle, the linear term $\bm{E}\bm{U}^{\star}(\bm{\Lambda}^{\star})^{-1}$ can be viewed as a correction term that helps improve the approximation fidelity.
As we shall see momentarily in Section~\ref{sec:community-detection-linf}, this subtle difference leads to sharper performance guarantees in applications like community recovery.

\paragraph{Entrywise estimation errors.} There is no shortage of applications where one cares more about the fine-grained estimation accuracy of the matrix rather than that of the low-rank factors. Fortunately, the $\ell_{2,\infty}$ theory derived in  Theorem~\ref{thm:UsgnH-Ustar-MUstar-general} in turn enables entrywise performance guarantees when estimating the matrix $\bm{M}^{\star}$. This is summarized in the following corollary, with the proof deferred to Section~\ref{sec:proof-cor-entrywise-error}.
\begin{corollary}
	\label{cor:entrywise-error-general}
	Consider the settings and assumptions in Section~\ref{sec:setup-general-theory}, and assume  further that
	$\sigma \kappa \sqrt{n\log n} \leq c_1 |\lambda_r^{\star}|$ for some sufficiently small constant $c_1>0$. Then with probability at least $1-O(n^{-5})$, one has
	\begin{align}
		\big\|\bm{U}\bm{\Lambda}\bm{U}^{\top}-\bm{M}^{\star}\big\|_{\infty} & \lesssim\sigma\kappa^{2}\mu r\sqrt{\frac{\log n}{n}}.
	\end{align}
\end{corollary}
Once again, it is instrumental to explain the result by specializing it to the simpler regime where $\kappa, \mu, r =O(1)$.  In this case, the finding of Corollary~\ref{cor:entrywise-error-general} reduces to
\begin{align}
	\big\|\bm{U}\bm{\Lambda}\bm{U}^{\top}-\bm{M}^{\star}\big\|_{\infty} & \lesssim\sigma  \sqrt{\frac{\log n}{n}}.
	\label{eq:ULambdaU-entrywise-simple}
\end{align}
In comparison, the Euclidean error of this spectral estimate satisfies (which follows by combining \eqref{eq:ULambdaU-Mstar-norm-denoising} with \eqref{eq:X-spectral-norm-iid-special-ramon} and \eqref{eq:assumption-B-sigma})
\begin{align}
	\big\|\bm{U}\bm{\Lambda}\bm{U}^{\top}-\bm{M}^{\star}\big\|_{\mathrm{F}} \leq 2\sqrt{2} \|\bm{E}\| \lesssim \sigma\sqrt{n}
\end{align}
with high probability, which is on the order of $n/\sqrt{\log n}$ times larger than the entrywise error bound \eqref{eq:ULambdaU-entrywise-simple}. In other words, the energy of the estimation error of the unknown matrix is also dispersed more or less across all matrix entries, a message that cannot be derived from classical matrix perturbation theory alone.

\paragraph{Leave-one-out analysis.}
As alluded to previously, the proof of Theorem~\ref{thm:UsgnH-Ustar-MUstar-general} relies heavily upon the leave-one-out analysis framework to decouple delicate statistical dependency. While the core idea bears close resemblance to the exposition in Section~\ref{sec:LOO-analysis-rank1-denoising}, implementing this idea rigorously for the general case requires considerably more effort. We defer a complete proof to Section~\ref{sec:proof-thm:UsgnH-Ustar-MUstar-general}.

\section{$\ell_{2,\infty}$ singular subspace perturbation under independent noise}
\label{sec:general-L2inf-SVD-perturbation}

The general theory presented in Section~\ref{sec:setup-general-theory-independent} applies only to symmetric matrices. It is not uncommon, however, to encounter scenarios where the matrix of interest $\bm{M}^{\star}$ is  asymmetric.
This motivates the need of extending the $\ell_{2,\infty}$ perturbation theory to accommodate more general matrices, which is the main content of the current section.

\subsection{Setup and notation}\label{sec:general-theory-setup-asymm}
\paragraph{Ground truth.}

Consider an unknown rank-$r$ matrix $\bm{M}^{\star}\in\mathbb{R}^{n_{1}\times n_{2}}$. Let
 $\bm{M}^{\star}=\bm{U}^{\star}\bm{\Sigma}^{\star}\bm{V}^{\star\top}$
represent the SVD of $\bm{M}^{\star}$,
where $\bm{U}^{\star}\in\mathbb{R}^{n_{1}\times r}$ (resp.~$\bm{V}^{\star}\in\mathbb{R}^{n_{2}\times r}$)
entails the top-$r$ left (resp.~right) singular vectors of $\bm{M}^{\star}$,
and $\bm{\Sigma}^{\star}=\mathsf{diag}([\sigma_{1}^{\star},\sigma_{2}^{\star},\cdots,\sigma_{r}^{\star}])$
is formed by the (nonzero) singular values of $\bm{M}^{\star}$. We  arrange the singular values $\{\sigma_i^{\star}\}$ in descending order (i.e., $\sigma_{1}^{\star}\geq \sigma_{2}^{\star} \geq \cdots \geq \sigma_{r}^{\star} > 0$). 
%
%and define the condition number of the matrix $\bm{M}^{\star}$ as $\kappa\coloneqq  {\sigma_{1}^{\star}}/{\sigma_{r}^{\star}}$.
%
% In accordance with (\ref{eq:defn-Ustar-incoherence}),

\paragraph{Key parameters.}
As usual,  $\mu$ stands for the incoherence parameter  of  $\bm{M}^{\star}$ (see Definition~\ref{assump:mc-incoherence}),
and  the condition number of the matrix $\bm{M}^{\star}$ is defined as $\kappa\coloneqq  {\sigma_{1}^{\star}}/{\sigma_{r}^{\star}}$.
Without loss of generality, it is assumed that $$n_1 \leq n_2$$ and we set $n \coloneqq n_1 + n_2$.

\paragraph{Observations and noise assumptions. }
Assume we have access to corrupted observations of $\bm{M}^{\star}$ as follows:
\[
	\bm{M}=\bm{M}^{\star}+\bm{E}\in\mathbb{R}^{n_{1}\times n_{2}},
\]
where $\bm{E}=[E_{i,j}]$ stands for a  noise or perturbation matrix. We impose the following conditions on  $\bm{E}$, which is a natural
adaptation of Assumption~\ref{assumption-noise-general} to the general case and covers a diverse array of scenarios.

\begin{assumption}
\label{assump:noise-general-rect}
The entries of $\bm{E}$ are independently generated obeying
\begin{align}
	\mathbb{E}[E_{i,j}] = 0, \quad \mathbb{E}[E_{i,j}^2]\leq \sigma^2, \quad |E_{i,j}|\leq B \quad \text{for all }i, j.
\end{align}
	Further,  assume that
	\begin{align}
		c_{\mathsf{b}} \coloneqq \frac{B}{\sigma  \sqrt{n_1/(\mu\log n)}} = O(1).
	\end{align}\end{assumption}

\paragraph{Goal and algorithm.}
Again, we aim at estimating $\bm{U}^{\star}$ and $\bm{V}^{\star}$, based
on the observation $\bm{M}$, using a spectral method. Specifically, let
\begin{equation}
\bm{M}=\left[\begin{array}{cc}
\bm{U} & \bm{U}_{\perp}\end{array}\right]\left[\begin{array}{cc}
\bm{\Sigma}\\
 & \bm{\Sigma}_{\perp}
\end{array}\right]\left[\begin{array}{c}
\bm{U}^{\top}\\
\bm{V}^{\top}
\end{array}\right]\label{eq:general-M-SVD}
\end{equation}
be the SVD of $\bm{M}$, in which $\bm{U}\bm{\Sigma}\bm{V}^{\top}$
is the rank-$r$ SVD (i.e., the singular values in $\bm{\Sigma}\coloneqq\mathsf{diag}([\sigma_{1},\cdots,\sigma_{r}])$
are larger than those in $\bm{\Sigma}_{\perp}$). The spectral method then
	deploys ($\bm{U},\bm{V}$) as an estimate of ($\bm{U}^{\star},\bm{V}^{\star}$).

\subsection{$\ell_{2,\infty}$ and $\ell_{\infty}$ theoretical guarantees}
\label{sec:theory-Linf-SVD-general}

We now present a theorem that generalizes Theorem~\ref{thm:UsgnH-Ustar-MUstar-general} and Corollary~\ref{cor:entrywise-error-general} to accommodate general (asymmetric and possibly rectangular) matrices. This can be accomplished via a standard ``symmetric dilation'' trick; the details can be found in Section~\ref{sec:proof-inf-rect}.  
% Yuxin: moved to Section 4.3.1
%Recall the definition of incoherent parameter $\mu$ in definition~\ref{assump:mc-incoherence} and condition number $\kappa$ that generalizes \eqref{eq:defn-kappa}

\begin{theorem}
\label{thm:2-inf-asymm}
Consider the settings and
assumptions in Section~\ref{sec:general-theory-setup-asymm}, and
define $\bm{H}_{\bm{U}}\coloneqq\bm{U}^{\top}\bm{U}^{\star}$ and
$\bm{H}_{\bm{V}}\coloneqq\bm{V}^{\top}\bm{V}^{\star}$. With probability
at least $1-O(n^{-5})$, one has
\begin{align}
	& \max\Big\{\|\bm{U}\mathsf{sgn}(\bm{H}_{\bm{U}})-\bm{U}^{\star}\|_{2,\infty},\, \|\bm{V}\mathsf{sgn}(\bm{H}_{\bm{V}})-\bm{V}^{\star}\|_{2,\infty} \Big\} \nonumber\\
	& \qquad\qquad\qquad\qquad \lesssim\frac{\sigma\sqrt{ r}\big(\kappa\sqrt{\frac{n_{2}}{n_{1}}\mu}+\sqrt{\log n}\big)}{\sigma_{r}^{\star}},\label{eq:main-result-2-infty-general-asymm}
\end{align}
provided that $\sigma\sqrt{n\log n}\leq c_{1}\sigma_{r}^{\star}$
for some sufficiently small constant $c_{1}>0$. In addition, if $\sigma\kappa\sqrt{n\log n}\leq c_{2}\sigma_{r}^{\star}$
for some small enough constant $c_{2}>0$, then the following holds with
probability at least $1-O(n^{-5})$:
\begin{equation}
	\|\bm{U}\bm{\Sigma}\bm{V}^{\top}-\bm{M}^{\star}\|_{\infty}\lesssim\sigma\kappa^{2}\mu r\sqrt{\frac{(n_2/n_1) \log n}{n_{1}}}.\label{eq:main-result-infty-general-asymm}
\end{equation}
\end{theorem}

The messages conveyed in Theorem~\ref{thm:2-inf-asymm} largely parallel those in Theorem~\ref{thm:UsgnH-Ustar-MUstar-general} and Corollary~\ref{cor:entrywise-error-general}.
For simplicity, let us discuss the implications when $\kappa, \mu, r = O(1)$ and $n_1 \asymp n_2$ (i.e., the aspect ratio of the matrix is $n_2 / n_1 = O(1)$).
In this scenario, Theorem~\ref{thm:2-inf-asymm} implies that
\begin{align*}
\|\bm{U}\mathsf{sgn}(\bm{H}_{\bm{U}})-\bm{U}^{\star}\|_{2,\infty}+\|\bm{V}\mathsf{sgn}(\bm{H}_{\bm{V}})-\bm{V}^{\star}\|_{2,\infty} & \lesssim\frac{\sigma\sqrt{\log n}}{\sigma_{r}^{\star}},\\
\|\bm{U}\bm{\Sigma}\bm{V}^{\top}-\bm{M}^{\star}\|_{\infty} & \lesssim\sigma  \sqrt{\frac{\log n}{n}},
\end{align*}
both of which bear close similarities to our previous observations \eqref{eq:Linf-eigen-bound-simpler} and \eqref{eq:ULambdaU-entrywise-simple}. Akin to our discussions in Section~\ref{sec:L2inf-Linf-eigen-theory}, these findings tell us that the  singular subspace estimation errors (resp.~the matrix estimation errors)  are fairly spread out across all rows of the singular subspace (resp.~all entries of the matrix).

%% file: chapters/mc_inf.tex
\section{Application: Entrywise guarantees for matrix completion}

To illustrate the utility of the fine-grained perturbation theory presented in previous sections, let us revisit the problem of matrix completion introduced in Section~\ref{sec:mc-l2} and apply our refined theory.

As a recap, the spectral method proposed for matrix completion proceeds by computing the best rank-$r$ approximation $\bm{U}\bm{\Sigma}\bm{V}^{\top}$ of the rescaled data matrix $\bm{M}=p^{-1}\mathcal{P}_{\Omega}(\bm{M}^{\star})$, where $p$ is the probability of each entry being observed, and $\mathcal{P}_{\Omega}(\cdot)$ denotes the Euclidean projection onto the set of matrices supported on the sampling set $\Omega$.
This time, we seek to characterize the $\ell_{2,\infty}$ error when estimating the true singular subspaces $\bm{U}^{\star}$ and $\bm{V}^{\star}$,
as well as the $\ell_{\infty}$ error when estimating the unknown matrix $\bm{M}^{\star}$, as stated below. As before, we set $n \coloneqq n_1 + n_2$.

\begin{theorem}
\label{thm:mc-inf}
Consider the settings and assumptions in Section~\ref{sec:problem-formulation-MC}, and
define $\bm{H}_{\bm{U}}\coloneqq\bm{U}^{\top}\bm{U}^{\star}$ and
$\bm{H}_{\bm{V}}\coloneqq\bm{V}^{\top}\bm{V}^{\star}$.
Suppose that $n_1\leq n_2$ and $n_{1}p\geq C\kappa^{4}\mu^{2}r^{2}\log n$ for some
sufficiently large constant $C>0$. Then with probability greater
than $1-O(n^{-5})$, we have
\begin{subequations}
\begin{align}
\max\{\|\bm{U}\mathsf{sgn}(\bm{H}_{\bm{U}})-\bm{U}^{\star}\|_{2,\infty},\,\,&\|\bm{V}\mathsf{sgn}(\bm{H}_{\bm{V}})-\bm{V}^{\star}\|_{2,\infty}\} \nonumber\\& \leq\kappa^{2}\sqrt{\frac{\mu^{3}r^{3}\log n}{n_{1}^{2}p}};
\label{eq:mc-inf-claim-1}\\
\|\bm{U}\bm{\Sigma}\bm{V}^{\top}-\bm{M}^{\star}\|_{\infty} & \lesssim\kappa^{2}\mu^{2}r^{2}\sqrt{\frac{\log n}{n_{1}^{3}p}}\|\bm{M}^{\star}\|.
\label{eq:mc-inf-claim-2}
\end{align}
\end{subequations}
\end{theorem}
\begin{proof}
	Recall our notation $\bm{E}= \bm{M}  -\bm{M}^{\star}= p^{-1}\mathcal{P}_{\Omega}(\bm{M}^{\star})-\bm{M}^{\star}$.
 It is straightforward to check
that $\bm{E}$ satisfies Assumption~\ref{assump:noise-general-rect} with
\begin{equation}
\sigma^{2}\coloneqq\frac{\|\bm{M}^{\star}\|_{\infty}^{2}}{p},\qquad\text{and}\qquad B\coloneqq\frac{\|\bm{M}^{\star}\|_{\infty}}{p}. \label{eq:mc-inf-noise}
\end{equation}
In addition, from the relation $B= c_{\mathsf{b}}\sigma\sqrt{n_1/({\mu}\log n)}$, it is seen that $c_{\mathsf{b}}= O(1)$
holds as long as $n_{1}p\gtrsim\mu\log n$.
With these preparations in place, the claims in Theorem~\ref{thm:mc-inf} follow
directly from Theorem~\ref{thm:2-inf-asymm} and the bound (\ref{eq:mc-entry-upper}) on $\|\bm{M}^{\star}\|_{\infty}$
(and hence on $\sigma$).
\end{proof}

%\begin{figure}
%\begin{center}
%	\includegraphics[width=0.55\textwidth]{figures/mc_inf_new.pdf} 
%\end{center}	
%	\caption{Numerical performance of spectral methods for matrix completion. Set $n=500$, $r=10$, and the underlying low-rank matrix $\bm{M}^{\star}$ as a product of two random orthonormal matrices. We vary the sampling probability $p$ from 0.25 to 0.9 with an equal space of 0.05. For each $p$, we draw 20 independent observation patterns and run the spectral method to recover the data matrix. The Euclidean error $\|\widehat{\bm{M}}-\bm{M}^{\star}\|_{\mathrm{F}}/\|\bm{M}^{\star}\|_{\mathrm{F}}$, the spectral norm error $\|\widehat{\bm{M}}-\bm{M}^{\star}\|/\|\bm{M}^{\star}\|$, and the $\ell_{\infty}$ norm error $\|\widehat{\bm{M}}-\bm{M}^{\star}\|_{\infty}/\|\bm{M}^{\star}\|_{\infty}$ are reported when averaged over 20 independent trials, where $\widehat{\bm{M}}=\bm{U}\bm{\Sigma}\bm{V}^{\top}$. 
%	\label{fig:mc-inf}}
%\end{figure}

In what follows, we compare the $\ell_{2,\infty}$ and $\ell_{2}$ performance
guarantees derived in the above theorem with those $\ell_{2}$ guarantees presented
in Section~\ref{sec:mc-l2-theory}; see Figure~\ref{fig:matrix-completion-motivation} in Section~\ref{sec:motivating_examples} for empirical performances. For the sake of brevity, we shall concentrate on the case where $\mu,\kappa,r=O(1)$.

\begin{itemize}

\item {\em $\ell_{2,\infty}$ singular space perturbation bounds. }
Comparing the $\ell_{2,\infty}$ performance guarantee (\ref{eq:mc-inf-claim-1}) with  Theorem~\ref{thm:mc-l2-subspace}, one sees that the $\ell_{2,\infty}$ perturbation bounds could be an order of $\sqrt{n_1}$ times smaller than the $\ell_{2}$ counterpart, showcasing the de-localization effect of the errors of the spectral estimates $\bm{U}$ and $\bm{V}$.

\item {\em Entrywise matrix estimation bounds. }
Furthermore, the entrywise error (\ref{eq:mc-inf-claim-2}) is about an order of $n_1$ times smaller than the corresponding Euclidean error predicted  in Theorem~\ref{thm:mc-l2-matrix}. This indicates that no entry in the resulting matrix estimate suffers from an error significantly higher than the average entrywise error.

\end{itemize}

%% file: chapters/community_detection_Linf.tex
\section{Application: Exact community recovery}
\label{sec:community-detection-linf}

Another application that benefits remarkably from the fine-grained eigenvector perturbation theory is community recovery. This section reexplores the stochastic block model studied in Section~\ref{sec:community-detection}, and develops significantly enhanced theoretical support for spectral clustering.

\subsection{Performance guarantees: Exact recovery}
\label{sec:performance-exact-recovery}

The focus of this section is simultaneous recovery of the community memberships of all vertices, 
which is termed {\em exact recovery} or {\em strong consistency} in the community detection literature \citep{abbe2017community}. This imposes a much stronger requirement than the {\em weak consistency} studied in Section~\ref{sec:theory-community-detection}. 

For the sake of conciseness, the theorem below concentrates on the challenging regime where $p,q\asymp \frac{\log n}{n}$,  corresponding to the lowest possible edge densities that allow for exact recovery. This is because, if $p < \frac{\log n}{n}$, then with high probability, one can find isolated vertices that are not connected with any edge in the graph \citep{durrett2007random}; hence, there will be absolutely no means to infer the community membership of these isolated vertices.  The theoretical guarantee is as follows. 
\begin{theorem}
	\label{thm:community-recovery-linf}
	Fix any constant $\varepsilon>0$, and consider the setting of Section~\ref{sec:setup-community-detection}. Suppose  $p=\frac{\alpha \log n}{n}$ and $q=\frac{\beta \log n}{n}$ for some sufficiently large constants $\alpha > \beta >0$.\footnote{In the current proof, the constants $\alpha$, $\beta$ might depend on the fixed choice of $\varepsilon$. Encouragingly, this restriction can be lifted; see~\citet{abbe2020entrywise} for details.} In addition, assume that
	\begin{align}
		\big( \sqrt{p} -\sqrt{q}\big)^2 \geq   2\left( 1+ \varepsilon  \right) \frac{\log n}{n} .
		\label{eq:H-pq-lower-bound-theorem}
	\end{align}
	With probability  $1-o(1)$, the spectral method in Section~\ref{sec:algorithm-community-detection} yields
	\[
		x_{i}=x_{i}^{\star}~~\text{for all }1\leq i\leq n, \quad~\text{or}\quad~ x_{i}=-x_{i}^{\star}~~\text{for all }1\leq i\leq n.
	\]
\end{theorem}
This theorem, which first appeared in \citet{abbe2020entrywise},  identifies a sufficient recovery condition in terms of the edge densities. 
The result substantially strengthens the $\ell_2$-based theory in Section~\ref{sec:theory-community-detection}, 
uncovering the capability of the spectral method in achieving not merely almost exact recovery in the average sense, 
but more appealingly,  exact community recovery that ensures correct labels of all vertices.  

A natural question arises as to whether the recovery condition \eqref{eq:H-pq-lower-bound-theorem} is improvable via more sophisticated algorithms. Answering this question requires information-theoretic thinking, that is, how to characterize a fundamental threshold---in terms of the difference of edge densities---below which exact recovery is deemed  infeasible. As has been demonstrated in \citet{abbe2014exact,mossel2015consistency,hajek2015achieving}, no algorithm whatsoever is able to achieve exact community recovery  if
\begin{equation}
	\big( \sqrt{p} -\sqrt{q}\big)^2 \leq  2(1-\varepsilon) \frac{\log n}{n} 
	\label{eq:IT-lower-bound-CD}
\end{equation}
for any constant $\varepsilon >0$. 
This fundamental lower bound, in conjunction with Theorem~\ref{thm:community-recovery-linf}, reveals a sharp phase transition behind the performance of the spectral method. In particular, its optimality is guaranteed all the way down to the information-theoretic threshold; see Figure~\ref{fig:sbm-inf} for numerical evidence.  

\begin{figure}[!t]
\begin{center}
	\includegraphics[width=0.55\textwidth]{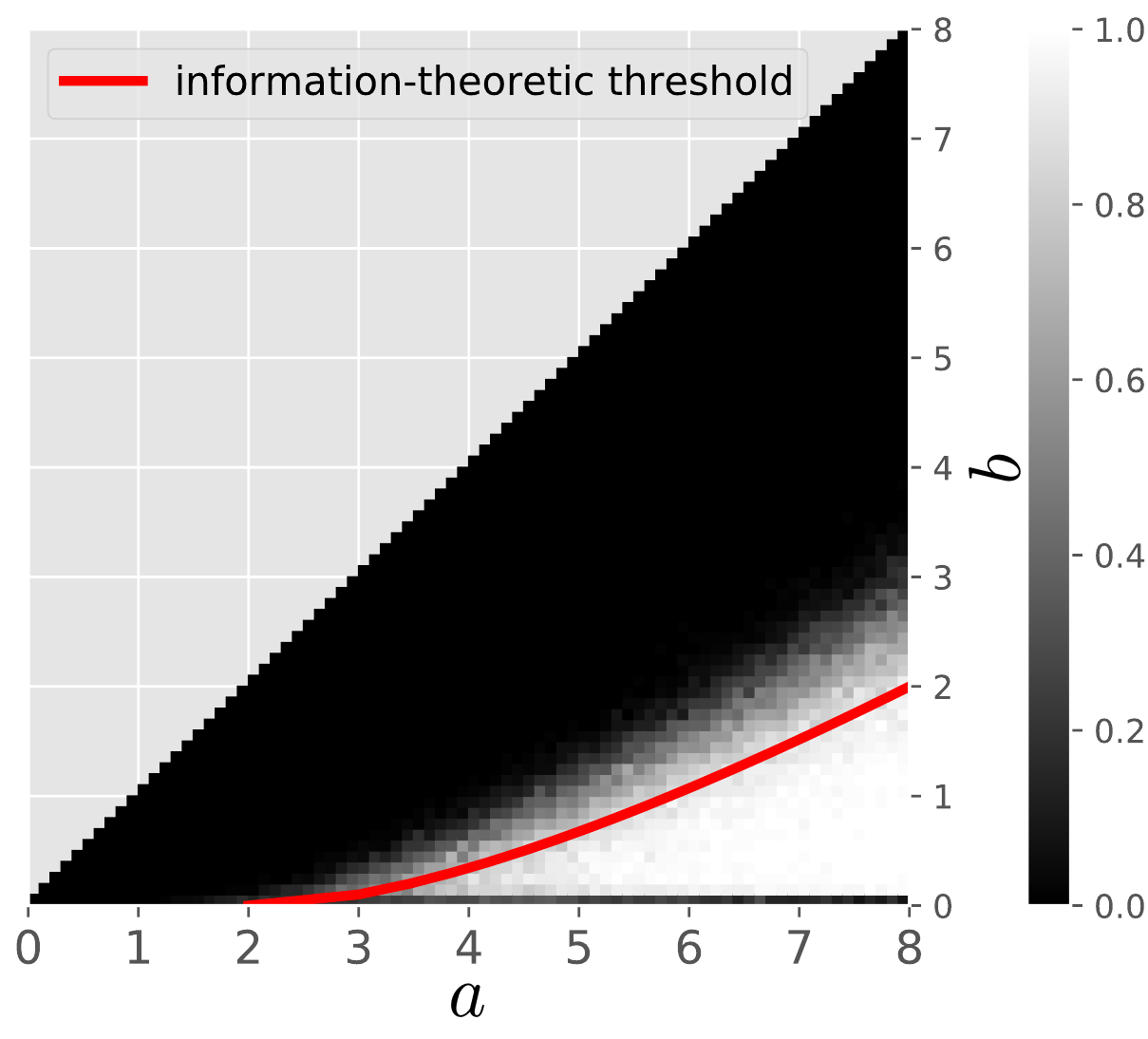} 
\end{center}	
	\caption{Phase transition of spectral methods for exact community recovery. Set $n=300$, $p=(a \log n) / n$, and $q = (b \log n) / n$. We vary $a,b$ from 0 to 8 with an equal space of 0.1. For each configuration of $a \geq b$, we conduct 100 Monte Carlo trials and report the empirical success rate for recovering the entire community structure correctly. 
	%White areas indicate complete success, while dark ones indicate complete failure. 
	The empirical phase transition occurs near the information-theoretic threshold (see \eqref{eq:IT-lower-bound-CD}).
	\label{fig:sbm-inf}}
\end{figure}

Given that the above information-theoretic threshold is specified in terms of $( \sqrt{p} -\sqrt{q})^2$,
the reader might naturally wonder what the operational meaning of this quantity is. As it turns out, this metric is a sort of distance measure between the two edge probability distributions under consideration. In truth, in the setting of Theorem~\ref{thm:community-recovery-linf}, this metric is intimately related to the squared Hellinger distance between two Bernoulli distributions.
\begin{definition}[Squared Hellinger distance]
\label{defn:Hellinger-distance}
Consider two distributions $P$ and $Q$ over a finite alphabet $\mathcal{Y}$. The squared Hellinger distance $\mathsf{H}^2(P\,\|\,Q)$ between $P$ and $Q$  is defined as follows	%
\begin{equation}
\mathsf{H}^2(P\,\|\,Q)\coloneqq\frac{1}{2}\sum\nolimits_{y\in \mathcal{Y}}\Big(\sqrt{P(y)}-\sqrt{Q(y)}\Big)^{2}.
\label{eq:defn-Hellinger-PQ}
\end{equation}
\end{definition}
In particular, consider the squared Hellinger distance between two Bernoulli distributions of interest $\mathsf{Bern}(p)$ and $\mathsf{Bern}(q)$, where we denote by $\mathsf{Bern}(p)$ the Bernoulli distribution with mean $p$. It is seen that \citep{chen2016information}
\begin{align*}
	\mathsf{H}^2\big( \mathsf{Bern}(p),   \mathsf{Bern}(q) \big) 
	&\coloneqq \frac{1}{2} \big( \sqrt{p} -\sqrt{q}\big)^2 + \frac{1}{2} \big( \sqrt{1-p} -\sqrt{1-q}\big)^2 \nonumber\\
	&= (1+o(1)) \frac{1}{2} \big( \sqrt{p} -\sqrt{q}\big)^2,
\end{align*}
when $p=o(1)$ and $q=o(1)$.\footnote{To justify this approximation, the following calculation suffices: \[ \sqrt{1-q}-\sqrt{1-p}=\frac{p-q}{\sqrt{1-p}+\sqrt{1-q}}=(1+o(1))\big(\sqrt{p}-\sqrt{q}\big)\big(\sqrt{p}+\sqrt{q}\big)=o\big(\sqrt{p}-\sqrt{q}\big).\]} 
The phase transition phenomenon identified in \eqref{eq:H-pq-lower-bound-theorem} and \eqref{eq:IT-lower-bound-CD} can then be alternatively described as
\begin{align*}
\text{\text{spectral method works}}\quad & \text{if }\mathsf{H}^{2}\big(\mathsf{Bern}(p),\mathsf{Bern}(q)\big)\geq(1+\varepsilon)\frac{\log n}{n}\\
\text{no algorithm works}\quad & \text{if }\mathsf{H}^{2}\big(\mathsf{Bern}(p),\mathsf{Bern}(q)\big)\leq(1-\varepsilon)\frac{\log n}{n}
\end{align*}
for an arbitrary small constant $\varepsilon >0$.

\subsection{Proof of Theorem~\ref{thm:community-recovery-linf}} 

We now turn to the proof of Theorem~\ref{thm:community-recovery-linf}. 
Without loss of generality, suppose that $x_i^{\star}=1$ for all $1\leq i\leq n/2$  and $x_i^{\star}=-1$ for all $i> n/2$, so that 
$\bm{u}^{\star} =\frac{1}{\sqrt{n}}  {\small
\left[\begin{array}{c}
\bm{1}_{n/2}\\
-\bm{1}_{n/2}
\end{array}\right] }$.

Recalling the matrix $\bM$ given in \eqref{eq:M-data-matrix-SBM} and its mean $\bM^{\star}$ in \eqref{eq:M-data-matrix-expectation-SBM}, 
one can immediately see that $\kappa=\mu=r=1$ for $\bM^{\star}$ in this application. 
Theorem \ref{thm:UsgnH-Ustar-MUstar-general} (cf.~\eqref{eq:UsgnH-MUstar-bound-theorem-general}) readily implies 
the existence of some $z\in\{1,-1\}$ such that
\begin{align}
 & \Big\| z\bm{u}-\frac{1}{\lambda^{\star}}\bm{M}\bm{u}^{\star}\Big\|_{\infty}
\lesssim\frac{\sigma}{|\lambda^{\star}|}+\frac{\sigma^{2}\sqrt{n\log n}+\sigma B \,{\log^{3/2}n}}{(\lambda^{\star})^{2}} 
\label{eq:implication-Thm-symmetry-1st-CD}
\end{align}
with probability at least $1-O(n^{-5})$. Additionally, it has already been explained in Section~\ref{sec:community-detection} that  
\[
	 B=1,\quad\sigma^{2}\leq\max\{p,q\}=p, \quad \text{and} \quad \lambda^{\star}={n(p-q)}/{2}.
\]
Substitution into \eqref{eq:implication-Thm-symmetry-1st-CD} reveals that
\begin{align}
	\big\| z\lambda^{\star}\bm{u}-\bm{M}\bm{u}^{\star} \big\|_{\infty} & \lesssim\sigma+\frac{\sigma^{2}\sqrt{n\log n}}{\lambda^{\star}}+\frac{\sigma B\,\log^{3/2}n}{\lambda^{\star}}\nonumber \\
 	& \leq  C \Big( \sqrt{p}+\frac{p\sqrt{\log n}}{\sqrt{n}(p-q)}+\frac{\sqrt{p}\log^{3/2}n}{n(p-q)} \Big)
	\label{eq:z-lambda-u-MU-bound-CD}
\end{align}
holds for some universal constant $C>0$.
As a result, a crucial step boils down to controlling $\bm{M}\bm{u}^{\star}$ in an entrywise manner: each element is a difference between two independent random binomial random variables and  is accomplished through the following lemma. 
\begin{lemma}
	\label{lemma:M-ustar-lower-bound-CD}
	Suppose that
\begin{align}
	\mathsf{H}_{p,q}^{2} \coloneqq \big( \sqrt{p} -\sqrt{q}\big)^2 \geq\left( 1+ \varepsilon  \right) \frac{2\log n}{n} 
	\label{eq:H-pq-lower-bound-lemma}
\end{align}
for some quantity $\varepsilon>0$.  Let $\varepsilon_0 \coloneqq \frac{\varepsilon\log n}{\sqrt{n}\log\frac{p(1-q)}{q(1-p)}}-\frac{1}{\sqrt{n}}$. 
Then with probability exceeding $1-n^{-\varepsilon/2}$, one has
\begin{align*}
	 \bm{M}_{l,\cdot}\bm{u}^{\star}   \geq \varepsilon_0 \,\,\,\text{for all }l\leq \frac{n}{2},
	\quad\text{and}\quad
	   \bm{M}_{l,\cdot} \bm{u}^{\star}  
	\leq - \varepsilon_0 \,\,\, \text{for all } l > \frac{n}{2}. 
	%\label{eq:M-ustar-bound-CD-lemma}
\end{align*}
\end{lemma}

We now return to analyze the entrywise behavior of $\bm{u}$. Note that
\begin{align*}
(z\lambda^{\star})u_{l} & \geq\bm{M}_{l,\cdot}\bm{u}^{\star}-\big| z\lambda^{\star}u_{l}-\bm{M}_{l,\cdot}\bm{u}^{\star}\big|  \quad\text{for all } l\leq n/2;\\
(z\lambda^{\star})u_{l} & \leq\bm{M}_{l,\cdot}\bm{u}^{\star}+\big| z\lambda^{\star}u_{l}-\bm{M}_{l,\cdot}\bm{u}^{\star}\big|  \quad\text{for all }l>n/2.
\end{align*}
This together with \eqref{eq:z-lambda-u-MU-bound-CD}, Lemma~\ref{lemma:M-ustar-lower-bound-CD} and $\lambda^{\star}>0$ yields that if
\begin{align}
	\frac{\varepsilon\log n}{\sqrt{n}\log\frac{p(1-q)}{q(1-p)}}  > \frac{1}{\sqrt{n}} 
	+  C \Big( \sqrt{p} + \frac{p\sqrt{\log n}}{\sqrt{n}(p-q)} + \frac{\sqrt{p}\log^{3/2}n}{n(p-q)} \Big), 
	\label{eq:epsilon-condition-CD-inf-123}
\end{align}
then it follows that
\[
zu_{l}>0\quad\text{for all }1\leq l\leq\frac{n}{2},\quad\text{and}\quad zu_{l}<0\quad\text{for all }l>\frac{n}{2},
\]
thus guaranteeing exact community recovery once the rounding procedure (based on the sign) is applied.

To finish up, it remains to validate Condition~\eqref{eq:epsilon-condition-CD-inf-123}. 
Fixing $\varepsilon>0$ to be a constant, we make the following observations. 
\begin{itemize}
	\item From the assumptions $\varepsilon \asymp 1$ and $p,q\asymp \frac{\log n}{n}$ (or $\alpha,\beta \asymp 1$), one has
	\[
		\frac{\varepsilon\log n}{\sqrt{n}\log\frac{p(1-q)}{q(1-p)}}=\frac{\varepsilon\log n}{\sqrt{n}\log\frac{(1+o(1))\alpha}{\beta}}
		\asymp \frac{\log n}{\sqrt{n}}\gg \sqrt{p}+\frac{1}{\sqrt{n}}.
	\]
	\item Turning to the term $\frac{p\sqrt{\log n}}{\sqrt{n}(p-q)}$,   we  observe that 
	\begin{align}
		& \log\frac{p(1-q)}{q(1-p)}=\log\Big(1+\frac{p-q}{q(1-p)}\Big)\leq\frac{p-q}{q(1-p)}\leq\frac{2(\alpha-\beta)}{\beta}, \notag\\
		& \qquad\quad \Longrightarrow \quad 
		\frac{\varepsilon\log n}{\sqrt{n}\log\frac{p(1-q)}{q(1-p)}}\geq\frac{\varepsilon\log n}{2\sqrt{n}}\,\frac{\beta}{\alpha-\beta},
		\label{eq:LB-CD-12345}
	\end{align}
	where in the last inequality of the  first line we have used the assumption $p=o(1)$ and hence $1/(1-p) \leq 2$. 
	Given that $\varepsilon, \alpha,\beta \asymp 1$, it is guaranteed that
	\[
		\eqref{eq:LB-CD-12345} \gg \frac{\alpha\sqrt{\log n}}{\sqrt{n}(\alpha-\beta)}=\frac{p\sqrt{\log n}}{\sqrt{n}(p-q)} .
	\]
	\item We then move on to the term $\frac{\sqrt{p}\log^{3/2}n}{n(p-q)}$. 
	If $\alpha / \beta \leq 2$ and  $\beta \geq\frac{200C^{2}}{\varepsilon^{2}} \geq \frac{100C^{2}\alpha/\beta}{\varepsilon^{2}}$, then one  has $\beta \geq \frac{10C\sqrt{\alpha}}{\varepsilon}$ and hence
	\[
		\eqref{eq:LB-CD-12345} \geq\frac{5C\sqrt{\alpha}\log n}{\sqrt{n}(\alpha-\beta)}=\frac{5C\sqrt{p}\log^{3/2}n}{n(p-q)} .
	\]
	 In addition, if $\alpha / \beta > 2$, then it follows that
	\begin{equation}
		\frac{\sqrt{p}\log^{3/2}n}{n(p-q)}=\frac{\sqrt{\alpha}\log n}{\sqrt{n}(\alpha-\beta)}\leq\frac{2\sqrt{\alpha}\log n}{\sqrt{n}\alpha}=\frac{2\log n}{\sqrt{n\alpha}},
		\label{eq:LB-CD-5678}
	\end{equation}
	where the inequality holds true since $\alpha - \beta > \alpha - \alpha/2 = \alpha / 2$. 
	Using the basic inequality $\log x \leq \sqrt{x}$ further leads  to
	\[
		\frac{\varepsilon\log n}{\sqrt{n}\log\frac{p(1-q)}{q(1-p)}}\geq\frac{\varepsilon\log n}{\sqrt{n}\log\frac{2\alpha}{\beta}}
		\geq\frac{\varepsilon\sqrt{\beta}\log n}{\sqrt{n}\sqrt{2\alpha}}\geq\frac{5C\sqrt{p}\log^{3/2}n}{n(p-q)}.
	\]
	Here, the first inequality holds since $p,q=o(1)$ and hence $\frac{1-q}{1-p}\leq 2$, 
	whereas the last relation relies on \eqref{eq:LB-CD-5678} and holds with the proviso that $\beta\geq200C^{2}/\varepsilon^{2}$.

\end{itemize}
The above calculations taken collectively establish Condition~\eqref{eq:epsilon-condition-CD-inf-123} under the assumptions of Theorem~\ref{thm:community-recovery-linf}, thus concluding the proof.

\begin{remark}
	It is worth pointing out that the bound \eqref{eq:UsgnH-Ustar-bound-theorem-general} in Theorem \ref{thm:UsgnH-Ustar-MUstar-general} is not sufficiently tight when establishing this result. 
	Instead, one needs to resort to the more refined bound \eqref{eq:UsgnH-MUstar-bound-theorem-general} in Theorem \ref{thm:UsgnH-Ustar-MUstar-general}, 
	which allows us to sharpen the error bound by explicitly accounting for the first-order error term $(\bm{M}-\bm{M}^{\star})\bm{u}^{\star}$. 
\end{remark}

\subsection{Proof of auxiliary lemmas}
\label{sec:proof-auxiliary-CD}

Before embarking on the proof of Lemma~\ref{lemma:M-ustar-lower-bound-CD}, we first record non-asymptotic tail bounds concerning log-likelihood ratios and a sum of Bernoulli random variables, which make apparent the role of the squared Hellinger distance \citep{tsybakov2009nonparm}.
\begin{lemma}
	\label{lemma:LLR-Hellinger}
	Consider two distributions $P$ and $Q$ over a finite alphabet $\mathcal{Y}$, and suppose that $P(y)\neq 0$ for all $y\in \mathcal{Y}$. Generate an independent sequence $\{y_i\}_{1\leq i\leq n}$ obeying $y_i \sim P$. Then for any $\zeta\in \mathbb{R}$ one has
\begin{equation}
\mathbb{P}\left\{ \sum_{i=1}^{n}\log\frac{Q(y_{i})}{P(y_{i})}\geq-n\zeta\right\} \leq\exp\left(-n\Big[\mathsf{H}^2(P\,\|\,Q)-\frac{\zeta}{2}\Big]\right),
\label{eq:prob-log-PQ-exp-Hel}
\end{equation}
where $\mathsf{H}^2(P\,\|\,Q)$ is the squared Hellinger distance between $P$ and $Q$ defined in \eqref{eq:defn-Hellinger-PQ}.
\end{lemma}
\begin{lemma}
\label{lem:sum-Bernoulli-CD}
Consider two sequences of independent random variables
\[
z_{i}\sim\mathsf{Bern}(p),\qquad w_{i}\sim\mathsf{Bern}(q),\qquad1\leq i\leq n,
\]
and suppose that $p>q$. For any $\xi\in \mathbb{R}$, it follows that
\begin{align*}
\mathbb{P}\left\{ \sum_{i=1}^{n}(z_{i}-w_i) \leq n\xi\right\}  & \leq\exp\left(-n\Big[\mathsf{H}_{p,q}^{2}-\frac{\xi}{2}\log\frac{p(1-q)}{q(1-p)}\Big]\right) ,
\end{align*}
where $\mathsf{H}_{p,q}^{2}  \coloneqq  \big(\sqrt{p}-\sqrt{q} \, \big)^{2}$.
%
%\begin{align}
%	.
%	\label{eq:defn-Hpq}
%\end{align}
% 
\end{lemma}
In what follows, we first establish Lemmas~\ref{lemma:LLR-Hellinger} and \ref{lem:sum-Bernoulli-CD}, and then return to prove Lemma~\ref{lemma:M-ustar-lower-bound-CD}.

\paragraph{Proof of Lemma~\ref{lemma:LLR-Hellinger}.}
Apply the Chernoff bound to yield 
\begin{align}
 & \mathbb{P}\left\{ \sum_{i=1}^{n}\log\frac{Q(y_{i})}{P(y_{i})}\geq-n\zeta\right\} {\leq}\frac{\mathbb{E}_{y_{i}\sim P}\left[\exp\left(\frac{1}{2}\sum_{i=1}^{n}\log\frac{Q(y_{i})}{P(y_{i})}\right)\right]}{\exp(-n\zeta/2)} \notag\\
 & \qquad \qquad\overset{(\mathrm{i})}{=}
 %\frac{\prod_{i=1}^{n}\mathbb{E}_{y_{i}\sim P}\left[\exp\left(\frac{1}{2}\log\frac{Q(y_{i})}{P(y_{i})}\right)\right]}{\exp(-n\zeta/2)}
 %= 
 \frac{\left(\mathbb{E}_{y\sim P}\Big[\left(\frac{Q(y)}{P(y)}\right)^{1/2}\Big]\right)^{n}}{\exp(-n\zeta/2)},
\label{eq:prob-log-PQ-UB}
\end{align}
where 
%(i) results from the Chernoff bound, and (ii) 
(i) holds due to the i.i.d.~assumption of the $y_{i}$'s.  In addition,
%It is seen from the definition of $P$ and $Q$ that
%
\begin{align}
\mathbb{E}_{y\sim P}\left[\left(\frac{Q(y)}{P(y)}\right)^{1/2}\right] & =\sum_{y}P(y)\left(\frac{Q(y)}{P(y)}\right)^{1/2}=\sum_{y}\sqrt{P(y)Q(y)} \notag\\
 & =1-\frac{1}{2}\sum_{y}\left(P(y)+Q(y)-2\sqrt{P(y)Q(y)}\right) \notag\\
  & =1-\frac{1}{2}\sum_{y}\left( \sqrt{P(y)} - \sqrt{Q(y)} \right)^2 \notag\\
 & =1-\mathsf{H}^2(P\,\|\,Q)\leq  \exp \big(-\mathsf{H}^2(P\,\|\,Q) \big),
 \label{eq:P-Q-Hellinger-bound}
\end{align}
where the second line follows since $\sum_yP(y) = \sum_yQ(y)=1$, and  the last line uses the definition \eqref{eq:defn-Hellinger-PQ} and the elementary inequality $1-x\leq \exp(-x)$. 
Substituting \eqref{eq:P-Q-Hellinger-bound} into \eqref{eq:prob-log-PQ-UB} concludes the proof.

\paragraph{Proof of Lemma~\ref{lem:sum-Bernoulli-CD}.}

Set  $y_i\coloneqq z_i - w_i$. The proof is built upon a mapping between $\sum_{i=1}^{n}y_{i}$ and a certain log-likelihood ratio. 
Specifically, let us introduce two distributions $P$ and $Q$ supported on $\{1,0,-1\}$:
\[
P(x)=\begin{cases}
p(1-q),\quad & \text{if }x=1,\\
q(1-p), & \text{if }x=-1,\\
pq+(1-p)(1-q), \quad & \text{if }x=0,
\end{cases} 
%\qquad \text{and}
\]
\[
Q(x)=\begin{cases}
q(1-p),\quad & \text{if }x=1,\\
p(1-q), & \text{if }x=-1,\\
pq+(1-p)(1-q), \quad & \text{if }x=0.
\end{cases}
\]
Apparently, $P$ (resp.~$Q$) corresponds to the distribution of $y_i$ (resp.~$-y_i$).  A key observation is that
\begin{align*}
\sum_{i=1}^{n}\log\frac{Q(y_{i})}{P(y_{i})} & =\sum_{i=1}^{n}\left\{ \mathbbm1\{y_{i}=1\}\log\frac{q(1-p)}{p(1-q)}+\mathbbm1\{y_{i}=-1\}\log\frac{p(1-q)}{q(1-p)}\right\} \\
 & =\sum_{i=1}^{n}y_{i}\log\frac{q(1-p)}{p(1-q)} ,
\end{align*}
which relies on the fact that $y_i$ is supported on $\{1,0,-1\}$. 
Recognizing that $\log\frac{p(1-q)}{q(1-p)}>0$ holds as long as
$p>q$ (since $q(1-p)<p(1-q)$), we can further derive
\begin{align*}
\mathbb{P}\left\{ \sum_{i=1}^{n}y_{i}\leq n\xi\right\}  & =\mathbb{P}\left\{ \frac{1}{\log\frac{p(1-q)}{q(1-p)}}\sum_{i=1}^{n}\log\frac{Q(y_{i})}{P(y_{i})}\geq-n\xi\right\} \\
 & \leq\exp\left(-n\Big[\mathsf{H}^2(P\,\|\,Q)-\frac{\xi}{2}\log\frac{p(1-q)}{q(1-p)}\Big]\right),
\end{align*}
where the last inequality comes from Lemma~\ref{lemma:LLR-Hellinger}. 
From the constructions of $P$ and $Q$ and the definition \eqref{eq:defn-Hellinger-PQ} of $\mathsf{H}^2(P\,\|\,Q)$, it is easily seen that 
\begin{align*}
\mathsf{H}^{2}(P\,\|\,Q) & =\Big(\sqrt{p(1-q)}-\sqrt{q(1-p)}\,\Big)^{2}
	= \frac{\big(p-q\big)^{2}}{\big(\sqrt{p(1-q)}+\sqrt{q(1-p)}\,\big)^{2}}\\
 & \geq\frac{\big(p-q\big)^{2}}{\big(\sqrt{p}+\sqrt{q}\,\big)^{2}}=\big(\sqrt{p}-\sqrt{q}\big)^{2} ,
\end{align*}
thus concluding the proof.

\paragraph{Proof of Lemma~\ref{lemma:M-ustar-lower-bound-CD}.}

Let us start by looking at the first entry of $\bm{M}\bm{u}^{\star}$. 
It is seen from the construction \eqref{eq:M-data-matrix-SBM} that
\begin{align}
\bm{M}_{1,\cdot}\bm{u}^{\star}=\bm{A}_{1,\cdot}\bm{u}^{\star}-\frac{p+q}{2}\big(\bm{1}^{\top}\bm{u}^{\star}\big)\bm{1}+pu_{1}^{\star}
\geq \bm{A}_{1,\cdot}\bm{u}^{\star} ,
%+ O\Big(\frac{1}{\sqrt{n}}\Big),
	\label{eq:connection-M1u-A1u-CD}
\end{align}
where we have used the fact that $\bm{1}^{\top}\bm{u}^{\star}=0$  and $u_1^{\star}>0$. 
The expression  $\bm{u}^{\star} =\frac{1}{\sqrt{n}}  {\small
\left[\begin{array}{c}
\bm{1}_{n/2}\\
-\bm{1}_{n/2}
\end{array}\right] }$ admits the following decomposition
\begin{align}
	\bm{A}_{1,\cdot}\bm{u}^{\star}
%=\frac{1}{\sqrt{n}}\sum_{i=2}^{n/2}\big(A_{1,i}-A_{1,i+n/2}\big)+\frac{1}{\sqrt{n}}A_{1,n/2+1}
	= \frac{1}{\sqrt{n}}\sum\nolimits_{i=1}^{n/2}\big(A_{1,i}-A_{1,i+n/2}\big) , 
	\label{eq:A1ustar-decompose-CD}
\end{align}
which can be controlled via Lemma~\ref{lem:sum-Bernoulli-CD}. 
%Here and throughout, we denote by $\mathsf{Bern}(p)$  the Bernoulli distribution with mean $p$.

Observe that $A_{1,i} \sim \mathsf{Bern}(p)$ for all $1< i\leq n/2$ and $A_{1,i} \sim \mathsf{Bern}(q)$ otherwise. Using the definitions of $z_i$ and $w_i$ in Lemma~\ref{lem:sum-Bernoulli-CD}, we obtain
\begin{align}
 & \mathbb{P}\left\{ \sum\nolimits_{i=1}^{n/2}\big(A_{1,i}-A_{1,i+n/2}\big)\leq\frac{n\zeta}{2}-1\right\}  
	\leq\mathbb{P}\left\{ \sum\nolimits_{i=1}^{n/2}\big(z_{i}-w_{i}\big)\leq\frac{n\zeta}{2}\right\}  \notag\\
	& \qquad \leq\exp\left(-\frac{n}{2}\Big[\mathsf{H}_{p,q}^{2}-\frac{\zeta}{2}\log\frac{p(1-q)}{q(1-p)}\Big]\right) \leq \frac{1}{n^{1+\delta}}
	\label{eq:A-sum-complicated-CD}
\end{align}
for some $\delta >0$, where the first inequality follows since $A_{1,1}=0\leq z_1+1$ (so that $A_{1,1}-A_{1,n/2+1}$ is stochastically dominated by $z_1 - w_1$), and the last inequality holds as long as
\begin{equation}
	\mathsf{H}_{p,q}^{2}-\frac{\zeta}{2}\log\frac{p(1-q)}{q(1-p)}
	\geq\frac{2(1+\delta)\log n}{n} ,
	\label{eq:H-pq-lower-bound-123}
\end{equation}
which we shall ensure at the end of the proof. Substituting \eqref{eq:A-sum-complicated-CD} into \eqref{eq:A1ustar-decompose-CD} and \eqref{eq:connection-M1u-A1u-CD} yields
\begin{align*}
\mathbb{P}\left\{ \bm{M}_{1,\cdot}\bm{u}^{\star}\leq \frac{n\zeta-2}{2\sqrt{n}} \right\}  
%& \leq\mathbb{P}\left\{ \bm{A}_{1,\cdot}\bm{u}^{\star}\leq\varepsilon_{0}-1/\sqrt{n}\right\} \\
 & \leq \mathbb{P}\left\{ \bm{A}_{1,\cdot}\bm{u}^{\star}\leq\frac{n\zeta-2}{2\sqrt{n}} \right\} \leq\frac{1}{n^{1+\delta}}.
\end{align*}
%
%which combined with \eqref{eq:connection-M1u-A1u-CD} leads to
% and the assumption ${1}/{\lambda^{\star}} = o(\varepsilon_0)$ leads to
%
%\[
%	\mathbb{P}\left\{ \bm{M}_{1,\cdot}\bm{u}^{\star} \leq  {\varepsilon_0}  \right\} \leq
%	\mathbb{P}\left\{ \bm{A}_{1,\cdot}\bm{u}^{\star} \leq  {\varepsilon_0}   \right\} 
%	\leq\frac{1}{n^{1+\delta}}.
%\]
%
Repeating the preceding analysis for $ \bm{M}_{l,\cdot}\bm{u}^{\star}$ with other  $l$'s  and taking the union bound, we see that with probability at least $1- n^{-\delta}$,
\begin{subequations}
	\label{eq:M-ustar-bound-CD}
\begin{align}
%\begin{split}
% & \frac{1}{\lambda^{\star}}\bm{A}_{l,\cdot}\bm{u}^{\star}\geq\frac{\sqrt{n}\zeta}{4\lambda^{\star}}\ \qquad\,\text{for all }l\leq\frac{n}{2}\\
% & \frac{1}{\lambda^{\star}}\bm{A}_{l,\cdot}\bm{u}^{\star}\leq-\frac{\sqrt{n}\zeta}{4\lambda^{\star}}\ \ \quad \text{for all }l>\frac{n}{2}
%\end{split}
	\bm{M}_{l,\cdot}\bm{u}^{\star} &\geq \frac{n\zeta-2}{2\sqrt{n}}
	%\frac{\sqrt{n}\zeta}{4\lambda^{\star}}
	\quad  &&\text{if } l\leq {n}/{2} \\
	%\quad
	%\text{and}\quad
	\bm{M}_{l,\cdot}\bm{u}^{\star}
	&\leq - \frac{n\zeta-2}{2\sqrt{n}}
	%-\frac{\sqrt{n}\zeta}{4\lambda^{\star}}
	\quad  &&\text{if } l> n/2
	%\big)
\end{align}
\end{subequations}
%
%\begin{equation}
%\begin{array}{cc}
%	& \frac{1}{\lambda^{\star}}\bm{M}_{l,\cdot}\bm{u}^{\star}  \geq\frac{\sqrt{n}\zeta-2}{2\lambda^{\star}},\ \quad~~\text{}l\leq\frac{n}{2}\\
%	& \frac{1}{\lambda^{\star}}\bm{M}_{l,\cdot}\bm{u}^{\star}  \leq-\frac{\sqrt{n}\zeta-2}{2\lambda^{\star}},\ \ ~~\text{}l>\frac{n}{2}
%\end{array}
%\end{equation}
%
hold simultaneously for all $1\leq l\leq n$.

Finally, it remains to ensure satisfaction of \eqref{eq:H-pq-lower-bound-123}. As it turns out, if the condition \eqref{eq:H-pq-lower-bound-lemma}
holds, then it suffices to take $\delta\leq \varepsilon / 2$ and 
$\zeta=\frac{2\varepsilon\log n}{n\log\frac{p(1-q)}{q(1-p)}}$. This completes the proof.

%% file: chapters/distribution_theory.tex
\section{Distributional theory and uncertainty quantification}
\label{sec:distribution-theory}

Thus far, we have demonstrated intriguing statistical performance of estimators developed based on spectral methods. 
As one can anticipate, the quality of a spectral estimator is largely affected by the imperfectness of data generating mechanisms (e.g., noise corruption, missing data). 
The uncertainty of the estimator due to these factors would inevitably influence any subsequent decision making based on it. 
Viewed in this light, it is recommended to accompany the estimator in hand with valid measures of uncertainty (or ``confidence''), 
in order to better inform decision makers.

Take the low-rank matrix estimation problem in Section~\ref{sec:setup-general-theory} for instance:  
an important uncertainty quantification task can be posed as the construction of a valid confidence
interval---based on the spectral estimator---that is likely to cover an unseen entry of the matrix of interest $\bm{M}^{\star}$. 
More precisely, for any location $(i,j)$ and any target coverage level $1-\alpha \in (0,1)$ (e.g., 95\%), we aim to identify a short interval---denoted by $\mathsf{CI}_{i,j}^{1-\alpha}$---based on the spectral estimator such that
\begin{equation}
	\mathbb{P}\big( M_{i,j}^{\star} \in \mathsf{CI}_{i,j}^{1-\alpha} \big) \approx 1-\alpha,  
\end{equation}
which essentially augments a point estimate into an interval that is guaranteed to cover the unknown with the pre-specified target probability.  
Note that the problem of constructing a valid confidence interval falls within the realm of {\em statistical inference} in the statistics literature,
which constitutes an important step beyond statistical estimation. 
Accomplishing this task in high dimension often calls for a refined statistical reasoning toolbox that offers quantitative distributional characterizations of the estimator. 

%The $\ell_{\infty}$ and $\ell_{2,\infty}$ theory presented in this chapter becomes particularly effective for this purpose.    

%calls for a modern suite
%of statistical reasoning tools, with the aim of providing comprehensive uncertainty evaluations for the estimation
%algorithms in use.
%i
%

%quantitative characterization of 

%statistical validity 

\subsection{Entrywise distributional guarantees}
\label{sec:distribution-general}

As a natural starting point to build confidence intervals, 
we seek to develop comprehensive understanding about the distribution of the spectral estimator. 
In general, obtaining a non-asymptotic yet tractable distributional characterization of a nonconvex estimator like the spectral method 
could be remarkably challenging. 
Fortunately, the $\ell_{\infty}$ and $\ell_{2,\infty}$ perturbation theory introduced previously (e.g., Theorem~\ref{thm:UsgnH-Ustar-MUstar-general}) allows one to make progress for some important scenarios.

Let us revisit the setting in Section~\ref{sec:setup-general-theory}, and consider the following estimator of the unknown low-rank matrix $\bm{M}^{\star}$:
\begin{align}
	\widehat{\bm{M}}=\big[\widehat{M}_{i,j} \big]_{1\leq i,j\leq n}=\bm{U}\bm{\Lambda}\bm{U}^{\top} ,
	\label{eq:estimator-M-hat-distribution}
\end{align}
obtained via the spectral method. The aim is to develop tractable distributional guarantees for each entry of $\widehat{\bm{M}}-\bm{M}^{\star}$.

Towards this end, we first examine whether our previous results shed light on certain distributional properties of $\widehat{\bm{M}}-\bm{M}^{\star}$. 
Informally,  Theorem~\ref{thm:UsgnH-Ustar-MUstar-general} (in particular, \eqref{eq:UsgnH-MUstar-bound-theorem-general}) reveals that
\begin{align}
	\bm{U}\mathsf{sgn}(\bm{H}) \approx \bm{M} \bm{U}^{\star} \big(\bm{\Lambda}^{\star}\big)^{-1} 
	= \bm{U}^{\star} + \bm{E} \bm{U}^{\star} \big(\bm{\Lambda}^{\star}\big)^{-1}. 
	\label{eq:U-sgn-H-first-order-approx-discuss}
\end{align}
Assuming tightness of this first-order approximation, one further derives 
\begin{align}
\bm{U}\bm{\Lambda}\bm{U}^{\top}-\bm{M}^{\star} & \overset{(\mathrm{i})}{\approx}\bm{U}\mathsf{sgn}(\bm{H})\bm{\Lambda}^{\star}\big(\bm{U}\mathsf{sgn}(\bm{H})\big)^{\top}-\bm{U}^{\star}\bm{\Lambda}^{\star}\bm{U}^{\star\top} \notag\\
 & \overset{(\mathrm{ii})}{\approx}\big(\bm{U}\mathsf{sgn}(\bm{H})-\bm{U}^{\star}\big)\bm{\Lambda}^{\star}\bm{U}^{\star\top}+\bm{U}^{\star}\bm{\Lambda}^{\star}\big(\bm{U}\mathsf{sgn}(\bm{H})-\bm{U}^{\star}\big)^{\top} 
 \notag\\
 & \overset{(\mathrm{iii})}{\approx}\bm{E}\bm{U}^{\star}\big(\bm{\Lambda}^{\star}\big)^{-1}\bm{\Lambda}^{\star}\bm{U}^{\star\top}+\bm{U}^{\star}\bm{\Lambda}^{\star}\big(\bm{E}\bm{U}^{\star}\big(\bm{\Lambda}^{\star}\big)^{-1}\big)^{\top} \notag\\
 & =\bm{E}\bm{U}^{\star}\bm{U}^{\star\top}+\bm{U}^{\star}\bm{U}^{\star\top}\bm{E}, 
	\label{eq:approx-ULambdaU-Mstar-discuss}
\end{align}
where (i) holds as long as $\mathsf{sgn}(\bm{H})\bm{\Lambda}^{\star} \mathsf{sgn}(\bm{H})^{\top} \approx \bm{\Lambda}$ (which has already been illuminated in the analysis 
of Corollary~\ref{cor:entrywise-error-general} and will be solidified momentarily),  
(ii) is obtained by dropping the higher-order term $\big(\bm{U}\mathsf{sgn}(\bm{H})-\bm{U}^{\star}\big)\bm{\Lambda}^{\star}\big(\bm{U}\mathsf{sgn}(\bm{H})-\bm{U}^{\star}\big)^{\top}$, and (iii) relies upon the approximation \eqref{eq:U-sgn-H-first-order-approx-discuss}.

Given that \eqref{eq:approx-ULambdaU-Mstar-discuss} is a linear map of the noise matrix $\bm{E}$, 
this essentially forms a first-order approximation of $\widehat{\bm{M}}$,
which in turn enables a tractable distributional theory for $\widehat{\bm{M}}$. 
Observe that each entry of the matrix in \eqref{eq:approx-ULambdaU-Mstar-discuss} 
is a weighted superposition of the independent zero-mean entries of $\bm{E}$.   
Equipped with this observation, some variant of the central limit theorem
suggests that each entry of $\widehat{\bm{M}} - \bm{M}^{\star}$ is approximately zero-mean Gaussian,
as formalized by the theorem below. For notational convenience, we shall define a projection matrix
	\begin{equation}
	%\bm{P}\coloneqq\bm{U}^{\star}\bm{U}^{\star\top},
		\bm{P}^{\star}= \big[ P_{i,j}^{\star} \big]_{1\leq i,j\leq n} \coloneqq \bm{U}^{\star}\bm{U}^{\star\top},
		\label{eq:defn-P-UU-T}
	\end{equation}	
and impose a lower bound requirement on the noise variance: 
\begin{equation}
	 \sigma_{\min}^2 \leq \sigma_{i,j}^2 \leq \sigma^2, \qquad 1\leq i,j\leq n. 
	 \label{eq:sigma-min-definition}
\end{equation}
\begin{theorem}
\label{thm:matrix-distribution-complete}
Suppose that the assumptions of Theorem~\ref{thm:UsgnH-Ustar-MUstar-general} hold. For any $1\leq i,j\leq n$, set
\begin{align}
	v_{i,j}^{\star}=
		\begin{cases}
			\sum_{l=1}^{n}\sigma_{i,l}^{2}P^{\star 2}_{l,j}+\sum_{l=1}^{n}P^{\star 2}_{i,l}\sigma_{l,j}^{2}+2\sigma_{i,j}^{2}P^{\star}_{i,i}P^{\star}_{j,j}, & \text{if }i\neq j,\\
			4\sum_{l=1}^{n}\sigma_{i,l}^{2}P^{\star 2}_{l,i}, & \text{if }i=j. 
		\end{cases}
	\label{eq:variance-formula-Mij-lem}
\end{align}
Assume that $\sigma/\sigma_{\min}=O(1)$, and that
\begin{equation}
	\frac{\big\|\bm{U}_{j,\cdot}^{\star}\big\|_{2}^{2}+\big\|\bm{U}_{i,\cdot}^{\star}\big\|_{2}^{2}}{\big\|\bm{U}^{\star}\big\|_{\mathrm{F}}^{2}}
	\gtrsim \frac{B^{2}\kappa^{2}\mu^{2}r^{2}\log^{2}n}{\sigma^{2}n^{2}}+\frac{\sigma^{2}\mu^{2}r\kappa^{4}\log^{3}n}{(\lambda_{r}^{\star})^{2}}.
	%\frac{B^{2}\mu^{2}r\log n}{n^{2}\sigma^{2}}+\frac{\sigma^{2}\mu^{2}r\kappa^{4}\log^{3}n}{(\lambda_{r}^{\star})^{2}}+\frac{\kappa^{2}\mu^{2}r^{2}\log^{2}n}{n^{2}}.
	\label{eq:Ui-Uj-condition-thm-distribution-1}
\end{equation}
Then the estimator \eqref{eq:estimator-M-hat-distribution} obeys
\[
	\sup_{z\in\mathbb{R}}\left|\mathbb{P}\left(\widehat{M}_{i,j}-M_{i,j}^{\star}\leq z\sqrt{v_{i,j}^{\star}}\,\right)-\Phi(z)\right|=o(1) ,
\]
where $\Phi(\cdot)$ represents the cumulative density function (CDF) of the standard Gaussian distribution.
\end{theorem}
The proof of this theorem is postponed to Section~\ref{sec:proof-thm:matrix-distribution-complete}. 
In a nutshell, Theorem~\ref{thm:matrix-distribution-complete} tells us that $\widehat{\bm{M}}$ is a nearly unbiased estimator of the truth $\bm{M}^{\star}$, as long as the signal strength---as captured by $\|\bm{U}^{\star}_{i,\cdot}\|_2$ and $\|\bm{U}^{\star}_{j,\cdot}\|_2$ when estimating the $(i,j)$-th entry---is sufficiently large (cf.~\eqref{eq:Ui-Uj-condition-thm-distribution-1}).  
The resulting estimation error in each entry is well approximated by a zero-mean Gaussian random variable, 
whose variance can be determined in a tractable fashion.  
As can be easily verified, the variance $v_{i,j}^{\star}$ is precisely the variance of the $(i,j)$-th entry of $\bm{E}\bm{U}^{\star}\bm{U}^{\star\top}+\bm{U}^{\star}\bm{U}^{\star\top}\bm{E}$ (as singled out in \eqref{eq:approx-ULambdaU-Mstar-discuss}). 
%In stark contrast to classical large-sample asymptotic analysis \citep{van2000asymptotic}, 
The above distributional theory is non-asymptotic, 
which lends itself well to high-dimensional applications.

\subsection{Inference and uncertainty quantification}
\label{sec:UQ-general}

The Gaussian approximation unveiled in Theorem~\ref{thm:matrix-distribution-complete},
which is dictated by a single parameter ${v}_{i,j}^{\star}$, 
paves the way for statistical inference and uncertainty quantification tailored to this model. 
In order to construct a valid confidence interval for each entry of $\bm{M}^{\star}$, 
everything boils down to identifying an estimator that approximates the variance parameter ${v}_{i,j}^{\star}$, ideally in a data-driven yet faithful manner. 

In view of the variance characterization \eqref{eq:variance-formula-Mij-lem}, 
computing ${v}_{i,j}^{\star}$ requires information about both the noise variances $\{\sigma_{i,j}^2\}_{1\leq i,j\leq n}$ and the projection matrix $\bm{P}^{\star}$ (cf.~\eqref{eq:defn-P-UU-T}). However, estimating the noise variances is in general statistically infeasible, given that we only have access to a single observation (i.e., $M_{i,j}$) related to each individual variance $\sigma_{i,j}^2$.  Fortunately, the variance  ${v}_{i,j}^{\star}$ involves only the summation or equivalently the average of these individual variances, whose stochastic errors will be averaged out.  This leads us to the following surrogate
\begin{equation}
	\widetilde{v}_{i,j}=\begin{cases}
		\sum_{l=1}^{n}{E}_{i,l}^{2}{P}_{l,j}^{\star 2}+\sum_{l=1}^{n}{P}_{i,l}^{\star 2}{E}_{l,j}^{2}+2{E}_{i,j}^{2}{P}^{\star}_{i,i}{P}_{j,j}^{\star}, & \text{if }i\neq j,\\
	4\sum_{l=1}^{n}{E}_{i,l}^{2}{P}_{l,i}^{\star 2}, & \text{if }i=j, 
\end{cases}
	\label{eq:v-ij-plugin-surrogate-123}
\end{equation}
which is clearly an unbiased estimator of ${v}_{i,j}^{\star}$. 
Given the statistical independence of $\{E_{i,j}\}_{i\geq j}$, 
we can expect to have $\widetilde{v}_{i,j}\approx v_{i,j}^{\star}$, owing to the concentration of measure.

%we propose to replace $\sigma_{i,j}^2$ by the corresponding entry $E_{i,j}^2$, or more realistically, an estimator $\widehat{E}_{i,j}^2$ of $E_{i,j}^2$. 

However, the above surrogate $\widetilde{v}_{i,j}$ remains practically incomputable, 
due to the absence of knowledge about both $\bm{E}$ and $\bm{P}^{\star}$.   
To address this issue, we propose the following plug-in estimator:
\begin{equation}
	\widehat{v}_{i,j}=\begin{cases}
	\sum_{l=1}^{n}\widehat{E}_{i,l}^{2}\widehat{P}_{l,j}^{2}+\sum_{l=1}^{n}\widehat{P}_{i,l}^{2}\widehat{E}_{l,j}^{2}+2\widehat{E}_{i,j}^{2}\widehat{P}_{i,i}\widehat{P}_{j,j}, & \text{if }i\neq j,\\
	4\sum_{l=1}^{n}\widehat{E}_{i,l}^{2}\widehat{P}_{l,i}^{2}, & \text{if }i=j, 
\end{cases}
	\label{eq:v-ij-plugin-estimator}
\end{equation}
where $\widehat{\bm{E}}=[\widehat{E}_{i,j}]_{1\leq i,j\leq n}$ and $\widehat{\bm{P}}=[\widehat{P}_{i,j}]_{1\leq i,j\leq n}$ stand for estimators of $\bm{E}$ and $\bm{P}^{\star}$, respectively. 
In particular, we employ the following specific estimators of $\bm{E}$ and $\bm{P}^{\star}$, again adopting the plug-in strategy:
\begin{subequations}
\label{eq:estimators-E-and-P}
\begin{align}
\widehat{\bm{E}} & \coloneqq\bm{M}-\bm{U}\bm{\Lambda}\bm{U}^{\top},\label{eq:estimator-noise-matrix-E}\\
\widehat{\bm{P}} & \coloneqq\bm{U}\bm{U}^{\top},\label{eq:estimator-projection-P}
\end{align}
\end{subequations}
where $\bm{U}$ and $\bm{\Lambda}$ are, as usual, computed via eigendecomposition of $\bm{M}$. 
For a prescribed coverage level $1-\alpha$ (with $0<\alpha <1$), we construct the following confidence interval for the $(i,j)$-th entry of $\bm{M}^{\star}$, 
motivated by the Gaussian approximation in Theorem~\ref{thm:matrix-distribution-complete}: 
\begin{align}
	\label{eq:confidence-interval-construction-general}
	\mathsf{CI}_{i,j}^{1-\alpha} \coloneqq \Big[\widehat{M}_{i,j}\pm\Phi^{-1}(1-\alpha/2)\sqrt{\widehat{v}_{i,j}}\,\Big].
\end{align}
Here and throughout, for any $b>0$, we let $[a\pm b]$ abbreviate the interval $[a-b,a+b]$, and we use $\Phi^{-1}(\cdot)$ to represent the inverse CDF of the standard Gaussian distribution.

As encouraging news, the above construction of entrywise confidence intervals is provably valid with high probability, as revealed by the following theorem. The proof is postponed to Section~\ref{sec:proof-thm:CI-general}. 
\begin{theorem}
	\label{thm:CI-general}
	Consider the settings and assumptions in Section~\ref{sec:setup-general-theory},  
	%the assumptions of Theorem~\ref{thm:UsgnH-Ustar-MUstar-general} hold.
	and suppose that $\sigma / \sigma_{\min} = O(1)$,  $\kappa^{4}\mu^{2}r^{2}\log n\leq n$ and $\sigma \kappa\sqrt{n\log n}\lesssim |\lambda_{r}^{\star}|$. Consider any $1\leq i,j\leq n$, and assume that
	\begin{align}
		\frac{\big\|\bm{U}_{j,\cdot}^{\star}\big\|_{2}^{2}+\big\|\bm{U}_{i,\cdot}^{\star}\big\|_{2}^{2}}{\big\|\bm{U}^{\star}\big\|_{\mathrm{F}}^{2}}\gtrsim\frac{B\kappa^{2}\mu^{2}r^{2}\log^{2}n}{\sigma n^{3/2}}+\frac{\sigma\mu^{2}r\kappa^{3}\log^{3}n}{|\lambda_{r}^{\star}|\sqrt{n}}.  
		\label{eq:Uj-Ui-lower-bound-thm}
	\end{align}
	For any fixed coverage level $1-\alpha \in (0,1)$, 
	the confidence interval $\mathsf{CI}_{i,j}^{1-\alpha}$ constructed in \eqref{eq:confidence-interval-construction-general} obeys
	\begin{align}
		%\left| 
		\mathbb{P} \Big( M_{i,j}^{\star} \in \mathsf{CI}_{i,j}^{1-\alpha}  
		\Big) 
		= 1-\alpha + o(1).
		% \right| =o(1).
	\end{align}
\end{theorem}
%
%It is noteworthy that: 
Theorem~\ref{thm:CI-general} confirms that the confidence interval proposed above meets the prescribed coverage requirement,
provided that the associated signal strength is not too low (see \eqref{eq:Uj-Ui-lower-bound-thm}). 
In addition to its statistical validity, the proposed procedure enjoys several features that make it practically appealing: 
\begin{itemize}
	\item {\em Adaptive to unknown noise levels and distributions.} The above inference procedure is fully data-driven, which does not require prior knowledge about the noise levels or noise distributions. As alluded to previously, it is in general impossible to estimate the noise variance in each entry, and hence a data-driven yet valid approach is of critical value. 

	\item {\em Adaptive to heteroskedastic noise.} Our statistical guarantees hold without relying on homogeneity of noise components. In other words, this inference procedure automatically accommodates heteroskedastic noise, a scenario where the variance of the noise components might vary across different locations.   
\end{itemize}

Careful readers might remark that Theorem~\ref{thm:CI-general} is concerned with statistical inference for a single entry. 
Interestingly, the distributional theory presented in Section~\ref{sec:distribution-general} (see also Lemma~\ref{thm:matrix-distribution} in the proof of Theorem~\ref{thm:matrix-distribution-complete}) might also be instrumental in pursuing simultaneous inference, namely, the problem of constructing a confidence region that simultaneously accounts for more than one unknown entries. We omit such an extension for the sake of conciseness.

%
%\end{proof}

%%
%\begin{align*}
% & \big\|\mathsf{sgn}(\bm{H})\bm{\Lambda}^{\star}\mathsf{sgn}(\bm{H})^{\top}-\bm{\Lambda}\big\|\\
% & \qquad\leq\big\|\mathsf{sgn}(\bm{H})\bm{\Lambda}^{\star}\mathsf{sgn}(\bm{H})^{\top}-\bm{H}\bm{\Lambda}^{\star}\bm{H}^{\top}\big\|+\big\|\bm{H}\bm{\Lambda}^{\star}\bm{H}^{\top}-\bm{\Lambda}\big\|\\
% & \qquad\lesssim\frac{\kappa\sigma^{2}n}{\lambda_{r}^{\star}}+\sigma\sqrt{r\log n}.
%\end{align*}

%\input{chapters/matrix_completion.tex}

\section{Application: Confidence intervals for matrix completion}
\label{sec:CI-matrix-completion}

As an illustration of the applicability of the inference procedure described in Section~\ref{sec:UQ-general}, 
we develop concrete consequences of Theorem~\ref{thm:CI-general} in application to noisy matrix completion---an extension of the formulation
in Section~\ref{sec:mc-l2} to noisy settings.

\paragraph{Model: noisy matrix completion.} 

Suppose that we are asked to reconstruct a symmetric rank-$r$ matrix
 $\bm{M}^{\star}=[M_{i,j}^{\star}]_{1\leq k,l \leq n}\in \mathbb{R}^{n\times n}$  
 with eigendecomposition $\bm{M}^{\star}=\bm{U}^{\star}\bm{\Lambda}^{\star}\bm{U}^{\star\top}$.
 We only get to acquire noisy observations of a subset of the entries of $\bm{M}^{\star}$; more precisely, 
there exists a sampling set $\Omega \subseteq [n]\times [n]$ such that we observe
\begin{align}
	M_{k,l}^{\star} + \eta_{k,l} \qquad & \text{if } (k,l)\in \Omega. 
	% M_{i,j} = \begin{cases} , \qquad & \text{if } (i,j)\in \Omega, \\ 0, &\text{else.}  \end{cases}
\end{align}
Here, $\{\eta_{k,l}\mid k\geq l\}$ denotes independent Gaussian noise obeying  
\begin{align}
	\eta_{k,l} = \eta_{l,k} \overset{\mathrm{i.i.d.}}{\sim} \mathcal{N}(0, \sigma_{\eta}^2), \qquad k\geq l. 
\end{align}
As before, we focus on the random sampling model such that each location $(k,l)$ with $k\geq l$ is included in the sampling set $\Omega$ independently with probability $p$. Further,
 assume that $\bm{M}^{\star}$ has eigenvalues obeying \eqref{eq:assumption-lambda-r-positive-setup}, condition number $\kappa$ (cf.~\eqref{eq:defn-kappa}), and  incoherence parameter $\mu$ (cf.~\eqref{eq:defn-Ustar-incoherence}). 
Can we build a confidence interval for each entry $M_{i,j}^{\star}$, on the basis of the output of the spectral method?

\paragraph{Computing entrywise confidence intervals.} 
In order to apply the inference procedure in Section~\ref{sec:UQ-general}, it suffices to determine the data matrix $\bm{M}=[M_{i,j}]_{1\leq i,j\leq n}$, which can be selected as usual. Specifically, a possible inference procedure proceeds as follows:
\begin{itemize}
	\item  Set $\bm{M}$ such that for any $1\leq i,j \leq n$, 
\begin{align}
	M_{i,j} = \begin{cases} \frac{1}{p} \big( M_{i,j}^{\star} + \eta_{i,j} \big), \qquad & \text{if } (i,j)\in \Omega, \\ 0, &\text{else}, \end{cases}
\end{align}
which clearly obeys $\mathbb{E}[\bm{M}]=\bm{M}^{\star}$. 

	\item Compute the estimate $\widehat{\bm{M}}$ (cf.~\eqref{eq:estimator-M-hat-distribution}) via the spectral method.

	\item For a given coverage level $1-\alpha$ and a given pair $(i,j)$, 
		construct the confidence interval $\mathsf{CI}_{i,j}^{1-\alpha}$ according to \eqref{eq:confidence-interval-construction-general}, 
		with auxiliary parameters provided in \eqref{eq:v-ij-plugin-estimator} and \eqref{eq:estimators-E-and-P}. 

\end{itemize}

\paragraph{Performance guarantees and implications.}
When specialized to noisy matrix completion, our inference theory in Theorem~\ref{thm:CI-general} leads to the following statistical guarantees. 
\begin{theorem}
	\label{thm:CI-noisy-matrix-completion}
	Consider the noisy matrix completion setting in this section. 
	Suppose that $\kappa^{4}\mu^{2}r^{2}\log n\leq n$, 
	%$\sigma \kappa\sqrt{n\log n}\lesssim |\lambda_{r}^{\star}|$ and 
	$\frac{\max_{k,l}|M_{k,l}^{\star}|}{\min_{k,l}|M_{k,l}^{\star}|}=O(1)$, 
	\begin{equation}
		np\gtrsim \kappa^{4}r\log n\quad\text{and}\quad\sigma_{\eta}\kappa\sqrt{\frac{n\log n}{p}}\lesssim|\lambda_{r}^{\star}|.  
		\label{eq:Mstar-noise-requirement-noisy-MC}
	\end{equation}
	Consider any $1\leq i,j\leq n$, and assume that
\begin{align}
	\frac{\big\|\bm{U}_{j,\cdot}^{\star}\big\|_{2}^{2}+\big\|\bm{U}_{i,\cdot}^{\star}\big\|_{2}^{2}}{\big\|\bm{U}^{\star}\big\|_{\mathrm{F}}^{2}} & \gtrsim\frac{\mu^{2}r^{2}\kappa^{4}\log^{3}n}{n\sqrt{np}}+\frac{\sigma_{\eta}\mu^{2}r\kappa^{3}\log^{3}n}{|\lambda_{r}^{\star}|\sqrt{np}}.
	\label{eq:Ui-Uj-lower-bound-thm-noisy-MC-146}
\end{align}
	For any fixed coverage level $1-\alpha \in (0,1)$, 
	the confidence interval $\mathsf{CI}_{i,j}^{1-\alpha}$ constructed in \eqref{eq:confidence-interval-construction-general} obeys
	\begin{align}
		\mathbb{P} \Big( M_{i,j}^{\star} \in \mathsf{CI}_{i,j}^{1-\alpha}  
		\Big) 
		= 1-\alpha + o(1).
	\end{align}
\end{theorem}

In order to help interpret the applicable range of Theorem~\ref{thm:CI-noisy-matrix-completion}, 
let us focus on the simple scenario with $\kappa, \mu, r\asymp 1$ to simplify discussion. 

\begin{itemize}

	\item First of all, Condition~\eqref{eq:Mstar-noise-requirement-noisy-MC} can be simplified as
	\[
		np  \gtrsim  \log n\quad\text{and}\quad \sigma_{\eta}\sqrt{\frac{n\log n}{p}}\lesssim |\lambda_r^{\star}|. 
	\]
	The first condition on the sampling size coincides with the fundamental requirement even if the goal is merely to enable reliable estimation \citep{candes2010NearOptimalMC}, 
	whereas the second condition on the signal-to-noise ratio is also necessary---up to some log factor---to ensure an estimation quality better than that of a random guess \citep[Theorem~3.3]{cai2019subspace}.

	\item Next, we move on to interpret the other condition \eqref{eq:Ui-Uj-lower-bound-thm-noisy-MC-146} imposed in our theory, which simplifies to
	\[
		\frac{\big\|\bm{U}_{j,\cdot}^{\star}\big\|_{2}^{2}+\big\|\bm{U}_{i,\cdot}^{\star}\big\|_{2}^{2}}{\big\|\bm{U}^{\star}\big\|_{\mathrm{F}}^{2}}
		\gtrsim\frac{\log^{3}n}{n\sqrt{np}}+\frac{\sigma_{\eta}\log^{3}n}{|\lambda_{r}^{\star}|\sqrt{np}}.
	\]
	Let us consider the most challenging case where $np \gtrsim  \mathrm{poly}\log n$ and $\sigma_{\eta}\sqrt{\frac{n\mathrm{poly}\log n}{p}}\lesssim |\lambda_r^{\star}|$ (for some sufficiently large poly-log factor). 
	In such a case, the above condition only requires 
	\[
		\frac{\big\|\bm{U}_{j,\cdot}^{\star}\big\|_{2}^{2}+\big\|\bm{U}_{i,\cdot}^{\star}\big\|_{2}^{2}}{\big\|\bm{U}^{\star}\big\|_{\mathrm{F}}^{2}}\gtrsim\frac{1}{n\mathrm{poly}\log n},
	\]
	indicating that the associated signal power $\|\bm{U}_{j,\cdot}^{\star}\|_{2}^{2}+\|\bm{U}_{i,\cdot}^{\star}\|_{2}^{2}$ 
		is allowed to be much smaller than the average signal power across all rows (which can be captured by 
		$\|\bm{U}^{\star}\|_{\mathrm{F}}^{2}/n$). 

\end{itemize}
In a nutshell, the validity of our inference procedure is ensured for broad settings. 
Additionally, we have conducted a series of numerical experiments to examine the entrywise distributions of $\bm{M}$. As illustrated in Figure~\ref{fig:MC-inference-numerics}, 
the normalized estimation error $(\widehat{v}_{i,j})^{-1/2} (\widehat{M}_{i,j} - M_{i,j}^{\star})$ is close in distribution to a standard Gaussian random variable, 
which corroborates our theory on the confidence interval construction.

Before concluding, we would like to remark that: 
while the distributional theory for spectral methods allows for valid construction of confidence intervals for an unseen entry,  
it is oftentimes not among the most effective statistical inference procedures that one can put forward. 
There exist other alternatives that are provably more efficient, including but not limited to inference procedures based on convex relaxation and nonconvex optimization \citep{chen2019inference,xia2021statistical}, and the ones based on more refined spectral methods \citep{yan2021inference,chernozhukov2021inference}.

\begin{figure}

	\begin{tabular}{cc}
		\includegraphics[width=0.45\textwidth]{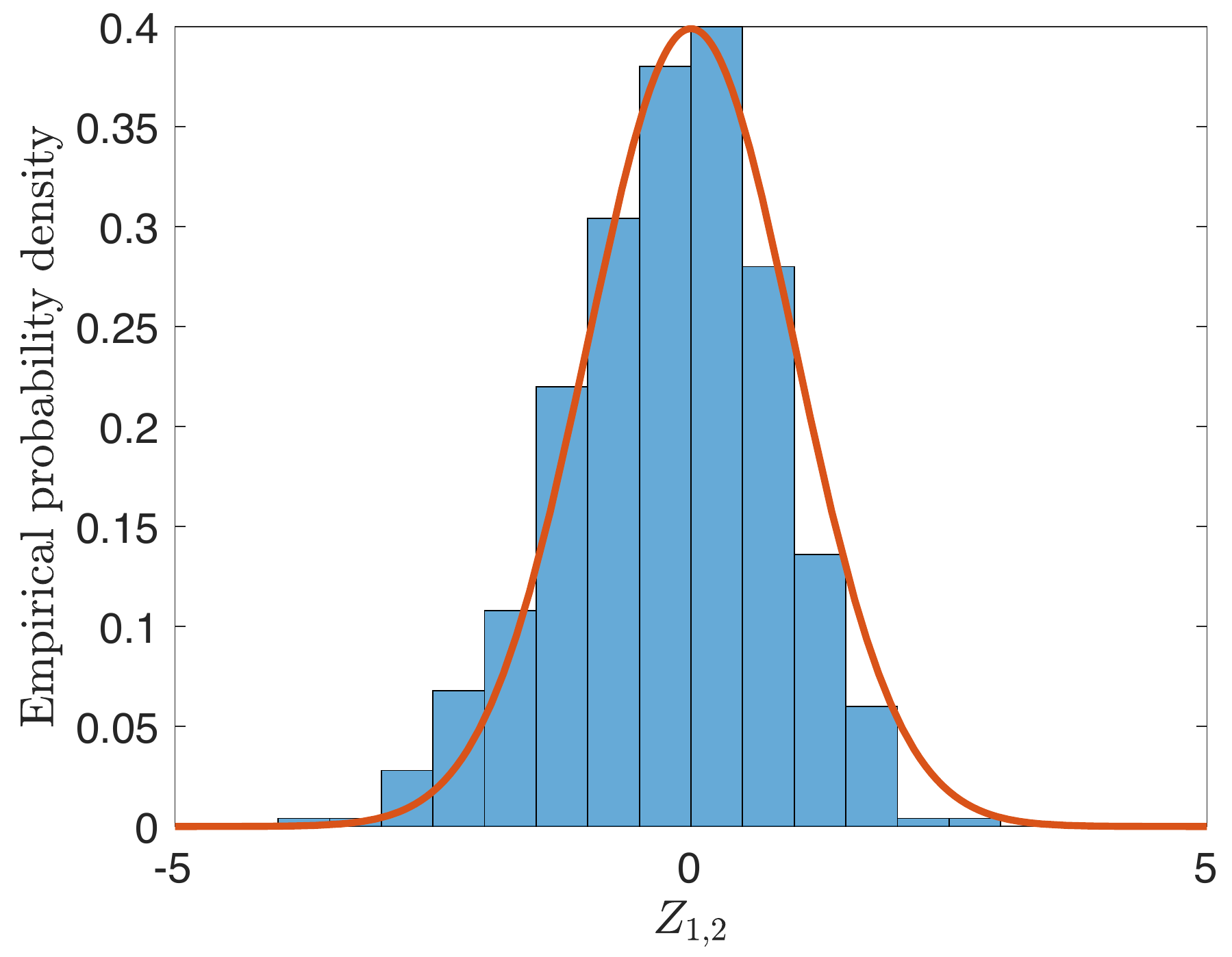} & \includegraphics[width=0.45\textwidth]{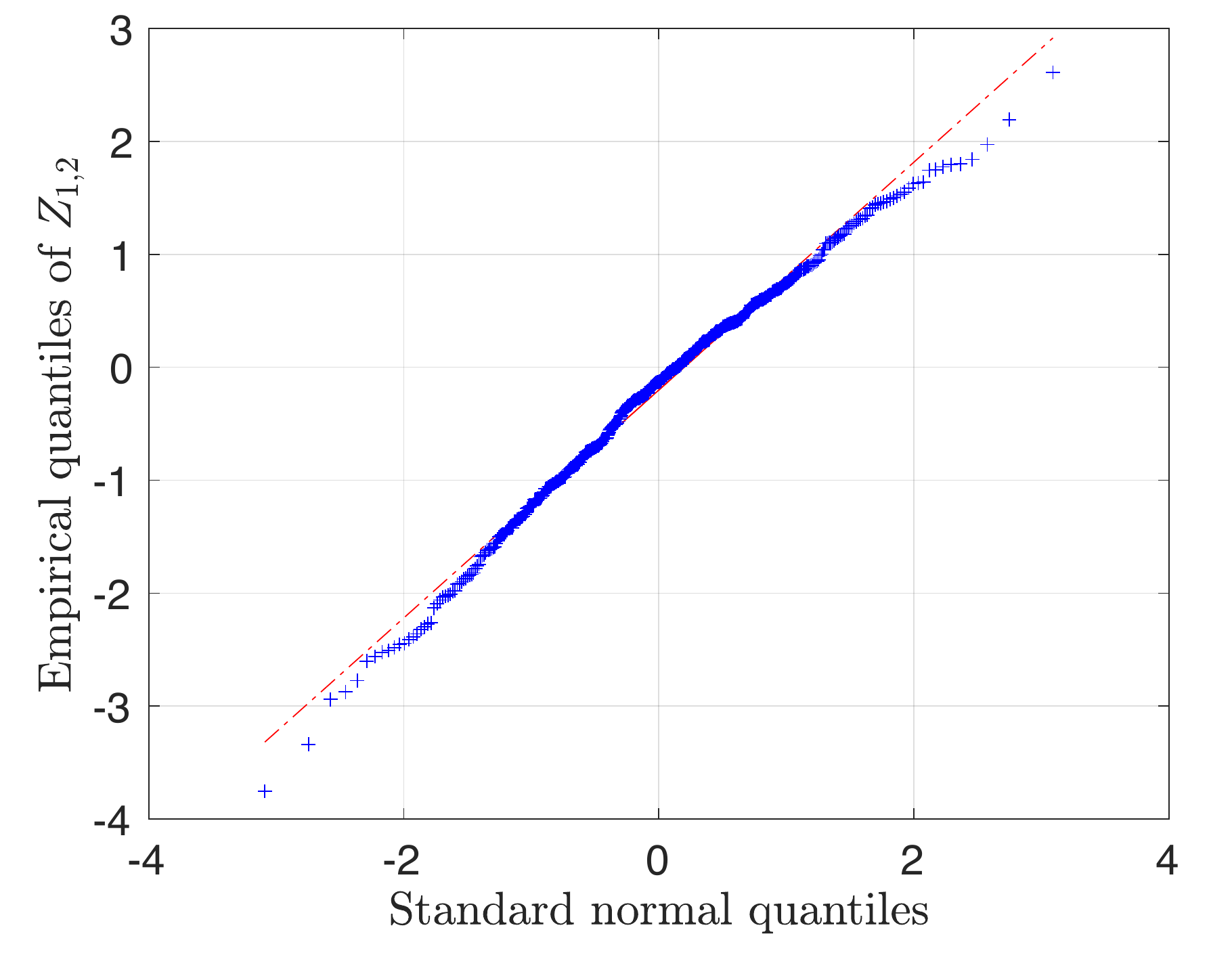} \tabularnewline
		(a)  & (b) \tabularnewline
	\end{tabular}
	
	\caption{Entrywise numerical distribution for noisy matrix completion. We generate $\bm{M}^{\star}=\bm{U}^{\star}\bm{U}^{\star\top}$ with $\bm{U}^{\star}$ being a random orthonormal matrix, and analyze the behavior of $Z_{1,2}= (\widehat{v}_{1,2})^{-0.5} ( \widehat{M}_{1,2}-M_{1,2}^{\star} )$, 
		where $\widehat{\bm{M}}$ is defined in \eqref{eq:estimator-M-hat-distribution} and $\widehat{v}_{i,j}$ is defined in \eqref{eq:v-ij-plugin-estimator}.
		The results are reported for 500 Monte Carlo trials when $p=0.3$, $n=1000$, $\sigma_{\eta} = 10^{-4}$, and $r=3$.  
		 (a) Histogram of the empirical distribution of $Z_{1,2}$; (b)  Q-Q (quantile-quantile) plot of $Z_{1,2}$ vs.~the standard normal distribution. \label{fig:MC-inference-numerics}}  
\end{figure}

\paragraph{Proof of Theorem~\ref{thm:CI-noisy-matrix-completion}.} 

Given that $\eta_{i,j}$ is a Gaussian random variable and hence possibly
unbounded, we find it convenient to introduce a truncated version
as follows
\[
	\widetilde{\eta}_{i,j}=\eta_{i,j}\mathbbm{1}\big\{|\eta_{i,j}|\leq 5\sigma_{\eta}\sqrt{\log n}\big\},\qquad1\leq i,j\leq n.
\]
and 
\[
	\widetilde{M}_{i,j} = \begin{cases} \frac{1}{p} \big( M_{i,j}^{\star} + \widetilde{\eta}_{i,j} \big), \qquad & \text{if } (i,j)\in \Omega, \\ 0, &\text{else}. \end{cases}
\]
Repeating the analysis in Section~\ref{sec:proof-eqn-noise-matrix-E-bound-denoising}, we can show that 
\[
	\mathbb{P}\big\{ \bm{M} = \widetilde{\bm{M}} \big\}
	= \mathbb{P} \big\{ \eta_{i,j} = \widetilde{\eta}_{i,j}, \forall i,j\in [n] \big\} 
	\geq 1- n^{-10},
\]
meaning that  $\bm{M}$ and $\widetilde{\bm{M}}$ are equivalent with high probability. 
As a result, we shall concentrate on validating the confidence interval computed based on $\widetilde{\bm{M}}$ in the subsequent analysis. %, unless otherwise noted 
Before proceeding, we record several key properties about  $\widetilde{\eta}_{i,j}$ as follows:
\begin{align}
	\mathbb{E}[ \widetilde{\eta}_{i,j}] = 0, ~~~ \mathbb{E}[ \widetilde{\eta}^2_{i,j}]=(1-o(1)) \sigma_{\eta}^2, 
	 ~~~ \big|\widetilde{\eta}_{i,j}\big|\leq  5\sigma_{\eta}\sqrt{\log n}.  
	 \label{eq:eta-property-CI-MC}
\end{align}

The proof follows by invoking Theorem~\ref{thm:CI-general}, as long as the conditions required therein are satisfied. 
To begin with, the associated variance parameters are given by
\begin{align*}
\sigma_{i,j}^{2} & \coloneqq\mathbb{E}\left[\big(\widetilde{M}_{i,j}-M_{i,j}^{\star}\big)^{2}\right]\\
 & =p\mathbb{E}\left[\left(\frac{1-p}{p}M_{i,j}^{\star}+\frac{1}{p}\widetilde{\eta}_{i,j}\right)^{2}\right]+(1-p)\big(M_{i,j}^{\star}\big)^{2}\\
 & = 
\frac{\left(1-p\right)^{2}}{p} \big(M_{i,j}^{\star}\big)^{2}+\frac{1}{p}\mathbb{E}\left[\widetilde{\eta}_{i,j}^{2}\right]+(1-p)\big(M_{i,j}^{\star}\big)^{2} \\
 & =\frac{1-p}{p}\big(M_{i,j}^{\star}\big)^{2}+\frac{1-o(1)}{p}\sigma_{\eta}^{2}, 
\end{align*}
thus leading to
\begin{align*}
\sigma_{\min}^{2} & \coloneqq \min_{i,j}\sigma_{i,j}^{2}=\frac{1-p}{p}\min_{i,j}\big(M_{i,j}^{\star}\big)^{2}+\frac{1-o(1)}{p}\sigma_{\eta}^{2},\\
\sigma^{2} & \coloneqq \max_{i,j}\sigma_{i,j}^{2}=\frac{1-p}{p} \|\bm{M}^{\star}\|_{\infty}^2 +\frac{1-o(1)}{p}\sigma_{\eta}^{2}.
\end{align*}
Apparently, $\sigma^2/\sigma_{\min}^{2}=O(1)$ holds true under the assumptions of Theorem~\ref{thm:CI-noisy-matrix-completion}. 
In addition, the random variables $\{\widetilde{M}_{i,j}-M_{i,j}^{\star}\}$ are all bounded obeying
\[
\big|\widetilde{M}_{i,j}-M_{i,j}^{\star}\big|\leq\frac{1-p}{p}\big|M_{i,j}^{\star}\big|+\frac{1}{p}\big|\widetilde{\eta}_{i,j}\big|
\leq\frac{(1-p)\|\bm{M}^{\star}\|_{\infty}+\sigma_{\eta}\sqrt{5\log n}}{p}\eqqcolon B.
\]
These bounds readily imply that
\begin{equation}
	\frac{B}{\sigma}\asymp\frac{~\frac{(1-p)\|\bm{M}^{\star}\|_{\infty}+\sigma_{\eta}\sqrt{5\log n}}{p}~}{\sqrt{\frac{1-p}{p}}\|\bm{M}^{\star}\|_{\infty}+\frac{1}{\sqrt{p}}\sigma_{\eta}}\lesssim\sqrt{\frac{\log n}{p}}. 
	\label{eq:B-sigma-ratio-noisy-MC}
\end{equation}

Moving to the condition $\sigma \kappa\sqrt{n\log n}\lesssim |\lambda_{r}^{\star}|$ in Theorem~\ref{thm:CI-noisy-matrix-completion}, it can be guaranteed if
	\begin{equation*}
		\|\bm{M}^{\star}\|_{\infty}\kappa\sqrt{\frac{(1-p)n\log n}{p}}
		\lesssim|\lambda_{r}^{\star}| 
		\quad \text{and} \quad
		\sigma_{\eta}\kappa\sqrt{\frac{n\log n}{p}}\lesssim|\lambda_{r}^{\star}|.  
		\label{eq:Mstar-noise-requirement-noisy-MC-135}
	\end{equation*}
Given the assumption  $\frac{\max_{k,l}|M_{k,l}^{\star}|}{\min_{k,l}|M_{k,l}^{\star}|}=O(1)$, one has 
\begin{equation}
	\label{eq:Mstar-inf-F-relation-noisy-MC}
	\|\bm{M}^{\star}\|_{\infty} \asymp \frac{1}{n} \|\bm{M}^{\star}\|_{\mathrm{F}} 
	\leq \frac{\sqrt{r}}{n} \|\bm{M}^{\star}\| 
	= \frac{\kappa \sqrt{r}}{n} |\lambda_{r}^{\star}|. 
\end{equation}
As a consequence, the condition $\sigma \kappa\sqrt{n\log n}\lesssim |\lambda_{r}^{\star}|$ can be ensured under Condition~\eqref{eq:Mstar-noise-requirement-noisy-MC}. 

It remains to certify Condition~\eqref{eq:Uj-Ui-lower-bound-thm}. 
By virtue of the above calculations of $\sigma$ and $B$ as well as the property \eqref{eq:B-sigma-ratio-noisy-MC}, it is easily seen that Condition~\eqref{eq:Uj-Ui-lower-bound-thm} is valid as long as the following holds:
\begin{align*}
\frac{\big\|\bm{U}_{j,\cdot}^{\star}\big\|_{2}^{2}+\big\|\bm{U}_{i,\cdot}^{\star}\big\|_{2}^{2}}{\big\|\bm{U}^{\star}\big\|_{\mathrm{F}}^{2}} & \gtrsim\frac{\kappa^{2}\mu^{2}r^{2}\log^{5/2}n}{n\sqrt{np}}+\frac{\sqrt{1-p}\,\|\bm{M}^{\star}\|_{\infty}\mu^{2}r\kappa^{3}\log^{3}n}{|\lambda_{r}^{\star}|\sqrt{np}}\\
 & \qquad+\frac{\sigma_{\eta}\mu^{2}r\kappa^{3}\log^{3}n}{|\lambda_{r}^{\star}|\sqrt{np}}.
\end{align*}
Taking this together with the relation \eqref{eq:Mstar-inf-F-relation-noisy-MC}, 
we can demonstrate straightforwardly that Condition~\eqref{eq:Uj-Ui-lower-bound-thm} is guaranteed to hold as long as Condition~\eqref{eq:Ui-Uj-lower-bound-thm-noisy-MC-146} is satisfied. 
This completes the proof.

%\paragraph{Statistical guarantees.} 

%% file: chapters/analysis_Linf.tex
%%%%%%%%%%%%%%%%%%%%%%%%%%%%%%%%%%%%%%%%%%%%%%%%%%%%%%%%%%%%%%%%%%%%%%

\section{Appendix A: Proof of Theorem~\ref{thm:UsgnH-Ustar-MUstar-general}}
\label{sec:proof-thm:UsgnH-Ustar-MUstar-general}

%In the rest of this chapter, we establish the results Theorem~\ref{thm:UsgnH-Ustar-MUstar-general} and Corollary~\ref{cor:entrywise-error-general}.
To simplify notation, we assume throughout the proof that $\lambda_r^{\star} > 0$, namely,
\begin{align}
	|\lambda_1^{\star}| \geq  \cdots \geq | \lambda_{r-1}^{\star} | \geq \lambda_r^{\star} > 0.
	\label{eq:assumption-lambda-r-positive}
\end{align}
%Now we turn to the proof of Theorem~\ref{thm:UsgnH-Ustar-MUstar-general}.
The challenge of the proof arises due to the complicated statistical dependency between $\bm{M}$ and $\bm{U}$, and the leave-one-out analysis paves a plausible path to decouple the dependency.

\subsection{Construction of leave-one-out auxiliary estimates}
\label{sec:construction-LOO-general}

As elucidated in the rank-1 matrix denoising example in Section~\ref{sec:rank-1-denoising-LOO}, the key to enabling fine-grained analysis is to seek assistance from  a collection of leave-one-out estimates.  Akin to Section~\ref{sec:LOO-estimates-denoising}, for each $1\leq l\leq n$, we construct  two auxiliary matrices  $\bm{M}^{(l)}$ and $\bm{E}^{(l)}=\big[E^{(l)}_{i,j} \big]_{1\leq i,j\leq n}$ as follows:
\begin{align}
\bm{M}^{(l)}\coloneqq\bm{M}^{\star}+\bm{E}^{(l)},
\quad
E_{i,j}^{(l)}  \coloneqq
\begin{cases}
E_{i,j}, & \text{if }i\neq l\text{ and }j\neq l,\\
0, & \text{else},
\end{cases}
\label{eq:Ml-construction-general}
\end{align}
which are generated by simply discarding all random noise incurred in the $l$-th column/row of the data matrix.
In addition, let $\lambda_1^{(l)},\cdots,\lambda_n^{(l)}$ be the eigenvalues of $\bm{M}^{(l)}$ sorted by
\begin{align}
	\big|\lambda_1^{(l)}\big| \geq \big|\lambda_2^{(l)}\big| \geq \cdots \geq \big| \lambda_{n}^{(l)} \big| ,
\end{align}
and denote by $\bm{u}_{i}^{(l)}$ the eigenvector of $\bm{M}^{(l)}$  associated with $\lambda_i^{(l)}$. The leave-one-out spectral estimates $\bm{U}^{(l)}$ and $\bm{\Lambda}^{(l)}$ are, therefore, given by
\begin{align}
	\bm{U}^{(l)}\coloneqq\big[\bm{u}_{1}^{(l)},\cdots,\bm{u}_{r}^{(l)}\big]\in\mathbb{R}^{n\times r};
	\quad
	\bm{\Lambda}^{(l)}\coloneqq\mathsf{diag}\big(\big[\lambda_{1}^{(l)},\cdots,\lambda_{r}^{(l)}\big]\big).	
\end{align}

We emphasize again that the main advantage of introducing the leave-one-out estimate $\bm{U}^{(l)}$ stems from its statistical independence from the $l$-th row of $\bm{M}$, which substantially simplifies the analysis for the $l$-th row of the estimate.
In principle, our analysis employs the leave-one-out estimates
to help decouple delicate statistical dependency  in a row-by-row fashion. Another crucial aspect of the analysis lies in the exploitation of the proximity of all these auxiliary estimates, a feature that is enabled by the ``stability'' of the spectral method.

\subsection{Preliminary facts}
\label{sec:preliminary-facts-Linf-general}

Before embarking on the $\ell_{\infty}$ and $\ell_{2,\infty}$ analyses,
we gather a couple of useful facts, whose proofs are postponed to
Section~\ref{sec:proof-auxiliary-lemmas-general}. In what follows,
$\bm{\Theta}$ (resp.~$\bm{\Theta}^{(l)}$) denotes a diagonal matrix
whose diagonal entries are the principal angles between $\bm{U}$
(resp.~$\bm{U}^{(l)}$) and $\bm{U}^{\star}$. In addition, we find
it helpful to introduce the following matrices
\begin{equation}
\bm{H}\coloneqq\bm{U}^{\top}\bm{U}^{\star}\qquad\text{and}\qquad\bm{H}^{(l)}\coloneqq\bm{U}^{(l)\top}\bm{U}^{\star},\label{eq:defn-H-Hl-general}
\end{equation}
which turn out to be close to being orthonormal.

The first set of results follows from the statistical nature of the
perturbation matrix $\bm{E}$ (cf.~Assumption~\ref{assumption-noise-general}),
which is immediately available from the matrix tail bounds. 
\begin{lemma}\label{lemma:general-noise-bound}
Consider the setting in Section~\ref{sec:setup-general-theory-independent}.
There is some constant $c_{2}>0$ such that with probability at least
$1-O(n^{-7})$,
\begin{equation}
\max_{l}\big\|\bm{E}^{(l)}\big\|\leq\|\bm{E}\|\leq c_{2}\sigma\sqrt{n}.\label{eq:noise-size-general}
\end{equation}
Moreover, for any fixed matrix $\bm{A}\in\mathbb{R}^{n\times d}$
with $d\leq n$, one has
\begin{align}
\|\bm{E}\bm{A}\|_{2,\infty}\leq4\sigma\sqrt{\log n}\,\|\bm{A}\|_{\mathrm{F}}+(6B\log n)\|\bm{A}\|_{2,\infty}.\label{eq:E-concentration-inf2-general}
\end{align}
 \end{lemma}
 %
%with probability at least $1-2n^{-5}$.

\begin{remark} In view of \eqref{eq:E-concentration-inf2-general},
with probability at least $1-2n^{-5}$,
\begin{align}
 & \|\bm{E}\bm{U}^{\star}\|_{2,\infty}\leq4\sigma\sqrt{\log n}\,\|\bm{U}^{\star}\|_{\mathrm{F}}+(6B\log n)\|\bm{U}^{\star}\|_{2,\infty}\nonumber \\
 & \qquad=4\sigma\sqrt{r\log n}+6B\sqrt{\frac{\mu r\log^{2}n}{n}}
 = (4+6c_{\mathsf{b}})\sigma\sqrt{r\log n},\label{eq:EU-2-inf-bound-general}
\end{align}
which relies on the definition \eqref{eq:defn-Ustar-incoherence}
and the definition of $c_{\mathsf{b}}$ in \eqref{eq:assumption-B-sigma}.
As a result,
\begin{align}
\big\|\bm{M}\bm{U}^{\star}\big\|_{2,\infty} & \leq\big\|\bm{M}^{\star}\bm{U}^{\star}\big\|_{2,\infty}+\big\|\bm{E}\bm{U}^{\star}\big\|_{2,\infty}\nonumber \\
&\leq \sqrt{\frac{\mu r}{n}}|\lambda_1^{\star}| + (4+6c_{\mathsf{b}})\sigma\sqrt{r\log n}, \label{eq:MU-2-inf-bound-general}
\end{align}
which follows from the fact $\bm{M}^{\star}\bm{U}^{\star}=\bm{U}^{\star}\bm{\Lambda}^{\star}$
(so that $\|\bm{M}^{\star}\bm{U}^{\star}\|_{2,\infty}\leq\big\|\bm{U}^{\star}\big\|_{2,\infty}\|\bm{\Lambda}^{\star}\|=\sqrt{\mu r/n}\,|\lambda_{1}^{\star}|$).
\end{remark}

With the size of the perturbations (i.e., $\|\bm{E}^{(l)}\|$ and
$\|\bm{E}\|$) under control, the $\ell_{2}$ perturbation theory
established in Chapter~\ref{cha:matrix-perturbation} leads to the following set of conclusions.
\begin{lemma}\label{lem:preliminary-result-general-iid} Suppose
that $c_{2}\sigma\sqrt{n}\leq(1-1/\sqrt{2})\lambda_{r}^{\star}$,
where $c_{2}$ is the same constant as in Lemma~\ref{lemma:general-noise-bound}.
Then with probability at least $1-O(n^{-7})$, one has \begin{subequations}
\label{eq:L2-perturbation-summary-general}
\begin{align}
\mathsf{dist}(\bm{U},\bm{U}^{\star}) & \leq\frac{2c_{2}\sigma\sqrt{n}}{\lambda_{r}^{\star}},~~ & \mathsf{dist}\big(\bm{U}^{(l)},\bm{U}^{\star}\big) & \leq\frac{2c_{2}\sigma\sqrt{n}}{\lambda_{r}^{\star}}, \label{eq:dist-U-Ustar-general}\\
\big\|\sin\bm{\Theta}\big\| & \leq\frac{c_{2}\sigma\sqrt{2n}}{\lambda_{r}^{\star}}, ~~ & \big\|\sin\bm{\Theta}^{(l)}\big\| & \leq\frac{c_{2}\sigma\sqrt{2n}}{\lambda_{r}^{\star}}, \label{eq:distp-U-Ustar-general}\\
\max_{1\leq j\leq r}|\lambda_{j}| & \geq\lambda_{r}^{\star}-c_{2}\sigma\sqrt{n}, & \max_{1\leq j\leq r}|\lambda_{j}^{(l)}| & \geq\lambda_{r}^{\star}-c_{2}\sigma\sqrt{n}, \label{eq:lambda-size-large-general}\\
\max_{j:j>r}|\lambda_{j}| & \leq c_{2}\sigma\sqrt{n}, & \max_{j:j>r}|\lambda_{j}^{(l)}| & \leq c_{2}\sigma\sqrt{n}\label{eq:lambda-size-small-general}
\end{align}
hold simultaneously for all $1\leq l\leq n$. In addition,
\begin{align}
\|\bm{U}\bm{H}-\bm{U}^{\star}\|_{\mathrm{F}}\leq\frac{2c_{2}\sigma\sqrt{rn}}{\lambda_{r}^{\star}}.\label{eq:dist-under-H-general}
\end{align}
 \end{subequations} %
with probability exceeding $1-2n^{-5}$. \end{lemma}

\begin{remark}\label{remark:perturbation-size-eigen-gap}
Lemmas~\ref{lemma:general-noise-bound} and \ref{lem:preliminary-result-general-iid} allow us to bound the eigengap and perturbation size as follows
\begin{subequations}\label{eq:lambda-rl-lambda-r1l-M-Ml-LB-step}
\begin{align}
	\big| \lambda_{r}^{(l)} \big| - \big|\lambda_{r+1}^{(l)}\big|
	& \geq \lambda_{r}^{\star}-  c_2 \sigma \sqrt{n} - c_2 \sigma \sqrt{n}
	\geq \lambda_{r}^{\star}/2,
	\label{eq:lambda-rl-lambda-r1l-LB-step}  \\
	\|\bm{M}-\bm{M}^{(l)}\| & \leq\|\bm{M}-\bm{M}^{\star}\|+\|\bm{M}^{\star}-\bm{M}^{(l)}\|=\|\bm{E}\|+\|\bm{E}^{(l)}\|  \notag\\
	& \leq2c_{2}\sigma\sqrt{n}\leq(1-1/\sqrt{2}) \big( \big| \lambda_{r}^{(l)} \big| - \big|\lambda_{r+1}^{(l)}\big| \big),
	\label{eq:M-Ml-size-step}
\end{align}
\end{subequations}
which are valid as long as $20 c_2 \sigma \sqrt{n} \leq \lambda_r^{\star}$. These will prove useful when bounding the approximation error of $\bm{U}$ using $\bm{U}^{(l)}$.
\end{remark}

Another collection of results is concerned with $\bm{H}$ and $\bm{H}^{(l)}$.
\begin{lemma} \label{lem:H-property-summary-general} Suppose that the
assumptions of Lemma~\ref{lem:preliminary-result-general-iid} hold. With
probability at least $1-O(n^{-7})$, \begin{subequations}
\label{eq:H-property-summary-general}
\begin{align}
\|\bm{H}^{-1}\| & \leq2,\quad & \big\|\big(\bm{H}^{(l)}\big)^{-1}\big\| & \leq2,\label{eq:H-inv-norm-bound-general}\\
\big\|\bm{H}-\mathsf{sgn}(\bm{H})\big\| & \leq\frac{2c_{2}^{2}n\sigma^{2}}{(\lambda_{r}^{\star})^{2}},~~ & \big\|\bm{H}^{(l)}-\mathsf{sgn}(\bm{H}^{(l)})\big\| & \leq\frac{2c_{2}^{2}n\sigma^{2}}{(\lambda_{r}^{\star})^{2}}\label{eq:H-sgnH-dist-general}
\end{align}
\end{subequations} hold simultaneously for all $1\leq l\leq n$.

\end{lemma}

\begin{remark}\label{remark:AH-norm-bound} As a consequence of \eqref{eq:H-property-summary-general}
and Proposition~\ref{prop:unitary_norm_relation}, one has \begin{subequations}
\label{eq:A-2inf-H-Hl-general}
\begin{align}
\vertiii{\bm{A}} & =\vertiiibig{\bm{A}\bm{H}\bm{H}^{-1}}\leq\vertiii{\bm{A}\bm{H}}\,\big\|\bm{H}^{-1}\big\|\leq2\vertiii{\bm{A}\bm{H}}, \label{eq:A-2inf-H-general}\\
\vertiii{\bm{A}} & \leq\vertiiibig{\bm{A}\bm{H}^{(l)}}\,\big\|\big(\bm{H}^{(l)}\big)^{-1}\big\|\leq2\vertiiibig{\bm{A}\bm{H}^{(l)}}\label{eq:A-2inf-Hl-general}
\end{align}
\end{subequations} for any matrix $\bm{A}$. Here, $\vertiii{\bm{A}}$ could either be the Frobenius norm or the $\ell_{2,\infty}$
norm $\|\cdot\|_{2,\infty}$. \end{remark}

\subsection{Leave-one-out analysis}

Now we move on to the main part of the analysis, which is further decomposed into four steps. 

\subsubsection{Step 1: decomposing the $\ell_{2,\infty}$ estimation error of $\bm{U}$}

By virtue of the proximity of  $\bm{H}$ and $\mathsf{sgn}(\bm{H})$ unveiled in Lemma~\ref{lem:H-property-summary-general}, we are allowed to employ $\bm{U}\bm{H}$ as a surrogate for $\bm{U}\mathsf{sgn}(\bm{H})$, which is more convenient to work with. As it turns out, the discrepancy between $\bm{U}\bm{H}$ and the first-order approximation $\bm{M}\bm{U}^{\star}(\bm{\Lambda}^{\star})^{-1}$, and the discrepancy between $\bm{U}\bm{H}$ and the truth,  can be bounded by three important terms, as asserted below. The proof is built upon elementary  algebra and basic $\ell_2$ perturbation bounds in Section~\ref{sec:preliminary-facts-Linf-general}, and is deferred to Section~\ref{sec:proof-auxiliary-lemmas-general}.

\begin{lemma}
\label{lem:UH-MUstar-diff-decompose}
Suppose that $2c_2\sigma\sqrt{n}\leq \lambda_r^{\star}$ for some sufficiently large constant $c_2>0$. Then with probability at least $1-O(n^{-7})$, one has
\begin{subequations}
	\label{eq:UH-MUstar-Ustar-diff-2inf-general}
	\begin{align}
		\big\|\bm{U}\bm{H}-\bm{M}\bm{U}^{\star}\big(\bm{\Lambda}^{\star}\big)^{-1}\big\|_{2,\infty} & \leq\mathcal{E}_{1}+\mathcal{E}_{2},
		\label{eq:UH-MUstarLambdastar-diff-2inf-general}\\
		\big\|\bm{U}\bm{H}-\bm{U}^{\star}\big\|_{2,\infty} & \leq \mathcal{E}_{1}+\mathcal{E}_{2}+ \mathcal{E}_3,
		\label{eq:UH-Ustar-diff-2inf-general}
	\end{align}
\end{subequations}
where
%
%\begin{subequations}
%\begin{align*}
%	\label{eq:UB-UH-MULambda-UB-lemma}
%
%\begin{array}{cc}
$\mathcal{E}_{1}\coloneqq\frac{2\|\bm{M}(\bm{U}\bm{H}-\bm{U}^{\star})\|_{2,\infty}}{\lambda_{r}^{\star}}$, $\mathcal{E}_{2}\coloneqq\frac{4\|\bm{M}\bm{U}^{\star}\|_{2,\infty}\|\bm{E}\|}{(\lambda_{r}^{\star})^{2}}$, and $\mathcal{E}_{3}\coloneqq\frac{\|\bm{E}\bm{U}^{\star}\|_{2,\infty}}{\lambda_{r}^{\star}}$.
%\end{array}
%
%\end{align*}
%\end{subequations}
%
\end{lemma}

Lemma~\ref{lem:UH-MUstar-diff-decompose} leaves us with three important terms to deal with. The term $\mathcal{E}_1$ is most complicated as it involves the product of two random matrices, whereas $\mathcal{E}_2$ and $\mathcal{E}_3$ can be controlled straightforwardly through our preliminary facts in Section~\ref{sec:preliminary-facts-Linf-general}.   Specifically,   the term $\mathcal{E}_2$ can be bounded by combining \eqref{eq:MU-2-inf-bound-general} with the bound \eqref{eq:noise-size-general} on $\bm{E}$ to obtain
%
% \begin{align*}
% \big\|\bm{M}\bm{U}^{\star}\big\|_{2,\infty} & \leq\big\|\bm{M}^{\star}\bm{U}^{\star}\big\|_{2,\infty}+\big\|\bm{E}\bm{U}^{\star}\big\|_{2,\infty}
%   \leq\big\|\bm{U}^{\star}\big\|_{2,\infty}\|\bm{\Lambda}^{\star}\|+\big\|\bm{E}\bm{U}^{\star}\big\|_{2,\infty}\\
%  & \leq \sqrt{\frac{\mu r}{n}}\big|\lambda_{1}^{\star}\big|+(4+6c_{1})\sigma\sqrt{r\log n},
% \end{align*}
%
% where the first identity makes use of the fact $\bm{M}^{\star}\bm{U}^{\star}=\bm{U}^{\star}\bm{\Lambda}^{\star}$ (so that $\|\bm{M}^{\star}\bm{U}^{\star}\|_{2,\infty}\leq \big\|\bm{U}^{\star}\big\|_{2,\infty}\|\bm{\Lambda}^{\star}\|$), and the last line relies on \eqref{eq:EU-2-inf-bound-general}. This combined with
% the bound \eqref{eq:L2-perturbation-summary-general} on $\|\bm{E}\|$ yields
%
\begin{align}
	% \frac{\big\|\bm{M}\bm{U}^{\star}\big\|_{2,\infty}\|\bm{E}\|}{(\lambda_{r}^{\star})^{2}}
	\mathcal{E}_2 & \leq \frac{4c_{2}\kappa\sigma\sqrt{\mu r}}{\lambda_{r}^{\star}}+\frac{4c_{2}(4+6c_{\mathsf{b}})\sigma^{2}\sqrt{rn\log n}}{(\lambda_{r}^{\star})^{2}}.
	\label{eq:MUstar-E-norm-lambda2-bound}
\end{align}
Regarding the term $\mathcal{E}_3$, the inequality \eqref{eq:EU-2-inf-bound-general} readily gives
\begin{align}
	%\big\|\bm{E}\bm{U}^{\star}\big(\bm{\Lambda}^{\star}\big)^{-1}\big\|_{2,\infty}\leq\frac{\big\|\bm{E}\bm{U}^{\star}\big\|_{2,\infty}}{\lambda_{r}^{\star}}
	\mathcal{E}_3 \leq\frac{(4+6c_{\mathsf{b}} )\sigma\sqrt{r\log n}}{\lambda_{r}^{\star}}.
	\label{eq:E3-bound-general}
\end{align}

Turning to controlling the remaining term $\mathcal{E}_{1}$, a closer inspection, however, reveals substantial challenges, due to the complicated statistical dependency between $\bm{M}$ and $\bm{U}$. To further complicate matters, the term $\mathcal{E}_1$---as we shall demonstrate momentarily---depend on some intrinsic properties of interest about $\bm{U}$ (e.g., $\|\bm{U}\bm{H}-\bm{U}^{\star}\|_{2,\infty}$),  which might lead to circular reasoning if not handled properly. In order to circumvent this issue, we intend to establish the following relation
\begin{equation}
	\mathcal{E}_{1}\leq\mathcal{E}_{1,1}+\rho_{1}\big\|\bm{U}\bm{H}-\bm{U}^{\star}\big\|_{2,\infty}
	\label{eq:E1-UB-E11-rho1}
\end{equation}
for some quantity $\mathcal{E}_{1,1}>0$ that does not involve $\|\bm{U}\bm{H}-\bm{U}^{\star} \|_{2,\infty}$  as well as some contraction factor $0<\rho_{1}\leq1/2$. Assuming the relation~\eqref{eq:E1-UB-E11-rho1} holds for the moment, we have the following useful claim (the proof is straightforward and again postponed to Section~\ref{sec:proof-auxiliary-lemmas-general}).
\begin{lemma}
\label{lem:UH-sequence-bounds-E123}
If Conditions \eqref{eq:UH-MUstar-Ustar-diff-2inf-general} and \eqref{eq:E1-UB-E11-rho1} hold with $0<\rho_{1}\leq1/2$, then we have
\begin{subequations}\label{eq:UH-sequence-bounds-E123}
\begin{align}
\big\|\bm{U}\bm{H}-\bm{U}^{\star}\big\|_{2,\infty} & \leq2\big(\mathcal{E}_{1,1}+\mathcal{E}_{2}+\mathcal{E}_{3}\big), \label{eq:UH-Ustar-E11-E2-E3}\\
\big\|\bm{U}\bm{H}-\bm{M}\bm{U}^{\star}\big(\bm{\Lambda}^{\star}\big)^{-1}\big\|_{2,\infty} & \leq2\big(\mathcal{E}_{1,1}+\mathcal{E}_{2}+\rho_{1}\mathcal{E}_{3}\big),
\label{eq:UH-MUstar-Lambdastar-E11-E2-E3}
\end{align}
\begin{align}
	& \big\|\bm{U}\mathsf{sgn}(\bm{H})-\bm{U}^{\star}\big\|_{2,\infty}  \leq4\big(\mathcal{E}_{1,1}+\mathcal{E}_{2}+\mathcal{E}_{3}\big)+\frac{4c_{2}^{2}\sigma^{2}\sqrt{\mu rn}}{(\lambda_{r}^{\star})^{2}}, \label{eq:UsgnH-Ustar-Lambdastar-E11-E2-E3} \\
	& \big\|\bm{U}\mathsf{sgn}(\bm{H})-\bm{M}\bm{U}^{\star}\big(\bm{\Lambda}^{\star}\big)^{-1}\big\|_{2,\infty} \nonumber \\
& \qquad \leq3\mathcal{E}_{1,1}+3\mathcal{E}_{2}+\Big(2\rho_{1}+\frac{8c_{2}^{2}\sigma^{2}n}{(\lambda_{r}^{\star})^{2}}\Big)\mathcal{E}_{3}
 +\frac{4c_{2}^{2}\sigma^{2}\sqrt{\mu rn}}{(\lambda_{r}^{\star})^{2}}.\label{eq:UsgnH-MUstar-Lambdastar-E11-E2-E3}
\end{align}
\end{subequations}
\end{lemma}
\begin{remark}
When $\mathcal{E}_3$ is the dominant term, the  bound   \eqref{eq:UH-MUstar-Lambdastar-E11-E2-E3}  might be stronger than \eqref{eq:UH-Ustar-E11-E2-E3} if $\rho_1$ is  small. % This subtle improvement proves beneficial for some applications like community detection.
\end{remark}

% Let us take a closer inspection of the three terms $\mathcal{E}_{1}$, $\mathcal{E}_{2}$ and $\mathcal{E}_3$.

%Suppose for the moment that there exist some quantities $\mathcal{E}_{1,1}>0$ and $0<\rho_{1}\leq1/2$ obeying

With this lemma in mind, everything boils down to (i) establishing the relation \eqref{eq:E1-UB-E11-rho1} and (ii) deriving a tight bound on $\mathcal{E}_{1,1}$, which form the main content of the rest of the proof. In light of the triangle inequality
\begin{align*}
	\big\|\bm{M}\big(\bm{U}\bm{H}-\bm{U}^{\star}\big)\big\|_{2,\infty}
	\leq
	\big\|\bm{E} \big(\bm{U}\bm{H}-\bm{U}^{\star}\big)\big\|_{2,\infty}
	+ \big\|\bm{M}^{\star}\big(\bm{U}\bm{H}-\bm{U}^{\star}\big)\big\|_{2,\infty},
\end{align*}
we dedicate the next two steps to bounding $\|\bm{E} (\bm{U}\bm{H}-\bm{U}^{\star} ) \|_{2,\infty}$ and $\|\bm{M}^{\star} (\bm{U}\bm{H}-\bm{U}^{\star}) \|_{2,\infty}$ respectively.

\subsubsection{Step 2: bounding $\|\bm{E} (\bm{U}\bm{H}-\bm{U}^{\star}) \|_{2,\infty}$ via leave-one-out analysis}

To obtain tight row-wise control of $\bm{E}\big(\bm{U}\bm{H}-\bm{U}^{\star}\big)$, one needs to carefully decouple the statistical dependency between $\bm{E}$ and $\bm{U}$,  which is where the leave-one-out idea comes into play.

\paragraph{Step 2.1: a convenient decomposition.}

We start by invoking the triangle inequality to decompose the target quantity as follows
\begin{align}
 & \big\|\bm{E}\big(\bm{U}\bm{H}-\bm{U}^{\star}\big)\big\|_{2,\infty}=\max_{l}\big\|\bm{E}_{l,\cdot}\big(\bm{U}\bm{H}-\bm{U}^{\star}\big)\big\|_{2}\nonumber \\
 & \leq\max_{l}\Big\{\big\|\bm{E}_{l,\cdot}\big(\bm{U}^{(l)}\bm{H}^{(l)}-\bm{U}^{\star}\big)\big\|_{2}+\big\|\bm{E}_{l,\cdot}\big(\bm{U}\bm{H}-\bm{U}^{(l)}\bm{H}^{(l)}\big)\big\|_{2}\Big\}\nonumber \\
 & \leq\max_{l}\Big\{\big\|\bm{E}_{l,\cdot}\big(\bm{U}^{(l)}\bm{H}^{(l)}-\bm{U}^{\star}\big)\big\|_{2}+\|\bm{E}\| \big\|\bm{U}\bm{H}-\bm{U}^{(l)}\bm{H}^{(l)}\big\|_{\mathrm{F}}\Big\}.
	\label{eq:decomposition-E-UH-Ustar}
\end{align}
In words, when controlling the $l$-th row of $\bm{E}\big(\bm{U}\bm{H}-\bm{U}^{\star}\big)$, we attempt to employ $\bm{U}^{(l)}\bm{H}^{(l)}$ as a surrogate of $\bm{U}\bm{H}$.  The benefits to be harvested from this decomposition are:
\begin{itemize}
	\item The statistical independence between $\bm{E}_{l,\cdot}$ and $\bm{U}^{(l)}\bm{H}^{(l)}$ allows for convenient upper bounds on $\big\|\bm{E}_{l,\cdot}\big(\bm{U}^{(l)}\bm{H}^{(l)}-\bm{U}^{\star}\big)\big\|_{2}$;
	
	\item  $\bm{U}\bm{H}$ and  $\bm{U}^{(l)}\bm{H}^{(l)}$ are expected to be exceedingly close, so that the discrepancy incurred by replacing $\bm{U}\bm{H}$ with $\bm{U}^{(l)}\bm{H}^{(l)}$ is negligible.
\end{itemize}
In what follows, we flesh out the proof details.

%With regards to the second term in \eqref{eq:MUh-Ustar-UB12},

%Our starting point is the following decomposition
%%
%\begin{align}
%	& \big\|\bm{M}\big(\bm{U}\bm{H}-\bm{U}^{\star}\big) \big\|_{2,\infty}  \notag\\
%  %\leq\big\|\bm{M}\big(\bm{U}\bm{H}-\bm{U}^{\star}\big)\big\|_{2,\infty}\big\|\big(\bm{\Lambda}^{\star}\big)^{-1}\big\|\\
% %& \qquad
%	%=\frac{\big\|\bm{M}\big(\bm{U}\bm{H}-\bm{U}^{\star}\big)\big\|_{2,\infty}}{\lambda_{r}^{\star}}
%	& \quad \leq   \big\|\bm{M}^{\star}\big(\bm{U}\bm{H}-\bm{U}^{\star}\big)\big\|_{2,\infty} + \big\|\bm{E}\big(\bm{U}\bm{H}-\bm{U}^{\star}\big)\big\|_{2,\infty}.
%	\label{eq:MUh-Ustar-UB12}
%\end{align}
%%

%In what follows, we shall take care of each of these two terms separately.
%

\paragraph{Step 2.2: the proximity of $\bm{U}\bm{H}$ and $\bm{U}^{(l)}\bm{H}^{(l)}$.}
Given that
\begin{align}
	\big\|\bm{U}\bm{H}-\bm{U}^{(l)}\bm{H}^{(l)}\big\|_{\mathrm{F}}
	& =\big\|\bm{U}\bm{U}^{\top}\bm{U}^{\star}-\bm{U}^{(l)}\bm{U}^{(l)\top}\bm{U}^{\star}\big\|_{\mathrm{F}} \notag\\
 	& \leq\big\|\bm{U}\bm{U}^{\top}-\bm{U}^{(l)}\bm{U}^{(l)\top}\big\|_{\mathrm{F}} \big\|\bm{U}^{\star}\big\| \notag\\
 	& =\big\|\bm{U}\bm{U}^{\top}-\bm{U}^{(l)}\bm{U}^{(l)\top}\big\|_{\mathrm{F}},
	\label{eq:UU-UlUl-UB01}
\end{align}
it boils down to bounding $\|\bm{U}\bm{U}^{\top}-\bm{U}^{(l)}\bm{U}^{(l)\top}\|_{\mathrm{F}}$.
Under simple conditions on the eigengap and the perturbation size (see \eqref{eq:lambda-rl-lambda-r1l-M-Ml-LB-step}  in Remark~\ref{remark:perturbation-size-eigen-gap}),  the Davis-Kahan theorem (cf.~Corollary~\ref{cor:davis-kahan-conclusion-corollary}) yields
\begin{align}
	\big\|\bm{U}\bm{U}^{\top}-\bm{U}^{(l)}\bm{U}^{(l)\top}\big\|_{\mathrm{F}}
	& \leq\frac{2\big\|\big(\bm{M}-\bm{M}^{(l)}\big)\bm{U}^{(l)}\big\|_{\mathrm{F}} }{ \big| \lambda_{r}^{(l)} \big| - \big|\lambda_{r+1}^{(l)} \big|} \nonumber\\
	& \leq\frac{4\big\|\big(\bm{M}-\bm{M}^{(l)}\big)\bm{U}^{(l)}\big\|_{\mathrm{F}} }{\lambda_{r}^{\star}}.
	\label{eq:UU-UlUl-UB1}
\end{align}

It remains to develop an upper bound on $\| (\bm{M}-\bm{M}^{(l)} )\bm{U}^{(l)} \|$.
The way we construct $\bm{M}^{(l)}$ (see Section~\ref{sec:construction-LOO-general}) allows us to express
\[
	\big(\bm{M}-\bm{M}^{(l)}\big)\bm{U}^{(l)}=\bm{e}_{l}\bm{E}_{l,\cdot}\bm{U}^{(l)}+  \big(\bm{E}_{\cdot,l} - E_{l,l} \bm{e}_l \big)  \bm{e}_{l}^{\top}\bm{U}^{(l)}.
\]
This together with the triangle inequality and the fact \eqref{eq:A-2inf-H-Hl-general} gives
\begin{align*}
 & \big\|\big(\bm{M}-\bm{M}^{(l)}\big)\bm{U}^{(l)}\big\|_{\mathrm{F}}\leq\big\|\bm{E}_{l,\cdot}\bm{U}^{(l)}\big\|_{2}+\big\|\bm{E}_{\cdot,l} - E_{l,l} \bm{e}_l \big\|_{2}\,\big\|\bm{U}^{(l)}\big\|_{2,\infty}\\
%	& \quad\leq\big\|\bm{E}_{l,\cdot}\bm{U}^{(l)}\big\|_{2}+\big\|\bm{E}\big\|\cdot\big\|\bm{U}^{(l)} \bm{H}^{(l)}  \big\|_{2,\infty} \big\| \big(\bm{H}^{(l)}\big)^{-1} \big\| \\
 & \quad\leq\big\|\bm{E}_{l,\cdot}\bm{U}^{(l)}\big\|_{2}+2\big\|\bm{E}\big\|\, \big\|\bm{U}^{(l)}\bm{H}^{(l)}\big\|_{2,\infty}\\
 & \quad\leq\big\|\bm{E}_{l,\cdot}\bm{U}^{(l)}\big\|_{2}+2\big\|\bm{E}\big\|\, \big\|\bm{U}\bm{H}\big\|_{2,\infty}+2\big\|\bm{E}\big\|\,\big\|\bm{U}\bm{H}-\bm{U}^{(l)}\bm{H}^{(l)}\big\|_{\mathrm{F}}.
\end{align*}
%
%where the last line arises from the triangle inequality.
%
Substitution into \eqref{eq:UU-UlUl-UB01} and \eqref{eq:UU-UlUl-UB1} gives
\begin{align*}
 & \big\|\bm{U}\bm{H}-\bm{U}^{(l)}\bm{H}^{(l)}\big\|_{\mathrm{F}} \leq \big\|\bm{U}\bm{U}^{\top}-\bm{U}^{(l)}\bm{U}^{(l)\top}\big\|_{\mathrm{F}} \\
 & \quad\leq\frac{4\big\|\bm{E}_{l,\cdot}\bm{U}^{(l)}\big\|_{2}+8\big\|\bm{E}\big\|\,\big\|\bm{U}\bm{H}\big\|_{2,\infty}+8\big\|\bm{E}\big\|\, \big\|\bm{U}\bm{H}-\bm{U}^{(l)}\bm{H}^{(l)}\big\|_{\mathrm{F}}}{\lambda_{r}^{\star}}.
\end{align*}
As long as $\|\bm{E}\|/\lambda_r^{\star}\leq 1/16$,
%(which is guaranteed if $16c_2\sigma\sqrt{n}\leq \lambda_r^{\star}$),
 one can further rearrange terms to obtain
\begin{align}
\big\|\bm{U}\bm{H}-\bm{U}^{(l)}\bm{H}^{(l)}\big\|_{\mathrm{F}} & \leq\frac{8\big\|\bm{E}_{l,\cdot}\bm{U}^{(l)}\big\|_{2}+16\big\|\bm{E}\big\|\, \big\|\bm{U}\bm{H}\big\|_{2,\infty}}{\lambda_{r}^{\star}}.
	\label{eq:UH-UlHl-diff-UB10}
\end{align}
In addition, the fact  \eqref{eq:A-2inf-H-Hl-general} combined with the triangle inequality yields
\[
	\tfrac{1}{2}\big\|\bm{E}_{l,\cdot}\bm{U}^{(l)}\big\|_{2}\leq \big\|\bm{E}_{l,\cdot}\bm{U}^{(l)}\bm{H}^{(l)}\big\|_{2}
	\leq \big\|\bm{E}_{l,\cdot}\big(\bm{U}^{(l)}\bm{H}^{(l)}-\bm{U}^{\star}\big)\big\|_{2}+ \big\|\bm{E}_{l,\cdot}\bm{U}^{\star}\big\|_{2},
\]
which taken collectively with \eqref{eq:UH-UlHl-diff-UB10} reveals that
%\yxc{check whether to use $UH$ or $U$}
%
\begin{align}
 & \big\|\bm{U}\bm{H}-\bm{U}^{(l)}\bm{H}^{(l)}\big\|_{\mathrm{F}}
	\leq\frac{16\big\|\bm{E}_{l,\cdot}\big(\bm{U}^{(l)}\bm{H}^{(l)}-\bm{U}^{\star}\big)\big\|_{2}}{\lambda_{r}^{\star}} \nonumber \\
	&  \quad + \frac{ 16 \Big\{ \big\|\bm{E}_{l,\cdot}\bm{U}^{\star}\big\|_{2}  + \big\|\bm{E}\big\| \, \big\|\bm{U}\bm{H}-\bm{U}^{\star}\big\|_{2,\infty} +  \big\|\bm{E}\big\|\, \big\|\bm{U}^{\star}\big\|_{2,\infty} \Big\} }{\lambda_{r}^{\star}}.
	\label{eq:UH-UlHl-diff-UB123}
\end{align}

\paragraph{Step 2.3: bounding $\|\bm{E}_{l,\cdot} (\bm{U}^{(l)}\bm{H}^{(l)}-\bm{U}^{\star} ) \|_{2}$.}

%Another term that appears in  \eqref{eq:decomposition-E-UH-Ustar}
%
%is $\|\bm{E}_{l,\cdot} (\bm{U}^{(l)}\bm{H}^{(l)}-\bm{U}^{\star} ) \|_{2}$, which we look into in this step.
Recognizing that $\bm{E}_{l,\cdot}$ is statistically independent of $\bm{U}^{(l)}$ (since $\bm{U}^{(l)}$ is computed without using $\bm{E}_{l,\cdot}$),
we invoke Lemma~\ref{lemma:general-noise-bound} (more precisely, we use the proof of this lemma) to demonstrate that
with probability exceeding $1-2n^{-6}$,
\begin{align}
 & \big\|\bm{E}_{l,\cdot}\big(\bm{U}^{(l)}\bm{H}^{(l)}-\bm{U}^{\star}\big)\big\|_{2} \notag\\
	%\leq4\sqrt{v\log n}+6L\log n\nonumber \\
 & \leq4\sigma\sqrt{\log n} \, \|\bm{U}^{(l)}\bm{H}^{(l)}-\bm{U}^{\star}\|_{\mathrm{F}}
	+ (6B \log n) \|\bm{U}^{(l)}\bm{H}^{(l)}-\bm{U}^{\star}\|_{2,\infty} \nonumber \\
	& \leq4\sigma \sqrt{\log n} \, \|\bm{U}\bm{H}-\bm{U}^{\star}\|_{\mathrm{F}} + (6B \log n) \|\bm{U}\bm{H}-\bm{U}^{\star}\|_{2,\infty} \nonumber \\
	& \qquad+ ( 10B\log n) \|\bm{U}\bm{H}-\bm{U}^{(l)}\bm{H}^{(l)}\|_{\mathrm{F}}
 \label{eq:El-UH-Ustar-UB-Bern1}
\end{align}
holds simultaneously for all $1\leq l\leq n$, where the last line results from the triangle inequality and the fact $4\sigma\sqrt{\log n}+6B\log n\leq 10B\log n$.

% using Assumption~\ref{eq:assumption-B-sigma}

\paragraph{Step 2.4: combining the above bounds.}

The careful reader would immediately remark that the inequalities \eqref{eq:UH-UlHl-diff-UB123} and \eqref{eq:El-UH-Ustar-UB-Bern1} are convoluted, both of which  involve the terms  $\|\bm{E}_{l,\cdot} (\bm{U}^{(l)}\bm{H}^{(l)}-\bm{U}^{\star} ) \|_{2}$ and $\|\bm{U}\bm{H}-\bm{U}^{(l)}\bm{H}^{(l)}\|_{\mathrm{F}}$. Fortunately, one can substitute \eqref{eq:El-UH-Ustar-UB-Bern1} into \eqref{eq:UH-UlHl-diff-UB123} to produce a cleaner bound. By doing so and exploiting the condition $320B\log n \leq \lambda_r^{\star}$, we  rearrange terms to reach
\begin{align*}
	& \big\|\bm{U}\bm{H}-\bm{U}^{(l)}\bm{H}^{(l)}\big\|_{\mathrm{F}}
	\leq \frac{32\big\|\bm{E}_{l,\cdot}\bm{U}^{\star}\big\|_{2}+32\big\|\bm{E}\big\|\, \big\|\bm{U}^{\star}\big\|_{2,\infty}}{\lambda_{r}^{\star}}\nonumber \\
	& +\frac{128\sigma\sqrt{\log n}\big\|\bm{U}\bm{H}-\bm{U}^{\star}\big\|_{\mathrm{F}}+(32c_2\sigma \sqrt{n} +192B\log n)\big\|\bm{U}\bm{H}-\bm{U}^{\star}\big\|_{2,\infty}}{\lambda_{r}^{\star}},
	%\label{eq:UH-UlHl-diff-UB456}
\end{align*}
where we have also used the upper bound on $\|\bm{E}\|$ derived in \eqref{eq:L2-perturbation-summary-general}.
Meanwhile, plugging the above inequality into \eqref{eq:El-UH-Ustar-UB-Bern1} yields
\begin{align}
	& \big\|\bm{E}_{l,\cdot}\big(\bm{U}^{(l)}\bm{H}^{(l)}-\bm{U}^{\star}\big)\big\|_{2}
	\leq \frac{320B\log n}{\lambda_{r}^{\star}}\big(\big\|\bm{E}_{l,\cdot}\bm{U}^{\star}\big\|_{2}+\big\|\bm{E}\big\|\,\big\|\bm{U}^{\star}\big\|_{2,\infty}\big) \nonumber\\
	&\quad + 5\sigma\sqrt{\log n}\,\|\bm{U}\bm{H}-\bm{U}^{\star}\|_{\mathrm{F}}+(7B\log n)\|\bm{U}\bm{H}-\bm{U}^{\star}\|_{2,\infty}  ,
	\label{eq:El-UH-Ustar-UB-Bern123}
\end{align}
provided that $\max \{ \sigma \sqrt{n}, B\log n\} \leq c_3\lambda_r^{\star}$ for some  small constant $c_3$.

Substituting the preceding two bounds into \eqref{eq:decomposition-E-UH-Ustar} and combining terms reveal the existence of some constant $c_4>0$ such that
\begin{align}
	&  \big\|\bm{E}\big(\bm{U}\bm{H}-\bm{U}^{\star}\big)\big\|_{2,\infty} \notag\\
	& \qquad \leq \alpha_0 + \alpha_1 + c_4\big( \sigma  \sqrt{n} + B\log n \big)\|\bm{U}\bm{H}-\bm{U}^{\star}\|_{2,\infty}
	\label{eq:El-UH-Ustar-UB-Bern123}
\end{align}
provided that $\max \{ \sigma \sqrt{n}, B\log n\} \leq c_3\lambda_r^{\star}$ for some constant $c_3>0$ small enough, where
\begin{align}
\begin{array}{ll}
 & \alpha_{0}\coloneqq\frac{32c_{2}\sigma\sqrt{n}+320B\log n}{\lambda_{r}^{\star}}\big(\big\|\bm{E}\bm{U}^{\star}\big\|_{2,\infty}+\big\|\bm{E}\big\|\,\big\|\bm{U}^{\star}\big\|_{2,\infty}\big) , \\
 & \alpha_{1}\coloneqq6\sigma\sqrt{\log n}\,\|\bm{U}\bm{H}-\bm{U}^{\star}\|_{\mathrm{F}} .
\end{array}	
\label{eq:defn-alpha-0-1-general}
\end{align}

Before continuing,  note that we are already well-equipped to bound the above two quantities. First, $\alpha_0$ can be bounded by
\begin{align}
\alpha_{0} & \leq \frac{32c_{2}\sigma\sqrt{n}+320B\log n}{\lambda_{r}^{\star}}\nonumber \\
 & \qquad\cdot\Big(4\sigma\sqrt{r\log n}+6B\sqrt{\frac{\mu r}{n}}\log n+c_{2}\sigma\sqrt{\mu r}\Big),
\label{eq:defn-alpha-0-general}
\end{align}
where we have used the bounds concerning $\bm{E}$ from Lemma~\ref{lemma:general-noise-bound} as well as the definition \eqref{eq:defn-Ustar-incoherence}.
Regarding $\alpha_1$, it is seen from (\ref{eq:dist-under-H-general}) that
\begin{align}
\alpha_{1}
  \leq\frac{12c_{2}\sigma^{2}\sqrt{rn\log n}}{\lambda_{r}^{\star}}.
	\label{eq:alpha1-bound-1234}
\end{align}

%Substituting the inequalities \eqref{eq:UH-UlHl-diff-UB123} and \eqref{eq:El-UH-Ustar-UB-Bern1} into \eqref{eq:decomposition-E-UH-Ustar} with a little algebra yields
%%
%\begin{align}
% & \big\|\bm{E}(\bm{U}\bm{H}-\bm{U}^{\star})\big\|_{2,\infty}\nonumber \\
% & \quad\leq4\sigma\sqrt{\log n}\|\bm{U}\bm{H}-\bm{U}^{\star}\|_{\mathrm{F}}+(6B\log n)\|\bm{U}\bm{H}-\bm{U}^{\star}\|_{2,\infty}\nonumber \\
% & \qquad+\frac{4\big(\|\bm{E}\|+10B\log n\big)\max_{l}\Big\{\big\|\bm{E}_{l,\cdot}\bm{U}^{(l)}\big\|_{2}+\big\|\bm{E}\big\|\cdot\big\|\bm{U}^{(l)}\big\|_{2,\infty}\Big\}}{\lambda_{r}^{\star}}
%	\label{eq:E-UH-Ustar-UB-234}
%\end{align}
%%

\subsubsection{Step 3: bounding $\|\bm{M}^{\star} (\bm{U}\bm{H}-\bm{U}^{\star} ) \|_{2,\infty}$}

We make the key observation that
\begin{align}
\big\|\bm{M}^{\star}\big(\bm{U}\bm{H}-\bm{U}^{\star}\big)\big\|_{2,\infty} & =\big\|\bm{U}^{\star}\bm{\Lambda}^{\star}\bm{U}^{\star\top}\big(\bm{U}\bm{H}-\bm{U}^{\star}\big)\big\|_{2,\infty} \notag\\
 & \leq\big\|\bm{U}^{\star}\big\|_{2,\infty} \|\bm{\Lambda}^{\star}\|\, \big\|\bm{U}^{\star\top}\big(\bm{U}\bm{H}-\bm{U}^{\star}\big)\big\| \notag\\
	& =  \sqrt{\frac{\mu r}{n}} \|\bm{\Lambda}^{\star}\|\, \big\|\bm{U}\bm{U}^{\top} - \bm{U}^{\star} \bm{U}^{\star\top} \big\|^{2}.
	\label{eq:Mstar-UH-Ustar-2-UB-123}
\end{align}
Here, the last identity holds true due to the following observation
\begin{align*}
\big\|\bm{U}^{\star\top}\big(\bm{U}\bm{H}-\bm{U}^{\star}\big)\big\| & =\big\|\bm{U}^{\star\top}\bm{U}\bm{U}^{\top}\bm{U}^{\star}-\bm{U}^{\star\top}\bm{U}^{\star}\big\|=\big\|\bm{U}^{\star\top}\bm{U}\bm{U}^{\top}\bm{U}^{\star}-\bm{I}\big\|\\
 & =\big\|\bm{Y}(\cos^{2}\bm{\Theta})\bm{Y}^{\top}-\bm{I}\big\|=\big\| \cos^{2}  \bm{\Theta} - \bm{I} \big\|\\
	& =\big\|\sin^{2}\bm{\Theta}\big\|=\big\|\sin\bm{\Theta}\big\|^2,
\end{align*}
where we denote by $\bm{X}(\cos \bm{\Theta}) \bm{Y}^{\top}$ the SVD of $\bm{U}^{\top}\bm{U}^{\star}$, with $\bm{X}$ and $\bm{Y}$ being orthonormal matrices and $\bm{\Theta}$ the diagonal matrix consisting of the principal angles between $\bm{U}$ and $\bm{U}^{\star}$ (see the definition in~\eqref{defn:principal-angles}). The above bounds combined with \eqref{eq:distp-U-Ustar-general} indicate that
\begin{align}
\big\|\bm{M}^{\star}\big(\bm{U}\bm{H}-\bm{U}^{\star}\big)\big\|_{2,\infty} & \leq
	%\big|\lambda_{1}^{\star}\big|\sqrt{\frac{\mu r}{n}}\,\frac{4c_{2}^{2}\sigma^{2}n}{\big(\lambda_{r}^{\star}\big)^{2}}
	\big|\lambda_{1}^{\star}\big| \sqrt{\frac{\mu r}{n}} \, \big\|\sin\bm{\Theta}\big\|^2
	% = \big|\lambda_{1}^{\star}\big| \sqrt{\frac{\mu r}{n}} \, \big\| \sin \bm{\Theta} \big\|^{2}
	\notag\\
	& \leq \frac{4c_{2}^{2}\kappa\sigma^{2}\sqrt{\mu rn}}{\lambda_{r}^{\star}}
 \eqqcolon \alpha_2.
	\label{eq:Mstar-UH-Ustar-2inf}
\end{align}

\subsubsection{Step 4: putting all pieces together}

Combining the bounds \eqref{eq:El-UH-Ustar-UB-Bern123} and \eqref{eq:Mstar-UH-Ustar-2inf} in Steps 2-3 and using the definition of $\mathcal{E}_1$ (see Lemma~\ref{lem:UH-MUstar-diff-decompose}) give
\begin{align}
\mathcal{E}_{1} & \leq\frac{2\big\|\bm{E}\big(\bm{U}\bm{H}-\bm{U}^{\star}\big)\big\|_{2,\infty}+2\big\|\bm{M}^{\star}\big(\bm{U}\bm{H}-\bm{U}^{\star}\big)\big\|_{2,\infty}}{\lambda_{r}^{\star}} \nonumber \\
	& \leq\mathcal{E}_{1,1}+\rho_{1}\|\bm{U}\bm{H}-\bm{U}^{\star}\|_{2,\infty},
	\label{eq:E1-UB-general-1}
\end{align}
where
\begin{equation}
	\mathcal{E}_{1,1}\coloneqq\frac{2\big(\alpha_{0}+\alpha_{1}+\alpha_{2}\big)}{\lambda_{r}^{\star}}
	\quad\text{and}\quad
	\rho_{1}\coloneqq\frac{2c_{4}(\sigma\sqrt{n}+B\log n)}{\lambda_{r}^{\star}}.
	\label{eq:defn-E1-rho1}
\end{equation}
This matches precisely the relation hypothesized in \eqref{eq:E1-UB-E11-rho1}. In particular, one has $0<\rho\leq 1/2$
as long as $4c_4(\sigma \sqrt{n}+ B\log n) \leq \lambda_r^{\star}$, which holds whenever $\sigma \sqrt{n \log n}\leq c_{\sigma} \lambda_r^{\star}$ for a sufficiently small $c_{\sigma}>0$ in view of our assumption on $B$ (cf. \eqref{eq:assumption-B-sigma}).

Recall that $\mathcal{E}_{1,1}$, $\mathcal{E}_{2}$, $\mathcal{E}_{3}$, $\alpha_0$, $\alpha_1$, $\alpha_2$ and $\rho_1$ have been controlled in \eqref{eq:defn-E1-rho1}, \eqref{eq:MUstar-E-norm-lambda2-bound}, \eqref{eq:E3-bound-general}, \eqref{eq:defn-alpha-0-general}, \eqref{eq:alpha1-bound-1234}, \eqref{eq:Mstar-UH-Ustar-2inf} and \eqref{eq:defn-E1-rho1}, respectively.
In addition, recall our assumption \eqref{eq:assumption-B-sigma} and suppose that $\sigma{\sqrt{n}}\leq c_{\sigma}\lambda_{r}^{\star}$ for some sufficiently small constant $c_{\sigma}>0$.
With these bounds and assumptions in mind, invoking Lemma~\ref{lem:UH-sequence-bounds-E123} and combining terms immediately conclude the proof.

\begin{remark} We shall also make note of an immediate consequence of the above argument as follows
\begin{align}
%\big\|\bm{U}\bm{H}-\bm{M}\bm{U}^{\star}\big(\bm{\Lambda}^{\star}\big)^{-1}\big\|_{2,\infty} & \lesssim\frac{\sigma\kappa\sqrt{\mu r}}{\lambda_{r}^{\star}}+\frac{c_{\mathsf{b}}\sigma^{2}\sqrt{rn}\log n}{\big(\lambda_{r}^{\star}\big)^{2}},
	\big\|\bm{U}\bm{H}-\bm{U}^{\star}\big\|_{2,\infty}&\lesssim\frac{\sigma\kappa\sqrt{\mu r}+\sigma\sqrt{r\log n}}{\lambda_{r}^{\star}},
	\label{eq:UH-MUstar-Lambda-star-2inf}
\end{align}
which will prove useful for deriving other important results.
\end{remark}

\subsection{Proof of auxiliary lemmas}
\label{sec:proof-auxiliary-lemmas-general}

\paragraph{Proof of Lemma~\ref{lemma:general-noise-bound}.}

Under Assumption~\ref{assumption-noise-general}, Theorem~\ref{thm:tighter-spectral-normal-ramon}
(in particular \eqref{eq:X-spectral-norm-iid-special-ramon}) reveals
the existence of some constant $c_{2}>0$ such that
\[
\max_{l}\big\|\bm{E}^{(l)}\big\|\leq\|\bm{E}\|\leq c_{2}\sigma\sqrt{n},
\]
holds with probability exceeding $1-O(n^{-7})$.

When it comes to $\bm{E}\bm{A}$, we proceed by viewing its $l$-th
row $\bm{E}_{l,\cdot}\bm{A}$ as a sum of independent random vectors
as follows
\[
\bm{E}_{l,\cdot}\bm{A}=\sum\nolimits _{j}E_{l,j}\bm{A}_{j,\cdot}\eqqcolon\sum\nolimits _{j}\bm{z}_{j},
\]
which can be controlled by the matrix Bernstein inequality. Specifically,
it is seen from Assumption~\ref{assumption-noise-general} that
\begin{align*}
v & \coloneqq\sum\nolimits _{j}\mathbb{E}\big[\|\bm{z}_{j}\|_{2}^{2}\big]%=\sum_{j}\mathbb{E}\big[E_{l,j}^{2}\big]\mathbb{E}\big[\|\bm{A}_{j,\cdot}\|_{2}^{2}\big]
\leq\sigma^{2}\sum\nolimits _{j}\big\|\bm{A}_{j,\cdot}\big\|_{2}^{2}=\sigma^{2}\|\bm{A}\|_{\mathrm{F}}^{2};\\
L & \coloneqq\max_{j}\|\bm{z}_{j}\|_{2}=\max_{j}|E_{l,j}|\,\big\|\bm{Z}_{j,\cdot}\big\|_{2}\leq B\|\bm{Z}\|_{2,\infty}.
\end{align*}
Invoke the matrix Bernstein inequality (cf.~Corollary~\ref{thm:matrix-Bernstein-friendly})
and take the union bound to demonstrate that: with probability exceeding
$1-2n^{-6}$,
\begin{align*}
\big\|\bm{E}_{l,\cdot}\bm{A}\big\|_{2} & \leq4\sqrt{v\log n}+6L\log n\leq4\sigma\sqrt{\log n}\,\|\bm{A}\|_{\mathrm{F}}+(6B\log n)\|\bm{A}\|_{2,\infty}
\end{align*}
holds simultaneously for all $1\leq l\leq n$, thus concluding the
proof.

\paragraph{Proof of Lemma~\ref{lem:preliminary-result-general-iid}.}

Lemma~\ref{lemma:general-noise-bound} tells us that
\begin{equation}
\max_{l}\big\|\bm{E}^{(l)}\big\|\leq\|\bm{E}\|\leq c_{2}\sigma\sqrt{n}\label{eq:noise-E-size-general}
\end{equation}
holds with probability exceeding $1-O(n^{-7})$. Repeating the argument
in the proof of Corollary~\ref{cor:davis-kahan-conclusion-corollary}
reveals that \begin{subequations}
\begin{align}
\max_{1\leq j\leq r}|\lambda_{j}| & \geq\lambda_{r}^{\star}-\|\bm{E}\|,\quad & \max_{1\leq j\leq r}|\lambda_{j}^{(l)}| & \geq\lambda_{r}^{\star}-\|\bm{E}^{(l)}\|,\\
\max_{j:j>r}|\lambda_{j}| & \leq\|\bm{E}\|, & \max_{j:j>r}|\lambda_{j}^{(l)}| & \leq\|\bm{E}^{(l)}\|,
\end{align}
\end{subequations} which taken together with \eqref{eq:noise-E-size-general}
validates \eqref{eq:lambda-size-large-general}-\eqref{eq:lambda-size-small-general}.

Regarding \eqref{eq:dist-U-Ustar-general} and \eqref{eq:distp-U-Ustar-general},
apply Corollary~\ref{cor:davis-kahan-conclusion-corollary} to reach
\begin{align}
\mathsf{dist}(\bm{U},\bm{U}^{\star})\leq\sqrt{2}\|\sin\bm{\Theta}\|\leq\frac{2\|\bm{E}\|}{\lambda_{r}^{\star}}\leq\frac{2c_{2}\sigma\sqrt{n}}{\lambda_{r}^{\star}}\label{eq:dist-UUstar-sin-Theta-general-UB123}
\end{align}
as long as $\|\bm{E}\|\leq c_{2}\sigma\sqrt{n}\leq(1-1/\sqrt{2})\lambda_{r}^{\star}$,
where we have used $\lambda_{r+1}^{\star}=0$. The bound on $\mathsf{dist}\big(\bm{U}^{(l)},\bm{U}^{\star}\big)$
follows from the same argument.

Additionally, regarding $\bm{U}\bm{H}-\bm{U}^{\star}$ we can derive
\begin{align}
 & \|\bm{U}\bm{H}-\bm{U}^{\star}\|_{\mathrm{F}}=\|\bm{U}\bm{U}^{\top}\bm{U}^{\star}-\bm{U}^{\star}\bm{U}^{\star\top}\bm{U}^{\star}\|_{\mathrm{F}}\nonumber \\
 & \qquad\quad\leq\|\bm{U}\bm{U}^{\top}-\bm{U}^{\star}\bm{U}^{\star\top}\|\,\|\bm{U}^{\star}\|_{\mathrm{F}} \nonumber\\
 &\qquad\quad=\|\sin\bm{\Theta}\|\,\|\bm{U}^{\star}\|_{\mathrm{F}}\leq\frac{2c_{2}\sigma\sqrt{rn}}{\lambda_{r}^{\star}},
\end{align}
where the last line arises from Lemma~\ref{prop:unitary-norm-property},
\eqref{eq:distp-U-Ustar-general}, and $\|\bm{U}^{\star}\|_{\mathrm{F}}=\sqrt{r}$.

\paragraph{Proof of Lemma~\ref{lem:H-property-summary-general}.}

We shall only prove the result for $\bm{H}$; the proof for $\bm{H}^{(l)}$
follows from identical arguments and is hence omitted.

From our discussion in Section~\ref{sec:introduction-distance-principal-angles},
one can express the SVD of $\bm{H}=\bm{U}^{\top}\bm{U}^{\star}$ as
$\bm{H}=\bm{X}(\cos\bm{\Theta})\bm{Y}^{\top}$, where the columns
of $\bm{X}$ (resp.~$\bm{Y}$) are the left (resp.~right) singular
vectors of $\bm{H}$, and $\bm{\Theta}$ is a diagonal matrix composed
of the principal angles between $\bm{U}$ and $\bm{U}^{\star}$. In
light of this and the definition (\ref{eq:defn-sgn-Z}), we can establish
\eqref{eq:H-sgnH-dist-general} as follows
\begin{align}
\big\|\bm{H}-\mathsf{sgn}(\bm{H})\big\| & =\big\|\bm{X}\big(\cos\bm{\Theta}-\bm{I}\big)\bm{Y}^{\top}\big\|=\big\|\bm{I}-\cos\bm{\Theta}\big\|\nonumber \\
 & \leq \big\|\bm{I}-\cos^{2}\bm{\Theta}\big\| = \|\sin\bm{\Theta}\|^{2}\nonumber \\
 & \leq\frac{2c_{2}^{2}\sigma^{2}n}{(\lambda_{r}^{\star})^{2}},
	\label{eq:H-sgnH-diff-UB-123}
\end{align}
where the middle line holds since $1-\cos\theta \leq 1 - \cos^2 \theta$,
and the last line follows from \eqref{eq:distp-U-Ustar-general}.

Coming back to the claim \eqref{eq:H-inv-norm-bound-general}, it
suffices to justify that $\sigma_{\min}(\bm{H})\geq1/2$. Recognizing
that $\mathsf{sgn}(\bm{H})=\bm{X}\bm{Y}^{\top}$, we see that all
singular values of $\mathsf{sgn}(\bm{H})$ equal 1. Thus, Weyl's inequality
together with (\ref{eq:H-sgnH-diff-UB-123}) gives
\[
\sigma_{\min}(\bm{H})\geq\sigma_{\min}\big(\mathsf{sgn}(\bm{H})\big)-\|\bm{H}-\mathsf{sgn}(\bm{H})\|\geq1-\frac{2c_{2}^{2}\sigma^{2}n}{(\lambda_{r}^{\star})^{2}}\geq\frac{1}{2},
\]
with the proviso that $2c_{2}\sigma\sqrt{n}\leq\lambda_{r}^{\star}$.

\paragraph{Proof of Lemma~\ref{lem:UH-MUstar-diff-decompose}.}
We start by connecting $\bm{U}\bm{H}\bm{\Lambda}^{\star}$ more explicitly with $\bm{M}\bm{U}^{\star}$ as follows (the invertibility of $\bm{\Lambda}$ can be deduced from \eqref{eq:lambda-size-large-general})
\begin{align}
	\bm{U}\bm{H}\bm{\Lambda}^{\star} & =\bm{U}\bm{\Lambda}\bm{\Lambda}^{-1}\bm{H}\bm{\Lambda}^{\star}
	%=\bm{M}\bm{U}\bm{\Lambda}^{-1}\bm{U}^{\top}\bm{U}^{\star}
	=\bm{M}\bm{U}\bm{\Lambda}^{-1}\bm{U}^{\top}\bm{U}^{\star}\bm{\Lambda}^{\star},\label{eq:UH-derivation-1}
\end{align}
which relies on the definition (\ref{eq:defn-H-Hl-general})
and the eigendecomposition $\bm{M}\bm{U}=\bm{U}\bm{\Lambda}$. In
addition, the eigendecomposition $\bm{M}^{\star}\bm{U}^{\star}=\bm{U}^{\star}\bm{\Lambda}^{\star}$ gives
\begin{align}
	\bm{U}^{\top}\bm{U}^{\star}\bm{\Lambda}^{\star}
	&=\bm{U}^{\top}\bm{M}^{\star}\bm{U}^{\star}=\bm{U}^{\top}\bm{M}\bm{U}^{\star}-\bm{U}^{\top}\bm{E}\bm{U}^{\star} \notag\\
	&= \bm{\Lambda}\bm{U}^{\top}\bm{U}^{\star}-\bm{U}^{\top}\bm{E}\bm{U}^{\star},
	\label{eq:HLambda-Lambda-H-connection}
\end{align}
which taken collectively with  (\ref{eq:UH-derivation-1}) demonstrates that
\begin{align}
\bm{U}\bm{H}\bm{\Lambda}^{\star} & =\bm{M}\bm{U}\bm{\Lambda}^{-1}\bm{\Lambda}\bm{U}^{\top}\bm{U}^{\star}-\bm{M}\bm{U}\bm{\Lambda}^{-1}\bm{U}^{\top}\bm{E}\bm{U}^{\star} \notag\\
 & =\bm{M}\bm{U}\bm{H}-\bm{M}\bm{U}\bm{\Lambda}^{-1}\bm{U}^{\top}\bm{E}\bm{U}^{\star} \notag\\
 & =\bm{M}\bm{U}^{\star}+\bm{M}\big(\bm{U}\bm{H}-\bm{U}^{\star}\big)-\bm{M}\bm{U}\bm{\Lambda}^{-1}\bm{U}^{\top}\bm{E}\bm{U}^{\star}.
\notag
%\label{eq:UH-derivation-2}
\end{align}
Consequently, the difference between $\bm{U}\bm{H}\bm{\Lambda}^{\star}$ and $\bm{M}\bm{U}^{\star}$ obeys
\begin{align}
 & \big\|\bm{U}\bm{H}\bm{\Lambda}^{\star}-\bm{M}\bm{U}^{\star}\big\|_{2,\infty}
  \leq \big\|\bm{M}\big(\bm{U}\bm{H}-\bm{U}^{\star}\big)\big\|_{2,\infty}  \notag\\
 %& \qquad\qquad\qquad\qquad\qquad\qquad
	& \qquad\qquad\qquad +\big\|\bm{M}\bm{U}\bm{\Lambda}^{-1}\bm{U}^{\top}\bm{E}\bm{U}^{\star}\big\|_{2,\infty}.
\label{eq:UH-MULambda-UB22}
\end{align}
%

%We then upper bound the two terms on the right-hand side of \eqref{eq:UH-MULambda-UB22} separately. Regarding the first term,
%we observe that
%%
%\begin{align}
%  \big\|\bm{M}\big(\bm{U}\bm{H}-\bm{U}^{\star}\big)\big(\bm{\Lambda}^{\star}\big)^{-1}\big\|_{2,\infty}
% & \leq \big\|\bm{M}\big(\bm{U}\bm{H}-\bm{U}^{\star}\big)\big\|_{2,\infty}\big\|\big(\bm{\Lambda}^{\star}\big)^{-1}\big\| \notag\\
% &=\frac{\big\|\bm{M}\big(\bm{U}\bm{H}-\bm{U}^{\star}\big)\big\|_{2,\infty}}{\lambda_{r}^{\star}} ,
%	\label{eq:M-UH-U-diff-UB1234}
%\end{align}
%%
%which follows from the assumption \eqref{eq:assumption-lambda-r-positive}.
%
Regarding the second term in \eqref{eq:UH-MULambda-UB22}, one can deduce that
%\lambda_{r}^{\star}
\begin{align}
 & \big\|\bm{M}\bm{U}\bm{\Lambda}^{-1}\bm{U}^{\top}\bm{E}\bm{U}^{\star}\big\|_{2,\infty} \notag\\
 & \quad \leq\big\|\bm{M}\bm{U}\big\|_{2,\infty}\big\|\bm{\Lambda}^{-1}\big\|\cdot\big\|\bm{U}\big\|\cdot\big\|\bm{E}\big\|\cdot\big\|\bm{U}^{\star}\big\| \notag\\
 & \quad \overset{(\mathrm{i})}{\leq}\frac{2\big\|\bm{M}\bm{U}\big\|_{2,\infty}\big\|\bm{E}\big\|}{\lambda_{r}^{\star}}\overset{(\mathrm{ii})}{\leq}\frac{4\big\|\bm{M}\bm{U}\bm{H}\big\|_{2,\infty}\big\|\bm{E}\big\|}{\lambda_{r}^{\star}} \notag\\
	& \quad \overset{(\mathrm{iii})}{\leq} \frac{4\big\|\bm{M}\big(\bm{U}\bm{H}-\bm{U}^{\star}\big)\big\|_{2,\infty}\big\|\bm{E}\big\|}{ \lambda_{r}^{\star}  }
	+ \frac{4\big\|\bm{M}\bm{U}^{\star}\big\|_{2,\infty}\big\|\bm{E}\big\|}{\lambda_{r}^{\star}} \notag\\
	& \quad \overset{(\mathrm{iv})}{\leq} \big\|\bm{M}\big(\bm{U}\bm{H}-\bm{U}^{\star}\big)\big\|_{2,\infty}
	+ \frac{4\big\|\bm{M}\bm{U}^{\star}\big\|_{2,\infty}\big\|\bm{E}\big\|}{\lambda_{r}^{\star}} .
	\label{eq:M-UH-U-diff-UB5678}
\end{align}
Here, (i) follows from the facts $\|\bm{U}\|=\|\bm{U}^{\star}\|=1$ and $|\lambda_r| \geq \lambda_r^{\star} - c_2\sigma\sqrt{n} \geq  \lambda_r^{\star}/2$ (see Lemma~\ref{lem:preliminary-result-general-iid}),
 (ii) holds due to \eqref{eq:A-2inf-H-Hl-general},
% the observation that $\big\|\bm{M}\bm{U}\big\|_{2,\infty}=\big\|\bm{M}\bm{U}\bm{H}\bm{H}^{-1}\big\|_{2,\infty}\leq\big\|\bm{M}\bm{U}\bm{H}\big\|_{2,\infty}\big\|\bm{H}^{-1}\big\|\leq2\big\|\bm{M}\bm{U}\bm{H}\big\|_{2,\infty}$ (given that $\|\bm{H}^{-1}\|\leq 2$ according to \eqref{eq:H-property-summary-general}),
(iii) invokes the triangle inequality,
whereas (iv) holds true provided that $4\|\bm{E}\|\leq \lambda^{\star}_r$.
Combine (\ref{eq:UH-MULambda-UB22}) and \eqref{eq:M-UH-U-diff-UB5678} to reach
\begin{align*}
 & \big\|\bm{U}\bm{H}\bm{\Lambda}^{\star}-\bm{M}\bm{U}^{\star}\big\|_{2,\infty}
  \leq 2\big\|\bm{M}\big(\bm{U}\bm{H}-\bm{U}^{\star}\big)\big\|_{2,\infty}
	+ \frac{4\big\|\bm{M}\bm{U}^{\star}\big\|_{2,\infty}\big\|\bm{E}\big\|}{\lambda_{r}^{\star}},
\end{align*}
which together with the fact that $\|(\bm{\Lambda}^\star)^{-1}\|=1/\lambda^{\star}_{r}$ and the elementary relation $\|\bm{A}\|_{2,\infty}=\|\bm{A}\bm{\Lambda}^\star(\bm{\Lambda}^\star)^{-1}\|_{2,\infty} \leq \|\bm{A}\bm{\Lambda}^\star\|_{2,\infty}\|(\bm{\Lambda}^\star)^{-1}\|$ yields the desired claim \eqref{eq:UH-MUstarLambdastar-diff-2inf-general}.

When it comes to the second claim \eqref{eq:UH-Ustar-diff-2inf-general}, combining \eqref{eq:UH-MUstarLambdastar-diff-2inf-general} with the triangle inequality
\begin{align*}
\big\|\bm{U}\bm{H}-\bm{U}^{\star}\big\|_{2,\infty} & = \big\|\bm{U}\bm{H}-\bm{M}^{\star}\bm{U}^{\star}\big(\bm{\Lambda}^{\star}\big)^{-1}\big\|_{2,\infty}\\
 & \leq\big\|\bm{U}\bm{H}-\bm{M}\bm{U}^{\star}\big(\bm{\Lambda}^{\star}\big)^{-1}\big\|_{2,\infty}+\big\|\bm{E}\bm{U}^{\star}\big(\bm{\Lambda}^{\star}\big)^{-1}\big\|_{2,\infty} \\
 & \leq\big\|\bm{U}\bm{H}-\bm{M}\bm{U}^{\star}\big(\bm{\Lambda}^{\star}\big)^{-1}\big\|_{2,\infty}
	+\big\|\bm{E}\bm{U}^{\star}\big\|_{2,\infty} \,/\, \lambda_r^{\star}
\end{align*}
immediately establishes the advertised bound. Here the last relation again arises from the elementary inequality $\|\bm{A}\bm{B}\|_{2,\infty}\leq \|\bm{A}\|_{2,\infty}\|\bm{B}\|$.

\paragraph{Proof of Lemma~\ref{lem:UH-sequence-bounds-E123}.}
First of all, taking Condition (\ref{eq:E1-UB-E11-rho1}) collectively with (\ref{eq:UH-Ustar-diff-2inf-general})
and rearranging terms yield (\ref{eq:UH-Ustar-E11-E2-E3}):
\begin{equation}
\big\|\bm{U}\bm{H}-\bm{U}^{\star}\big\|_{2,\infty}\leq\frac{1}{1-\rho_{1}}\big(\mathcal{E}_{1,1}+\mathcal{E}_{2}+\mathcal{E}_{3}\big)\leq2\big(\mathcal{E}_{1,1}+\mathcal{E}_{2}+\mathcal{E}_{3}\big),\label{eq:UH-Ustar-E11-E2-E3-1}
\end{equation}
where the last inequality follows from $\rho\leq 1/2$. Substituting (\ref{eq:E1-UB-E11-rho1}) and (\ref{eq:UH-Ustar-E11-E2-E3})
into (\ref{eq:UH-MUstarLambdastar-diff-2inf-general}) then gives \eqref{eq:UH-MUstar-Lambdastar-E11-E2-E3}:
\begin{align}
\big\|\bm{U}\bm{H}-\bm{M}\bm{U}^{\star}\big(\bm{\Lambda}^{\star}\big)^{-1}\big\|_{2,\infty} & \leq\mathcal{E}_{1,1}+2\rho_{1}\big(\mathcal{E}_{1,1}+\mathcal{E}_{2}+\mathcal{E}_{3}\big)+\mathcal{E}_{2}\nonumber \\
 & \leq2\mathcal{E}_{1,1}+2\mathcal{E}_{2}+2\rho_{1}\mathcal{E}_{3},\label{eq:UH-MUstar-Lambdastar-E11-E2-E3-1}
\end{align}
where once again we use the assumption that $\rho\leq 1/2$. 
In addition, the following observation connects $\bm{U}\bm{H}$ with
$\bm{U}\mathsf{sgn}(\bm{H})$:
\begin{align}
 & \big\|\bm{U}\bm{H}-\bm{U}\mathsf{sgn}(\bm{H})\big\|_{2,\infty}
   \leq \|\bm{U}\|_{2,\infty}\big\|\bm{H}-\mathsf{sgn}(\bm{H})\big\|\nonumber \\
 & \qquad \overset{(\mathrm{i})}{\leq}\frac{2c_{2}^{2}\sigma^{2}n}{(\lambda_{r}^{\star})^{2}}\big\|\bm{U}\big\|_{2,\infty}\overset{(\mathrm{ii})}{\leq}\frac{4c_{2}^{2}\sigma^{2}n}{(\lambda_{r}^{\star})^{2}}\big\|\bm{U}\bm{H}\big\|_{2,\infty}\nonumber \\
 & \qquad \overset{(\mathrm{iii})}{\leq}\frac{4c_{2}^{2}\sigma^{2}n}{(\lambda_{r}^{\star})^{2}}\big\|\bm{U}\bm{H}-\bm{U}^{\star}\big\|_{2,\infty}+\frac{4c_{2}^{2}\sigma^{2}n}{(\lambda_{r}^{\star})^{2}}\sqrt{\frac{\mu r}{n}},
 \label{eq:UH-UsgnH-bound-iterate}
\end{align}
where (i) results from (\ref{eq:H-sgnH-dist-general}), (ii) relies on (\ref{eq:A-2inf-H-general}), and (iii) comes from the triangle
inequality $\|\bm{U}\bm{H}\|_{2,\infty}\leq \|\bm{U}\bm{H}-\bm{U}^{\star}\|_{2,\infty} + \|\bm{U}^{\star} \|_{2,\infty}$
and the definition (\ref{eq:defn-Ustar-incoherence}).

The preceding bound
 together with the triangle inequality gives \eqref{eq:UsgnH-Ustar-Lambdastar-E11-E2-E3}:
\begin{align*}
\big\|\bm{U}\mathsf{sgn}(\bm{H})-\bm{U}^{\star}\big\|_{2,\infty} & \leq\big\|\bm{U}\bm{H}-\bm{U}^{\star}\big\|_{2,\infty}+\big\|\bm{U}\bm{H}-\bm{U}\mathsf{sgn}(\bm{H})\big\|_{2,\infty}\\
 & \leq\Big(1+\frac{4c_{2}^{2}\sigma^{2}n}{(\lambda_{r}^{\star})^{2}}\Big)\big\|\bm{U}\bm{H}-\bm{U}^{\star}\big\|_{2,\infty}+\frac{4c_{2}^{2}\sigma^{2}n}{(\lambda_{r}^{\star})^{2}}\sqrt{\frac{\mu r}{n}}\\
 & \leq4\big(\mathcal{E}_{1,1}+\mathcal{E}_{2}+\mathcal{E}_{3}\big)+\frac{4c_{2}^{2}\sigma^{2}\sqrt{\mu rn}}{(\lambda_{r}^{\star})^{2}},
\end{align*}
where the last inequality holds as long as $4c_{2}^{2}\sigma^{2}n\leq(\lambda_{r}^{\star})^{2}$. Additionally, the inequalities \eqref{eq:UH-MUstar-Lambdastar-E11-E2-E3-1} and \eqref{eq:UH-UsgnH-bound-iterate} further allow us to deduce \eqref{eq:UsgnH-MUstar-Lambdastar-E11-E2-E3}:
\begin{align*}
 & \big\|\bm{U}\mathsf{sgn}(\bm{H})-\bm{M}\bm{U}^{\star}\big(\bm{\Lambda}^{\star}\big)^{-1}\big\|_{2,\infty} \\
 & \quad \leq\big\|\bm{U}\bm{H}-\bm{M}\bm{U}^{\star}\big(\bm{\Lambda}^{\star}\big)^{-1}\big\|_{2,\infty}+\big\|\bm{U}\bm{H}-\bm{U}\mathsf{sgn}(\bm{H})\big\|_{2,\infty}\\
 % & \quad\leq2\mathcal{E}_{1,1}+2\mathcal{E}_{2}+2\rho_{1}\mathcal{E}_{3}+\frac{4c_{2}^{2}\sigma^{2}n}{(\lambda_{r}^{\star})^{2}}\big\|\bm{U}\bm{H}-\bm{U}^{\star}\big\|_{2,\infty}+\frac{4c_{2}^{2}\sigma^{2}n}{(\lambda_{r}^{\star})^{2}}\sqrt{\frac{\mu r}{n}}\\
 & \quad\leq2\mathcal{E}_{1,1}+2\mathcal{E}_{2}+2\rho_{1}\mathcal{E}_{3}+\frac{8c_{2}^{2}\sigma^{2}n}{(\lambda_{r}^{\star})^{2}}\big(\mathcal{E}_{1,1}+\mathcal{E}_{2}+\mathcal{E}_{3}\big)+\frac{4c_{2}^{2}\sigma^{2}\sqrt{\mu rn}}{(\lambda_{r}^{\star})^{2}}\\
 & \quad\leq3\mathcal{E}_{1,1}+3\mathcal{E}_{2}+\Big(2\rho_{1}+\frac{8c_{2}^{2}\sigma^{2}n}{(\lambda_{r}^{\star})^{2}}\Big)\mathcal{E}_{3}+\frac{4c_{2}^{2}\sigma^{2}\sqrt{\mu rn}}{(\lambda_{r}^{\star})^{2}} .
\end{align*}
Here, the penultimate line combines (\ref{eq:UH-Ustar-E11-E2-E3}), (\ref{eq:UH-MUstar-Lambdastar-E11-E2-E3})
and (\ref{eq:UH-UsgnH-bound-iterate}), while the last inequality relies on the
assumption $8 c_{2}^{2}\sigma^{2}n\leq(\lambda_{r}^{\star})^{2}$.

%% file: chapters/analysis_entrywise.tex
\section{Appendix B: Proof of Corollary~\ref{cor:entrywise-error-general}}
\label{sec:proof-cor-entrywise-error}

% Preview source code from paragraph 14 to 15

Moving on to the proof of Corollary \ref{cor:entrywise-error-general},
we start by pointing out the main issue that deserves particular attention.
Roughly speaking, we have learned from Theorem \ref{thm:UsgnH-Ustar-MUstar-general}
(and its analysis) that $\bm{U}^{\star}\approx\bm{U}\bm{H}$ under
mild conditions, which naturally suggests that
\[
\bm{M}^{\star}=\bm{U}^{\star}\bm{\Lambda}^{\star}\bm{U}^{\star\top}\approx\bm{U}\bm{H}\bm{\Lambda}^{\star}\bm{H}^{\top}\bm{U}^{\top}.
\]
As a result, in order to enable $\bm{M}^{\star}\approx\bm{U}\bm{\Lambda}\bm{U}^{\top}$,
one would need to ensure
%\[
%\bm{U}\bm{H}\bm{\Lambda}^{\star}\bm{H}^{\top}\bm{U}^{\top}\approx\bm{U}\bm{\Lambda}\bm{U}^{\top},
%\]
%which would happen as long as 
$\bm{H}\bm{\Lambda}^{\star}\bm{H}^{\top}\approx\bm{\Lambda}$.

The above argument, while highly informal, reveals the core idea underlying
the proof. Our proof is based upon the following observation
\begin{align}
	& \big\|\bm{U}\bm{\Lambda}\bm{U}^{\top}-\bm{U}^{\star}\bm{\Lambda}^{\star}\bm{U}^{\star\top}\big\|_{\infty}
  \leq\underset{\eqqcolon\mathcal{\gamma}_{1}}{\underbrace{\big\|\bm{U}\big(\bm{\Lambda}-\bm{H}\bm{\Lambda}^{\star}\bm{H}^{\top}\big)\bm{U}^{\top}\big\|_{\infty}}}\nonumber \\
 & \quad\quad\quad\quad+
	\underset{\eqqcolon\gamma_{2}}{\underbrace{\big\|\bm{U}\bm{H}\bm{\Lambda}^{\star}\bm{H}^{\top}\bm{U}^{\top}-\bm{U}^{\star}\bm{\Lambda}^{\star}\bm{U}^{\star\top}\big\|_{\infty}}},
\label{eq:ULambdaU-YLambdaY-decompose}
\end{align}
which leaves us with two terms to cope with.

% We now move on to the proof of Corollary~\ref{cor:entrywise-error-general}.   In order to study the estimation error of $\bm{U}\bm{\Lambda}\bm{U}^{\top}$,

\subsubsection{Step 1: bounding $\gamma_1$}

Regarding $\gamma_1$ defined in \eqref{eq:ULambdaU-YLambdaY-decompose}, it is seen that
\begin{align}
	\gamma_{1} & \leq\big\|\bm{U}\big\|_{2,\infty}^{2}\big\|\bm{\Lambda}-\bm{H}\bm{\Lambda}^{\star}\bm{H}^{\top}\big\|\leq\frac{4\mu r}{n}\big\|\bm{\Lambda}-\bm{H}\bm{\Lambda}^{\star}\bm{H}^{\top}\big\|.
	\label{eq:gamma1-UB-12345}
\end{align}
In the last relation, we have exploited the fact that
\begin{align}
	\big\|\bm{U}\big\|_{2,\infty}&=\big\|\bm{U}\mathsf{sgn}(\bm{H})\big\|_{2,\infty}
\leq\big\|\bm{U}^{\star}\big\|_{2,\infty}+\big\|\bm{U}\mathsf{sgn}(\bm{H})-\bm{U}^{\star}\big\|_{2,\infty} \notag\\
	&\leq 2\sqrt{\mu r/n},
	\label{eq:incoherence-U-estimate}
\end{align}
where the last inequality relies on (\ref{eq:UsgnH-Ustar-bound-theorem-general})
and the assumption $\sigma\sqrt{n}(\kappa+\sqrt{\log n})\leq c_{1}\lambda_{r}^{\star}$
for some sufficiently small constant $c_{1}>0$.

It then boils down to bounding $\|\bm{\Lambda}-\bm{H}\bm{\Lambda}^{\star}\bm{H}^{\top}\|$.
Towards this, it is  seen from the identity (\ref{eq:HLambda-Lambda-H-connection})
and the definition $\bm{H}=\bm{U}^{\top}\bm{U}^{\star}$ that
\begin{align*}
\bm{H}\bm{\Lambda}^{\star}\bm{H}^{\top}-\bm{\Lambda} & =\bm{\Lambda}\bm{H}\bm{H}^{\top}-\bm{U}^{\top}\bm{E}\bm{U}^{\star}\bm{H}^{\top}-\bm{\Lambda},
\end{align*}
which together with the triangle inequality reveals that
\begin{align}
\big\|\bm{H}\bm{\Lambda}^{\star}\bm{H}^{\top}-\bm{\Lambda}\big\| & \leq\big\|\bm{\Lambda}\big(\bm{H}\bm{H}^{\top}-\bm{I}\big)\big\|+\big\|\bm{U}^{\top}\bm{E}\bm{U}^{\star}\bm{H}^{\top}\big\|.
	\label{eq:HLambda-H-ambda-bound-2terms}
\end{align}
The rest of this step is devoted to controlling the above two terms.

With regards to the first term on the right-hand side of (\ref{eq:HLambda-H-ambda-bound-2terms}), we make the observation that
\[
\|\bm{H}\bm{H}^{\top}-\bm{I}\|=\|\cos^{2}\bm{\Theta}-\bm{I}\|=\|\sin^{2}\bm{\Theta}\|\leq\frac{2c_{2}^2\sigma^{2}n}{(\lambda_{r}^{\star})^{2}},
\]
where, as usual, $\bm{\Theta}$ denotes a diagonal matrix composed
of the principal angles between $\bm{U}$ and $\bm{U}^{\star}$, and the last inequality results from Lemma~\ref{lem:preliminary-result-general-iid}. This
combined with Weyl's inequality and Lemma \ref{lem:preliminary-result-general-iid} leads to
\begin{align}
\|\bm{\Lambda}(\bm{H}\bm{H}^{\top}-\bm{I})\| & \leq\|\bm{\Lambda}\|\,\|\bm{H}\bm{H}^{\top}-\bm{I}\|\leq(\|\bm{\Lambda}^{\star}\|+\|\bm{E}\|)\,\|\bm{H}\bm{H}^{\top}-\bm{I}\|\nonumber \\
 & \leq\big(\big|\lambda_{1}^{\star}\big|+c_{2}\sigma\sqrt{n}\big)\frac{2c_{2}^2\sigma^{2}n}{(\lambda_{r}^{\star})^{2}}\leq\frac{4c_{2}^2\kappa\sigma^{2}n}{\lambda_{r}^{\star}},
\label{eq:Lambda-HHT-I-bound}
\end{align}
provided that $c_{2}\sigma\sqrt{n}\leq\lambda_{r}^{\star}\leq|\lambda_{1}^{\star}|$.

When it comes to the second term on the right-hand side of (\ref{eq:HLambda-H-ambda-bound-2terms}),
let us introduce an orthonormal matrix $\bm{R}\coloneqq\arg\min_{\bm{Q}\in\mathcal{O}^{r\times r}}\|\bm{U}\bm{Q}-\bm{U}^{\star}\|$, which helps us derive
\begin{align}
\big\|\bm{U}^{\top}\bm{E}\bm{U}^{\star}\bm{H}^{\top}\big\| & \leq\big\|\bm{U}^{\top}\bm{E}\bm{U}^{\star}\big\|=\big\|\bm{R}^{\top}\bm{U}^{\top}\bm{E}\bm{U}^{\star}\big\| \notag\\
 & \leq\big\|\bm{U}^{\star\top}\bm{E}\bm{U}^{\star}\big\|+\big\|\big(\bm{U}\bm{R}-\bm{U}^{\star}\big)^{\top}\bm{E}\bm{U}^{\star}\big\| \notag\\
	& \leq\big\|\bm{U}^{\star\top}\bm{E}\bm{U}^{\star}\big\|+\|\bm{E}\|\,\mathsf{dist}\big(\bm{U},\bm{U}^{\star}\big). \label{eq:Utop-E-U-H-bound12}
\end{align}
Here, the first inequality holds since $\|\bm{H}\|=\|\bm{U}^{\top}\bm{U}^{\star}\|\leq1$,
% the second inequality relies on Remark \ref{remark:AH-norm-bound},
while the last line follows since $\|\bm{U}^{\star}\|=1$ and $\|\bm{U}\bm{R}-\bm{U}^{\star}\|=\mathsf{dist}(\bm{U},\bm{U}^{\star})$.
In addition, we claim that with probability at least $1-2n^{-7}$,
\begin{equation}
\|\bm{U}^{\star\top}\bm{E}\bm{U}^{\star}\|\leq(6+12 c_{\mathsf{b}})\sigma\sqrt{r\log n}.
	\label{eq:claim-Utop-E-Ustar-spectral-norm}
\end{equation}
If this claim were valid, then one could continue the derivation (\ref{eq:Utop-E-U-H-bound12}) and invoke Lemma \ref{lem:preliminary-result-general-iid} to demonstrate that
\begin{equation}
\big\|\bm{U}^{\top}\bm{E}\bm{U}^{\star}\bm{H}^{\top}\big\|\leq(6+12 c_{\mathsf{b}})\sigma\sqrt{r\log n}+\frac{2c_{2}^{2}\sigma^{2}n}{\lambda_{r}^{\star}}.
	\label{eq:UB-UEUHtop-UB1}
\end{equation}

To finish up, substituting \eqref{eq:Lambda-HHT-I-bound} and \eqref{eq:UB-UEUHtop-UB1} into \eqref{eq:HLambda-H-ambda-bound-2terms} yields
\begin{align}
	\big\|\bm{H}\bm{\Lambda}^{\star}\bm{H}^{\top}-\bm{\Lambda}\big\|\leq\frac{6c_{2}^2\kappa\sigma^{2}n}{\lambda_{r}^{\star}}+ (6+12c_{\mathsf{b}})\sigma\sqrt{r\log n} ,
	\label{eq:H-Lambda-star-H-Lambda-dist}
\end{align}
which combined with \eqref{eq:gamma1-UB-12345} leads to
\begin{equation}
\gamma_{1}\lesssim \frac{\sigma^{2} \kappa\mu r}{\lambda_{r}^{\star}}+\frac{\sigma\mu\sqrt{r^{3}\log n}}{n}.
\label{eq:gamma1-UB-2346}
\end{equation}

\subsubsection{Step 2: bounding $\gamma_2$}

%Before proceeding, we first gather a simple bound on $\|\bm{Y}\|_{2,\infty}$.
%It is seen from the definition of $\bm{Y}$ that
%\begin{align}
%\|\bm{Y}\|_{2,\infty} & \leq\big\|\bm{U}^{\star}\big\|_{2,\infty}+\big\|\bm{E}\bm{U}^{\star}\big(\bm{\Lambda}^{\star}\big)^{-1}\big\|_{2,\infty}\leq\sqrt{\frac{\mu r}{n}}+\frac{\big\|\bm{E}\bm{U}^{\star}\big\|_{2,\infty}}{\lambda_{r}^{\star}}\nonumber \\
% & \overset{(\mathrm{i})}{\leq}\sqrt{\frac{\mu r}{n}}+\frac{10\sigma\sqrt{r\log n}}{\lambda_{r}^{\star}}\overset{(\mathrm{ii})}{\leq}2\sqrt{\frac{\mu r}{n}},\label{eq:Y-2-inf-bound-123}
%\end{align}
%where (i) comes from (\ref{eq:EU-2-inf-bound-general}) and the assumption
%$c_{\mathsf{b}}\leq1$, and (ii) is valid as long as $10\sigma\sqrt{n\log n}\leq\lambda_{r}^{\star}$.
%In addition, it follows from (\ref{eq:UH-MUstar-Lambda-star-2inf})
%that
%\begin{align}
%\big\|\bm{U}\bm{H}-\bm{Y}\big\|_{2,\infty} & \lesssim\frac{\sigma\kappa\sqrt{\mu r}}{\lambda_{r}^{\star}}+\frac{c_{\mathsf{b}}\sigma^{2}\sqrt{rn}\log n}{\big(\lambda_{r}^{\star}\big)^{2}}\lesssim\frac{\sigma\kappa\sqrt{\mu r\log n}}{\lambda_{r}^{\star}},\label{eq:UH-MUstar-Lambda-star-2inf-1}
%\end{align}
%with the proviso that $\sigma\sqrt{n\log n}\lesssim\lambda_{r}^{\star}$.

Before proceeding, we recall from (\ref{eq:UH-MUstar-Lambda-star-2inf}) that
\begin{align}
	\big\|\bm{U}\bm{H}-\bm{U}^{\star}\big\|_{2,\infty}
	\lesssim  \frac{\sigma\kappa\sqrt{\mu r}+\sigma\sqrt{r\log n}}{\lambda_{r}^{\star}}
	\lesssim  \frac{\sigma\kappa\sqrt{\mu r \log n}}{\lambda_{r}^{\star}}
	\label{eq:UH-MUstar-Lambda-star-2inf-1}
\end{align}
%
% with the proviso that $\sigma\sqrt{n\log n}\lesssim\lambda_{r}^{\star}$.

Recognizing the basic decomposition
\begin{align*}
	\gamma_2 = \big\| \bm{U}\bm{H}\bm{\Lambda}^{\star}\bm{H}^{\top}\bm{U}^{\top}-\bm{U}^{\star}\bm{\Lambda}^{\star}\bm{U}^{\star\top} \big\|_{\infty}
	= \big\| \bm{A}_1 + \bm{A}_1^{\top} + \bm{A}_2 \big\|_{\infty}
%(\bm{U}\bm{H}-\bm{U}^{\star})\bm{\Lambda}^{\star}(\bm{U}\bm{H}-\bm{U}^{\star})^{\top}\\
% & \quad\quad\quad+(\bm{U}\bm{H}-\bm{U}^{\star})\bm{\Lambda}^{\star}\bm{U}^{\star\top}+\bm{U}^{\star}\bm{\Lambda}^{\star}(\bm{U}\bm{H}-\bm{U}^{\star})^{\top},
\end{align*}
with $\bm{A}_{1}\coloneqq(\bm{U}\bm{H}-\bm{U}^{\star})\bm{\Lambda}^{\star}\bm{U}^{\star\top}$
and $\bm{A}_{2}\coloneqq(\bm{U}\bm{H}-\bm{U}^{\star})\bm{\Lambda}^{\star}(\bm{U}\bm{H}-\bm{U}^{\star})^{\top}$,
we can control each of these terms separately.
Firstly, observe that
\begin{align*}
\big\|\bm{A}_{1}\big\|_{\infty} & \leq\|\bm{U}\bm{H}-\bm{U}^{\star}\|_{2,\infty}\big\|\bm{U}^{\star}\big\|_{2,\infty}\big\|\bm{\Lambda}^{\star}\big\|\\
 & \lesssim\big|\lambda_{1}^{\star}\big|\sqrt{\frac{\mu r}{n}}\cdot\frac{\sigma\kappa\sqrt{\mu r\log n}}{\lambda_{r}^{\star}}\asymp\sigma\kappa^{2}\mu r\sqrt{\frac{\log n}{n}},
\end{align*}
where we have made use of (\ref{eq:UH-MUstar-Lambda-star-2inf-1}). Similarly,
\begin{align*}
	\big\|\bm{A}_{2}\big\|_{\infty} & \leq\|\bm{U}\bm{H}-\bm{U}^{\star}\|_{2,\infty}^{2}\big\|\bm{\Lambda}^{\star}\big\|
	\lesssim\big|\lambda_{1}^{\star}\big|\frac{\sigma^{2}\kappa^{2}\mu r\log n}{\big(\lambda_{r}^{\star}\big)^{2}}\\
 	& \asymp\frac{\sigma^{2}\kappa^{3}\mu r\log n}{\lambda_{r}^{\star}}\lesssim\sigma\kappa^{2}\mu r\sqrt{\frac{\log n}{n}},
\end{align*}
provided that $\sigma\kappa\sqrt{n\log n}\lesssim\lambda_{r}^{\star}$.
Consequently,
\begin{align*}
% \big\|\bm{U}\bm{H}\bm{\Lambda}^{\star}\bm{H}^{\top}\bm{U}^{\top}-\bm{Y}\bm{\Lambda}^{\star}\bm{Y}^{\top}\big\|_{\infty}
\gamma_2 & \leq2\big\|\bm{A}_{1}\big\|_{\infty}+\big\|\bm{A}_{2}\big\|_{\infty}\lesssim\sigma\kappa^{2}\mu r\sqrt{\frac{\log n}{n}}.
\end{align*}

\subsubsection{Step 3: putting all this together}

%Combining the above bounds, we demonstrate that
%%
%\begin{align}
%	& \big\|\bm{U}\bm{\Lambda}\bm{U}^{\top}-\bm{Y}\bm{\Lambda}^{\star}\bm{Y}^{\top}\big\|_{\infty}
%	\leq\gamma_{1}+\gamma_{2} \notag\\
%& \lesssim\sigma\kappa^{2}\mu r\sqrt{\frac{\log n}{n}}+\frac{\kappa\mu r\sigma^{2}}{\lambda_{r}^{\star}}+\frac{\sigma\mu r\sqrt{r\log n}}{n}
%  \asymp\sigma\kappa^{2}\mu r\sqrt{\frac{\log n}{n}}.
%\end{align}
%%

Combining the above bounds, we demonstrate that
\begin{align}
	& \big\|\bm{U}\bm{\Lambda}\bm{U}^{\top}-\bm{U}^{\star}\bm{\Lambda}^{\star}\bm{U}^{\star\top} \big\|_{\infty}
	\leq\gamma_{1}+\gamma_{2} \notag\\
& \lesssim\sigma\kappa^{2}\mu r\sqrt{\frac{\log n}{n}}+\frac{\kappa\mu r\sigma^{2}}{\lambda_{r}^{\star}}+\frac{\sigma\mu r\sqrt{r\log n}}{n}
  \asymp\sigma\kappa^{2}\mu r\sqrt{\frac{\log n}{n}},
\end{align}
provided that $\sigma\sqrt{n} \lesssim \lambda_r^{\star}$. This concludes the proof of Corollary~\ref{cor:entrywise-error-general}, as long as the claim \eqref{eq:claim-Utop-E-Ustar-spectral-norm} can be validated.

\subsubsection{Proof of the claim \eqref{eq:claim-Utop-E-Ustar-spectral-norm}}

% Preview source code for paragraph 3

Let us start by expressing $\bm{U}^{\star\top}\bm{E}\bm{U}^{\star}$
as a sum of independent random matrices as follows
\[
\bm{U}^{\star\top}\bm{E}\bm{U}^{\star}=\sum_{i,j:\,i\geq j}E_{i,j}\Big\{\big(\bm{U}_{i,\cdot}^{\star}\big)^{\top}\bm{U}_{j,\cdot}^{\star}+\big(\bm{U}_{j,\cdot}^{\star}\big)^{\top}\bm{U}_{i,\cdot}^{\star}\Big\}\eqqcolon\sum_{i,j:\,i\geq j}\bm{Z}_{i,j}.
\]
From the elementary inequality $(\bm{A}+\bm{A}^{\top})^{2}\preceq2\bm{A}\bm{A}^{\top}+2\bm{A}^{\top}\bm{A}$, we have
\begin{align*}
\mathbb{E}\big[\bm{Z}_{i,j}^{2}\big] & \preceq\sigma^{2}\Big\{\big(\bm{U}_{i,\cdot}^{\star}\big)^{\top}\bm{U}_{j,\cdot}^{\star}+\big(\bm{U}_{j,\cdot}^{\star}\big)^{\top}\bm{U}_{i,\cdot}^{\star}\Big\}\Big\{\big(\bm{U}_{i,\cdot}^{\star}\big)^{\top}\bm{U}_{j,\cdot}^{\star}+\big(\bm{U}_{j,\cdot}^{\star}\big)^{\top}\bm{U}_{i,\cdot}^{\star}\Big\}^{\top}\\
 & \preceq2\sigma^{2}\Big\{\big\|\bm{U}_{j,\cdot}^{\star}\big\|_{2}^{2}\big(\bm{U}_{i,\cdot}^{\star}\big)^{\top}\bm{U}_{i,\cdot}^{\star}+\big\|\bm{U}_{i,\cdot}^{\star}\big\|_{2}^{2}\big(\bm{U}_{j,\cdot}^{\star}\big)^{\top}\bm{U}_{j,\cdot}^{\star}\Big\},
\end{align*}
thus indicating that
\begin{align*}
v & =\Big\|\sum_{i,j:i\geq j}\mathbb{E}\big[\bm{Z}_{i,j}^{2}\big]\Big\|\leq2\sigma^{2}\Big\|\sum_{i=1}^{n}\sum_{j=1}^{n}\big\|\bm{U}_{j,\cdot}^{\star}\big\|_{2}^{2}\big(\bm{U}_{i,\cdot}^{\star}\big)^{\top}\bm{U}_{i,\cdot}^{\star}\Big\|\\
 & =2\sigma^{2}\|\bm{U}^{\star}\|_{\mathrm{F}}^{2}\Big\|\sum_{i}\big(\bm{U}_{i,\cdot}^{\star}\big)^{\top}\bm{U}_{i,\cdot}^{\star}\Big\|
	=2\sigma^{2}r \big\|\bm{U}^{\star\top}\bm{U}^{\star}\big\| =2\sigma^{2}r.
\end{align*}
In addition, each matrix $\bm{Z}_{i,j}$ can be bounded in size by
\[
\max_{i,j}\|\bm{Z}_{i,j}\|\leq2\max_{i,j}\big|E_{i,j}\big|\max_{i,j}\big\|\bm{U}_{i,\cdot}^{\star}\big\|_{2}\big\|\bm{U}_{j,\cdot}^{\star}\big\|_{2}\leq2B\frac{\mu r}{n}\eqqcolon L,
\]
where we have used the definition of the incoherence parameter $\mu$.
Apply the matrix Bernstein inequality (see Corollary \ref{thm:matrix-Bernstein-friendly}) to reach
\begin{align*}
\big\|\bm{U}^{\star\top}\bm{E}\bm{U}^{\star}\big\| & \leq4\sqrt{v\log n}+6L\log n\leq\sigma\sqrt{32r\log n}+ 12B\frac{\mu r\log n}{n}\\
 & \leq(6+ 12c_{\mathsf{b}})\sigma\sqrt{r\log n}
\end{align*}
with probability exceeding $1-2n^{-7}.$ Here, the last line holds since
\[
B\frac{\mu r\log n}{n}\leq c_{\mathsf{b}}\sigma\sqrt{\frac{n}{\mu\log n}}\cdot\frac{\mu r\log n}{n}=c_{\mathsf{b}}\sigma\sqrt{r\log n}\sqrt{\frac{\mu r}{n}}\leq c_{\mathsf{b}}\sigma\sqrt{r\log n},
\]
which relies on the assumption (\ref{eq:assumption-B-sigma}) and the basic fact $\mu\leq n/r$.

%% file: chapters/analysis_Linf_asym.tex
\section{Appendix C: Proof of Theorem~\ref{thm:2-inf-asymm}}
\label{sec:proof-inf-rect}

\paragraph{A symmetrization trick. }

As alluded to previously, the proof is built  on a ``symmetric dilation'' trick that helps symmetrize a general matrix. We start with the following definition.  

\begin{definition}[Symmetric dilation]
\label{defn:sym-dilation}
For any matrix $\bm{A}\in\mathbb{R}^{n_{1}\times n_{2}}$, its symmetric dilation $\mathcal{S}(\bm{A})\in\mathbb{R}^{(n_{1}+n_{2})\times(n_{1}+n_{2})}$ is defined to be 
\[
\mathcal{S}(\bm{A})=\left[\begin{array}{cc}
\bm{0} & \bm{A}\\
\bm{A}^{\top} & \bm{0}
\end{array}\right].
\]
 \end{definition}
Apart from the symmetry of $\mathcal{S}(\bm{A})$,
which is immediate from its definition, the main benefit of the symmetric
dilation lies in the correspondence between the eigendecomposition
of $\mathcal{S}(\bm{A})$ and the singular value decomposition of
$\bm{A}$. More specifically, let $\bm{U}\bm{\Sigma}\bm{V}^{\top}$
be the SVD of $\bm{A}$. Then one has the following eigendecomposition
for $\mathcal{S}(\bm{A})$: 
\begin{equation}
\mathcal{S}(\bm{A})=\frac{1}{\sqrt{2}}\left[\begin{array}{cc}
\bm{U} & \bm{U}\\
\bm{V} & -\bm{V}
\end{array}\right]\cdot\left[\begin{array}{cc}
\bm{\Sigma} & \bm{0}\\
\bm{0} & -\bm{\Sigma}
\end{array}\right]\cdot\frac{1}{\sqrt{2}}\left[\begin{array}{cc}
\bm{U} & \bm{U}\\
\bm{V} & -\bm{V}
\end{array}\right]^{\top}. \label{eq:dilation-eigen}
\end{equation}
Here, the columns of $\frac{1}{\sqrt{2}}\left[{\scriptsize \begin{array}{cc}
\bm{U} & \bm{U}\\
\bm{V} & -\bm{V}
\end{array}}\right]$
are orthonormal and represent the eigenvectors of $\mathcal{S}(\bm{A})$, whereas
$\left[{\scriptsize \begin{array}{cc}
\bm{\Sigma} & \bm{0}\\
\bm{0} & -\bm{\Sigma}
\end{array}}\right]$ contains all (non-zero) eigenvalues
of $\mathcal{S}(\bm{A})$. 

Utilizing this ``symmetric dilation'' trick, we can translate the observation
model $\bm{M}=\bm{M}^{\star}+\bm{E}$ into the following equivalent
form
\[
\mathcal{S}(\bm{M})=\mathcal{S}(\bm{M}^{\star})+\mathcal{S}(\bm{E}),
\]
which is in line with the symmetric observation model stated in
(\ref{eq:M-noisy-copy-general}).

\paragraph{Verifying conditions. }

% (cf.~Theorem~\ref{thm:UsgnH-Ustar-MUstar-general} and Corollary~\ref{cor:entrywise-error-general})
To invoke the general theory in Section \ref{sec:setup-general-theory-independent},
one is required to first examine the spectral properties of $\mathcal{S}(\bm{M}^{\star})$
and $\mathcal{S}(\bm{M})$, as well as the  assumptions on the noise part
$\mathcal{S}(\bm{E})$. 

%from Section~\ref{sec:general-theory-setup-asymm} 

Recall that $\bm{M}^{\star}=\bm{U}^{\star}\bm{\Sigma}^{\star}\bm{V}^{\star\top}$, which together with the relation (\ref{eq:dilation-eigen}) reveals that: (i) $\mathcal{S}(\bm{M}^{\star})$ has rank $2r$ and condition number $\kappa$; (ii) the nonzero eigenvalues of $\mathcal{S}(\bm{M}^{\star})$ and the corresponding eigenvectors are reflected respectively in the  matrices
\begin{equation}
\overline{\bm{\Lambda}}^{\star}\coloneqq\left[\begin{array}{cc}
\bm{\Sigma}^{\star} & \bm{0}\\
\bm{0} & -\bm{\Sigma}^{\star}
\end{array}\right]	
\quad\text{and}\quad
\overline{\bm{U}}^{\star}\coloneqq\frac{1}{\sqrt{2}}\left[\begin{array}{cc}
\bm{U}^{\star} & \bm{U}^{\star}\\
\bm{V}^{\star} & -\bm{V}^{\star}
\end{array}\right]
	\label{eq:defn-spectral-S-M-star}
\end{equation}
Similarly, given that the SVD of $\bm{M}$ is $\bm{M}=\bm{U}\bm{\Sigma}\bm{V}^{\top}+\bm{U}_{\perp}\bm{\Sigma}_{\perp}\bm{V}_{\perp}^{\top}$, we see that the $2r$-leading eigenvalues of $\mathcal{S}(\bm{M})$ and the corresponding eigenvectors are represented respectively by the matrices
\[
\overline{\bm{\Lambda}}\coloneqq\left[\begin{array}{cc}
\bm{\Sigma} & \bm{0}\\
\bm{0} & -\bm{\Sigma}
\end{array}\right]
\quad\text{and}\quad
\overline{\bm{U}}\coloneqq\frac{1}{\sqrt{2}}\left[\begin{array}{cc}
\bm{U} & \bm{U}\\
\bm{V} & -\bm{V}
\end{array}\right]. 
\]
 Further, the incoherence parameter $\overline{\mu}$
of $\mathcal{S}(\bm{M}^{\star})$ (cf.~(\ref{eq:defn-Ustar-incoherence}))
obeys
\begin{align}
\overline{\mu} & \coloneqq\frac{(n_{1}+n_{2})}{2r}\|\overline{\bm{U}}^{\star}\|_{2,\infty}^{2}
	\overset{(\text{i})}{=}  \frac{(n_{1}+n_{2})}{2r}\max\big\{\|\bm{U}^{\star}\|_{2,\infty}^{2},\|\bm{V}^{\star}\|_{2,\infty}^{2}\big\} \notag\\
 & \overset{(\text{ii})}{\leq}\frac{(n_{1}+n_{2})}{\min\{n_{1},n_{2}\}}\frac{\mu}{2}\overset{(\text{iii})}{=}\frac{(n_{1}+n_{2})}{2n_{1}}\mu.
	\label{eq:mu-bar-general}
\end{align}
Here, the relation (i) is based on the definition (\ref{eq:defn-spectral-S-M-star}),
the  inequality (ii) follows from the incoherence of $\bm{M}^{\star}$
(cf.~Definition~\ref{assump:mc-incoherence}), while the last one (iii)
holds under the assumption $n_{1}\leq n_{2}$. 

When it comes to the ``symmetrized'' noise part, it is straightforward to verify
that under Assumption~\ref{assump:noise-general-rect}, 
the matrix
$\mathcal{S}(\bm{E})$ satisfies Assumption~\ref{assumption-noise-general}
with precisely the quantities $\sigma, B$ and $c_{\mathsf{b}}$.

\paragraph{$\ell_{2,\infty}$ and $\ell_{\infty}$ guarantees. }

With the above preparations in place, 
%we are ready to invoke Theorem~\ref{thm:UsgnH-Ustar-MUstar-general} and Corollary~\ref{cor:entrywise-error-general}
%to demonstrate finer performance guarantees for the spectral
%estimates $\bm{U},\bm{V}$.
%
apply Theorem~\ref{thm:UsgnH-Ustar-MUstar-general} (more specifically (\ref{eq:UH-MUstar-Lambda-star-2inf}))
 to demonstrate that
 \begin{align}
\|\overline{\bm{U}}\,\overline{\bm{U}}^{\top}\overline{\bm{U}}^{\star}-\overline{\bm{U}}^{\star}\|_{2,\infty} & \lesssim\frac{\sigma\kappa\sqrt{\overline{\mu}r}+\sigma\sqrt{r\log n}}{\sigma_{r}^{\star}} \notag\\
 & \lesssim \frac{\sigma\kappa\sqrt{n_{2}\mu r/n_{1}}+\sigma\sqrt{r\log n}}{\sigma_{r}^{\star}}, 
	 \label{eq:mc-inf-first-step}
\end{align}
where the last relation follows from \eqref{eq:mu-bar-general}. Further, note that 
\begin{align*}
\overline{\bm{U}}\,\overline{\bm{U}}^{\top}\overline{\bm{U}}^{\star}-\overline{\bm{U}}^{\star} & =\frac{1}{\sqrt{2}}\left[\begin{array}{cc}
\bm{U}\bm{H}_{\bm{U}}-\bm{U}^{\star}, & \bm{U}\bm{H}_{\bm{U}}-\bm{U}^{\star}\\
\bm{V}\bm{H}_{\bm{V}}-\bm{V}^{\star}, & -(\bm{V}\bm{H}_{\bm{V}}-\bm{V}^{\star})
\end{array}\right],
%\overline{\bm{U}}\overline{\bm{U}}^{\top}\overline{\bm{U}}^{\star}-\overline{\bm{U}}^{\star}  =\frac{1}{\sqrt{2}}\left[\begin{array}{cc}
%\bm{U}\bm{U}^{\top}\bm{U}^{\star}-\bm{U}^{\star}, & \bm{U}\bm{U}^{\top}\bm{U}^{\star}-\bm{U}^{\star}\\
%\bm{V}\bm{V}^{\top}\bm{V}^{\star}-\bm{V}^{\star}, & -(\bm{V}\bm{V}^{\top}\bm{V}^{\star}-\bm{V}^{\star})
%\end{array}\right],
\end{align*}
where $\bm{H}_{\bm{U}}\coloneqq \bm{U}^{\top} \bm{U}^{\star}$ and $\bm{H}_{\bm{V}}\coloneqq \bm{V}^{\top} \bm{V}^{\star}$. Combining this with the upper bound (\ref{eq:mc-inf-first-step}) then yields
\[
	\max\big\{\|\bm{U}\bm{H}_{\bm{U}}-\bm{U}^{\star}\|_{2,\infty},\|\bm{V}\bm{H}_{\bm{V}}-\bm{V}^{\star}\|_{2,\infty} \big\}
	\lesssim \frac{\sigma\kappa\sqrt{n_{2}\mu r/n_{1}}+\sigma\sqrt{r\log n}}{\sigma_{r}^{\star}}.
\]
We can then repeat the same analysis  as in the proof of Lemma~\ref{lem:UH-sequence-bounds-E123}
to connect $\|\bm{U}\bm{H}_{\bm{U}}-\bm{U}^{\star}\|_{2,\infty}$
(resp.~$\|\bm{V}\bm{H}_{\bm{V}}-\bm{V}^{\star}\|_{2,\infty}$)
with $\|\bm{U}\mathsf{sgn}(\bm{H}_{\bm{U}})-\bm{U}^{\star}\|_{2,\infty}$
(resp.~$\|\bm{V}\mathsf{sgn}(\bm{H}_{\bm{V}})-\bm{V}^{\star}\|_{2,\infty}$),
and obtain the desired bound in (\ref{eq:main-result-2-infty-general-asymm}); the details are omitted here for the sake of conciseness. 

We now proceed to the second claim (\ref{eq:main-result-infty-general-asymm}).
Towards this, invoke Corollary~\ref{cor:entrywise-error-general} and the inequality \eqref{eq:mu-bar-general} to obtain 
\[
	\big\|\overline{\bm{U}}\,\overline{\bm{\Lambda}}\,\overline{\bm{U}}^{\top}-\mathcal{S}(\bm{M}^{\star}) \big\|_{\infty}\lesssim\sigma\kappa^{2}\overline{\mu}r\sqrt{\frac{\log n}{n}}
	\lesssim \sigma\kappa^{2}\mu r\sqrt{\frac{n\log n}{n_{1}^{2}}}. 
\]
This in conjunction with the following observations 
\begin{align*}
	\overline{\bm{U}}\,\overline{\bm{\Lambda}}\,\overline{\bm{U}}^{\top} & =\mathcal{S}\big(\bm{U}\bm{\Sigma}\bm{V}^{\top} \big) \\
	\big\|\mathcal{S} \big(\bm{U}\bm{\Sigma}\bm{V}^{\top}\big)- \mathcal{S}\big(\bm{M}^{\star}\big) \big\|_{\infty} &=  \big\|\bm{U}\bm{\Sigma}\bm{V}^{\top}-\bm{M}^{\star} \big\|_{\infty}
\end{align*}
immediately establishes the second claim.

%% file: chapters/analysis_distribution.tex
\section{Appendix D: Proof of Theorem~\ref{thm:matrix-distribution-complete}}
\label{sec:proof-thm:matrix-distribution-complete}

As before (see \eqref{eq:assumption-lambda-r-positive}), we assume throughout the proof that $\lambda_r^{\star}>0$ for the purpose of simplifying notation.

\subsection{Proof outline}

We now outline the proof of our distributional guarantees in Theorem~\ref{thm:matrix-distribution-complete}. 
The first step consists of justifying the heuristic first-order approximation in \eqref{eq:approx-ULambdaU-Mstar-discuss}. This is stated in the lemma below, with the proof deferred to Section~\ref{sec:sub-proof-thm:matrix-distribution}. 
\begin{lemma}[First-order approximation]
\label{thm:matrix-distribution}
Suppose that the assumptions of Theorem~\ref{thm:UsgnH-Ustar-MUstar-general} hold. Then with probability at least $1-O(n^{-5})$, one can write
\begin{subequations}
\label{eq:U-M-decomposition-thm}
\begin{align}
\bm{U}\mathsf{sgn}(\bm{H})-\bm{U}^{\star} & =\underset{\eqqcolon\,\bm{Z}}{\underbrace{\bm{E}\bm{U}^{\star}\big(\bm{\Lambda}^{\star}\big)^{-1}}}+\bm{\Psi},
	\label{eq:U-Ustar-decomposition-Z-Psi-thm}\\
\bm{M}-\bm{M}^{\star} & =\underset{\eqqcolon\,\bm{W}}{\underbrace{\bm{E}\bm{U}^{\star}\bm{U}^{\star\top}+\bm{U}^{\star}\bm{U}^{\star\top}\bm{E}}}+\bm{\Phi}
	\label{eq:M-Mstar-diff-decomposition-thm}
\end{align}
\end{subequations}
for some matrices $\bm{\Psi}$ and $\bm{\Phi}$ obeying
\begin{subequations}
\label{eq:Psi-Phi-bound-thm}
\begin{align}
\|\bm{\Psi}\|_{2,\infty} & \lesssim\frac{\sigma^{2}\kappa\sqrt{\mu rn\log^{2}n}}{(\lambda_{r}^{\star})^{2}}+\frac{\kappa\sigma^{2}\sqrt{\mu rn}}{(\lambda_{r}^{\star})^{2}}+\frac{\sigma}{\lambda_{r}^{\star}}\sqrt{\frac{\mu r^{2}\log n}{n}},\\
\big\|\bm{\Phi}\big\|_{\infty} & \lesssim\frac{\sigma^{2}\mu r\kappa^{2}\log n}{\lambda_{r}^{\star}}+\frac{\sigma\kappa\mu\sqrt{r^{3}\log n}}{n}.
\end{align}
\end{subequations}
\end{lemma}
\begin{remark}
	In addition to quantifying the goodness of the approximation \eqref{eq:M-Mstar-diff-decomposition-thm}, Lemma~\ref{thm:matrix-distribution} also delivers a more refined characterization for the first-order approximation $\bm{U}\mathsf{sgn}(\bm{H})-\bm{U}^{\star}\approx \bm{E}\bm{U}^{\star}\big(\bm{\Lambda}^{\star}\big)^{-1}$ in comparison to Theorem~\ref{thm:UsgnH-Ustar-MUstar-general}.  As it turns out, this result \eqref{eq:U-Ustar-decomposition-Z-Psi-thm} also assists in performing statistical inference on the low-rank factors $\bm{U}^{\star}$. The interested reader is referred  to \citet{yan2021inference} for details. 
\end{remark}

In turn, Lemma~\ref{thm:matrix-distribution} motivates one to pin down the distribution of the matrix $\bm{W}$ in \eqref{eq:M-Mstar-diff-decomposition-thm}. This can be accomplished by invoking the Berry-Esseen Theorem (e.g., \citet[Theorem 3.7]{chen2010normal}), which gives rise to the following distributional characterization. The proof of this lemma can be found in Section~\ref{sec:proof-lem:normality-spectral}. 
\begin{lemma}[Gaussian approximation]
	\label{lem:normality-spectral}
	Suppose that the assumptions of Theorem~\ref{thm:UsgnH-Ustar-MUstar-general} hold, and that 
	%
	%assume that
%
\begin{align}
	\label{eq:Uj-Ui-lower-bound-lem-2}
	\frac{\big\|\bm{U}_{j,\cdot}^{\star}\big\|_{2}^{2}+\big\|\bm{U}_{i,\cdot}^{\star}\big\|_{2}^{2}}{\big\|\bm{U}^{\star}\big\|_{\mathrm{F}}^{2}}
	& \gtrsim \frac{B^{2}\kappa^2 \mu^{2}r^2\log ^2 n}{\sigma_{\min}^{2}n^{2}}
	+ \frac{\sigma^{4}\mu^{2}r\kappa^{4}\log^{3}n}{\sigma_{\min}^{2}(\lambda_{r}^{\star})^{2}} .
\end{align}
%
%Condition~\ref{eq:Uj-Ui-lower-bound-lem-2} holds. 
%
	Let $\bm{W}=\bm{E}\bm{U}^{\star}\bm{U}^{\star\top} + \bm{U}^{\star}\bm{U}^{\star\top}\bm{E}$. For any $1\leq i,j\leq n$, one has
	\begin{subequations}
	\begin{align}
		\sup_{z\in\mathbb{R}}\left|\mathbb{P}\left( W_{i,j}\leq z \sqrt{v_{i,j}^{\star}}\right) -\Phi(z)\right| &=o(1) , \label{eq:eqn-normality-spectral}\\
		\big\|\bm{\Phi}\big\|_{\infty} &= o\left( \sqrt{ v_{i,j}^{\star} } \right), 
		\label{eq:Phi-inf-norm-o-vij}
	\end{align}
	\end{subequations}
	where $v_{i,j}^{\star}$ is defined in \eqref{eq:variance-formula-Mij-lem}, and $\Phi(\cdot)$ denotes the CDF of the standard Gaussian distribution.  
\end{lemma}

To finish up, invoke Lemma~\ref{thm:matrix-distribution} and \eqref{eq:Phi-inf-norm-o-vij} to yield
\[
	M_{i,j}-M_{i,j}^{\star}=W_{i,j}+\delta_{\phi}\sqrt{v_{i,j}^{\star}}\qquad\text{with }\delta_{\phi}=o(1).
\]
A little algebra further gives
\begin{align*}
 & \left|\mathbb{P}\left( M_{i,j}-M_{i,j}^{\star}\leq z\sqrt{v_{i,j}^{\star}}\right) -\Phi(z)\right|=\left|\mathbb{P}\left( W_{i,j}\leq\left(z-\delta_{\phi}\right)\sqrt{v_{i,j}^{\star}}\right) -\Phi(z)\right|\\
 & \qquad \leq \left|\mathbb{P}\left( W_{i,j}\leq\left(z-\delta_{\phi}\right)\sqrt{v_{i,j}^{\star}}\right)  -\Phi(z-\delta_{\phi})\right|+\left|\Phi(z-\delta_{\phi})-\Phi(z)\right|\\
 & \qquad\leq o(1)+\delta_{\phi}=o(1),
\end{align*}
where the last line invokes Lemma~\ref{lem:normality-spectral} and the fact that $|\Phi(u)-\Phi(v)|\leq |u-v|$ for any $u,v\in \mathbb{R}$. 
This completes the proof of  Theorem~\ref{thm:matrix-distribution-complete}, as long as Lemmas~\ref{thm:matrix-distribution} and \ref{lem:normality-spectral} can be established. 
The rest of this section is thus devoted to proving Lemmas~\ref{thm:matrix-distribution} and \ref{lem:normality-spectral}.

\subsection{Proof of Lemma~\ref{thm:matrix-distribution}}
\label{sec:sub-proof-thm:matrix-distribution}

Before proceeding to the proof, we make note of several preliminary facts that are all direct consequences of the analysis of Theorem~\ref{thm:UsgnH-Ustar-MUstar-general} and Corollary~\ref{cor:entrywise-error-general}. The proof of these preliminary results can be found in Section~\ref{sec:proof-auxiliary-distribution}.
\begin{lemma}
\label{lem:bound-Delta1-Delta2}
With probability exceeding $1-O(n^{-5})$, one can write
\begin{subequations}\label{eq:decomp_soda}
\begin{align}
\bm{U}\bm{\Lambda}\mathsf{sgn}(\bm{H}) & =\bm{M}\bm{U}^{\star}+\bm{\Delta}_{1} ,
			\label{eq:U-Lambda-sgnH-decompose}\\
\mathsf{sgn}(\bm{H})\bm{\Lambda}^{\star}\mathsf{sgn}(\bm{H})^{\top} & =\bm{\Lambda}+\bm{\Delta}_{2}
			\label{eq:sgnH-Lambdastar-decompose}
\end{align}
\end{subequations}
for some matrices $\bm{\Delta}_{1}$ and $\bm{\Delta}_{2}$ obeying
\begin{subequations}
\begin{align}
  \big\| \bm{\Delta}_1 \big\|_{2,\infty} & \lesssim 
\frac{\sigma^{2}\kappa\sqrt{\mu rn\log^{2}n}}{\lambda_{r}^{\star}}, \label{eq:Delta-1-U-R-Lambda}\\
  \big\|\bm{\Delta}_2\big\|
%\\
% & \qquad\leq\big\|\mathsf{sgn}(\bm{H})\bm{\Lambda}^{\star}\mathsf{sgn}(\bm{H})^{\top}-\bm{H}\bm{\Lambda}^{\star}\bm{H}^{\top}\big\|+\big\|\bm{H}\bm{\Lambda}^{\star}\bm{H}^{\top}-\bm{\Lambda}\big\|\\
 %& \qquad
  &\lesssim\frac{\kappa\sigma^{2}n}{\lambda_{r}^{\star}}+\sigma\sqrt{r\log n}.
   \label{eq:sgn-H-Lambdastar-sgn-H-Lambda}
\end{align}
\end{subequations}
\end{lemma}

We are now ready to embark on the proof of Lemma~\ref{thm:matrix-distribution}. 
In order to analyze the behavior of $\bm{U}\mathsf{sgn}(\bm{H})$, 
we first point out the following decomposition: 
\begin{align*}
\bm{U}\mathsf{sgn}(\bm{H})\bm{\Lambda}^{\star} & =\bm{U}\bm{\Lambda}\mathsf{sgn}(\bm{H})+\bm{U}\big(\mathsf{sgn}(\bm{H})\bm{\Lambda}^{\star}-\bm{\Lambda}\mathsf{sgn}(\bm{H})\big)\\
 & =\bm{M}\bm{U}^{\star}+\bm{\Delta}_{1}+\bm{U}\big(\mathsf{sgn}(\bm{H})\bm{\Lambda}^{\star}\mathsf{sgn}(\bm{H})^{\top}-\bm{\Lambda}\big)\mathsf{sgn}(\bm{H})\\
 & =\bm{M}^{\star}\bm{U}^{\star}+\bm{E}\bm{U}^{\star}+\bm{\Delta}_{1}+\bm{U}\bm{\Delta}_{2} \, \mathsf{sgn}(\bm{H})\\
 & =\bm{U}^{\star}\bm{\Lambda}^{\star}+\bm{E}\bm{U}^{\star}+\bm{\Delta}_{1}+\bm{U}\bm{\Delta}_{2}\, \mathsf{sgn}(\bm{H}),
\end{align*}
where the second and the third identities result from Lemma~\ref{lem:bound-Delta1-Delta2} (cf.~\eqref{eq:decomp_soda}). 
The key point of this decomposition is to establish a connection between $\bm{U}\mathsf{sgn}(\bm{H})\bm{\Lambda}^{\star}$ and $\bm{U}^{\star}\bm{\Lambda}^{\star} + \bm{E}\bm{U}^{\star}$, with the assistance of the matrices $\bm{\Delta}_1$ and $\bm{\Delta}_2$ studied in Lemma~\ref{lem:bound-Delta1-Delta2}. 
A little algebra then yields
\begin{align}
\bm{U}\mathsf{sgn}(\bm{H})-\bm{U}^{\star}=\underset{\eqqcolon\,\bm{Z}}{\underbrace{\bm{E}\bm{U}^{\star}\big(\bm{\Lambda}^{\star}\big)^{-1}}} & +\underset{\eqqcolon\,\bm{\Psi}}{\underbrace{\bm{\Delta}_{1}\big(\bm{\Lambda}^{\star}\big)^{-1}+\bm{U}\bm{\Delta}_{2}\,\mathsf{sgn}(\bm{H})\big(\bm{\Lambda}^{\star}\big)^{-1}}}.
\label{eq:U-Ustar-decomposition-Z-Psi}
\end{align}
In view of Lemma~\ref{lem:bound-Delta1-Delta2}, the residual matrix
$\bm{\Psi}$ obeys
\begin{align}
\|\bm{\Psi}\|_{2,\infty} & \leq\big\|\bm{\Delta}_{1}\big\|_{2,\infty}\big\|\big(\bm{\Lambda}^{\star}\big)^{-1}\big\|+\big\|\bm{U}\big\|_{2,\infty}\big\|\bm{\Delta}_{2}\big\|\,\big\|\mathsf{sgn}(\bm{H})\big\|\,\big\|\big(\bm{\Lambda}^{\star}\big)^{-1}\big\|\nonumber \\
 & \lesssim\frac{1}{\lambda_{r}^{\star}}\big\|\bm{\Delta}_{1}\big\|_{2,\infty}+\frac{1}{\lambda_{r}^{\star}}\sqrt{\frac{\mu r}{n}}\big\|\bm{\Delta}_{2}\big\|\nonumber \\
 & \lesssim\frac{\sigma^{2}\kappa\sqrt{\mu rn\log^{2}n}}{(\lambda_{r}^{\star})^{2}}+\frac{\kappa\sigma^{2}\sqrt{\mu rn}}{(\lambda_{r}^{\star})^{2}}+\frac{\sigma}{\lambda_{r}^{\star}}\sqrt{\frac{\mu r^{2}\log n}{n}}
\label{eq:Psi-two-inf-norm-bound}
\end{align}
as claimed.

The next step lies in analyzing the matrix estimator $\bm{M}=\bm{U}\bm{\Lambda}\bm{U}^{\top}$.
Towards this, we make the observation that
\begin{align}
\bm{M}-\bm{M}^{\star} & =\bm{U}\bm{\Lambda}\bm{U}^{\top}-\bm{M}^{\star}\nonumber \\
 & =\bm{U}\big(\mathsf{sgn}(\bm{H})\bm{\Lambda}^{\star}\mathsf{sgn}(\bm{H})^{\top}\big)\bm{U}^{\top}-\bm{U}\bm{\Delta}_{2}\bm{U}^{\top}-\bm{M}^{\star}\nonumber \\
 & =(\bm{U}^{\star}+\bm{Z}+\bm{\Psi})\bm{\Lambda}^{\star}(\bm{U}^{\star}+\bm{Z}+\bm{\Psi})^{\top}-\bm{U}\bm{\Delta}_{2}\bm{U}^{\top}-\bm{U}^{\star}\bm{\Lambda}^{\star}\bm{U}^{\star\top}\nonumber \\
 & =\bm{Z}\bm{\Lambda}^{\star}\bm{U}^{\star\top}+\bm{U}^{\star}\bm{\Lambda}^{\star}\bm{Z}^{\top}+\bm{\Phi}\nonumber \\
 & =\bm{E}\bm{U}^{\star}\bm{U}^{\star\top}+\bm{U}^{\star}\bm{U}^{\star\top}\bm{E}+\bm{\Phi},\label{eq:M-Mstar-diff-decomposition}
\end{align}
where the second line relies on Lemma~\ref{lem:bound-Delta1-Delta2} (cf.~\eqref{eq:sgnH-Lambdastar-decompose}),
the third identity makes use of (\ref{eq:U-Ustar-decomposition-Z-Psi}),
and the residual matrix $\bm{\Phi}$ is defined as
\begin{equation}
\bm{\Phi}\coloneqq\bm{\Psi}\bm{\Lambda}^{\star}\big(\bm{U}\mathsf{sgn}(\bm{H})\big)^{\top}+\bm{U}^{\star}\bm{\Lambda}^{\star}\bm{\Psi}^{\top}+\bm{Z}\bm{\Lambda}^{\star}\big(\bm{U}\mathsf{sgn}(\bm{H})-\bm{U}^{\star}\big)^{\top}-\bm{U}\bm{\Delta}_{2}\bm{U}^{\top}.\label{eq:defn-Phi-residual}
\end{equation}
In addition, it is seen from (\ref{eq:EU-2-inf-bound-general}) that
\begin{equation}
\|\bm{Z}\|_{2,\infty}\leq\big\|\bm{E}\bm{U}^{\star}\big(\bm{\Lambda}^{\star}\big)^{-1}\big\|_{2,\infty}\leq\frac{1}{\lambda_{r}^{\star}}\big\|\bm{E}\bm{U}^{\star}\big\|_{2,\infty}\lesssim\frac{\sigma\sqrt{r\log n}}{\lambda_{r}^{\star}}.\label{eq:Z-two-inf-norm-bound}
\end{equation}
Consequently, one can deduce that
\begin{align*}
\big\|\bm{\Phi}\big\|_{\infty} & \leq\big\|\bm{\Lambda}^{\star}\big\|\left\{ \|\bm{Z}\|_{2,\infty}^{2}+\|\bm{Z}\|_{2,\infty}\big\|\bm{U}\mathsf{sgn}(\bm{H})-\bm{U}^{\star}\big\|_{2,\infty}\right\} \\
 & \quad+\big\|\bm{\Lambda}^{\star}\big\|\|\bm{\Psi}\|_{2,\infty}\left(\big\|\bm{U}\big\|_{2,\infty}+\|\bm{U}^{\star}\|_{2,\infty}\right)+\|\bm{U}\|_{2,\infty}^{2}\|\bm{\Delta}_{2}\|\\
 & \lesssim\big|\lambda_{1}^{\star}\big|\left\{ \frac{\sigma^{2}r\log n}{(\lambda_{r}^{\star})^{2}}+\frac{\sigma\sqrt{r\log n}}{\lambda_{r}^{\star}}\cdot\frac{\sigma\kappa\sqrt{\mu r\log n}}{\lambda_{r}^{\star}}\right\} \\
 & \quad+\big|\lambda_{1}^{\star}\big|\sqrt{\frac{\mu r}{n}}\left\{ \frac{\sigma^{2}\kappa\sqrt{\mu rn\log^{2}n}}{(\lambda_{r}^{\star})^{2}}+\frac{\kappa\sigma^{2}\sqrt{\mu rn}}{(\lambda_{r}^{\star})^{2}}+\frac{\sigma}{\lambda_{r}^{\star}}\sqrt{\frac{\mu r^{2}\log n}{n}}\right\} \\
 & \quad+\frac{\mu r}{n}\left\{ \frac{\kappa\sigma^{2}n}{\lambda_{r}^{\star}}+\sigma\sqrt{r\log n}\right\} \\
 & \lesssim\frac{\sigma^{2}\mu r\kappa^{2}\log n}{\lambda_{r}^{\star}}+\frac{\sigma\kappa\mu\sqrt{r^{3}\log n}}{n},
\end{align*}
where the second inequality follows from (\ref{eq:Psi-two-inf-norm-bound}),
(\ref{eq:Z-two-inf-norm-bound}), \eqref{eq:incoherence-U-estimate}, Theorem~\ref{thm:UsgnH-Ustar-MUstar-general}, and Lemma~\ref{lem:bound-Delta1-Delta2}.

\subsection{Proof of Lemma~\ref{lem:normality-spectral}}
\label{sec:proof-lem:normality-spectral}

In what follows, we shall only focus on the case with $i\neq j$. The case with $i=j$ can be analyzed in an analogous manner; we omit it for the sake of brevity.
Before proceeding to the proof, we make note of a couple of basic facts about $\bm{P}^{\star}$ that will prove useful.  
The first property asserts that, for any $1\leq j\leq n$, 
\begin{align}
\sum_{l=1}^{n}P_{j,l}^{\star2} & =\sum_{l=1}^{n}P_{l,j}^{\star2}=\big\|\big(\bm{U}^{\star}\bm{U}^{\star\top}\big)_{\cdot,j}\big\|_{2}^{2}=\big\|\bm{U}^{\star}\big(\bm{U}_{j,\cdot}^{\star}\big)^{\top}\big\|_{2}^{2}=\bm{U}_{j,\cdot}^{\star}\bm{U}^{\star\top}\bm{U}^{\star}\big(\bm{U}_{j,\cdot}^{\star}\big)^{\top}\nonumber \\
 & =\big\|\bm{U}_{j,\cdot}^{\star}\big\|_{2}^{2}=P_{j,j}^{\star}.
	\label{eq:P-jl-square-sum}
\end{align}
The second property is concerned with the term $\|\bm{P}^{\star}\|_{\infty}$:  
\begin{equation}
	\|\bm{P}^{\star}\|_{\infty}=\|\bm{U}^{\star}\bm{U}^{\star\top}\|_{\infty}\leq\|\bm{U}^{\star}\|_{2,\infty}^{2}\leq\frac{\mu r}{n},
	\label{eq:P-inf-norm-UB}
\end{equation}
where the last inequality follows from the incoherence assumption.

In view of the definition \eqref{eq:defn-P-UU-T} of $\bm{P}^{\star}$, we can express $\bm{W}=\bm{E}\bm{P}^{\star}+\bm{P}^{\star}\bm{E}$, which reveals that 
\begin{equation}
	W_{i,j}
	%=\sum_{l=1}^{n}E_{i,l}P_{l,j}+\sum_{l=1}^{n}P_{i,l}E_{l,j}
	=\sum_{l:\,l\neq j}E_{i,l}P^{\star}_{l,j}+\sum_{l:\,l\neq i}P^{\star}_{i,l}E_{l,j}+E_{i,j}(P^{\star}_{i,i}+P^{\star}_{j,j}).
\end{equation}
In other words,  $W_{i,j}$ can be viewed as a weighted sum of independent random variables $\{E_{i,l}\mid l\neq j\} \cup \{E_{l,j}\mid l\neq i\} \cup \{E_{i,j}\}$. 
To pin down the distribution of $W_{i,j}$, we resort to a non-asymptotic
version of the celebrated Berry-Esseen Theorem; 
see \citet[Theorem 3.7]{chen2010normal} for a proof using Stein's method.
\begin{theorem}[The Berry-Esseen bound]\label{thm:berry-esseen}
Let $\xi_{1},\ldots,\xi_{n}$ be independent zero-mean random variables satisfying $\sum_{i=1}^{n}\mathsf{Var}(\xi_{i})= v$.
Then the quantity $S= \frac{1}{\sqrt{v}} \sum_{i=1}^{n}\xi_{i}$ satisfies 
\[
	\sup_{z\in\mathbb{R}}\big|\,\mathbb{P}\left(S\leq z\right)-\Phi\left(z\right)\big|\leq10\gamma,\qquad\text{where}\ \gamma
	=\sum_{i=1}^{n} \frac{\mathbb{E}\big[|\xi_{i}|^{3}\big] }{ v^{3/2} }.
\]
%
%where $\Phi(\cdot)$ represents the CDF of a standard Gaussian distribution. 
\end{theorem}

According to the Berry-Esseen bound (cf.~Theorem~\ref{thm:berry-esseen}), 
proving the approximate Gaussianity of $W_{i,j}$ boils down to characterizing the second and the third moments of these random variables under consideration.

Let us start with the variance statistics. 
Given that $\{E_{i,j}\mid i\geq j\}$ are independently generated, 
%the variance  $v_{i,j}^{\star}  \coloneqq\mathsf{Var}(W_{i,j}) $ 
we can straightforwardly see that
\begin{align}
	v_{i,j}^{\star}   %\coloneqq\mathsf{Var}(W_{i,j})
	&=\sum_{l:\,l\neq j}\sigma_{i,l}^{2}P_{l,j}^{\star 2}+\sum_{l:\,l\neq i}P_{i,l}^{\star 2}\sigma_{l,j}^{2}+\sigma_{i,j}^{2}(P_{i,i}^{\star}+P_{j,j}^{\star})^{2} \notag\\
	&= \mathsf{Var}(W_{i,j}).
	\label{eq:variance-formula-Mij}
\end{align}
We now develop a lower bound on this variance term.  Given
that $\bm{P}^{\star} \succeq \bm{0}$, one has $P_{i,i}^{\star},P_{j,j}^{\star}\geq0$,
%and $(P^{\star}_{i,i}+P^{\star}_{j,j})^{2}\geq P_{i,i}^{\star 2}+P_{j,j}^{\star 2}$,
which combined with \eqref{eq:P-jl-square-sum} reveals that
\begin{align}
v_{i,j}^{\star} & 
%\geq\sigma_{\min}^{2}\Bigg\{\sum_{l:\,l\neq i}P_{l,j}^{\star 2}+\sum_{l:\,l\neq j}P_{i,l}^{\star 2}+P_{i,i}^{\star 2}+P_{j,j}^{\star 2}\Bigg\}
\geq \sigma_{\min}^{2}\Bigg\{\sum_{l=1}^{n}P_{l,j}^{\star 2}+\sum_{l=1}^{n}P_{i,l}^{\star 2}\Bigg\} 
%\notag\\
% & =\sigma_{\min}^{2}\Big\{\big\|\big(\bm{U}^{\star}\bm{U}^{\star\top}\big)_{\cdot,j}\big\|_{2}^{2}+\big\|\big(\bm{U}^{\star}\bm{U}^{\star\top}\big)_{i,\cdot}\big\|_{2}^{2}\Big\} \notag\\
% & =\sigma_{\min}^{2}\Big\{\big\|\bm{U}^{\star}\big(\bm{U}_{j,\cdot}^{\star}\big)^{\top}\big\|_{2}^{2}+\big\|\bm{U}_{i,\cdot}^{\star}\bm{U}^{\star\top}\big\|_{2}^{2}\big\} \notag\\
% & 
=\sigma_{\min}^{2}\Big\{\big\|\bm{U}_{j,\cdot}^{\star}\big\|_{2}^{2}+\big\|\bm{U}_{i,\cdot}^{\star}\big\|_{2}^{2}\Big\}.
	\label{eq:vij-lower-bound-12}
\end{align}

Next, we move on to bound the third moments. 
Utilizing the independence of $\{E_{i,j}\mid i\geq j\}$ once again gives
\begin{align}
	\gamma & \coloneqq\frac{\sum_{l:\,l\neq j}\mathbb{E}\big[\big|E_{i,l}\big|^{3}\big]\big|P_{l,j}^{\star} \big|^{3}+\sum_{l:\,l\neq i}\big|P_{i,l}\big|^{3}\mathbb{E}\big[\big|E_{l,j}\big|^{3}\big]+\mathbb{E}\big[\big|E_{i,j}\big|^{3}\big]\big|P_{i,i}^{\star} +P_{j,j}^{\star} \big|^{3}}{\big(v_{i,j}^{\star}\big)^{3/2}} \notag\\
 & \leq\frac{2B\|\bm{P}^{\star} \|_{\infty}}{\big(v_{i,j}^{\star}\big)^{3/2}}\left\{ \sum_{l:\,l\neq j}\mathbb{E}\big[E_{i,l}^{2}\big] P_{l,j}^{\star 2}+\sum_{l:\,l\neq i} P_{i,l}^{\star 2}\mathbb{E}\big[E_{l,j}^{2}\big]+\mathbb{E}\big[E_{i,j}^{2}\big]\big|P_{i,i}^{\star} +P_{j,j}^{\star} \big|^{2}\right\} \notag\\
 & =\frac{2B\|\bm{P}^{\star} \|_{\infty}}{\big(v_{i,j}^{\star}\big)^{3/2}}\cdot v_{i,j}^{\star}=\frac{2B\|\bm{P}^{\star} \|_{\infty}}{\big(v_{i,j}^{\star}\big)^{1/2}}, 
	\label{eq:gamma-UB-123}
\end{align}
where we have used the assumption that $|E_{i,j}|\leq B$, and the last line arises from the expression \eqref{eq:variance-formula-Mij}. 
Substituting \eqref{eq:P-inf-norm-UB} and \eqref{eq:vij-lower-bound-12} into \eqref{eq:gamma-UB-123} and invoking the elementary identity $\|\bm{U}^{\star} \|_{\mathrm{F}}^2=r$, 
we arrive at
\begin{align*}
\gamma & \leq\frac{2B\mu r}{n\sigma_{\min}\sqrt{\big\|\bm{U}_{j,\cdot}^{\star}\big\|_{2}^{2}+\big\|\bm{U}_{i,\cdot}^{\star}\big\|_{2}^{2}}}=\frac{2B\mu\sqrt{r}\,\big\|\bm{U}^{\star}\big\|_{\mathrm{F}}}{n\sigma_{\min}\sqrt{\big\|\bm{U}_{j,\cdot}^{\star}\big\|_{2}^{2}+\big\|\bm{U}_{i,\cdot}^{\star}\big\|_{2}^{2}}}
	= o(1),
\end{align*}
provided that Condition~\eqref{eq:Uj-Ui-lower-bound-lem-2} holds.

With the above calculations in place, invoking Theorem~\ref{thm:berry-esseen} immediately leads to 
\[
\sup_{z\in\mathbb{R}}\left|\mathbb{P}\left( W_{i,j}\leq\sqrt{v_{i,j}^{\star}}z\right) -\Phi(z)\right|\leq10\gamma=o(1). 
\]
as claimed in \eqref{eq:eqn-normality-spectral}.

Finally, we turn to proving the bound \eqref{eq:Phi-inf-norm-o-vij}. 
By virtue of Lemma~\ref{thm:matrix-distribution} and \eqref{eq:vij-lower-bound-12}, we know that 
\begin{align*}
\big\|\bm{\Phi}\big\|_{\infty} & \lesssim\frac{\sigma^{2}\mu r\kappa^{2}\log n}{\lambda_{r}^{\star}}+\frac{\sigma\kappa\mu\sqrt{r^{3}\log n}}{n}\\
 & =o\left(\sigma_{\min}\sqrt{\big\|\bm{U}_{j,\cdot}^{\star}\big\|_{2}^{2}+\big\|\bm{U}_{i,\cdot}^{\star}\big\|_{2}^{2}}\right)\leq o\big(\sqrt{v_{i,j}^{\star}}\big) , 
\end{align*}
with the proviso that
\begin{align}
	& \sqrt{\big\|\bm{U}_{j,\cdot}^{\star}\big\|_{2}^{2}+\big\|\bm{U}_{i,\cdot}^{\star}\big\|_{2}^{2}}  \gtrsim\frac{\sigma^{2}\mu r\kappa^{2}\log^{3/2}n}{\sigma_{\min}\lambda_{r}^{\star}}+\frac{\sigma\kappa\mu\sqrt{r^{3}}\log n}{\sigma_{\min}n} \notag\\
 & \qquad =\left(\frac{\sigma^{2}\mu\sqrt{r}\kappa^{2}\log^{3/2}n}{\sigma_{\min}\lambda_{r}^{\star}}+\frac{\sigma\kappa\mu r\log n}{\sigma_{\min}n}\right)\big\|\bm{U}^{\star}\big\|_{\mathrm{F}}. 
	\label{eq:condition-7878}
\end{align}
Recognizing the trivial bound $\sigma^2 = \max_{i,j}\mathbb{E}[E_{i,j}^2]\leq B^2$, we know that Condition~\eqref{eq:condition-7878} holds as long as
\begin{align}
	\frac{\big\|\bm{U}_{j,\cdot}^{\star}\big\|_{2}^{2}+\big\|\bm{U}_{i,\cdot}^{\star}\big\|_{2}^{2}}{\big\|\bm{U}^{\star}\big\|_{\mathrm{F}}^{2}}
	& \gtrsim \frac{\sigma^{4}\mu^{2}r\kappa^{4}\log^{3}n}{\sigma_{\min}^{2}(\lambda_{r}^{\star})^{2}} + \frac{B^{2}\kappa^2 \mu^{2}r^2\log ^2 n}{\sigma_{\min}^{2}n^{2}}, 
	%+\frac{\sigma^{2}\kappa^{2}\mu^{2}r^{2}\log^{2}n}{\sigma_{\min}^{2}n^{2}}.
	%\label{eq:Uj-Ui-lower-bound-lem-2}
\end{align}
which is precisely Condition~\eqref{eq:Uj-Ui-lower-bound-lem-2}. This concludes the proof of Lemma~\ref{lem:normality-spectral}. 

%as claimed. 

%
%\[
%\frac{B\|\bm{P}\|_{\infty}}{\big(v_{i,j}^{\star}\big)^{1/2}}\lesssim\frac{\sigma_{\max}\sqrt{n/(\mu^{2}\log n)}\cdot\frac{\mu r}{n}}{\sigma_{\min}\Big\{\Big\|\bm{U}_{j,\cdot}^{\star}\Big\|_{2}+\Big\|\bm{U}_{i,\cdot}^{\star}\Big\|_{2}\Big\}}\asymp\frac{1}{\sqrt{\log n}}\frac{\|\bm{U}^{\star}\|_{\mathrm{F}}}{\sqrt{n}\Big\{\big\|\bm{U}_{j,\cdot}^{\star}\big\|_{2}+\big\|\bm{U}_{i,\cdot}^{\star}\big\|_{2}\Big\}}
%\]
%

%And we know that 
%\[
%B\lesssim\sigma_{\max}\sqrt{n/(\mu^{2}\log n)}.
%\]

\subsection{Proof of Lemma~\ref{lem:bound-Delta1-Delta2}}
\label{sec:proof-auxiliary-distribution}

%\paragraph{Proof of Lemma~\ref{lem:bound-Delta1-Delta2}.}
To begin with, let us begin by proving \eqref{eq:Delta-1-U-R-Lambda}. 
From the definition of the quantity $\mathcal{E}_{1}$ (see Lemma~\ref{lem:UH-MUstar-diff-decompose}), we have
\begin{align*}
 & \big\|\bm{M}(\bm{U}\bm{H}-\bm{U}^{\star})\big\|_{2,\infty}=\lambda_{r}^{\star}\mathcal{E}_{1}/2\\
 & \qquad\lesssim\left(\alpha_{0}+\alpha_{1}+\alpha_{2}\right)+\left(\sigma\sqrt{n}+B\log n\right)\big\|\bm{U}\bm{H}-\bm{U}^{\star}\big\|_{2,\infty}\\
 & \qquad\lesssim\left(\alpha_{0}+\alpha_{1}+\alpha_{2}\right)+\left(\sigma\sqrt{n}+B\log n\right)\frac{\sigma\kappa\sqrt{\mu r\log n}}{\lambda_{r}^{\star}}\\
 & \qquad\asymp\frac{\sigma^{2}\kappa\sqrt{\mu rn\log^{2}n}}{\lambda_{r}^{\star}},
\end{align*}
where the first inequality comes from \eqref{eq:E1-UB-general-1} and \eqref{eq:defn-E1-rho1},
the second inequality is a consequence of Theorem~\ref{thm:UsgnH-Ustar-MUstar-general}, and 
the last line relies on our previous bounds on $\alpha_0,\alpha_1,\alpha_2$ (see \eqref{eq:defn-alpha-0-general}, \eqref{eq:alpha1-bound-1234} and \eqref{eq:Mstar-UH-Ustar-2inf}) and holds as long as $B\lesssim\sigma\sqrt{n/(\mu\log n)}$. 
Additionally, from the elementary identity $\bm{M}\bm{U}\bm{H}=\bm{U}\bm{\Lambda}\bm{H}$, we obtain
\begin{align*}
\big\|\bm{U}\bm{\Lambda}\mathsf{sgn}(\bm{H})-\bm{M}\bm{U}\bm{H}\big\|_{2,\infty} & =\big\|\bm{U}\bm{\Lambda}\mathsf{sgn}(\bm{H})-\bm{U}\bm{\Lambda}\bm{H}\big\|_{2,\infty}\\
 & \leq\big\|\bm{U}\big\|_{2,\infty}\big\|\bm{\Lambda}\big\|\,\big\|\mathsf{sgn}(\bm{H})-\bm{H}\big\|\\
 & \lesssim\sqrt{\frac{\mu r}{n}} |\lambda_{1}^{\star} | \cdot\frac{\sigma^{2}n}{(\lambda_{r}^{\star})^{2}}=\frac{\kappa\sigma^{2}\sqrt{\mu rn}}{\lambda_{r}^{\star}},
\end{align*}
where the last line results from Lemma~\ref{lem:H-property-summary-general}, the fact \eqref{eq:incoherence-U-estimate}, and the following inequality
\begin{equation}
	\|\bm{\Lambda}\| \leq \|\bm{\Lambda}^{\star}\| + \|\bm{E}\|\leq |\lambda_1^{\star}| + O(\sigma\sqrt{n}) \leq 2|\lambda_1^{\star}|. 
	\label{eq:Lambda-norm-upper-bound}
\end{equation}
Taking together the above bounds and applying the triangle inequality 
immediately establish~\eqref{eq:Delta-1-U-R-Lambda}.

Next, we turn to the proof of the bound \eqref{eq:sgn-H-Lambdastar-sgn-H-Lambda}. 
Note that it has been shown in (\ref{eq:H-Lambda-star-H-Lambda-dist}) that
\[
\big\|\bm{H}\bm{\Lambda}^{\star}\bm{H}^{\top}-\bm{\Lambda}\big\|\lesssim\frac{\kappa\sigma^{2}n}{\lambda_{r}^{\star}}+\sigma\sqrt{r\log n}.
\]
In addition, the triangle inequality leads to
\begin{align*}
 & \big\|\mathsf{sgn}(\bm{H})\bm{\Lambda}^{\star}\mathsf{sgn}(\bm{H})^{\top}-\bm{H}\bm{\Lambda}^{\star}\bm{H}^{\top}\big\|\\
 & \qquad\leq\big\|\mathsf{sgn}(\bm{H})\bm{\Lambda}^{\star}\big(\mathsf{sgn}(\bm{H})-\bm{H}\big)^{\top}\big\|+\big\|\big(\mathsf{sgn}(\bm{H})-\bm{H}\big)\bm{\Lambda}^{\star}\bm{H}^{\top}\big\|\\
 & \qquad\leq\left(\big\|\mathsf{sgn}(\bm{H})\big\|+\big\|\bm{H}\big\|\right)\,\big\|\bm{\Lambda}^{\star}\big\|\,\big\|\mathsf{sgn}(\bm{H})-\bm{H}\big\|\\
 & \qquad\leq  2 |\lambda_{1}^{\star}| \, \big\|\mathsf{sgn}(\bm{H})-\bm{H}\big\|\\
 & \qquad\lesssim |\lambda_{1}^{\star}| \frac{\sigma^{2}n}{\big(\lambda_{r}^{\star}\big)^{2}}=\frac{\kappa\sigma^{2}n}{\lambda_{r}^{\star}},
\end{align*}
where the third line follows since $\big\|\mathsf{sgn}(\bm{H})\big\|=1$
and $\|\bm{H}\|\leq\|\bm{U}\|\|\bm{U}^{\star}\|=1$, and the last
inequality comes from (\ref{eq:H-sgnH-diff-UB-123}). Combining
the above two results and invoking the triangle inequality lead to
the advertised bound \eqref{eq:sgn-H-Lambdastar-sgn-H-Lambda}.

\section{Appendix E: Proof of Theorem~\ref{thm:CI-general}} 
\label{sec:proof-thm:CI-general}

With the distributional guarantees in Theorem~\ref{thm:matrix-distribution-complete} in place, 
the only remaining task boils down to verifying the statistical accuracy of the variance estimator $\widehat{v}_{i,j}$. 
This can be achieved via the following lemma, whose proof is provided in Section~\ref{sec:proof-lem:variance-estimation-guarantee}.  
\begin{lemma}
	\label{lem:variance-estimation-guarantee}
	Suppose that the assumptions of Theorem~\ref{thm:UsgnH-Ustar-MUstar-general} hold. In addition, assume that $\kappa^{4}\mu^{2}r^{2}\log n\leq n$, $\sigma\sqrt{n}\lesssim|\lambda_{r}^{\star}|/\kappa$ and
	\begin{align}
		\big\|\bm{U}_{j,\cdot}^{\star}\big\|_{2}^{2}+\big\|\bm{U}_{i,\cdot}^{\star}\big\|_{2}^{2}\gtrsim\frac{B\sigma\kappa^{2}\mu^{2}r^{2}\log n}{\sigma_{\min}^2 n^{3/2}}+\frac{\sigma^{3}\kappa\mu^{2}r^{2}\log n}{\sigma_{\min}^2|\lambda_{r}^{\star}|\sqrt{n}}. 
		\label{eq:Uj-Ui-lower-bound-lem}
	\end{align}
	With probability exceeding $1-O(n^{-5})$, one has
	\begin{align}
		\big|\widehat{v}_{i,j}-v_{i,j}^{\star}\big| =o\big( v_{i,j}^{\star} \big). 
	\end{align}
\end{lemma}

Lemma~\ref{lem:variance-estimation-guarantee} essentially enables us to express
\[
	\widehat{v}_{i,j} = (1 + \zeta_v)^2 v_{i,j}^{\star} \qquad \text{with } \zeta_v = o(1). 
\]
As a consequence, we can further demonstrate that
\begin{align*}
 & \left|\mathbb{P}\left( M_{i,j}^{\star}\in\mathsf{CI}_{i,j}^{1-\alpha}\right) -(1-\alpha)\right|=\left|\mathbb{P}\left( \big| \widehat{M}_{i,j}-M_{i,j}^{\star}\big|\leq z_{\alpha/2}\sqrt{\widehat{v}_{i,j}}\right) -(1-\alpha)\right|\\
 & =\left|\mathbb{P}\left( \big| \widehat{M}_{i,j}-M_{i,j}^{\star}\big|\leq(1+\zeta_{v})z_{\alpha/2}\sqrt{v_{i,j}^{\star}}\right) -(1-\alpha)\right|\\
 & \leq\left|\Phi\big((1+\zeta_{v})z_{\alpha/2}\big)-\Phi\big(-(1+\zeta_{v})z_{\alpha/2}\big)-(1-\alpha)\right|+o(1)\\
 & \leq\left|\Phi\big(z_{\alpha/2}\big)-\Phi\big(-z_{\alpha/2}\big)-(1-\alpha)\right|+2\left|\Phi\big((1+\zeta_{v})z_{\alpha/2}\big)-\Phi\big(z_{\alpha/2}\big)\right|+o(1)\\
 & \leq2\zeta_{v}z_{\alpha/2}=o(1),
\end{align*}
where the first inequality follows from Theorem~\ref{thm:matrix-distribution-complete} and $z_{\alpha/2}:= \Phi^{-1}(1-\alpha/2)$, the second inequality applies the triangle inequality, 
and the validity of the last line can be seen from the basic fact $|\Phi(u)-\Phi(v)| \leq |u-v|$.

When $\sigma_{\min}\asymp \sigma$, Condition~\eqref{eq:Uj-Ui-lower-bound-lem} simplifies to
\begin{align}
	\frac{\big\|\bm{U}_{j,\cdot}^{\star}\big\|_{2}^{2}+\big\|\bm{U}_{i,\cdot}^{\star}\big\|_{2}^{2}}{\big\|\bm{U}^{\star}\big\|_{\mathrm{F}}^{2}}
	\gtrsim\frac{B\kappa^{2}\mu^{2}r\log n}{\sigma n^{3/2}}+\frac{\sigma\kappa\mu^{2}r\log n}{|\lambda_{r}^{\star}|\sqrt{n}}.		 
		\label{eq:Uj-Ui-lower-bound-lem-31}
\end{align}
We still need to ensure that Condition~\eqref{eq:Ui-Uj-condition-thm-distribution-1} is satisfied.  
It is seen that 
\begin{align*}
	\frac{B^{2}\kappa^{2}\mu^{2}r^{2}\log^{2}n}{\sigma^{2}n^{2}}+\frac{\sigma^{2}\mu^{2}r\kappa^{4}\log^{3}n}{(\lambda_{r}^{\star})^{2}}&\lesssim\frac{B\kappa^{2}\mu^{2}r^{2}\log^{2}n}{\sigma n^{3/2}}+\frac{\sigma\mu^{2}r\kappa^{3}\log^{3}n}{|\lambda_{r}^{\star}|\sqrt{n}},
\end{align*}
%
%where the first inequality has made use of the fact that $B\geq \sigma$, and the second inequality 
which holds if $\sigma\kappa \sqrt{n} \lesssim |\lambda_r^{\star}|$ and $B\lesssim \sigma\sqrt{n}$. 
As a result, if Condition~\eqref{eq:Uj-Ui-lower-bound-thm} holds, then both \eqref{eq:Uj-Ui-lower-bound-lem-31} and \eqref{eq:Ui-Uj-condition-thm-distribution-1}
are satisfied. 
This finishes the proof, as long as  Lemma~\ref{lem:variance-estimation-guarantee} can be established.

%%
%\begin{align*}
% & \frac{B^{2}\mu^{2}r\log n}{n^{2}\sigma^{2}}+\frac{\sigma^{2}\mu^{2}r\kappa^{4}\log^{3}n}{(\lambda_{r}^{\star})^{2}}+\frac{\kappa^{2}\mu^{2}r^{2}\log^{2}n}{n^{2}}\\
% & \lesssim\frac{B^{2}\kappa^{2}\mu^{2}r^{2}\log^{2}n}{n^{2}\sigma^{2}}+\frac{\sigma^{2}\mu^{2}r\kappa^{4}\log^{3}n}{(\lambda_{r}^{\star})^{2}}\lesssim\frac{B\kappa^{2}\mu^{2}r^{2}\log^{2}n}{\sigma n^{3/2}}+\frac{\sigma\mu^{2}r\kappa^{3}\log^{3}n}{|\lambda_{r}^{\star}|\sqrt{n}},
%\end{align*}
%%

\subsection{Proof for Lemma~\ref{lem:variance-estimation-guarantee}}
\label{sec:proof-lem:variance-estimation-guarantee}

As before, we shall only present the proof for the case with $i\neq j$ for the sake of conciseness. 
In order to justify the goodness of the estimator $\widehat{v}_{i,j}$, 
we find it convenient to first look at the surrogate estimator introduced in \eqref{eq:v-ij-plugin-surrogate-123}, i.e., 
\begin{equation}
	\widetilde{v}_{i,j}=\sum_{l=1}^{n}E_{i,l}^{2}P_{l,j}^{\star 2}+\sum_{l=1}^{n}P_{i,l}^{\star 2}E_{l,j}^{2}+2E_{i,j}^{2}P_{i,i}^{\star}P_{j,j}^{\star}, 
	\label{eq:v-surrogate}
\end{equation}
%
% which can also be viewed as a plug-in estimator if an oracle reveals all information about $\bm{E}$ and $\bm{P}^{\star}$. 
In the sequel, our proof consists of two main steps: 
\begin{itemize}
	\item Show that the surrogate $\widetilde{v}_{i,j}$ is a reliable estimate of the truth ${v}^{\star}_{i,j}$, namely, $\widetilde{v}_{i,j}\approx {v}^{\star}_{i,j}$. 
	\item Show that the estimator in use and the surrogate estimator are sufficiently close, namely, $\widehat{v}_{i,j}\approx \widetilde{v}_{i,j}$.
\end{itemize}

\subsubsection{Step 1: show that $\widetilde{v}_{i,j}\approx {v}^{\star}_{i,j}$} 
Firstly, the fact that the $E_{i,j}$'s are zero-mean random variables
immediately reveals that $\widetilde{v}_{i,j}$ is an unbiased estimate
of $v_{i,j}^{\star}$, that is, 
\[
\mathbb{E}\big[\widetilde{v}_{i,j}\big]=v_{i,j}^{\star}.
\]
Secondly, given that the $E_{i,j}$'s are statistically independent,
we intend to invoke the Bernstein inequality to control the difference $\widetilde{v}_{i,j}-v_{i,j}^{\star}=\widetilde{v}_{i,j}-\mathbb{E}\big[\widetilde{v}_{i,j}\big]$.
To do so, one first calculates that
\begin{align*}
L_{0} & \coloneqq\max\left\{ \max_{l}E_{i,l}^{2}P_{l,j}^{\star 2},\,\max_{l}E_{l,j}^{2}P_{i,l}^{\star 2},\,E_{i,j}^{2}(P_{i,i}^\star +P_{j,j}^\star )^{2}\right\} \leq4B^{2}\|\bm{P}^\star \|_{\infty}^{2}
\end{align*}
and
\begin{align*}
V_{0} & \coloneqq\sum_{l:\,l\neq j}\mathsf{Var}\big(E_{i,l}^{2}\big)P_{l,j}^{\star 4}+\sum_{l:\,l\neq i}\mathsf{Var}\big(E_{l,j}^{2}\big)P_{i,l}^{\star 4}+\mathsf{Var}\big(E_{i,j}^{2}\big)(P_{i,i}^\star +P_{j,j}^\star )^{4}\\
 & \lesssim\sum_{l=1}^{n}\mathbb{E}\big[E_{i,l}^{4}\big]P_{l,j}^{\star 4}+\sum_{l=1}^{n}\mathbb{E}\big[E_{l,j}^{4}\big]P_{i,l}^{\star 4}+\mathbb{E}\big[E_{i,j}^{4}\big]P_{i,i}^{\star 2}P_{j,j}^{\star 2}\\
 & \lesssim B^{2}\|\bm{P}^\star \|_{\infty}^{2}\left\{ \sum_{l=1}^{n}\mathbb{E}\big[E_{i,l}^{2}\big]P_{l,j}^{2}+\sum_{l=1}^{n}\mathbb{E}\big[E_{l,j}^{2}\big]P_{i,l}^{\star 2}+\mathbb{E}\big[E_{i,j}^{2}\big]P_{i,i}^\star P_{j,j}^\star \right\} \\
 & \lesssim \sigma^{2} B^{2} \|\bm{P}^\star \|_{\infty}^{2}\left\{ \sum_{l=1}^{n}P_{l,j}^{\star 2}+\sum_{l=1}^{n}P_{i,l}^{\star 2}\right\} \\
 & \lesssim \sigma^{2} B^{2} \|\bm{P}^\star \|_{\infty}^{3},
\end{align*}
where the last line follows from \eqref{eq:P-jl-square-sum}. 
%since, for any $1\leq j\leq n$, 
%%
%\[
%\sum_{l=1}^{n}P_{l,j}^{2}=\big\|\bm{U}^{\star}\big(\bm{U}_{j,\cdot}^{\star}\big)^{\top}\big\|_{2}^{2}=\big\|\bm{U}_{j,\cdot}^{\star}\big\|_{2}^{2}=P_{j,j}\leq\|\bm{P}\|_{\infty}.
%\]
%
Invoking the Bernstein inequality (cf.~Corollary~\ref{thm:matrix-Bernstein-friendly}) reveals that with probability exceeding $1-O(n^{-5})$, 
\begin{align}
\big|\widetilde{v}_{i,j}-v_{i,j}^{\star}\big| & =\left|\widetilde{v}_{i,j}-\mathbb{E}\left[\widetilde{v}_{i,j}\right]\right|\lesssim\sqrt{V_{0}\log n}+L_{0}\log n \notag\\
 & \lesssim\sigma B\sqrt{\|\bm{P}^\star \|_{\infty}^{3}\log n}+B^{2}\|\bm{P}^\star \|_{\infty}^{2}\log n \notag\\
 & \lesssim\frac{\sigma B\mu^{3/2}r^{3/2}\sqrt{\log n}}{n^{3/2}}+\frac{\mu^{2}r^{2}B^{2}\log n}{n^{2}} ,
	\label{eq:tilde-v-star-v-diff}
\end{align}
where the last inequality results from \eqref{eq:P-inf-norm-UB}.

\subsubsection{Step 2: show that $\widehat{v}_{i,j}\approx \widetilde{v}_{i,j}$} 
In order to accomplish this, we are in need of controlling the difference between $E_{i,j}$ (resp.~$P^\star_{i,j}$) and $\widehat{E}_{i,j}$ (resp.~$\widehat{P}_{i,j}$). 
To this end, apply the entrywise estimation guarantees in Corollary~\ref{cor:entrywise-error-general} to yield
\begin{align}
\big\|\widehat{\bm{E}}-\bm{E}\big\|_{\infty} & =\big\|\big(\bm{M}-\bm{U}\bm{\Lambda}\bm{U}^{\top}\big)-\big(\bm{M}-\bm{M}^{\star}\big)\big\|_{\infty} \notag\\
 & =\big\|\bm{U}\bm{\Lambda}\bm{U}^{\top}-\bm{M}^{\star}\big\|_{\infty}\lesssim\sigma\kappa^{2}\mu r\sqrt{\frac{\log n}{n}},
	\label{eq:entrywise-estimation-E}
\end{align}
and as a result,
\begin{equation}
\big\|\widehat{\bm{E}}\big\|_{\infty}\leq\big\|\bm{E}\big\|_{\infty}+\big\|\widehat{\bm{E}}-\bm{E}\big\|_{\infty}
	\lesssim B+\sigma\kappa^{2}\mu r\sqrt{\frac{\log n}{n}}
	\asymp B,
	\label{eq:entrywise-estimation-E-1}
\end{equation}
provided that $B\gtrsim \sigma\kappa^{2}\mu r\sqrt{\frac{\log n}{n}}$ (which is trivially satisfied if $\kappa^{2}\mu r\sqrt{\frac{\log n}{n}}\leq 1$). Moving on to the error term $\widehat{P}_{i,j}-P_{i,j}^\star $, we observe that
\begin{align*}
 & \big\|\widehat{\bm{P}}-\bm{P}^\star \big\|_{\infty} =\big\|\bm{U}\bm{U}^{\top}-\bm{U}^{\star}\bm{U}^{\star\top}\big\|_{\infty}\\
 & \qquad \leq\big\|\big(\bm{U}\mathsf{sgn}(\bm{H})-\bm{U}^{\star}\big)\bm{U}^{\star\top}\big\|_{\infty}+\big\|\bm{U}\mathsf{sgn}(\bm{H})\big(\bm{U}\mathsf{sgn}(\bm{H})-\bm{U}^{\star}\big)^{\top}\big\|_{\infty}\\
 & \qquad \leq\big\|\bm{U}\mathsf{sgn}(\bm{H})-\bm{U}^{\star}\big\|_{2,\infty}\left\{ \big\|\bm{U}^{\star}\big\|_{2,\infty}+\big\|\bm{U}\mathsf{sgn}(\bm{H})\big\|_{2,\infty}\right\} \\	
 & \qquad \leq\big\|\bm{U}\mathsf{sgn}(\bm{H})-\bm{U}^{\star}\big\|_{2,\infty}\left\{ \big\|\bm{U}^{\star}\big\|_{2,\infty}+\big\|\bm{U}\big\|_{2,\infty}\right\} .
\end{align*}
Taking this  together with the bound \eqref{eq:UsgnH-Ustar-bound-theorem-general} in Theorem~\ref{thm:UsgnH-Ustar-MUstar-general}, 
the incoherence assumption, and the inequality \eqref{eq:incoherence-U-estimate}, we arrive at
\begin{align}
	\big\|\widehat{\bm{P}}-\bm{P}^\star \big\|_{\infty} & \lesssim \frac{(\sigma\kappa\sqrt{\mu r}+\sigma\sqrt{r\log n})}{|\lambda_{r}^{\star}|} \sqrt{\frac{\mu r}{n}}. \label{eq:P-hat-Linf-norm-diff}
\end{align}
This taken together with \eqref{eq:P-inf-norm-UB} indicates that
\begin{align}
\big\|\widehat{\bm{P}}\big\|_{\infty} & \leq\big\|\bm{P}^\star \big\|_{\infty}+\big\|\widehat{\bm{P}}-\bm{P}^\star \big\|_{\infty}\nonumber \\
 & \lesssim \frac{\mu r}{n}+\frac{(\sigma\kappa\sqrt{\mu r}+\sigma\sqrt{r\log n})}{|\lambda_{r}^{\star}|}\sqrt{\frac{\mu r}{n}}\asymp\frac{\mu r}{n},\label{eq:P-hat-Linf-norm}
\end{align}
with the proviso that $\sigma\sqrt{n}\lesssim|\lambda_{r}^{\star}|/\kappa$ and $\sigma\sqrt{n\log n}\lesssim|\lambda_{r}^{\star}|$.

Armed with the preceding bounds, we are now positioned to control $\widehat{v}_{i,j} - \widetilde{v}_{i,j}$. From the definition of $\widetilde{v}_{i,j}$ and $v_{i,j}^{\star}$, we recognize that
\begin{align}
\big| \widehat{v}_{i,j} - \widetilde{v}_{i,j} \big| & \leq\underset{\eqqcolon\,\alpha_{1}}{\underbrace{\left|\sum_{l=1}^{n}\left(\widehat{E}_{i,l}^{2}\widehat{P}_{l,j}^{2}-E_{i,l}^{2}P_{l,j}^{\star 2}\right)\right|}}+\underset{\eqqcolon\,\alpha_{2}}{\underbrace{\left|\sum_{l=1}^{n}\left(\widehat{P}_{i,l}^{2}\widehat{E}_{l,j}^{2}-P_{i,l}^{\star 2}E_{l,j}^{2}\right)\right|}}\nonumber \\
 & \qquad+2\, \underset{\eqqcolon\,\alpha_{3}}{\underbrace{\left|\widehat{E}_{i,j}^{2}\widehat{P}_{i,i}\widehat{P}_{j,j}-E_{i,j}^{2}P_{i,i}^\star P_{j,j}^\star \right|}},
	\label{eq:v-tilde-bound-alpha-123}
\end{align}
leaving us with three terms to cope with. 
Regarding the first term $\alpha_1$ on the right-hand side of \eqref{eq:v-tilde-bound-alpha-123}, it can be easily verified that
\begin{align}
\alpha_{1} & \leq
\left|\sum_{l=1}^{n}\left(\widehat{E}_{i,l}^{2}\widehat{P}_{l,j}^{2}-E_{i,l}^{2}\widehat{P}_{l,j}^{2}\right)\right|+\left|\sum_{l=1}^{n}\left(E_{i,l}^{2}\widehat{P}_{l,j}^{2}-E_{i,l}^{2}P_{l,j}^{\star2}\right)\right| \notag\\
 & \leq \left(\max_{l}\left|\widehat{E}_{i,l}^{2}-E_{i,l}^{2}\right|\right)\sum_{l=1}^{n}\widehat{P}_{l,j}^{2}+\left(\max_{l}\left| \widehat{P}_{l,j}^{2} - P_{l,j}^{\star 2} \right|\right)\sum_{l=1}^{n}E_{i,l}^{2} \notag\\
	& \overset{\mathrm{(i)}}{\leq} \left(\|\bm{E}\|_{\infty}+\|\widehat{\bm{E}}\|_{\infty}\right)\big\|\widehat{\bm{E}}-\bm{E}\big\|_{\infty}\widehat{P}_{j,j} \notag\\
	&\qquad +\left(\|\bm{P}^\star \|_{\infty}+\|\widehat{\bm{P}}\|_{\infty}\right)\big\|\widehat{\bm{P}}-\bm{P}^\star \big\|_{\infty}\|\bm{E}\|^{2} \notag\\
	& \overset{\mathrm{(ii)}}{\lesssim} B\sigma\kappa^{2}\mu r\sqrt{\frac{\log n}{n}}\cdot\frac{\mu r}{n}+\frac{\mu r}{n}\cdot\frac{(\sigma\kappa\sqrt{\mu r}+\sigma\sqrt{r\log n})}{|\lambda_{r}^{\star}|}\sqrt{\frac{\mu r}{n}}\cdot\sigma^{2}n.  \notag
	%\label{eq:alpha1-ii-UB}
\end{align}
Here, (i) makes use of \eqref{eq:P-jl-square-sum} (with $\bm{P}^{\star}$ replaced by $\widehat{\bm{P}}$), 
whereas (ii) holds true due to \eqref{eq:entrywise-estimation-E}, \eqref{eq:entrywise-estimation-E-1}, \eqref{eq:P-inf-norm-UB}, \eqref{eq:P-hat-Linf-norm-diff}, \eqref{eq:P-hat-Linf-norm}, and Lemma~\ref{lemma:general-noise-bound}.  
The second term $\alpha_2$ on the right-hand side of \eqref{eq:v-tilde-bound-alpha-123} can be bounded in the same manner and we omit it here for brevity. 
When it comes to the last term $\alpha_3$ on the right-hand side of \eqref{eq:v-tilde-bound-alpha-123}, one has
\begin{align*}
\alpha_{3} & \leq\left|\widehat{E}_{i,j}^{2}-E_{i,j}^{2}\right|\left|\widehat{P}_{i,i}\widehat{P}_{j,j}\right|+ E_{i,j}^{2} \left|\widehat{P}_{i,i}\widehat{P}_{j,j}-P_{i,i}^\star P_{j,j}^\star \right|\\
 & \lesssim\big\|\widehat{\bm{E}}-\bm{E}\big\|_{\infty}\left\{ \big\|\bm{E}\big\|_{\infty}+\big\|\widehat{\bm{E}}\big\|_{\infty}\right\} \big\|\widehat{\bm{P}}\big\|_{\infty}^{2} \notag\\
	& \qquad +B^{2}\big\|\widehat{\bm{P}}- \bm{P}^{\star} \big\|_{\infty}\left\{ \big\|\bm{P}^\star \big\|_{\infty}+\big\|\widehat{\bm{P}}\big\|_{\infty}\right\} \\
 & \lesssim\sigma\kappa^{2}\mu r\sqrt{\frac{\log n}{n}}\cdot B\cdot\left(\frac{\mu r}{n}\right)^{2}+B^{2}\cdot\frac{(\sigma\kappa\sqrt{\mu r}+\sigma\sqrt{r\log n})}{|\lambda_{r}^{\star}|}\sqrt{\frac{\mu r}{n}}\cdot\frac{\mu r}{n} \\
 & \lesssim B\sigma\kappa^{2}\mu r\sqrt{\frac{\log n}{n}}\cdot\frac{\mu r}{n}+\frac{\mu r}{n}\cdot\frac{(\sigma\kappa\sqrt{\mu r}+\sigma\sqrt{r\log n})}{|\lambda_{r}^{\star}|}\sqrt{\frac{\mu r}{n}}\cdot\sigma^{2}n, 
\end{align*}
where the penultimate inequality is a consequence of \eqref{eq:entrywise-estimation-E}, \eqref{eq:entrywise-estimation-E-1}, \eqref{eq:P-inf-norm-UB}, \eqref{eq:P-hat-Linf-norm-diff} and \eqref{eq:P-hat-Linf-norm},
and the last line holds true as long as $\mu r \leq n$ (cf.~\eqref{eq:mu-constraint-general}) and $B\lesssim \sigma \sqrt{n}$. 
Combining the above inequalities allows one to reach
\begin{align*}
 & \big| \widehat{v}_{i,j} - \widetilde{v}_{i,j} \big|  \leq\alpha_{1}+\alpha_{2}+\alpha_{3}\\
 & \qquad \lesssim  
	 B\sigma\kappa^{2}\mu r\sqrt{\frac{\log n}{n}}\cdot\frac{\mu r}{n}+\frac{\mu r}{n}\cdot\frac{\sigma\kappa\sqrt{\mu r}+\sigma\sqrt{r\log n}}{|\lambda_{r}^{\star}|}\sqrt{\frac{\mu r}{n}}\cdot\sigma^{2}n \\
 & \qquad \asymp \frac{B\sigma\kappa^{2}\mu^{2}r^{2}\sqrt{\log n}}{n^{3/2}}+\frac{\sigma^{3}\kappa\mu^{2}r^{2}+\sigma^{3}\mu^{3/2}r^{2}\sqrt{\log n}}{|\lambda_{r}^{\star}|\sqrt{n}}
\end{align*}
with probability at least $1-O(n^{-5})$. 
%provided that $\mu r\leq n$. 

\subsubsection{Step 3: combining the above bounds} 

Putting together the results in the previous steps, we can readily derive
\begin{align}
\big|\widehat{v}_{i,j}-v_{i,j}^{\star}\big| & \leq\big|\widetilde{v}_{i,j}-v_{i,j}^{\star}\big|+ \big| \widehat{v}_{i,j} - \widetilde{v}_{i,j} \big|  \notag\\
 & \lesssim\left(\frac{\sigma B\mu^{3/2}r^{3/2}\sqrt{\log n}}{n^{3/2}}+\frac{\mu^{2}r^{2}B^{2}\log n}{n^{2}}\right) \notag\\
 & \qquad+\left(\frac{B\sigma\kappa^{2}\mu^{2}r^{2}\sqrt{\log n}}{n^{3/2}}+\frac{\sigma^{3}\kappa\mu^{2}r^{2}+\sigma^{3}\mu^{3/2}r^{2}\sqrt{\log n}}{|\lambda_{r}^{\star}|\sqrt{n}}\right) \notag\\
	& \asymp \frac{B\sigma\kappa^{2}\mu^{2}r^{2}\sqrt{\log n}}{n^{3/2}}+\frac{\sigma^{3}\kappa\mu^{2}r^{2}\sqrt{\log n}}{|\lambda_{r}^{\star}|\sqrt{n}},  \label{eq:vij-vhat-ij-gap-bound}
 %& \lesssim\frac{1}{\sqrt{\log n}}v_{i,j}^{\star},
\end{align}
where the last relation is guaranteed as long as $B\lesssim \sigma \sqrt{n/\log n}$. 
Consequently, if Condition~\eqref{eq:Uj-Ui-lower-bound-lem} holds, 
%$\frac{B\sigma\kappa^{2}\mu^{2}r^{2}\log n}{n^{3/2}}+\frac{\sigma^{3}\kappa\mu^{2}r^{2}\log n}{|\lambda_{r}^{\star}|\sqrt{n}}
%\lesssim\sigma_{\min} ^2 \big\{ \|\bm{U}_{j,\cdot}^{\star}\|_{2}^{2}+\|\bm{U}_{i,\cdot}^{\star}\|_{2}^{2} \big\} $, 
then it follows from \eqref{eq:vij-vhat-ij-gap-bound} and the lower bound \eqref{eq:vij-lower-bound-12} that
\begin{align*}
	\big|\widehat{v}_{i,j}-v_{i,j}^{\star}\big| \lesssim \frac{1}{\sqrt{\log n}} v_{i,j}^{\star}, 
\end{align*}
thus concluding the proof.

%% file: chapters/conclusion.tex
\chapter{Concluding remarks and open problems}
\label{chapter:conclusion}

This monograph offered a coherent statistical treatment for spectral methods,
resulting in appealing theoretical guarantees for a wide spectrum of data science applications ranging from structured signal reconstruction and factor analysis to
clustering and ranking.
The important role of statistical thinking cannot be overstated.
As has been illuminated, the suite of modern statistical techniques not merely empowers classical $\ell_2$ perturbation theory by delivering tight Euclidean error bounds,
but also enables fine-grained $\ell_\infty$ and $\ell_{2,\infty}$ performance guarantees that cannot be derived from classical matrix perturbation theory alone.
We highlighted a unified recipe that underlies our application-driven analyses, which will be readily applicable to tackle many other problems.

The vignettes presented herein only reflect the tip of an iceberg regarding the capability of spectral methods.
There are multiple aspects about spectral methods that remain inadequately explored and are worthy of future investigation. We conclude this monograph by pointing out a few of them.

%as they not only enable the classical $\ell_2$ perturbation analysis, but also provide finer-grained $\ell_\infty$ or $\ell_{2,\infty}$ perturbation analysis, through careful treatments of concentration inequalities and statistical decoupling.

\begin{itemize}

	\item {\em Precise performance characterization.} The analysis herein falls short of pinpointing a precise trade-off curve between the statistical accuracy and sample complexity of spectral methods, and might even be off by some logarithmic factor.  For algorithms that exhibit order-wise equivalent behavior,  comparing their performances requires finer statistical characterization, ideally with sharp pre-constants.

	\item {\em Handling dependency structure.} Thus far, the $\ell_{\infty}$ and $\ell_{2,\infty}$ perturbation theory we have presented  is restricted to the case where the entries of the data samples are independently generated. There is no shortage of applications where the data samples might exhibit across-entry dependency, examples including blind deconvolution \citep{ahmed2013blind} and phase retrieval with coded diffraction patterns \citep{candes2015CDP,gross2017improved}. Handling such scenarios might require ideas beyond the current leave-one-out framework.

	% \item {\em Heavy-tailed noise distributions.}   A large part of this monograph focuses on bounded and/or sub-Gaussian noise distributions.   For noise distributions with even heavier tails (e.g., the ones  with only bounded second or fourth moments), extensions might be possible by leveraging techniques such as truncation, shrinkage,  adaptive Huber estimation, and median of means; see  \citet[Chapter~10]{fan2020statistical} and \citet{fan2020robust} and the references therein.  It would be interesting to see how these techniques can be adapted to tackle the problems listed herein in the presence of heavier-tailed error distributions. 

%    \item {\em Heavier-tail error distributions.}   Most of the work in this monograph is based on the assumption of sub-Gaussian errors.  This can easily be extended to the class of sub-exponential distributions with suitable modifications.  For error distributions with heavier tails such as those  with only bounded second or fourth moments, extensions are also possible by leveraging techniques such as truncation, shrinkage,  adaptive Huber estimation, and median of means.  There is a number of recent developments in these directions.  See Chapter~10 of \citet{fan2020statistical}, \citet{fan2020robust}, and references therein.  It will be interesting to see how these techniques can be adopted to problems such as low-rank matrix denoising and matrix completion.
		
	\item {\em Functional estimation.} In many decision making applications, what ultimately matters might not be full information about an eigenvector of a matrix, but rather, some deterministic functions (e.g., certain linear functionals or polynomials) about the entries of this eigenvector. However, naive ``plug-in'' estimators---namely, estimating the eigenvector first and plugging it into the target functional---might suffer from significant estimation bias, even in the case of a linear functional.  A systematic bias-correction paradigm is therefore needed to enable optimal functional estimation.

	\item {\em Small eigengaps.} All theory presented in this monograph imposes a stringent requirement on the associated eigengap, that is, it  needs to exceed the spectral norm of a noise or perturbation matrix.  While this eigengap criterion might be unavoidable in generic matrix perturbation theory (which takes a worst-case perspective), there is often no statistical lower bound that rules out the possibility of reliable eigenspace estimation when the eigengap drops below the perturbation size.
		It would be of fundamental importance to understand how a small eigengap impacts the efficacy of spectral methods under various statistical models.

  \item {\em Weak and sparse factors.}  As mentioned previously, low-rank matrices often admit factor-model interpretations.  
In many applications, one has to deal with weak factors, on which only a small fraction of the variables have non-negligible loadings.  
This gives rise to sparse patterns on the loading matrix or the eigenvectors of the covariance matrix. To utilize such a sparsity structure,  a simple method is to apply marginal screening techniques \citep{fan2008sure,fan2008high}.  Examples of this kind include supervised PCA \citep{bair2006prediction},    PCA on ``targeted predictors'' \citep{bai2008forecasting}, and sparse PCA \citep{zou2006sparse,JohLu09,ma2013sparse}. It remains to develop a more systematic and unified theory concerning how to efficiently exploit such
 special structures in low-rank factorizations, taking into account both statistical and computational considerations.

	\item {\em Heterogeneous missing patterns.} When it comes to missing data, the theory presented herein adopts a uniform sampling model where every entry is independently observed with the same probability. In practice, however, one might encounter non-uniform sampling mechanisms, where the sampling probabilities are non-identical across different entries. How to develop an effective spectral method to automatically account for heterogeneous observation patterns, ideally without knowing the detailed sampling probabilities {\em a priori}?

	\item {\em Confidence regions and hypothesis testing for individual eigenvectors.} Given the output of a spectral method, one might be asked to produce valid confidence regions for an unknown individual eigenvector of interest, a task that has not been fully resolved by the existing literature.  Another closely related task is hypothesis testing for individual eigenvectors: given two random samples, how to develop viable statistical tests regarding whether these two samples are associated with the same individual eigenvectors or not. 
		An even more challenging task is concerned with performing efficient statistical inference on some deterministic functions of an individual eigenvector, which remains largely unknown.

 % \item {\em Hypothesis testing for individual eigenvectors and eigen-spaces.}  Given two random samples, one naturally asks if their leading vectors or eigenspaces are the same. How to develop viable tests to answer this kind of questions remains open.

		%Ideally, we would wish to construct confidence intervals without prior knowledge about the noise variance. Additionally, it would also be desirable for such inferential procedures to be fully adaptive to heteroskedastic noise.

\end{itemize}

%% file: chapters/ack.tex
\begin{acknowledgements}

The authors thank the Editor-in-Chief Prof.~Michael Jordan for his encouragements, and the publisher Mike Casey for his editorial help.

We are deeply indebted to our wonderful collaborators who have contributed significantly to, and helped shape our perspectives into, the materials presented herein,
including Emmanuel Abbe, Changxiao Cai, Emmanuel Cand\`es, Yanxi Chen, Chen Cheng, Yonina Eldar, Yingying Fan, Haoyu Fu, Andrea Goldsmith, Leonidas Guibas, Qixing Huang, Govinda Kamath, Tracy Ke, Gen Li, Yuanxin Li, Yingbin Liang, Yuan Liao, Junwei Lu, Yue Lu, Jinchi Lv,  Vincent Monardo, H.~Vincent Poor, Changho Suh, Tian Tong, David Tse, Bingyan Wang, Kaizheng Wang, Weichen Wang, Yuting Wei, Yuling Yan, Zhuorang Yang, Huishuai Zhang, Yuchen Zhou, Yiqiao Zhong, and Ziwei Zhu.
We owe our particular gratitude to Yuling Yan, 
who has generously helped with most materials presented in Sections~\ref{sec:distribution-theory}-\ref{sec:CI-matrix-completion}. 
We also thank Bingyan Wang and Chen Dan for their helpful comments about an early version of this monograph, Changxiao Cai for his help in producing Figure~\ref{fig:tensor-completion-illustration}, and Kaizheng Wang for his help in generating Figure~\ref{fig:sbm-inf}.

We gratefully acknowledge the generous financial support of multiple agencies. More specifically, Y.~Chen acknowledges the support by the AFOSR YIP award FA9550-19-1-0030, the ONR grant N00014-19-1-2120, the ARO YIP award W911NF-20-1-0097, the ARO grant W911NF-18-1-0303,
the NSF grants CCF-1907661, IIS-1900140, IIS-2100158 and DMS-2014279, and the Princeton SEAS innovation award;
Y.~Chi has been supported in part by the ONR under the grants N00014-18-1-2142 and N00014-19-1-2404,
by the ARO under the grant W911NF-18-1-0303, and by the NSF under the grants
CAREER ECCS-1818571, CCF-1901199, CCF-1806154, CCF-2007911, CCF-2106778 and ECCS-2126634;
and J.~Fan has been supported in part by the ONR grant N00014-19-1-2120,
the NSF grants DMS-1662139, DMS-1712591, DMS-2053832,  DMS-2052926, and the NIH grant R01-GM072611, and the Princeton SEAS innovation award.
Part of this work was done while Y.~Chen was visiting the Simons Institute for the Theory of Computing.

Last but not least, this work would not come to existence without the continuing support of our families, especially during the difficult time of COVID-19 pandemic when this monograph was completed. Y.~Chen thanks Yuting Wei for bringing love and encouragement everyday during the writing of this monograph.  Y.~Chi is deeply grateful to her parents, husband, and daughter for being the silver lining in the pandemic. J.~Fan enjoys gratefully his wife and daughters' company and thanks them for compassionate support.
%and Esca, a COVID-19 Husky, for .  
C.~Ma thanks Xinyi Liu for her unfailing support, and Pidan the Cat for bringing surprises and joys everyday. This monograph is dedicated to them.

\end{acknowledgements}